\documentclass[11pt]{book}
\usepackage[
backend=bibtex,
style=numeric,
doi=false,
isbn=true,
url=false,
eprint=false,
abbreviate=false,
sorting=none
]{biblatex}
\usepackage[english]{babel}
\usepackage[utf8]{inputenc}
\usepackage[T1]{fontenc}
\usepackage[includefoot, a4paper,left=3cm,right=2.5cm,top=3cm,bottom=3cm, twoside]{geometry}
\usepackage[dvipsnames,table]{xcolor}
\usepackage{stmaryrd}
\usepackage{amssymb}
\usepackage{amsmath, nccmath}
\usepackage{dsfont}
\usepackage{bm}
\usepackage{libertine}
\usepackage[pdftex]{graphicx}
\usepackage{caption}
\usepackage{subcaption}
\usepackage{float}
\usepackage{multirow}
\usepackage{makecell}
\usepackage{siunitx}
\usepackage{apalike}
\usepackage{wrapfig}
\usepackage{tikz}
\usetikzlibrary{backgrounds}
\usetikzlibrary{arrows,shapes}
\usetikzlibrary{tikzmark}
\usetikzlibrary{calc}
\newcommand{\highlight}[2]{\colorbox{#1!17}{$\displaystyle #2$}}

\renewcommand{\highlight}[2]{\colorbox{#1!17}{#2}}

\usepackage{enumitem}
\usepackage{epigraph}
\usepackage{lscape}
\usepackage{mdframed}
\usepackage{titletoc}
\usepackage{tocloft}
\usepackage{booktabs}
\usepackage{arydshln}
\usepackage{fontawesome} 
\usepackage{calc}
\usepackage[newparttoc]{titlesec} 
\usepackage[export]{adjustbox}
\usepackage[bottom]{footmisc}
\usepackage{appendix}
\usepackage{tcolorbox}
\usepackage[onehalfspacing]{setspace}
\usepackage[colorlinks=true,
citecolor=linkColor,linkcolor=linkColor]{hyperref}
\usepackage{fancyhdr}
\usepackage{lipsum}
\usepackage{tcolorbox}
\tcbuselibrary{breakable} 
\usepackage[capitalize,nameinlink]{cleveref}
\crefname{section}{Sec.}{Secs.}
\Crefname{section}{Section}{Sections}
\Crefname{table}{Table}{Tables}
\crefname{table}{Tab.}{Tabs.}
\crefname{property}{Property}{Properties}
\crefname{appproperty}{Property}{Properties}
\crefname{appendix}{App.}{Apps.}
\crefname{chapter}{Chap.}{Chaps.}
\Crefname{chapter}{Chapter}{Chapters}
\crefname{part}{Part}{Parts}

\usepackage{tabto}
\usepackage{pdfpages}

\usepackage{tabularx}
\usepackage{array}

\usepackage{algorithm}
\usepackage{algpseudocode}

\usepackage{bookmark}
\usepackage{rotating}

\usepackage{stmaryrd}
\usepackage[normalem]{ulem}
\useunder{\uline}{\ul}{}

\newcolumntype{P}[1]{>{\centering\arraybackslash}m{#1}}

\newcommand{\midsepremove}{\aboverulesep = 0mm \belowrulesep = 0mm}
\midsepremove
\newcommand{\midsepdefault}{\aboverulesep = 0.605mm \belowrulesep = 0.984mm}
\midsepdefault

\usepackage{ntheorem}
\theoremstyle{break}
\newtheorem{property}{Property}
\newtheorem{appproperty}{Property}
\newtheorem{proposition}{Proposition}
\newtheorem{appproposition}{Proposition}
\theoremstyle{nonumberplain}
\newtheorem{proof}{Proof}

\definecolor{linkColor}{HTML}{000000}
\definecolor{hrefcolor}{HTML}{750000}
\definecolor{footerhighlight}{HTML}{750000}
\definecolor{secnum}{RGB}{13,151,225}
\definecolor{ptcbackground}{RGB}{240,240,240}
\definecolor{ptctitle}{RGB}{0,177,235}
\definecolor{THEMECOLOR}{HTML}{750000}

\newcommand{\rbf}[1]{\textcolor{red}{\textbf{#1}}}
\newcommand{\bbf}[1]{\textcolor{blue}{\textbf{#1}}}

\newcommand{\rtd}{\textbf{\textcolor{BrickRed}{\faThumbsDown}}}
\newcommand{\gtu}{\textbf{\textcolor{OliveGreen}{\faThumbsUp}}}

\newcommand{\bfsc}[1]{\textsc{\textbf{#1}}}

\renewbibmacro{in:}{%
	\ifentrytype{article}
	{}
	{\bibstring{in}%
		\printunit{\intitlepunct}}}

\titleformat{\paragraph}
{\normalfont\normalsize\bfseries}{\theparagraph}{1em}{}

\titleformat{\chapter}[display]
{\normalfont\color{THEMECOLOR}}
{\Huge\color{THEMECOLOR}\bfseries\sffamily\textsc\chaptertitlename\hspace*{2mm}%
	\begin{tikzpicture}[baseline={([yshift=-0.7ex]current bounding box.center)}]
		\node[fill=THEMECOLOR,rectangle,text=white, rounded corners=0.8mm] {\thechapter};
	\end{tikzpicture}
}
{1ex}
{\titlerule[1.5pt]\vspace*{1.ex}\Huge\color{black}\sffamily\textsc}
[]

\titleformat{name=\chapter,numberless}[display]
{\normalfont\color{black}}
{}
{1ex}
{\huge\textsc}
[]

\titleformat{\part}[display]
{\huge\bfseries}
{	
	\begin{tikzpicture}[overlay,remember picture,shift=(current page.north west)]
		\begin{scope}[x={(current page.north east)},y={(current page.south west)}]
			\draw[THEMECOLOR, line width=8pt] ($(1,1)+(-3.7cm,+1cm)$) -- ($(0,1)+(1cm,1cm)$) -- ($(0,0)+(1cm,-1cm)$) -- ($(1,0)+(-1cm,-1cm)$) -- ($(1,1)+(-1cm,+1cm)$) -- ($(1,1)+(-2.5cm,+1cm)$);
			\end{scope}
	\end{tikzpicture}%
	\Huge\color{THEMECOLOR}\sffamily\textsc\partname\nobreakspace\thepart\\
	\vspace{5mm}\\
	\color{THEMECOLOR}\titlerule[3pt]}
{0mm}
{\centering\huge\bfseries\color{THEMECOLOR}\sffamily\textsc}
[\vspace{2ex}{\titlerule[3pt]}]

\titlecontents{part}[0pc]
{\vspace*{8mm}}
{\color{THEMECOLOR}\large\sffamily\bfseries\filcenter \partname~\thecontentslabel: }
{\large\sffamily\bfseries\filcenter \thecontentslabel}
{}
[\color{THEMECOLOR}\titlerule]

\titlespacing{\section}{0mm}{1.5em}{-0.5\parskip}
\titlespacing{\subsection}{0mm}{0.5em}{-\parskip}
\titlespacing{\subsubsection}{0mm}{0.5em}{-\parskip}
\titlespacing{\paragraph}{0mm}{0.5em}{-\parskip}
\titlecontents{psection}[4.5em]{\addvspace{0.0em}}{\color{THEMECOLOR}\contentslabel{2.0em}}{\hspace*{-3em}}{\color{THEMECOLOR}\titlerule*[0.75em]{.}\contentspage}[\vspace{-5pt}]
\titlecontents{psubsection}[6.5em]{}{\color{THEMECOLOR}\contentslabel{2.6em}}{\hspace*{-2.6em}}{\color{THEMECOLOR}\titlerule*[0.75em]{.}\contentspage}[\vspace{-5pt}]

\let\oldmainmatter\mainmatter
 \renewcommand{\mainmatter}{%
  \addtocontents{toc}{\protect\addvspace{15pt}}%
  \oldmainmatter%
 }

\titlecontents{chapter}[0pc]
{\vspace*{3mm}}
{\sffamily\bfseries\hspace{5mm} \rlap{
		\begin{tikzpicture}[baseline={([xshift=-1mm,yshift=-0.7ex]current bounding box.center)}]
			\node[fill=THEMECOLOR,rectangle,text=white,rounded corners=0.5mm,minimum width=\cftchapnumwidth-5pt, minimum height=\cftchapnumwidth]{\thecontentslabel\quad};
		\end{tikzpicture}
	}%
	\qquad}
{\sffamily\bfseries \thecontentslabel}
{\bfseries\hfill\contentspage}

\newcommand\PartialToC{%
\startcontents[chapters]%
\hypersetup{linkcolor=linkColor,linktoc=all}

\begin{mdframed}[backgroundcolor=ptcbackground,hidealllines=true]
\printcontents[chapters]{p}{1}[2]{
	\begin{tcolorbox}[width=\textwidth+3\fboxsep\relax, colframe=THEMECOLOR, colback=THEMECOLOR, arc=3mm, sharp corners=east]
		\color{white}\bfseries\sffamily \large\textsc{Contents}
	\end{tcolorbox}
	\vspace{-10mm}
	
	}
\end{mdframed}%
}

\newcommand\PartialToCCh{%
\startcontents[chapters]%
\hypersetup{linkcolor=hrefcolor,linktoc=all}%

\begin{mdframed}[backgroundcolor=ptcbackground,hidealllines=true]
	\begin{tcolorbox}[width=\textwidth+3\fboxsep\relax, colframe=THEMECOLOR, colback=THEMECOLOR, arc=3mm, sharp corners=east]
			
		\color{white}\bfseries\sffamily \large\textsc{Contents}
	\end{tcolorbox}
	\vspace{-10mm}
	\begingroup
	\let\pagebreak\relax
	\printcontents[chapters]{p}{1}[2]{}
	\endgroup
	\end{mdframed}%

}

\newcommand{\chapabstract}[1]{
	\vspace{-1cm}
    \begin{quote}
        \singlespacing\small
		\large{\textbf{\textcolor{THEMECOLOR}{\textsc{Abstract}}}}\\
        \textcolor{THEMECOLOR}{\rule{13.5cm}{0.5pt}}
        \small{#1}
        \vskip-4mm
        \rule{13.5cm}{0.5pt}
\end{quote}}

\newtcolorbox{rmk}[1][]{
	breakable,
	colback=white,
	left skip=1cm,
	coltitle=black,
	fonttitle=\bfseries,
	bottomrule=0pt,
	toprule=0pt,
	leftrule=4pt,
	rightrule=0pt,
	titlerule=0pt,
	arc=0pt,
	outer arc=0pt,
	colframe=THEMECOLOR   
}	

\newcommand{\HRule}{\rule{\linewidth}{0.7mm}}

\newcommand{\ie}{\emph{i.e.},~}
\newcommand{\eg}{\emph{e.g.},~}
\DeclareMathOperator*{\argmax}{arg\,max}
\DeclareMathOperator*{\argmin}{arg\,min}

\addto\captionsenglish{
  \renewcommand{\contentsname}%
    {\rmfamily\mdseries\textsc{Table of Contents}\\ \color{THEMECOLOR}\rule{\textwidth}{2pt}}%
}

\addto\captionsenglish{
  \renewcommand{\listfigurename}%
    {\rmfamily\mdseries\textsc{List of Figures}\\ \color{THEMECOLOR}\rule{\textwidth}{2pt}}%
}

\addto\captionsenglish{
  \renewcommand{\listtablename}%
    {\rmfamily\mdseries\textsc{List of Tables}\\ \color{THEMECOLOR}\rule{\textwidth}{2pt}}%
}

\usepackage{etoolbox}

\assignpagestyle{\part}{partplain}

\providecommand{\keywords}[1]{\textbf{\textit{Keywords---}} #1}

\providecommand{\keywordsfr}[1]{\textbf{\textit{Mots-Clé---}} #1}

\usepackage[section=section]{glossaries}
\newglossary{symbols}{syo}{syi}{List of Symbols}    
\makenoidxglossaries

\glsnoexpandfields

\newglossaryentry{aaa_Model}{
    type=symbols,
    sort={aaa_Model},
    name={$\mathcal{F}(\cdot, \theta)$},
    description={General purpose model with parameters $\theta$}}
\newglossaryentry{aaa_params}{
    type=symbols,
    sort={aaa_params},
    name={$\theta$},
    description={Parameters of a model, \eg a neural network}}  
\newglossaryentry{aaa_image}{
    type=symbols,
    sort={aaa_image},
    name={$I$},
    description={An image}
}
\newglossaryentry{aaa_class_set}{
    type=symbols,
    sort={aaa_class_set},
    name={$\mathcal{C}$},
    description={A set of classes}
}
\newglossaryentry{aaa_class_set_base}{
    type=symbols,
    sort={aaa_class_set_base},
    name={$\mathcal{C}_{\mathrm{base}}$},
    description={Base class set}
}
\newglossaryentry{aaa_class_set_novel}{
    type=symbols,
    sort={aaa_class_set_novel},
    name={$\mathcal{C}_{\mathrm{novel}}$},
    description={Novel class set}
}
\newglossaryentry{aaa_novel_set}{
    type=symbols,
    sort={aaa_novel_set},
    name={$\mathcal{D}_{\mathrm{novel}}$},
    description={A set containing examples of the novel classes and their annotations, used for adaptation and/or fine-tuning, a.k.a support set}
}
\newglossaryentry{aaa_base_set}{
    type=symbols,
    sort={aaa_base_set},
    name={$\mathcal{D}_{\mathrm{base}}$},
    description={A set containing training examples of the base classes}
}
\newglossaryentry{aaa_bbox}{
    type=symbols,
    sort={aaa_bbox},
    name={$b$},
    description={A ground truth bounding box: $b=[x,y,w,h]$}
}
\newglossaryentry{aaa_bbox_pred}{
    type=symbols,
    sort={aaa_bbox_pred},
    name={$\hat{b}$},
    description={A predicted bounding box:  $\hat{b}=[\hat{x},\hat{y},\hat{w},\hat{h}]$}
}
\newglossaryentry{aaa_class}{
    type=symbols,
    sort={aaa_class},
    name={$c$},
    description={A semantic class, $c \in \mathcal{C}$}
}
\newglossaryentry{aaa_class_scores}{
    type=symbols,
    sort={aaa_class_scores},
    name={$l$},
    description={A vector of classification scores, $l \in [0,1]^{\mathcal{C}}$}
}
\newglossaryentry{aaa_query_features}{
    type=symbols,
    sort={aaa_query_features},
    name={$F_q$},
    description={Features extracted with a backbone from a query image $I_q$}
}
\newglossaryentry{aaa_support_features}{
    type=symbols,
    sort={aaa_support_features},
    name={$F_s^c$},
    description={Features extracted with a backbone from a support image $I_s^c$, it represents examples from the class $c$}
}
\newglossaryentry{aaa_k_shot}{
    type=symbols,
    sort={aaa_k_shot},
    name={$K$},
    description={Number of shots, \ie the number of examples available for a novel class}
}
\newglossaryentry{aaa_n_ways}{
    type=symbols,
    sort={aaa_n_ways},
    name={$N$},
    description={Number of novel classes in the few-shot setting}
}
\newglossaryentry{aaa_backbone}{
    type=symbols,
    sort={aaa_backbone},
    name={$f$},
    description={Backbone of a visual model (for classification or detection). Often implemented as a large CNN or Transformer model}
}
\newglossaryentry{aaa_neck}{
    type=symbols,
    sort={aaa_neck},
    name={$g$},
    description={Neck of a detection model, it generally consists of a Feature Pyramidal Network}
}

\newglossaryentry{aaa_head}{
    type=symbols,
    sort={aaa_head},
    name={$h$},
    description={Detection or classification head, often implemented as a lightweight CNN or MLP}
}

\newglossaryentry{aaa_proposal_number}{
    type=symbols,
    sort={aaa_proposal_number},
    name={$N_p$},
    description={Number of proposals boxes in DiffusionDet}
}

\newglossaryentry{siou_gamma}{
    type=symbols,
    sort={siou_gamma},
    name={$\gamma$},
    description={Scale-Adaptative Intersection over Union strength parameter}
}

\newglossaryentry{siou_kappa}{
    type=symbols,
    sort={siou_kappa},
    name={$\kappa$},
    description={Scale-Adaptative Intersection over Union scaling parameter}
}
\newglossaryentry{aaa_detect_label}{
    type=symbols,
    sort={aaa_detect_label},
    name={$\boldsymbol{\mathrm{y}}$},
    description={A detection label, constituted of a bounding box and a class label (or a classification score vector)}
}
\newglossaryentry{pfrcn_objectness}{
    type=symbols,
    sort={pfrcn_objectness},
    name={$o$},
    description={Objectness score for an object}
}
\newglossaryentry{pfrcn_roi_features}{
    type=symbols,
    sort={pfrcn_roi_features},
    name={$\xi_i$},
    description={Features extracted from Region of Interest $i$}
}
\newglossaryentry{pfrcn_z}{
    type=symbols,
    sort={pfrcn_z},
    name={$z$},
    description={Intermediary features within a network, \eg embedding vector or latent variable}
}
\newglossaryentry{pfrcn_phipsi}{
    type=symbols,
    sort={pfrcn_phipsi},
    name={$\Phi^c \mathrm{\,and\, } \Psi^c$},
    description={Prototype vectors for class $c$ in the RPN and the detection of Prototypical Faster R-CNN, respectively}
}
\newglossaryentry{aaf_gamma_s}{
    type=symbols,
    sort={aaf_gamma_s},
    name={$\gamma_s \mathrm{\,and\, } \gamma_q$},
    description={Global attention operators in AAF framework for query and support images, respectively}
}
\newglossaryentry{aaf_lambda_s}{
    type=symbols,
    sort={aaf_lambda_s},
    name={$\lambda_s \mathrm{\,and\, } \lambda_q$},
    description={Spatial alignment operators in AAF framework for query and support images, respectively}
}
\newglossaryentry{aaf_omega}{
    type=symbols,
    sort={aaf_omega},
    name={$\Omega$},
    description={Fusion operators in AAF framework}
}
\newglossaryentry{aaf_algined}{
    type=symbols,
    sort={aaf_algined},
    name={$A_q \mathrm{\,and\, } A_s^c$},
    description={Aligned features in AAF framework for a query image and a support image of class $c$, respectively}
}
\newglossaryentry{aaf_highlighted}{
    type=symbols,
    sort={aaf_highlighted},
    name={$H_q \mathrm{\,and\, } H_s^c$},
    description={Highlighted features in AAF framework for a query image and a support image of class $c$, respectively}
}
\newglossaryentry{aaf_merged}{
    type=symbols,
    sort={aaf_merged},
    name={$M_q^c$},
    description={Merged query features with support features for class $c$ in AAF framework}
}
\newglossaryentry{diff_alphabeta}{
    type=symbols,
    sort={diff_alphabeta},
    name={$\alpha_t \mathrm{\,and\, } \beta_t$},
    description={Diffusion noise variance parameters, $\alpha_t = 1- \beta_t$}
}
\newglossaryentry{diff_varsigma}{
    type=symbols,
    sort={diff_varsigma},
    name={$\varsigma$},
    description={DiffusionDet noise clamping parameter for box generation}
}
\newglossaryentry{siou_criterion}{
    type=symbols,
    sort={siou_criterion},
    name={$\mathfrak{C}$},
    description={A box similarity criterion, \eg IoU, GIoU, or SIoU}
}

\newglossaryentry{fsod}
{
    name=FSOD,
    description={Few-Shot Object Detection}
}
\newglossaryentry{od}
{
    name=OD,
    description={Object Detection}
}
\newglossaryentry{fsl}
{
    name=FSL,
    description={Few-Shot Learning}
}
\newglossaryentry{fsc}
{
    name=FSC,
    description={Few-Shot Classification}
}
\newglossaryentry{rsi}
{
    name=RSI,
    description={Remote Sensing Image}
}
\newglossaryentry{geoint}
{
    name=GEOINT,
    description={Geospatial Intelligence}
}
\newglossaryentry{sgd}
{
    name=SGD,
    description={Stochastic Gradient Descent}
}

\newglossaryentry{pdf}
{
    name=PDF,
    description={Probability Density Function}
}
\newglossaryentry{elbo}
{
    name=ELBO,
    description={Evidence Lower BOund}
}
\newglossaryentry{ode}
{
    name=ODE,
    description={Ordinary Differential Equation}
}
\newglossaryentry{llm}
{
    name=LLM,
    description={Large Language Model}
}
\newglossaryentry{mlp}
{
    name=MLP,
    description={Multi Layer Perceptron}
}
\newglossaryentry{gfsod}
{
    name=G-FSOD,
    description={Generalized Few-Shot Object Detection}
}
\newglossaryentry{fpn}
{
    name=FPN,
    description={Feature Pyramidal Network, introduced in \cite{lin2017feature}}
}
\newglossaryentry{cnn}
{
    name=CNN,
    description={Convolutional Neural Network}
}
\newglossaryentry{dl}
{
    name=DL,
    description={Deep Learning}
}
\newglossaryentry{ssl}
{
    name=SSL,
    description={Self-Supervised Learning}
}
\newglossaryentry{knn}
{
    name=$k$-NN,
    description={$k$-Nearest Neighbors}
}
\newglossaryentry{nms}
{
    name=NMS,
    description={Non-Maximal Suppression}
}
\newglossaryentry{gsd}
{
    name=GSD,
    description={Ground Sampling Distance}
}
\newglossaryentry{iou}
{
    name=IoU,
    description={Intersection over Union}
}
\newglossaryentry{siou}
{
    name=SIoU,
    description={Scale-Adaptative Intersection over Union}
}
\newglossaryentry{api}
{
    name=API,
    description={Application Programming Interface}
}
\newglossaryentry{onnx}
{
    name=ONNX,
    description={Open Neural Network Exchange}
}
\newglossaryentry{pca}
{
    name=PCA,
    description={Principal Component Analysis}
}
\newglossaryentry{cmos}
{
    name=CMOS,
    description={Complementary Metal-Oxide-Semiconductor}
}
\newglossaryentry{lidar}
{
    name=LIDAR,
    description={Light Detection and Ranging}
}
\newglossaryentry{dm}
{
    name=DM,
    description={Diffusion Model}
}
\newglossaryentry{ldm}
{
    name=LDM,
    description={Latent Diffusion Model}
}
\newglossaryentry{kd}
{
    name=KD,
    description={Knowledge Distillation}
}
\newglossaryentry{yolo}
{
    name=YOLO,
    description={You Only Look Once}
}
\newglossaryentry{pfrcnn}
{
    name=PFRCNN,
    description={Prototypical Faster R-CNN }
}
\newglossaryentry{xqsa}
{
    name=XQSA,
    description={Cross-Scale Query-Support Alignment}
}
\newglossaryentry{aaf}
{
    name=AAF,
    description={Alignment Attention Fusion framework}
}
\newglossaryentry{lr}
{
    name=LR,
    description={Learning Rate}
}
\newglossaryentry{cd}
{
    name=CD,
    description={Cross-Domain}
}
\newglossaryentry{fpga}
{
    name=FPGA,
    description={Field-Programmable Gate Array}
}
\newglossaryentry{gpu}
{
    name=GPU,
    description={Graphical Processing Unit}
}
\newglossaryentry{rpn}
{
    name=RPN,
    description={Region Proposal Network}
}
\newglossaryentry{roialign}
{
    name=RoI Align,
    description={Region of Interest Alignment}
}
\newglossaryentry{smb}
{
    name=SMB,
    description={Small and Midsize Business}
}
\newglossaryentry{uav}
{
    name=UAV,
    description={Unmanned Aerial Vehicle}
}

\glsaddall
\glsfindwidesttoplevelname

\begin{document}
\fancypagestyle{empty}{
	\fancyhf{}
	\renewcommand{\headrulewidth}{0 pt}
	\renewcommand{\footrulewidth}{0 pt}
}
\pagestyle{empty}
\newpage
\includepdf[noautoscale, width=\paperwidth, pages=-]{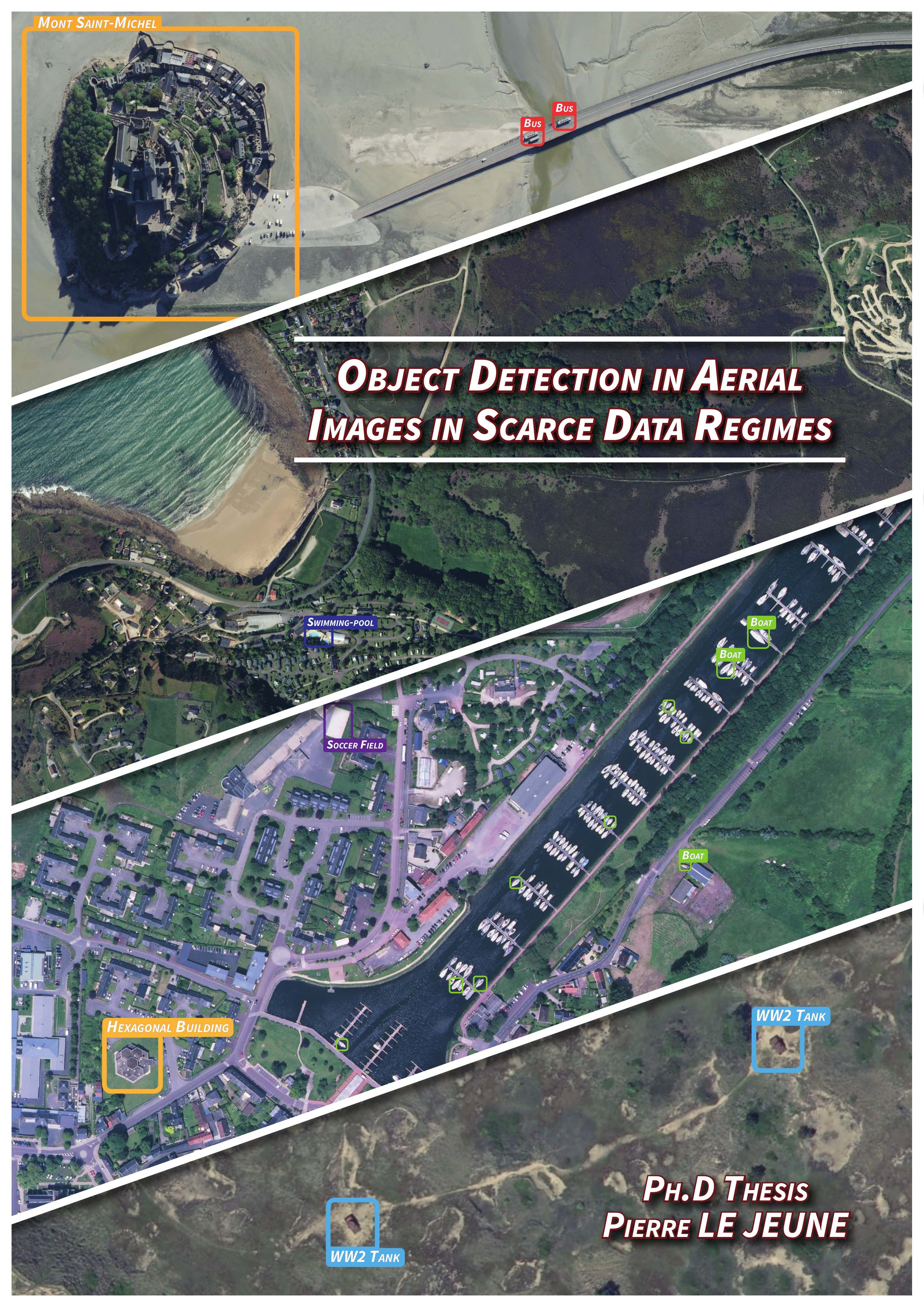}

\begin{titlepage}
    \begin{center}
        \begin{tabular}{@{\hskip -1cm}c@{\hskip 1cm}c@{\hskip 1cm}c@{\hskip 1cm}c}
            \includegraphics[height=1.2cm]{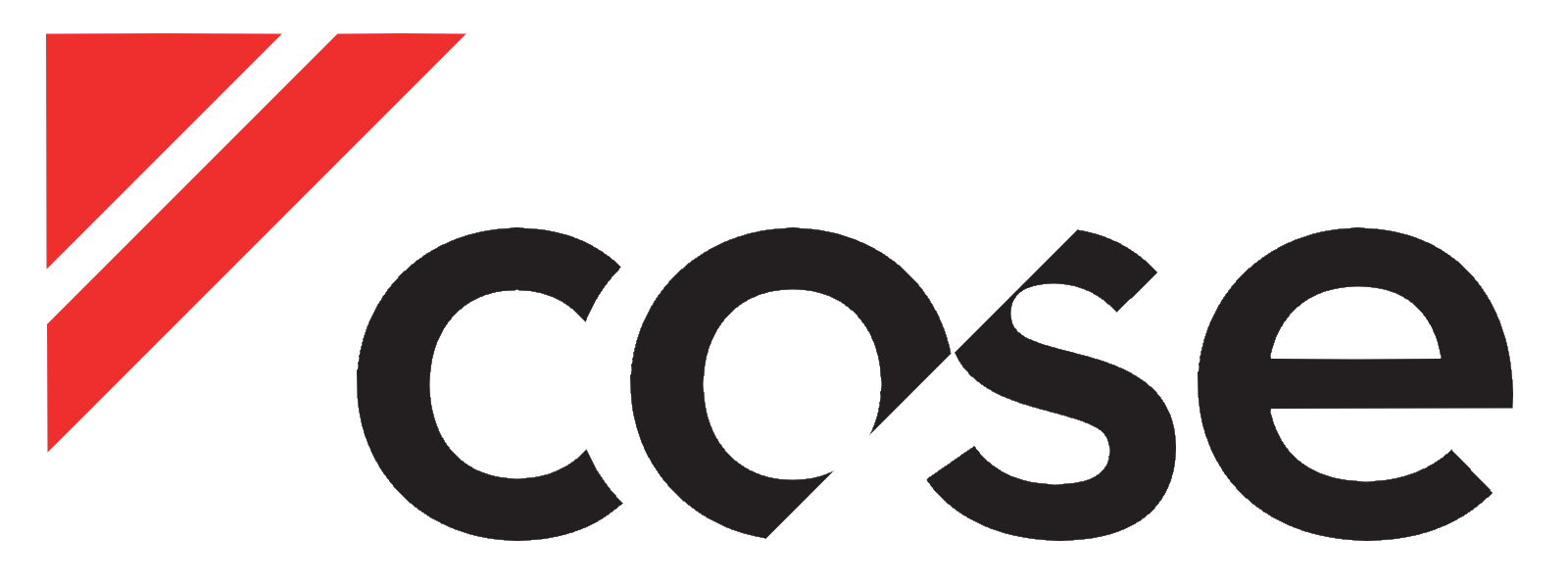} &
            \includegraphics[height=1.2cm]{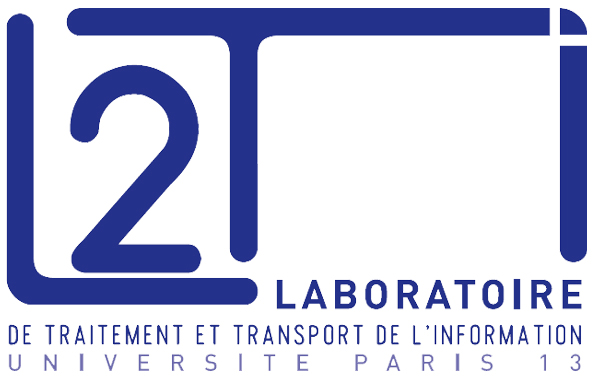} &
            \includegraphics[height=1.2cm]{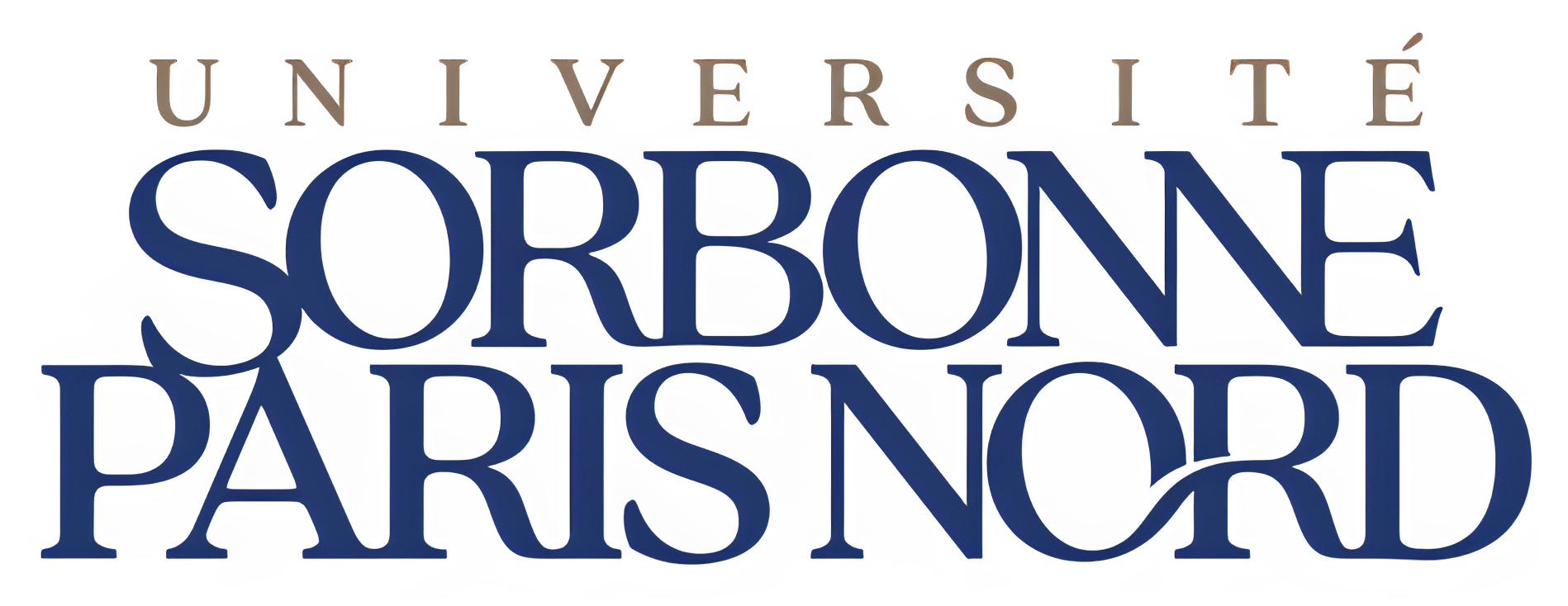} &
            \includegraphics[height=1.2cm]{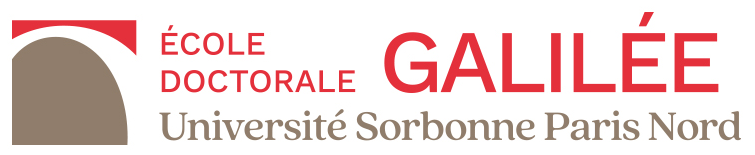}\\
        \end{tabular}
    \end{center}

    \begin{center}

        \vfill
        \vspace{6mm}

        \HRule \\[4mm]
        { \huge \bfseries \textsc{Object Detection in Aerial Images in Scarce Data Regimes} }\\[4mm]
        {\LARGE \textsc{Détection d'objets dans des images aériennes en cas de faible supervision}}\\
        \HRule \\
    \end{center}
    
    \vfill
    
    \begin{center}
        {\Large \textsc{Thèse de Doctorat}}\\[.5cm]
        Presentée et soutenue le 03-10-2023 par :\\
        \textsc{\Large Pierre LE JEUNE}\\[0.5cm] 

        en vue de l'obtention du grade de\\
        \textsc{\Large Docteur en Informatique}\\[0.8cm]

        \large{
            École doctorale Galilée (ED 146), Université Sorbonne Paris Nord\\
            Laboratoire de Traitement et Transport de l'Information (L2TI, UR 3043)\\
            COSE\\[0.5cm]
        }

        \large{\textbf{Directrice : } \textsc{Anissa MOKRAOUI}}
        
    \end{center}
    
    \vspace{1cm}

    \begin{center}
        \textbf{Membres du Jury}\\[0.2cm]
        
        \begin{tabular}{llr}
            \bfsc{Céline HUDELOT}  &  Professeur --- \textit{CentraleSupélec} & \bfsc{Rapportrice}\\
            \bfsc{Paul HONEINE}  &  Professeur --- \textit{Université de Rouen Normandie} & \bfsc{Rapporteur} \\
            \bfsc{Emmanuel DELLANDREA}  &  MCF HDR --- \textit{\'Ecole Centrale de Lyon} & \bfsc{Examinateur}\\
            \bfsc{Ismail BEN AYED}  &  Professeur --- \textit{ETS Montréal} & \bfsc{Examinateur}\\
            \bfsc{Fangchen FENG}  &  MCF --- \textit{Université Sorbonne Paris Nord} & \bfsc{Examinateur}\\
            \bfsc{Hervé GUIOT}  &  Directeur\,  --- \textit{COSE} & \bfsc{Invité} \\               
            \bfsc{Anissa MOKRAOUI}  &  Professeur --- \textit{Université Sorbonne Paris Nord} & \bfsc{Directrice} \\
        \end{tabular}
        \vspace{10mm}
        
    \end{center}	
    
    \newpage
\end{titlepage}

\pagestyle{fancy}

\renewcommand{\sectionmark}[1]{\markright{\thesection~- ~#1}}
\renewcommand{\chaptermark}[1]{\markboth{\chaptername~\thechapter~-~ #1}{}}

\fancypagestyle{main}{
\fancyhf{} 
\fancyhead[LE]{\scshape\leftmark}
\fancyhead[RO]{\scshape\rightmark}
\fancyfoot[RO]{\vspace{8mm}\vfootline\hskip\linepagesep\hspace{1mm}\thepage}
\fancyfoot[LE]{\vspace{8mm}\thepage\hskip\linepagesep\hspace{1mm}\vfootline}
\newskip\linepagesep \linepagesep 5pt\relax
  \def\vfootline{%
	\begingroup\color{footerhighlight}\rule[-990pt]{2pt}{1000pt}\endgroup}
\renewcommand{\headrulewidth}{0 pt}
\renewcommand{\footrulewidth}{0 pt}
}

\fancypagestyle{plain}{%
  \fancyhf{}%
  \fancyfoot[RO]{\vspace{8mm}\vfootline\hspace{1mm}\hskip\linepagesep\thepage}
	\fancyfoot[LE]{\vspace{8mm}\thepage\hspace{1mm}\hskip\linepagesep\vfootline}
	\newskip\linepagesep \linepagesep 5pt\relax
	\def\vfootline{%
		\begingroup\color{footerhighlight}\rule[-990pt]{2pt}{1000pt}\endgroup}
	\renewcommand{\headrulewidth}{0 pt}
	\renewcommand{\footrulewidth}{0 pt}
}


\fancypagestyle{partplain}{%
  \fancyhf{}%
  \fancyfoot[RO]{\vspace{8mm}\vfootline\hskip\linepagesep\color{THEMECOLOR}\thepage}
	\fancyfoot[LE]{\vspace{8mm}\color{THEMECOLOR}\thepage\hskip\linepagesep\vfootline}
	\newskip\linepagesep \linepagesep 5pt\relax
	\def\vfootline{%
		\begingroup\color{THEMECOLOR}\rule[-990pt]{2pt}{1000pt}\endgroup}
	\renewcommand{\headrulewidth}{0 pt}
	\renewcommand{\footrulewidth}{0 pt}
}

\pagestyle{plain}
\frontmatter
{
	\setlength{\parskip}{.7em}
	
	\titlespacing*{\section}{0pt}{.9em}{.8em}
	\renewcommand{\baselinestretch}{1.1}
	\sloppy

	\phantomsection
\clearpage

\section*{\huge \rmfamily\mdseries\textsc{Abstract}}
\addcontentsline{toc}{chapter}{Abstract} 
\vspace{-5mm}
Most contributions on Few-Shot Object Detection (FSOD) evaluate their methods on
natural images only, yet the transferability of the announced performance is not
guaranteed for applications on other kinds of images. We demonstrate this with
an in-depth analysis of existing FSOD methods on aerial images and observed a
large performance gap compared to natural images. Small objects, more numerous
in aerial images, are the cause for the apparent performance gap between natural
and aerial images. As a consequence, we improve FSOD performance on small
objects with a carefully designed attention mechanism. In addition, we also
propose a scale-adaptive box similarity criterion, that improves the training
and evaluation of FSOD methods, particularly for small objects. We also
contribute to generic FSOD with two distinct approaches based on metric learning
and fine-tuning. Impressive results are achieved with the fine-tuning method,
which encourages tackling more complex scenarios such as Cross-Domain FSOD. We
conduct preliminary experiments in this direction and obtain promising results.
Finally, we address the deployment of the detection models inside COSE's
systems. Detection must be done in real-time in extremely large images (more
than 100 megapixels), with limited computation power. Leveraging existing
optimization tools such as TensorRT, we successfully tackle this engineering
challenge.

\keywords{Object Detection, Few-Shot Learning, Few-Shot Object Detection,
Cross-Domain Adaptation, Deep Learning, Computer Vision, Intersection over
Union, Attention Mechanism, Diffusion, Query-Support Alignment}

\section*{\huge \rmfamily\mdseries\textsc{Résumé}}
\addcontentsline{toc}{chapter}{Résumé} 
\vspace{-5mm}
La plupart des contributions en Détection d'Objets \textit{Few-Shot} (FSOD)
évaluent leurs méthodes uniquement sur des images naturelles, ne garantissant
pas la transférabilité de leur performance à d'autres types d'images. Nous
démontrons ceci avec une analyse des méthodes FSOD existantes sur des images
aériennes et observons un large écart comparé aux images naturelles. Les petits
objets, plus nombreux dans les images aériennes, sont responsables de cet écart.
Ainsi, nous proposons d'améliorer la détection des petits objets avec un
mécanisme d'attention dédié. En plus, nous proposons un nouveau critère de
similarité pour boîtes englobantes, adaptatif à la taille. Il améliore
l'entraînement et l'évaluation des modèles FSOD, en particulier pour les petits
objets. Nous contribuons aussi au FSOD classique avec deux approches distinctes
basées sur le \textit{metric learning} et le \textit{fine-tuning}. Des résultats
impressionnants sont obtenus avec cette dernière méthode, ce qui encourage son
application à des scénarios plus complexes comme la détection \textit{Few-Shot}
\textit{Cross-Domain}. Finalement, nous abordons le déploiement de modèles de
détection au sein des systèmes de COSE qui doivent détecter les objets en temps
réel sur de très grandes images (plus de 100 mégapixels), avec des ressources de
calcul limitées.


\keywordsfr{Détection d'objet, Apprentissage profond, Apprentissage frugal,
Adaptation au domaine, Mécanisme d'attention, Diffusion, \textit{Intersection
over Union}, Alignement query-support}


	\cleardoublepage
	
}

{
	\hypersetup{linkcolor=black}
	
	{
		\phantomsection
		\addcontentsline{toc}{chapter}{Table of Contents}
        \tableofcontents
    }
	
	{
		\cleardoublepage
		\phantomsection
		\addcontentsline{toc}{chapter}{List of Figures}
		\listoffigures
	}

	{
		\cleardoublepage
		\phantomsection
		\addcontentsline{toc}{chapter}{List of Tables}
		\listoftables
		
	}

}

\cleardoublepage
\phantomsection
\pagebreak
\chapter*{Glossary}
\addcontentsline{toc}{chapter}{Glossary}

\printnoidxglossary[type=symbols, title={List of Symbols},nonumberlist,style=alttree]
\addcontentsline{toc}{section}{List of Symbols}

\printnoidxglossary[title={List of Acronyms},nonumberlist,style=alttree,nogroupskip=true]
\addcontentsline{toc}{section}{List of Acronyms}

\setlength{\parskip}{.7em}

\titlespacing*{\section}{0pt}{.9em}{.8em}
\renewcommand{\baselinestretch}{1.1}
\sloppy



\mainmatter

\pagestyle{main}


\markboth{\textsc{Introduction}}{\textsc{Introduction}}
\chapter*{Introduction}
\addcontentsline{toc}{chapter}{Introduction} 
\epigraph{\itshape  If a machine is expected to be infallible, it cannot also be intelligent.}{-- Alan Turing}

As an introduction to this thesis manuscript, we present the industrial context
and the motivation behind this project. First, we introduce the company COSE and
the \textit{Laboratoire de Traitement et Transport de l'Information} (L2TI) that
collaborated on this CIFRE PhD project. Then, we briefly describe what object
detection is and how the industrial constraints that weigh upon COSE influenced our
study toward low-data regimes and few-shot learning. Next, we carry out an
overview of the structure of the manuscript, with an individual summary
describing each chapter. Finally, we gather the contributions that came out of
this project. This includes research articles, accepted or submitted to peer
review conferences and journals, as well as open-source code contributions. 

\section{Industrial Context, Motivation and Objectives}
\vspace{-1em}
This PhD thesis originates from a collaboration between the L2TI laboratory from
\textit{Université Sorbonne Paris Nord} (USPN) and the company COSE. The L2TI
was founded in 1998 and is a member of the CNRS Research Federation MathSTIC (FR
3734) which includes the \textit{Laboratoire Analyse, Géométrie et Applications}
(LAGA), UMR 7539 and the \textit{Laboratoire d'Informatique de Paris Nord}
(LIPN), UMR 7030. Two main research teams coexist in the L2TI. The Multimedia
team focuses on visual information analysis and processing, while the Network
team targets information transport and network questions. This thesis falls
within the scope of the Multimedia team. 

COSE\footnote{\href{https://www.cose.fr/}{https://www.cose.fr/}} is a
highly innovative SMB with around 20 employees. It is a first-tier government
provider in the aeronautic and defense sector. COSE was born from an INRIA
start-up in the 1990s and has integrated research excellence at the heart of its
industrial process. While being relatively small, COSE has multidisciplinary
teams with expertise in various fields such as mechanic, electronic, navigation,
automation and embedded software. Its size gives COSE remarkable agility and
cost-effectiveness in comparison to its main competitors. This competitive
advantage has allowed COSE to build strong partnerships with major actors in the
aeronautic and defense areas. 

COSE develops, produces and supports aerial observation camera systems and
onboard equipment. These products are mainly designed for military use and must
therefore conform to strict quality criteria. The relationship with military
forces is handled by the Directorate General of Armaments (DGA), which is one of
the main clients of COSE. Among others, COSE currently relies on three products
that are in use by French military forces around the world (see \cref{fig:cose_product_intro}): 

\begin{itemize}[nolistsep]
    \item[-] \textbf{GlobalScanner}: a high-resolution imaging embedded system
    that provides real-time and georeferenced images. It consists of a
    high-resolution, stabilized linear sensor that can be integrated into various
    carriers such as helicopters, aircraft or UAVs. It comes with powerful
    software to operate the camera and manage image streams.  
    \item[-] \textbf{Strike}: a stabilization arm for helicopters to improve
    high-precision rifle accuracy. It improves shot accuracy and drastically
    reduces collateral damage.
    \item[-] \textbf{POD Xplorer}: a multifunctional pod for various carriers. Its
    purpose is to embed various types of payloads such as optical sensors,
    LIDAR, scientific equipment or inertial sensors.
\end{itemize}

\begin{figure}
    \centering
    \begin{subfigure}[c]{0.31\textwidth}
        \includegraphics[width=\textwidth]{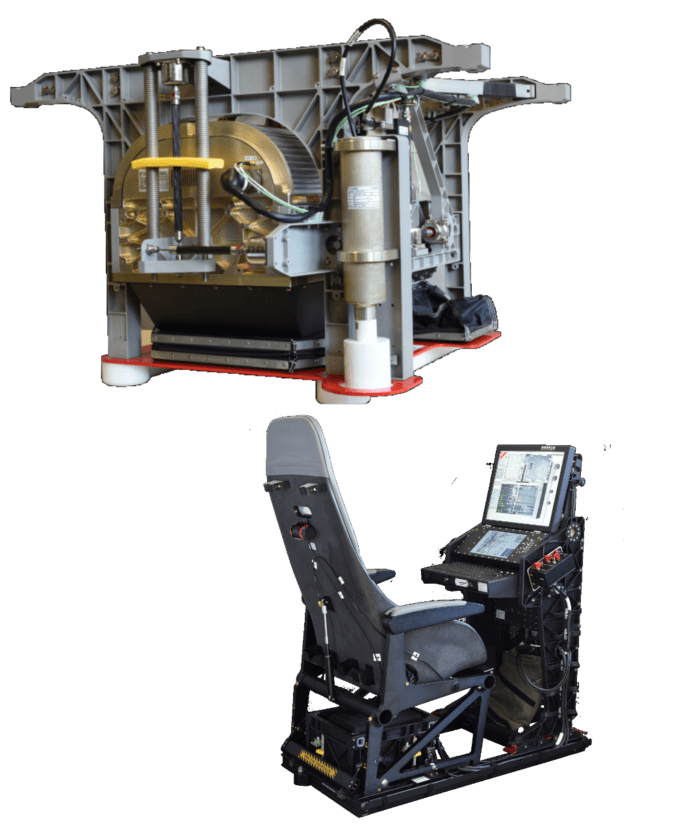}
        \caption{GlobalScanner Camera and Observation Unit.}
    \end{subfigure}
    \hfill
    \begin{subfigure}[c]{0.31\textwidth}
        \vspace{4.5mm}
        \includegraphics[width=\textwidth]{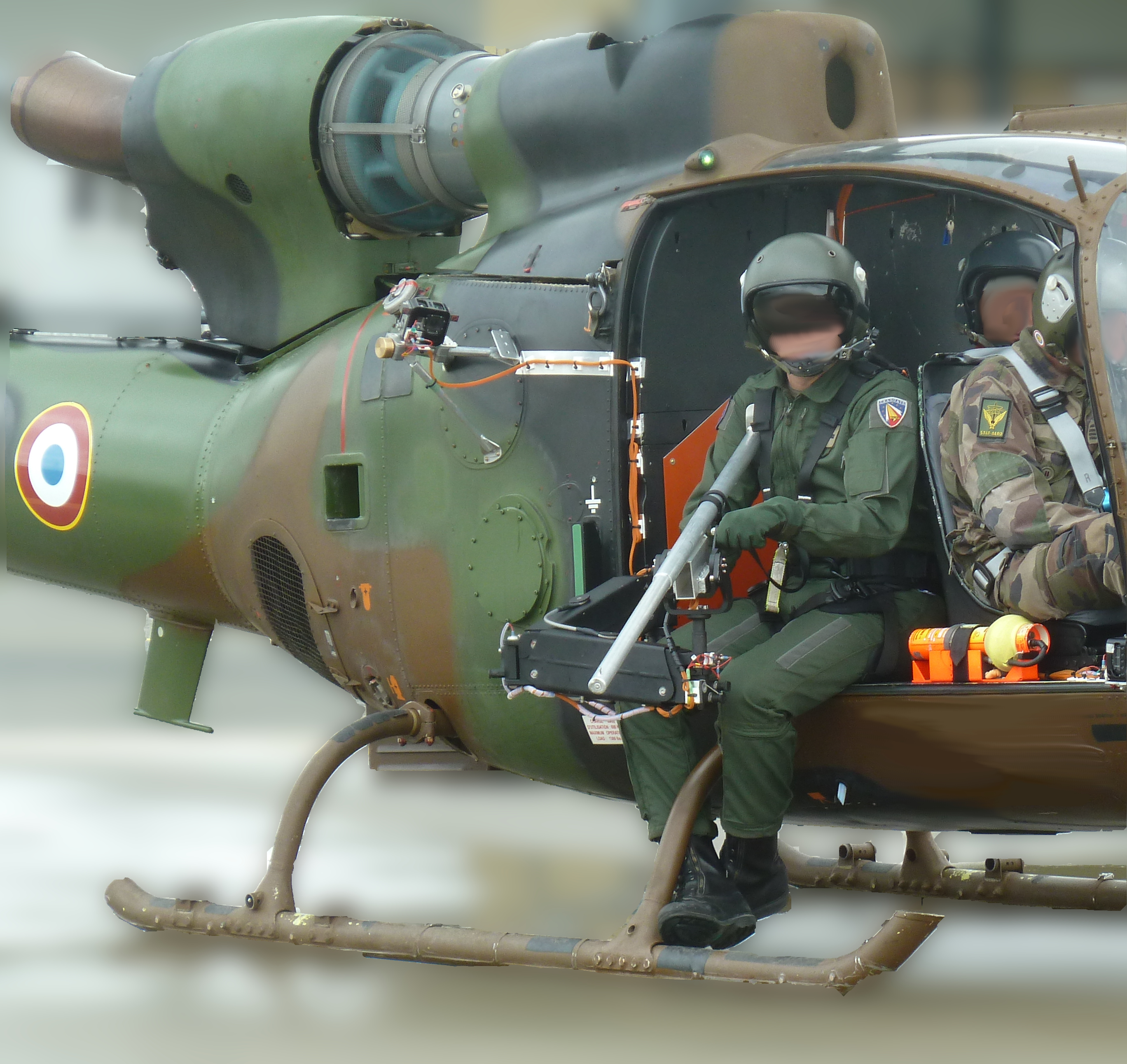}
        \vspace{3.5mm}
        \caption{Strike stabilization arm mounted on a Gazelle helicopter.}
        
    \end{subfigure}
    \hfill
    \begin{subfigure}[c]{0.31\textwidth}
        \vspace{11mm}
        \includegraphics[width=\textwidth]{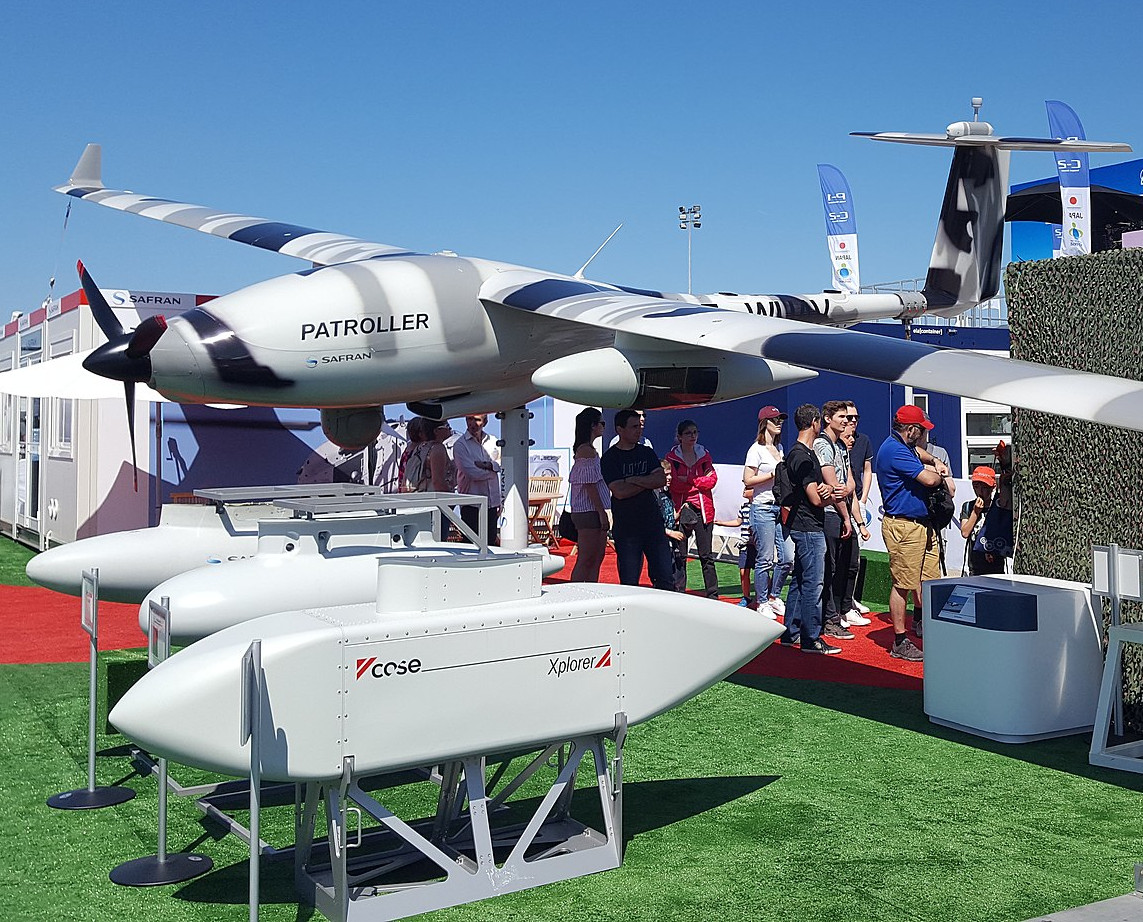}
        \vspace{7mm}
        \caption{POD Xplorer next to SAFRAN's Patroller at 2019 Paris Air Show.}
    \end{subfigure}
    \caption{Illustration of the three main products developed at COSE: GlobalScanner, Strike and POD Xplorer.}
    \label{fig:cose_product_intro}
\end{figure}

Recently, COSE started the CAMELEON project to replace the decades-old
GlobalScanner system. Its objective is to improve GlobalScanner in every aspect.
First, the linear sensor will be replaced by high-resolution CMOS matrix
sensors. With up to six sensors per system, CAMELON will be able to cover
extremely large areas with high resolution. The images will also largely
overlap, enabling precise 3D reconstruction of the flown-over areas, which is
especially important for mission preparation and risk analysis. CAMELEON will
also come with an improved software stack from mission planning to image
analysis and visualization. This PhD project is part of this software redesign.
The amount of image data acquired each second by the new system will be
overwhelming for a single photo interpreter as done with GlobalScanner.
Furthermore, standard communication streams will not be sufficient to send
entire images in real-time. Therefore, relevant information must be extracted
from the images, automatically and on edge. To this end, CAMELEON must integrate
intelligent algorithms able to find relevant structures and information inside
the mass of pixels acquired each second. These pieces of information are often
called Geospatial Intelligence (GEOINT). They consist of evidence of human
activity precisely georeferenced, with any kind of supplementary metadata (\eg
weather conditions or user annotation). Such evidence can be buildings, crop
fields, vehicles, or even animals. It is illustrated in \cref{fig:geoint_intro}.
In most cases, these are salient objects and can be detected in the images. Once
an object has been localized in an image, its precise location can be derived
from the carrier position, the direction of the camera and the digital elevation
model used, which produces a GEOINT. The GEOINT is then enriched with relevant
information about the object: what is the object? Is it dangerous? Is it moving?
Even though this seems to require a human appraisal, some of these questions can
be answered automatically. The main objective of this PhD project is to develop
models that will be able to produce GEOINT automatically. It will need to
localize objects and infer relevant metadata about them. It should drastically
increase the efficiency of photo-interpreters who will then be able to manage
the ever-increasing amount of data generated by aerial intelligence systems, in
particular within the CAMELEON project. 

\begin{figure}
    \centering
    \includegraphics[width=0.9\textwidth]{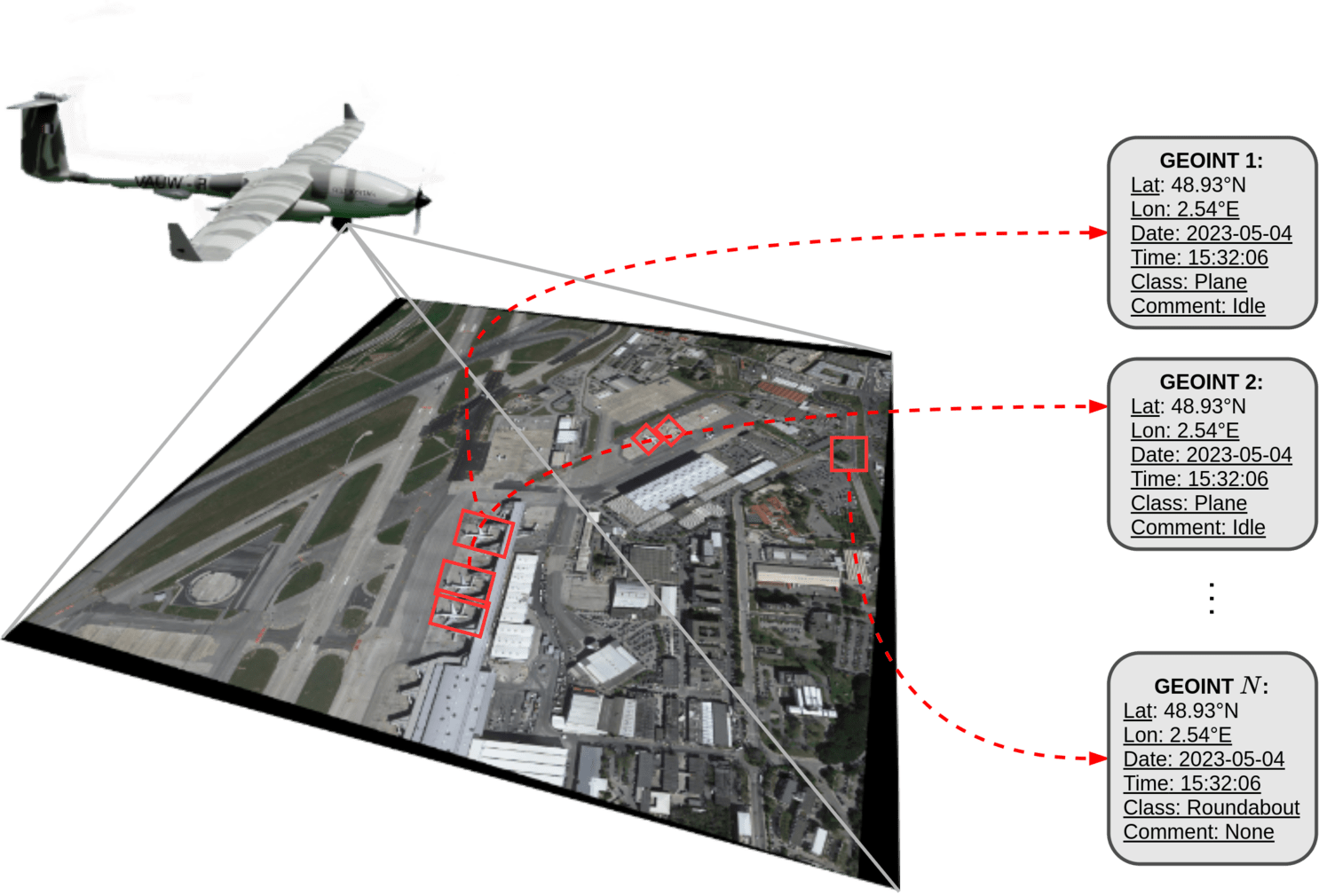}
    \caption{Description of Geospatial Intelligence (GEOINT).}
    \label{fig:geoint_intro}
\end{figure}

Object Detection (OD) is a crucial part of creating GEOINTs. In computer vision,
object detection is the task of localizing and classifying all objects visible
in an image. Of course, the notion of an object needs to be defined more precisely,
otherwise anything in the image can be considered of interest. A pre-defined set
of semantic classes $\mathcal{C}$ is fixed so that a clear distinction can be
made between objects of interest (\ie the ones we want to detect, also called
foreground objects) and background objects (\ie those we are not interested in).
Based on this distinction, the task of detecting objects can be split into two
sub-tasks. \textbf{1)} Localizing all the objects (foreground and background):
this can be done by finding the coordinates of the center of the objects, a
rectangular bounding box or even a precise segmentation mask for each object. In
general, the object detection task in computer vision is associated with
bounding box localization. \textbf{2)} Classifying the objects localized in
step 1). It consists in first filtering out background objects and then,
assigning a class label $c \in \mathcal{C}$ to each foreground object. Research
interest in the detection task dates back to the early 2000s when the
Viola-Jones object detector \cite{viola2001rapid} was first introduced. Since
then, plenty of algorithms have been proposed to improve both the speed and
quality of the detection. A breakthrough occurred in 2013 with the
first uses of deep convolutional networks for detection, namely OverFeat
\cite{sermanet2013overfeat} and R-CNN \cite{girshick2014rich}. These methods
paved the way for more elaborated deep-learning-based detectors.
Deep-learning detectors are often referred to as learning-based approaches as
they mainly rely on the \textit{learning from data} paradigm and supervised
learning. They contrast with earlier detection methods (also called traditional
methods) which often build upon hand-crafted features. A thorough review of both
traditional and learning-based object detectors is available in
\cref{sec:review_od}. Learning-based approaches have now established complete
dominance over traditional methods in terms of detection quality while having
reasonable speed performance. Therefore, most of this project will focus on
learning-based algorithms.   

\vspace{3mm}
The choice of deep-learning-based detectors may seem puzzling for COSE on-edge
applications. Computing resources are limited inside the carrier. The payload
must be as light as possible, so we cannot afford to embed heavy Graphical
Processing Units (GPU) enabled machines, designed to run deep learning models.
In addition, on-board power supplies cannot provide enough energy to run such
hardware. However, light-weight, energy-efficient GPUs exist, such as the Nvidia
Xavier and Orin Series, which are perfectly suited for deep-learning inference.
Nevertheless, another constraint remains, images must be processed in real-time.
The CAMELEON system is designed to take about 1 image every second. It is a
rather low frame rate but due to the sensor size and the number of cameras, this
represents a data stream of several hundreds of megapixels per second. To
process such a massive amount of images every second, the detection models must
be as light and efficient as possible. Fortunately, tools capable of optimizing
the inference of deep learning models exist. \cref{chap:integration} will
present these tools and how they can be leveraged to build detection models fast
enough for COSE's applications. This solves the issues related to the deployment
and inference of such models; however, a major concern remains: how to train
these object detectors? 

\vspace{3mm}
Learning-based methods and especially deep learning models heavily rely on data
to be trained. In general, the overall performance of a model highly depends on
the amount and quality of annotated data available during the training. For the
detection task, collecting large annotated datasets is time-consuming and
expensive. In some cases, it is even impossible. In the medical domain, for
instance, privacy-preserving regulations often prevent the use of personal data.
For military applications, this is even harder as potential training data are
classified. This is problematic for the training of data-hungry methods such as
deep learning. Fortunately, there are some learning strategies much more
data-efficient. These methods are usually referred to as \textit{few-shot} or
\textit{low-shot learning}, and thorough reviews of these methods will be
presented in \cref{sec:fsl}. While there are plenty of approaches to Few-Shot
Learning (FSL), all follow the same basic principle. First, learn generic
knowledge about a related task (source task), second, adapt to the target task.
These two training phases are referred to as base training and fine-tuning. In
the case of detection, a \textit{task} can designate a set of classes to be
detected, this problem is then called Few-Shot Object Detection (FSOD). A large
annotated dataset containing annotations of objects belonging to
$\mathcal{C}_{\text{base}}$ is available. The source task is to detect these
objects. Then, the target task is to detect objects from the so-called
\textit{novel classes}, only provided with a limited number of annotations.
\cref{chap:fsod} provides an in-depth review of existing work in this area.
Generally, the target task is performed on similar images as the ones seen
during base training. However, the target task can also be done with different
kinds of images, \eg the source task can be learned from natural images while
the target task on aerial or medical images. This is called \textit{Cross-Domain
Adaptation}. It complexifies significantly the problem, but it is a much more
realistic scenario in the industry. Collected datasets can only approximate the
real data distributions. Discrepancies between the acquisition settings (\ie
camera, lights etc.) and the application settings almost always produce a
performance drop. In medical imagery, this is a typical issue as different
scanners will not produce exactly similar images. This prevents training models
on scan collection from one hospital and deploying them in another. The military
use case is another critical example. The confidentiality of the images, and the
ever-changing environment and objects of interest make it difficult to build
robust detection algorithms. 

Given the constraints of COSE, the main objective of this project is
to develop data-efficient object detection methods based on few-shot learning.
We orient our research on the Few-Shot Object Detection
problem, \ie the adaptation to novel classes. While a detailed overview of the
thesis will be presented in the next section, we outline here the main parts of
this work. First, we conduct in \cref{part:literature} a thorough review of the
literature about object detection, few-shot learning and finally few-shot object
detection. Then, we propose three distinct FSOD approaches in
\cref{part:contributions_fsod}. This part includes experiments in the
Cross-Domain setting, inside \cref{chap:diffusion}. Our experiments mainly focus on publicly available
aerial datasets due to the lack of private datasets inside the company. These
datasets contain detection annotations and will be presented in
\cref{chap:od}. \cref{part:iou} presents an alternative to the Intersection over
Union, a bounding box similarity measure extensively employed in Object
Detection. Its use for both model evaluation and training are discussed in
\cref{chap:siou_metric}. Finally, \cref{chap:integration} provides details about
the deployment of object detection models according to the needs of COSE.

\section{Overview of the thesis}
\vspace{-1em}

This section outlines the content of each chapter of this thesis. This takes the
form of a small abstract per chapter. These abstracts will be repeated for
convenience at the beginning of the corresponding chapters.  
\subsection*{\cref{part:literature}: Literature Review on Object Detection,
Few-Shot Learning and Few-Shot Object Detection} 

The first part of this thesis is composed of three chapters. The two first
present the literature about Object Detection, Few-Shot Learning and Few-Shot
Object Detection. Then, the third chapter explores the challenges of applying
Few-Shot Object Detection on aerial images and presents our first contribution:
an analysis of these difficulties.

\noindent
\textbf{\cref{chap:od}: \nameref{chap:od}} \\
Object Detection and Few-Shot
Learning are two relevant subfields from the Computer Vision and Machine
Learning fields. Both are necessary to build detection techniques able to
generalize from limited data. Hence, this chapter reviews both Object Detection
and Few-Shot Learning. Both problems are defined, and detailed reviews of the
respective literature are conducted.

\noindent
\textbf{\cref{chap:fsod}: \nameref{chap:fsod}} \\
This chapter presents the task of detection in the few-shot regime
and reviews the existing literature about it. Few-Shot Object Detection (FSOD)
is at the crossroads of Object Detection and Few-Shot Learning, and therefore,
extensively relies on these two fields explored in \cref{chap:od}. Just as for
classification, various directions are explored in the literature to tackle the
detection task in the few-shot regime which will be presented in detail.
Finally, this chapter focuses on the aerial image application of FSOD methods
and extensions of the few-shot setting. 

\noindent
\textbf{\cref{chap:aerial_diff}: \nameref{chap:aerial_diff}} \\
The detection task becomes extremely challenging when limited annotated data is
available. In this chapter, we explore the reasons behind this difficulty. In
particular, we focus on the case of aerial images for which it is even harder to
apply FSOD techniques. It turns out that small objects are especially
challenging for the FSOD task and are the main source of error in remote sensing
images.

\textit{\textbf{Chapter's contributions}:
\begin{itemize}[nolistsep,topsep=2pt]
    \item[\faFileTextO] P. Le Jeune and A. Mokraoui, "Improving Few-Shot Object Detection through a Performance Analysis on Aerial and Natural Images," 2022 30th European Signal Processing Conference (EUSIPCO), Belgrade, Serbia, 2022, pp. 513-517, doi: 10.23919/EUSIPCO55093.2022.9909878.
    \item[\faFileTextO] P. Le Jeune and A. Mokraoui, "Amélioration de la détection d’objets few-shot à travers une analyse de performances sur des images aériennes et naturelles." GRETSI 2022, XXVIIIème Colloque Francophone de Traitement du Signal et des Images, Nancy, France
\end{itemize}
}

\subsection*{\cref{part:contributions_fsod}: Improving Few-Shot Object Detection through Various Approaches}
The second part of this thesis presents our main contributions to the Few-Shot
Object Detection (FSOD) field. Each chapter proposes a novel approach to
addressing the FSOD problem and discusses its pros and cons compared to
existing methods. These contributions led to several accepted articles in
international and national conferences and journals.

\noindent
\textbf{\cref{chap:prcnn}: \nameref{chap:prcnn}} \\
Prototypical Faster
R-CNN (PFRCNN) is a novel approach for FSOD based on metric learning. It embeds
prototypical networks inside the Faster R-CNN detection framework, specifically
in place of the classification layers in the RPN and the detection head. PFRCNN
is applied to synthetic images generated from the MNIST dataset and to real
aerial images with DOTA dataset. The detection performance of PFRCNN is slightly
disappointing but sets a first baseline on DOTA. However, the experiments
conducted with PFRCNN provide relevant information about the design choices for
FSOD approaches.

\textit{\textbf{Chapter's contributions}:
\begin{itemize}[nolistsep,topsep=2pt]
    \item[\faFileTextO] P. L. Jeune, M. Lebbah, A. Mokraoui and H. Azzag, "Experience feedback using Representation Learning for Few-Shot Object Detection on Aerial Images," 2021 20th IEEE International Conference on Machine Learning and Applications (ICMLA), Pasadena, CA, USA, 2021, pp. 662-667, doi: 10.1109/ICMLA52953.2021.00110.
\end{itemize}
}

\noindent
\textbf{\cref{chap:aaf}: \nameref{chap:aaf}}\\
Fair comparison is extremely challenging in the Few-Shot Object Detection task
as plenty of architectural choices differ from one method to another.
Attention-based approaches are no exception, and it is difficult to assess which
mechanisms are the most efficient for FSOD. In this chapter, we propose a highly
modular framework to implement existing techniques and design new ones. It
allows for fixing all hyperparameters except for the choice of the attention
mechanism. Hence, a fair comparison between various mechanisms can be made.
Using the framework, we also propose a novel attention mechanism specifically
designed for small objects.

\textit{\textbf{Chapter's contributions}:
\begin{itemize}[nolistsep,topsep=2pt]
    \item[\faPaperPlaneO] P. Le Jeune and A. Mokraoui, "A Comparative Attention Framework for Better Few-Shot Object Detection on Aerial Images", Submitted at the Elsevier Pattern Recognition journal.
    \item[\faFileTextO] P. Le Jeune and A. Mokraoui, "Cross-Scale Query-Support Alignment Approach for Small Object Detection in the Few-Shot Regime", Accepted at the IEEE International Conference on Image Processing 2023 (ICIP).
\end{itemize}
}

\noindent
\textbf{\cref{chap:diffusion}: \nameref{chap:diffusion}}\\
Previous chapters explore few-shot object detection with metric
learning and attention-based techniques. This chapter logically focuses on the
last major approach for FSOD: fine-tuning. Based on DiffusionDet, a recent
detection model leveraging diffusion models, we build a simple but efficient
fine-tuning strategy. The resulting method, called FSDiffusionDet, achieves
state-of-the-art FSOD on aerial datasets and competitive performance on natural
images. Extensive experimental studies explore the design choices of the
fine-tuning strategy to better understand the key components required to achieve
such quality. Finally, these impressive results allow considering more complex
settings such as cross-domain scenarios, which are especially relevant for COSE.

\textit{\textbf{Chapter's contributions}: This chapter describes very
recent work, and we plan to submit research articles to present these results.}

\vspace{1em}
\subsection*{\cref{part:iou}: Rethinking Intersection Over Union}

This part contains only one chapter which presents a contribution orthogonal to
the approaches proposed in \cref{part:contributions_fsod} as it questions the
relevance of the Intersection over Union, a key component of object detection
pipelines. 

\noindent
\textbf{\cref{chap:siou_metric}: \nameref{chap:siou_metric}}\\
Intersection over Union (IoU) is not an optimal box similarity measure
for evaluating and training object detectors. For evaluation, it is too strict
with small objects and does not align well with human perception. For training,
it provides a poor balance between small and large objects to the detriment of
small ones. We propose Scale-adaptative Intersection over Union (SIoU), a
parametric alternative that solves the shortcomings of IoU. We provide
empirical and theoretical arguments for the superiority of SIoU through in-depth
analysis of various criteria.

\textit{\textbf{Chapter's contributions}:
\begin{itemize}[nolistsep,topsep=2pt]
    \item[\faPaperPlaneO] P. Le Jeune and A. Mokraoui, "Rethinking Intersection Over Union for Small Object Detection in Few-Shot Regime", Submitted at the International Conference on Computer Vision 2023 (ICCV).
    \item[\faFileTextO] P. Le Jeune and A. Mokraoui, "Extension de l'\textit{Intersection over Union} pour améliorer la détection d'objets de petite taille en régime d'apprentissage few-shot", Accepted at GRETSI 2023.
\end{itemize}
}

\subsection*{\cref{part:application}: Prototyping and Industrial Application}
Finally, the last part of this thesis presents our industrial contributions.
This part is crucial for COSE as it bridges the gap between research
advancements and real-world applications. Therefore, the only chapter of this
part discusses the engineering aspects of object detection and is not associated
with any academic contribution. 

\noindent
\textbf{\cref{chap:integration}: \nameref{chap:integration}}\\
Detection models are often heavy and are not well suited for
    COSE's application. In this chapter, we first present in detail the CAMELEON
    system and its constraints. Then, we study the influence of the model size
    on the performance and present useful tools and tricks to accelerate the
    inference. Finally, we explain how the detection models are deployed inside
    the CAMELEON prototype and how they perform on aerial images.

\section{Summary of the Contributions}
\vspace{-1em}

\textbf{International Conference Articles}
\begin{itemize}[nolistsep]
    \item[\faFileTextO] P. Le Jeune, M. Lebbah, A. Mokraoui and H. Azzag, "Experience feedback using Representation Learning for Few-Shot Object Detection on Aerial Images," 2021 20th IEEE International Conference on Machine Learning and Applications (ICMLA), Pasadena, CA, USA, 2021, pp. 662-667, doi: 10.1109/ICMLA52953.2021.00110.
    \item[\faFileTextO] P. Le Jeune and A. Mokraoui, "Improving Few-Shot Object Detection through a Performance Analysis on Aerial and Natural Images," 2022 30th European Signal Processing Conference (EUSIPCO), Belgrade, Serbia, 2022, pp. 513-517, doi: 10.23919/EUSIPCO55093.2022.9909878.
    \item[\faFileTextO] P. Le Jeune and A. Mokraoui, "Cross-Scale Query-Support Alignment Approach for Small Object Detection in the Few-Shot Regime", Accepted at the IEEE International Conference on Image Processing 2023 (ICIP).
\end{itemize}

\textbf{National Conference Articles}
\begin{itemize}[nolistsep]
    \item[\faFileTextO] P. Le Jeune and A. Mokraoui, "Amélioration de la détection d’objets few-shot à travers une analyse de performances sur des images aériennes et naturelles." GRETSI 2022, XXVIIIème Colloque Francophone de Traitement du Signal et des Images, Nancy, France.
    \item[\faFileTextO] P. Le Jeune and A. Mokraoui, "Extension de l'\textit{Intersection over Union} pour améliorer la détection d'objets de petite taille en régime d'apprentissage few-shot", GRETSI 2023, XXIXème Colloque Francophone de Traitement du Signal et des Images, Grenoble, France.
\end{itemize}

\textbf{Submitted Articles}
\begin{itemize}[nolistsep]
    \item[\faPaperPlaneO] P. Le Jeune and A. Mokraoui, "A Comparative Attention Framework for Better Few-Shot Object Detection on Aerial Images", Submitted at the Elsevier Pattern Recognition journal.
    \item[\faPaperPlaneO] P. Le Jeune and A. Mokraoui, "Rethinking Intersection Over Union for Small Object Detection in Few-Shot Regime", Submitted at the International Conference on Computer Vision 2023 (ICCV).
\end{itemize}

\textbf{Oral Presentations}\\
During the PhD, I had the opportunity to give talks in various occasions listed
below:
\begin{itemize}[nolistsep]
    \item[-] L2TI's scientific day (Dec. 2020).
    \item[-] \textit{Prototypical Faster R-CNN for Few-Shot Object Detection on Aerial Images}, DeepLearn Summer School 2021, Las Palmas de Gran Canaria (Jul. 29, 2021).
    \item[-] \textit{Prototypical Faster R-CNN for Few-Shot Object Detection on Aerial Images} at a GDR-ISIS meeting: \textit{Vers un apprentissage pragmatique dans
    un contexte de données visuelles labellisées limitées}, Paris, (Nov. 26, 2021).
    \item[-] L2TI's Doctoral seminar (Mar. 2022 and Feb. 2023).
    \item[-] \textit{Few-Shot Object Detection on Aerial Images}, Seminar at ETS Montreal (Sep. 28, 2022).
\end{itemize}

\textbf{Internships Supervision}\\
I supervised four internships over the three years of this PhD, three inside the
company and one at within the L2TI:
\begin{itemize}[nolistsep]
    \item[-] Conception et mise en oeuvre d’algorithmes de suivi d’objets dans des images
    aériennes (March-August 2021 -- COSE).
    \item[-] Optimisation et intégration d’algorithmes de détection d’objets dans un système
    embarqué (March-August 2022 -- COSE).
    \item[-] Self-supervised learning for Few-shot Object Detection (April-August 2022 – L2TI au
    travers du LabCom IRISER).
    \item[-] Détection d'objets few-shot par visual transformers sur des images
    Aériennes (March-August 2023 -- COSE and L2TI through the LabCom IRISER). 
\end{itemize}

In addition to the supervision of two internships, I am actively involved inside
the LabCom IRISER\footnote{\href{https://www-l2ti.univ-paris13.fr/iriser/}{Link to the LabCom IRISIER's website}} which is a
joint laboratory between COSE, the L2TI and the LIPN. It was created one year
after the beginning of my PhD at the instigation of my academic and industrial
supervisors.

\textbf{Open-source Software}\\
In the course of the various project I conducted during this PhD, I wrote
multiple open-source Python packages that can be found on GitHub: 
\begin{itemize}[nolistsep]
    \item[\faGithub] \href{https://github.com/pierlj/proto_faster_rcnn}{Prototypical Faster R-CNN}
    \item[\faGithub] \href{https://github.com/pierlj/aaf_framework}{AAF framework}
    \item[\faGithub] \href{https://github.com/pierlj/pycocosiou}{Pycocosiou}
    \item[\faGithub] \href{https://github.com/pierlj/fs_diffusiondet}{FSDiffusionDet}
\end{itemize}

\addtocounter{chapter}{-1}  
\stepcounter{chapter}

\renewcommand{\theHchapter}{FR\thechapter}

\chapter*{Introduction (Français)}
\addcontentsline{toc}{chapter}{Introduction (Français)} 
\epigraph{\itshape  Si une machine doit être infaillible, alors elle ne peut pas aussi être intelligente.}{-- Alan Turing}

Pour introduire ce manuscrit de thèse, le contexte industriel et les motivations
de ce projet sont présentés. D'abord, sont introduits l'entreprise COSE et le
Laboratoire de Traitement et Transport de l’Information (L2TI) qui ont
collaboré sur cette thèse CIFRE. Ensuite, nous décrivons ce qu'est la détection
d'objets dans le cadre de la vision par ordinateur et comment les contraintes
industrielles liées à COSE ont orienté la thèse vers l'apprentissage frugal (dit
\textit{few-shot}). Dans un second temps, la structure de ce
manuscrit est exposée en présentant un résumé individuel pour chaque chapitre. Enfin, une
dernière partie liste les différentes contributions apportées au cours de ce
projet, cela inclut des articles de recherche publiés ou soumis dans des
conférences nationales et internationales.

\phantomsection
\section{Contexte industriel, motivation et objectifs}
\vspace{-1em}

Cette thèse a pour origine la collaboration entre le laboratoire L2TI de
l'Université Sorbonne Paris Nord (USPN) et la société COSE. Le L2TI a été fondé
en 1998 et est un membre de la Fédération de Recherche MathSTIC du CNRS (FR
3734) qui inclut également deux laboratoires CNRS: le Laboratoire Analyse,
Géométrie et Applications (LAGA), UMR 7539 et le Laboratoire d'Informatique de
Paris Nord (LIPN), UMR 7030. Ces laboratoires sont tous rattachés à l'Institut
Galilée. Deux équipes de recherche cohabitent dans le L2TI. D'abord, l'équipe
Multimédia, qui se concentre sur le traitement et l'analyse de l'information
visuelle et audio. Ensuite, l'équipe Réseaux, qui travaille sur le transport de
l'information et les communications. Ce projet de thèse s'inscrit logiquement
dans l'équipe Multimédia.

COSE \footnote{\href{https://www.cose.fr/home/}{https://www.cose.fr/}} est une
PME innovante d'environ 20 salariés. C'est un fournisseur de rang 1 de l'état
dans le secteur de l'aéronautique et de la défense. COSE est né en tant que
startup de l'INRIA dans les années 90 et la recherche est toujours au cœur de
son processus industriel. Bien que relativement petite, COSE possède des équipes
pluridisciplinaires de haut niveau dans des domaines tels que la mécanique,
l'électronique, la navigation, l'automatique et les systèmes embarqués. La
taille de COSE lui confère une agilité et une efficacité remarquable comparée à
ces principaux compétiteurs. Cet avantage permet à l'entreprise de créer des
partenariats forts avec les acteurs majeurs de l'aéronautique et de la défense. 

COSE développe, produit et maintient des systèmes de renseignements aéroportés
et des équipements embarqués en tout genre. Ces produits sont principalement
destinés à un usage militaire et sont donc soumis à des critères de qualité
stricts. La relation entre les forces armées et COSE est gérée par la Direction
Générale de l'Armement (DGA) qui est de fait l'un des principaux clients de
COSE. COSE a pour l'instant trois produits principaux dans sa gamme que les forces
françaises utilisent pour différentes missions (voir
\cref{fig:cose_product_intro_fr}) :

\begin{itemize}[nolistsep]
    \item[-] \textbf{GlobalScanner} : un système de caméra embarquée qui produit
    des images haute résolution et géoréférencées en temps réel. Le système est
    constitué d'un capteur linéaire de très grande résolution. Ce capteur est
    stabilisé et intégré au sein d'une enceinte mécanique qui peut être intégrée
    sous différents types d'aéronefs (hélicoptère, avion, drone, etc.). Le
    capteur est connecté à un poste de contrôle et une suite logicielle
    permettant de piloter la caméra et de gérer les flux d'images.
    \item[-] \textbf{Strike} : un bras de stabilisation d'arme à feu pour
    hélicoptère. Il améliore sensiblement la précision des tireurs et réduit les
    risques de dommages collatéraux. 
    \item[-] \textbf{POD Xplorer} : un pod multifonction pouvant être attaché en
    dessous de différents types de porteurs. Il permet d'embarquer simplement
    des charges utiles variées comme des capteurs optiques, des LIDARs, ou des
    équipements scientifiques.
\end{itemize}

\begin{figure}
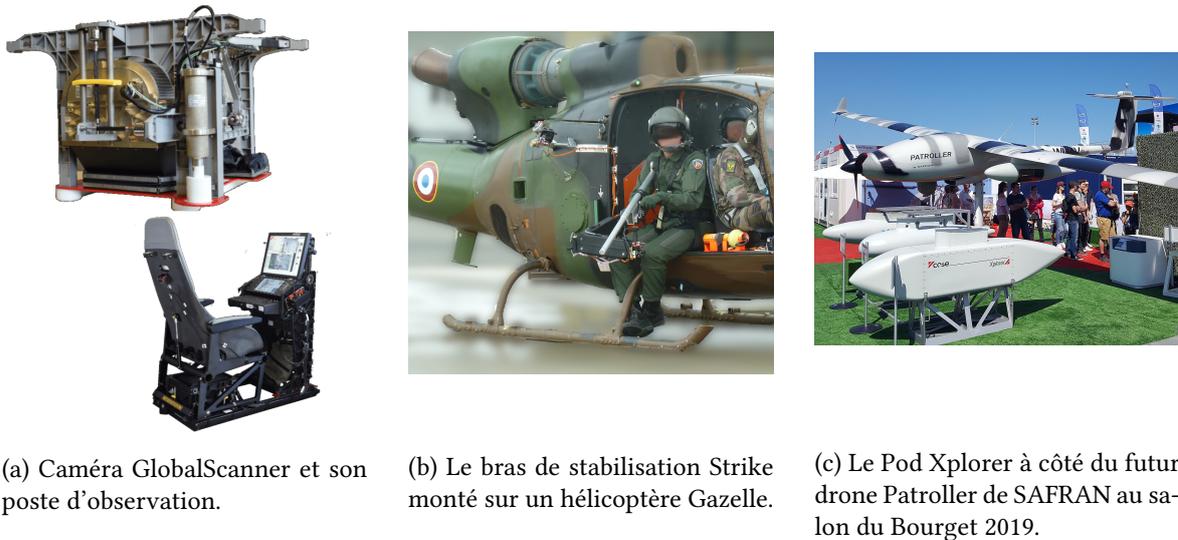

    \centering
    \begin{subfigure}[c]{0.31\textwidth}
        \includegraphics[width=\textwidth]{images/global_scan.png}
        \caption{Caméra GlobalScanner et son poste d'observation.}
    \end{subfigure}
    \hfill
    \begin{subfigure}[c]{0.31\textwidth}
        \vspace{4.5mm}
        \includegraphics[width=\textwidth]{images/strike_flou.JPG}
        \vspace{3.5mm}
        \caption{Le bras de stabilisation Strike monté sur un hélicoptère Gazelle.}
    \end{subfigure}
    \hfill
    \begin{subfigure}[c]{0.31\textwidth}
        \vspace{11mm}
        \includegraphics[width=\textwidth]{images/pod+patroller.jpg}
        \vspace{7mm}
        \caption{Le Pod Xplorer à côté du futur drone Patroller de SAFRAN au salon du Bourget 2019.}
    \end{subfigure}
    \caption[]{Illustration des trois produits phares de COSE: GlobalScanner, Strike et le POD Xplorer.}
    \label{fig:cose_product_intro_fr}
\end{figure}

Récemment, COSE a lancé le projet CAMELEON afin de remplacer GlobalScanner. Son
objectif premier est de surpasser GlobalScanner dans tous les aspects.
Premièrement, le capteur linéaire sera remplacé par un capteur matriciel CMOS de
haute résolution. CAMELEON pourra embarquer jusqu'à six capteurs avec des
orientations différentes afin de couvrir des très grandes zones au sol tout en
conservant une grande définition. Les images ainsi acquises auront beaucoup de
recouvrement afin de permettre la reconstruction 3D des zones survolées. C'est
un aspect extrêmement important de la préparation de mission et la gestion des
risques pour les forces armées. CAMELEON proposera également une amélioration
complète du logiciel d'observation et notamment en ce qui concerne l'analyse et
le traitement des images. La quantité d'images obtenues chaque seconde par le
système sera trop importante pour être analysée par un seul photo-interprète. De
plus, les moyens de communication standards n'ont pas un débit suffisant pour
transmettre les images en temps réel. Ainsi, il est nécessaire d'extraire les
informations stratégiques des images, automatiquement et à bord. CAMELEON doit
donc être doté d'algorithmes intelligents et efficaces afin d'extraire les
informations pertinentes en temps réel. Cette thèse s'inscrit dans la refonte
logicielle de CAMELEON et tente de répondre aux contraintes du projet. Les
informations extraites des images sont souvent appelées \textit{GEOspaital
INTelligence} (GEOINT). Il s'agit principalement de preuve d'activité humaine,
précisément géoréférencées ainsi que de méta-données en tout genre (\eg les
conditions météorologiques ou des annotations de l'interprète). Le concept de
GEOINT est illustré dans la Figure \ref{fig:geoint_intro_fr}. Il peut s'agir de
bâtiments, de champs, de véhicules ou même d'animaux. Dans la plupart des cas,
ce sont des objets saillants qui sont visibles dans les images aériennes. Une
fois qu'un objet a été localisé dans l'image, sa géolocalisation précise peut
être calculée en fonction de la position du porteur, de l'angle de la caméra et
du modèle numérique de terrain utilisé, cela produit ainsi un GEOINT. Il peut
ensuite être enrichi avec des informations supplémentaires, pertinentes pour
l'opération : quel est cet objet ? Est-ce une menace ? Est-il en mouvement ?
Même si ces questions semblent requérir le jugement humain, on peut en réalité
souvent y répondre automatiquement. L'objectif principal de cette thèse est de
produire des modèles capables d'automatiser la création de GEOINT. Pour cela,
ces modèles devront localiser les objets d'intérêt et inférer les méta-données
pertinentes en lien avec ces objets. Aidés par ces outils, les photo-interprètes
seront bien plus efficaces et pourront gérer des masses d'images toujours plus
grandes. 

\begin{figure}
    \centering
    \includegraphics[width=0.95\textwidth]{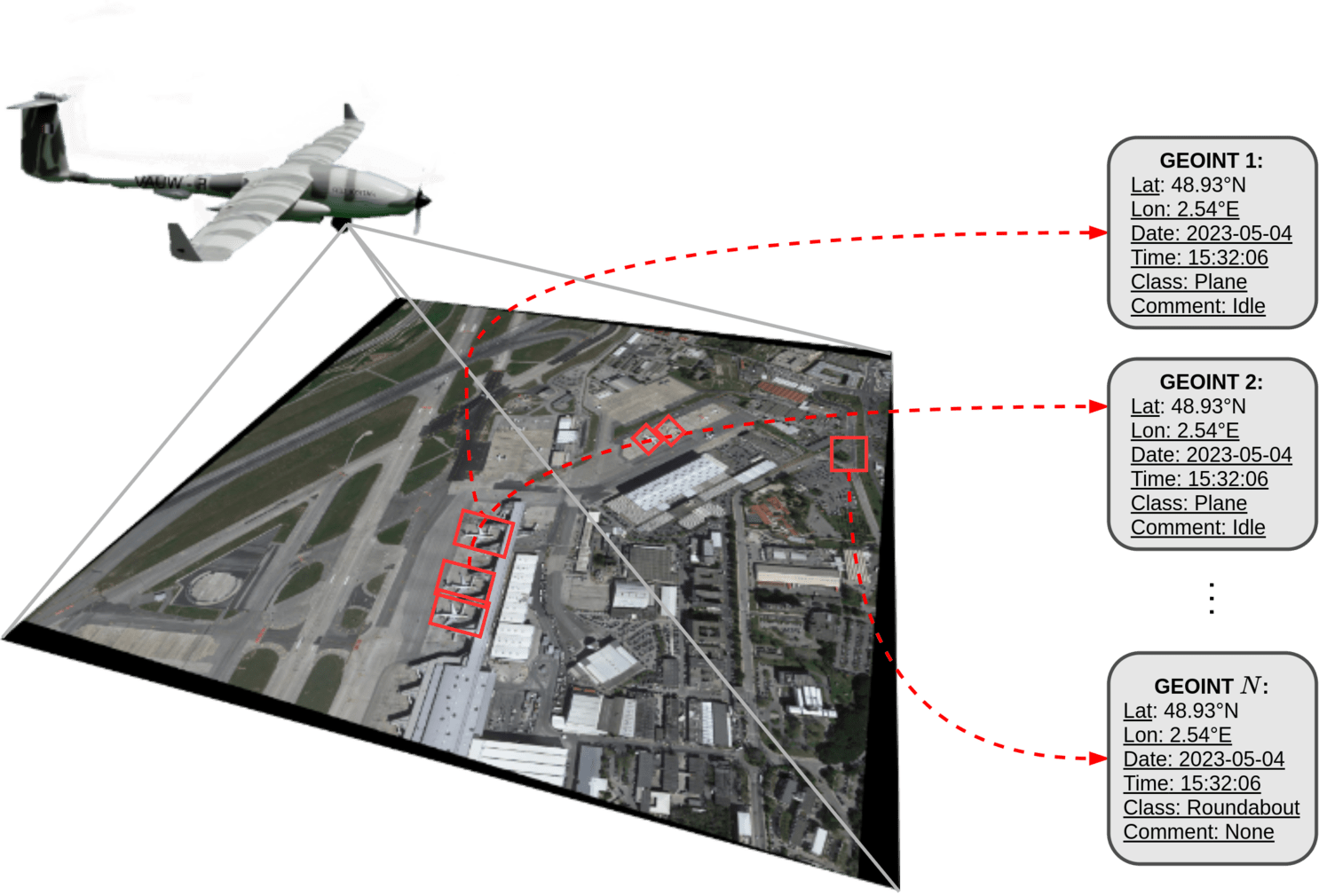}
    \caption[]{Illustration d'un renseignement géospatial (GEOINT).}
    \label{fig:geoint_intro_fr}
\end{figure}

La détection d'objets est une étape cruciale de la création de GEOINTs. En
vision par ordinateur, la détection d'objets consiste à localiser et classifier
tous les objets visibles dans une image. Bien sûr, la notion d'objet doit être
définie de manière plus précise, sinon tout ce qui se trouve dans l'image peut
être considéré comme étant d'intérêt. Un ensemble prédéfini de classes
sémantiques $\mathcal{C}$ est fixé afin d'établir une distinction claire entre
les objets d'intérêt (ceux que l'on souhaite détecter) et les objets de
l'arrière-plan (ceux qui ne nous intéressent pas). Sur la base de cette
distinction, la tâche de détection d'objets peut être divisée en deux
sous-tâches. \textbf{1)} Localiser tous les objets (objets d'intérêt et objets
de l'arrière-plan) : cela peut être fait en trouvant les coordonnées du centre
des objets, d'une boîte englobante rectangulaire ou même un masque de
segmentation précis pour chaque objet. En général, la tâche de détection
d'objets en vision par ordinateur est associée à la localisation par boîte
englobante. \textbf{2)} Classer les objets localisés à l'étape 1). Il s'agit
d'abord de filtrer les objets de l'arrière-plan, puis d'attribuer une classe $c
\in \mathcal{C}$ à chaque objet d'intérêt. La recherche sur la détection
d'objets a pour origine le début des années 2000, lorsque le détecteur d'objets
Viola-Jones \cite{viola2001rapid} a été introduit pour la première fois. Depuis,
de nombreux algorithmes ont été proposés pour améliorer à la fois la vitesse et
la qualité de la détection. Une avancée remarquable s'est produite en 2013 avec
les premières utilisations de réseaux de neurones convolutifs pour la détection,
notamment avec OverFeat \cite{sermanet2013overfeat} et R-CNN
\cite{girshick2014rich}. Ces méthodes ont ouvert la voie à des détecteurs de
plus en plus élaborés basés sur l'apprentissage profond. Ces détecteurs reposent
principalement sur le paradigme de l'apprentissage machine et notamment
l'apprentissage supervisé. Ils diffèrent des méthodes de détection antérieures
(appelées méthodes traditionnelles) qui s'appuient souvent sur des
caractéristiques manuelles. Une revue détaillée des détecteurs d'objets
traditionnels et basés sur l'apprentissage est disponible dans la section
\labelcref{sec:review_od}. Les approches basées sur l'apprentissage ont désormais
établi une domination complète sur les méthodes traditionnelles en termes de
qualité de détection, tout en offrant des temps d'exécution plus rapides.
Par conséquent, la majeure partie de ce projet se concentrera sur les
algorithmes basés sur l'apprentissage.

Le choix de détecteurs basés sur l'apprentissage profond peut sembler compliqué
pour les applications embarquées telles que celles de COSE. Les
ressources de calcul sont limitées une fois en vol. La charge utile doit être
aussi légère que possible, COSE ne peut donc pas nous permettre d'intégrer de
lourdes machines dotées de cartes graphiques (GPU) conçues pour exécuter des
modèles d'apprentissage profond. De plus, les alimentations embarquées ne
peuvent pas fournir suffisamment d'énergie pour faire fonctionner un tel
matériel. Cependant, il existe des GPU légers et économes en énergie, tels que
les gammes Nvidia Xavier et Orin, qui conviennent parfaitement à des systèmes
embarqués. Néanmoins, une autre contrainte subsiste : les images doivent être
traitées en temps réel. Le système CAMELEON est conçu pour prendre environ une
image par seconde. Il s'agit d'une fréquence d'image assez faible, mais en
raison de la taille du capteur et du nombre de caméras, cela représente un flux
de données de plusieurs centaines de mégapixels par seconde. Pour traiter une
telle quantité d'images par seconde, les modèles de détection doivent être aussi
légers et efficaces que possible. Heureusement, il existe des outils capables
d'optimiser l'inférence des modèles d'apprentissage profond. Le
chapitre \ref{chap:integration} présentera ces outils et comment ils peuvent être
utilisés pour construire des modèles de détection suffisamment rapides pour les
applications de COSE. Cela résout les problématiques liées au déploiement et à
l'inférence des modèles, cependant, une préoccupation majeure subsiste : comment
entraîner ces détecteurs d'objets ?

Les méthodes basées sur l'apprentissage, et en particulier les modèles
d'apprentissage profond, reposent fortement sur les données pour leur
entraînement. En général, les performances globales d'un modèle dépendent de la
quantité et de la qualité des données annotées disponibles lors de
l'entraînement. Pour la détection, la collecte de grands ensembles de données
annotées est chronophage et coûteuse. Dans certains cas, il est même impossible
de rassembler de tels ensembles de données pour l'entraînement. Dans le domaine
médical, par exemple, les réglementations empêchent souvent l'utilisation de
données personnelles. Dans le domaine militaire, cela est encore plus difficile
car les données d'entraînement sont classifiées et ne peuvent être divulguées en
aucune circonstance. Cela est problématique pour l'entraînement des méthodes
d'apprentissage profond, gourmandes en données. Heureusement, il existe des
stratégies d'apprentissage beaucoup plus efficaces en termes de données. Ces
méthodes sont généralement désignées comme l'apprentissage frugal
(\textit{few-shot learning} (FSL) en anglais). Une revue détaillée de ces
méthodes sera présentée dans la section \labelcref{sec:fsl}. Bien qu'il existe
de nombreuses approches différentes pour l'apprentissage \textit{few-shot},
elles suivent souvent le même principe de base. Premièrement, elles apprennent
des connaissances générales sur une tâche connexe (tâche source), puis elles
s'adaptent à une tâche cible. Ces deux phases d'entraînement sont appelées
entraînement de base et \textit{fine-tuning}. Dans le cas de la détection, une
\textit{tâche} désigne un ensemble de classes à détecter, et ce problème est
alors appelé détection d'objets \textit{few-shot} (FSOD en anglais). Un grand
ensemble de données annotées contenant des annotations d'objets appartenant à
$\mathcal{C}_{\text{base}}$ est disponible. La tâche source consiste à détecter
ces objets. Ensuite, la tâche cible consiste à détecter des objets des
\textit{classes nouvelles}, en disposant uniquement d'un nombre limité
d'annotations. Le chapitre \labelcref{chap:fsod} fournit une revue approfondie
des travaux existants dans ce domaine. Dans le cas général, la tâche cible est
réalisée sur des images similaires à celles vues pendant l'entraînement de base.
Cependant, la tâche cible peut aussi être réalisée avec différents types
d'images. Par exemple, la tâche source peut être apprise à partir d'images
naturelles tandis que la tâche cible porte sur des images aériennes ou
médicales. On appelle cela l'adaptation au domaine. Cela complexifie
considérablement le problème, mais c'est un scénario beaucoup plus réaliste dans
l'industrie. Les ensembles de données collectés ne peuvent qu'approximer la
distribution réelle des données pour un problème spécifique. Les divergences
entre les paramètres d'acquisition (appareil photo, éclairage, etc.) et les
conditions réelles entraînent presque toujours une baisse des performances. En
imagerie médicale, il s'agit d'un problème classique car différents scanners ne
produiront pas exactement les mêmes images. Cela empêche de former des modèles
sur une collection d'images provenant d'un hôpital et de les déployer dans un
autre. Le cas d'utilisation militaire est un autre exemple critique. La
confidentialité des images, l'environnement et les objets d'intérêt en constante
évolution rendent difficile la construction d'algorithmes de détection robustes.

Compte tenu des contraintes industrielles de COSE, l'objectif principal de ce
projet de thèse est de développer des méthodes de détection d'objets efficaces
en termes de données, basées sur des stratégies d'apprentissage
\textit{few-shot}. Nous avons choisi d'orienter nos recherches principalement
sur le problème de la détection d'objets \textit{few-shot}, c'est-à-dire
l'adaptation aux nouvelles classes. Bien qu'un aperçu détaillé de la thèse soit
présenté dans la prochaine section, ici les principales parties de ce travail
sont décrites. Tout d'abord, dans la partie \labelcref{part:literature} présente
une revue approfondie de la littérature sur la détection d'objets,
l'apprentissage frugal et enfin la détection d'objets \textit{few-shot}.
Ensuite, trois approches distinctes sont proposées pour la détection d'objets
\textit{few-shot} dans la partie \ref{part:contributions_fsod}. Cette partie
comprend également des expériences dans le cadre de l'adaptation au domaine,
principalement dans la section \ref{sec:cda_analysis}. Nos expériences se
concentrent principalement sur des jeux de données d'images aériennes disponibles
publiquement, faute d'ensembles de données privées disponibles au sein de
l'entreprise. Ces jeux de données contiennent des annotations de détection et seront
présentés dans le chapitre \ref{chap:od}. La partie \ref{part:iou} présente une
alternative à l'\textit{Intersection over Union} (IoU), une mesure de similarité
des boîtes englobantes largement utilisée en détection d'objets. Son utilisation
pour l'évaluation et l'entraînement des modèles est discutée dans le chapitre
\ref{chap:siou_metric}. Enfin, le chapitre \ref{chap:integration} fournit des
détails sur le déploiement des modèles de détection d'objets en fonction des
contraintes de COSE.

\section{Plan de la thèse}
\vspace{-1em}

Cette section présente un aperçu du contenu de chaque chapitre de cette thèse.
Cela prend la forme d'un court résumé par chapitre. Ces résumés seront répétés
pour plus de commodité au début des chapitres correspondants. 

\subsection*{Partie \ref{part:literature} : Revue de la littérature sur la détection
d'objets, l'apprentissage \textit{few-shot} et la détection
\textit{few-shot}} 

La première partie de cette thèse est composée de trois chapitres. Les deux
premiers présentent la littérature sur la détection d'objets, l'apprentissage
\textit{few-shot} et la détection d'objets à faible échantillonnage. Le
troisième chapitre, quant à lui, explore les défis liés à l'application de la
détection d'objets \textit{few-shot} à des images aériennes et présente notre
première contribution : une analyse détaillée de ces difficultés.

\noindent
\textbf{\cref{chap:od} : Détection d'objets, apprentissage \textit{few-shot} et adaptation aux domaines} \\
La détection d'objets et l'apprentissage \textit{few-shot} sont des
sous-domaines de la vision par ordinateur et de
l'apprentissage automatique. Les deux sont nécessaires pour développer des
techniques de détection capables de généraliser à partir de données limitées.
Par conséquent, ce chapitre passe en revue à la fois la détection d'objets et
l'apprentissage \textit{few-shot}. Les deux problèmes sont définis et des
revues détaillées des littératures respectives sont réalisées.

\noindent
\textbf{\cref{chap:fsod} : Détection d'objets \textit{few-shot}} \\
Ce chapitre présente la tâche de détection dans le régime \textit{few-shot} et
passe en revue la littérature existante sur ce sujet. La détection d'objets
\textit{few-shot} (FSOD) se situe à l'intersection de la détection d'objets et
de l'apprentissage \textit{few-shot}, et repose donc largement sur ces deux
domaines explorés dans le chapitre \ref{chap:od}. Tout comme pour la
classification, différentes approches sont explorées dans la littérature pour
aborder la tâche de détection en régime \textit{few-shot}. Enfin, ce chapitre se
concentre sur l'application de la détection d'objets \textit{few-shot} sur des
images aériennes et sur les extensions du régime \textit{few-shot}. 

\noindent
\textbf{\cref{chap:aerial_diff} : Analyse des difficultés liées à la détection \textit{few-shot}} \\
La tâche de détection devient extrêmement difficile lorsque les données annotées
sont limitées. Dans ce chapitre, les raisons derrière ces
difficultés sont explorées. En particulier, nous nous concentrons sur le cas des images
aériennes pour lesquelles il est encore plus difficile d'appliquer des
techniques de détection \textit{few-shot}. Il s'avère que les petits objets sont
particulièrement difficiles à localiser en régime \textit{few-shot} et
sont la principale source d'erreur dans les images aériennes.

\textit{\textbf{Contributions liées à ce chapitre }:
\begin{itemize}[nolistsep,topsep=2pt]
    \item[\faFileTextO] P. Le Jeune and A. Mokraoui, "Improving Few-Shot Object Detection through a Performance Analysis on Aerial and Natural Images," 2022 30th European Signal Processing Conference (EUSIPCO), Belgrade, Serbia, 2022, pp. 513-517, doi: 10.23919/EUSIPCO55093.2022.9909878.
    \item[\faFileTextO] P. Le Jeune and A. Mokraoui, "Amélioration de la détection d’objets few-shot à travers une analyse de performances sur des images aériennes et naturelles." GRETSI 2022, XXVIIIème Colloque Francophone de Traitement du Signal et des Images, Nancy, France.
\end{itemize}
}

\vspace{2em}
\subsection*{\cref{part:contributions_fsod} : Amélioration de la détection
\textit{few-shot} à travers plusieurs approches} La deuxième partie de cette
thèse présente nos principales contributions dans le domaine de la détection
d'objets \textit{few-shot} (FSOD). Chaque chapitre propose une nouvelle approche
pour aborder le problème de FSOD et discute de ses avantages et inconvénients
par rapport aux méthodes existantes. Ces contributions ont abouti à plusieurs
articles acceptés et soumis dans des conférences et des revues internationales et
nationales.

\noindent
\textbf{Partie \ref{chap:prcnn} : Retour d'expérience sur l'apprentissage de métrique pour FSOD} \\
Prototypical Faster R-CNN (PFRCNN) est une approche innovante pour la détection
d'objets \textit{few-shot} (FSOD) basée sur l'apprentissage de métriques. Elle
intègre des réseaux  de prototypes (\textit{prototypical networks}) à
l'intérieur de Faster R-CNN, plus précisément à la place
des couches de classification dans le RPN et la tête de détection. PFRCNN est
appliqué à des images synthétiques générées à partir de l'ensemble de données
MNIST et à des images aériennes réelles avec le jeu de données DOTA. Les
performances de détection de PFRCNN sont légèrement décevantes, mais elles
établissent un premier point de repère sur DOTA. Les expériences
menées avec PFRCNN fournissent des informations pertinentes sur les choix de
conception pour les approches FSOD.

\textit{\textbf{Contributions liées à ce chapitre } :
\begin{itemize}[nolistsep,topsep=2pt]
    \item[\faFileTextO] P. L. Jeune, M. Lebbah, A. Mokraoui and H. Azzag, "Experience feedback using Representation Learning for Few-Shot Object Detection on Aerial Images," 2021 20th IEEE International Conference on Machine Learning and Applications (ICMLA), Pasadena, CA, USA, 2021, pp. 662-667, doi: 10.1109/ICMLA52953.2021.00110.
\end{itemize}
}

\noindent
\textbf{\cref{chap:aaf} : Un environnement modulaire pour la détection \textit{few-shot} basée sur des mécanismes d'attention}\\
Comparer de manière équitable différents modèles est extrêmement difficile en
détection d'objets \textit{few-shot} car de nombreuses options architecturales
diffèrent d'une méthode à une autre. Les approches basées sur l'attention ne
font pas exception, et il est difficile d'évaluer quels mécanismes sont les plus
efficaces pour le FSOD. Ce chapitre présente un environnement modulaire
pour réimplémenter les techniques existantes et concevoir de nouvelles
approches. Il permet de fixer tous les hyperparamètres à l'exception du
mécanisme d'attention et de les comparer de manière équitable. En
utilisant cet environnement, nous proposons également un nouveau mécanisme
d'attention spécifiquement conçu pour les petits objets.

\textit{\textbf{Contributions liées à ce chapitre} :
\begin{itemize}[nolistsep,topsep=2pt]
    \item[\faPaperPlaneO] P. Le Jeune and A. Mokraoui, "A Comparative Attention Framework for Better Few-Shot Object Detection on Aerial Images", Soumis à Elsevier Pattern Recognition journal.
    \item[\faFileTextO] P. Le Jeune and A. Mokraoui, "Cross-Scale Query-Support Alignment Approach for Small Object Detection in the Few-Shot Regime", Accepté à the IEEE International Conference on Image Processing 2023 (ICIP).
\end{itemize}
}

\noindent
\textbf{\cref{chap:diffusion} : FSDiffusionDet: un détecteur \textit{few-shot} basé sur les modèles de diffusion et une stratégie de \textit{fine-tuning}}\\
Les chapitres précédents explorent la détection d'objets \textit{few-shot} en
utilisant l'apprentissage métrique et les techniques basées sur l'attention. Ce
chapitre se concentre logiquement sur la dernière grande approche pour le FSOD :
le \textit{fine-tuning}. En nous basant sur DiffusionDet, un récent modèle de
détection utilisant des modèles de diffusion, nous construisons une stratégie de
\textit{fine-tuning} simple et efficace, baptisée FSDiffusionDet. FSDiffusionDet
surpasse état de l'art en FSOD sur des jeux de données aériens et obtient des
performances compétitives sur les images naturelles. Des études expérimentales
approfondies explorent les choix de conception de la stratégie de
\textit{fine-tuning} afin de mieux comprendre les composantes clés nécessaires
pour atteindre une telle qualité. Enfin, ces résultats impressionnants
permettent de considérer des scénarios plus complexes comme l'adaptation à de
nouveaux domaines, ce qui est particulièrement pertinent pour COSE.

\textit{\textbf{Contributions liées à ce chapitre} : Ce chapitre décrit des
travaux très récents et nous planifions de soumettre des articles de recherche
qui les présenterons.}

\subsection*{\cref{part:iou}: Repenser l'\textit{Intersection over Union}} 

Cette partie ne contient qu'un seul chapitre qui présente une contribution
indépendante des approches proposées dans la partie précédente. Ce chapitre
remet en question la pertinence de l'\textit{Intersection sur Union} (IoU), un élément
clé des modèles de détection d'objets.

\noindent
\textbf{Partie \ref{chap:siou_metric}: \textit{Intersection over Union} adaptable à la taille des objets}\\
L'\textit{Intersection sur Union} (IoU) n'est pas une mesure de similarité de
boîte englobante optimale pour l'évaluation et l'entraînement des détecteurs
d'objets. Pour l'évaluation, elle est trop stricte avec les petits objets et ne
correspond pas bien à la perception humaine. Pour l'entraînement, elle crée un
déséquilibre entre les petits et grands objets souvent au détriment des petits.
Nous proposons l'Intersection sur Union adaptative à l'échelle (appelée SIoU),
une alternative paramétrable qui résout les lacunes de l'IoU. Des arguments
empiriques et théoriques sont avancés pour démontrer la supériorité de la SIoU
grâce à une analyse approfondie de celle-ci et d'autres critères existants.

\textit{\textbf{Contributions liées à ce chapitre} :
\begin{itemize}[nolistsep,topsep=2pt]
    \item[\faPaperPlaneO] P. Le Jeune and A. Mokraoui, "Rethinking Intersection Over Union for Small Object Detection in Few-Shot Regime", Soumis à International Conference on Computer Vision 2023 (ICCV).
    \item[\faFileTextO] P. Le Jeune and A. Mokraoui, "Extension de l'\textit{Intersection over Union} pour améliorer la détection d'objets de petite taille en régime d'apprentissage few-shot", GRETSI 2023, XXIXème Colloque Francophone de Traitement du Signal et des Images, Grenoble, France.
\end{itemize}
}

\subsection*{Partie \ref{part:application}: Prototypage et déploiement industriel}

Enfin, la dernière partie de cette thèse présente nos contributions
industrielles. Cette partie est cruciale pour COSE car elle comble l'écart entre
les avancées de la recherche et les applications du monde réel. Par conséquent,
le seul chapitre de cette partie aborde les aspects techniques de la détection
d'objets et n'est associé à aucune contribution académique.

\noindent
\textbf{\cref{chap:integration}: Intégration dans les prototypes de COSE}\\
Les modèles de détection sont souvent lourds et ne conviennent pas bien à
l'application de COSE. Ce chapitre présente d'abord en détail le
système CAMELEON et ses contraintes. Ensuite, nous étudions l'influence de la
taille du modèle sur ses performances et présentons des outils et des astuces
utiles pour accélérer l'inférence. Enfin, nous expliquons comment les modèles de
détection sont déployés à l'intérieur du prototype CAMELEON et comment ils se
comportent sur des images aériennes.

\section{Résumé des contributions}
\vspace{-1em}

\textbf{Articles de conférences internationales}
\begin{itemize}[nolistsep,topsep=2pt]
    \item[\faFileTextO] P. L. Jeune, M. Lebbah, A. Mokraoui and H. Azzag, "Experience feedback using Representation Learning for Few-Shot Object Detection on Aerial Images," 2021 20th IEEE International Conference on Machine Learning and Applications (ICMLA), Pasadena, CA, USA, 2021, pp. 662-667, doi: 10.1109/ICMLA52953.2021.00110.
    \item[\faFileTextO] P. Le Jeune and A. Mokraoui, "Improving Few-Shot Object Detection through a Performance Analysis on Aerial and Natural Images," 2022 30th European Signal Processing Conference (EUSIPCO), Belgrade, Serbia, 2022, pp. 513-517, doi: 10.23919/EUSIPCO55093.2022.9909878.
    \item[\faFileTextO] P. Le Jeune and A. Mokraoui, "Cross-Scale Query-Support Alignment Approach for Small Object Detection in the Few-Shot Regime", Accepté à the IEEE International Conference on Image Processing 2023 (ICIP).
\end{itemize}

\textbf{Articles de conférences nationales}

\begin{itemize}[nolistsep,topsep=2pt]
    \item[\faFileTextO] P. Le Jeune and A. Mokraoui, "Amélioration de la détection d’objets few-shot à travers une analyse de performances sur des images aériennes et naturelles." GRETSI 2022, XXVIIIème Colloque Francophone de Traitement du Signal et des Images, Nancy, France.
    \item[\faFileTextO] P. Le Jeune and A. Mokraoui, "Extension de l'\textit{Intersection over Union} pour améliorer la détection d'objets de petite taille en régime d'apprentissage few-shot", GRETSI 2023, XXIXème Colloque Francophone de Traitement du Signal et des Images, Grenoble, France.
\end{itemize}

\textbf{Articles soumis}

\begin{itemize}[nolistsep,topsep=2pt]
    \item[\faPaperPlaneO] P. Le Jeune and A. Mokraoui, "A Comparative Attention Framework for Better Few-Shot Object Detection on Aerial Images", Soumis à Elsevier Pattern Recognition journal.
    \item[\faPaperPlaneO] P. Le Jeune and A. Mokraoui, "Rethinking Intersection Over Union for Small Object Detection in Few-Shot Regime", Soumis à International Conference on Computer Vision 2023 (ICCV).
\end{itemize}

\textbf{Présentations orales}\\
Au cours de cette thèse, j'ai eu l'opportunité de donner plusieurs présentations orales :
\begin{itemize}[nolistsep,topsep=2pt]
    \item[-] Journée scientifique du L2TI (Déc. 2020).
    \item[-] \textit{Prototypical Faster R-CNN for Few-Shot Object Detection on
    Aerial Images} DeepLearn Summer School 2021, Las Palmas de Gran Canaria
    (29 Juil. 2021).
    \item[-] Prototypical Faster R-CNN for Few-Shot Object Detection on Aerial
    Images à la journée du GDR-ISIS : \textit{Vers un apprentissage pragmatique dans
    un contexte de données visuelles labellisées limitées}, Paris, (26 Nov. 2021).
    \item[-] Séminaires des doctorants (Mar. 2022 and Fév. 2023).
    \item[-] Séminaire à l'ETS Montreal: \textit{Few-Shot Object Detection on Aerial Images} (28 Sep. 2022). 
\end{itemize}

\textbf{Supervision de stages}\\
J'ai supervisé 4 stages au cours de ces trois ans de thèse, 3 au sein de l'entreprise et un au L2TI :
\begin{itemize}[nolistsep,topsep=2pt]
    \item[-] Conception et mise en oeuvre d’algorithmes de suivi d’objets dans des images
    aériennes (Mars-Août 2021 -- COSE).
    \item[-] Optimisation et intégration d’algorithmes de détection d’objets dans un système
    embarqué (Mars-Août 2022 -- COSE).
    \item[-] Apprentissage auto-supervisée pour la détection d'objets few-shot (Avril-Août 2022 – L2TI au
    travers du LabCom IRISER).
    \item[-] Détection d'objets few-shot par visual transformers sur des images
    aériennes (Mars-Août 2023 -- COSE et L2TI via le LabCom IRISER). 
\end{itemize}

En plus des supervisions de deux stages, je suis activement impliqué dans le
LabCom IRISER\footnote{\href{https://www-l2ti.univ-paris13.fr/iriser/}{Lien vers
le site du LabCom IRISER}}, un laboratoire commun entre COSE, le L2TI et le
LIPN. Ce laboratoire commun a été crée un an environ après le début de ma thèse
sous l'impulsion de mes superviseurs académique et industriel. 

\textbf{Logiciels libres}\\
Au travers des différents projets qui ont constitués cette thèse, j'ai développé
plusieurs package Python qui se trouvent en accès libre sur GitHub :  
\begin{itemize}
    \item[\faGithub] \href{https://github.com/pierlj/proto_faster_rcnn}{Prototypical Faster R-CNN}
    \item[\faGithub] \href{https://github.com/pierlj/aaf_framework}{AAF framework}
    \item[\faGithub] \href{https://github.com/pierlj/pycocosiou}{Pycocosiou}
    \item[\faGithub] \href{https://github.com/pierlj/fs_diffusiondet}{FSDiffusionDet}
\end{itemize}


\part{Literature Review on Object Detection, \\Few-Shot Learning, and Few-Shot Object~Detection}%
\label{part:literature}
\chapter{Object Detection, Few-Shot Learning and Cross-Domain Adaptation}
\label{chap:od}


\chapabstract{Object Detection and Few-Shot Learning are two relevant challenges
from the Computer Vision and Machine Learning fields. Both are necessary to
build detection techniques able to generalize from limited data. Hence, this
chapter reviews both Object Detection and Few-Shot Learning. The two problems
are defined and detailed reviews of the respective literature are conducted.} 

\PartialToC

As briefly presented in the introduction, this PhD project lies at the
intersection of three sub-domains of Machine Learning: Object Detection (OD),
Few-Shot Learning (FSL), and Cross-Domain Adaptation (CDA). In this chapter, we
define more precisely what these three fields are and outline the main existing
contributions in the literature. We start by introducing the
main computer vision problems these fields address and the related notations
adopted in this manuscript. Then, we provide a review of existing
works in each area, and finally, we link them with the industrial needs of COSE.  

\section{Object Detection}
\label{sec:review_od}
\vspace{-0.5em}
\subsection{Problem Definition}
\label{sec:od_problem}
Object detection consists in localizing and classifying all objects of interest
visible in an image $I$. There are multiple terms to explain this statement.
First, the notion of the object of interest is defined according to a fixed set
of semantic classes $\mathcal{C}$. The objects of interest are the ones that
belong to one class $c \in \mathcal{C}$. Of course, one can question the
belonging of an object to a class. A class could be ambiguous for multiple
reasons. Given the quality of the image, it can be difficult to determine the
class of the object depicted. For instance, in a satellite photograph, a car
could be so small that it cannot be perceived by a human observer. Another issue
is the slackness of our concept of objects, one word can refer to multiple
objects (\eg spring, game, or chest), and our vocabulary is organized
hierarchically (\eg the class \textit{vehicle} contains many classes including
\textit{car} and \textit{truck}). One could go even further by questioning the
very concept of objects in our mind (see for instance exemplar-based vs.
prototype-based concept theories), but it would have more to do with cognitive
science than computer vision. Generally, for object detection, these
complications are not considered. One object can only belong to one class and
whether it belongs to the class or not falls under the common sense of the
observer. Most of the time, this is established with the ground truth
annotations of human experts. This explanation is rather obvious, but keep in
mind that this is a simplification, this will be useful when the notion of an
object gets blurrier in the case of few-shot learning and few-shot object
detection.

The detection task consists in finding all occurrences of the objects of
interest in the images, \ie the image coordinates of each object. This can be
done in several ways, by locating a single pixel inside each object, by drawing
a rectangular bounding box, or by computing a precise mask around it. The former
setting is barely used as it can be quite ambiguous, all points in an object are
equivalent and no supplementary information (size and shape) can be inferred
from this representation. Traditionally, object detection leverages bounding
boxes to localize the objects, and precise masks are reserved for the Instance
Segmentation task. A bounding box is generally determined by four coordinates,
it can be the coordinates of two diagonally opposed points on the box or the
coordinates of one point plus the width and height of the box, see
\cref{fig:od_bbox_form} for more details about boxes representations. In the
following, we adopt the latter definition of a bounding box: the first two
coordinates denote the $x$ and $y$ image coordinates of the top-left corner of
the box, while the last two represent the width $w$ and height $h$ of the box. A
bounding box $b$ is then denoted as follows:%
\begin{equation}
    b = [x, y, w, h]^T.
\end{equation}
\begin{figure}
    \centering
    \includegraphics[width=\textwidth]{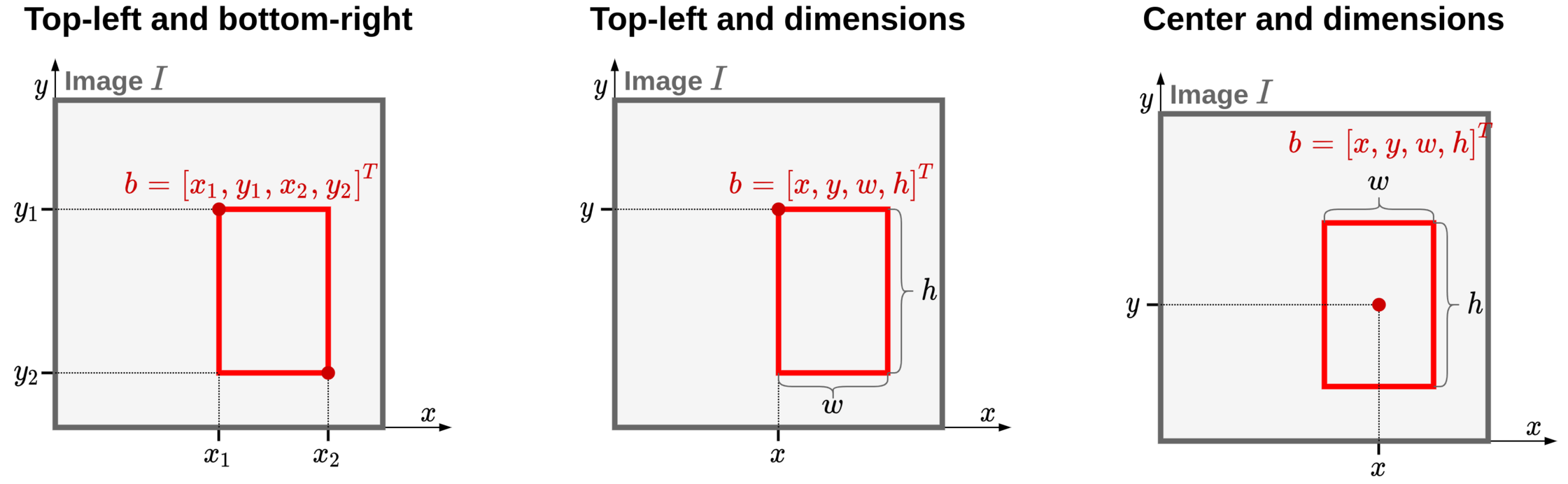}
    \caption[Bounding boxes representations]{Three different box representations: top-left and bottom-right
    points, top-left point, width and height, and center point width and height.
    Many more representations exist but they will not be presented here. In this
    manuscript, the second representation, top-left point width and height, will
    be used exclusively.}
    \label{fig:od_bbox_form}
\end{figure}%
As the goal of object detection is to localize \textit{and} classify the objects of
interest, each bounding box must be associated with a class $c \in \mathcal{C}$.
Therefore, we define the detection label as a tuple of a bounding box
and a class label:%
\begin{equation}
    \boldsymbol{\mathrm{y}} = (b, c).
\end{equation}
An image may contain more than one object and therefore, each image $I$ is
associated with a set of detection labels $\mathcal{Y} = \{\boldsymbol{\mathrm{y}}_i\}_{i=1}^{N_I}$,
where $N_I$ is the number of objects in image $I$. Hence, solving an object
detection task is to find a detector model $\mathcal{F}(\cdot, \theta)$ with
parameters $\theta$ able to output a set of predicted labels given an input
image $I$: %
\begin{equation}
    \mathcal{F}(I, \theta) = \hat{\mathcal{Y}} = \{\hat{\boldsymbol{\mathrm{y}}}_i\}_{i=1}^{M_I}.
\end{equation}
We employ here a slight abuse of notation by calling the output of a detector
$\boldsymbol{\mathrm{y}}$. Indeed, a detector predicts a classification score
$l^c$ for each class $c \in \mathcal{C}$. Therefore,
$\hat{\boldsymbol{\mathrm{y}}} = (\hat{b}, \{\hat{l}^c\}_{c \in \mathcal{C}})$
and $\hat{c} = \argmax_{c \in \mathcal{C}} l^c$. For convenience, we denote $l
\in [0,1]^{|\mathcal{C}|}$ the vector of classification scores. Hatted symbols
represent the model's outputs. Note that the number of detections found by the
model $M_I$ may not equal the number of objects present in the image as the
detector can either miss some objects or output false detections. The proximity
between the predicted labels $\hat{\mathcal{Y}}$ and the ground truth labels
$\mathcal{Y}$ determines the performance of the model $\mathcal{F}(\cdot,
\theta)$. Hence, finding an optimal detection model is to find a set of optimal
parameters, which minimizes the distance between predicted and ground truth
labels: 
\useshortskip
\begin{equation}
    \theta^* = \argmin_{\theta} d(\hat{\mathcal{Y}}_{\theta}, \mathcal{Y}),
\end{equation}
where $d$ is a distance measure between $\hat{\mathcal{Y}}_{\theta}$, the labels
predicted by $\mathcal{F}(\cdot, \theta)$ and the ground truth labels
$\mathcal{Y}$. Of course, there are plenty of valid approaches to
measure the proximity between two sets of detection labels, some will be
introduced in \cref{sec:od_evaluation}.

For COSE, detecting the objects of interest in an image is a crucial step. This
step is necessary for the creation of GEOINTs. From the bounding boxes
coordinates in the image, and the carrier position (latitude, longitude, and
altitude) and attitude (pitch, roll, and yawn), one can determine the precise
locations of the objects on Earth. These computations also involve camera
properties and orientations, but they will not be addressed in this manuscript.

\subsection{Evaluation of Object Detectors}
\label{sec:od_evaluation}
Before jumping into the Object Detection literature, let's introduce the most
commonly used metrics employed to assess the quality of the detection models. As
mentioned in the previous section assessing the detection performance of a model
consists in comparing the set of predicted detection labels $\hat{\mathcal{Y}}$
with the set of ground truth labels $\mathcal{Y}$ (typically made by a human
observer). In the previous section, we defined the set of detection labels over
an image. However, to better assess the generalization capabilities of the
detectors, their evaluation is always conducted on a relatively large
\textit{test set} of images. Therefore, we extend the notation
$\hat{\mathcal{Y}}$ and $\mathcal{Y}$ as the sets of predicted and ground truth
labels (respectively) over all test images. 

\subsubsection{Average Precision and mean Average Precision}
The most commonly used metrics for Object Detection are the Average Precision
(AP) and its extension in the multiclass setting, the mean Average Precision
(mAP). The AP is formally defined as the area under the precision-recall curve:
\begin{equation}
    \text{AP} = \int_{0}^{1} \text{Prec}(r) \,dr , 
\end{equation}
where $\text{Prec}$ denotes the Precision and $r$, the Recall. The Precision and
Recall are two well-known metrics often used to evaluate classification
problems. They are defined respectively as the ratio of true positive labels
over the positive predicted labels and the ratio of the positive labels over the
positive true labels. \Cref{fig:od_precision_recall} clearly illustrates these
definitions with a classification confusion matrix. Note that the notion of True
Positive (TP), True Negative (TN), False Positive (FP) and False Negative (FN)
introduced in the figure will be leveraged throughout the manuscript.

\begin{figure}
    \centering
    \begin{subfigure}[c]{0.75\textwidth}
        \includegraphics[width=\textwidth]{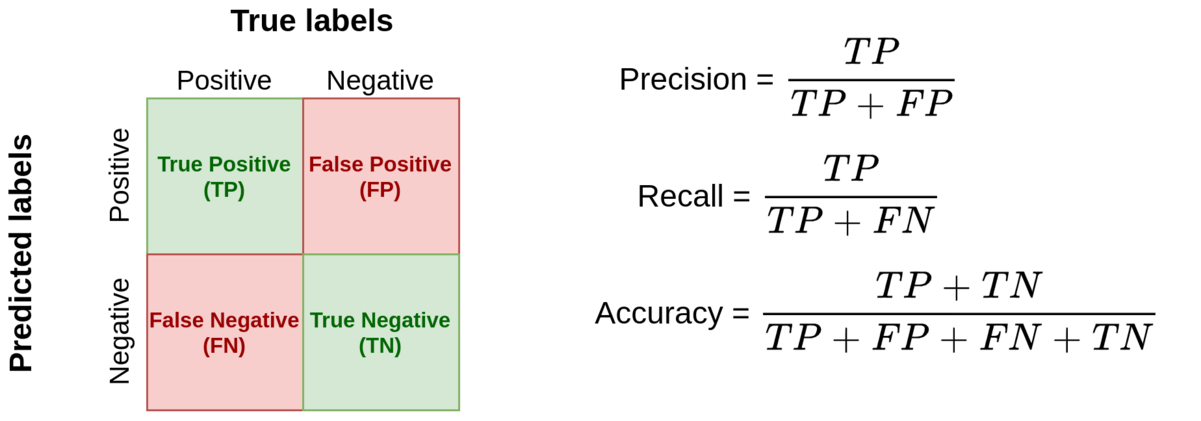} 
        \caption{Classification confusion matrix and definition of the Precision, Recall and Accuracy metrics.}
        \label{fig:od_precision_recall}
    \end{subfigure}
    \hfill
    \begin{subfigure}[c]{0.2\textwidth}
        \includegraphics[width=\textwidth]{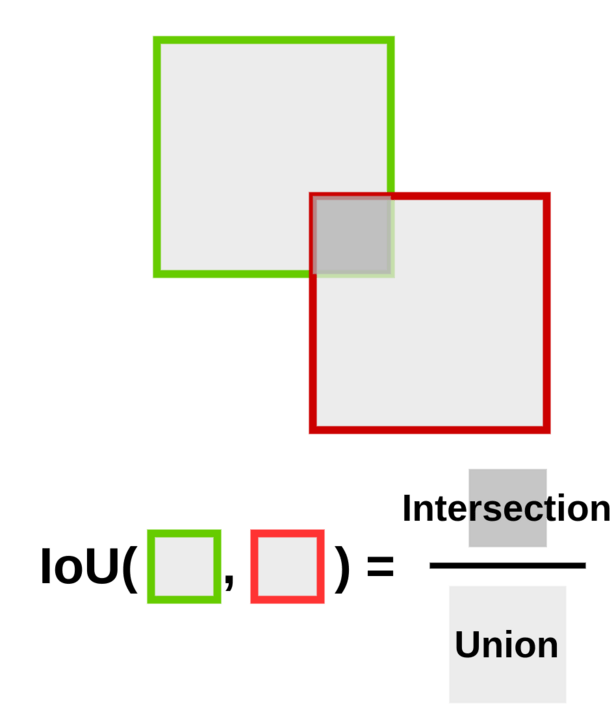} 
        \caption{IoU for rectangular bounding boxes.}
        \label{fig:od_IoU}
    \end{subfigure}
    \caption[Precision, Recall, Accuracy and IoU definitions]{Definition of the
    Precision, Recall and Accuracy metrics (a) as well as the box similarity
    criterion Intersection over Union (b).}
    \label{fig:od_IoU_PR}
\end{figure}

For now, forget about the classification part of the detection and consider only
the localization problem. How can we define the four corners of the confusion
matrix for bounding boxes? An approach is to consider a detection as TP if the
predicted bounding box has the same coordinates as one true detection label.
However, it is extremely challenging to get a perfect positioning of the
predicted boxes. First, the detector generally leverages regression techniques
to predict box coordinates and outputs continuous box coordinates. This is
incompatible with the annotated box coordinates, which are usually discrete.
Rounding errors can cause TP to become FP. Then, from an application viewpoint,
pixel-perfect bounding boxes are not necessary. Therefore, it is generally
admitted that a true positive detection is a box close enough to a ground truth
box. The similarity between two bounding boxes is almost always computed with
the Intersection over Union (IoU). Then, TPs are the boxes that have an IoU with
a true box above a fixed threshold (typically 0.5). However, the IoU may not be an
optimal criterion in certain cases as we will discuss in \cref{chap:siou_metric}.

The Intersection over Union, also known as the Jaccard index, is a well-known
similarity measure between sets $A$ and $B$: 
\begin{equation}
    \text{IoU}(A, B) = \frac{|A \cap B|}{|A \cup B|}.
\end{equation}
Besides its application in statistics, the IoU is widely used in computer vision to
assess the quality of visual tasks such as detection and segmentation. The IoU
can compute how close two sets of pixels are and thus gives a similarity measure
between ground truth and the model prediction. Here, we focus on the detection
task, therefore the IoU can be written in terms of coordinates of the boxes $b_1
=[x_1, y_1, w_1, h_1]^T$ and $b_2 =[x_2, y_2, w_2, h_2]^T$: 
\begin{align}
    \label{eq:iou_bbox}
    &\begin{aligned}
        \mathcal{A}_{\text{inter}} = &\max\big(0, \max(x_1, x_2) - \min(x_1 + w_1, x_2 + w_2)\big)\\
          &\qquad  * \max\big(0, \max(y_1, y_2) - \min(y_1 + h_1, y_2 + h_2)\big),
      \end{aligned} \\
    &\text{IoU}(b_1, b_2) = \frac{\mathcal{A}_{\text{inter}}}{w_1h_1 + w_2h_2 - \mathcal{A}_{\text{inter}}}.
\end{align}

The IoU is, therefore, a crucial part of the evaluation protocol of the object
detectors as it conditions which predicted bounding boxes are TP, and which are
FP. The IoU threshold limit determines how close to the ground truth the
predicted boxes must be to be considered TP. In the Pascal VOC detection
challenge \cite{everingham2010pascal}, this threshold was set to 0.5, this has
been the gold standard for a few years. However, it changed when the more
challenging Microsoft COCO dataset \cite{lin2014microsoft} was proposed. The
authors of the COCO challenge compute the AP at several thresholds (ranging from
0.5 to 0.95) and average the values. While this is the current standard for
object detection evaluation, the few-shot object detection literature still uses
the former Pascal VOC AP as it is an easier metric. Hence, we will use both of
these metrics in the manuscript. We will denote them as $\text{AP}_{0.5}$ and
$\text{AP}_{0.5:0.95}$ respectively. 

Before computing the precision-recall curve, the predicted labels must be ranked
by confidence scores. It is usually possible to derive a confidence score along
with the bounding box coordinates and labels from a detection model. This can be
for instance the highest class probability score. Once the detections are ranked
according to confidence scores, one can compute the running precision
$\text{Prec}_k$ and recall $\text{r}_k$ by taking only the top-$k$ bounding
boxes. Then, it is possible to plot the precision as a function of the recall by
plotting the points $(\text{r}_k, \text{Prec}_k)$, for $1 \leq k \leq
|\hat{\mathcal{Y}}|$. This generally gives a zig-zag shaped curve
$\text{Prec}(\text{r})$ as visible in \cref{fig:od_pr_curve}. Therefore, it may
not be easy to compute the area under the curve, \ie the AP. A few tricks were
introduced in \cite{salton1983introduction}, and later popularized with the
Pascal VOC challenge \cite{everingham2010pascal}. They consist in taking an
interpolated precision-recall curve:
\begin{equation}
    \text{Prec}_{\text{interpolated}}(\text{r}) = \max_{\tilde{r} \geq r} \, \text{Prec}(\tilde{r}), 
\end{equation} 

and discretize the area computation over 11 equally spaced points along the
recall axis. Hence, the original AP definition becomes:
\begin{equation}
    \text{AP} = \frac{1}{11} \sum\limits_{i=0}^{10} \text{Prec}_{\text{interpolated}}(0.1 \times i).
\end{equation}

So far, we only discuss the evaluation without taking into account the class of
the bounding boxes. In order to take this into account, the AP is computed
independently for each class and noted $\text{AP}_c$. The predicted boxes are
now considered true positive only if they have a sufficient IoU with a ground
truth box and if they have the same class. The mean Average Precision is defined
as follows:
\begin{equation}
    \text{mAP} = \frac{1}{|\mathcal{C}|} \sum\limits_{c \in \mathcal{C}} \text{AP}_c.
\end{equation}

The mAP is largely the most employed metric for object detection and most of our
analysis will be based on it. However, there exist complementary metrics that
are worth presenting here.

\subsubsection{mean Average Precision per object size}
The AP and mAP can be computed only on certain object sizes. The principle is
simple, simply filter the sets of predicted and true labels to keep boxes of a
certain size before the AP computation. This distinction was introduced in the
COCO challenge \cite{lin2014microsoft}, with three different object sizes: 
\begin{itemize}[nolistsep]
    \item[-] \textbf{Small}: boxes with an area $a=wh$ smaller than $a \leq 32^2$ pixels.
    \item[-] \textbf{Medium}: boxes with $32^2 < a \leq 96^2$.
    \item[-] \textbf{Large}: boxes with $a > 96^2$.  
\end{itemize}
\begin{figure}
    \centering
    \begin{subfigure}[c]{0.49\textwidth}
        \includegraphics[width=\textwidth]{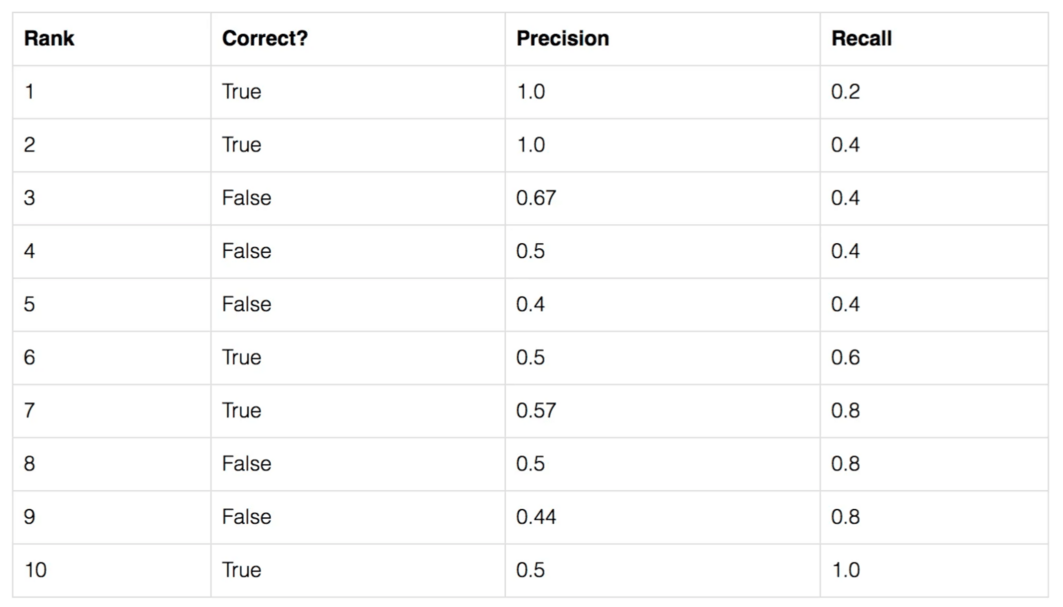}
    \end{subfigure}
    \hfill
    \begin{subfigure}[c]{0.49\textwidth}
        \includegraphics[width=\textwidth]{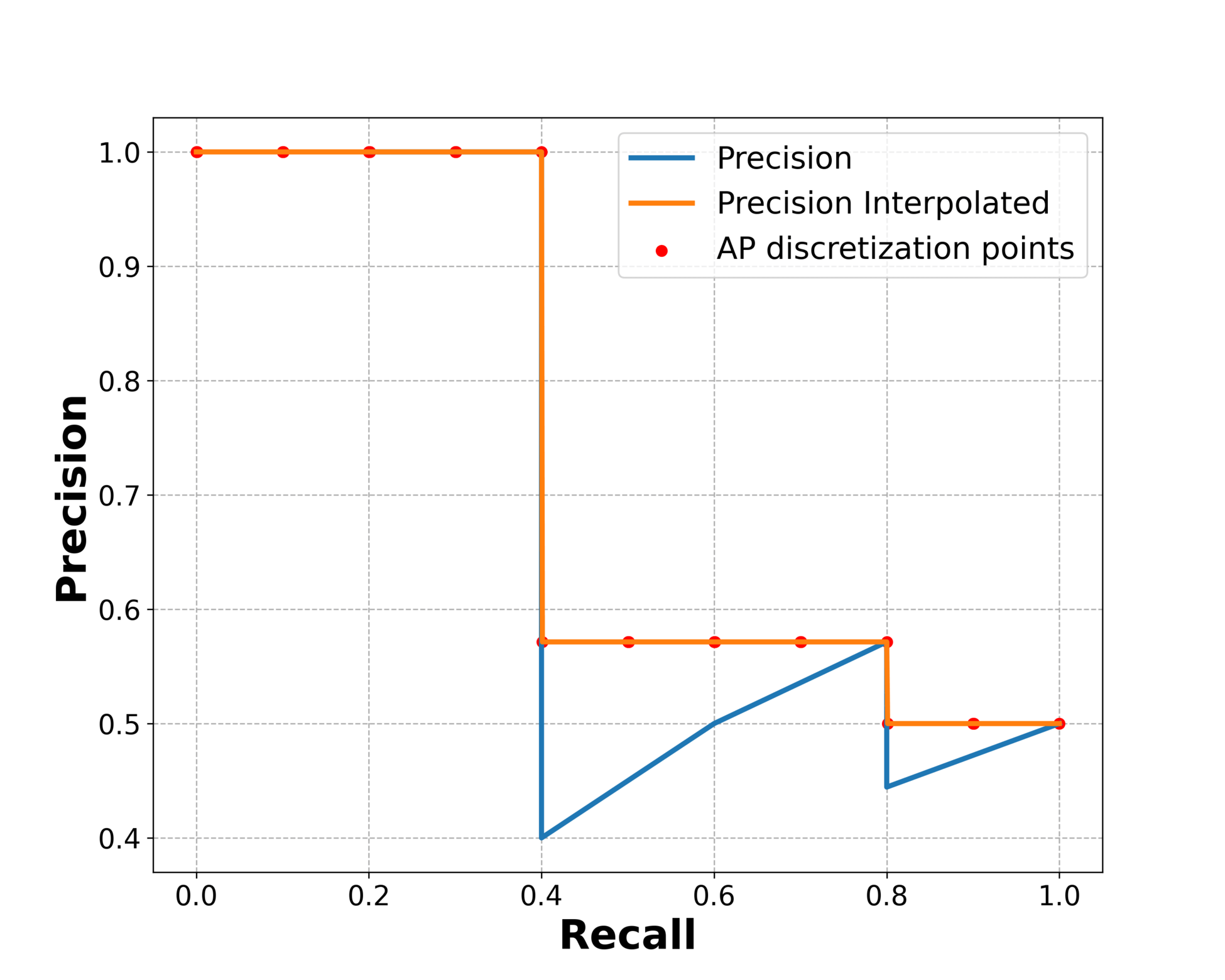}

    \end{subfigure}
    \caption[Precision-Recall curve computation]{Example of a Precision-Recall curve and its interpolated variant.}
    \label{fig:od_pr_curve}
\end{figure}

This distinction is extremely relevant for COSE applications as the objects of
interest in aerial images are often small and because most detectors struggle to
detect them. Even worse, this issue is reinforced in the few-shot regime as we
will see in the following chapters. Therefore, $\text{mAP}^S$, $\text{mAP}^M$
and $\text{mAP}^L$ will be extensively employed in our analysis. 

\subsubsection{Average Recall}

Similar to the average precision, the Average Recall (AR) is computed as the
area under the Recall-IoU curve: 
\begin{equation}
    \text{AR} = 2 \int_{0.5}^1 \text{Recall}(\upsilon) \,d\upsilon, 
\end{equation}

where $\upsilon$ denotes the IoU threshold (ranging from 0.5 to 1) for the
recall computation. Similarly to the AP for different object sizes, the average
recall can be declined according to the maximum number of detection allowed.
Basically, this controls the size of $\hat{\mathcal{Y}}$, the more detection the
model can output, the less likely it is to miss an object of interest. Although
it can be a critical metric in some applications (\eg lesion detection), this is
not essential for COSE's applications.

\subsubsection{Average Precision shortcomings and Alternative}
Even though AP is widely used in the computer vision community, it has several
shortcomings. Similar values of AP can be obtained from very different
precision-recall curves, hiding different detectors' behavior. The ranking of
the confidence scores makes the AP sensitive to the prediction confidence.
Finally, the interpolation trick from \cite{salton1983introduction} may cause
large errors when the number of instances of the class is small. These drawbacks
were highlighted in \cite{oksuz2018localization}, which proposes an alternative
metric, the Localization-Recall-Precision (LRP). This metric is an aggregation
of three metrics based on the box regression error, the precision, and the
recall, under a certain confidence threshold. Hence, LRP fixes some of the AP's
shortcomings.

More recently, \cite{jena2023beyond} also outlined two detection issues that are
not spotted by AP, namely \textit{spatial hedging} and \textit{category
hedging}. Spatial hedging comes from the fact that low-confidence duplicates
(slightly perturbated spatially) of a box do not degrade the AP value, instead
having a lot of these duplicates generally improves the AP. However, these
duplicates are mostly burdensome from an application viewpoint. The authors even
highlight some tricks in recent object detectors that boost AP while increasing
the number of duplicates. Category hedging comes from duplicated boxes with
different classes. Consequently, the authors proposed two novel metrics to
specifically assess spatial and category hedging: Duplicate Confusion (DC) and
Naming Error (NE). Note that LRP partly assesses spatial hedging as well.

\subsection{Literature review about Object Detection}
\label{sec:od_literature}

In this section, we review the Object Detection literature but only present the
most ground-breaking works. For an exhaustive review of object detection, we
defer the reader to two popular surveys \cite{zaidi2022survey,zou2023object}.
This section is divided into three parts, traditional object detection,
CNN-based OD and Transformer based OD. These correspond to three phases in the
development of object detection techniques. This is highlighted by the timeline
in \cref{fig:od_timeline}, which summarizes the history of the object detection
field.

\begin{figure}[h]
    \centering
    \includegraphics[width=\textwidth]{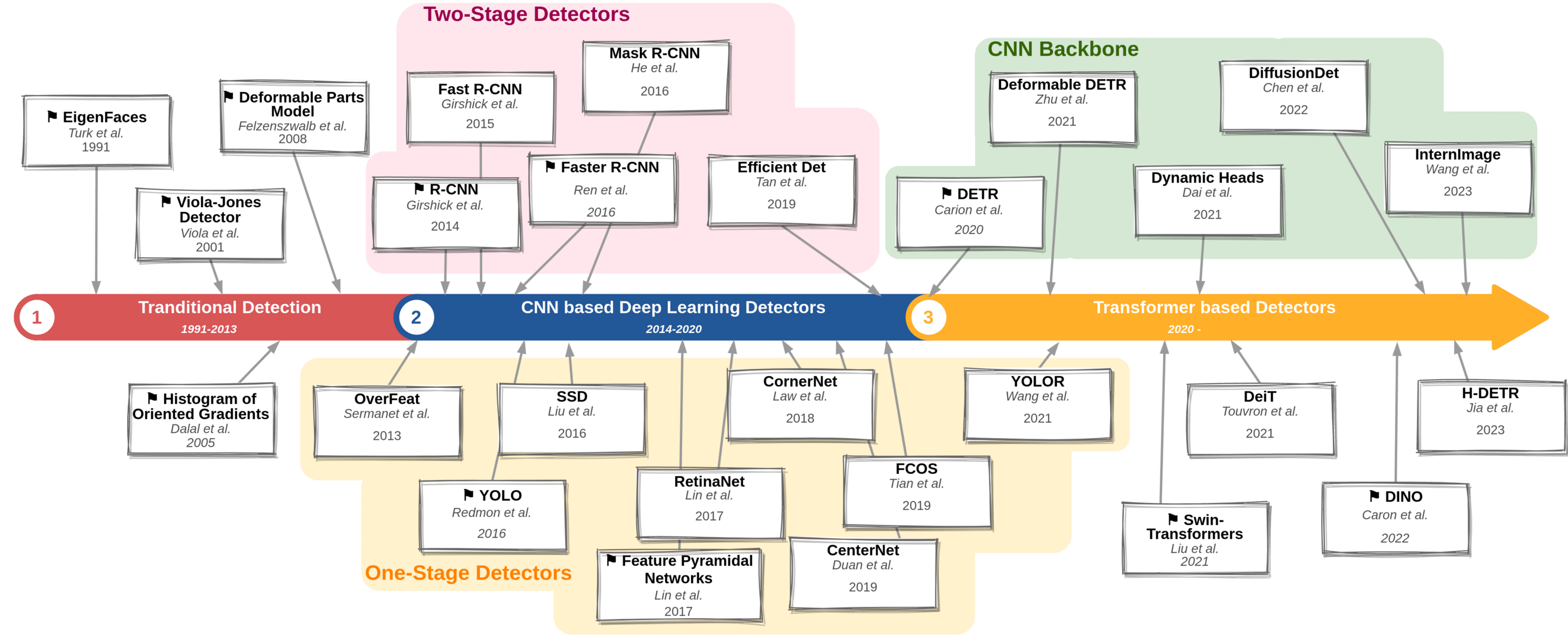}
    \caption[Object Detection timeline]{Timeline of the Object Detection literature, it includes some of
    the most relevant works in the field of Object Detection. Papers marked with a flag are
    the most ground-breaking works, some will be detailed in \cref{sec:od_literature}}
    \label{fig:od_timeline}
\end{figure}

\subsubsection{Traditional Object Detection Approaches}
\label{sec:od_traditional}

The very beginnings of the object detection field date back to the early 1990s.
It began with an easier one-class problem: face detection. Of course, there were
prior works addressing this task, but they mainly focused on face recognition
and not detection. The difference is slight, the recognition task only asks
whether there is a face or not in an image. This field gained substantial
interest over the 1970s and 1980s with seminal works such as
\cite{goldstein1971identification,goldstein1972man,kaya1972basic,Yuille1989FeatureEF}.
However, it was only in 1991 that the detection task was first tackled by
EigenFaces \cite{turk1991eigenfaces}. In this work, the authors perform a
Principal Component Analysis (PCA) on a set of face images. The PCA
returns a set of eigenvectors (denoted as \textit{EigenFaces}) that span the
face space. Applying the EigenFaces on a sliding window over the images
allows creating \textit{faceness} maps and therefore localizing faces. Following
this pioneer contribution, many face detectors were proposed in the 1990s, for
instance
\cite{bichsel1994human,sung1998example,vaillant1994original,rowley1998neural}.
We called this section "Traditional Object Detection Approaches" in contrast
with the two following sections that are deep learning-based approaches.
However, note that a significant proportion of the methods developed during the
1990s actually leverage neural networks. In EigenFaces
\cite{turk1991eigenfaces}, for instance, the authors discuss an implementation
of their system using Multi-Layer Perceptron (MLP) for fast parallel
computation. Similarly, \cite{vaillant1994original,rowley1998neural} exploit
and train neural networks for the task of face detection. We will see in the next
section that plenty of these ideas will be re-used 20 years later by current
deep-learning-based detectors. The missing pieces in these early days of object
detection were large annotated datasets and dedicated hardware such as Graphical
Processing Units (GPUs).

In the late 1990s, the object detection task as we know it today with multiple
classes of interest was still far off, but some research groups began to apply
it to objects other than faces. As an example, \cite{moghaddam1995probabilistic}
introduced a general probabilistic model for the object of interest used for the
visual search of faces and hands. Later, \cite{papageorgiou1998general} stepped
aside from the PCA-like object representations and proposed generic learnable
features based on Haar wavelets transform. This was successfully applied to
pedestrian detection. An extension with slightly more elaborated features
proposed the Viola-Jones detector \cite{viola2001rapid}. But it mainly provided
several tricks for fast computation, achieving robust real-time detection. 

In the 2000s, plenty of works shared the same strategy as the Viola-Jones
detector: learning a set of elaborated features and classifying regions in the
input images by comparison with the set of features. Improvements were made
using more and more complex feature sets
\cite{ronfard2002learning,mikolajczyk2004human} and the popularization of
Support Vector Machines (SVM) classifiers\cite{papageorgiou2000trainable}. This
strategy led to the well-known Histogram of Orient Gradient (HOG)
\cite{dalal2005histograms}, which was first applied to pedestrian detection.
However, this method was one of the first to tackle multi-class object detection
in the first Pascal VOC challenge \cite{everingham2010pascal} in 2005. 

The Pascal VOC challenge quickly became a reference in the Object Detection
field, with increasing difficulty over the years (more classes and more
images). The winners of the following editions 2006, 2007 and 2008 all took
inspiration from the HOG features. In particular,
\cite{felzenszwalb2008discriminatively} employs several tricks to improve the
detection quality based on HOG features. Among those, Deformable Part Models (DPM),
\ie representing each object as a set of its parts provide significant
improvements. It also leverages pyramid features and hard examples mining which
are common components of recent object detectors. DPM were then
refined with for instance the Grammar Models \cite{girshick2011object} and
Star Models \cite{girshick2011object}.

\subsubsection{Object Detection in the Deep Learning Era}
\label{sec:od_dl_era}

While there were a few attempts to solve object detection with neural networks
during the 1990s, all were limited to single-class problems and lagged behind
state-of-the-art in terms of detection quality and speed. However, this changed
with the popularization of fast parallel processing units (GPUs) and the
creation of large image datasets. In 2012, AlexNet \cite{krizhevsky2017imagenet}
was introduced for image classification with deep convolutional networks (CNNs).
Since AlexNet, deep learning was successfully applied to most tasks in computer
vision including Object Detection.  

From the beginning of the deep object detection era, two schools of thought
emerged: one-stage detectors and two-stage detectors. As two-stage detectors
were proposed first, we will present them first here.

\paragraph*{Two-Stage Detectors}
Regions with CNN features or R-CNN \cite{girshick2014rich} is one of the first
attempts to tackle the task of Object Detection with CNNs. This marks a
significant performance improvement over the previous methods (about 20\% mAP
improvement over the best DPMs on the 2010 Pascal VOC challenge). The idea
behind R-CNN is to leverage the classification power of CNNs such as AlexNet for
the detection task. It first employs Selective Search
\cite{uijlings2013selective}, a class-agnostic object locator, to generate
region proposals. For each region proposal, the corresponding part in the input
image is cropped and fed to a CNN pre-trained on ImageNet. The CNN outputs
high-dimensional feature vectors which are then classified by an SVM for each
class. The CNN is fine-tuned for the detection task by replacing its final
classification layer with N+1 class logits (one additional class for the
background) and training with a Cross-Entropy (CE) loss function. Proposal
regions with an IoU of 0.5 with a ground truth box are selected as
\textit{positive proposals} and the model is trained to classify these regions
with the label of their corresponding ground truth. The other proposals (denoted
\textit{negative proposals}) are selected as background examples. The authors
train classification SVMs instead of using the classification outputs of the CNN
as they observe higher performance with SVMs. In addition, they train a linear
bounding box regressor to refine the coordinates of each region of interest,
following the most recent advances with DPM. OverFeat
\cite{sermanet2013overfeat} was another attempt to solve detection with CNN.
Although it did win the ImageNet Detection Challenge in 2013, it is largely
outperformed by R-CNN and the corresponding paper was never published. Following
R-CNN, the first author Ross Girscick proposed two successive extensions.
Firstly, Fast R-CNN \cite{girshick2015fast}, mainly improves over R-CNN in terms
of speed. It introduces a Region of Interest (RoI) Pooling layer which extracts
RoI features directly from the features maps of the entire image. This is
largely inspired by Spatial Pyramid Pooling \cite{he2015spatial} which consists
in pooling the features of an RoI with multiple binning resolutions and
concatenating the outputs. RoI Pooling saves a lot of unnecessary forward passes
through the CNN (R-CNN performs a forward pass for each RoI). Then, they dropped
the SVM classifiers for the CNN outputs and integrated a bounding box regressor
at the end of the CNN as a parallel branch to the classification layer. The
training is done in a similar fashion as in R-CNN, they simply added a
regression loss function, computed only on the positive proposals to train the
bounding box regression branch. Secondly, Faster R-CNN \cite{ren2015faster}
introduced the Region Proposal Network (RPN) to replace Selective Search as the
proposal generation algorithm. Selective Search provides almost exhaustive
proposals but is slow. The RPN is a lightweight CNN that densely predicts
proposal coordinates and an \textit{objectness} score at each position in the
feature map. The box predictions are done as coordinate shifts from a set of
pre-defined \textit{anchor} boxes. After objectness filtering, this produces a
reasonable number of proposals that can be processed with Fast R-CNN. The RPN is
trained like Fast R-CNN with a similar loss function. A binary classification
loss as the RPN classifies each location between objects or background (with the
objectness score), and a regression loss. Example selection remains unchanged
compared to Fast R-CNN. The training of both the RPN and Fast R-CNN is done in
an end-to-end manner. Faster R-CNN achieves superior performance compared to its
predecessors, but most importantly, it unlocks real-time detection with deep
learning-based two-stage approaches. 

\paragraph*{One-Stage Detectors}
One-stage detectors appeared slightly later than two-stage ones. The reason for
this is probably because two-stage models were the logical continuity of the
sliding-window-based older detectors. These approaches are highly inefficient as
they process parts of the input images many times. While this is reduced in
modern two-stage detectors, they still have redundancies that limit their
inference speed. In 2015, Redmon et al. proposed You Only Look Once (YOLO)
\cite{redmon2016you}, a first detector to avoid all redundancy as it only needs
to look once at each part of the image. The main idea behind YOLO is to
reformulate OD as a regression problem and not a classification one. Prior
detectors solve OD by classifying regions of the input image, \ie classifying
given a region of interest. YOLO instead regresses the box coordinates and
classifies the object jointly. 

The main principle of YOLO is to split the input image into an $S\times S$ grid
and predict bounding boxes, confidence scores and class probabilities for each
cell in the grid. Each cell is "in charge" of detecting objects that are located
within its boundaries. To deal with cases where more than one object is visible
in one cell, $B$ bounding boxes and confidence scores are predicted per cell
($B=2$ in the original paper). To keep model size constrained, the class
probabilities are predicted only once per cell. This assumption limits the model
to predict boxes of one class only per cell. The YOLO architecture is based on a
deep CNN followed by two fully connected layers. The grid separation is directly
implemented inside the architecture since the input size is fixed. YOLO is
trained in an end-to-end fashion with a typical detection loss function. This
loss function includes a regression part for box coordinates and a
classification part, both implemented as L2 loss functions. Just like other
object detectors, YOLO has an example selection strategy to compute its loss.
Each ground truth box is attributed to the cell where its center is located and
then to the box with the highest IoU. Thus, YOLO is extremely fast compared to
the two-stage approaches  (50 to 100 fps depending on the configuration for YOLO
compared to less than 15 fps for Faster R-CNN). However, this speed improvement
comes at the cost of slightly lower detection quality. 

Just like the R-CNN family, YOLO was extended several times by its original
authors and even later perpetuated by other research groups. In YOLOv2
\cite{redmon2017yolo9000}, several improvements are introduced, including a
lighter and fully convolutional architecture, decoupled class probabilities for
each box, anchors boxes as in Faster R-CNN and a hierarchical word structure for
refined classification. It also proposes several tricks and loss function
adjustments to stabilize training. YOLOv2 hence achieves both higher detection
performance and speed. YOLOv3 \cite{redmon2018yolov3} is then introduced in an
unpublished paper by the same authors, presenting several incremental
improvements. Again, it outperforms its previous version both in quality and
speed. After YOLOv3, its authors decided to quit the Object Detection research
for ethical reasons, but many other groups continued to refine the YOLO
framework. The race for the best performance and speed includes numerous
versions of YOLO: YOLOv4 \cite{bochkovskiy2020yolov4}, PP-YOLO
\cite{long2020pp}, YOLOX \cite{yolox2021}, YOLOv6 \cite{li2022yolov6}, YOLOR
\cite{wang2021you} and YOLOv7 \cite{wang2022yolov7}. Each of these works has its
share of marginal changes involving elaborated loss design, structure change,
augmentation techniques and refined anchors generation. Note that YOLOv5
\footnote{https://github.com/ultralytics/yolov5} and
YOLOv8\footnote{https://github.com/ultralytics/ultralytics} also exist but only
as popular code repositories on GitHub, without any detailed report about their
contributions. 

YOLO models have a rich development history but are now reduced to marginal
changes and implementation tricks. This is extremely useful from an engineering
view, but it is less relevant from a research perspective. Nevertheless, the
YOLO framework inspired plenty of other one-stage detectors. In particular, some
detectors drop the use of anchors boxes and instead detect objects with
keypoints. CornerNet \cite{law2018cornernet} for instance produces heatmaps to
determine the position of two corner points for each object, preventing the use
of boxes at all. CenterNet \cite{duan2019centernet} is a refinement of CornerNet
that only outputs a center point and infer the box dimensions from the keypoint
features. Both CornerNet and CenterNet output a keypoint heatmap for each class
and involve sophisticated post-processing to obtain the predicted bounding
boxes. FCOS \cite{tian2019fcos} simplifies this by directly predicting boxes
from each feature location.

It is also worth mentioning Single Shot Detector (SSD) \cite{liu2016ssd}, which
was proposed slightly after YOLO. It is also a one-stage detector, but unlike
the first YOLO version, it is fully convolutional. This has several advantages
as it predicts boxes densely on the images (higher recall, better detection of
small objects) and it adapts better to different input image sizes. But most
importantly, SSD leverages features from various scales for the predictions
which dramatically improves the detection of small targets. Although this idea
was introduced by SSD, it was popularized later with Feature Pyramid Networks
(FPN) \cite{lin2017feature}. We will discuss these advancements in the following
section as well as the choice of the CNN architecture choice.

One-stage object detectors were at first lagging behind two-stage detectors in
terms of detection performance. However, the recent progress tends to close this
gap, making the one-stage detectors the standard choice in the industry as they
offer the best speed/performance tradeoff.  

\paragraph*{Backbone network choice}
In the OD literature, it is common to denote the main features' extractor CNN as
the \textit{backbone} of the network. Then, the lightweight module designed for
classification and box regression on top of the backbone is logically called the
\textit{detection head}. What is placed between the backbone and the head (\eg
FPN and RPN) are sometimes referred to as the \textit{neck}.
\cref{fig:od_architecture_choice} highlights these three main components of the
object detector structure and outlines some design choices for each component.  
The backbone has an extremely important role in object detection as it extracts
the features on which the classification and regression modules work. The choice
of the backbone has been driven by the advances in classification, specifically
the most common backbones have largely proven their capacities on ImageNet.
First AlexNet \cite{krizhevsky2017imagenet} was used by R-CNN, then VGG networks
\cite{simonyan2014very} for Fast R-CNN, Faster R-CNN, YOLO, SSD and many others.
These were quickly replaced by Residual Networks (ResNets) \cite{he2016deep}
which provide a large improvement in ImageNet classification, and consequently
in the detection task. Following the extensions upon ResNets, object detectors
successfully adopted WideResNet \cite{zagoruyko2016wide}, ResNext
\cite{xie2017aggregated} or EfficientNet \cite{tan2019efficientnet}. More
recently, the backbone network shifted from CNN to visual transformers, but this
review will be conducted in \cref{sec:od_recent_advance}. 

Now, the backbones alone are not sufficient to extract relevant features for
Object Detection. Backbone networks are originally designed to deal with curated
images where one main object is visible and often located at the center of the
image. Thus, backbones are not well-suited to deal with objects of various sizes
and locations. An alternative to this issue is to leverage pyramidal features.
This is not a very innovative idea as this was largely employed by face
detectors during the 1990s and HOG models later (see \cref{sec:od_traditional}).  
The actual innovation is to integrate this inside the CNN architecture. This was
introduced first in SSD \cite{liu2016ssd}, but it was popularized with Feature
Pyramid Networks (FPN) \cite{lin2017feature}. FPNs combine features with various
resolutions from the input image at a single resolution (\ie it is not necessary
to perform the forward pass on multiple rescaled versions of the same image). FPNs
extract intermediate feature maps in the backbone and aggregate them in a
bottom-up manner (in contrast to the top-down processing of the forward pass).
This bottom-up computation path is generally implemented with deconvolution
layers, such as in Deconvolution SSD \cite{fu2017dssd}. FPNs are plug-and-play
modules that can be attached to most backbone networks and significantly
improves the detection performance, especially for small objects. In two-stage
detectors, FPN are often combined with Region of Interest Alignment Layer (RoI
Align) to extract RoI features from the most relevant feature level (\ie
according to the RoI size). RoI Align was introduced in Mask R-CNN
\cite{he2017mask}, an extension of Faster R-CNN for Instance Segmentation, that
uses a FPN. In one-stage detectors, the detection is carried out on all feature
levels output by the FPN, deeper levels are responsible for detecting larger
objects. 

Of course, plenty of contributions were proposed to improve upon FPNs. Path
Aggregation Net (PANet) \cite{liu2018path} for instance adds another top-down
path before aggregating features from multiple levels with an Adaptive Feature
Pooling layer. However, the design of the FPN architecture is not trivial and
requires lots of trial and error. To find optimal FPN designs, NAS-FPN
\cite{ghiasi2019fpn} proposed to apply principles of Neural Architecture
Search for the design of FPNs and achieved superior detection performance.
However, these questions are not so relevant to us as they mainly focus on
Auto-ML problematics.  

\begin{figure}
    \centering
    \includegraphics[width=\textwidth]{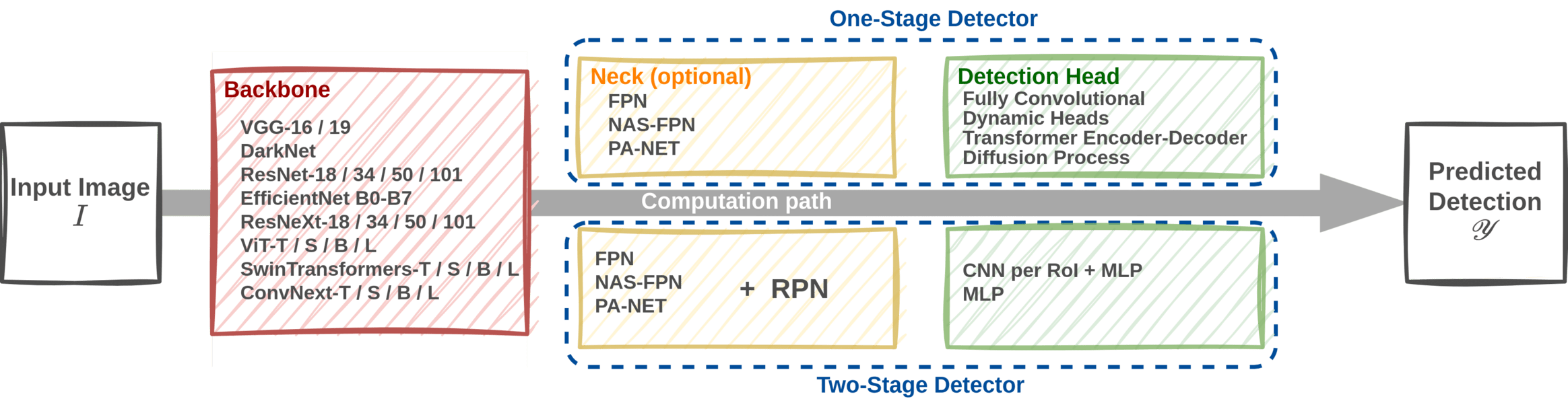}
    \caption[Object Detection architectural choices]{Possible architectural
     choices for Object Detector design.}
    \label{fig:od_architecture_choice}
\end{figure}

\paragraph*{Non-Maximal Suppression}
Note that these methods often output numerous detections and require elaborated
filtering schemes to prevent duplicates. Among others the Non-Maximal
Suppression (NMS) became quite popular. In the
case of largely overlapping boxes (\ie when the IoU is above a fixed threshold),
NMS keeps only the most confident box to prevent duplicates. This is done
separately for each class so that objects from different classes may
overlap. It improves the visual quality of the detection, but it
can slightly degrade the detection performance when dealing with crowded scenes.
For instance, Faster R-CNN employs the NMS both on the outputs of the RPN and
on the final set of bounding boxes.

\subsubsection{Recent Advances in Object Detection}
\label{sec:od_recent_advance}

A Transformer is a network architecture based on a multi-head self-attention
mechanism. It was proposed in the context of Natural Language Processing (NLP)
in 2017 by Vaswani et al. \cite{vaswani2017attention}. Since then, it became an
essential component of most NLP applications. The original idea behind
transformers is to represent the relations between different words in a
sentence. For instance, the subject and a pronoun in a phrase must be strongly
connected as they designate the same object. The self-attention mechanism from
the Transformers is specifically built to adaptively compute these relations
between words. As Transformers reformed the entire NLP field, vision models
started to embed similar mechanisms to model long-range dependencies between
parts of an image. Vision Transformers (ViT) \cite{dosovitskiy2020vit} and Image
Transformers \cite{parmar2018image} are one of the first attempts to solve image
classification with Transformers and achieve considerable improvements over CNN
baselines. They achieve this by dividing an input image into several
patches that they treat just like words in a sentence.

Consequently, most vision tasks were quickly influenced by this new
architecture. Object detection is no exception and in 2020, DEtection with
TRansformers (DETR) \cite{carion2020end} is introduced as a first attempt to
solve object detection using visual Transformers. Due to the time complexity of
the transformers blocks ($\mathcal{O}(H^2W^2)$, where $H$ and $W$ are the height
and width of the image respectively), it is unreasonable to directly apply ViT
for OD as input images are generally quite large. Instead, DETR leverages a
ResNet backbone to extract relevant features but implements the detection head
with transformers blocks. The head is divided into two parts, an encoder and a
decoder. The encoder combines backbone features with positional encodings (\ie
fixed vector whose role is to keep track of the location of the patch
processed). The decoder takes as input a set of object queries, learnable
vectors that represent various positions in the images (similarly to positional
encodings). Their role is to condition the detection toward a specific part of
the image. Encoded image features are integrated inside the decoder through
cross-attention layers. Finally, a light MLP predicts the box coordinates and
class for each object query. The training of DETR is similar to prior object
detectors (\ie a classification and a regression loss function). However, the
matching between predicted and ground truth boxes differs. DETR tackles OD as a
set prediction problem, \ie it predicts a set of bounding boxes as a whole and
compares it with the set of ground truth boxes. This differs from prior
detectors which often employ box-to-box matching. In DETR, the matching is done
by finding an optimal permutation of the sets (with the Hungarian algorithm)
according to a cost involving boxes' positions and class labels. 

Although DETR is a significant milestone in the detection landscape, it does not
yield impressive performance gains over existing work. It even has considerable
drawbacks, its inference is slow, it struggles with small objects and its
training is one order of magnitude slower than prior detectors. Fortunately,
several extensions mitigate these issues. First Deformable DETR
\cite{zhu2021deformable}, reduces the convergence time and improves detection
performance with a deformable attention module. Deformable attention drastically
reduces the amount of computation required and can process images with higher
resolution. Deformable Attention Module is the twin of Deformable Convolutions
\cite{dai2017deformable} but for transformers. Similarly, H-Deformable-DETR
\cite{jia2022detrs} builds upon Deformable DETR with improved matching
techniques which accelerate training further.

While the previous methods leverage transformers to make detections, they still
rely on large CNN models as the backbone. This choice is also questioned by
recent advances in visual transformer architectures. On the one hand,
Data-efficient image Transformers DeiT \cite{touvron2021training} and
Bidirectional Encoder representation from Image Transformers (BEiT)
\cite{bao2022beit} both improve ViT's accuracy and training strategy. Both DeiT
and BEiT show similar fine-tuning properties as CNNs, unlocking their
application for various downstream tasks including object detection. On the
other hand, ConViTs \cite{d2021convit} and Swin Transformers
\cite{liu2021swin,liu2022swin} make the attention computation much faster with
spatial inductive bias and hierarchical structure respectively. 

Another source of improvement for the detection backbones comes from
self-supervised training. Recent advances in large-scale self-supervised
training for classification are now being adapted to other visual tasks. DINO
\cite{caron2021emerging,oquab2023dinov2} pre-trains both CNN and Transformer
based backbones in a contrastive way with carefully designed augmentation
schemes to obtain more robust visual features. Using these pre-trained backbones
generally boosts the performance on many visual tasks, at least when applied to
sufficiently similar images. The very recent Segment Anything Model (SAM)
\cite{kirillov2023segment} also falls under the same category called \textit{foundation
models}. Even though no derivative work has been published yet, the capacities
of SAM are promising for object detection.

Thanks to this progress, transformers-based backbones are about to replace CNN
in most computer vision applications, including Object Detection.
Nevertheless, some CNNs backbones are still proposed and seem to keep up with
the rapid progress of transformers-based backbones. Among these works, ConvNeXt
\cite{liu2022convnet} brings several improvements over the original ResNet
architecture to outperform Swin Transformer backbones. Closely related,
InternImage \cite{wang2022internimage} proposes an extension of Deformable
Convolution to scale CNNs architectures as much as transformers (which was
limited before). Even if the current hype is directed toward Transformers-based
backbones, CNNs are not defeated yet. As an example, DynamicHeads
\cite{dai2021dynamic} is a recent object detector based on a ResNext backbone
achieving very competitive results on the COCO dataset. It leverages attention
mechanisms inside the detection head, but not Transformers modules as in DETR.
Another recent detector based on a CNN backbone is DiffusionDet
\cite{chen2022diffusiondet}. It adapts the diffusion models (currently very
popular for image generation) to box prediction and obtains convincing
performance as well.

To summarize, recent advances in Object Detection have largely followed the
Transformer "revolution". First, with improved detection heads (the DETR
family), and then, with improved backbones, based on Transformers but also
revamped CNNs.

\subsubsection{Object Detection on Aerial Images}
\label{sec:od_rsi}

Most of the works presented in the above sections focus on the object detection
task applied to natural images. Aerial images differ significantly from natural
images, they do not contain any perspective, objects can be arbitrarily rotated,
and they have a greater object size variance. Given this, it seems obvious that
some adjustments are required to adapt popular detectors to aerial images. 

\paragraph*{Oriented Bounding Boxes}
As objects can be oriented in any direction, some aerial object detection
datasets give annotations as oriented bounding boxes. This slightly changes the
problem, but most detectors can easily be extended to deal with rotated boxes.
The bounding box formulation can be extended so that it is rotated, the
regression layer must then be adapted to predict a rotation angle
\cite{liu2017learning,xie2021oriented}, more than four coordinates
\cite{xia2018dota,xu2020gliding}, or to use rotated RoI, for instance with RoI
Transform \cite{ding2018learning}.  

\paragraph*{Small Object Detection}
Aerial images contain objects with great size variance due to discrepancies in
the shot conditions (altitude, sensor resolution, camera focal length, etc.). In
addition, they also have smaller objects than natural images. To deal with the
object size variance, it is possible to leverage supplementary information such
as the ground sample distance (GSD). The ground sample distance represents the
size of one pixel on the ground. Based on GSD, a model can infer the size of the
RoI and therefore refine its predictions, as done by GSDet \cite{li2021gsdet}. 
However, object size variance is generally a limited issue compared to detecting
small objects, which remains an open challenge in object detection. Many
attempts were made to solve this issue using specific architecture design
\cite{deng2022extended}, multiscale training \cite{singh2018sniper,
singh2021scale}, data-augmentation \cite{kisantal2019augmentation} or
super-resolution \cite{rabbi2020small, Bai2018SODMTGANSO,Ferdous2019SuperRD}.
Additionally, Normalized Wasserstein Distance (NWD) \cite{wang2021nwd} proposes
an alternative to the IoU loss specifically designed for detecting small
objects. It consists in computing the Wasserstein distance between two Gaussian
distributions fitted on the two bounding boxes compared. Moreover, NWD is not
only used as a loss function but also as an example selection criterion.

\subsubsection{Training Object Detectors}
\label{sec:training_od}

While we briefly presented how object detectors are trained in the previous
sections, we did not give much detail about the loss functions and the
optimization process. We remedy this here by reviewing the loss functions of
several popular object detection benchmarks.

\paragraph*{Loss functions for Object Detection}
As object detectors must solve both classification and regression tasks, most
detection loss functions are divided into two components, a classification loss and a
regression loss. Plenty of choices exist for both components. For the
classification loss, the most common choice is the Cross-Entropy loss:
\begin{equation}
    \mathcal{L}^{\text{cls}}_{\text{CE}}(\hat{\boldsymbol{\mathrm{y}}_i}, \boldsymbol{\mathrm{y}}_i) = - \log(\hat{l}_i^{c_i}), 
\end{equation}

where $\hat{l}_i^{c_i}$ is the predicted probability that the box $i$ contains an
object of class $c_i$, $c_i$ being the true label ($\boldsymbol{\mathrm{y}}_i = (b_i, c_i)$).
However, some alternatives such as the L1 or L2 losses over the class probabilities are
also employed:
\begin{equation}
    \mathcal{L}^{\text{cls}}_{\text{L1}}(\hat{\boldsymbol{\mathrm{y}}_i}, \boldsymbol{\mathrm{y}}_i) = \| \hat{l}_i - l_i \|_2^2,
\end{equation}

where $\hat{l}_i$ denotes the class probability vector
($\hat{l}_i=\{\hat{l}_i^c\}_{c \in \mathcal{C}}$) and $l_i$ is the one-hot encoded
true probability vector of box $i$. Another very common classification loss
function in recent detector is the Focal Loss (FL) function \cite{lin2017focal},
which was introduced with the RetinaNet object detector. Focal loss is designed
to address the class imbalance issue that is inherent to dense object detectors
(the background class is much more represented than other classes). Focal loss
reduces the relative loss of well-classified examples so that the learning
process focuses on misclassified objects. It is defined as follows:
\begin{equation}
    \mathcal{L}^{\text{cls}}_{\text{FL}}(\hat{\boldsymbol{\mathrm{y}}_i}, \boldsymbol{\mathrm{y}}_i) = - \alpha_{c_i} (1-\hat{l}_i^{c_i})^{\gamma}\log(\hat{l}_i^{c_i}),
\end{equation}

where $\alpha_{c_i}$ is an inverse class-frequency parameter and $\gamma$
controls how much FL reduces the contribution of well-classified examples to the
loss. In RetinaNet, the authors leverage only a binary version of FL as they
tackle the multi-class classification problem as $|\mathcal{C}|$ binary
classification problems. They replaced the Softmax activation function on the
classification layer with a Sigmoid activation and classified the box as either
background or foreground for each class independently. This binary formulation
of the classification task will be extensively re-used by derivative detectors.  

For the regression part, a greater variety of loss functions exists in the
literature. L1 and L2 losses are common in early CNN-based detectors. The
Smooth L1 (or Huber Loss \cite{huber1992robust}) is a variant of the L1 loss
leveraged by the Faster R-CNN family of detectors. It combines the L1 and L2
norms to get a smooth loss function around 0. Next, UnitBox \cite{yu2016unitbox}
introduced the IoU loss function as the new standard for box regression
training: 
\begin{align}
    \mathcal{L}^{\text{reg}}_{\text{IoU}}(\hat{\boldsymbol{\mathrm{y}}_i}, \boldsymbol{\mathrm{y}}_i) &= - \log(\text{IoU}(\hat{b_i}, b_i)), \quad &&\text{Log version}\\
    \mathcal{L}^{\text{reg}}_{\text{IoU}}(\hat{\boldsymbol{\mathrm{y}}_i}, \boldsymbol{\mathrm{y}}_i) &=  1 - \text{IoU}(\hat{b_i}, b_i). \quad &&\text{Linear version}
\end{align}

Following the IoU loss, several extensions were proposed, \eg Generalized IoU
\cite{rezatofighi2019generalized}, Distance-IoU \cite{zheng2020distance}, or
$\alpha$-IoU \cite{he2021alpha}, these will be reviewed later in
\cref{chap:siou_metric}. To summarize, \cref{tab:od_losses} gives an overview of
the loss functions used by common object detectors.

\begin{table}[h]
    \centering
    
    \resizebox{\columnwidth}{!}{%
    \begin{tabular}{@{\hspace{3mm}}lll@{\hspace{3mm}}}
    \toprule[1pt]
                 & \textbf{Classification}                                                                                                                                              & \textbf{Regression}                                                                                              \\ \midrule
    \textbf{Faster R-CNN} \cite{ren2015faster} & \begin{tabular}[c]{@{}l@{}}Cross-Entropy for the detection head\\ Binary Cross Entropy for the RPN\end{tabular}                                                      & SmoothL1 Loss for head and RPN                                                                          \\ \midrule
    \textbf{YOLO} \cite{redmon2016you}        & \begin{tabular}[c]{@{}l@{}}L2 Loss on class probability vector (for grid cell containing an object)\\ L2 Loss on true class probability (for all cells)\end{tabular} & \begin{tabular}[c]{@{}l@{}}L2 Loss on box center\\ L2 Loss on square root of box dimensions\end{tabular} \\ \midrule
    \textbf{RetinaNet} \cite{lin2017focal}    & Binary Focal Loss                                                                                                                                                    & SmoothL1 Loss                                                                                           \\ \midrule
    \textbf{UnitBox} \cite{yu2016unitbox}     & Binary Cross Entropy                                                                                                                                                 & IoU Loss (log)                                                                                          \\ \midrule
    \textbf{FCOS} \cite{tian2019fcos}       & Binary Focal Loss                                                                                                                                                    & GIoU Loss (linear)                                                                                      \\ \midrule
    \textbf{DETR}  \cite{carion2020end}       & Cross Entropy                                                                                                                                                        & L1 Loss and GIoU Loss                                                                                   \\ \midrule
    \textbf{DiffusionDet} \cite{chen2022diffusiondet} & Binary Focal Loss                                                                                                                                                    & L1 Loss and GIoU Loss                                                                                   \\ \bottomrule[1pt]
    \end{tabular}%
    }
    \caption[Loss functions of Object Detectors]{Summary of the loss functions used in several object detection frameworks.}
    \label{tab:od_losses}
    \end{table}

\paragraph*{Example selection strategy}
In the previous paragraph, we presented the various loss functions employed in
the Object Detection literature. For simplicity, we defined these losses for a
couple of predicted and ground truth detection labels
$(\hat{\boldsymbol{\mathrm{y}}_i}, \boldsymbol{\mathrm{y}}_i)$. In reality, the
losses are computed as the sum of all such couples (over one or multiple
images). However, it is not straightforward to build these couples as there may
be more predictions than ground truths, missed objects, or false detection. Each
detector has its own strategy to operate the matching between prediction and
ground truth. These strategies were briefly presented in the previous sections,
but we regrouped them inside \cref{tab:od_matching_strat} for clarity.

\begin{table}[h]
    \centering
    \resizebox{\columnwidth}{!}{%
    \begin{tabular}{@{\hspace{3mm}}ll@{\hspace{3mm}}}
    \toprule[1pt]
                 & \textbf{Matching Strategy} \\ \midrule
    \textbf{Faster R-CNN} \cite{ren2015faster}          & \begin{tabular}{@{}l@{}}- Select RoI with at least 0.5 IoU with a GT as Positive Samples (PS) and RoI with low IoU (< 0.1) \\\hspace{2mm}as Negative Samples (NS). \\- \textbf{Classification loss} is computed on all selected samples \\\hspace{2mm}(PS with the corresponding GT class and NS with the background class). \\- \textbf{Regression Loss} is only computed with PS. \end{tabular}               \\ \midrule
    \textbf{YOLO} \cite{redmon2016you}                  & \begin{tabular}{@{}l@{}}- Select PS as grid cells in which there is a GT center point and assign the highest IoU boxes in \\\hspace{2mm}case of multiple GT in one cell. \\\hspace{2mm}All others are NS. \\- Classification done separately on PS and NS. \\-Regression loss with PS only.  \end{tabular}                 \\ \midrule
    \textbf{FCOS} \cite{tian2019fcos}                   & \begin{tabular}{@{}l@{}}- Select PS as feature map locations that fall inside a GT, all others are NS. \\\hspace{2mm}If multiple GT for the same location take the smallest one. \\- Classification on PS and NS. \\- Regression on PS only. \end{tabular}                \\ \midrule
    \textbf{DETR}  \cite{carion2020end}                 & \begin{tabular}{@{}l@{}}- 1-to-1 optimal prediction and GT assignment according to localization and classification cost. \\ - \textit{No-object} are added to the GT set when the predictions are more numerous. \\- Classification loss is computed for all matched couples. \\- Regression loss only for couples with an actual GT. \end{tabular}               \\ \midrule
    \textbf{H-DETR} \cite{jia2022detrs}                 & \begin{tabular}{@{}l@{}}- DETR matching. \\- Supplementary 1-to-many matching with duplicated and augmented GT for training. \end{tabular}                  \\ \midrule
    \textbf{DiffusionDet} \cite{chen2022diffusiondet}   & - DETR matching                  \\ \bottomrule[1pt]
    \end{tabular}%
    }
    \caption[Prediction-Ground truth matching of Object Detectors]{Brief description of some existing prediction ground truth matching strategy in existing object detectors.}
    \label{tab:od_matching_strat}
    \end{table}

\subsection{Datasets Presentation}
There exist numerous object detection datasets in the literature. We present in
this section four of them in detail as they will be extensively used in this
manuscript. These datasets are Pascal VOC \cite{everingham2010pascal}, MS COCO
\cite{lin2014microsoft}, DOTA \cite{xia2018dota} and DIOR \cite{li2020object}.
We choose these datasets because they were the most relevant and widespread
datasets of natural (Pascal VOC and COCO) and aerial (DOTA and DIOR) images at
the beginning of this project. Some other datasets will be punctually used,
especially for the cross-domain experiments in \cref{chap:diffusion}, they will
be presented in detail there. However, we draw up a non-exhaustive list in
\cref{tab:od_dataset_overview} of the most well-known object detection datasets in
the literature.

\begin{table}[h]
    \centering
    \resizebox{\columnwidth}{!}{%
    \begin{tabular}{@{}ccccc@{}}
    \toprule[1pt]
    \textbf{Image Type / Application}    & \textbf{Dataset Name} & \textbf{\# Classes}   & \textbf{\# Images}    & \textbf{\# Instances}   \\ \midrule
    \multirow{4}{*}{Natural}             & Pascal VOC \cite{everingham2010pascal}              & 20                   & 11.5k                & 27k                  \\
                                         & COCO  \cite{lin2014microsoft}                   & 80                   & 117k                 & 1.5M                 \\
                                         & LVIS  \cite{gupta2019lvis}                   & 1203                 & 100k                 & 1.3M                 \\
                                         & Object365  \cite{Shao_2019_ICCV}              & 365                  & 2M                   & 30M                  \\ \midrule
    \multirow{2}{*}{Autonomous Vehicle}  & KITTI  \cite{Geiger2012CVPR}                  & 11                   & 7k                   & 80k                  \\
                                         & BDD100k  \cite{yu2020bdd100k}                & 10                   & 400k                 & 3M                   \\ \midrule
    \multirow{3}{*}{Pedestrian}          & CityPerson  \cite{Shanshan2017CVPR}             & 1                    & 3k                   & 19k                  \\
                                         & TinyPerson  \cite{yu2020scale}             & 1                    & 1610                 & 72k                  \\
                                         & CrowdHuman  \cite{shao2018crowdhuman}             & 1                    & 15k                  & 340k                 \\ \midrule
    \multirow{5}{*}{Aerial}              & COWC   \cite{mundhenk2016large}                  & 1                    & \multicolumn{1}{l}{} & 33k                  \\
                                         & DOTA  \cite{xia2018dota}                   & 16                   & 2.8k (megapixels)    & 220k                 \\
                                         & DIOR   \cite{li2020object}                  & 20                   & 23k                  & 190k                 \\
                                         & xView  \cite{lam2018xview}                  & 60                   & 1.1k (megapixels)    & 1M                   \\
                                         & FAIR -1M \cite{sun2022fair1m}               & 37                   & 15k (megapixels)     & 1M                   \\ \midrule
    \multirow{2}{*}{Agricultural / Food} & DeepFruits                               &7                      & 457                       & 2.5k \\
                                         & Oktobeerfest \cite{tum2019oktoberfest}      & 15                   & 1k                   & 2.5k                 \\ \midrule
    \multirow{4}{*}{Other Modalities}    & ClipArt \cite{inoue2018cross}            & 32                   & 5k                   & 13k                  \\
                                         & LogoDet  \cite{wang2022logodet}           & 3000                 & 159k                 & 194k                 \\
                                         & SIXray \cite{miao2019sixray}             & 6                    & 9k                   & 1M                   \\
                                         & DroneVehicle  \cite{sun2022drone}      & 5                    & 56k                  & 191k                 \\ \bottomrule[1pt]
    \end{tabular}%
    }
    \caption[Overview of existing detection datasets]{Overview of existing detection datasets}
    \label{tab:od_dataset_overview}
    \end{table}

\subsubsection{Natural Images}
Natural images are the kind of image humans are the most familiar with,
therefore it is logically the most common application in Computer Vision. Object
detection is no exception and most proposed detectors are developed to process
natural images. Hence, our analysis must be conducted as well on natural images
even though the industrial interest of COSE is more towards aerial imagery. To
this end, we choose Pascal VOC and MS COCO as our main sources of natural images.

\textbf{Pascal VOC \cite{everingham2010pascal} --}
The Pascal VOC challenge took place every year between 2005 and 2012. This
competition defined the object detection problem as we know it today. The last
version of the dataset Pascal VOC 2012 includes images of various sizes and
aspect ratios. Each image is annotated with horizontal bounding boxes around
objects belonging to 20 classes. Examples of images and a list of classes,
ordered by the number of occurrences, are available in \cref{fig:od_pascal_voc}. 

\textbf{MS COCO \cite{lin2014microsoft} --}
MS COCO is an extension of Pascal VOC which includes much more images and
classes. The set of images is completely distinct from Pascal VOC, but all
classes in Pascal VOC are included in COCO. Similarly, \cref{fig:od_coco}
presents image examples and a list of COCO classes.

\subsubsection{Aerial Images}
The overall goal of this project is to detect objects from aerial images. Aerial
images are sometimes associated with low-altitude drone images. These images are
halfway between natural and aerial images as they often preserve some
perspective. Remote Sensing Images (RSI), \ie acquired from planes or satellites
with nadir-oriented cameras are much closer to COSE's application. In this
manuscript, we refer to this kind of image both as aerial or RSI images. There
exist a few publicly available datasets of such images. We have chosen two of
them based on the ground resolution of the images (in agreement with COSE
systems) and their availability at the beginning of this project. 

\textbf{DOTA \cite{xia2018dota} --}
DOTA contains images coming from Google Earth and distinct satellites Jilin-1
and Gaofen-2 (with roughly 1m spatial resolution GSD). Images from DOTA are
large, ranging from 800 to 4000 pixels in width and objects are annotated with
oriented bounding boxes. To ease the handling of the images, we prepared DOTA by
tiling all images into $512 \times 512$ patches with a 50\% overlap and
converted the annotations to horizontal bounding boxes. \cref{fig:od_dota}
presents images and the class list for DOTA. 

\textbf{DIOR \cite{li2020object} --}
DIOR is very similar to DOTA. It contains only images scrapped from Google Earth
and has slightly more classes than DOTA. The images are already tiled at $800
\times 800$ pixels and boxes are horizontal. \cref{fig:od_dior} presents images
from DIOR and the list of classes.

\vspace{1em}
\section{Few-Shot Learning: Learning with Limited Data}
\label{sec:fsl}
\vspace{-1em}
As presented in the introduction, COSE faces a substantial challenge in the
design of its imaging systems: the lack of real-case images and unknown objects
of interest. All methods described in \cref{sec:review_od} require large
annotated training sets to achieve reasonable detection performance, which is
misaligned with COSE's constraints. This issue is common in the industry, most
computer vision problems lack large annotated datasets, and therefore the direct
application of research contributions is often challenging. Fortunately, there
exists an entire research field dedicated to learning with limited annotated
data. The main paradigm in this field is to learn a closely related task with
sufficient data and adapt to the real task with limited annotations. Two kinds
of adaptation can be considered: class adaptation and domain adaptation. Given a
computer vision task such as classification, the former consists in learning to
classify objects or images among a set of classes and then adapt to another set
of classes. This is usually called Few-Shot Classification (FSC). While
classification is not the primary interest of COSE, it is worth exploring the
FSC literature as it is an older field, much more developed than FSOD and
because FSC lays the foundation for tackling more complex tasks in the few-shot
regime. On the other hand, domain adaptation consists in adapting to different
kinds of images, \eg different seasons, weather conditions, general
environments, etc. In the strict definition of domain adaptation, the classes of
interest remain the same. However, the setting when both the classes and the
domain change is also studied in the literature. It is more challenging, but it
better reflects the industrial needs such as COSE's. In this section, we review
both kinds of adaptations for the classification problem. Even though it is not
a task of interest for COSE, understanding few-shot adaptation strategies is
crucial before addressing the more challenging problem of Few-Shot Object
Detection which we reserve for \cref{chap:fsod}.   

\subsection{Few-shot Classification}
\vspace{0.5em}
\subsubsection{Problem Definition}
Classification is a simpler problem than detection. Given a set of classes
$\mathcal{C}$ and an input image $I$, one wants to find the class $c \in
\mathcal{C}$ that is depicted by $I$. Of course, the higher considerations
briefly presented in \cref{sec:od_problem} about how to properly define the
membership of an image to a class still holds. For classification as well, the
class membership is determined by human appraisal and common sense. Solving a
classification task is to find a model $\mathcal{F}(\cdot, \theta)$ that outputs a class label for a given
input image $I$:
\begin{equation}
    \mathcal{F}(I, \theta) = \hat{c} \in \mathcal{C}.
\end{equation}
Deep Learning based models proved to be particularly adapted to the
classification task in a fully supervised setting (\ie provided with
sufficiently large annotated datasets). This was supported amongst others by
LeNet \cite{lecun1998gradient} for digit classification, and by AlexNet
\cite{krizhevsky2017imagenet} and ResNet \cite{he2016deep} for ImageNet
classification. However, the classification task in this form is not a topic for
this section, and we refer to \cite{rawat2017deep} for a complete review of
existing works in this field. 

\begin{figure}
    \centering
    \begin{subfigure}[position]{\textwidth}
        \includegraphics[width=\textwidth]{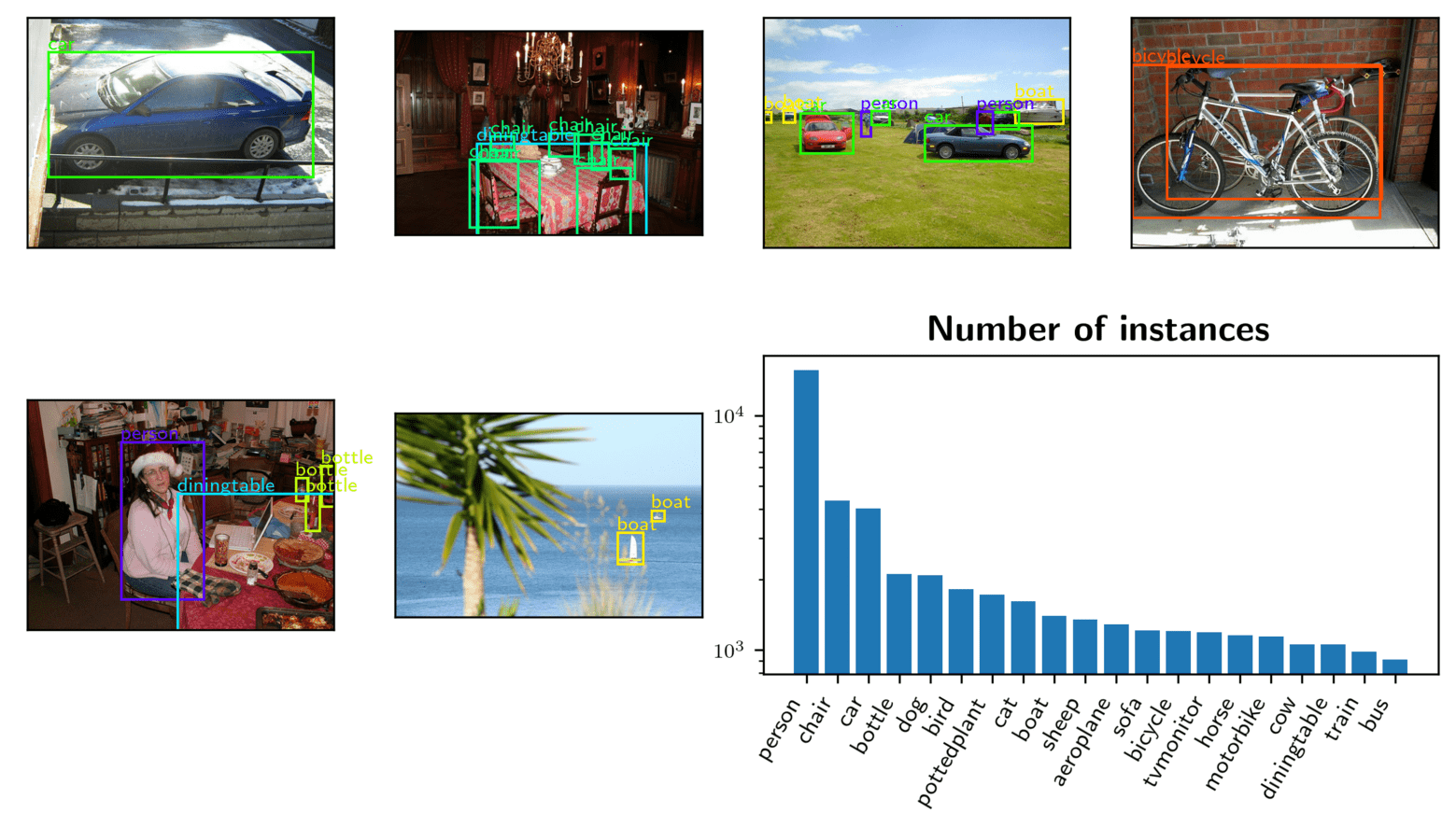}
        \caption[]{Examples of Pascal VOC images and class repartition on the training split.}
        \label{fig:od_pascal_voc}
    \end{subfigure}
    \begin{subfigure}[position]{\textwidth}
        \includegraphics[width=\textwidth]{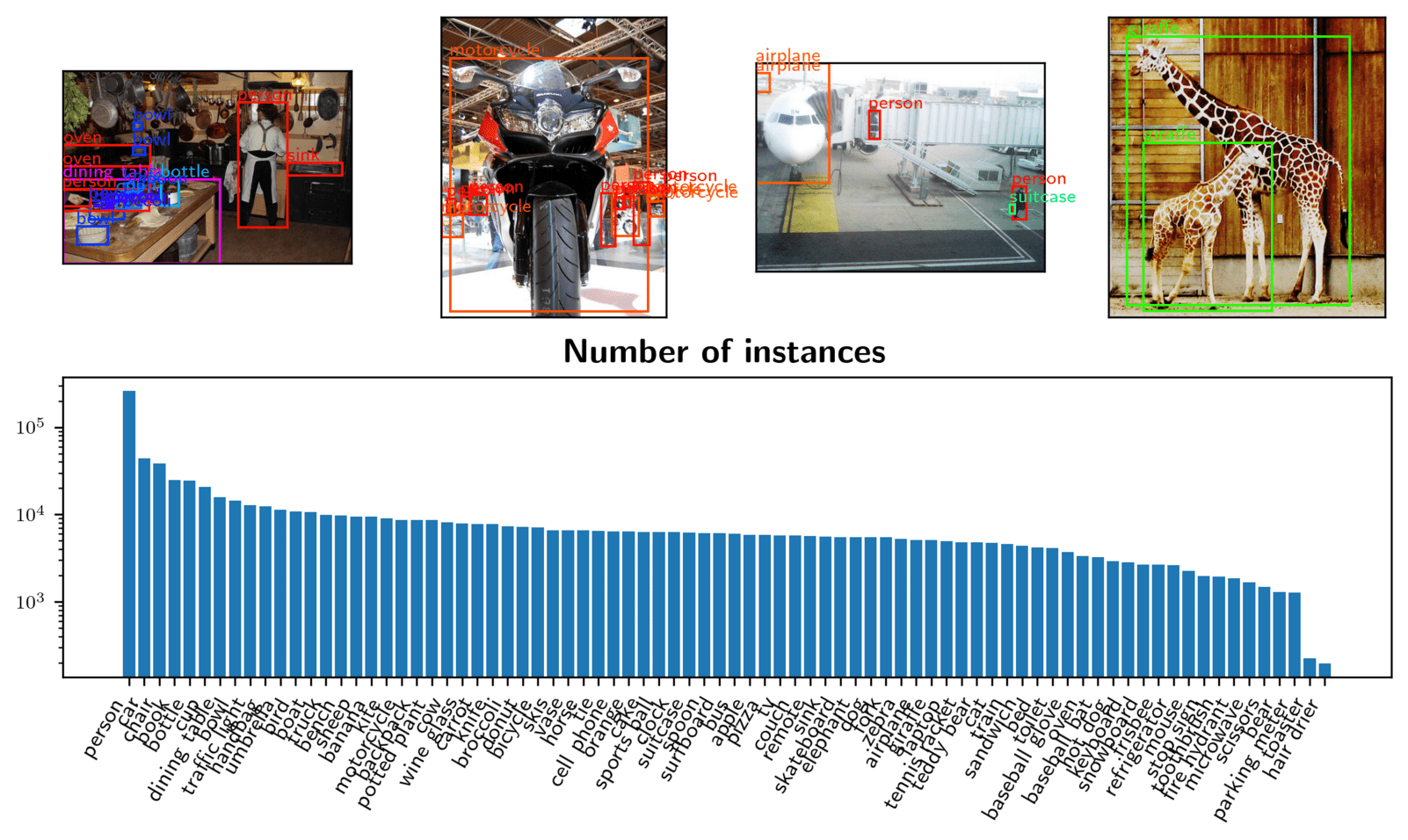}
        \caption[]{Examples of MS COCO images and class repartition on the training split.}
        \label{fig:od_coco}
    \end{subfigure}
    \caption[Overview of Pascal VOC and MS COCO datasets]{Image examples for the Natural images dataset Pascal VOC and MS COCO.}
    \label{fig:od_dataset_natural}
\end{figure}

\begin{figure}
    \centering
    \begin{subfigure}[position]{\textwidth}
        \includegraphics[width=\textwidth]{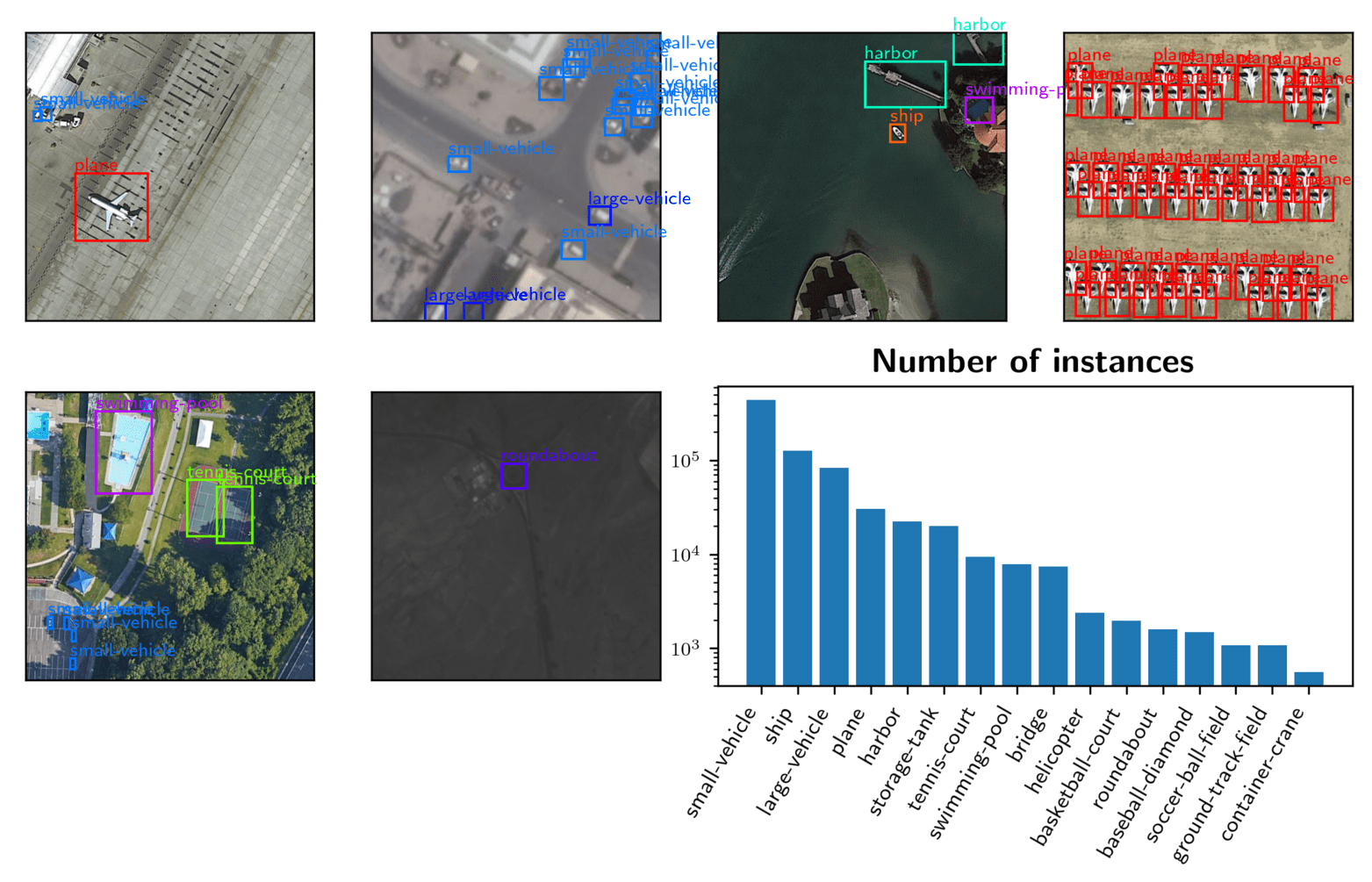}
        \caption[]{Examples of DOTA images and class repartition on the training split.}
        \label{fig:od_dota}
    \end{subfigure}
    \begin{subfigure}[position]{\textwidth}
        \includegraphics[width=\textwidth]{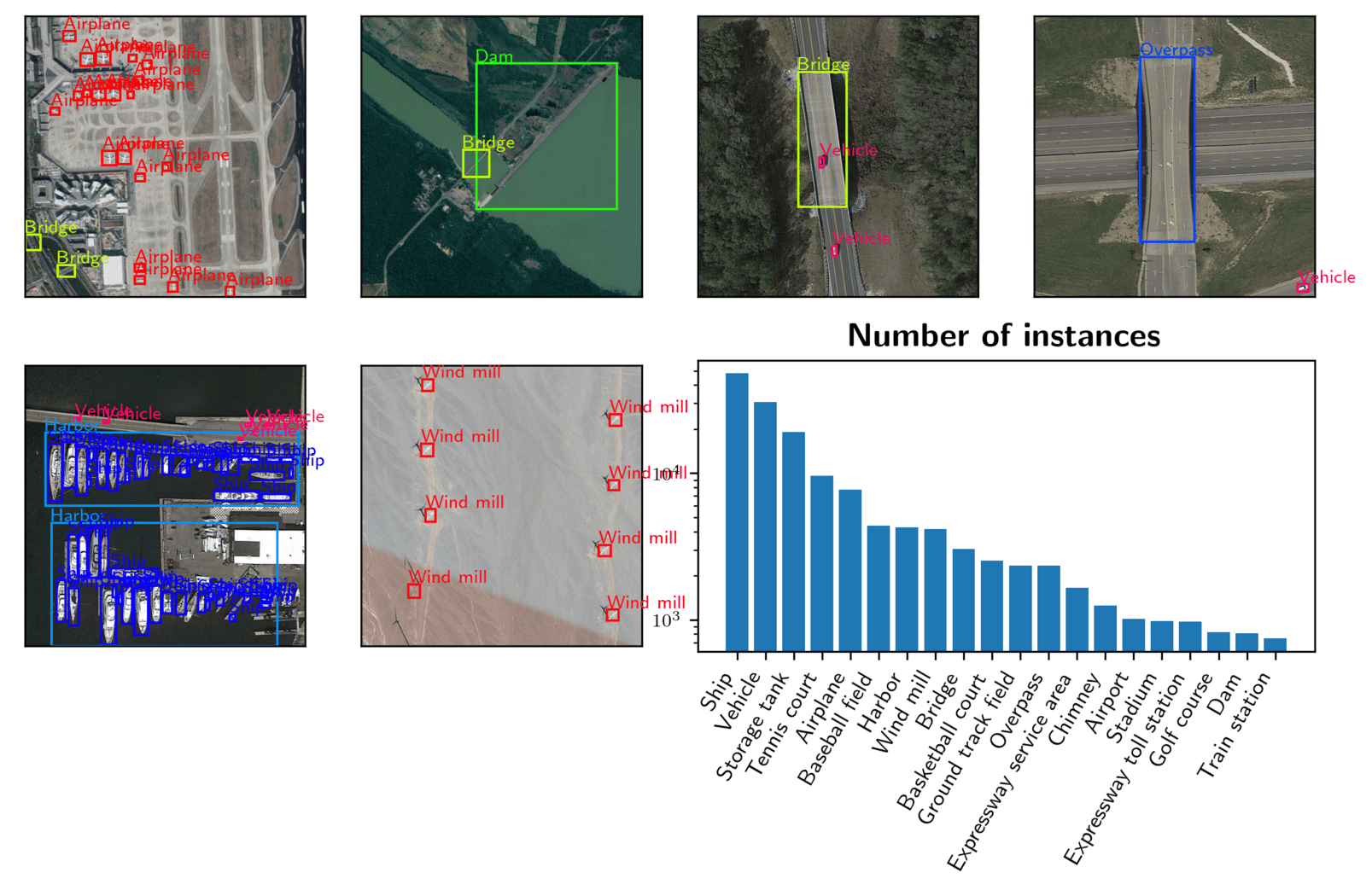}
        \caption[]{Examples of DIOR images and class repartition on the training split.}
        \label{fig:od_dior}
    \end{subfigure}
    \caption[Overview of DOTA and DIOR datasets]{Image examples for the Aerial images dataset DOTA and DIOR.}
    \label{fig:od_dataset_aerial}
\end{figure}

In the few-shot setting, the classification goal remains the same, predicting
the class of an image. The input image to an FSL model is usually denoted as a
\textit{query image}, and therefore, the test set is called the \textit{query
set}. What changes between the few-shot and regular settings is the amount of
annotated images available to train the model. In the literature, the expression
$N$-way $K$-shot classification designates the task of classifying images
amongst $N$ different classes only provided with $K$ annotated examples per
class. The $NK$ images constitute the \textit{novel dataset}, in contrast to the
\textit{base dataset} which contains an arbitrary number of annotations for
another set of classes. In the few-shot literature, the novel dataset is often
called the \textit{support set}, and its elements \textit{support examples}.
Similarly, the sets of classes of the base and novel datasets are called the
\textit{base classes} set (noted $\mathcal{C}_{\text{base}}$) and \textit{novel
classes} set (noted $\mathcal{C}_{\text{novel}}$) respectively. Specifically, we
have:
\begin{align}
    \mathcal{D}_{\text{base}} &= \left\{(I_i, c_i)\right\}_{1 \leq i \leq |\mathcal{D}_{\text{base}}|} \quad c_i \in \mathcal{C}_{\text{base}}, \\
    \mathcal{D}_{\text{novel}} &= \bigcup_{c \in \mathcal{C}_{\text{novel}}}\left\{(I_k^c, c)\right\}_{1 \leq k \leq K}.
\end{align}

As mentioned above, $\mathcal{D}_{\text{base}}$ is used to train the model
during a first phase called \textit{base training}. During this phase, the model
has access to plenty of annotated data and is trained in a supervised manner to
classify images within $\mathcal{C}_{\text{base}}$. It is noteworthy to point
out that this supervised base training is not the only possible choice. Recent
advances in Self-Supervised Learning (SSL)
\cite{chen2020simple,grill2020bootstrap,he2020momentum,caron2020unsupervised}
proved that SSL is a competitive alternative to supervised base training. 

After base training, the novel dataset is leveraged to adapt the model to
classify the novel classes. Hence, the few-shot classification task can be seen
as predicting the class label from the input image and the novel dataset: 
\useshortskip
\begin{equation}
    \mathcal{F}(I, \mathcal{D}_{\text{novel}}) = \hat{c} \in \mathcal{C}.
\end{equation}
The model adaptation generally starts with small architectural modifications,
such as replacing the final classification layer with a novel layer randomly
initialized and with the right number of outputs (\eg if the numbers of base and
novel classes differ). Then, several approaches exist for adjusting the model to
the novel classes given the novel dataset. We identify here four different
adaptation strategies and will present each of them in the next sections. These
strategies are: fine-tuning, metric-learning, meta-learning and
attention-mechanisms. However, there are no clear boundaries between these four
areas, \cref{fig:fsl_taxonomy} illustrates the interactions between the various
strategies and gives a few examples for each category. We propose this taxonomy
as it suits well the few-shot object detection field. Hence, reviewing FSC
through this lens helps to understand how these techniques could be extended for
detection. However, there exist much more detailed taxonomies and reviews about
FSC in the literature, \cite{wang2020generalizing,song2022comprehensive} are
worthy examples. Note that the novel dataset can also be used during inference,
so that adaptation is done "on the fly". This is called \textit{transductive
inference} and will be presented in \cref{sec:transductive_cls}. 

In the most common few-shot setting, we have $\mathcal{C}_{\text{base}} \cap
\mathcal{C}_{\text{novel}} = \emptyset$, meaning that there are no common
classes between the base and novel sets. Of course, some works focus on relaxing
these assumptions, we will outline some of them in \cref{sec:generalized_fs_setting}

\begin{figure}
    \centering
    \includegraphics[width=\textwidth]{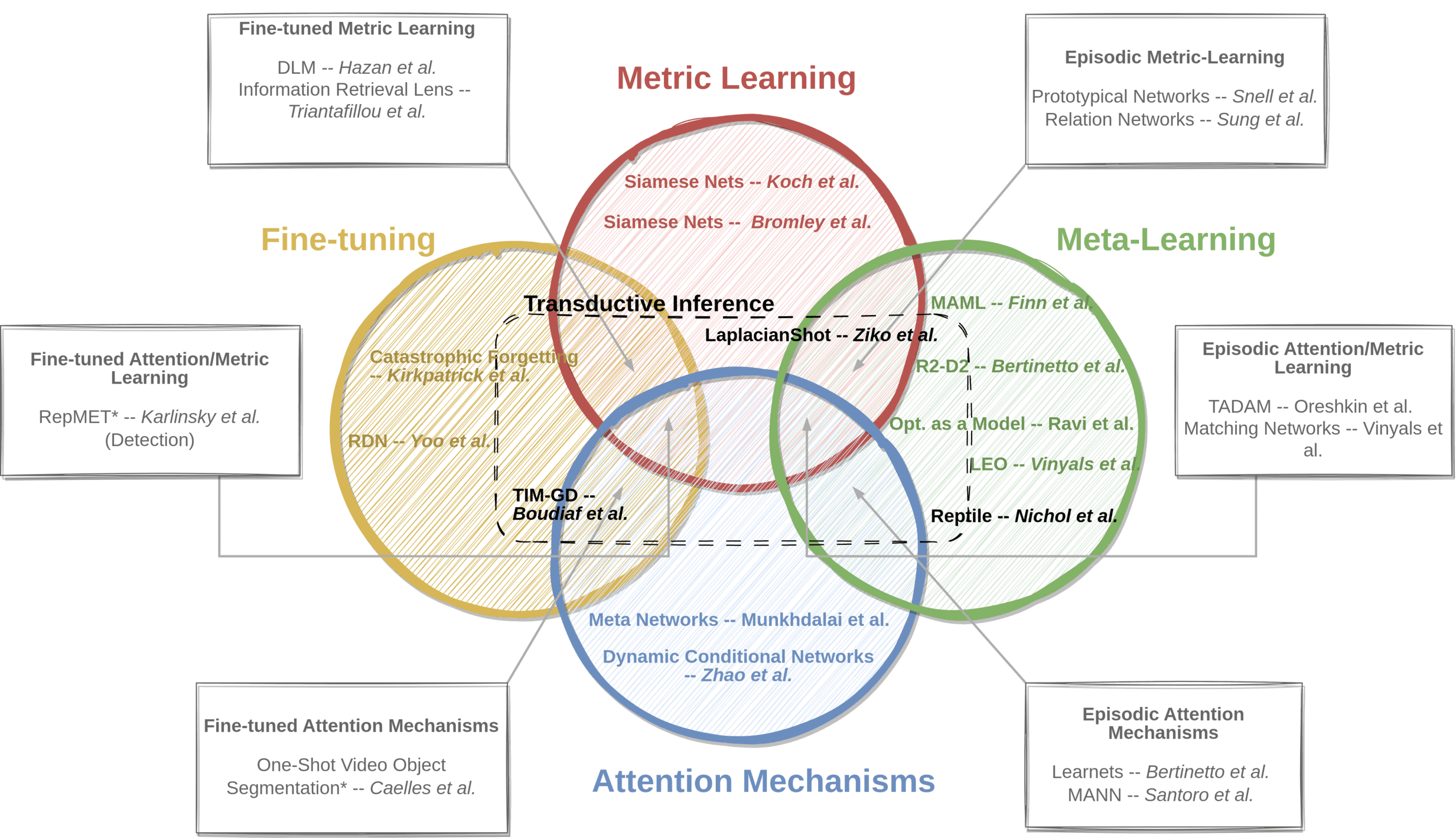}
    \caption[Taxonomy of Few-Shot Classification]{Taxonomy of the Few-Shot
    Classification literature. To illustrate each category of the taxonomy, a
    few papers are selected as representatives among others. Papers marked with a
    * are not solving the FSC task but are included in this figure as no
    contribution in the literature tackles classification from this
    perspective.}
    \label{fig:fsl_taxonomy}
\end{figure}

\subsubsection{Fine-tuning}
Probably the most straightforward way to tackle the FSC task is to employ
fine-tuning or transfer learning. This method trains the model on
$\mathcal{D}_{\text{base}}$ and then the model weights are fine-tuned using
$\mathcal{D}_{\text{novel}}$ with only the few examples available. It works well
but the fine-tuned models are prone to strong overfitting and catastrophic
forgetting \cite{kirkpatrick2017overcoming}. Overfitting on the novel set is
problematic as it means that the fine-tuned model will have poor generalization
capabilities, \ie its performance will be way lower on the test set than on the
training set. Catastrophic forgetting is a more subtle issue. It happens when
the performance of the fine-tuned model on the base classes drops. In the case
of the simple FSC it is a subtle issue, but it becomes more challenging when
dealing with extended setups such as generalized FSL and Continual Learning (see
\cref{sec:generalized_fs_setting} for more details). However, the authors of
\cite{kirkpatrick2017overcoming} propose regularization penalties to be applied
during fine-tuning that alleviate both the overfitting and catastrophic
forgetting. Specifically, the penalty prevents the fine-tuned weights from being too
far from the pre-trained weights. Similarly, \cite{yoo2018efficient} proposes
several regularization loss functions and a grouped parameter update during
fine-tuning to overcome the overfitting. Closely related,
\cite{triantafillou2017few} leverages Direct Loss Minimization's theorem
\cite{hazan2010direct} to optimize the model's weights with an Information
Retrieval Loss at inference. Although fine-tuning is a relatively simple
approach for FSL, it was not much explored in the case of classification. We
will see in \cref{chap:fsod} that it has gained more attention recently for more
complex tasks.

\subsubsection{Metric Learning}
Metric learning is a branch of deep learning which consists of learning
self-organized representation spaces, \ie similar inputs should have similar
representation in the \textit{embedding space}. It was first introduced with
Siamese Networks for signature verification \cite{bromley1993signature} and
later for face identification \cite{chopra2005learning}. The idea behind the
Siamese Networks is to leverage two copies of the same model and feed them two
different images. The output of the two networks should be similar if the input
images look similar as well. Siamese networks were then applied for one-shot
image classification by \cite{koch2015siamese}. It was one of the first attempts
to solve this task using deep neural networks. Features from the query and
support images are extracted by the siamese nets and then compared by a final
prediction layer. This final layer takes as input the difference between the
features of the query and support images. Its role is to assess whether the
features are similar enough so that the two images belong to the same class.
Following Siamese Networks, a series of works was proposed based on the same
principle. These contributions are embodied by Prototypical Networks
\cite{snell2017prototypical}. ProtoNets replace the final prediction layer of
the siamese networks with a linear classifier and extend the metric learning
framework for multi-class and increased number of shots. Specifically, the
features of all support examples of one class are aggregated to form
\textit{class prototypes} and query features are classified according to the
class of the closest prototype. Many variants of the ProtoNets were then
proposed. Inspired by Siamese Networks, Relation Networks
\cite{sung2018learning} replace the linear classifier of the Prototypical
Network with a small MLP trained to predict a similarity score based on the
query features and a prototype. The difference with Siamese Networks is that
this is done with each class prototype allowing Relation Nets to address
multi-class problems. Other extensions include prototypes rectification for
intra and extra-class variance \cite{liu2020prototype}, semi-supervised prototypes
refinement \cite{ren2018meta} and multiple prototypes per class
\cite{deuschel2021multi}. It is essential to note that Prototypical Networks and
their extensions leverage episodic training strategies borrowed to
meta-learning. This strategy consists in dividing the training into shorter
\textit{episodes}. During each episode, the model is trained for a random task,
generally a subset of the novel classes (only the $\mathcal{D}_{\text{novel}}$
dataset is considered by these approaches). The episodic strategy follows the
"learning to learn" paradigm and mimics the adaptation process the model will
undergo at test time. 

The episodic strategy in the context of metric learning is first used in
Matching Networks \cite{vinyals2016matching}, an earlier work than Prototypical
Network. Matching Networks are more inspired by the meta-learning techniques,
hence the episodic strategy. Two networks are trained jointly, one to extract
the support features and one for the query features. However, the way query and
support features are combined differs from other metric-learning methods. The
authors leverage an attention mechanism to compute the predicted class
probabilities as a similarity-aware weighting of the support examples labels. Closely
related, Task Dependent Adaptive Metric (TADAM) \cite{oreshkin2018tadam} learns a
task representation and adapts its embedding network through a task-conditioning
layer which resembles an attention mechanism. The class prototypes are then
task-dependent and an image is classified according to the most similar
prototype. Matching Network and TADAM are therefore at the intersection of three
approaches for FSL: metric-learning, meta-learning and attention mechanism.
These are reviewed in the two following sections.   

\subsubsection{Meta-Learning}
As hinted at the end of the previous section, meta-learning's paradigm is
"learning to learn”. This was the main motivation for the episodic training
strategy described there. Each episode forces the model to adapt to new classes,
repeating these episodes should overall increase the adaptation capacities of
the model. However, the concept of meta-learning goes even further. 

This concept was introduced and popularized during the 1990s
\cite{schmidhuber1987evolutionary,thrun1998learning}. At the time, meta-learning
was mostly applied in the context of policy learning, with evolutionary or
reinforcement learning methods. It was brought up-to-date for the few-shot
classification by Model-Agnostic Meta-Learning (MAML) \cite{finn2017model} which
proposes to directly train the initial weights of a classifier so that it will
quickly adapt to a given task. The optimization is done in a nested manner. At
the inner level, a task is sampled (like with the episodic training strategy)
and the classifier is initialized with the current initial weights. A few
gradient steps are performed on the classifier with respect to the task
objective function. Then, at the outer level, the initial weights are updated
through gradient descent on the task loss value computed with the trained
classifier. The meta-update converges to a set of initial weights that make the
classifier "easy to train" on any task. However, MAML does not take task
information into account for the weight initialization, and it is unrealistic to
find truly task-agnostic initializations. Therefore, \cite{lee2018gradient} extends
MAML to only choose a subset of the optimal initial parameters to initialize the
classifier based on task information. Orthogonally, some contributions integrate
uncertainty in the weight initialization
\cite{franceschi2018bilevel,kim2018bayesian}, others make the training easier
\cite{behl2019alpha,rusu2018meta} or provide a theoretical framework and
guarantees \cite{fallah2021generalization}.

Similarly, some meta-learning techniques propose learning the optimization
process instead of the weight initialization. This is the case of Optimization
as a Model \cite{ravi2016optimization} which trains a LSTM meta-learner to
output gradient updates for the classifier network. The meta-learner takes as
input the weights of the classifier and the gradients computed on a given task.
The recurrent nature of the meta-learner helps to keep track of the previous
error signals and update consequently the weights of the classifier. Close to
MAML, the meta-learner is updated after several weight updates (with different
tasks) based on the loss values of the classifier on a test set. 

Another meta-learning direction is introduced with Ridge Regression
Differentiable Discriminator (R2D2) \cite{bertinetto2018meta}. It consists in
teaching a model “to use” standard machine learning tools such as Ridge or Lasso
Regression. These techniques often have closed-form solutions and are fast to
compute when few data are available. In R2D2, a CNN is trained as a feature
extractor as a meta-learner, while the classifier's weights are computed with a
ridge regression from the support set. The meta-learner CNN is trained to
extract features that will generate optimal classifier weights through ridge
regression.  

\subsubsection{Attention-Based methods}
As an alternative to the rather complex and heavy meta-learning methods, a line
of work followed the MAML principle but focused only on some layers rather than
on the complete classifier. This originates with LearNets
\cite{bertinetto2016learning} that are trained to output the weights of a
convolution kernel from a support example. The kernel is then used in a
\textit{Dynamic Convolution Layer} (DCL) inside the classifier, which in the end
predicts a class-membership score (according to the class of the support
example). When multiple classes are available in the support set, the DCL is
applied with each class features independently and the classifications are done
in a binary fashion. This can be understood as an attention mechanism between
the query and support features. The dynamic convolution acts as a filter
responsive to the support features. To put it another way, the locations in the
query feature map that are similar to the support features will be highlighted.
The dynamic convolution sets the focus of the classifier on features from the
support class. Self-Attention (SA) is probably the most common form of attention
mechanism used in the current deep learning literature. It highlights similarity
links between subparts of the input (\eg different locations in an image, or two
words in a sentence). Here, with DCL, this is done with two distinct images: a
query and a support image. Several other works build upon this idea. Simple
Neural Attentive Learner (SNAIL) \cite{mishra2017simple} designs more complex
attention blocks, based on transformers, to perform the classifier's adaptation.
Although their primary goal is to tackle few-shot reinforcement learning with
temporal convolutional layers (to deal with causality), they apply it
successfully to FSC as well. The $NK$ images in the support set are fed in
random order and the query image is given last. Very similar to SNAIL,
CrossTransformers \cite{doersch2020crosstransformers} assemble an attention
module to combine query and support features. The crucial difference with SNAIL
is the preservation of the spatiality of the features. Most previous works
aggregate the support features to perform the adaptation, losing the spatial
information of the support image. Inspired by the recent progress of ViT,
CrossTransformers manage to adapt the classifier while preserving spatial
information. Just like many metric learning methods presented in the previous
section, some of the attention-based techniques discussed above borrow an
episodic training strategy from meta-learning. 

The query-support attention mechanism can also be interpreted as a conditioning
of the classifier's input based on the support features. That is the view adopted
by Dynamic Conditional Network \cite{zhao2018dynamic}. The general idea is very
similar to LearNet except for the training which is not done in an episodic
manner. Conditionally Shifted Neurons (CSN) \cite{munkhdalai2018rapid} see the
adaptation as the conditioning of the classifier activations. The meta-learner
outputs shift values that are added to the pre- or post-activation values in the
network. The shifts are computed from a \textit{task description} stored in a
memory. The task description regroups the activation of all the layers of the
classifier fed with the support images. The use of a memory bank is widespread
along with attention-based mechanisms for FSC. Memory Augmented Neural Network
(MANN) \cite{santoro2016meta} epitomizes this line of work. It leverages a
\textit{controller} (\ie a small network) to read and write in the memory. The
controller generates a key from an input, which is then used to either add a new
entry in the memory or retrieve already stored information. The retrieval is
done through an attention-like mechanism. The memory is built throughout a task
episode adaptively. When a new memory is added, if a similar memory is already
stored, the new memory refines the existing one to build more relevant
representations. Plenty of contributions took inspiration from MANN. Adaptive
Posterior Learning (APL) \cite{ramalho2019adaptive} refines the memory writing
process to store only "surprising" memories. Meta Networks
\cite{munkhdalai2017meta} also leverage external memory only accessible to the
meta-learner in charge of adapting the classifier. \cite{xu2017few} proposes a
second \textit{abstract memory} which stores refined information relevant for
the current task. All these memories are generally wiped when the task is
modified. However, \textit{life-long memory} \cite{kaiser2017learning} also
benefits the FSC even though it is mostly exploited for slightly more challenging
tasks such as Generalized FSL or Continual Learning (see
\cref{sec:generalized_fs_setting}). 

\subsubsection{Transductive Inference}
\label{sec:transductive_cls}
A recent line of work tackles the few-shot classification problem with
transductive inference. Transductive inference, in contrast to inductive
inference, consists of leveraging labeled and unlabeled images and classifying
all unlabeled data points at the same time. Conversely, inductive inference
deals with each data point independently. Of course, deploying such methods in
an industrial scenario requires having multiple data points available at test
time. For real-time applications, it is generally not practical. However, in the
case of COSE and the detection task, this is largely applicable. Very large
images cannot be processed as a whole, they must be divided into smaller images.
In addition, for the detection task, an image is often seen as a collection of
objects or regions of interest. COSE's use case is therefore rather well-suited
for transductive inference. Therefore, we briefly review the recent advances of
transductive learning for FSC in this section.

Transductive inference is an old concept of statistical learning that was
popularized under this name and for machine learning applications by Vladimir
Vapnik in the 1990s \cite{vapnik1999overview}. As mentioned above, in the
transductive setting, an entire unlabeled dataset (\eg a test dataset or a query
set in the few-shot context) is available at test time. Transductive methods
leverage information contained both in the support and query set to make
predictions. Before application on FSC, transduction was applied to regular
classification on small-size datasets, yielding significant improvements over
inductive methods. Amongst them, Transductive Support Vector Machines
\cite{joachims1999transductive} extends the well-known SVM
\cite{cortes1995support} to make use of unlabeled information to refine the
class separation margins. Another direction is taken by \cite{zhou2003learning}
which derives an iterative method that propagates known labels to unlabeled data
points according to their similarity. Recently, TransBoost
\cite{belhasin2022transboost} even applied transduction to the entire ImageNet
dataset with significant accuracy gain over inductive methods. The authors
propose a fine-tuning approach to refine trained neural networks to perform
better on a specific test set. It takes both the training set and the unlabeled
test set to compute a regularization loss function that penalizes similar images
to be classified differently by the network. 

Transductive inference is especially effective in the few-shot context as the
limited labeled data is often not enough to provide sufficient supervision.
Leveraging additional unlabeled data is therefore highly beneficial. Various
approaches were proposed to make use of this supplementary information within
the already existing few-shot frameworks. Probably the most straightforward
approach is to fine-tune pre-trained models with additional regularization loss
based on the labeled and unlabeled data. This is the direction taken by
\cite{dhillon2019baseline} which compares the few-shot performance of several
methods against a simple transductive fine-tuned baseline. Similarly,
Transductive Episodic-wise Adaptive Metric (TEAM) \cite{qiao2019transductive}
and Transductive Information Maximization with Gradient Descent (TIM-GD)
\cite{boudiaf2020information} also both leverage fine-tuning objectives to
refine the model before transductive inference. This resembles semi-supervised
learning which fine-tunes models with additional unlabeled data, \ie different
from the training and test set. However, the essence of transductive learning
lies more in adapting the inference based on the additional information rather
than fine-tuning the model. This can be done by direct optimization of an
objective function with regularization as in LaplacianShot
\cite{ziko2020laplacian}, TIM \cite{boudiaf2020information} or Cross-Attention
Networks \cite{hou2019cross}. Many propositions iteratively propagate the known
labels to unlabeled data points within a graph structure
\cite{kim2019edge,liu2018learning,hu2021graph,rodriguez2020embedding,lazarou2021iterative}.
But there also exist contributions that exploit transductive inference through
metric learning with \eg Prototype Rectification\cite{liu2020prototype} or
Meta-Confidence Transduction \cite{kye2020meta}, which meta-learn a distance
metric. Meta-learning based methods also get their transductive extension, such
as Reptile \cite{nichol2018first} which extends MAML to perform transductive
inference by leveraging information shared by test samples through the batch
normalization layers.

\subsubsection{Extending the Few-Shot Setting}
\label{sec:generalized_fs_setting}

The few-shot setup that we described in previous sections is limited and makes a few
assumptions: 
\begin{enumerate}[nolistsep]
    \item The set of base and novel classes are known in advance.
    \item At test time, only the performance on novel classes matters.
    \item Novel classes are only added once and all at the same time. 
\end{enumerate}

These assumptions significantly simplify the problem, but these are relaxed by
different sub-fields of few-shot learning. In some aspects, the few-shot
detection can be seen as a relaxation of these assumptions. Various tasks,
similar to few-shot classification, exist in the literature.
\cite{geng2020recent} provides a comprehensive taxonomy of these tasks. We will
briefly present in this section some relevant extensions of the few-shot
classification for COSE's application and the detection task.
\cref{tab:few-shot-task} provides an overview of these tasks and their
differences in terms of goal and available data.

\paragraph*{Few-shot Open-set Recognition}
Open-set classification assumes that some classes are unknown during training
(\ie the training dataset is incomplete) and deals with these classes. Instances
of unknown classes can be rejected or identified as unknown classes. It models
real use cases better as test data can be contaminated by classes not included
in the training set. Object detection can be assimilated as an open-set problem
as objects belonging to a fixed set of classes must be localized while rejecting
everything else as background. The training set can only contain a limited
variety of background examples and new instances of background will be presented
to the model at test time. There are plenty of approaches for Open-set
Recognition, but we will not review them here in detail and refer the reader to
a complete survey \cite{geng2020recent} about it. Instead, we simply outline the
general principle behind algorithms that tackle this problem. Two main
approaches coexist in the literature, discriminative and generative approaches.
The former ones propose techniques to distinguish between known and unknown
classes using discriminative information, \eg distance to class representations
\cite{bendale2016towards}. The latter leverage generating models to
hallucinate negative examples as additional training data \cite{jo2018open}. Of
course, transduction also helps in this case and \textit{outlierness} score
\cite{boudiaf2023open} can be computed using the unlabelled examples available
at test time. 

This holds for open-set recognition, but in few-shot there are additional
complexities. Not only the classes from the query set may be unknown (\ie not
even in the support set), but the support set could provide irrelevant
information for the current task. This setup, introduced in
\cite{martin2022towards}, is not common in the FSL literature even if it is of
great interest from an industrial perspective. It is also quite relevant from
the few-shot detection point of view as the detection support examples can embed
irrelevant information for the task.

\paragraph*{Generalized Few-Shot Classification}
Up to now, we presented the few-shot classification problem as only adapting a
model to classify novel classes. However, it can sometimes be relevant to keep
the possibility of classifying classes from the base dataset. Often the
adaptation significantly reduces the performance on base classes, this
phenomenon is known as the catastrophic forgetting
\cite{kirkpatrick2017overcoming}. When both base and novel classes are of
interest, the task is called \textit{generalized few-shot learning}. This can be
achieved with several tricks such as doing the inference with both base-trained and
fine-tuned models. But it is also possible with careful extension and
fine-tuning of the model, \eg via disentangling base and novel class predictions
\cite{gidaris2018dynamic}.

\paragraph*{Continual Learning}
Generalized few-shot is an intermediary step toward continual or life-long
learning which consists in continuously adapting the model with novel classes.
This is way more challenging but also resembles the industrial setting better.
While extremely relevant from COSE's perspective, we choose not to tackle this
problem in this PhD project as it seems more sensible to address first the
already challenging few-shot setting for the detection. In addition, continual
learning often leverages complex learning scheme such as task rehearsal
\cite{robins1995catastrophic} or adaptive model architectures \cite{li2019learn}
to prevent forgetting classes or tasks.

\begin{table}[]
    \centering
    \resizebox{\columnwidth}{!}{%
    \begin{tabular}{@{}cccc@{}}
    \toprule[1pt]
    \textbf{Task}           & \textbf{Classes of interest}  & \textbf{Novel supervision} & \textbf{Query-support interaction} \\ \midrule
    Regular Classification  & $\mathcal{C}_{\text{base}}$                                                                                           & None                       & None                               \\  \midrule
    Few-shot Classification & $\mathcal{C}_{\text{novel}}$                                                                                          & $K$ examples per novel class                          & $\mathcal{C}_{\text{query}} = \mathcal{C}_{\text{support}}$                            \\  \midrule
    Zero-shot Classification& $\mathcal{C}_{\text{novel}}$                                                                                          & External information (\eg class labels)             & None                               \\  \midrule
    Generalized FSC         & $\mathcal{C}_{\text{base}} \cup \mathcal{C}_{\text{novel}}$                                                           & $K$ examples per novel class                          & $\mathcal{C}_{\text{query}} = \mathcal{C}_{\text{support}}$                            \\  \midrule
    FS Open-set Recognition & $\mathcal{C}_{\text{novel}} \cup \mathcal{C}_{\text{unknown}}$                                                        & \begin{tabular}{@{}l@{}}$K$ examples per novel class \\ None for unknown classes\end{tabular}                & $\mathcal{C}_{\text{query}} \subset \mathcal{C}_{\text{support}}$  or $\mathcal{C}_{\text{query}} \supset \mathcal{C}_{\text{support}}$          \\  \midrule
    Continual Learning      & $\mathcal{C}_{\text{base}} \cup \left(\bigcup_i \mathcal{C}_{\text{novel}}^i\right)$                                      & $K$ examples per novel class                          & None                               \\  \midrule
    Few-Shot Object Detection& $\mathcal{C}_{\text{novel}} \backslash \mathcal{C}_{\text{background}}$                        & $K$ annotated images per novel class                  & $\mathcal{C}_{\text{query}} = \mathcal{C}_{\text{support}}$                            \\ \bottomrule[1pt]
    \end{tabular}%
    } \caption[Few-Shot Learning and related tasks]{Summary of the various flavors of classification tasks existing
    in the literature. The second column, classes of interest, denotes what is
    the overall goal of the task. The last column presents the possible class
    setup encountered both in the query and support set
    ($\mathcal{C}_{\text{query}} $ and $\mathcal{C}_{\text{support}}$
    respectively.). $\mathcal{C}_{\text{unknown}}$ represents additional classes
    that should be identified in the open-set setting. In the detection task,
    $\mathcal{C}_{\text{background}}$ denotes all object classes that can be
    present in the background and that should not be detected.}
    \label{tab:few-shot-task}
    \end{table}

\subsection{Cross Domain Adaptation}
\label{sec:cda}

Sometimes, there are significant discrepancies between images from train and test
datasets. We discussed in the previous section the discrepancies in terms of
classes: classes encountered at test time may differ from annotated training
classes.  However, training and test images can also have different aspects. For
instance, autonomous vehicle perception systems could be trained only with
daylight images and encounter nighttime images once deployed. The train and test
image spaces are denoted as \textit{source domain} and \textit{target domain} in
the Cross-Domain Adaptation (CDA) literature. Specifically, a domain consists in
an image space $\mathcal{I}$ and a marginal probability distribution $p(I)$ over
it: 
\begin{equation}
    \mathcal{M} = \left\{\mathcal{I}, p(I) \right\}, \quad\quad I \in \mathcal{I}.
\end{equation}

CDA aims at adapting a model trained for a specific task on a source domain
$\mathcal{M}_{\text{source}} = \left\{\mathcal{I}_{\text{source}},
p_{\text{source}}(I) \right\}$ to perform the same (or another) task on the
target domain $\mathcal{M}_{\text{target}} = \left\{\mathcal{I}_{\text{target}},
p_{\text{target}}(I) \right\}$. For simplicity, we restrain the scope of this
section to the classification task. Hence, when the task changes from source and
target, the set of classes changes as well. We denote these sets as
$\mathcal{C}_{\text{source}}$ and $\mathcal{C}_{\text{target}}$ to comply with
the CDA notations. Note that these sets of classes correspond to the base and
novel classes in the FSC context. Generally, in the CDA literature, a limited
amount of annotated data is available for the target domain which prevents
direct supervised training. However, if a closely related source domain with
sufficient available data is available, adaptation to the target domain is
possible with limited data. Accordingly, cross-domain adaptation and few-shot
learning are closely related problems. In this section, we review the two kinds
of CDA, with and without label shift. COSE's industrial application 
contains CDA's problematics as the imaging systems can be deployed to
different theaters of operations for which no images were available during
training.

\subsubsection{Domain Adaptation without class shift}
There exists a slight difference between Domain Adaptation (DA) and what is
sometimes called Few-shot Domain Adaptation (FSDA) in the literature. This
difference lies in the amount of available data in the target domain. FSDA
methods have access to fewer target examples than regular DA. This distinction
is not relevant as in both cases, there is not enough target data to perform
directly supervised training (although additional unlabeled target data is often
leveraged). Therefore, we choose to review both DA and FSDA at once. This review
is not exhaustive, and we refer the reader to \cite{wang2018deep} for a more
complete overview of Domain Adaptation. Following this survey, we divide our
review into two parts, discrepancy-based adaptation and generative modeling approaches. 

\paragraph*{Discrepancy-based Adaptation}
The simplest way to adapt a model to a target domain is to fine-tune it on the
few available target data. The model is first trained on the source domain to
learn the task. Then, fine-tuning is done on the target domain with some tricks
to avoid overfitting. These tricks consist in reducing the discrepancies
between source and target features. For instance,
\cite{tzeng2015simultaneous} fine-tunes on the target domain with a regular
cross-entropy loss but leverages additional loss functions to minimize domain
confusion with additional unlabeled target images. Similarly,
\cite{tzeng2015simultaneous} has been extended with semi-supervised consistency
\cite{gebru2017fine} and contrastive \cite{motiian2017unified} losses. Following
the same principle, a number of works
\cite{ghifary2014domain,long2015learning,tzeng2014deep,long2017deep}
leverage additional losses based on the Maximum Mean Discrepancy (MMD) or close
extensions. MMD is a distance measure between probability distributions. In the
context of DA, it can be leveraged to assess the shift from source to target
domain for a given class. Employing MMD-based loss functions allows these
methods to learn domain invariant features and therefore improve cross-domain
generalization. As an example, Central Moment Discrepancy (CMD)
\cite{zellinger2017central} proposes an approximation of MMD to derive a
discrepancy regularizer. This regularization is computed over all layers of the
model to enforce features from all levels to be domain invariant. Other
contributions developed relatively similar techniques based on other criteria
such as Kullback-Leiber divergence \cite{zhuang2015supervised}, or correlation
alignment \cite{sun2016return}. 

The methods presented above all fine-tune the models from feature discrepancies.
However, as the task remains the same, it is reasonable to assume that optimal
weights for the source and target domains are related. Following this idea,
\cite{rozantsev2018beyond} proposes a weight regularization to prevent
fine-tuning to find weights too different from source weights. Closely related,
\cite{li2016revisiting} proposes to only change Batch Normalization's statistics
to adapt to the target domain.

Finally, advances in adversarial learning provided new ways to address DA by
minimizing source and target discrepancies in an adversarial setup. This is
embodied by \cite{ganin2016domain} and \cite{tzeng2017adversarial} which both
jointly train a domain discriminator along with the target feature extractor in
an adversarial fashion. The trained extractor embeds images in a shared
source-target feature space on which the source classifier can perform well.

\paragraph*{Generative Modeling}
Another approach to domain adaptation is to artificially generate target data.
This is particularly easy with discriminative approaches based on Generative
Adversarial Networks \cite{goodfellow2020generative}. GANs were extended to
perform domain translation with CoGAN \cite{liu2016coupled}, Pix-2-Pix
\cite{isola2017image} and CycleGAN \cite{CycleGAN2017}. The source domain images
can then be converted into source-target image pairs which greatly facilitate
domain adaptation with methods similar to the ones described in the previous
paragraph. This is done for instance in CyCADA \cite{hoffman2018cycada}. Of
course, GANs are not the only available generative models suitable for this
task. Recent advances in image generation leveraging Diffusion Processes
\cite{sohl2015deep,ho2020denoising} unveil new possibilities for domain
adaptation following existing work about generative domain adaptation as
done very recently by \cite{zhang2023diffusion}.

Closely related, Deep Reconstruction Networks (DRCN) \cite{ghifary2016deep}
jointly learn to classify and reconstruct images from multiple domains. The
model is trained to classify source images and reconstruct target images. This
strategy enforces the learning of domain-invariant features and largely improves
domain adaptation. Similar approaches have been proposed with disentangled
domain-invariant and domain-specific representations \cite{bousmalis2016domain},
or adversarial reconstruction \cite{kim2017learning}.

\subsubsection{Cross-Domain Adaptation with class shift}

Cross-Domain Few-Shot Classification (CD-FSC) designates problems where both
classes and domain change at the same time. This complexifies further the
learning, but it is closer to real-case scenarios and developing such techniques
will ease the deployment of classification techniques. It is particularly
interesting for COSE as it solves two major issues regarding training visual
recognition systems for surveillance applications: undefined objects of interest
and changing image appearance. This setting is relatively new in the few-shot
literature and has been popularized in particular by the creation of
Meta-Dataset \cite{triantafillou2019meta}. Meta-Dataset is a benchmark for
CD-FSC. It gathers 10 existing classification datasets and proposes a simple
testing scenario: pre-train on ImageNet then fine-tune on each dataset
individually with limited annotations. 

Most of the proposed techniques for solving CD-FSC borrow from both the few-shot
learning and domain adaptation fields. Plenty of approaches are then based on
the meta-learning strategy, pre-training on the source dataset and fine-tuning
episodically on the target domain and novel classes. Meta-FDMixup
\cite{fu2021meta} for instance trains episodically a classifier with additional
domain discriminant losses computed on an augmented query set (mixing-up source
and target domain -- MixUp \cite{zhang2017mixup} is a well-known augmentation
technique). Meta-FDMixup, is later extended with a dynamic mixup strategy by
Target Guided Dynamic Mixup (TGDM) \cite{zhuo2022tgdm}. Another merger of FSL and
DA techniques is Domain-Adaptive Prototypical Networks (DAPN)
\cite{zhao2021domain}, which extends prototypical networks with a domain
adaptation module for prototype alignment, trained in an adversarial fashion.
Closely related, \cite{chen2022cross} proposes a bi-directional prototype
alignment. Another line of work tackles CD-FSC through the prism of
distillation, for instance, \cite{fu2022me} first trains two "experts" networks
to perform the FSC task on both domains independently. Then, a student network
is trained to match the output of both teachers using distillation techniques.
It results in a student network able to deal with both domains identically.
Similarly, Universal Representation Learning (URL) \cite{li2021universal}
distills knowledge learned from $K$ classifiers trained on $K$ distinct domains into
a single cross-domain model. This is achieved by adding lightweight domain
adaptation modules between the feature extraction module and the classification
layer. Overall these techniques all involve complex training strategies and
architectural designs which are not very convenient for industrial deployment,
replication, or future extensions. To counter this, ReFine \cite{oh2022refine}
proposes a simple fine-tuning strategy that only re-initializes the last layers
of the model before fine-tuning to facilitate domain adaptation. Much simpler
than concurrent approaches, it yields competitive results.

Finally, some other works \cite{yao2021cross,zhang2022few,hu2022adversarial}
study an even harder task when target domain data are completely unlabeled. We
will not review this kind of approach as it is out of the industrial scope of
COSE.

\section{Conclusion}
This chapter presents the object detection and few-shot learning fields, both
necessary to the conception of few-shot object detectors. For object detection,
notations and problem definitions are given in detail, as well as a list of
popular evaluation metrics and datasets. A thorough review of existing works
redraws decades of progress in this field and helps understand how
state-of-the-art detection has been achieved. Similarly, for Few-Shot Learning,
this chapter gives the key definitions to understand the stakes of the few-shot
problem. An overview of the few-shot literature also provides relevant insights
about how to adapt perception models in low-data regimes. This prospecting work
greatly helps in understanding what is relevant from a research perspective and
what directions to follow according to the industrial needs of COSE.  

\chapter{Few-Shot Object Detection}
\label{chap:fsod}

\chapabstract{This chapter presents the task of detection in the few-shot regime
and reviews the existing literature about it. Few-Shot Object Detection (FSOD)
is at the crossroads of Object Detection and Few-Shot Learning, and therefore,
extensively relies on these two fields explored in \cref{chap:od}. Just as for
classification, various directions are explored in the literature to tackle the
detection task in the few-shot regime which will be presented in detail.
Finally, this chapter focuses on the aerial image application of FSOD methods
and extensions of the few-shot setting. } \PartialToC

The company COSE is developing CAMELEON, an intelligent airborne surveillance system to
automatically detect objects of interest. The detection algorithm must be
adaptative as the objects can change from one operation to another. Therefore,
the most relevant direction to explore is the Few-Shot Object Detection (FSOD)
task. In this chapter, we properly define the FSOD setting and present an
exhaustive review of the current literature. We also explain how detection
datasets can be leveraged for FSOD and how the proposed methods are evaluated.

\section{Problem definition}
\vspace{-1em}
Unsurprisingly, the Few-shot Object Detection task aims to detect objects just
as regular object detection but under the few-shot constraints. Specifically,
given an input image $I$, FSOD's goal is to learn a detection model
$\mathcal{F}(\cdot, \theta)$, with parameters $\theta$, able to adapt to new
classes ($\mathcal{C}_{\text{novel}}$) from only a limited number of examples.
Just as for the few-shot classification problem, two datasets are available, a
base dataset with plenty on annotations of base classes instances
$\mathcal{C}_{\text{base}}$ and a novel dataset (also called support set) with
$K$ annotated images for each novel class:
\begin{align}
    \mathcal{D}_{\text{base}} &= \left\{(I_i^{c_i}, \mathcal{Y}_i^{c_i})\right\}_{1 \leq i \leq |\mathcal{D}_{\text{base}}|} \quad c_i \in \mathcal{C}_{\text{base}}, \\
    \mathcal{D}_{\text{novel}} &= \bigcup_{c \in \mathcal{C}_{\text{novel}}}\left\{(I_k^c, \mathcal{Y}_i^c)\right\}_{1 \leq k \leq K},
\end{align}
where $I_k^c$ is an image containing at least one instance of the class $c$, and
$\mathcal{Y}_k^c$ is the corresponding annotation set (bounding box and label)
for the image $I_k^c$, filtered to contain only class $c$ instances. Note that
there could be more than $K$ annotations per class as multiple objects of the
same class can be visible on one image. This setting is commonly used in the
FSOD literature and called $N$-ways $K$-shots object detection. Conversely,
keeping only one annotation to comply with the few-shot classification setting
can be problematic as it provides wrong supervision to the model. This issue
will be elaborated further in \cref{chap:aerial_diff}. Hence, based on the input
image and the support set, the few-shot detection model $\mathcal{F}(\cdot,
\theta)$ should predict bounding boxes and labels for all instances of classes
$\mathcal{C}_{\text{novel}}$:
\begin{equation}
    \mathcal{F}(I, \mathcal{D}_{\text{novel}}) = \hat{\mathcal{Y}} = \{\hat{\boldsymbol{\mathrm{y}}}_i\}_{i=1}^{M_I} = \{(\hat{b}_i, \hat{c}_i)\}_{i=1}^{M_I}, \quad\quad \text{with } \hat{c}_i \in \mathcal{C}_{\text{novel}}. 
\end{equation}

This setup resembles the FSC setting described in \cref{chap:od}, but brings
some complications. While the sets of base and novel classes are disjoint, FSOD
must deal with the background. Any object that does not belong to either the
base or novel class sets is considered background. Therefore, an object detector
can encounter unknown classes at test time and must be able to ignore them. No
information about the background classes is available which makes it even more
difficult to discriminate between classes of interest and background. From this
perspective, FSOD is closer to the few-shot open-set recognition problem than
FSC. In addition, multiple different classes of interest can be depicted within
a single image. Distinct objects can overlap in the image and their features
(potentially from different classes) can blend, making recognition challenging.
This is reinforced as the objects get smaller, their features get noisier and
can be misclassified more easily. This stands for the query images but also for
support images which increases the difficulty compared to FSC. 

\section{Review of the Few-Shot Object Detection Literature}
\vspace{-1em}
Even though FSOD is a natural extension of FSC, the difficulties mentioned above
prevent the direct use of FSC techniques, just as classification techniques may
be extended for detection. Of course, the main principles for adapting
classification models to the few-shot setting can be reused, but they need to be
carefully adjusted to take care of the supplementary challenges of the detection
task. Hence, as for FSC, the detection models are first trained on the base
dataset and then adapted to novel classes with the support set. This adaptation
can be done in many ways, often based on FSC approaches. Therefore, we adopt the
same organization as for \cref{sec:fsl} and divide our review into four
distinct parts: fine-tuning, metric learning, meta-learning and attention-based
approaches. \cref{fig:fsod_timeline} outlines the organization and the
temporality of the FSOD field. FSOD is a relatively new challenge and only
started 2 years before this PhD project. \cref{tab:fsod_comparison} provides an
almost exhaustive overview of the literature about FSOD. The reader can refer to
several surveys \cite{huang2021survey,kohler2021few,liu2022few} about FSOD for
more thorough reviews. However, note that these surveys are already a few years
old, which is already a lot compared to the recency of the field.  

\midsepremove
\begin{table}[]
    \centering
    \resizebox{\textwidth}{!}{%
    \rowcolors{2}{gray!25}{white}
    \begin{tabular}{@{\hspace{5pt}}lcccP{4cm}cP{5.5cm}P{3cm}@{\hspace{5pt}}}
    \toprule[1pt]
    \textbf{Approach}           & \textbf{Abbreviation} & \textbf{Venue}       & \textbf{Date}        & \textbf{Detection Framework}        & \textbf{Multiscale}  & \textbf{Datasets}                              & \textbf{Aerial / Natural Images}  \\ \midrule
    \cellcolor{white}                                                                                     & FRW    \cite{kang2019few}                                 & ICCV                 & 2019                 & YOLO                      & No                   & Pascal / COCO                                  & Natural                    \\
    \cellcolor{white}                                                                                     & OSOD-CACE     \cite{wallach2019one}                       & NEURIPS              & 2019                 & Faster RCNN               & Yes                  & Pascal / COCO                                  & Natural                    \\
    \cellcolor{white}                                                                                     & Meta R-CNN    \cite{han2021meta}                          & ICCV                 & 2019                 & Faster RCNN               & No                   & Pascal / COCO                                  & Natural                    \\\rowcolor{OliveGreen!50}
    \cellcolor{white}                                                                                     & FSOD-RSI      \cite{deng2020few}                          & TGRS                 & 2020                 & YOLO                      & Yes                  & DIOR / NWPU VHR                                & Aerial                     \\
    \cellcolor{white}                                                                                     & ARPN          \cite{fan2020few}                           & CVPR                 & 2020                 & Faster RCNN               & Yes                  & COCO                                           & Natural                    \\
    \cellcolor{white}                                                                                     & VEOW          \cite{xiao2020few}                          & ECCV                 & 2020                 & Faster RCNN               & Yes                  & Pascal / COCO                                  & Natural                    \\
    \cellcolor{white}                                                                                     & KT            \cite{kim2020few}                           & SMC                  & 2020                 & Faster RCNN               & Yes                  & Pascal                                         & Natural                    \\
    \cellcolor{white}                                                                                     & OSOD-WFT      \cite{li2020one}                            & Preprint             & 2020                 & FCOS                      & Yes                  & Pascal / COCO / ImageNet Loc                   & Natural                    \\
    \cellcolor{white}                                                                                     & ONCE          \cite{perez2020incremental}${\ddagger}$     & CVPR                 & 2020                 & Center Net                & No                   & Pascal / COCO / Deepfashion                    & Natural                    \\\rowcolor{OliveGreen!50}
    \cellcolor{white}                                                                                     & WSAAN         \cite{xiao2020fsod}                         & TAEORS               & 2021                 & Faster RCNN               & Yes                  & RSOD / NWPU VHR                                & Aerial                     \\\rowcolor{OliveGreen!50}
    \cellcolor{white}                                                                                     & FSOD-FPDI     \cite{gao2021fast}                          & MDPI                 & 2021                 & FCOS                      & Yes                  & DOTA /  NWPU VHR                               & Aerial                     \\
    \cellcolor{white}                                                                                     & Meta-FRCNN    \cite{han2022meta}                          & AAAI                 & 2022                 & Faster RCNN               & Yes                  & Pascal / COCO                                  & Natural                    \\
    \cellcolor{white}                                                                                     & Meta-DETR     \cite{zhang2021meta}                        & TPAMI                & 2021                 & DETR                      & No                   & Pascal / COCO                                  & Natural                    \\
    \cellcolor{white}                                                                                     & DRL           \cite{liu2021dynamic}                       & Preprint             & 2021                 & Faster RCNN               & Yes                  & Pascal / COCO                                  & Natural                    \\
    \cellcolor{white}                                                                                     & DANA          \cite{chen2021should}                       & TM                   & 2021                 & Faster RCNN               & Yes                  & Pascal / COCO                                  & Natural                    \\
    \cellcolor{white}                                                                                     & SP            \cite{xu2021few}                            & Access               & 2021                 & Faster RCNN               & Yes                  & Pascal / COCO                                  & Natural                    \\
    \cellcolor{white}                                                                                     & JCACR         \cite{chu2021joint}                         & ICIP                 & 2021                 & YOLO                      & Yes                  & Pascal / COCO                                  & Natural                    \\
    \cellcolor{white}                                                                                     & TI-FSOD       \cite{li2021transformation}                 & CVPR                 & 2021                 & Faster RCNN               & Yes                  & Pascal / COCO                                  & Natural                    \\ \rowcolor{OliveGreen!50}
    \cellcolor{white}                                                                                     & SAM           \cite{huang2021few}                         & MDPI                 & 2021                 & Faster RCNN               & No                   & NWPU VHR-10 / DIOR                             & Aerial                     \\
    \cellcolor{white}                                                                                     & FSOD-FCT      \cite{han2022few}                           & CVPR                 & 2022                 & Faster RCNN               & No                   & Pascal / COCO                                  & Natural                    \\
    \cellcolor{white}                                                                                     & SAR-DRM       \cite{chen2022few}                          & TGRS                 & 2022                 & Faster RCNN               & No                   & FUSAR-GEN                                      & Aerial ${\mathsection}$                    \\
    \cellcolor{white}                                                                                     & FSOD-PSI      \cite{ouyang2022few}                        & JDT                  & 2022                 & YOLO                      & Yes                  & Pascal / COCO                                  & Natural                    \\
    \cellcolor{white}                                                                                     & SAFT          \cite{zhao2022semantic}                     & CVPR                 & 2022                 & FCOS                      & Yes                  & Pascal / COCO                                  & Natural                    \\
    \cellcolor{white}                                                                                     & APSP          \cite{lee2022few}                           & WACV                 & 2022                 & Faster RCNN               & No                   & Pascal / COCO                                  & Natural                    \\
    \cellcolor{white}                                                                                     & KFSOD         \cite{zhang2022kernelized}                  & CVPR                 & 2022                 & Faster RCNN               & Yes                  & Pascal / COCO                                  & Natural                    \\\rowcolor{OliveGreen!50}
    \cellcolor{white}                                                                                     & FSODS         \cite{zhou2022fsods}                        & TGRS                 & 2022                 & YOLO                      & Yes                  & SMCDD-FS                                       & Aerial ${\mathsection}$                   \\\rowcolor{OliveGreen!50}
    \cellcolor{white}                                                                                     & TIN-FSOD      \cite{liu2023transformation}                & Arxiv                & 2023                 & Faster RCNN               & Yes                  & NWPU VHR/ DIOR / HRRSD                         & Aerial                     \\
    \multirow{-34}{*}[0mm]{\cellcolor{white}\textbf{Attention}}                                           & FSOD-ICF      \cite{jiang2023few}                         & WACV                 & 2023                 & Faster RCNN               & Yes                  & Pascal / COCO                                  & Natural     \\\midrule
    \cellcolor{white}                                                                                     & PNPDet        \cite{zhang2021pnpdet}                      & WACV                 & 2021                 & Center Net                & No                   & Pascal / COCO                                  & Natural                   \\
    \cellcolor{white}                                                                                     & UPE           \cite{wu2021universal}                      & ICCV                 & 2021                 & Faster RCNN               & Yes                  & Pascal / COCO                                  & Natural                   \\
    \multirow{-3}{*}[0mm]{\cellcolor{white}\parbox{3.5cm}{\textbf{Attention / \newline Metric Learning}}} & GenDet        \cite{liu2021gendet}                        & NNLS                 & 2021                 & FCOS                      & Yes                  & Pascal / COCO                                  & Natural                 \\\midrule
    \cellcolor{white}                                                                                     & RepMet        \cite{karlinsky2019repmet}                  & CVPR                 & 2018                 & Faster RCNN               & Yes                  & Pascal / ImageNet Loc                          & Natural                    \\
    \cellcolor{white}                                                                                     & RN-FSOD       \cite{yang2020restoring}                    & NEURIPS              & 2020                 & Faster RCNN               & Yes                  & Pascal / ImageNet Loc                          & Natural                    \\
    \cellcolor{white}                                                                                     & MDODD         \cite{zhao2021morphable}${\dagger}$         & ICCV                 & 2021                 & Faster RCNN               & No                   & Pascal / COCO                                  & Natural                    \\
    \cellcolor{white}                                                                                     & FSCE          \cite{sun2021fsce}                          & CVPR                 & 2021                 & Faster RCNN               & Yes                  & Pascal / COCO                                  & Natural                    \\
    \multirow{-5}{*}[0mm]{\cellcolor{white}\textbf{Metric learning}}                                      & GD-FSOD       \cite{wu2021generalized}                    & NEURIPS              & 2021                 & Faster RCNN               & Yes                  & Pascal / COCO                                  & Natural   \\\midrule
    \cellcolor{white}                                                                                     & LSTD          \cite{chen2018lstd}                         & AAAI                 & 2018                 & Faster RCNN               & Yes                  & Pascal / COCO / ImageNet Loc                   & Natural                    \\
    \cellcolor{white}                                                                                     & MSPSR         \cite{wu2020multi}                          & ECCV                 & 2020                 & Faster RCNN               & Yes                  & Pascal / COCO                                  & Natural                    \\
    \cellcolor{white}                                                                                     & TFA           \cite{wang2020frustratingly}                & ICML                 & 2020                 & Faster RCNN               & Yes                  & Pascal / COCO / LVIS                           & Natural                    \\
    \cellcolor{white}                                                                                     & WOFG          \cite{fan2021generalized}${\dagger}$        & CVPR                 & 2021                 & Faster RCNN               & Yes                  & Pascal / COCO                                  & Natural                    \\
    \cellcolor{white}                                                                                     & Hallu-FSOD    \cite{zhang2021hallucination}               & CVPR                 & 2021                 & Faster RCNN               & Yes                  & Pascal / COCO                                  & Natural                    \\\rowcolor{OliveGreen!50}
    \cellcolor{white}                                                                                     & DHP           \cite{wolf2021double}                       & ICCVW                & 2021                 & Faster RCNN               & Yes                  & iSAID /  NWPU VHR                              & Aerial                     \\
    \cellcolor{white}                                                                                     & LVC           \cite{kaul2022label}                        & CVPR                 & 2021                 & Faster RCNN               & No                   & Pascal / COCO                                  & Natural                    \\
    \cellcolor{white}                                                                                     & FSCN          \cite{li2021few}                            & CVPR                 & 2021                 & Faster RCNN               & Yes                  & Pascal / COCO                                  & Natural                    \\
    \cellcolor{white}                                                                                     & FADI          \cite{cao2021few}                           & NEURIPS              & 2021                 & Faster RCNN               & Yes                  & Pascal / COCO                                  & Natural                    \\
    \cellcolor{white}                                                                                     & DeFRCN        \cite{qiao2021defrcn}                       & ICCV                 & 2021                 & Faster RCNN               & Yes                  & Pascal / COCO                                  & Natural                    \\\rowcolor{OliveGreen!50}
    \cellcolor{white}                                                                                     & SIMPL         \cite{xu2022simpl}                          & TAEORS               & 2022                 & YOLO                      & No                   & xView (plane only)                             & Aerial                     \\
    \cellcolor{white}                                                                                     & DETReg        \cite{bar2022detreg}                        & CVPR                 & 2022                 & Deformable DETR           & Yes                  & COCO                                           & Natural                    \\
    \cellcolor{white}                                                                                     & CFA           \cite{guirguis2022cfa}${\dagger}$           & CVPRW                & 2022                 & Faster RCNN               & No                   & Pascal / COCO                                  & Natural                    \\\rowcolor{OliveGreen!50}
    \cellcolor{white}                                                                                     & CIR           \cite{wang2022context}                      & TGRS                 & 2022                 & Faster RCNN               & Yes                  & NWPU VHR-10 / DIOR                             & Aerial                     \\
    \cellcolor{white}                                                                                     & NIMPE         \cite{liu2022novel}                         & ICASSP               & 2022                 & Faster RCNN               & Yes                  & COCO                                           & Natural                    \\
    \cellcolor{white}                                                                                     & HDA           \cite{she2022fast}                          & IROS                 & 2022                 & Faster RCNN               & Yes                  & COCO                                           & Natural                    \\
    \cellcolor{white}                                                                                     & MDB           \cite{wu2022multi}                          & LNCS                 & 2022                 & Faster RCNN               & No                   & Pascal / COCO                                  & Natural                    \\
    \cellcolor{white}                                                                                     & DCB           \cite{gao2022decoupling}${\dagger}$         & NEURIPS              & 2022                 & Faster RCNN               & Yes                  & Pascal / COCO                                  & Natural                    \\\rowcolor{OliveGreen!50}
    \cellcolor{white}                                                                                     & CPP-FSOD      \cite{lin2023effective}                     & Preprint             & 2023                 & Faster RCNN               & Yes                  & Pascal / COCO                                  & Natural                    \\
    \multirow{-21}{*}[0mm]{\cellcolor{white}\parbox{3cm}{\textbf{Fine-tuning Strategy}}}                  & I-DETR        \cite{dong2022incremental}${\ddagger}$      & AAAI                 & 2023                 & Deformable DETR           & No                   & Pascal / COCO                                  & Natural      \\\midrule
    \cellcolor{white}                                                                                     & MetaDet       \cite{wang2019meta}                         & ICCV                 & 2019                 & Faster RCNN               & No                   & Pascal / COCO                                  & Natural                \\
    \multirow{-2}{*}[0mm]{\cellcolor{white}\textbf{Meta-Learning}}                                        & Sylph         \cite{yin2022sylph}${\ddagger}$             & CVPR                 & 2022                 & Faster RCNN               & No                   & COCO / LVIS                                    & Natural         \\\midrule
    \cellcolor{white}                                                                                     & TL-ZSOD       \cite{rahman2019transductive}               & ICCV                 & 2019                 & RetinaNet                 & Yes                  & COCO                                           & Natural         \\
    \multirow{-2}{*}[0mm]{\cellcolor{white}\parbox{3.5cm}{\textbf{Zero-shot \newline Object Detection}}}  & ML-CMP        \cite{han2022multimodal}                    & Preprint             & 2022                 & Faster RCNN               & No                   & Pascal / COCO                                  & Natural           \\ \midrule
    \cellcolor{white}                                                                                     & OA-FSUI2IT    \cite{zhao2022oa}                           & AAAI                 & 2022                 & Faster RCNN               & Yes                  & Multiple datasets                              & Natural                   \\
    \cellcolor{white}                                                                                     & Acro FOD     \cite{gao2022acrofod}                        & ECCV                 & 2022                 & YOLO                      & Yes                  & Multiple datasets                              & Natural                   \\
    \cellcolor{white}                                                                                     & CD-CutMix     \cite{nakamura2022few}                      & ACCV                 & 2022                 & Faster RCNN               & No                   & Multiple datasets                              & Natural                   \\\rowcolor{OliveGreen!50}
    \cellcolor{white}                                                                                     & CD-FSOD       \cite{xiong2022cd}                          & Preprint             & 2022                 & Faster RCNN               & Yes                  & Multiple datasets                              & Aerial                    \\\rowcolor{OliveGreen!50}
    \multirow{-5}{*}[0mm]{\cellcolor{white}\textbf{Cross-domain}}                                         & CD-MDB        \cite{lee2022rethinking}                    & ECCV                 & 2022                 & Faster RCNN               & Yes                  & Multiple datasets                              & Aerial         \\ \bottomrule[1pt]
    \end{tabular}%
    } \caption[FSOD literature review]{List of the most relevant contributions
    to the Few-Shot Object Detection field. These works are grouped according to
    the general approach employed to tackle FSOD and sorted by their year of
    publication. Green rows signify that the methods were applied to aerial
    images and $\mathsection$ indicates that these images are SAR
    images. $\dagger$ signals that it was applied to generalized FSOD while
    $\ddagger$ means that it was developed in an incremental setting.}
    \label{tab:fsod_comparison}
\end{table}
\midsepdefault

\begin{figure}
    \centering
    \includegraphics[width=\textwidth]{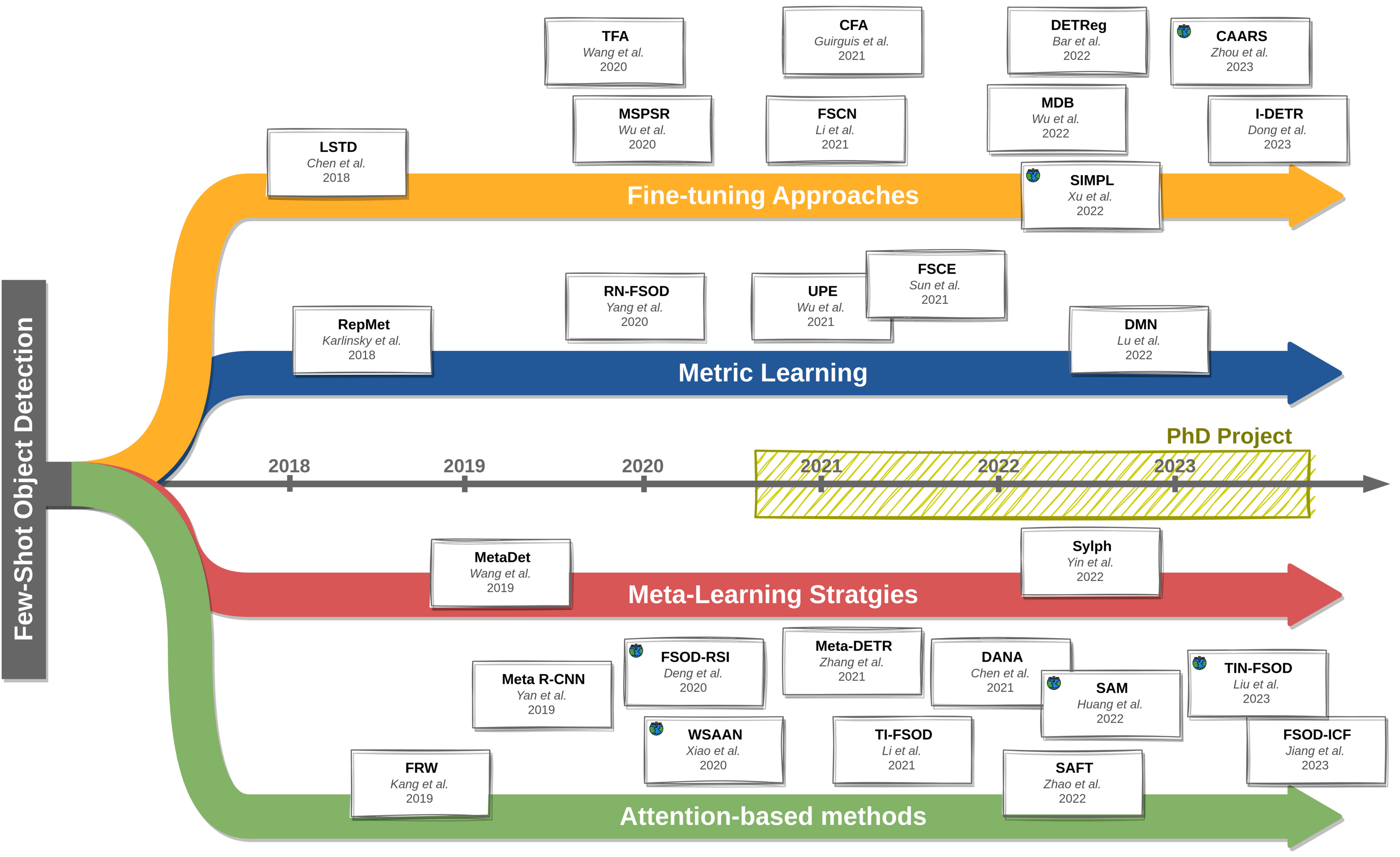}
    \caption[FSOD literature timeline]{Timeline of the FSOD literature, several
     works are included as milestones for each of the four kinds of approaches
     to FSOD: Fine-tuning, Metric Learning, Meta-Learning and Attention-based
     approaches. The yellow hatched rectangle represents the duration of this
     PhD project.}
    \label{fig:fsod_timeline}
\end{figure}

\subsection{Fine-tuning}

Fine-tuning is the simplest approach to tackle FSOD, the principle is quite
straightforward and similar to FSC: train a detection model to detect base
classes on a large dataset and then fine-tune it on novel classes with the few
available annotations. This is leveraged by Low Shot Transfer Detector (LSTD)
\cite{chen2018lstd}, a pioneer work on FSOD. It first trains a Faster R-CNN on a
base dataset and fine-tunes it on a support set containing only some examples of
the novel classes. Regularization losses are introduced to prevent overfitting.
Before fine-tuning, the last layer of the classifier branch is replaced with a
randomly initialized layer with the right number of outputs (\ie the number of
novel classes $|\mathcal{C}_{\text{novel}}|$). Closely related,
\cite{wang2020frustratingly} leverages the same idea without any additional loss.
Instead, they propose a Two-stage Fine-tuning Approach (TFA), which freezes most
of the network after base training. TFA is then extended by Constraint-based
Fine-tuning Approach (CFA) \cite{guirguis2022cfa} which leverages a technique
borrowed from Continual Learning: Average Gradient Episodic Memory. It applies
orthogonality constraints to the gradient during fine-tuning to prevent
forgetting base knowledge. This mostly helps for the generalized FSOD setting,
but it is also beneficial for regular FSOD. Another extension on top of the
basic fine-tuning approaches is to add a refinement step to filter the bounding
boxes predicted by the fine-tuned network. For instance, \cite{li2021few}
proposes a Few-Shot Correction Network (FSCN) whose goal is to assist the
detector classification branch. It is trained directly on the false positive of
the detector to specifically target challenging situations. Similarly,
\cite{kaul2022label} leverages a kNN classifier to “verify” the predicted labels
and lightweight bounding regressors to “correct” the predicted localizations.
Multi-scale Positive Sample Refinement (MSPSR) \cite{wu2020multi} also proposes
a proposal refinement strategy by leveraging a multiscale refinement branch. It
provides a better balance between positive and negative samples and makes both
base training and fine-tuning more efficient.

Another line of work addresses the FSOD problem through an augmentation
perspective. It circumvents the low few-annotated examples and the overfitting
risk by enriching the support set with more or less elaborated augmentations. An
easy and effective solution is to crop and paste novel instances directly inside
base images \cite{lin2023effective}. During fine-tuning, images from the base
dataset and support examples are randomly sampled. The support examples are
cropped and pasted into the base images. This significantly boosts the
fine-tuning procedure and improves FSOD performance. Similarly, Synthetic object
IMPLantation (SIMPL) \cite{xu2022simpl} leverages 3D models for each novel class
to generate high-quality augmented images. SIMPL completely blends the augmented
object inside the image, whereas \cite{lin2023effective} pastes some background
around the novel object as well. SIMPL leverages external information about the
classes and requires access to 3D models of the classes which is not always
possible. However, this opens an opportunity for addressing the even more
challenging zero-shot setting.  
Pushing even further, \cite{zhang2021hallucination} proposes a generative model
to enrich the support set and improve detection quality. The
\textit{hallucinator} model is trained jointly with the detector in an EM-like
procedure. First, the hallucinator is trained with the detector
classification loss (the detector is kept frozen). Then, the detector is trained
while the hallucinator provides more support examples (with the hallucinator now
frozen).

Finally, some other works leverage the fine-tuning strategy with other tricks.
Novel Instances Mining with Pseudo-Margin Evaluation (NIMPE) \cite{liu2022novel}
build a mining network to extract pseudo-labels from the base dataset. Reference
\cite{wu2022multi} fine-tunes both the classifier and the regressor of Faster
R-CNN with an additional distillation loss based on pseudo-labels. Pseudo-labels
are computed in a metric learning fashion between the query feature and the
prototype features stored in a memory bank. Few-shot object detection via
Association and DIscrimination (FADI) \cite{cao2021few} splits the fine-tuning
step into association and discrimination steps. During the association step, the
network is fine-tuned to map novel classes onto base classes. It leverages the
well-structured base class representation space learned during base training and
separates novel classes. Then, the discrimination step disentangles base and
novel class representations with a dedicated margin loss. 

Considering FSOD as a hierarchical refinement \cite{she2022fast} is also a
viable option as it breaks down base classes into novel classes. While this
setup is certainly relevant for many applications, it differs from the commonly
adopted FSOD setting.  

\subsection{Metric-learning based methods}
\label{sec:fsod_metric}
Metric-learning-based methods are extensively employed for few-shot
classification. Metric learning is designed especially for classification and
cannot handle bounding box regression. Thus, it cannot be directly applied to
object detection. However, several attempts were made to tackle FSOD with
metric-learning techniques, mostly replacing the classification branch of the
model with prototypical networks or closely related variants and keeping the
regression branch unchanged. Of course, even the classification adaptation is
not straightforward as object detection includes a special background class that
should be processed with care. Among these attempts, RepMet
\cite{karlinsky2019repmet} learns class representative vectors to classify
Regions of Interest (RoI) in Faster R-CNN according to their distance to the
closest class prototype. Class vectors are initialized with support image
representations and then fine-tuned via backpropagation. The fine-tuning is
based on a cross-entropy loss and a margin metric loss which favors tight
clusters in the embedding space. The background class probability is computed as
the complementary probability of the most probable class. Closely related,
Plug-and-Play Detectors (PNPDet) \cite{zhang2021pnpdet} learns prototype vectors
as well as scale factors. In addition, they replace the Euclidean distance from
RepMet with a Cosine similarity measure. Similarly, FSCE \cite{sun2021fsce} adds
a contrastive head on top of a pre-trained detector during fine-tuning. This
head outputs embedding for each RoI. A contrastive loss is optimized to bring
closer the representations of same-class RoI and repel RoI with no objects.
Likewise, \cite{zhao2021morphable} leverages prototypes as well but deals with the
background class separately with a learnable binary classifier. 

Plenty of other works leverage class prototypes for the classification part of
the detector. However, various tricks are proposed in the literature to improve
the quality and use of the prototypes. For instance, Universal Prototype
Enhancement (UPE) \cite{wu2021universal} refines prototypes with affine
transformation to convert image-level representations into object-level
prototypes much more adapted to the detection task. Also, it does not leverage
the prototypes directly as a classifier but rather uses them to adapt query
features before classification and regression. Similarly, GenDet (Generate
Detectors from Few Shots) \cite{liu2021gendet} combines the technique from
RepMet and UPE, \ie learnable prototypes to adapt the query features. Negative
prototypes can also be learned to better deal with the background class
\cite{yang2020restoring}. Finally, some contributions manage two sets of
prototypes, arguing that one set is not optimal for adapting features for both
the classification and the regression. Decoupled Metric Network
\cite{lu2022decoupled} introduces a decoupled representation transformation to
adapt class prototypes for either classification or regression. Likewise,
\cite{wu2021generalized} splits the representations using Singular Value
Decomposition. Eigenvectors corresponding to the largest singular values
represent the main source of variance. The authors argue that this accounts for
the general adaptation between base and novel classes. They are leveraged for
adapting query features both for regression and classification. Other
eigenvectors only represent the inter-class variance and therefore, are only
used in the classification branch. The methods described in this paragraph are
slightly different from the ones at the beginning of this section. They all use
their representation vectors to update query features before a learnable
classification and regression module, instead of using them for direct
classification (\eg distance to the closest prototype). They highly resemble
some attention-based methods that will be presented in the next section.

\subsection{Attention-based methods}
\label{sec:fsod_attention}
As we briefly broached at the end of the previous section, a common technique
for FSOD is to adapt the features from the query image based on the support
images. This can be understood as an attention mechanism between the query and
support features as it highlights locations in the query feature map that are
similar to the support images. Another way to see this is to think of the
attention mechanism as an adaptive filtering layer. It filters the query
features map according to the support features. Highlighted locations in the
query map show features similar to the support images. Following the attention
module, the regression and classification are performed independently per class,
often using a shared, class-agnostic detection head (see
\cref{fig:fsod_attention_principle}). The query-support combination module takes
the query feature maps and the features from all novel classes as input
and outputs class-specific query feature maps. This will be explained in more
detail in \cref{chap:aaf} which presents a general framework to subsume
existing attention mechanisms for FSOD. To summarize, what we call here
attention-based FSOD methods are techniques that leverage support information to
adapt query features before the classification and regression branches.
Following this definition the methods presented at the end of
\cref{sec:fsod_metric} can be interpreted as attention methods. However, they
are presented from a metric learning perspective which is why they are not
discussed in the current section (they will be classified as "Attention/Metric
Learning” methods in \cref{tab:fsod_comparison}). 

\begin{figure}[h]
    \centering
    \includegraphics[width=\textwidth]{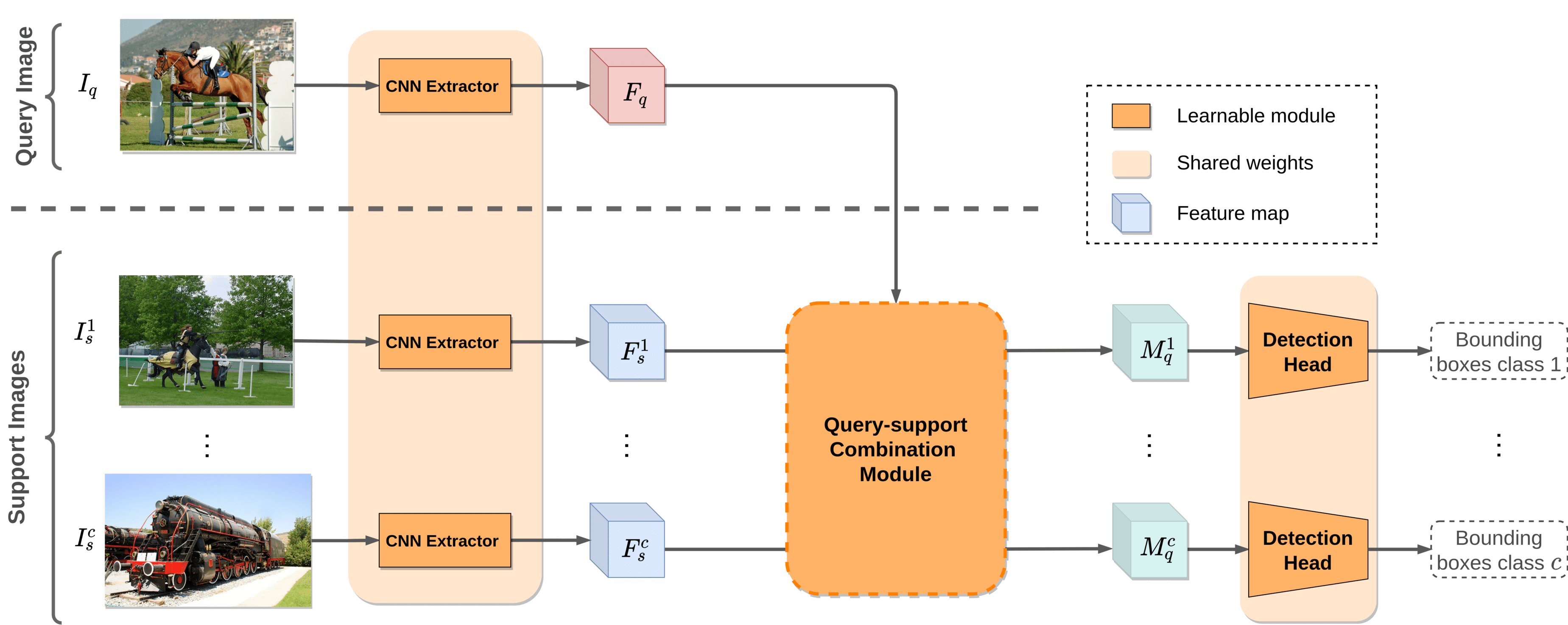}
    \caption{Attention-based FSOD principle}
    \label{fig:fsod_attention_principle}
\end{figure}

A seminal work in this field introduces Feature ReWeighting (FRW
) \cite{kang2019few}, which trains a reweighting module along with a YOLO
detector. The reweighting module produces class-specific feature vectors with a
Global Pooling layer (GP) applied on the support feature maps. These are then
channel-wise multiplied by the query features extracted by the backbone.
Hence, class-specific query features are generated, and the detection head
computes predictions for each class separately. This technique has been widely
re-used in the following literature, with other detection frameworks: Faster
R-CNN \cite{yan2019meta,fan2020fsod,xiao2020fsod,xiao2020few,kim2020few,wallach2019one},
CenterNet \cite{perez2020incremental,zhang2021pnpdet} or FCOS \cite{li2020one}.
The class reweighting vectors can be enriched by several tricks to
improve feature filtering. A few works
\cite{kim2020few,xiao2020fsod,liu2021dynamic} employ Graph Convolutional Networks
(GCNs) to combine and refine the reweighting vectors before the combination
modules. \cite{gao2021fast} finds optimal vectors through iterative
optimization. Others leverage multi-scale features to enrich class
representation \cite{deng2020few,gao2021fast}. 
 
This channel-wise multiplication between query and support features is a simple
form of attention. It can be thought of as an adaptive convolution layer, whose
weights depend on the support features. Incidentally, it is often interpreted
and implemented as such by existing works, approaching the meta-learning
paradigm. However, more complex attention mechanisms have been leveraged in the
literature. The incentive behind this improvement was the loss of spatial
support information and background feature contamination with the GP layer.
First, \cite{xu2021few} proposes a self-attention module to better highlight the
support object features and prevent background contamination. Very similar, Dual
AwareNess Attention (DANA) \cite{chen2021should} introduces a background
attenuation block for the same reason. However, DANA also leverages an alignment
mechanism to combine query and support features without losing spatial
information. This alignment module is quite close to the visual transformers'
attention. It encodes the features from the query images as \textit{queries} and
the features from a support image as \textit{keys} and \textit{values}.
\textit{Queries} here refer to the query-key-value (QKV) formulation of the
transformers, the correspondence with the query features is fortuitous. Queries
and keys are combined to form an attention map, which represents the similarity
between the query and support image patches. Then, the dot product between the
attention map and the values produces the aligned support features. It can be
understood as an alignment as it re-organizes spatially the support features to
match the spatial dimension of the query map, according to the similarity
between query and support. The underlying idea is that the same class objects in
the query and support images will likely have different aspects or poses.
Therefore, a direct comparison between the feature maps is often irrelevant.
The alignment procedure moves support features to similar locations in the query
map. This is illustrated in \cref{fig:fsod_spatial_align}, but more details will
be given in \cref{chap:aaf}. 

\begin{figure}
    \centering
    \includegraphics[width=0.9\textwidth,trim=50 0 0 0,clip]{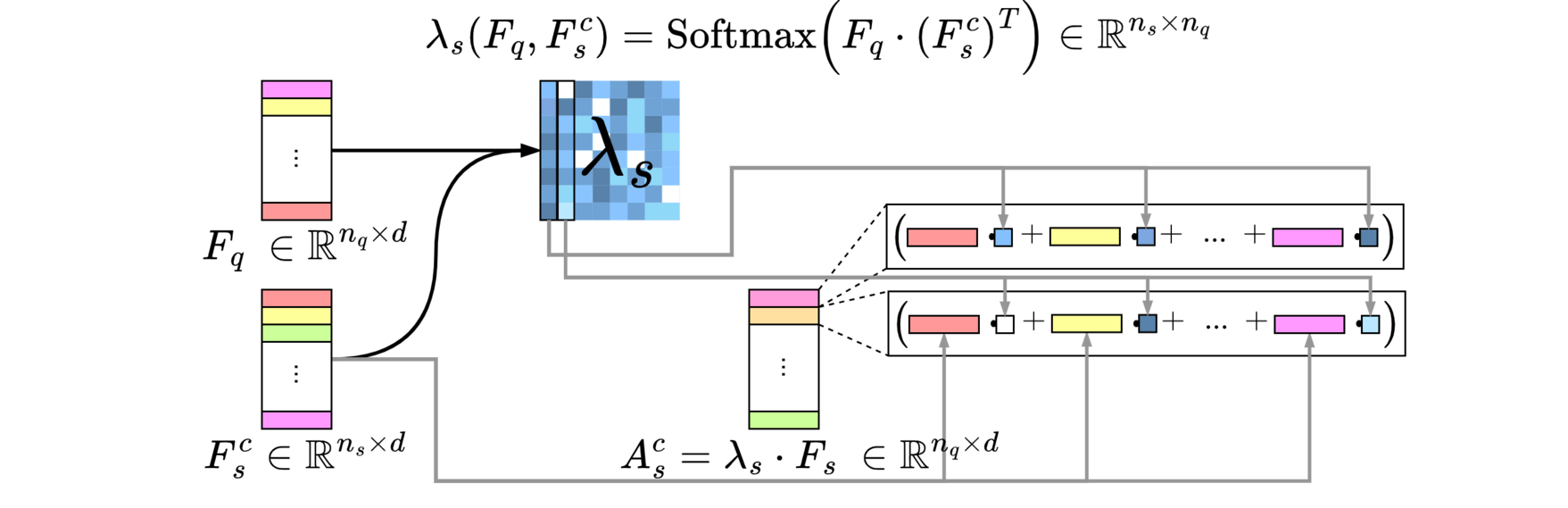}
    \caption[Query-Support Spatial Alignment]{Spatial alignment between query and support feature maps.
    Similarity matrix is computed as an outer product between the feature maps. For sake of
    clarity, maps are reshaped as 2-D matrix where the first dimension controls
    the spatial position in the map: $n_q$ positions for the query and $n_s$ for
    the support. $d$ is the number of channels. Similar colors mean that
    features are similar.}
    \label{fig:fsod_spatial_align}
\end{figure}

Hence, this alignment mechanism combines query and support features,
highlighting their similarity without losing spatial information. The same
technique is leveraged by several works
\cite{yan2019meta,han2022meta,han2022few,zhao2022semantic} with slight
variations. Similarly, \cite{lee2022few} uses QKV attention with globally pooled
support features. Instead, the authors propose to compute attention between the
query images and all support images for a class at once. In other works, the
attention is generally aggregated per class. Following the same idea, Meta-DETR
\cite{zhang2021meta} computes attention between a query image and all support
images at once. However, the authors do this for all classes at the same time
and replace the binary classification with a multi-class classification layer
(unlike most methods discussed above). To achieve this, a task encoding module
adapts the features for a specific task (\ie to the classes of interest) before
the classification head.

Of course, these are not the only attention mechanisms existing in the FSOD
literature. Some works derive other kinds of attention achieving
competitive results. Kernelized FSOD (KFSOD) \cite{zhang2022kernelized} proposes
elaborated kernel functions to combine query and support features in various
ways, which can be interpreted as attention. Differently, \cite{fan2020fsod}
trains three distinct branches that combine query and support in different
fashions, globally, locally, and patch-to-patch. Dynamic Relevance Learning
(DRL) \cite{liu2021dynamic} proposes a simpler way to combine query and support
features by simple point-wise operations (concatenation, multiplication, and
subtraction) after global pooling.   

In addition to these attention mechanisms, some works also propose additional
loss functions to improve the quality of the extracted features (query and
support). As an example, Transformation Invariant FSOD (TI-FSOD)
\cite{li2021transformation} leverages two losses to enforce robust feature
extraction. These losses are implemented as a distance between original and
augmented query or support features. The same principle is also proposed for
application on remote sensing images by \cite{liu2023transformation}. Another
technique is proposed by \cite{jiang2023few}, which computes a regularization
loss between two randomly sampled subsets of an RoI
feature. This regularization enforces consistency and robustness in the feature
space making the detection easier. Likewise, \cite{chu2021joint} derives a
reconstruction loss function by computing a low-rank matrix and reconstructing
the extracted query and support features. It enforces relevant latent structure
and alignment between query and support features. Finally, \cite{ouyang2022few}
introduces a loss function to promote orthogonality between classes in the
feature space.

\subsection{Meta-Learning}

While Meta-Learning methods are very common for FSC, they are much rarer in
the FSOD literature. It is explained easily by the difficulty of the task and
the complexity of meta-learning approaches. They often require training a
meta-learner model for generating weights or gradient updates to a smaller
classification network. However, detection models are much larger than the
classification ones, which makes the meta-learning approaches impractical for
FSOD. Nevertheless, the FSOD literature borrows some techniques from the
meta-learning field. In particular, most attention and metric-learning-based
methods for FSOD are trained using an episodic training strategy (see
\cref{tab:fsod_comparison}). In addition, many works are presented within the
meta-learning perspective because of the episodic training, here are some
examples: Meta Faster R-CNN \cite{han2022meta}, Meta R-CNN \cite{wu2020meta},
Meta-DETR \cite{zhang2021meta}, and GenDet \cite{liu2021gendet}. 

Yet, there are a few attempts at solving FSOD with meta-learning approaches.
MetaDet \cite{wang2019meta} extends Faster R-CNN with the MAML framework.
Specifically, they choose to generate weights only for the detection head of
Faster R-CNN, considering that the feature extractor and RPN are class-agnostic
and do not need adaptation. This significantly reduces the size of the generated
weights and makes it possible to use MAML. During base training, only the
detector is trained on the base dataset. Then a fine-tuning phase occurs, the
meta-learner is trained to predict the weight of the detection head only
provided with support examples for the novel classes. At the same time, the
detector is fine-tuned on the support set, with all class-agnostic parts frozen.
At test time, the meta-learner predicts weights for the novel classes to extend
the detector's head and the detector can be used as a regular detector.
Similarly, Sylph \cite{yin2022sylph} applies the same idea but only to the
classification branch of the detector, assuming that the regression is also
class-agnostic.

This section draws an almost exhaustive list of the contributions to the FSOD
field (see \cref{tab:fsod_comparison}). As for FSC, several research tracks
explore FSOD. However, meta-learning is a lot less popular approach for
detection compared to classification. Instead, attention-based methods (often
trained with an episodic strategy) are the mainstream approaches. Nevertheless,
there is no consensus about the best way to tackle FSOD, and fine-tuning or
metric-learning contributions are often simpler and still competitive. 

\section{Few-Shot Object Detection on Remote Sensing Images}
\vspace{-1em}
Few-Shot Object Detection is a relatively recent field in computer vision and so
far, it has been applied mostly to natural images and in particular on Pascal
VOC and MS COCO datasets (\cref{sec:fsod_dataset} explains how they are prepared
for the few-shot setting). However, there are a few contributions that apply
FSOD techniques on Remote Sensing Images (RSI), these are highlighted in green
in \cref{tab:fsod_comparison}. RSI are notoriously more challenging than natural
images for the detection task. Objects are smaller and more numerous, they can
be arbitrarily oriented, and the background is often more complex. Therefore,
object detection methods applied to RSI often comprise some tricks to better
deal with the specificities of RSI. FSOD techniques applied to RSI follow the
same trend. Among these tricks, the use of multiscale features is certainly the
most common. For instance, FSOD-RSI \cite{deng2020few} extends FRW with three
levels features map to better deal with small objects. Similarly,
\cite{gao2021fast,zhou2022fsods,wang2022context} leverage multiscale features
for either the query image, support images or both. FSOD applied to RSI is based
on attention-mechanisms
\cite{deng2020few,gao2021fast,zhou2022fsods,huang2021few,chen2022few,liu2023transformation}
or fine-tuning strategies
\cite{wolf2021double,xu2022simpl,wang2022context,zhou2022few}, but to our
knowledge, there is no FSOD method based completely on a metric-learning
approach.

Among these contributions, some tackle the detection problem with other
modalities than visible light. Indeed, this is a problem of interest as earth
observation is often conducted with non-visible light. Image quality is highly
dependent on the weather conditions, and half the earth at night is
unobservable with visible light. Therefore, a lot of applications rather use
infrared light or Synthetic-Aperture Radar (SAR). Two articles tackle few-shot
detection in SAR images \cite{chen2022few,zhou2022fsods}, yet without any
notable adjustment to account for the modality change.

These contributions are of particular interest for COSE as the goal is to design
efficient detection methods for high-resolution images. Some extensions for the
CAMELEON project are already planned with multi-spectral images and LIDAR.
Hence, methods able to adapt from one modality to another are especially valuable.

\section{Extension of the Few-Shot Object Detection Setting}
\vspace{-1em}
\label{sec:fsod_extensions}
As for Few-Shot Classification, the Few-Shot Object Detection setup has
several extensions. These settings are more challenging but reflect better
real-life use cases.

\paragraph*{One-shot and Zero-shot Object Detection}
First, in the case of extremely limited annotations, object detection is still
achievable. One-Shot OD has been addressed by several works
\cite{li2020one,wallach2019one,chu2021joint,zhao2022semantic} that we
present in the above section. These approaches are not different from the
few-shot setting, it simply is more difficult. However, in the zero-shot
setting, it becomes even more challenging as no image example is available for
the novel classes. The common approach in this setting is to leverage semantic
representations from the class labels and condition the detection on this
information \cite{bansal2018zero}. Recently introduced large
language-visual models such as CLIP \cite{radford2021learning} provide strong
improvements for various zero-shot tasks and object detection is no exception.
For instance, \cite{han2022multimodal} trains a prompt generator in a
meta-learning fashion to condition the detection on novel classes.
Alternatively, DINO \cite{caron2021emerging} and DETReg \cite{bar2022detreg}
conceive strong self-supervised pre-training schemes specifically adapted for
object detection, which translate into impressive performance in a low shot
setting. Finally, \cite{rahman2019transductive} leverage a transductive
pseudo-labeling approach to improving zero-shot detection. To our knowledge,
this is the only transductive method applied to few-shot detection.

\paragraph*{Generalized and Incremental FSOD}
The goal of FSOD is to adapt to novel classes; however, in many cases, the
performance on base classes matters as well. This is the case in the Generalized
FSOD (G-FSOD) setting where we are interested in detecting both base and novel
classes. The incremental setting extends G-FSOD with several adaptations to
novel classes without forgetting the previously seen classes. For G-FSOD, a
naive approach is to train two detectors: one on base classes and one on all
classes (base and novel), as a fine-tuned version of the first one. Outputs from
both detectors are combined at test time to achieve better performance on base
and novel classes. This is adopted by \cite{fan2021generalized} with two Faster
R-CNN. However, other contributions propose more sophisticated methods. CFA
\cite{guirguis2022cfa}, for instance, proposes a regularization loss to prevent
forgetting base classes. Alternatively, \cite{gao2022decoupling} duplicates the
detection head to process separately the foreground and background samples and
prevents classification bias toward base classes. In the incremental setting,
the naive approach from \cite{fan2021generalized} does not scale as it would
require duplicating the detector each time novel classes are added. Instead,
\cite{perez2020incremental,yin2022sylph} train a meta-network to generate
classifier weights for novel classes on-the-fly unlocking convenient adaptation.
Incremental DETR \cite{dong2022incremental} adopts a different strategy based on
fine-tuning and distillation to prevent forgetting previously seen classes. 

\paragraph*{Cross-domain Few-Shot Object Detection}
Last but not least, Cross-Domain FSOD (CD-FSOD) tackles the few-shot object
detection task in the context of domain adaptation. CD-FSOD aims at designing
methods able to generalize to new kinds of images. Just as for classification,
two sub-tasks have been explored in the literature: CD-FSOD with and without
class shift. For COSE, both tasks are relevant but CD-FSOD with class
shift precisely corresponds to their application. Indeed, once a system is in
operation, it will likely encounter new objects and domains. While the images
will always be taken from above, their general aspect may change a lot due to
weather conditions, different landscapes or carrier altitude. Therefore, solving
CD-FSOD with class shift is crucial for the CAMELEON system. 

However, this field remains barely untouched. To our knowledge, only a few
contributions tackle CD-FSOD. First, without class shift, several works address
this problem with augmentation-based approaches. The idea is to leverage the few
target examples to augment the source images so that they become plausible
samples from the target domain. For instance, \cite{gao2022acrofod} proposes a
directive data augmentation procedure that optimally augments the source
examples, so their features are close to the features of the target examples.
The detector is then trained as a regular detector on the augmented examples.
Likewise, \cite{wang2019few,zhao2022oa} propose source-to-target translation
networks that convert source images into target images. These networks are
trained adversarially with a discriminator that aims to distinguish between
domains. Closely related, Cross-Domain CutMix \cite{nakamura2022few} crafts an
augmentation technique that mixes two domains by cropping and pasting objects
from the target domain into source images and vice-versa. \\

The methods discussed above assume that the source and target domains share the
same label space, \ie they have the same classes. This assumption allows for
building simple domain augmentation approaches as once the source images are
translated into the target domain, their annotations can be directly used for
training. When classes shift between source and target domains, this is no
longer an option. In this very challenging setting, most FSOD methods showcase
poor results and adaptation capabilities. Naive fine-tuning is often the best
alternative in this case. To our knowledge, there are still no contributions
that tackle this task. Only two articles \cite{xiong2022cd,lee2022rethinking}
pave the way for future research in this direction with dedicated benchmarks and
baselines. Both propose to study the generalization capabilities of FSOD methods
on various datasets after a large base training on COCO dataset. First,
\cite{xiong2022cd} combines three datasets: 1) ArTaxOr \cite{artaxor2020}, a
close-up insect images dataset, 2) UODD \cite{jiang2021underwater}, an
underwater image dataset, and 3) DIOR. The authors propose a simple
self-distillation strategy, similar to self-supervised approaches, during the
fine-tuning on the new domains. Second, \cite{lee2022rethinking} builds a more
complete benchmark called Multi-dOmain FSOD (MoFSOD) with 10 different target
domains. However, the authors study the influence of two different source
datasets COCO and LVIS \cite{gupta2019lvis}. They also provide a domain distance
measure that assesses the similarity between a dataset and COCO. This measure is
the recall of a detector trained on COCO applied to a dataset in a
class-agnostic manner. Intuitively, if a dataset is close to COCO (in terms of
classes and aspects), the trained detector will detect a lot of objects (even if
the classes are wrong) and will have a high recall. Based on this similarity
measure they study the impact of freezing some layers of the detector during
the fine-tuning. Previous works recommend only fine-tuning the detection head
while keeping the backbone frozen. However, \cite{lee2022rethinking} shows that
this is true only for sufficiently similar datasets. In other words, when the
source-target gap is large, it is better to fine-tune the model entirely for
better adaptation. Unfortunately, the authors did not provide an easy-to-use
meta-dataset for future research and addressing CD-FSOD remains challenging
due to complex initial data processing.

\section{Dataset preparation and evaluation in the Few-Shot setting}
\label{sec:fsod_dataset}
\vspace{-1em}

\subsection{Adapting detection datasets in the Few-Shot setting}
There are no specific datasets for Few-Shot Object Detection. Instead, regular
detection datasets can be adapted to the few-shot setting. In this section, we
describe this process in the case of the four datasets on which this PhD project
mainly focuses: DOTA \cite{xia2018dota}, DIOR \cite{li2020object}, Pascal VOC
\cite{everingham2010pascal} and COCO \cite{lin2014microsoft}. 

The conversion of a dataset for the $N$-ways $K$-shots setting is
straightforward. First, the set of classes is divided into two sets: the base
and novel class sets (with $|\mathcal{C}_{\text{novel}}| = N$). The class split
for each dataset is fixed by common practices (for Pascal VOC and COCO) or taken
at random (for DOTA and DIOR) when no convention is set in the literature.
\cref{tab:fsod_class_splits} gives the class split that will be used throughout
this PhD thesis. Then, the instances of the novel classes are filtered from the
dataset to keep only $K$ images per novel class. This filtering operation is
performed in two steps:

\begin{enumerate}
    \item for each novel class $c\in \mathcal{C}_{\text{novel}}$, $K$ images
    containing at least one instance of class $c$ are selected as the support
    examples. They constitute the support set.
    \item the instances of the novel classes are removed from all other images
    in the dataset. 
\end{enumerate}

This choice is motivated by the presence of \textit{distractors}. This concept
and further explanations about how they influence the few-shot training in
detection will be presented in \cref{sec:fsod_distractors}. It is important to
note that as FSOD is a recent field, several preparation techniques coexist in
the literature. However, the one described above seems the most reasonable and
common in current FSOD works. 

Nevertheless, the existence of various preparation settings makes the comparison
with existing methods difficult as this choice is rarely discussed in the
articles. It may be challenging to figure out precisely how the datasets were
prepared, not to mention the choice of the training strategy (\ie episodic or
direct fine-tuning). The reported FSOD performance in the literature should then
be regarded with a critical eye.

\begin{table}[]
    \centering
    \resizebox{0.85\textwidth}{!}{%
    \rowcolors{2}{gray!25}{white}
    \begin{tabular}{@{\hspace{5pt}}lP{6cm}P{6cm}@{\hspace{5pt}}}
    \toprule[1pt]
                        & \textbf{Novel classes}                              & \textbf{Base classes}                                                                                                                                                              \\ \midrule 
    \textbf{Pascal VOC} & bird, bus, cow, motorbike, sofa                                        & aeroplane, bicycle, boat, bottle, car, cat, chair, diningtable, dog, horse, person, pottedplant, sheep, train, tvmonitor                                                                                                                                              \\ \addlinespace[0.5em]
    \textbf{MS COCO}    & person, bicycle, car, motorcycle, airplane, bus, train, boat, bird, cat, dog, horse, sheep, cow, bottle, chair, couch, potted plant, dining table, tv & truck, traffic light, fire hydrant, stop sign, parking meter, bench, elephant, bear, zebra, giraffe, backpack, umbrella, handbag, tie, suitcase, frisbee, skis, snowboard, sports ball, kite, baseball bat, baseball glove, skateboard, surfboard, tennis racket, wine glass, cup, fork, knife, spoon, bowl, banana, apple, sandwich, orange, broccoli, carrot, hot dog, pizza, donut, cake, bed, toilet, laptop, mouse, remote, keyboard, cell phone, microwave, oven, toaster, sink, refrigerator, book, clock, vase, scissors, teddy bear, hair drier, toothbrush \\ \addlinespace[0.5em]
    \textbf{DOTA}       & storage-tank, tennis-court, soccer-ball-field                                              & plane, ship, baseball-diamond, basketball-court, ground-track-field, harbor, bridge, small-vehicle, large-vehicle, roundabout, swimming-pool, helicopter, container-crane                                                                                                                                                    \\ \addlinespace[0.5em]
    \textbf{DIOR}       & airplane, baseball field, tennis court, train station, wind mill                                        & airport, basketball court, bridge, chimney, dam, expressway service area, expressway toll station, golf course, ground track field, harbor, overpass, ship, stadium, storage tank, vehicle                                                                                                                                              \\ \bottomrule[1pt]
    \end{tabular}%
    } 
    \caption[Base / Novel class splits]{Base / Novel class splits for the different datasets used
    throughout this thesis. The novel classes in the COCO dataset correspond to
    all classes in Pascal VOC.}
    \label{tab:fsod_class_splits}
    \end{table}

\subsection{Evaluation protocol for Few-Shot Object Detection}
\label{sec:fsod_evaluation}
The common practice in FSC is to randomly sample a support set from the training
split of the whole dataset, adapt the model with it (through fine-tuning or
direct adaptation), and finally, make the predictions and compute the relevant
metrics on the test split. This is repeated many times with different support
sets and the scores are finally averaged to give a robust evaluation of the
generalization capabilities of the models. 

For detection, the same principle should be applied to get a robust assessment
of the models' performance. However, the adaptation of such models is often
quite long compared to classification models. Indeed, detection models are much
larger than classification ones, thus they take more time to adapt to novel
classes. Before going further, we need to distinguish two approaches, on the
one hand, the fine-tuning strategy and on the other hand all other strategies
(\ie metric learning, meta-learning and attention-based methods). The main
distinction is that the latter use the support set during inference whereas
fine-tuning approaches only leverage it during the second phase of training.
\cref{fig:fsod_evaluation_strategies} illustrates the two different approaches
for FSOD model evaluation and exhibits a time estimation for training and
evaluating one model, following the general recommendations of FSC (\ie at least
100 repetitions with various support sets). 

\paragraph*{Evaluation of fine-tuning FSOD approaches}
Repeated evaluation requires fine-tuning the base model (\ie the model after
base training) with various support sets. Fine-tuning FSOD methods can take up
to a few hours and repeated evaluation may take days \footnote{a typical setup
in FSC is to repeat 100 times the adaptation, even if an FSOD model takes only
30 minutes to adapt to the novel classes through fine-tuning, the robust
evaluation would take almost 2 days.}, which is not practical. A reasonable
compromise is to perform a limited number of runs (between 10 and 30), which is
sufficient according to empirical studies in \cite{wang2020frustratingly}. Even
though, robust evaluation is still an intensive process in this setting. 

\paragraph*{Evaluation of other FSOD approaches}
Other methods adapt to novel classes at inference time given a support set.
Therefore, they can be more robustly evaluated (in a reasonable time) than their
fine-tuning counterparts. Adaptation is often fast compared to the fine-tuning
phase and can be more easily repeated. However, detection models that are based
on metric learning or attention still require a fine-tuning phase at least for
the regression branch. A support set must be used for this as
well and its choice influences the performance of the model, even if
adaptation is repeated multiple times after the fine-tuning. Yet this
fine-tuning step is often even more time-consuming than basic fine-tuning
approaches as the models are augmented with costly adaptation modules. In this
case, the common setting is to repeat the adaptation at inference multiple
times after only a single fine-tuning of the model. 
\vspace{-1mm}
These settings provide a sweet spot between evaluation robustness and
computation time. They will be employed in our all experiments unless specified,
both for fine-tuning, metric learning and attention-based approaches. 

\begin{figure}
    \centering
    \includegraphics[width=0.9\textwidth,trim=0 10 0 0,clip]{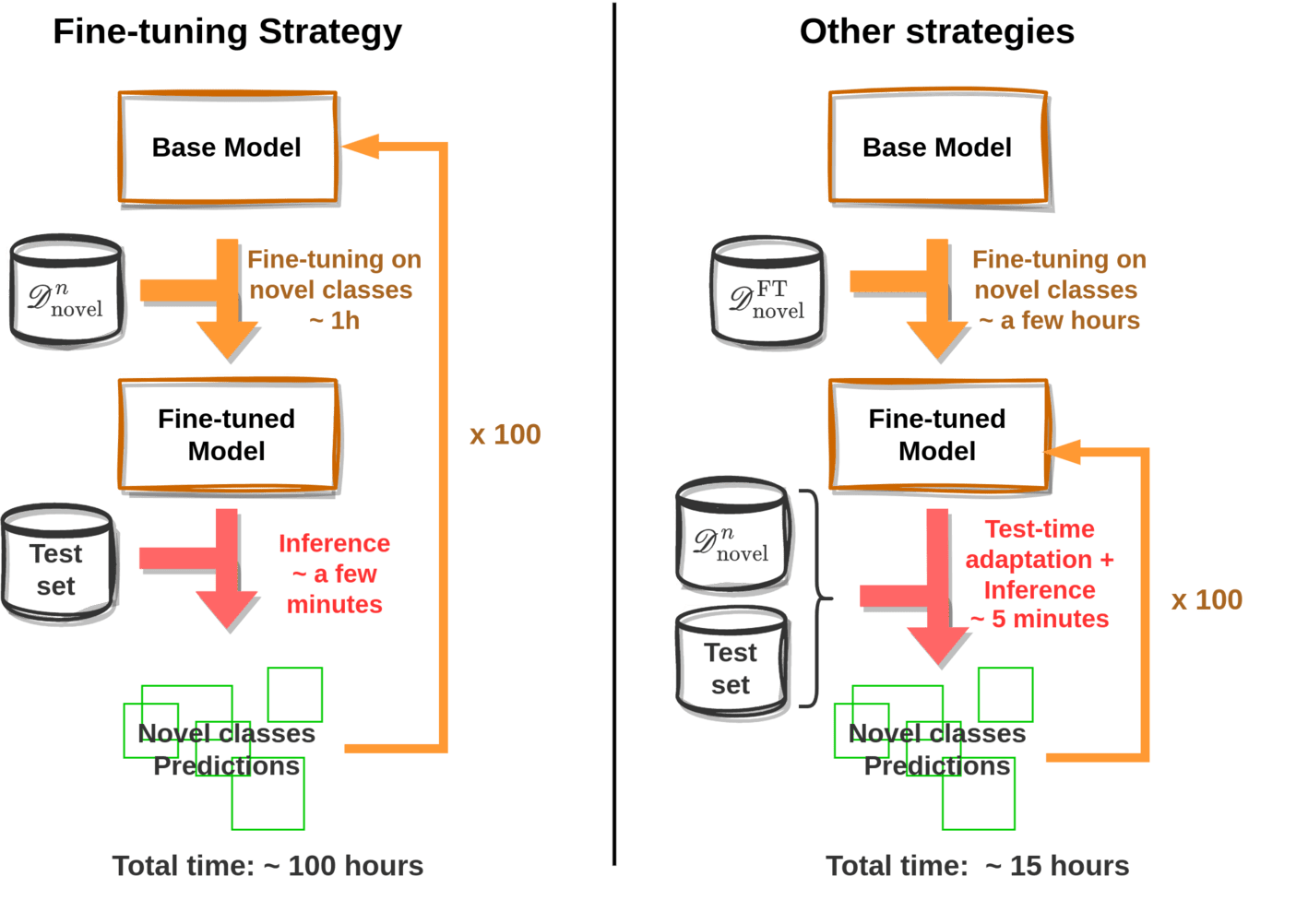}
    \caption[Distinct evaluation processes for fine-tuning and other
    methods]{Illustration of two evaluation processes existing in the
    literature. One for fine-tuning methods (left) and one available for all
    other methods (right).}
    \label{fig:fsod_evaluation_strategies}
    \vspace{-1.5em}
\end{figure}

\vspace{-2mm}
\section{Conclusion}
\vspace{-1em}
In this chapter, we reviewed the FSOD literature. This field is relatively
recent and fastly growing. It has significantly evolved since the beginning of
this project. In 2020, most FSOD works were based on attention-based approaches,
yet fine-tuning techniques are now getting more and more interest. This review
helps to understand the main directions that have already been explored and the
relevant tracks that need to be pursued.

\chapter{Understanding the Challenges of Few-Shot Object Detection}
\label{chap:aerial_diff}

\chapabstract{The detection task becomes extremely challenging when limited
annotated data is available. In this chapter, we explore the reasons behind this
difficulty. In particular, we focus on the case of aerial images for which it is
even harder to apply FSOD techniques. It turns out that small objects are
especially challenging for the FSOD task and are the main source of poor
performance in remote sensing images.} 
\vspace{-6mm}
\textit{
    \begin{itemize}[noitemsep]
        \item[\faFileTextO] P. Le Jeune and A. Mokraoui, "Improving Few-Shot Object Detection through a Performance Analysis on Aerial and Natural Images," 2022 30th European Signal Processing Conference (EUSIPCO), Belgrade, Serbia, 2022, pp. 513-517, doi: 10.23919/EUSIPCO55093.2022.9909878.
        \item[\faFileTextO] P. Le Jeune and A. Mokraoui, "Amélioration de la détection d’objets few-shot à travers une analyse de performances sur des images aériennes et naturelles." GRETSI 2022, XXVIIIème Colloque Francophone de Traitement du Signal et des Images, Nancy, France
    \end{itemize}
}

\PartialToC


In this chapter, we present our first contribution to the FSOD field.
Specifically, this section presents an analysis of the difficulties of going
from a regular to a few-shot data regime for the detection task, especially
for aerial images. 

\vspace{1em}
\section{Distractors in the Few-Shot Data Regime}
\label{sec:fsod_distractors}
\vspace{-1em}

First, changing the set of classes of interest during the training procedure
(\ie between base training and fine-tuning) is problematic. Classes considered
as background can become objects of interest, which goes against the knowledge
acquired during the base training phase. This is embodied by the concept of
\textit{distractors}, introduced in \cite{li2021few}. It refers to examples that
provide wrong supervision to a model during training. We choose to refine this
concept in two categories: \textit{self-distractors} and \textit{co-occurrence
distractors}. Self-distraction occurs when annotated and non-annotated instances
of a class are visible in the same image. The non-annotated instances are called
self-distractors. This can happen when there are annotation mistakes in a
dataset, but it can also happen in the few-shot settings. For instance, if
annotations of support examples are filtered (\eg to keep only one annotation
per support image), all the filtered instances will become self-distractors
during fine-tuning. When fine-tuning on the support examples, the annotated
instances of the novel classes will provide correct supervision to the model.
However, the non-annotated instances will be considered as background and wrong
supervision will be propagated in the model. This explains why it is more
sensible to keep all annotations of the novel classes in the support set even
though it does not fully comply with the original $N$-ways $K$-shots setting. In
the literature, this choice is barely discussed and early works in the field
employ either the strict one annotation per image sampling or the
\textit{self-distractor-free} sampling described above. Using
self-distractor-free sampling often results in improved performance; however, no
analysis was conducted to explain the origins of these gains (either coming from
more examples or thanks to more coherent supervision). In our experiments, both
setups were used as this issue was encountered in the middle of this project.
We will clearly specify what setting is used for all our experiments.

The second type of distraction, \textit{co-occurrence distraction}, happens when
novel class instances are visible in images during base training. Their
annotations have been filtered out, therefore they are considered as background.
Of course, this makes sense during base training as the novel classes are by
definition unknown at this point. However, in this setup, the model is
specifically trained to consider these objects as background, whereas if no
co-occurrences of the novel and base classes were allowed, no background
supervision would be given to novel class instances. This could be achieved by
removing all images containing such co-occurrences from the dataset. However,
this type of distractor is much less frequent than the self-distractors (see
\cref{fig:fsod_cooccurrences}). In addition, these mostly occur during the base
training phase, and even if they provide incorrect supervision, fine-tuning will
rectify it. Therefore, we choose to keep images with co-occurrence distractors
during base training.

\begin{figure}
    \centering
    \includegraphics[width=\textwidth]{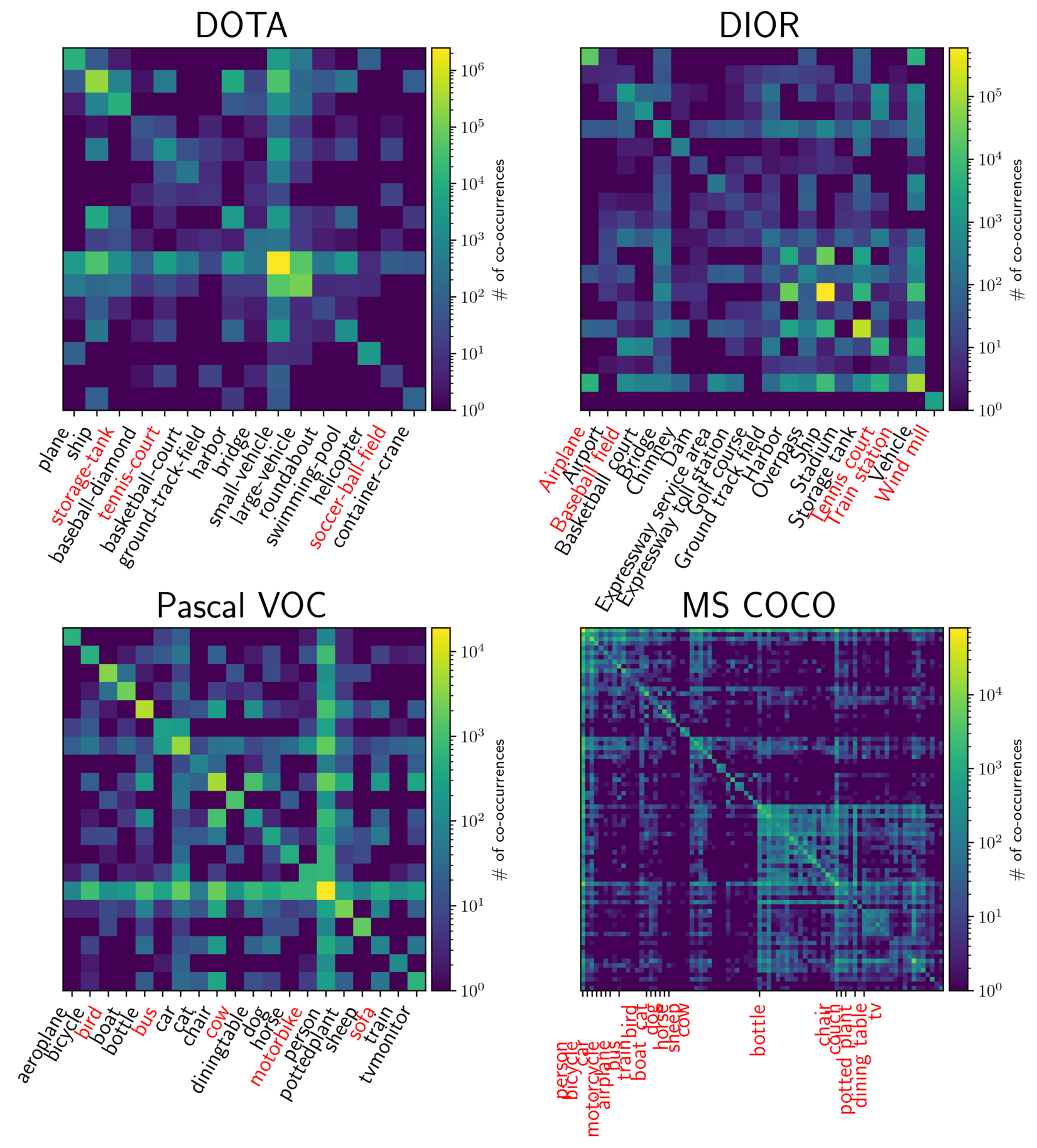}
    \caption[Class co-occurrences in DOTA, DIOR, Pascal VOC and MS COCO]{Co-occurrences between the classes of the four datasets of interest
     DOTA, DIOR, Pascal VOC and MS COCO. Novel classes are highlighted in red.
     For MS COCO only novel class labels are shown for clarity.}
    \label{fig:fsod_cooccurrences}
\end{figure}

\vspace{-0.5em}
\section{The Increased Challenge of Aerial Images}
\vspace{-1em}
The difficulties described in the previous section are not specific to any kind
of image. However, it appears from the scarce literature and our experiments
that applying FSOD on aerial images is much more difficult than on natural ones. 

At the beginning of this PhD, only very few works addressed FSOD in aerial
images. Among those, \textit{Few-Shot Object Detection via Feature Re-Weighting}
(FRW) \cite{xiao2020few} and  \textit{Few-Shot Object Detection With
Self-Adaptive Attention Network for Remote Sensing Images} (WSAAN)
\cite{deng2020few} were the most popular. These two works were not evaluated on
the same datasets, preventing any useful comparison. In addition, many
architectural choices differ from one another, for instance, the underlying
detection frameworks and the backbones. Therefore, we choose to implement these
methods within a single framework, preventing most architectural discrepancies.
Our proposed framework will be described thoroughly in \cref{chap:aaf}.   
We also re-implemented Dual AwareNess Attention (DANA) \cite{chen2021should} as
it was one of the best-performing methods on COCO dataset at that time. We
analyze here the behaviors of these three FSOD techniques both on natural and
aerial images. The main idea is to compare the performance of the three methods
in regular and few-shot data regimes. The regular data regime corresponds to the
vanilla detector (\ie without any modification for the few-shot setting) and
with full access to the novel class annotations in the dataset (\ie no
annotation filtering). In our re-implementation of FRW, WSAAN and DANA, the
underlying detector is FCOS \cite{tian2019fcos}. Thus, the regular baseline
is an FCOS detector trained on the full datasets. To conduct this experiment, we
only select DOTA, DIOR and Pascal VOC as they have roughly the same number of
classes. COCO however has 4 times more classes which brings
additional complexities. The performance results of the regular baseline are
available in \cref{tab:fsod_regular_baseline}. Specifically, FCOS is trained on
each dataset with full access to the annotations for both base and novel
classes. Then the mAP is computed on each class individually and averaged on
base and novel classes separately. This gives an overview of the performance of
the model respectively on base and novel classes in a regular data regime.

\begin{table}[]
    \centering
    \resizebox{0.85\textwidth}{!}{%
    \begin{tabular}{@{\hspace{3mm}}cclcclcc@{\hspace{3mm}}}
    \toprule[1pt]
    \multicolumn{2}{c}{\textbf{DOTA}} &  & \multicolumn{2}{c}{\textbf{DIOR}} &  & \multicolumn{2}{c}{\textbf{Pascal VOC}} \\ \midrule
    Base Classes mAP   & Novel mAP   &  & Base mAP   & Novel mAP   &  & Base mAP      & Novel mAP      \\
    60.87      & 69.69       &  & 72.82      & 81.48       &  & 65.47         & 68.02          \\ \bottomrule[1pt]
    \end{tabular}%
    } \caption[Regular baseline performance]{Regular baseline performance (mAP
    with a 0.5 IoU threshold) on DOTA, DIOR and Pascal VOC datasets (\ie trained
    with all annotations). The baseline model is FCOS \cite{tian2019fcos}, 
    trained on all classes (base and novel) and with all available annotations
    in each dataset. Then the mAP is computed on base and novel classes
    separately.}
    \vspace{-1em}
    \label{tab:fsod_regular_baseline}
\end{table}

It seems tempting here to extrapolate the FSOD performance on DOTA and DIOR from
the performance on Pascal VOC. The regular baseline (FCOS) achieves similar
performance on these two datasets, which contain the same number of classes and
roughly the same number of images. Thus, one could have expected close FSOD
performance on these datasets. This is quite different from the actual results
reported in \cref{tab:fsod_few_shot_baselines}. The FSOD performance on DOTA
and DIOR is significantly lower compared with the results on Pascal VOC. To
better visualize this finding, \cref{fig:fsod_baseline_res} represents the
few-shot performance as dark bars while regular baseline performance as lighter
rectangles. The height of the rectangle is set as the mAP on either the base or
novel classes (in blue and red respectively). This clearly illustrates the
different behaviors of FSOD methods applied on aerial or natural images.

\begin{table}[]
    \centering
    \resizebox{\textwidth}{!}{%
    \begin{tabular}{@{}cccccccccccccccccccccccccccccc@{}}
        \toprule[1pt]
        \multicolumn{1}{l}{} & \multicolumn{9}{c}{\textbf{DOTA}}                                                                                                                                                                                        &           &\multicolumn{9}{c}{\textbf{DIOR}}         &         & \multicolumn{9}{c}{\textbf{Pascal VOC}}                                                                                                                                                                                                                                                                                                                                                                                               \\ \cmidrule(lr){2-9} \cmidrule(lr){12-19} \cmidrule(lr){21-28}
                             & \multicolumn{2}{c}{\textbf{FRW}}                      & \textbf{} & \multicolumn{2}{c}{\textbf{WSAAN}}                                        & \textbf{} & \multicolumn{2}{c}{\textbf{DANA}}                                       & \textbf{}        &                      &\multicolumn{2}{c}{\textbf{FRW}}                      & \textbf{} & \multicolumn{2}{c}{\textbf{WSAAN}}                                 & \textbf{} & \multicolumn{2}{c}{\textbf{DANA}}                       &                      &\multicolumn{2}{c}{\textbf{FRW}}                      & \textbf{} & \multicolumn{2}{c}{\textbf{WSAAN}}                                 & \textbf{} & \multicolumn{2}{c}{\textbf{DANA}}                                         \\ \midrule
        $\boldsymbol{K}$     & \underline{Base} & \underline{Novel}                  &           & \underline{Base}                     & \underline{Novel}                  &           & \underline{Base}                   & \underline{Novel}                  &                  &                      & \underline{Base} & \underline{Novel}                 &           & \underline{Base}                    & \underline{Novel}&           & \underline{Base}                   & \underline{Novel}              &                      & \underline{Base} & \underline{Novel}                 &           & \underline{Base}                    & \underline{Novel}&           & \underline{Base}                   & \underline{Novel}                                \\
        1                    & 47.24             & \textcolor{red}{\textbf{13.35}}     &           & 45.55                                 & 12.19                               &           & \textcolor{blue}{\textbf{49.78}}    & \textcolor{red}{\textbf{12.52}}     & \textbf{}        &                      & 56.67             & 16.92                              &           & 56.41                                & 15.48             &           & \textcolor{blue}{\textbf{58.78}}    & \textcolor{red}{\textbf{20.64}} &                      & 59.92             & 28.22                              &           & 61.70                                & 30.94             &           & \textcolor{blue}{\textbf{62.58}}    & \textcolor{red}{\textbf{32.82}}                    \\
        3                    & 46.50             & \textcolor{red}{\textbf{25.32}}     &           & 44.18                                 & 24.42                               &           & \textcolor{blue}{\textbf{49.67}}    & 20.70                               &                  &                      & 58.05             & 25.08                              &           & 51.72                                & 13.84             &           & \textcolor{blue}{\textbf{59.14}}    & \textcolor{red}{\textbf{27.26}} &                      & 63.34             & 31.12                              &           & 63.52                                & \textcolor{red}{\textbf{42.19}}             &           & \textcolor{blue}{\textbf{64.18}}    & 33.95                   \\
        5                    & 48.60             & 29.57                               &           & 47.56                                 & \textcolor{red}{\textbf{31.44}}     &           & \textcolor{blue}{\textbf{53.49}}    & 24.96                               &                  &                      & 60.75             & 32.58                              &           & 60.79                                & 30.32             &           & \textcolor{blue}{\textbf{62.12}}    & \textcolor{red}{\textbf{34.16}} &                      & 64.35             & \textcolor{red}{\textbf{46.33}}                             &           & 64.68                                & 46.16             &           & \textcolor{blue}{\textbf{65.20}}    & 42.59                     \\
        10                   & 48.52             & \textcolor{red}{\textbf{37.10}}     &           & 46.72                                 & 35.12                               &           & \textcolor{blue}{\textbf{53.25}}    & 34.39                               &                  &                      & 61.47             & \textcolor{red}{\textbf{35.56}}    &           & \textcolor{blue}{\textbf{61.88}}     & 33.41             &           & \textcolor{blue}{\textbf{62.49}}    & \textcolor{red}{\textbf{36.43}} &                      & 63.16             & 48.71                              &           & \textcolor{blue}{\textbf{65.27}}     & \textcolor{red}{\textbf{51.70}}             &           & 65.03    & 50.30                    \\ \bottomrule[1pt]
        \end{tabular}%
    } \caption[Performance comparison between FRW, WSAAN and DANA on DOTA, DIOR
    and Pascal VOC datasets]{Comparison of $\text{mAP}_{0.5}$ of several methods
    on DOTA, DIOR and Pascal VOC datasets. For each method, mAP is reported for
    different numbers of shots $K \in \{1, 3, 5, 10\}$ and separately for base
    and novel classes. Blue and red values represent the best performance on
    base and novel classes respectively, for each dataset.}
    \label{tab:fsod_few_shot_baselines}
    \end{table}

\begin{figure}[t]
    \centering
    \includegraphics[width=0.9\textwidth]{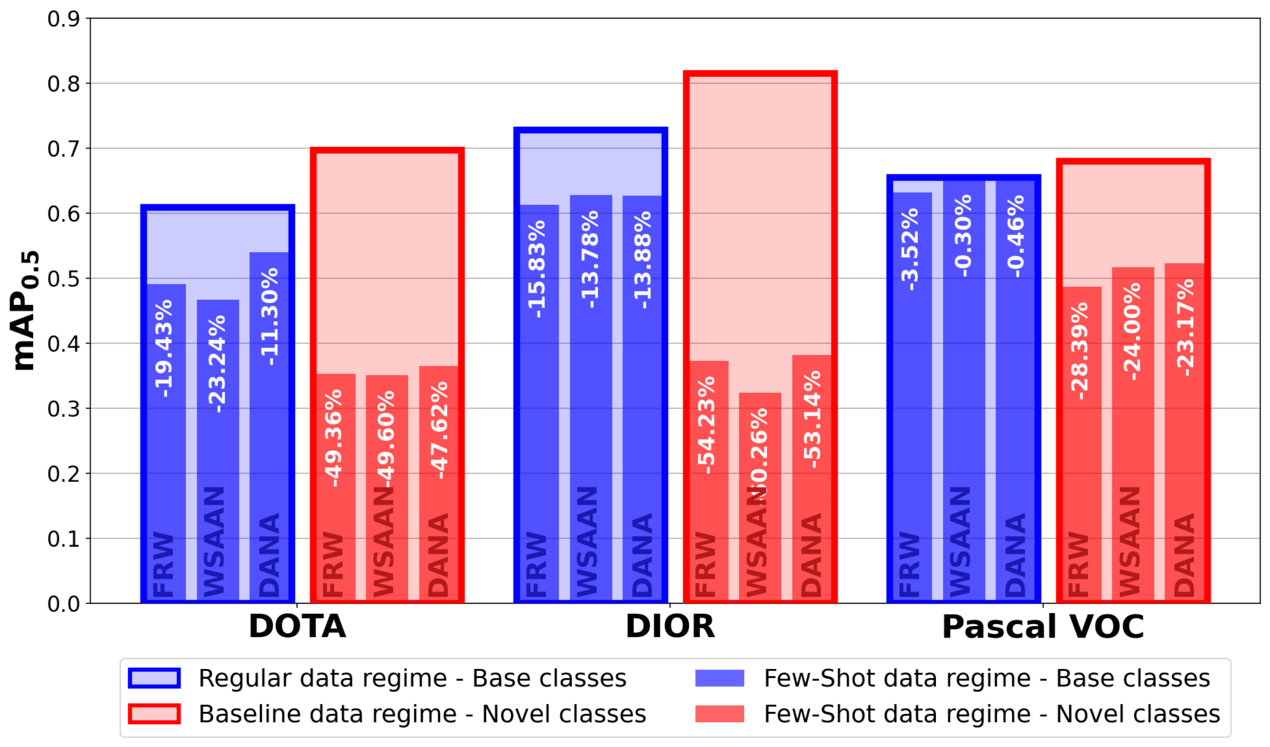}
    \caption[Baseline comparison between FRW, WSAAN and DANA on DOTA, DIOR
    and Pascal VOC datasets]{Performance comparison between FCOS trained in a regular data
    regime versus three few-shot baselines, FRW, WSAAN and DANA (all based on FCOS
    as well) on three datasets: DOTA, DIOR and Pascal VOC.}
    \label{fig:fsod_baseline_res}
\end{figure}

It is generally irrelevant to compare the performance of a method from one
dataset to another, especially with images of different natures. Each dataset
has its own characteristics (resolution, intra-class variety, color range, etc.)
and therefore a given model will not perform equally well on two distinct
datasets according to a pre-defined performance metric. Hence, we cannot
compare the absolute performance of a FSOD method on Pascal VOC and DOTA and the
previous extrapolation is not valid. Nevertheless, there is a pattern: FSOD
methods work consistently better on natural images compared to aerial images. To
understand this phenomenon, we need a way to fairly compare the FSOD performance
across several datasets. To this end, we propose to look at the relative
performance of the FSOD methods against the regular baseline (i.e. FCOS in
our case) using the following metric: 
\begin{equation}
    \label{eq:rmap}
    \text{RmAP} =
            \frac{\text{mAP}_{\text{FSOD}} -\text{mAP}_{\text{Baseline}}}{\text{mAP}_{\text{Baseline}}}.
\end{equation}

RmAP assesses how well a FSOD method is performing on different datasets
compared with the regular detection performance. Hence, it represents how much
performance is lost when switching from the regular to the few-shot regime. This
is exactly what is illustrated in \cref{fig:fsod_baseline_res}, white
percentages are $\text{RmAP}$ values. $\text{RmAP}$ is significantly lower on
DOTA and DIOR compared to Pascal VOC, both for base and novel classes. This way,
we can quantitatively confirm the intuition emerging from
\cref{tab:fsod_few_shot_baselines}: FSOD works better on natural images.

We hypothesize that this performance gap is mainly due to differences in the
object sizes within the datasets. In aerial images, objects are much smaller on
average. This is already an issue for object detection: small objects are
challenging to find. The paradigm of current vision models is to have deep
feature representations with increasing fields of view. The Field of View (FoV)
of a specific layer is the area in the input image that influences the value of
one location in the feature map of that layer. In deeper layers, the FoV is
often quite large compared to small objects' size and object features are
diluted with their irrelevant and noisy surroundings. Thus, it reduces the
activation strength at the object location, and the object can easily be missed.
Feature Pyramidal Networks \cite{lin2017feature} and various other tricks were
introduced to solve this issue, as discussed in \cref{sec:od_rsi}. However, this
problem is largely amplified for FSOD. It is still difficult to detect small
objects, but in addition, they are poor examples for adapting the model (either
through fine-tuning or direct adaptation).

\begin{figure}[h]
    \centering
    \includegraphics[width=0.95\textwidth]{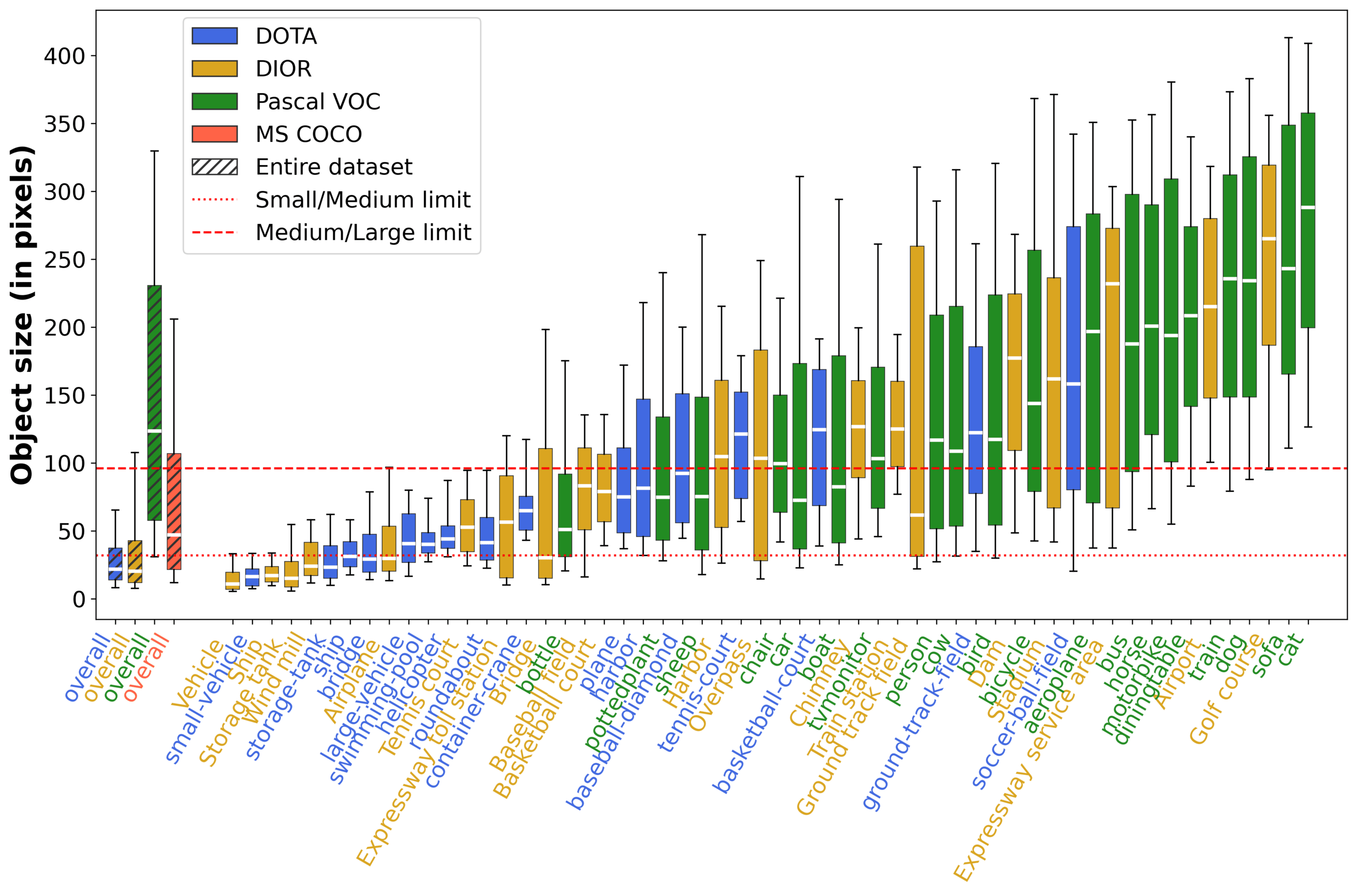}
    \caption[Per class object size analysis on DOTA, DIOR and Pascal VOC.]{Box
    plot of objects size in DOTA, DIOR and Pascal VOC and MS COCO. On the left
    side, boxes represent the overall size distribution in each dataset. On the
    right side, the distributions are split by class and ordered by average
    size. As MS COCO contains 80 classes, we choose not to include the per class
    box plots for it in this plot.}
    \label{fig:obj_sizes}
\end{figure}

To support this hypothesis, we first conduct a brief size analysis of the four
datasets DOTA, DIOR, Pascal VOC and MS COCO (see \cref{fig:obj_sizes}). Aerial
datasets contain far smaller objects than natural ones. Plus, in aerial
datasets, the size of objects in different classes differs a lot. Some classes
contain only small objects, while others only large objects. In Pascal VOC, this
class' size variance is limited. We argue that it is more difficult for the
model to extract relevant information from small support examples but also to
learn more diverse features to deal with greater objects' size variance.
Incidentally, this partly explains the greater difficulty of MS COCO.
\vspace{-1mm}
To support this claim, we conduct a per-class performance analysis on DOTA, DIOR
and Pascal VOC. The results of this comparison are available in
\cref{fig:perf_by_class}. In this figure, the performance is reported per class
against the average size of the class. The first row reports absolute mAP values
(with 0.5 IoU threshold) both for FRW and FCOS (baseline). In the second row,
the mAP gap between the FRW and the baseline is plotted against the objects'
size. We did not report RmAP values for the sake of visualization. RmAP can take
large values (e.g. when the regular baseline mAP is low) and this squeezes the
interesting part of the plot in a narrow band around 0. Larger objects are
easier to detect. It is true in both data regimes, but this trend is reinforced
in the few-shot regime (in the first row, the blue trend lines are steeper than
the black ones). This is observed for base classes but not always for novel
classes, probably because the trends on novel classes are not reliable due to
the limited number of points. \cref{fig:fsod_rmap} shows a more reliable trend
for novel classes when the results from the three datasets are aggregated.
Finally, the few-shot methods, which leverage support information to condition
the detection can surpass the baseline in some cases. All three methods here are
attention-based, and therefore, benefit from having support examples available
during inference to condition the detection. This would not be the case with
fine-tuning approaches. However, this seems advantageous only when the objects
are large. On the contrary, when the objects are small, the performance is
degraded. It confirms that small objects are poor examples to condition the
detection on. For novel classes; however, the performance is always below the
baseline, even if the gap shrinks with larger objects. This is expected as the
network only received weak supervision for these classes. 
\vspace{-1mm}
This comparative analysis confirms that detecting small objects is a very
difficult task in the few-shot regime. It is hard to extract useful information
from small support objects. Even worse, this information can be detrimental for
the detection. Existing FSOD methods are not designed to deal with small
objects, hence the application of these methods on aerial images does not yield
satisfactory results. It is therefore crucial to develop FSOD techniques that
target specifically small objects. Incidentally, we will address this point in
\cref{sec:xscale} and in \cref{chap:siou_metric}.    

\vspace{-1em}
\section{Conclusion}
\vspace{-1em}
In this chapter, we presented our first contribution to the FSOD field with an
analysis of the challenges raised by the few-shot regime for the detection task.
These difficulties are reinforced when FSOD is applied to aerial images as they
contain smaller objects. This gives a clear direction for this PhD project:
improving the handling of small objects in FSOD methods. To this end, we
dedicate \cref{sec:xscale} and the entire \cref{part:iou} of this thesis.

\begin{figure}[]
    \centering
    \includegraphics[width=\textwidth]{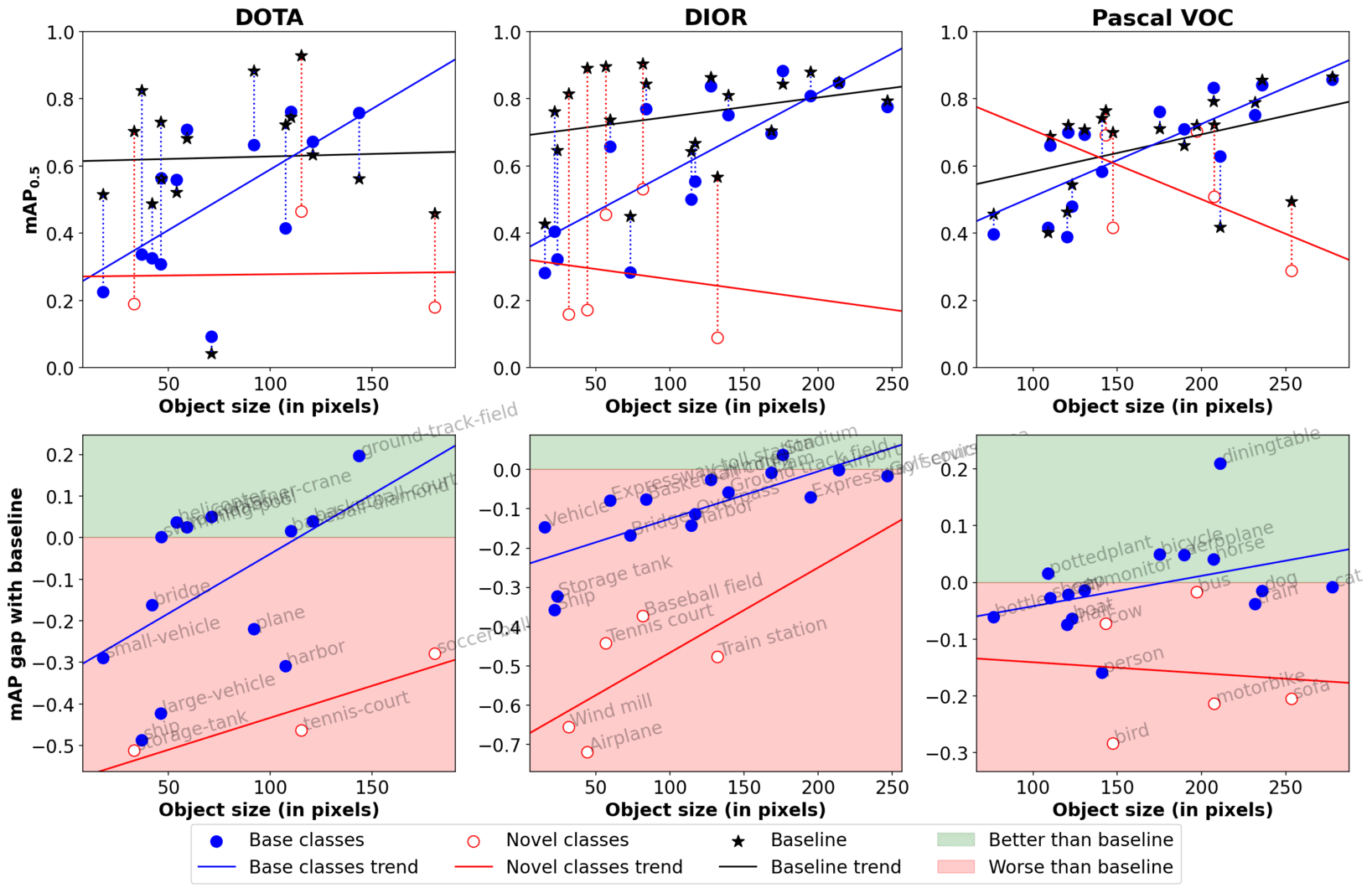}
    \caption[Per class performance analysis for FRW on DOTA, DIOR and Pascal
     VOC against object size]{Performance comparison between FRW baseline -- with 10 shots -- (blue
     and red dots) and regular baseline (black stars) on three different
     datasets: DOTA, DIOR and Pascal VOC. \textbf{(top)} Mean average
     performance of the two methods plotted per class against average object
     size. \textbf{(bottom)} gap between FRW baseline and regular baseline, per
     class. Positive values indicate better performance than the regular
     baseline.}
    \label{fig:perf_by_class}
\end{figure}

\begin{figure}[]
    \centering
    \includegraphics[width=\textwidth]{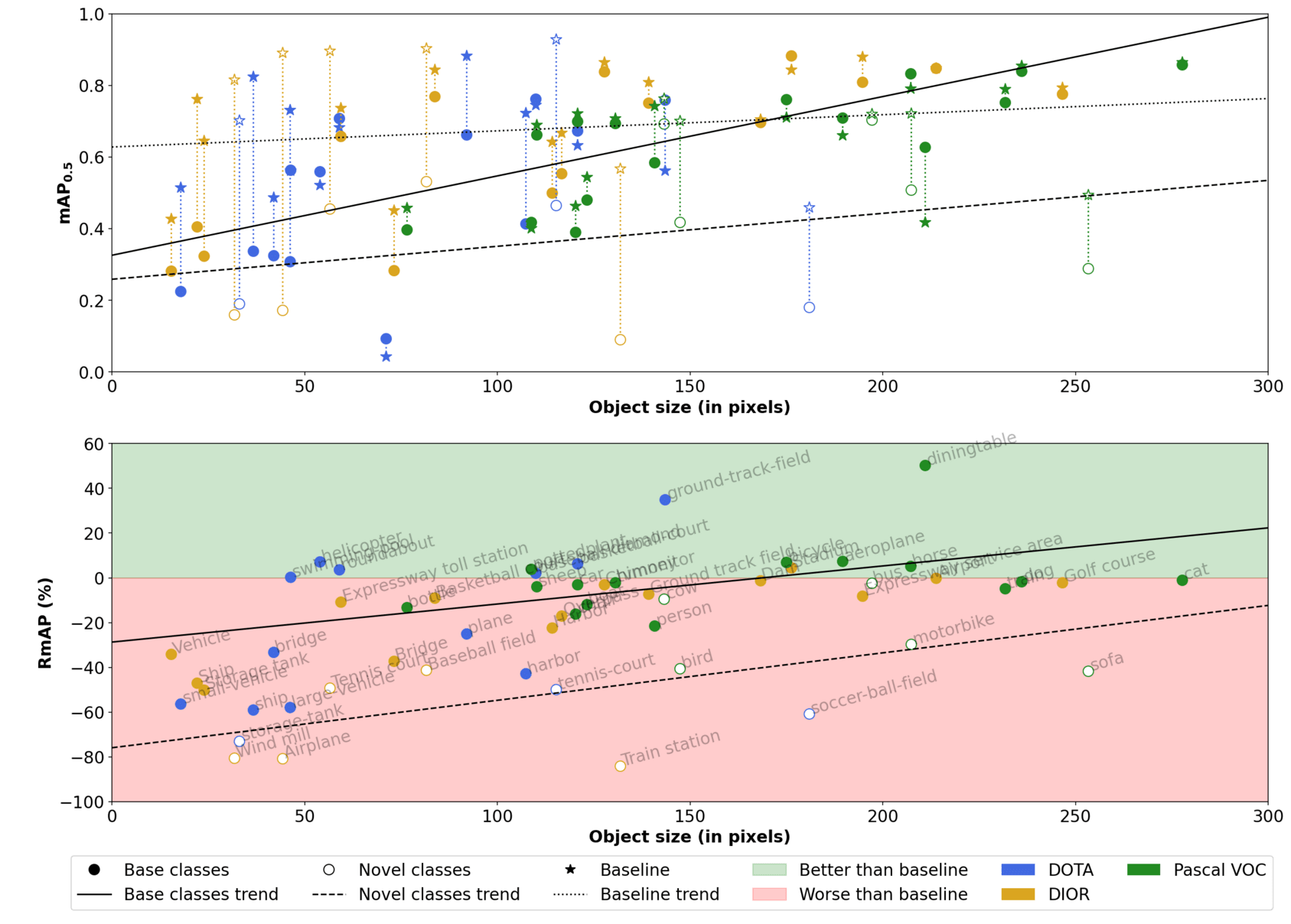}
    \caption[Per class performance analysis for FRW on all datasets]{Comparison
    of FRW (FSOD) and FCOS (baseline) performance against object size on the
    three datasets DOTA, DIOR and Pascal VOC together. \textbf{(top)} absolute
    $\text{mAP}_{0.5}$ values. \textbf{(bottom)} RmAP computed against the
    regular baseline. The RmAP plot has been cropped for visualization reasons.
    Only the class \textit{container-crane} from DOTA dataset is not visible
    (with a RmAP of 150\%).}
    \label{fig:fsod_rmap}
\end{figure}

\part{Improving Few-Shot Object Detection through Various Approaches}%
\label{part:contributions_fsod}

\makeatletter
\let\savedchap\@makechapterhead
\def\@makechapterhead{\vspace*{-1.5cm}\savedchap}
\chapter{Experience Feedback about Metric Learning for FSOD}
\let\@makechapterhead\savedchap
\makeatletter
\label{chap:prcnn}

\vspace{-1em}
\chapabstract{Prototypical Faster R-CNN (PFRCNN) is a novel approach for FSOD
based on metric learning. It embeds prototypical networks inside the Faster
R-CNN detection framework, specifically in place of the classification layers in
the RPN and the detection head. PFRCNN is applied to synthetic images generated
from the MNIST dataset and to real aerial images with DOTA dataset. The
detection performance of PFRCNN is slightly disappointing but sets a first
baseline on DOTA. However, the experiments conducted with PFRCNN provide
relevant information about the design choices for FSOD approaches.}
\vspace{-0.5em}
\textit{
\begin{itemize}[noitemsep]
    \item[\faFileTextO] P. L. Jeune, M. Lebbah, A. Mokraoui and H. Azzag, "Experience feedback using Representation Learning for Few-Shot Object Detection on Aerial Images," 2021 20th IEEE International Conference on Machine Learning and Applications (ICMLA), Pasadena, CA, USA, 2021, pp. 662-667, doi: 10.1109/ICMLA52953.2021.00110.
\end{itemize}
}

\PartialToC
\pagebreak

As a first step into the Few-Shot Object Detection field, we proposed a naive
approach to solve the detection task in the few-shot regime. To give some
context, at the beginning of this project, FSOD was a very recent domain and
very few articles tackled this challenging task, especially applied to aerial
images. Therefore, we took inspiration from the Object Detection and Few-Shot
Classification literature embodied respectively by Faster R-CNN
\cite{ren2015faster} and Prototypical Networks \cite{snell2017prototypical}.
This chapter presents Prototypical Faster R-CNN, our first attempt at solving
FSOD. We begin by presenting the motivation behind this contribution and its
main principle. Then, the training procedure and several tricks are proposed to
improve training stability and detection quality. Finally, Prototypical Faster
R-CNN is applied to synthetic and aerial images to assess its generalization
capabilities and understand its limitations. 

\section{Motivation and Principle}
\vspace{-1em}
In 2020, most of the FSOD literature was focused on attention-based approaches
(see \cref{tab:fsod_comparison}); however, the simplicity and success of the
metric learning classification models was tempting. Thus, we proposed
Prototypical Faster R-CNN (PFRCNN), an extension of Faster R-CNN based on metric
learning. The key idea is to replace the classification layers from Faster R-CNN
(\ie in the Region Proposal Network (RPN) and in the Classification head) with
prototypical networks. It is similar to RepMet \cite{karlinsky2019repmet} that
leverages class-representative vectors in the classification head. However,
there are two major differences with PFRCNN. First, RepMet only replaces the
classification layer in the second stage of Faster R-CNN, not in the RPN. Hence,
the adaptation to novel classes is only done in the second stage. Even if the
RPN is presented as a class-agnostic detector, it specializes in the classes
seen during training. As only base classes are annotated during the first phase
of training, objects from novel classes will be filtered out by the RPN, leaving
no chance for the second stage to detect them. Even if it is trained to have a
high recall, the RPN will mostly generate proposals on base classes, which is
harmful in a few-shot regime. Second, RepMet learns the class-generative vectors
from fine-tuning on the few available examples of the novel classes. Instead, a
prototypical network computes its \textit{prototypes} directly from the few
available examples. Finally, Prototypical Networks can adapt to novel classes
without any fine-tuning. Hopefully, this property would transfer to Faster R-CNN
by replacing its classification layer with such malleable modules. For COSE's
application, this would be ideal as the detection model could adapt “on the
fly“ at a low cost. 

\section{Prototypical Faster R-CNN for FSOD}
\label{sec:proto_faster}
\vspace{-1em}
Before explaining in detail how the prototypical networks can be embedded into
Faster R-CNN, let us define a few notations and detail the functioning of Faster
R-CNN. The backbone, RPN and detection head are respectively denoted as $f$,
$g$, and $h$. The backbone extracts feature $F_q$ from the input -- or query --
image $I_q$:
\begin{equation}
    f(I_q)= F_q.
\end{equation}

The backbone extracts features at multiple scales using an FPN, but for
simplicity, we regroup all these features into one notation: $F_q$. The backbone
is a ResNet-50, the FPN extracts features from 3 different levels with
respective strides 8, 16, and 32. The RPN takes $F_q$ as input and computes both
proposals boxes $\bar{b}_i$ and objectness scores $o_i$ for all locations in the
feature maps:
\begin{equation}
    g(F_q) = \left\{(\bar{b}_i, o_i)\right\}_{i=1}^M,
\end{equation}

where $M$ is the number of generated boxes. It changes with the
number of anchor boxes defined per location, in our case, it is set to 3. Hence,
$M$ is three times the number of locations in the feature map. Then, the top
1000 boxes with the highest objectness scores are selected to extract the proposals
features $\xi_i$ with the RoI Align layer:
\begin{equation}
    \text{RoIAlign}(F_q, \bar{b}_i) = \xi_i.
\end{equation}

Finally, the detection head $h$ outputs classification scores for each proposal
from its features and refines its box coordinates:
\begin{equation}
    h(\xi_i) = \hat{\boldsymbol{\mathrm{y}}}_i = (\hat{b_i}, \hat{l}_i),
\end{equation}

where $l_i \in [0,1]^{|\mathcal{C}| + 1}$ is a vector of classification
scores. There is one more element in $l_i$ than in $\mathcal{C}$ because Faster
R-CNN deals with background as a class. 

\subsection{Extending Faster R-CNN with Prototypical Networks}
To replace the classification layer in Faster R-CNN by prototypical networks, we
propose to change the output dimension of the last layer in the classification
branches of both the RPN and the head. That way, instead of producing a
classification (or objectness) score per box, these networks output embedding
vectors. Each vector represents the information contained inside the
corresponding box. We denote these embedding vectors of the RPN and the
classification head $z_i^{\text{RPN}}$ and $z_i^{\text{head}}$ respectively.
Their dimension is set to 128 ($z_i^{\cdot} \in \mathbb{R}^{128}$) and is kept
fixed in all our experiments. Hence, the outputs of the RPN and the detection
head become: 
\begin{align}
    g(F_q) &= \left\{(\bar{b}_i, z_i^{\text{RPN}})\right\}, \\
    h(\xi_i) &= (\hat{b_i}, z_i^{\text{head}}).
\end{align}

Then, the objectness and classification scores for each proposal are computed
with prototypical networks based on class prototypes computed from support
examples. Prototypes are computed from the support set $\left\{(I_k^c,
b_k^c)\right\}_{\substack{1 \leq k\leq K \\ c \in \mathcal{C}_{\text{novel}}}}$.
Specifically, each support image is fed into the backbone to extract its
features $F_k^c$ and then the example features are extracted with RoI Align:
\begin{align}
    z_{k,c}^{\text{RPN}} &= g(F_k^c), \\
    \Phi_k^c &= \text{RoIAlign}(z_{k,c}^{\text{RPN}}, b_k^c), \\
    z_{k,c}^{\text{head}} &= h(F_k^c), \\
    \Psi_k^c &= \text{RoIAlign}(z_{k,c}^{\text{head}}, b_k^c).
\end{align}

This gives RPN features and classification features for each support image,
denoted $\Phi_k^c$ and $\Psi_k^c$ respectively. Note a slight abuse of notation
here, when only the embedding part of $g$ and $h$ is used to project the
features extracted by the backbone (\ie not the regression part). When multiple
examples are available for a class (\ie $K\geq 1$), their embeddings are
averaged to get one prototype per class:
\begin{align}
    \Phi^c &= \frac{1}{K} \sum\limits_{k=1}^{K} \Phi_k^c,\\
    \Psi^c &= \frac{1}{K} \sum\limits_{k=1}^{K} \Psi_k^c.
\end{align}

\begin{figure}
    \centering
    \includegraphics[width=\textwidth]{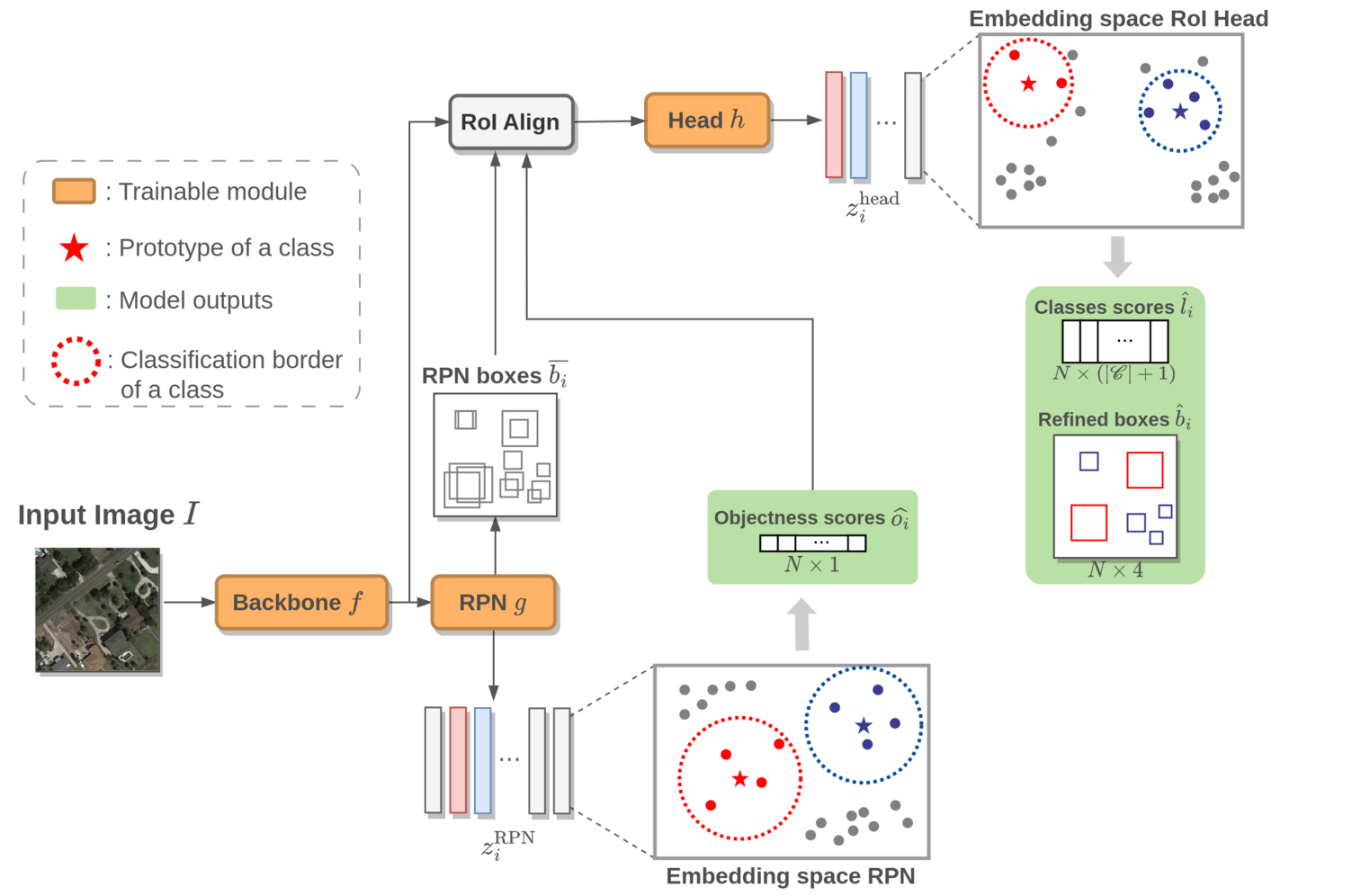}
    \caption[Prototypical Faster R-CNN architecture]{Illustration of the
    architecture of Prototypical Faster R-CNN.}
    \label{fig:proto_faster_rcnn}
\end{figure}

For each proposal, we compute the classification score for class $c$ as the
likelihood of the region in the input image representing an object of class $c$.
To do so, we suppose that the class distributions over the embedding space are
Gaussian distributions centered on the class prototypes. Hence, the
classification score of the proposal $i$ for class $c$ is:  
\begin{equation}
    \hat{l}_i^c = p(z_i^{\text{head}}|c) = \exp\Big(\frac{-\mathrm{d}(z_i^{\text{head}}, \Psi^c)^2}{2\sigma^2}\Big),
\end{equation}

where $\mathrm{d}$ is a distance measure over the representation space, in our
experiments, $\mathrm{d}$ is the Euclidean distance. Note that in our case, the
embeddings are normalized after their computation, therefore the Euclidean
distance is equivalent to the Cosine Similarity. $\sigma$ is the standard
deviation of the distribution and is set to $0.5$ in our experiments. In Faster
R-CNN, the background is considered as a class as well, the corresponding score
can be derived from the other class scores as follows:
\begin{equation}
    \hat{l}_i^{\varnothing} = p(z_i^{\text{head}}|\varnothing) = 1 - \max\limits_{c \in \mathcal{C}} \hat{l}_i^c, 
\end{equation}

where $\varnothing$ denotes the background class. 

In the RPN the objectness computation is very similar to the classification
score in the head. However, only two classes are considered: foreground and
background. The foreground class is seen as a mixture of Gaussians (\ie a
mixture of all foreground classes) and is approximated as the maximum score
among all classes for stability reasons:
\begin{equation}
    \hat{o}_i = \max\limits_{c \in \mathcal{C}} \hat{l}_i^c. 
\end{equation}

These modifications make Faster R-CNN able to adapt to novel classes. Computing
prototypes for novel classes allows direct adaptation of the whole detection
model and not simply the detection head as in RepMet. However, with these
changes, the model also requires a different training scheme to ensure that the
prototypes are properly leveraged and classes are not only memorized.

\subsection{Training Procedure}
Before presenting the changes with the Faster R-CNN training procedure, we present
here what remains unchanged: the loss functions and the example selection.
Faster R-CNN is trained using four distinct loss functions, two for the RPN and
two for the detection head:
\begin{align}
    & \mathcal{L}_{reg}^{\text{RPN}}(b_i^{\text{RPN}}, \bar{b}_i^{\text{RPN}}) &  & = \text{SmoothL1Loss}(b_i^{\text{RPN}}, \bar{b}_i^{\text{RPN}}), \\
    & \mathcal{L}_{obj}^{\text{RPN}}(o_i, \hat{o}_i)                       &  & = \hat{o}_i \log(o_i) + (1-\hat{o}_i) \log (1-o_i),          \\
    & \mathcal{L}_{reg}^{\text{head}}(b_i^{\text{head}}, \hat{b}_i^{\text{head}}) &  & = \text{SmoothL1Loss}(b_j^{\text{head}}, \hat{b}_j^{\text{head}}), \\
    & \mathcal{L}_{cls}^{\text{head}}(c_i, \hat{c}_i)                       &  & = - \log(l_i^c),
 \end{align}

where $b_i^{\text{RPN}}$ and $o_i$ are the ground truth targets for the
regression and classification branches of the RPN. Similarly,
$b_i^{\text{head}}$ and $c_i$ are the target for the detection head. During
training, not all boxes are selected for computing the losses. The generated
boxes (or proposals for the RPN) are separated into two groups: positive
examples, i.e. boxes with an overlap of at least 0.7 with a ground truth
annotation, and negative examples which represent the background class. The
classification losses are computed over all examples, while the regression
losses only take into account the positive boxes. This remains unchanged for
Prototypical Faster R-CNN.

However, the training of PFRCNN is done episodically, following the
Meta-Learning paradigm and the training scheme proposed in Prototypical Networks
\cite{snell2017prototypical}. The motivation behind such training is to mimic
the setup that will be encountered at test time and prevent base classes
memorization. Indeed, during base training, only annotations from the base
classes are available. Training the model with all base classes at the same time
could lead to overfit the base class set, at the cost of adaptation. The
episodic training consists in sampling a subset of classes
$\mathcal{C}_{\text{ep}} \subset \mathcal{C}_{\text{base}}$ and train the model
to detect only these classes for a few training steps. Such a training phase is
called an episode. The episodes are then repeated over and over until
convergence. During each episode, a \textit{query set} and a \textit{support
set} are sampled from the original dataset. The support set contains the
examples that will be leveraged for the prototypes computation. On the other
hand, the query set is exploited as a small training set. The loss is computed
on the query set and between each update of the model, the prototypes are
re-computed from the same support set. The update of the prototype is not
necessary between each training step, but since the model's weights are updated,
the class representations also change. Additionally, the episodic strategy
allows for mimicking the test time setting. If there are $N$ novel classes with
$K$ support images at test time, the episodes can reproduce this even though the
dataset has a lot more classes and data. Episode after episode, the model will
encounter new class combinations and support examples, in the end, it should
learn to generalize to novel classes from a few examples, according to
Meta-Learning claims. 

To build the support set, for each class $c \in \mathcal{C}_{\text{ep}}$, we
select images containing objects of class $c$ and disregard all other objects
(i.e. their annotations are not included in the support set but the image is not
masked, so they are still visible). If there is more than one object $c$ in the
image, only one is selected randomly as the annotated example. This prevents
having more than $K$ examples per class. The query set contains
$K_{\text{query}}$ images for each of its $N$ classes, this means at least
$K_{\text{query}}$ examples for each class, but this number can be larger as
more than one object is present in the images. As for the support set, the
annotations with class labels not in $\mathcal{C}_{\text{ep}}$ are discarded.
This sampling procedure prevents the occurrence of \textit{self-distractors} but
not co-occurrence distractors (see \cref{sec:fsod_distractors}).

Once the base training is done, the network can directly be applied to novel
classes through direct adaptation from the prototypes (see
\cref{fig:fsod_evaluation_strategies}). However, the adaptation is only
performed in the classification parts of the model, regression branches are not
modified. This is certainly sub-optimal and therefore, we provide a fine-tuning
scheme to remedy this. This fine-tuning is done exactly as the base training
phase, in an episodic manner except that the episode classes are sampled from
both base and novel class sets: $\mathcal{C}_{\text{ep}} \subset
(\mathcal{C}_{\text{base}} \cup \mathcal{C}_{\text{novel}})$. In this case, the
examples of novel classes are the same in the query and support sets so that
the total number of support examples remains fixed.   

\subsection{Iterative improvements}
\label{sec:improvements}
PFRCNN, as described in the previous section, denoted the baseline, does not
perform well on aerial images (see \cref{tab:pfrcnn_ablation}). Therefore, we
introduce a series of improvements to improve the performance of the model.

In this section, we propose a series of improvements on top of the PFRCNN
baseline described above. Indeed, when tested on aerial images, vanilla PFRCNN
yields relatively poor performance (see \textit{baseline} performance in
\cref{tab:pfrcnn_ablation}). To remedy this, we introduce several training tricks. 

\paragraph*{Hard negative example mining}
One issue encountered with the baseline is the detection of the base classes
regardless of support examples. Basically, it detects base classes even though
no prototypes are provided for these classes: this is base class memorization.
Although this improves performance when base class prototypes are provided, it
produces lots of false positive detections when novel classes are wanted. To
address this, we propose to sample hard negative examples to encourage the model
to detect support classes only. The main idea is to take advantage of the
annotations for classes not selected in the current task to find hard negative
examples, i.e. classes that the network could have memorized from previous tasks
but should not be detected during this episode. When starting a new episode, it
is likely that the model still produces detection for objects annotated in one
of the previous episodes if it does not rely on the support information. Even
though these objects are not annotated in the new task, their annotations are
available in the dataset (because they belong to base classes). Therefore, these
annotations can be used to find examples that should be considered as background
for the current task. They are different from the background examples that
do not contain any class of the dataset, which are referred to as easy negative
examples and are much more numerous. Explicitly sampling these hard negative
examples encourages the network to detect only objects annotated in the support
set.

\paragraph*{Moving average prototypes}
Another issue with the baseline is that the prototypes can change abruptly,
either when the network is updated or when the support set changes. We argue
that this causes some training instabilities. To prevent such rapid
modification of the prototypes, an exponential moving average is introduced to
smooth the disruption. Hence, $\bar{\Phi}^c_{t+1} = \alpha \Phi^c_t + (1 - \alpha)
\bar{\Phi}^c_{t}$. $\alpha$ is set to $0.1$ in our experiments. $\bar{\Phi}^c_{t}$
is the averaged prototype for class $c$ at iteration $t$, while $\Phi^c_t$ is
the prototype computed from the support set, for class $c$ at iteration $t$. 

\paragraph*{Background clustering}
Lastly, the baseline shows a poor separation of novel class representations (see
\cref{fig:pfrcnn_dota_tsne}). This leads to poor performance with novel classes at
test time. In order to solve this, inspiration is drawn from
\cite{caron2018deep}. At each iteration, they fit a K-means on the learned
representations. This gives pseudo-labels to train the network for
classification in a self-supervised manner. Similarly, we propose to fit a
K-means on the negative embeddings (i.e. representing boxes not matched by any
ground truth object). From the resulting pseudo-labels a contrastive loss
function (Triplet Loss \cite{hoffer2015deep}) is computed. The triplets are
formed with embeddings labeled identically by the K-means. It encourages the
network to organize the negative examples into tight and separated clusters.
This will eventually discover semantic clusters that represent novel objects.

\paragraph*{Ablation study} 
In order to assess the relevance of the tricks formulated in
previous paragraphs, a small ablation study is conducted on DOTA dataset. The
results of this analysis can be found in \cref{tab:pfrcnn_ablation}. On the one hand,
the introduction of hard examples mining and moving average prototypes improves
consistently the novel classes mAP in the one-shot setting. On the other hand,
background clustering greatly reduces the performance on base classes, while
achieving similar results on novel classes. According to this analysis, we chose
to keep only hard example mining and the moving average as it combines the best
base and novel classes performance.

\begin{table}[h]
    \centering
    \resizebox{0.75\textwidth}{!}{%
    \begin{tabular}{@{\hspace{3mm}}ccccccccccc@{\hspace{3mm}}}
    \toprule[1pt]
       & \multicolumn{2}{c}{\textbf{PFRCNN Baseline}} & \multicolumn{2}{c}{\textbf{+HEM}}& & \multicolumn{2}{c}{\textbf{+MA}}& & \multicolumn{2}{c}{\textbf{+BC}} \\ \midrule
    \textbf{K}  & Base                 & Novel                 & Base            & Novel          & & Base           & Novel          & & Base           & Novel           \\ \midrule
    1  & \bbf{35.5}                 & 2.1                   & 31.2            & 4              & & 26.5           & \rbf{6.9}            & & 13.3           & 4.3             \\
    3  & \bbf{35.9}                 & 2.7                   & 35.6            & 2.3            & & 33.9           & 3.5            & & 14.5           & \rbf{4.1}             \\
    5  & 34.3                 & 3.8                   & \bbf{41.2}            & 3.3            & & 37             & 4.2            & & 18.2           & \rbf{4.7}             \\
    10 & 30.4                 & 4.1                   & \bbf{34.3}            & 2.6            & & 35.1           & \rbf{5.9}           & & 14.8           & 2.6             \\ \bottomrule[1pt]
    \end{tabular}%
    } \caption[PFRCNN training tricks ablation study]{Ablation study about the
    training tricks described in section \ref{sec:improvements}. Each column
    corresponds to the addition of each trick on top of the previous one. HEM, MA and BC
    correspond respectively to Hard Example Mining (HEM), Moving Average (MA) prototypes
    and Background Clustering (BC). Detection performance is reported as mAP with a
    0.5 IoU threshold. Blue and red values represent the best performance on
    base and novel classes respectively.}
    \label{tab:pfrcnn_ablation}
\end{table}

\section{Performance on Artificial Data}
\vspace{-1em}
Before applying PFRCNN on aerial images, we test it on an artificial dataset
with reduced difficulty. This gives a hint about the capacities of the model
on real data.

\subsection{MNIST-LOC Dataset}
As an artificial dataset, we leveraged MNIST-LOC. This dataset is not a
published work but rather a toy example sometimes mentioned in the literature.
It consists in creating artificial images with the handwritten digit images from
the original MNIST dataset \cite{deng2012mnist}. For each image in MNIST-LOC, a
random number of MNIST digits are sampled and placed randomly in the image with
a random scale. This creates a potentially infinite dataset but with limited
variability. For our experiments, we build a dataset with 20k images in the
training split and 2k images both for the test and validation splits. The
dataset has around 120k annotated objects, which translates to approximately 12k
instances per class. An overview of the dataset is provided in
\cref{fig:mnist_fig}.

\begin{figure}
    \centering
    \includegraphics[width=\textwidth]{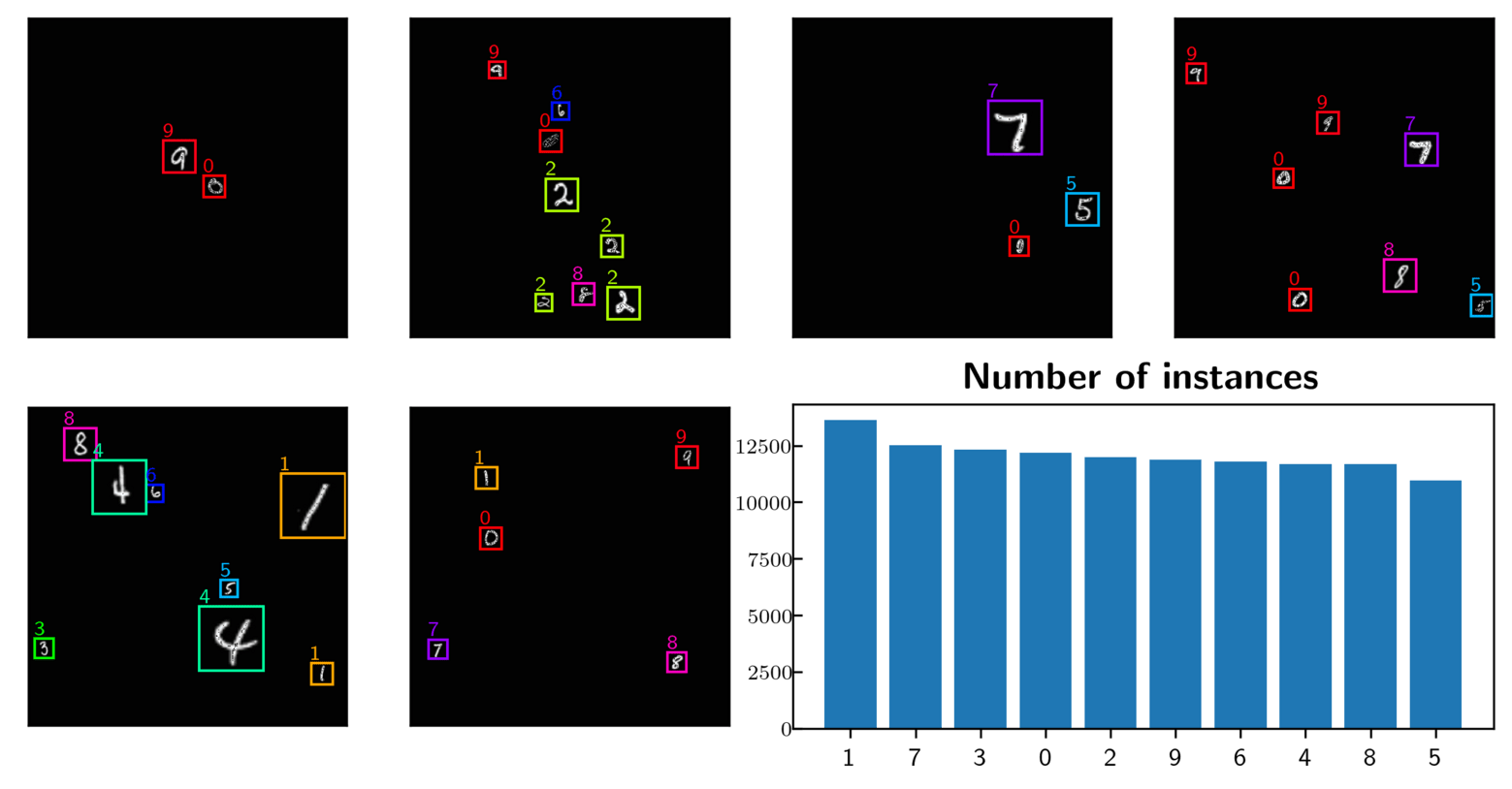}
    \caption{Images of MNIST-LOC dataset and classes repartition.}
    \label{fig:mnist_fig}
\end{figure}

Compared to a real dataset, MNIST-LOC is far more simple. The background is
uniform which simplifies the localization of the objects. Then, the
class occurrences are uniformly distributed. And finally, the intra-class
variance is reduced as MNIST is an easy dataset for the classification task.

\subsection{Implementation details}
We provide here some of the implementation details for training Prototypical
Faster R-CNN, but a complete list of the hyperparameters and their values is
available in our
GitHub\footnote{\href{https://github.com/pierlj/proto_faster_rcnn}{Link to Prototypical
Faster R-CNN repository.}}.
The optimization is done with Adam optimizer \cite{kingma2014adam} and a
learning rate of $1e-4$. The backbone network is pre-trained on ImageNet and its
first two layers are kept frozen during training. Three classes are selected as
novel classes and are reserved for evaluation, the 7 others are kept as base
classes. Each episode is constituted of $K_{\text{query}} = 5$ images per class,
\ie 15 images per episode.

\subsection{Detection performance on MNIST-LOC}
We present the performance results on MNIST-LOC in \cref{tab:pfrcnn_mnist}. This
table reports the mean Average Precision (mAP) with an IoU threshold of 0.5 (see
\cref{sec:od_evaluation} for more details about mAP). The results are given with
multiple values of $K$, the number of support examples, and two distinct
splits of the base and novel classes. The evaluation is done on an unseen test
set, from which the support examples are sampled as well. The table provides the
mAP both for base and novel classes separately as we do not consider the
generalized few-shot setting. 

First, it can be seen from this table that the performance in a regular data
regime (\ie vanilla Faster R-CNN with all annotations) is high. This confirms
that MNIST-LOC is a fairly simple dataset and that the detection task is way
easier on this dataset than on real ones. It is important to note that these
values cannot be directly compared with the performance values in the few-shot
regime as the number of classes is different. In the regular regime, the
classification problem has 10 classes whereas, in the few-shot regime, it only
has three (3-ways $K$ shots setting, even for base classes). Then, the few-shot
performance of PFRCNN on base classes is also quite high, approaching one as the
number of shots grows. However, for novel classes, this is different, the mAP
values are way lower in this case and fall below an acceptable threshold for any
industrial use case. To get a better grasp on these results,
\cref{fig:pfrcnn_mnist_quant} gives detection examples on MNIST-LOC dataset for
base and novel classes. For base classes, two distinct support sets are
leveraged between rows 1 and 2 (with different classes, \ie
$\mathcal{C}_{\text{ep}}^1 \neq \mathcal{C}_{\text{ep}}^2$). For base classes,
the detection is almost perfect, which represents well the scores from
\cref{tab:pfrcnn_mnist}. However, for novel classes, there are undesired
detections of base classes and a lot of confusion between novel classes. 

\begin{table}[]
    \centering
    \resizebox{0.7\textwidth}{!}{%
    \begin{tabular}{@{\hspace{3mm}}ccccc@{\hspace{3mm}}}
    \toprule[1pt]
                         & \multicolumn{2}{c}{\textbf{Split 1: [0, 1, 4]}} & \multicolumn{2}{c}{\textbf{Split 2: [3, 5, 8]}} \\ \midrule
    \textbf{Data regime} & \textbf{Base Classes}     & \textbf{Novel Classes}    & \textbf{Base Classes}     & \textbf{Novel Classes}    \\ \midrule
    
    1 shot      & 94.86             & 21.86             & 92.46             & 19.43             \\
    3 shots     & 95.70             & 20.39             & 94.82             & 21.22             \\
    5 shots     & 95.10             & 24.34             & 94.95             & 21.73             \\
    10 shots    & 95.86             & 23.19             & 93.11             & 20.17             \\ \midrule
    \textit{Faster R-CNN}     & \textit{76.86}             & \textit{96.33}             & \textit{84.29}             & \textit{79.01}             \\ \bottomrule[1pt]
    \end{tabular}%
    } \caption[PFRCNN performance on MNIST-LOC]{PFRCNN performance on MNIST-LOC
    dataset with two distinct class splits. On the left, novel classes are 0, 1
    and 4, while on the right novel classes are 3, 5, and 8. In both cases, all
    other classes belong to the base class set. Performance is reported as
    $\text{mAP}_{0.5}$. The last row reports the performance of a vanilla Faster
    R-CNN trained in a regular data regime, \ie with all available annotations
    in the dataset. For Faster R-CNN, per-class performance is averaged over
    base and novel classes separately to compare with the few-shot techniques.}
    \label{tab:pfrcnn_mnist}
    \end{table}

\begin{figure}
    \centering
    \includegraphics[width=\textwidth]{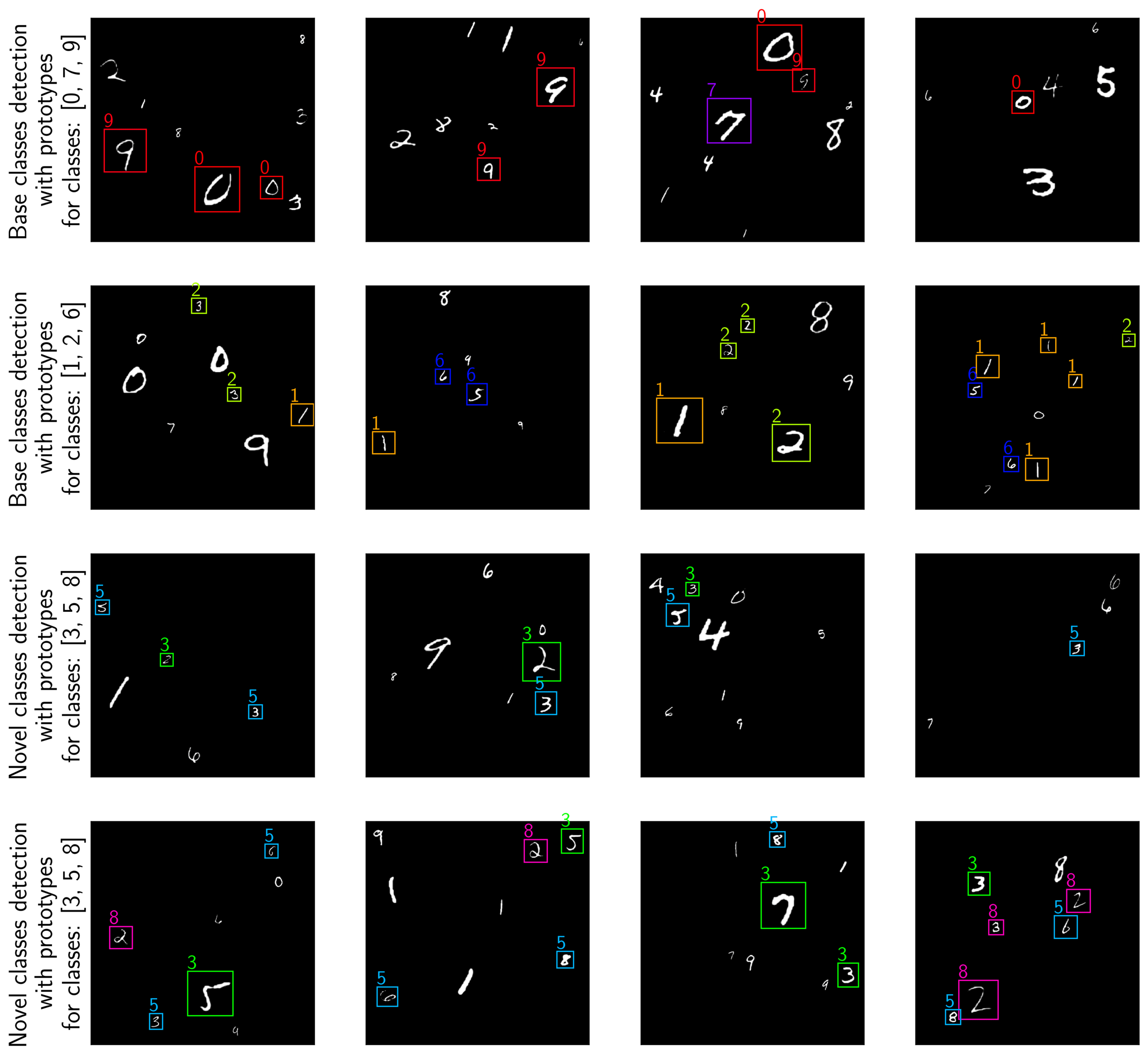}
    \caption[Prototypical Faster R-CNN results on MNIST-LOC dataset]{Prototypical Faster
     R-CNN qualitative detection results on MNIST-LOC dataset, on base and novel
     classes. Predictions are done without fine-tuning and with $K=1$.}
    \label{fig:pfrcnn_mnist_quant}
\end{figure}

\vspace{1em}
\section{Difficulties on Aerial Images}
\vspace{-1em}
While the Prototypical Faster R-CNN is challenged on synthetic images, it has
more serious difficulties with real images. In this section, we present the
detection result of PFRCNN applied on aerial images, specifically on DOTA and
DIOR datasets.

First, \cref{tab:pfrcnn_dota} gathers the performance results of PFRCNN on DOTA
dataset for base and novel classes. As for MNIST-LOC, two distinct class splits
are experimented: Split A with \textit{plane}, \textit{ship} and
\textit{tennis-court} and Split B with \textit{harbor}, \textit{roundabout} and
\textit{helicopter}. Following the same configuration as in the previous
section, we report the performance with $\text{mAP}_{0.5}$ for base and novel
classes independently. The results on base classes are much lower than with
MNIST-LOC, but it makes sense as the detection task in DOTA is also much more
difficult. Nevertheless, the base classes' performance is much lower than the
regular setup (\ie Faster R-CNN trained on the whole DOTA). For novel
classes, a similar performance drop is observed, making PFRCNN unfit for any
industrial application. 

\begin{table}[]
    \centering
    \resizebox{0.7\textwidth}{!}{%
    \begin{tabular}{@{\hspace{3mm}}ccccc@{\hspace{3mm}}}
    \toprule[1pt]
                                & \multicolumn{2}{c}{\textbf{Split A}}  & \multicolumn{2}{c}{\textbf{Split B}}  \\
    \multicolumn{1}{c}{$K$ shots} & Base classes & Novel Classes & Base classes & Novel Classes \\ \midrule
    
    1                           & 27.5 $\pm$ 1.0   & 4.7 $\pm$ 2.0     & 41.5 $\pm$ 3.0   & 8.0  $\pm$ 1.0    \\
    3                           & 35.2 $\pm$ 2.0   & 2.4 $\pm$ 1.0     & 39.2 $\pm$ 3.0   & 10.1 $\pm$ 2.0    \\
    5                           & 39.0 $\pm$ 1.0   & 3.8 $\pm$ 1.0     & 43.4 $\pm$ 2.0   & 12.1 $\pm$ 1.0    \\
    10                          & 38.4 $\pm$ 2.0   & 4.1 $\pm$ 1.0     & 41.4 $\pm$ 3.0   & 10.1 $\pm$ 2.0    \\ \midrule
    \textit{Faster R-CNN}            & \textit{65.62}   & \textit{90.96}    & \textit{73.21}   & \textit{69.77}    \\ \bottomrule[1pt]
    \end{tabular}%
    } \caption[PFRCNN performance on DOTA dataset]{PFRCNN performance on DOTA dataset with two distinct class
    splits. Split A has classes \textit{plane}, \textit{ship} and
    \textit{tennis-court} and Split B has \textit{harbor}, \textit{roundabout}
    and \textit{helicopter}. In both cases, all other classes belong to the base
    class set. Performance is reported as $\text{mAP}_{0.5}$. The last row reports the performance of a vanilla Faster
    R-CNN trained in a regular data regime, \ie with all available annotations
    in the dataset. For Faster R-CNN, per-class performance is averaged over
    base and novel classes separately to compare with the few-shot techniques.}
    \label{tab:pfrcnn_dota}
    \end{table}

Nonetheless, these experiments are not useless and provide relevant insights
about the FSOD task and its difficulties. For instance, with the MNIST-LOC
dataset, almost no difference could be seen between splits. With DOTA, much
better performance is achieved on Split B than Split A. It indicates some
interactions between classes, some combinations are more difficult than others.
These considerations were not taken into account in the design of PFRCNN and
should be overcome to achieve reasonable few-shot detection.

Despite its limited performance, Prototypical Faster R-CNN is one of the first
approaches to tackle FSOD on remote sensing images from a metric learning
perspective. In addition, this method does not need any fine-tuning.
All previous results were given from a simple adaptation to the novel classes at
inference time with novel prototypes. We also experimented with an additional
fine-tuning step, especially to refine the regression branches of the model.
This was performed on DOTA with Split A and the results are available in
\cref{tab:pfrcnn_ft}. Fine-tuning with the few available support examples helps
significantly to boost the detection quality on novel classes, but it remains
insufficient for COSE's application. Interestingly, after fine-tuning a common
property of few-shot methods emerges: the more examples are provided, the higher
the performance. It was not the case without fine-tuning. With Split A, the best
performance is achieved with $K=1$, with Split B, it increases until $K=5$ and
then decreases with $K=10$. This indicates that the management of more shots is
difficult within PFRCNN. It suggests that support examples features may not be
trivially aggregated as it can produce irrelevant prototypes. This can
happen when a class has a great variety and thus a multimodal distribution in
the embedding space.

\begin{table}[]
    \centering
    \resizebox{0.6\textwidth}{!}{%
    \begin{tabular}{@{\hspace{3mm}}cccc@{\hspace{3mm}}}
    \toprule[1pt]
                                & \multicolumn{2}{c}{\textbf{Without fine-tuning}}   & \textbf{With fine-tuning} \\ \midrule
    \multicolumn{1}{c}{K shots} & Base classes      & Novel Classes         & Novel Classes \\\midrule
    1                           & 27.5              & 4.7                   & \textbf{7.5 }       \\
    3                           & 35.2              & 2.4                   & \textbf{9.3 }       \\
    5                           & 39.0              & 3.8                   & \textbf{11.3}       \\
    10                          & 38.4              & 4.1                   & \textbf{11.6}       \\ \bottomrule[1pt]
    \end{tabular}%
    }
    \caption{PFRCNN performance comparison with and without fine-tuning.}
    \label{tab:pfrcnn_ft}
    \end{table}

\begin{figure}
    \centering
    \includegraphics[width=\textwidth]{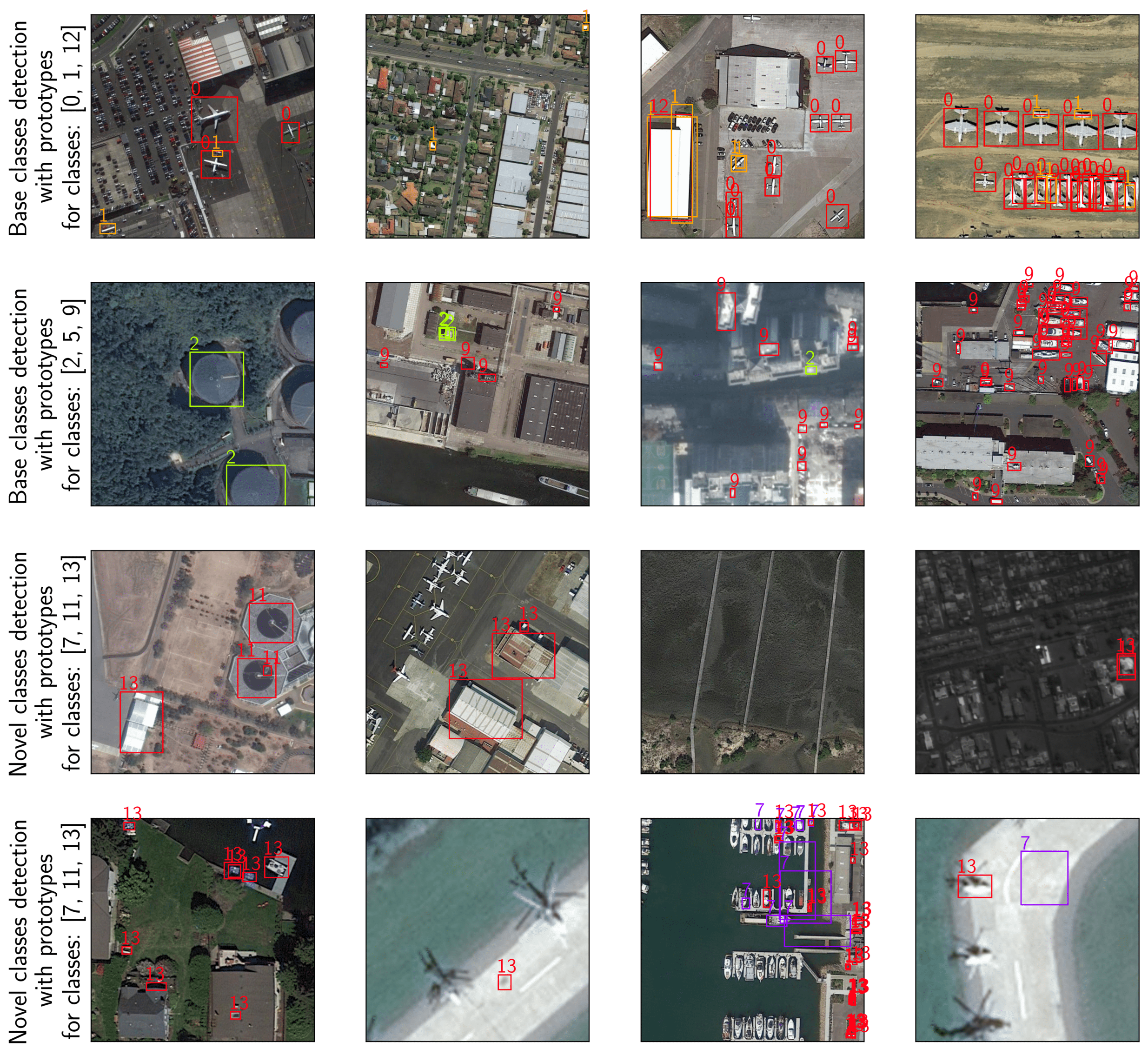}
    \caption[Prototypical Faster R-CNN results on DOTA]{Prototypical Faster
        R-CNN qualitative detection results on DOTA dataset, on base and novel
        classes. Predictions are done without fine-tuning and with $K=1$.}
    \label{fig:pfrcnn_dota_quant}
\end{figure}

Just as for MNIST-LOC, we provide qualitative results of the FSOD on DOTA with
PFRCNN. These are available in \cref{fig:pfrcnn_dota_quant}. The detection is
satisfactory (but not perfect) on the base classes. However, the bounding boxes
and labels for novel classes (bottom 2 rows) are mostly incorrect. Some
confusion between base and novel classes occurs. For instance, in the left-most
image in the third row, water tanks are mistaken as roundabouts. Of course,
these two classes look similar in practice and that makes them difficult to
distinguish. To better understand why this confusion happens, we investigate the
embedding space of PFRCNN through TSNE visualization (see
\cref{fig:pfrcnn_dota_tsne}). This figure is made by collecting the embedding
vectors of all proposals over an entire query set, and then by reducing their
dimension using the TSNE algorithm \cite{van2008visualizing}. Class-specific
clusters are well-formed in the representation space, but some classes overlap
which explains the confusion. Representations of these classes may be close to
another class prototype and get misclassified. This is especially true for novel
classes which overlap over base classes, explaining their poor
performance. For the example above, the misclassification of the two water tanks
is easily understood from the TSNE plot as these two classes almost perfectly
overlap (class \textit{storage-tank} in dark green and \textit{roundabout} in
pink).

\begin{figure}
    \centering
    \includegraphics[width=0.9\textwidth]{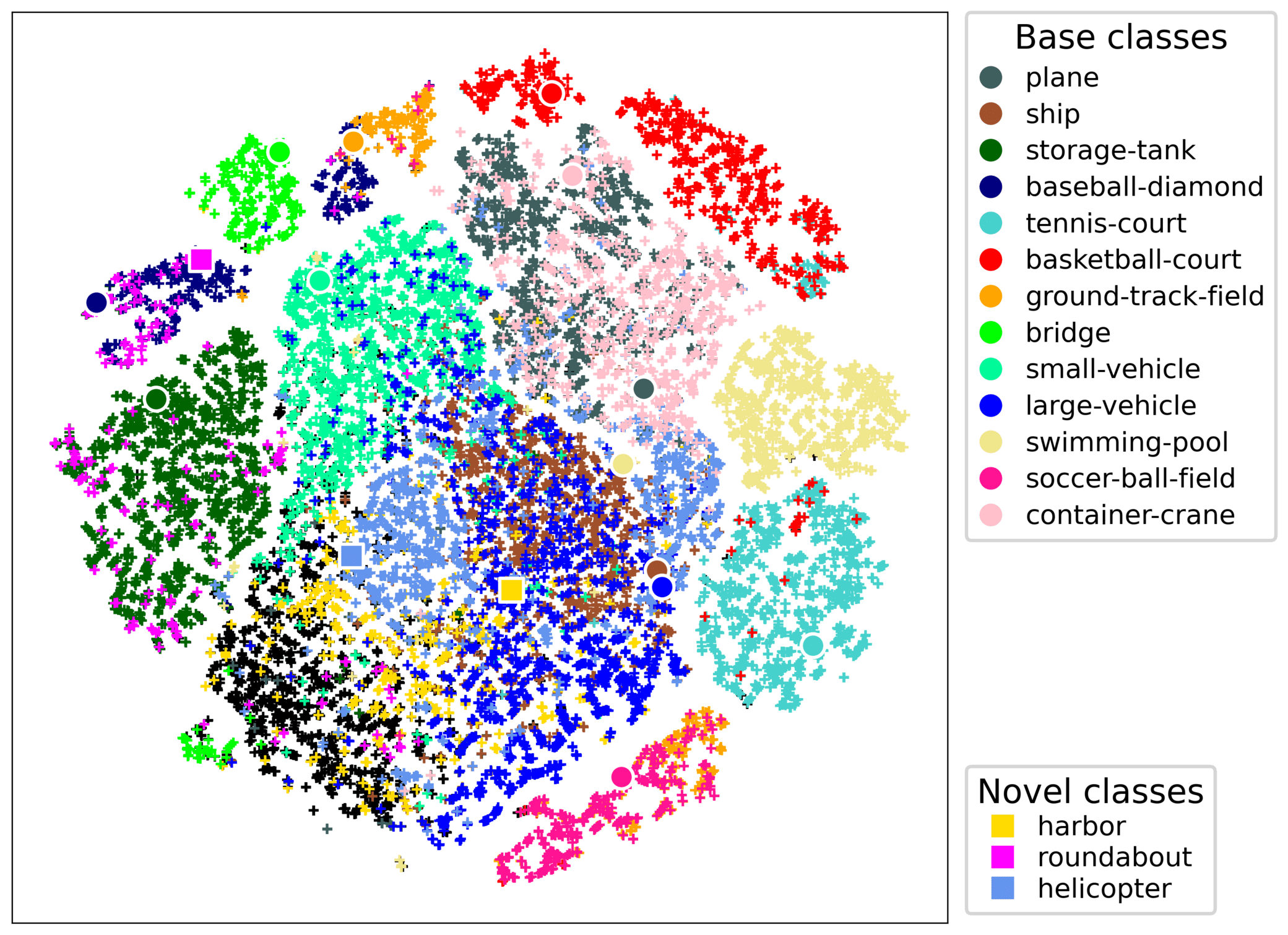}
    \caption[TSNE visualization of the embedding space of PFRCNN on DOTA]{TSNE
    visualization of the trained embedding space of PFRCNN on DOTA. Each point
    represents the projection of RoI in the embedding space. Large circles and
    squares respectively denote the prototypes of base and novel classes. Black
    points denote background proposals.}
    \label{fig:pfrcnn_dota_tsne}
\end{figure}

\section{Insights and conclusion}
\vspace{-1em}
From the results presented above, one question arises: is representation
learning a suitable choice for few-shot object detection? Metric-learning
methods are competitive with state-of-the-art for few-shot classification but
seem less appropriate for FSOD. Prototypical Faster R-CNN is
a first attempt to apply prototypical networks to FSOD. The very few
FSOD approaches based on Metric-Learning often leverage other tricks such as
carefully designed fine-tuning or attention mechanisms. Furthermore, at the
beginning of this PhD, there were only two contributions solving FSOD with
metric learning: RepMet \cite{karlinsky2019repmet} and RN-FSOD
\cite{yang2020restoring}. More investigation was therefore needed in this
direction. Of course, the poor results of PFRCNN alone are not enough to
conclude that all metric-learning-based approaches are inappropriate.
Nevertheless, metric-learning FSOD methods are now in the minority in the current
literature, which indirectly confirms their inadequacy. 

Despite the relatively poor performance of PFRCNN, our experiments provide
useful insights for future designs of FSOD methods. First fine-tuning is crucial
for FSOD. It yields significant performance gains compared to models only trained
on base classes. This makes sense as the adaptation of the model with the
prototypes is only performed in the classification branches. The regression
branches are therefore unprepared for the localization of novel classes. Of
course, having a method that does not require any fine-tuning is highly
desirable from an industrial perspective, but that should not come at the cost
of poor performance. Then, Faster R-CNN may not be the best detector choice for
few-shot extensions. Indeed, its two-stage structure duplicates the number of
modifications required for the adaptation to novel classes. Even if some works
argue that the RPN is class-agnostic, it is still trained to only detect base
classes while discarding everything else, including potential novel classes. The
RPN must then be adapted to novel classes as well. It makes the few-shot
extension more cumbersome, with more parameters and more causes for failure.
One-stage detectors certainly are a more sensible choice. Finally, the episodic
training strategy may also be inadequate for detection. It complexifies greatly
the training and introduces distractors (this concept is explained in
\cref{sec:fsod_distractors}). At the beginning of each episode, a subset of
classes (either base or novel depending on the training phase) is sampled.
Annotations from all other classes are discarded during the episode, yet the
training images still contain instances of other classes. These distractors are
confusing for the model. Of course, for classification, the episodic strategy
forces the model to establish connections between support examples and the query
images. But it is much simpler as the query images only contain one object
belonging to one of the episode classes. The presence of already-seen classes
(and potential future classes) inside query images certainly makes the episodic
training strategy suboptimal for the detection task. 

The latter paragraph formulates a few assumptions based on our observations of
the design and training of PFRCNN. Of course, it would be wise to conduct
dedicated experiments to confirm these hypotheses. For instance, carefully
designed synthetic images could help to experiment with the distractors'
influence in a controlled way. 

To conclude, this chapter presents Prototypical Faster R-CNN, a fully
metric-learning-based approach for the Few-Shot Object Detection task. This is
one of the first methods proposed in this category. Despite its relatively poor
performance on real images, it can adapt to novel classes without any
fine-tuning, which is still a rare property in the current literature. Finally,
the experiments conducted with PFRCNN provide relevant insights about the FSOD
task and will help in the design of future approaches.

\chapter{Attention Framework for Fair FSOD Comparison}
\label{chap:aaf}

\chapabstract{ Fair comparison is extremely challenging in the Few-Shot Object Detection task
as plenty of architectural choices differ from one method to another.
Attention-based approaches are no exception, and it is difficult to assess which
mechanisms are the most efficient for FSOD. In this chapter, we propose a highly
modular framework to implement existing techniques and design new ones. It
allows for fixing all hyperparameters except for the choice of the attention
mechanism. Hence, a fair comparison between various mechanisms can be made.
Using the framework, we also propose a novel attention mechanism specifically
designed for small objects.}
    \textit{
    \begin{itemize}[noitemsep]
        \item[\faPaperPlaneO] P. Le Jeune and A. Mokraoui, "A Comparative Attention Framework for Better Few-Shot Object Detection on Aerial Images", Submitted at the Elsevier Pattern Recognition journal.
        \item[\faFileTextO] P. Le Jeune and A. Mokraoui, "Cross-Scale Query-Support Alignment Approach for Small Object Detection in the Few-Shot Regime", Accepted at the IEEE International Conference on Image Processing 2023 (ICIP).
    \end{itemize}
    }
\vspace{2em}

\PartialToC

\section{Framework Presentation and Motivation}
\vspace{-1em}
As seen in \cref{chap:prcnn}, metric-learning approaches are not the optimal
choice for the FSOD task. The early FSOD literature has been dominated by
attention-based methods, which probably are a more sensible alternative. Plenty
of contributions in this domain leverage attention mechanisms for solving the
detection task. However, it is difficult to make fair comparisons between the
various mechanisms. Each method is proposed with its own choice of detection
framework, backbone, hyperparameters, loss function, augmentations and training
strategy. Thus, it is difficult to demonstrate the superiority of one attention
mechanism over others. Furthermore, there is no consensus in the FSOD field
about a proper way to evaluate the models. This can change from one work to
another and is also a source of variance preventing fair comparison in the
literature. To this end, we propose a modular framework called the
Alignment-Attention-Fusion (AAF) framework. The goal of this framework is to
allow the implementation of various attention mechanisms while keeping all other
parameters fixed. Looking closely at the existing attention-based method in the
literature (see \cref{sec:fsod_attention}), three main types of attention
mechanisms can be observed: Spatial Alignment, Global Attention and Direct
Feature Fusion. Therefore the AAF framework is structured around these three
components. The framework proposes first a mathematical formalism to present and
define existing and future mechanisms. Second, a modular Python
package\footnote{\url{https://github.com/pierlj/aaf_framework}} allows easy
implementation of attention-based methods inside a controlled detection
environment to ensure fair comparisons. In the following sections, we will
present the framework in detail and conduct fair comparison experiments with it.
Finally, a novel attention mechanism will be presented, it is designed through
the AAF framework and specifically tackles the small objects to improve
the detection performance on aerial images.

\section{Alignment Attention Fusion Framework}
\label{sec:framework}
\vspace{-1em}
In \cref{sec:fsod_attention}, three main components of attention mechanisms for
FSOD have been identified: Spatial Alignment, Global Attention and Direct
Feature Fusion. Most attention-based FSOD methods rely on one or more of these
components. Thi section will cover the Alignment, Attention and Fusion (AAF) framework,
whose purpose is to provide a flexible environment to implement attention-based
FSOD methods. Before jumping into the definition of the AAF framework, let's
recall briefly the main principle of attention-based FSOD, illustrated by
\cref{fig:aaf_attention_principle}. The goal of the attention module is to combine
the information from the query image and the support examples. Specifically, the
query features are compared with class-specific features computed from the
support set. This comparison highlights similar parts in the query image and the
support examples, yielding class-specific query features. The detection is then
performed separately for each class.   

\begin{figure}[]
    \centering
    \includegraphics[width=\textwidth]{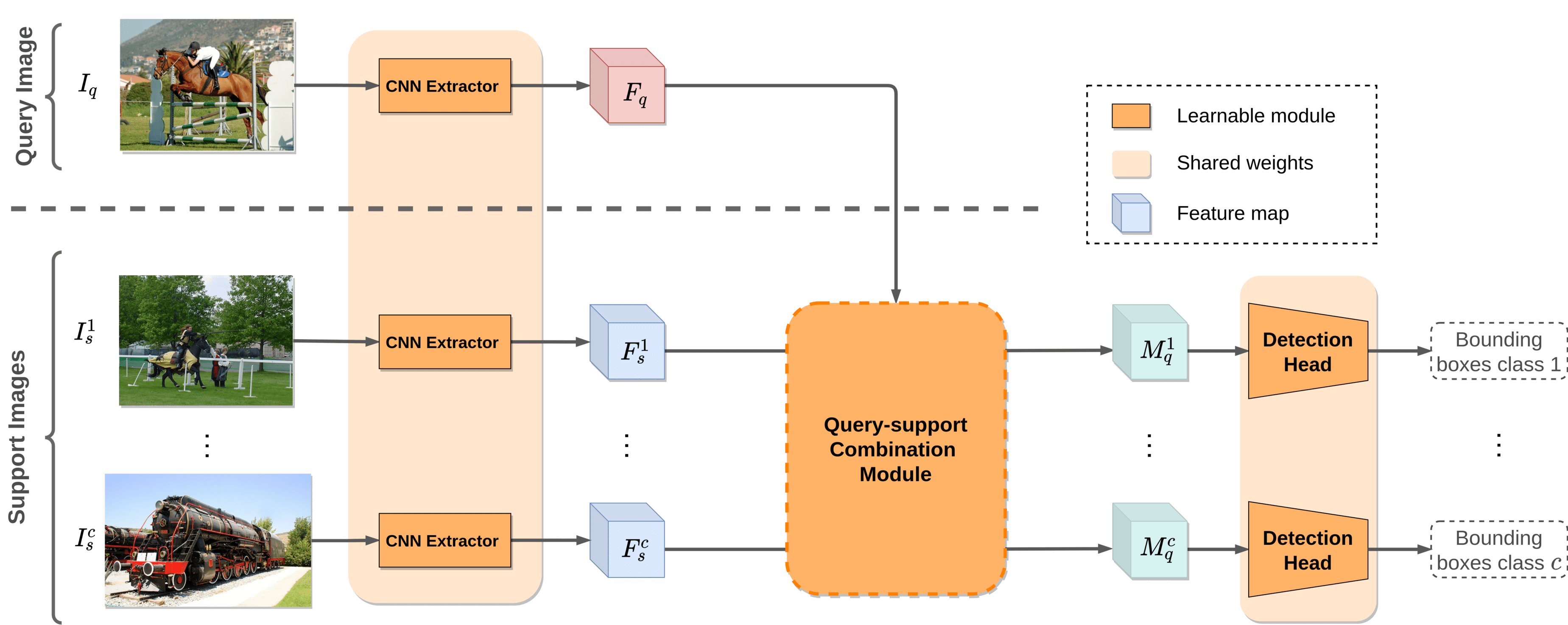}
    \caption[Attention-based Few-Shot Object Detection
    principle]{Attention-based Few-Shot Object Detection principle. Query and
    support images are processed by the backbone before being fed to the
    query-support combination block. Detection is then performed independently
    on each class.}
    \label{fig:aaf_attention_principle}
\end{figure}

The AAF framework takes as input the features from the query image $F_q$ as well
as the features extracted from every support image $F_s^c$ for $c \in
\mathcal{C}$ (if more than one example is available per class -- $K > 1$ --,
their features are averaged). It outputs class-specific query features $M_q^c$
in which features relative to class $c$ are highlighted. The framework is
divided into three parts as shown in \cref{fig:aaf_framework}, which provides an overview
of the framework. Each component is described below in dedicated sections with
examples of possible design choices. Even though this framework is presented
from the perspective of object detection, it could be applied to any kind of
few-shot visual tasks (\eg classification or segmentation).

\begin{figure}[h]
    \centering
    \includegraphics[width=\textwidth]{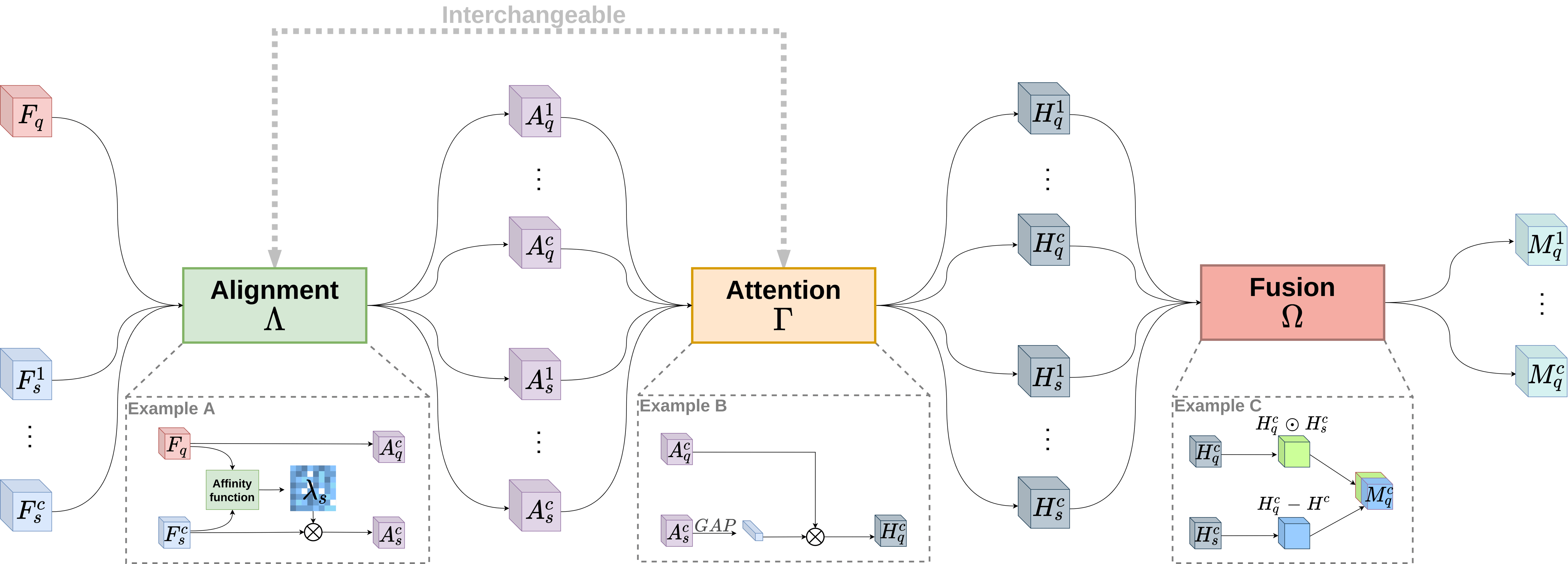}
    \caption[Overview of the AAF Framework]{The Alignment Attention Fusion (AAF)
    framework is composed of three components: spatial alignment $\Lambda$,
    global attention $\Gamma$ and a fusion layer $\Omega$. Examples for each
    module are depicted, these come from FSOD methods in the literature. Ex.
    A is presented in \cite{han2021meta}, Ex. B in \cite{kang2019few} and
    Ex. C in \cite{liu2021dynamic}. The colors chosen in this diagram match
    the colors used in
    \cref{eq:align1,eq:align2,eq:global1,eq:global2,eq:fusion}.}
    \label{fig:aaf_framework}
    \vspace{-1em}
\end{figure}

\subsection{Query-Support Alignment}
The alignment module, denoted $\Lambda$, spatially aligns the features from the
query and the support. It is unlikely that objects of the same class appear at
the same position inside query and support images (in comparison to the
classification task where objects to classify are in the center of the image).
This issue is commonly avoided by pooling the support map and using it as a
class-specific reweighting vector. But as discussed in
\cref{sec:fsod_attention}, this trick loses the spatial information about the
support object, which can be detrimental for detection. Instead, an alignment
based on similarity can be done between query and support feature maps. The idea
is to re-organize one feature map by comparing it with the other so that similar
features are spatially close in the maps (see \cref{fig:fsod_spatial_align}).
The alignment module is defined as follows:
\vspace{2em}
\begin{align}
    \tikzmarknode{g}{\highlight{Purple}{$A_q^c$}} &= \tikzmarknode{l}{\highlight{ForestGreen}{$\lambda_q(F_q , F_s^c)$}} \tikzmarknode{q}{\highlight{red}{$F_q$}},\label{eq:align1}\\
    \tikzmarknode{t}{\highlight{Purple}{$A_s^c$}} &= \tikzmarknode{l}{\highlight{ForestGreen}{$\lambda_s(F_q, F_s^c)$}} \tikzmarknode{s}{\highlight{blue}{$F_s^c$}}.\label{eq:align2}
\end{align}
\begin{tikzpicture}[overlay,remember picture,>=stealth,nodes={align=left,inner ysep=1pt},<-]
    \path (q.north) ++ (0em,1em) node[anchor=south west,color=red!67] (query){$\substack{\text{Query Features } \\ \in \, \mathbb{R}^{n\times d}}$};
    \draw [color=red!57](q.north) |- ([xshift=-0.3ex,color=red]query.north east);
    \path (s.south) ++ (0,-1em) node[anchor=north west,color=blue!67] (support){$\substack{\text{Support Features for class } c \\ \in \, \mathbb{R}^{m\times d}}$};
    \draw [color=blue!57](s.south) |- ([xshift=-0.3ex,color=blue]support.south east);
    \path (l.south) ++ (0,-1em) node[anchor=north east,color=ForestGreen!80] (lambda){$\substack{\text{Affinity matrices } \\ \in \, \mathbb{R}^{n\times m}}$};
    \draw [color=ForestGreen!80](l.south) |- ([xshift=-0.3ex,color=ForestGreen]lambda.south west);
    \path (g.north) ++ (0,1em) node[anchor=south east,color=Purple!80] (gfeat){$\substack{\text{Aligned features }}$};
    \draw [color=Purple!80](g.north) |- ([xshift=-0.3ex,color=Purple]gfeat.north west);
\end{tikzpicture}
\vspace{0.5em}

The definition of the matrices $\lambda_q$ and $\lambda_s$ determines how
features are aligned. They are mostly derived from a similarity measure between
query and support features. This formulation is close to the non-local blocks
described in \cite{wang2018non} and is at the heart of visual transformers
\cite{vaswani2017attention}. Transformers attention can be understood as an
alignment of the value to match the query-key similarity. This formulation
allows easy implementation of different feature alignments by changing the
expression of the affinity matrices. As an example, Meta Faster R-CNN
\cite{han2021meta} leverages an alignment module with affinity matrices
$\lambda_s(F_q , F_s^c) = F_q\cdot (F_s^q)^T$ and $\lambda_q(F_q , F_s^c) = I$
(see Example A in \cref{fig:aaf_framework}). Only the support features are
aligned so that they match query features. It is important to mention that
alignment alone does not combine query and support features. It rather
reorganizes spatially the query or support features. However, once the features
are aligned, their sizes match, which allows direct comparison through
element-wise operations (within the fusion layer).

\subsection{Global Attention}
The global attention module, denoted  $\Gamma$, combines global information of
the supports and the query. It highlights class-specific features and softens
irrelevant information for the task. This module is defined as follows:
\vspace{-0.5em}
\begin{align}
    \tikzmarknode{g}{\highlight{darkgray}{$H_q^c$}} &= \tikzmarknode{l}{\highlight{RedOrange}{$\gamma_q$}}(A_q^c , A_s^c),\label{eq:global1}\\
    \tikzmarknode{g}{\highlight{darkgray}{$H_s^c$}} &= \tikzmarknode{l}{\highlight{RedOrange}{$\gamma_s$}}(A_q^c, A_s^c).\label{eq:global2}
\end{align}
\begin{tikzpicture}[overlay,remember picture,>=stealth,nodes={align=left,inner ysep=1pt},<-]
    \path (l.south) ++ (0,-1em) node[anchor=north west,color=RedOrange!80] (lambda){\footnotesize Global attention operators};
    \draw [color=RedOrange!80](l.south) |- ([xshift=-0.3ex,color=RedOrange]lambda.south east);
    \path (g.south) ++ (0,-1em) node[anchor=north east,color=darkgray!80] (gfeat){\footnotesize Highlighted features};
    \draw [color=darkgray!80](g.south) |- ([xshift=-0.3ex,color=darkgray]gfeat.south west);
\end{tikzpicture}

The global attention operators $\gamma_q$ and $\gamma_s$ combine the global
information from their inputs and highlight features accordingly. This is
generally done through channel-wise multiplication. In this way, class-specific
features are highlighted, while features not relevant to the class are softened.
Changing the definition of $\gamma_q$ and $\gamma_s$ allows the implementation
of a wide variety of global attention mechanisms. This technique largely
resembles the principle of the Learnet \cite{bertinetto2016learning} introduced
for FSC. For instance, reference \cite{kang2019few} pools the support maps with
a global max pooling operation ($\text{GP}$) into a reweighting vector and
reweights the query features through channel-wise multiplication: $\gamma_q(A_q^c, A_s^c) = A_q^c
\circledast GP(A_s^c)$  and $\gamma_s(A_q^c, A_s^c) = A_s^c$ (see Example B in
\cref{fig:aaf_framework}).

\subsection{Fusion Layer}
The purpose of the fusion component is to combine highlighted query and support
maps. This is only applicable when the maps have the same spatial dimension. It
is mostly used alongside the alignment module as it does not combine the
information from the support and the query but only reorganizes the maps. In
particular, when support and query maps do not have the same spatial dimension,
aligning support maps with query maps fixes the size discrepancy. The fusion
module is then defined as follows:
\begin{equation}
    \label{eq:fusion}
    \tikzmarknode{g}{\highlight{TealBlue}{$M_q^c$}} = \tikzmarknode{l}{\highlight{Mahogany}{$\Omega$}}(H_q^c , H_s^c).
\end{equation}
\begin{tikzpicture}[overlay,remember picture,>=stealth,nodes={align=left,inner ysep=1pt},<-]
    \path (l.south) ++ (0,-1em) node[anchor=north west,color=Mahogany!80] (lambda){\footnotesize Fusion operator};
    \draw [color=Mahogany!80](l.south) |- ([xshift=-0.3ex,color=Mahogany]lambda.south east);
    \path (g.south) ++ (0,-0.9em) node[anchor=north east,color=TealBlue!80] (gfeat){\footnotesize Merged query features};
    \draw [color=TealBlue!80](g.south) |- ([xshift=-0.3ex,color=TealBlue]gfeat.south west);
\end{tikzpicture}

The highlighted maps can be combined through any point-wise operation, addition
$\oplus$, multiplication~$\odot$, subtraction $\ominus$, concatenation $[\cdot,
\cdot ]$ or more sophisticated ones. An example of such a fusion module is
presented in \cite{liu2021dynamic}. The fusion operator concatenates the results
of the addition and the subtraction of the highlighted features: $\Omega(H_q^c,
H_s^c) = [H_q^c \oplus H_s^c, H_q^c \ominus H_s^c]$ (see Example C in
\cref{fig:aaf_framework}). The point-wise operators can also contain small
trainable models such as in \cite{han2021meta}, where small CNNs (\eg
$\psi_{dot}$, $\psi_{sub}$, and $\psi_{cat}$) are applied after the point-wise
operators, but before the concatenation: $\Omega(H_q^c, H_s^c) =
[\psi_{dot}(H_q^c \odot H_s^c), \psi_{sub}(H_q^c \ominus H_s^c),
\psi_{cat}([H_q^c, H_s^c])]$.

Except for the fusion layer which must be applied last, spatial alignment and
global attention can be applied in any order. This flexibility is required to
cover methods that apply global attention before alignment, such as DANA
\cite{chen2021should}. The whole architecture of the AAF framework is
illustrated in \cref{fig:aaf_framework}, in which examples from the previous
sections are also depicted. The modular structure of the framework enables the
implementation of a wide variety of attention mechanisms while keeping all other
hyperparameters fixed. In this way, it is a useful tool for FSOD methods
comparison.

\subsection{Implementation details}
\label{sec:aaf_implementation_details}
Before presenting the results of a fair comparison between several FSOD
approaches re-implemented in the AAF framework, we must review its
implementation. To make the comparisons fair, some design choices are kept fixed
in the framework. The backbone is a ResNet-50 with a 3-layers Feature Pyramid
Network on top. It extracts features at 3 different levels, which helps the
network to detect objects of various sizes. The detection head is based on FCOS
\cite{tian2019fcos}, a one-stage detector (a choice motivated by insights from
\cref{chap:prcnn}). The head is composed of a few convolutional layers with
ReLU activations followed by two branches (convolutional as well) for
classification and regression respectively. The AAF framework is applied between
the backbone and the detection head. As features are extracted at multiple
levels, attention mechanisms are also implemented to work at different scales.
This may differ from the original implementations, but most methods are already
designed to work at multiscale (see \cref{tab:fsod_comparison}). The model is
trained in an episodic manner. During each episode, a subset $\mathcal{C}_{ep}
\subset \mathcal{C}$ of the classes is randomly sampled. A query set is sampled
for each episode, containing 100 images per class. This set only contains
annotations of the episode classes and is leveraged for the loss computation and
the optimization of the model. A support set is also sampled at the beginning of
each episode containing $K$ examples for each episode class. The support
examples are used through the attention mechanisms to condition the detection on
the episode classes. 

The training is divided into two phases \textit{base training} and
\textit{fine-tuning}. During base training, only base classes can be sampled
($\mathcal{C}_{ep} \subset \mathcal{C}_{base}$) and one image per class is
drawn for the support set ($K=1$). The optimization is done with SGD and a
learning rate of $\num{1e-3}$ for 1000 episodes. During \textit{fine-tuning},
the backbone is frozen, the learning rate is divided by 10, and the episode
classes can be sampled from $\mathcal{C}_{base} \cup \mathcal{C}_{novel}$, with
at least one novel class per episode. Examples from novel classes are selected
among the $K$ examples sampled once before fine-tuning. Each model is fine-tuned
separately for different values of $K \in {1,3,5,10}$. During both training
phases, the same loss function is optimized, as defined in FCOS (see
\cref{tab:od_losses}).

\subsubsection{Augmentations and Cropping Strategies}
\label{sec:aaf_aug_crop}
Some existing works leverage sophisticated training strategies (e.g. auxiliary
loss functions \cite{fan2020few}, hard examples mining
\cite{zhang2021hallucination} or multiscale training \cite{deng2020few}). While
this certainly improves the quality of the detections, it introduces new
parameters to tune and makes the comparison with other works difficult. As the
focus of this study is on attention mechanisms, we choose not to reimplement all
these improvements. However, to remain competitive with existing works, we
propose a novel augmentation pipeline specifically designed for object
detection. It is defined in the next paragraph which includes a cumulative study
of its different components on DOTA. In addition, we discuss the choice of the
support extraction strategy. Basically, this refers to how the support examples
are extracted from the support images since most parts of these images are
irrelevant for the task. From our analysis, it seems that this design choice
significantly influences the model performance (see
\cref{tab:cropping_methods}). However, it is barely discussed in the FSOD
literature. We find that the best strategy is to crop the support example and
resize it to a fixed-size patch. This strategy is thus fixed for all our
experiments.

\paragraph*{Augmentations}
To improve the performance of the methods implemented in the AAF framework and
be competitive with existing works, we propose an augmentation pipeline
specifically designed for detection. Some regular augmentation techniques cannot
be directly applied for object detection as they can completely mask objects from
the image. This deteriorates the training as the model will not be able to
detect hidden objects, but it will be penalized anyway. 

First, we apply random horizontal and vertical flips (only for aerial images)
and color jitter. As it does not remove entire objects, these can be applied
directly to the images. However, some other classical techniques such as random
crop-resize and random cut-out cannot be applied directly. Therefore, we developed
object-preserving random crop-resize and cut-out. The main idea is
to apply these transformations at the object level and not at the image level.
This ensures that objects of interest are still visible in the transformed
image. For crop-resize, a non-empty subset of the objects in the image is
randomly sampled. An overall bounding box is computed around all these objects
and the cropped area is randomly drawn between this box and the image borders.
Hence, it guarantees the presence of at least one object inside the cropped
image. For cut-out, the principle is similar, instead of cutting out a random
part of the image, the cut is applied at the object level so that it does not
hide out entire objects. \cref{fig:aaf_augmentations} compares the two
proposed augmentations with their naive implementations. 

\begin{figure}[h]
    \centering
    \includegraphics[width=0.8\textwidth]{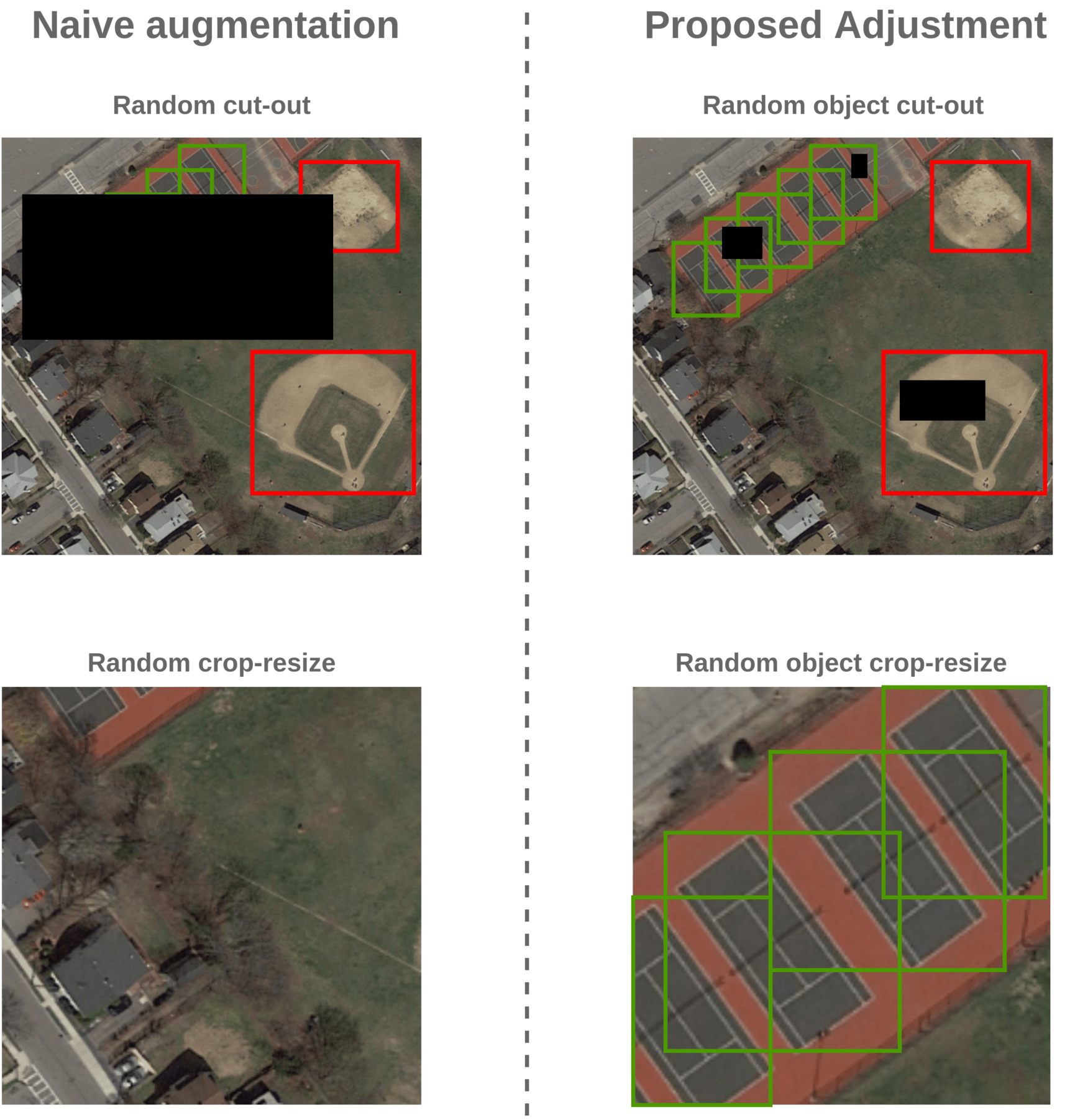}
    \caption[Object-level augmentation techniques]{Difference between naive
    augmentation techniques (left) and our adaptation to object detection
    (right). The proposed transformations are applied at the object-level to
    preserve objects integrity.}
    \label{fig:aaf_augmentations}
    \vspace{-1em}
\end{figure}

We performed a cumulative study to assess the benefits of each component of the
augmentation pipeline. To do so, we implemented Feature ReWeighting (FRW)
\cite{kang2019few}, a well-known FSOD technique, within the AAF framework. FRW
is then trained on DOTA dataset. This experiment is summarized in
\cref{tab:augmentation-performance}. It shows that the augmentation is
beneficial for the performance on novel classes but detrimental for base
classes. Surprisingly, performance drops on base classes with augmentation. 
More specifically, it seems that image flips are responsible for the performance
loss on base classes (see first and second columns in
\cref{tab:augmentation-performance}). Base classes performance drops when adding
flips but remains mostly constant when adding other types of augmentations. One
crucial difference between flips and other augmentations is that we choose to
apply flips also on support examples. This choice was made to increase the
diversity of the support set during fine-tuning. For novel classes, only a few
images are available as support during fine-tuning, and we want to avoid
overfitting these examples. Although other types of augmentations could have
been employed for this, we wanted to avoid disrupting too much of the
information in the support. This choice may be the cause of the performance drop
on base classes. To verify this hypothesis, we conduct a few more experiments
disabling the flip in the support set. With the \textit{default} cropping
strategy (see next paragraph), the experiments confirm the hypothesis: no
performance drop is observed when supports are not flipped. However, support
augmentation certainly interacts with the support cropping strategy, thus we
also tried with the \textit{same-size} cropping strategy. Surprisingly, it does
not produce similar results, and in this case, flipping support examples is
actually beneficial for base classes performance. This suggests a complex
interaction between augmentation on the support set and the cropping strategy.
The choice made in our experiments may not be optimal in this regard, and a
deeper analysis of this interaction should be conducted in future work (\eg
studying the effect of various augmentations on a synthetic dataset to have
better control over the images). Finally, the base class performance loss is
compensated by clear improvements on novel classes. As this is the main goal of
FSOD, we choose to adopt the original augmentation pipeline, including flips in
the support set, for all our experiments. Other augmentations are not applied to
the support set to prevent disrupting their representations and therefore the
conditioning of the model.

\begin{table}[]
    \centering
    
    \resizebox{0.7\textwidth}{!}{%
    \begin{tabular}{@{\hspace{3mm}}ccccccc@{\hspace{3mm}}}
        \toprule[1pt]
        \textbf{\# Shots}   & \textbf{}      & \textbf{Baseline} & \textbf{+ Flip} & \textbf{+ Color} & \textbf{+ Cutout} & \textbf{+ Crop} \\ \midrule
        \multirow{2}{*}{1}  & \textbf{Base}  & \bbf{48.83}    & 45.80           & 45.96	         &47.20	             &45.68 \\
                            & \textbf{Novel} & 6.15             & 5.25             & 6.92            &6.44	            &\rbf{10.03}\\ \addlinespace[1mm]
        \multirow{2}{*}{3}  & \textbf{Base}  & \bbf{51.06}    & 47.70           & 47.03	         &46.10	             &45.22 \\
                            & \textbf{Novel} & 14.41             & 18.59           & 18.59	         &19.74	             &\rbf{21.95} \\ \addlinespace[1mm]
        \multirow{2}{*}{5}  & \textbf{Base}  & \bbf{52.66}    & 49.38           & 50.09	         &50.28	             &48.74 \\
                            & \textbf{Novel} & 19.25             & 23.71           & 25.08	         &25.01	             &\rbf{25.95} \\ \addlinespace[1mm]
        \multirow{2}{*}{10} & \textbf{Base}  & \bbf{53.84}    & 50.80           & 50.77	         &50.41	             &50.27 \\
                            & \textbf{Novel} & 28.56             & 31.23           & 28.08	         &34.13	             &\rbf{35.95} \\ \bottomrule[1pt]
        \end{tabular}
    } \caption[Cumulative study about augmentations]{Cumulative study of the
    proposed augmentation techniques on DOTA using FRW \cite{kang2019few}.
    $\text{mAP}_{0.5}$ is reported for different number of shots.}
    \label{tab:augmentation-performance}
    \end{table}

\paragraph*{Support example cropping strategy}
The support information is only located inside the example bounding box.  The
remaining part of the support image mostly contains irrelevant information
concerning the object class. Therefore, extracting features from the whole
support images is not necessary. But features contained only inside the object's
bounding box might not be sufficient as well. Context can be extremely valuable
in certain cases, especially for small objects. For instance, a car and a small
boat could easily be mistaken without context. Close surroundings of the objects
can help for recognition. 

A common strategy for support information extraction is proposed by
\cite{kang2019few}. They concatenate in the channel direction the support image with the support
object's binary mask (rectangular, computed from the bounding box) and pass this
to an extractor network. This has two main drawbacks. First, it is necessary to
compute features from the entire support image, which is a loss of resources.
Second, the same network cannot be used for extracting features in query and
support images as it needs an additional input channel to process the mask.
Hence, the network cannot be pretrained beforehand. This design choice is rarely
discussed, if ever mentioned, in the literature.

In this section, we explore this design choice by implementing several
extraction strategies. We did not reimplement the technique from
\cite{kang2019few} as it requires to have two different networks for support and
query feature extraction. However, some of our techniques are quite close to
what they proposed. These techniques are described below and
\cref{fig:aaf_cropping} illustrates most of them:

\begin{itemize}[nolistsep]
    \item[-]  \textbf{Default}: the most naive extraction technique. It consists in
    cropping the support image around the support object at a fixed size (e.g. $128\times128$).
    Objects larger than this are simply resized to fit in the patch.
    \item[-]  \textbf{Context-padding}: the cropping occurs exactly as with the
    default strategy, but areas around the objects are masked out. This is close
    to what was proposed by \cite{kang2019few}.
    \item[-]  \textbf{Reflection}: context is replaced by reflection of the
    object. In the case of small objects, the support patch can easily be dominated either
    by irrelevant information or by zeros when using the latter two extraction methods.
    \item[-]  \textbf{Same-size}: all objects are resized to fill the
    support patch entirely (preserving the aspect ratio). It does not change the process
    for large objects, but it prevents smaller objects from being dominated by
    irrelevant information.
    \item[-] \textbf{Multi-scale}: the object is resized and cropped at 3
    different scales. Each scale is responsible for matching small, medium and
    large objects in query images.
    \item[-] \textbf{Mixed}: it is a combination of the default strategy and
    \textit{same-size}. Small objects (i.e. $\sqrt{wh} < 32$) are extracted
    using the default strategy. Larger objects ($\sqrt{wh} \geq 32$) are resized
    into a patch of $128 \times 128$ pixels. Therefore, small objects are not
    upscaled more than 4 times, as they are using the \textit{resize} strategy.
\end{itemize}

\begin{figure}[]
    \centering
    \includegraphics[width=\textwidth]{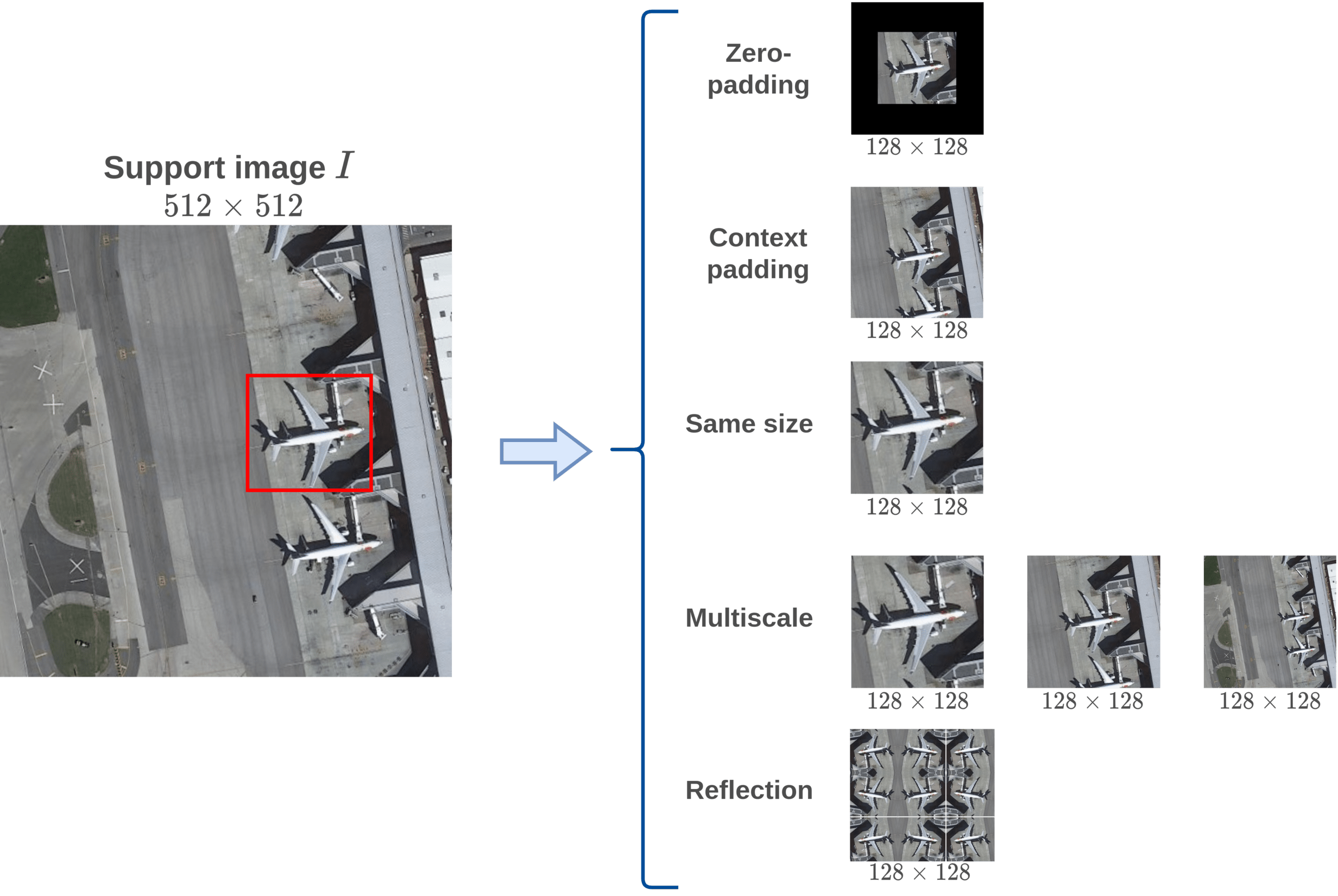}
    \caption[Support example cropping strategies]{Illustration of the different cropping strategies tested. The
    mixed strategy is not illustrated as it is a combination of \textit{default}
    and \textit{same-size}.}
    \label{fig:aaf_cropping}
\end{figure}

These strategies are compared in Table \ref{tab:cropping_methods}. Even though
\textit{same-size} gets the best overall results on novel classes (regardless of
object sizes), there is no clear conclusion. It is outperformed by both
\textit{reflection} and \textit{mixed} for base classes. No method outperforms
the others on all object sizes, not even the ones designed to be more robust to
size (\textit{multiscale} and \textit{mixed}). The latter two techniques
introduce discrepancies in the features: objects of similar size (\eg from both
sides of the size limit) are processed differently, resulting in really different
features. As \textit{same-size} gives the best results on novel classes, we
choose to use this strategy for all our experiments. Yet, in the light of our
performance analysis in \cref{chap:aerial_diff}, we can understand some results
from \cref{tab:cropping_methods}. The \textit{multiscale} strategy does not
perform very well as it introduces small objects features which seems
detrimental for the good conditioning of the network. On the contrary,
\textit{same-size} only generates large objects as support which is a better
strategy. Finally, \textit{reflection} performs surprisingly well for small
objects while preserving their small size. The redundancy generated by the
reflection of such small objects certainly reinforces the object's features.

\begin{table}[]
    \centering
    
    \resizebox{0.8\textwidth}{!}{%
    \begin{tabular}{@{\hspace{3mm}}lccccccccc@{\hspace{3mm}}}
        \toprule
        \multirow{2}{*}{\textbf{Cropping Strategy}}  & \multicolumn{4}{c}{\textbf{Base classes}}                                &  & \multicolumn{4}{c}{\textbf{Novel classes}}                          \\ \cmidrule(lr){2-5} \cmidrule(lr){7-10}
                                 & \textbf{All}    & \textbf{S}      & \textbf{M}      & \textbf{L}    &  & \textbf{All}    & \textbf{S}      & \textbf{M}      & \textbf{L}                \\ \midrule
        \textbf{Default}         & 49.80	     & 24.26	      & 57.36          & \bbf{63.90}            &  & 24.97          &	7.66	     &24.31           &34.48   \\
        \textbf{Context padding} & 50.03	     & 24.96	      & 59.72          & 63.03                  &  & 29.60          & \rbf{10.71}	         &26.76           &50.91   \\
        \textbf{Same-size}       & 50.63	     & \bbf{30.83}	      & 59.44          & 62.47              &  & \rbf{32.19}          &	8.32	     &33.21           &\rbf{56.54}   \\
        \textbf{Multiscale}      & \bbf{51.44}	     & 29.03	      & 59.78          & 63.27              &  & 26.95          &	8.44	     &\rbf{33.48}           &45.63   \\
        \textbf{Reflection}      & 50.28	     & 26.05	      & 59.36          & 62.49                  &  & 25.50          &	7.22	     &20.66           &44.29   \\
        \textbf{Mixed}           & 50.95	     & 27.16	      & \bbf{60.48}          & 60.67            &  & 27.96          &	9.51	     &26.52           &48.89   \\ \bottomrule
        \end{tabular}%
    } \caption[Influence of support cropping strategies on FSOD
    performance]{Comparison of support extraction strategies on base and novel
    classes with DOTA dataset and FRW method with 10 shots. $\text{mAP}_{0.5}$
    is reported on all objects and separately on objects of different sizes:
    small (S), medium (M) and large (L).}
    \label{tab:cropping_methods}
    \end{table}

\subsection{Fair comparison of Few-Shot Object Detection Methods with AAF}
\label{sec:aaf_fair_comparison}

To showcase the flexibility of the proposed AAF framework, we reimplement and
compare multiple existing works. Some methods described in \cref{chap:fsod} are
selected: FRW \cite{kang2019few}, WSAAN \cite{xiao2020fsod}, DANA
\cite{chen2021should}, Meta R-CNN \cite{han2021meta} and DRL
\cite{liu2021dynamic} (see \cref{tab:fsod_comparison}). They have been chosen
because they represent well the variety of attention mechanisms available in the
literature and according to their popularity. FRW is based on class-specific
reweighting vectors, WSAAN has a more sophisticated global attention and
computes reweighting vectors inside a graph structure. DANA and Meta R-CNN
leverage query-support alignment in different manners and DRL only uses a
sophisticated fusion layer. Each of these methods has been reimplemented within
the AAF framework. Of course, some details differ from the original
implementations, but the purpose of this comparison is to compare only the
query-support combination module. In particular, the backbone and the training
strategy (losses and episode tasks) may differ. We first conduct such a
comparative experiment on Pascal VOC \cite{everingham2010pascal} and MS COCO
\cite{lin2014microsoft} datasets. On these datasets, the performance achieved by
our implementations is close (\ie within 2 points of mAP$_{0.5}$) to the values reported
in the original papers. Then, we use the framework to compare the performance of
some methods on two aerial datasets:  DOTA \cite{xia2018dota} and DIOR
\cite{li2020object}.
\subsubsection{Evaluation protocol}
The evaluation is also conducted in an episodic manner, following
recommendations from \cite{huang2021survey}. The performance is averaged over
multiple episodes, each containing 500 examples for each class and this
operation is repeated multiple times with randomly sampled support sets. The
query and support examples are drawn from the test set, thus the support
examples are different from the ones used during fine-tuning. This prevents
overestimations of the performance due to overfitting on the support examples.
To ensure a fair comparison between the various methods, the same random seed is
used for all evaluations, thus the support and query examples are the same. 

\subsubsection{Natural images}
\label{sec:aaf_natural_images}
\cref{tab:aaf_result_voc} gather the results on Pascal VOC. First, compared to
the FCOS baseline, a slight performance drop on base classes is observed. This
is expected, even if the model has seen a lot of examples of these classes
during training, its predictions are still conditioned on a few examples, which
can sometimes be misleading. On the other hand, performance on novel classes is
significantly lower than the FCOS baseline, especially for low numbers of shots.
The number of shots is crucial for performance on novel classes. The higher the
number of shots, the better the network performs. On average, with 10 examples
per class, the network achieves $0.2$ higher mAP$_{0.5}$ than with 1 example. More
examples provide more precise and robust class representations, improving the
detection. The same phenomenon is observed with base classes to a lesser extent
($+ 0.04$ mAP from 1 to 10 shots). \cref{fig:aaf_method_compare}  displays these
trends clearly, both for base and novel classes. In addition,
\cref{fig:aaf_map_per_class} provides the same results split by class. An
interesting observation from this last figure is the very good detection
performance for the novel class \textit{sheep}. This can be explained easily
from the presence of the class \textit{horse} in the base set. The model has
seen a lot of examples of horses during base training, which makes it learn
visual attributes common with a sheep (\eg four legs, hair and grassy
background). Such a class similarity makes the novel class detection much easier.
Some authors do leverage this fact, as for instance \cite{cao2021few} which
first associates a base class to each novel class before learning to
discriminate between them.

The behavior just described is expected from any few-shot object detection
method. Moreover, performance values are close to what is reported in the
original papers of the reimplemented methods. These results are not the same as
many architectural choices differ from the proposed methods (e.g. backbone,
class splits, etc.). Nevertheless, it confirms that the proposed AAF framework
is flexible enough to implement a wide variety of attention mechanisms. It is
therefore an appropriate tool to compare and design query-support attention
mechanisms.

\begin{table}[]
    \centering
    \resizebox{\textwidth}{!}{%
    \begin{tabular}{@{\hspace{1mm}}cccccccccccccc@{}}
    \toprule[1pt]
                      & \multicolumn{2}{c}{\textbf{FRW} \cite{kang2019few}} &   \multicolumn{2}{c}{\textbf{WSAAN} \cite{xiao2020fsod}} &   \multicolumn{2}{c}{\textbf{DANA} \cite{chen2021should}} &   \multicolumn{2}{c}{\textbf{Meta R-CNN} \cite{han2021meta}} &   \multicolumn{2}{c}{\textbf{DRL} \cite{liu2021dynamic}}&  &  \multicolumn{2}{c}{\textit{\textbf{FCOS Baseline}}} \\ \hline
    $\boldsymbol{K}$  & \underline{Base}         & \underline{Novel}        &   \underline{Base}        & \underline{Novel}            &   \underline{Base}      & \underline{Novel}               &   \underline{Base}         & \underline{Novel}          &   \underline{Base}    & \underline{Novel}               &  &  \underline{\textit{Base}}& \underline{\textit{Novel}}        \\
    1                 & 59.92	                 & 28.22	                &61.70	                    &30.95	                       &62.58	                 &\rbf{32.82}	                           &57.85	                    &30.16	                     &\bbf{64.18}	                 &27.05                            &  &  \multirow{4}{*}{\textit{65.47}}  &   \multirow{4}{*}{\textit{68.02}}                   \\
    3                 & 63.33	                 & 31.12	                &63.52	                    &\rbf{42.19}	                       &\bbf{64.18}	             &33.95	                           &58.70	                    &36.79	                     &61.74	                 &29.64                            &  &                           &                          \\
    5                 & 64.35	                 & \rbf{46.33}	                &64.68	                    &46.16	                       &65.20	                 &42.58	                           &62.14	                    &40.75	                     &\bbf{66.45}	                 &37.34                            &  &                           &                          \\
    10                & 63.16	                 & 48.71	                &65.27	                    &\rbf{51.70}	                       &65.03	                 &50.30	                           &63.38	                    &49.45	                     &\bbf{66.98}	                 &47.99                            &  &                           &  \\ \bottomrule[1pt]
    \end{tabular}%
    } \caption[FSOD performance of 5 existing methods in AAF
    framework on Pascal VOC]{Performance comparison between five selected methods on Pascal
    VOC. All are reimplemented with the proposed AAF framework. Mean average
    precision is reported for each method on base and novel classes separately
    and for various numbers of shots ($K \in \{1,3,5,10\}$).}
    \label{tab:aaf_result_voc}
    \end{table}

\begin{figure}
    \centering
    \includegraphics[width=0.6\textwidth, trim=10 0 0 10,clip]{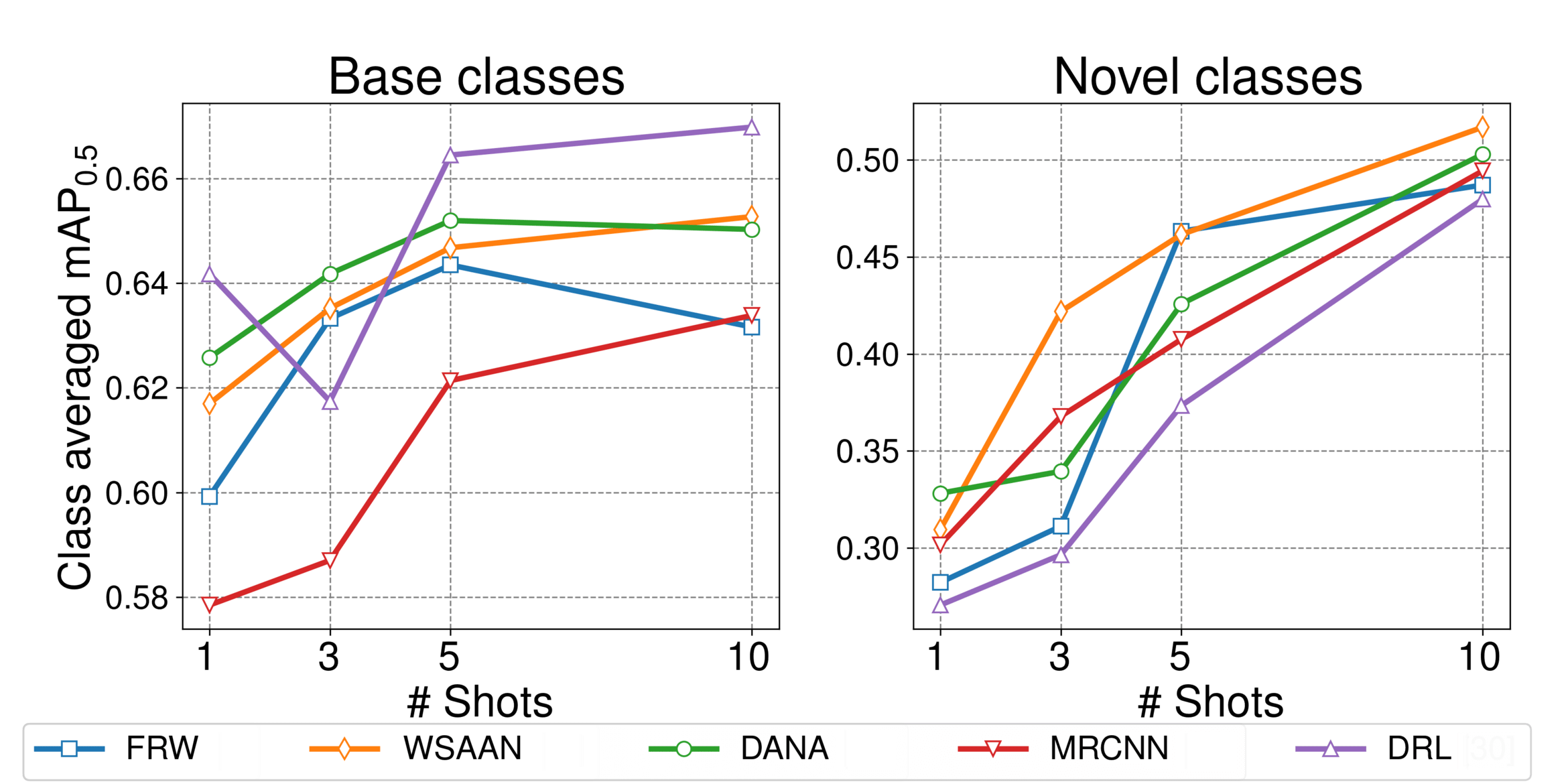}
    \caption[Influence of the number of shots on FSOD performance on Pascal
            VOC]{Evolution of $\text{mAP}_{0.5}$ with the number of shots
            averaged on base and novel classes separately. Each line represents
            one of the reimplemented methods.}
    \label{fig:aaf_method_compare}
\end{figure}

DRL is arguably the simplest method among the five selected as it leverages only
a fusion layer. It combines query features with the features of each support
image through concatenation and point-wise operations, creating class-specific
query features. It is therefore the closest to the regular FCOS functioning.
This explains the very good performance on base classes and lower mAP on novel
classes, compared to the baseline. Regarding the other methods, FRW and WSAAN
can be easily compared as both are based on global attention. The only
difference is how the class-specific vectors are computed. In FRW, they are
globally pooled from the support feature map. However, WSAAN combines the same
vectors with query features in a graph. This certainly provides better
class-specific features and in the end, better results both on base and novel
sets. The remaining methods, DANA and Meta R-CNN both leverage spatial alignment.
While it seems to bring quite an improvement for DANA over FRW and DRL, the gain
is smaller for Meta R-CNN. In both methods, spatial alignment is not used alone. It
is combined with other attention mechanisms. In DANA, a Background Attenuation
block (i.e. a global self-attention mechanism) is applied to the support features to
highlight class-relevant features and soften background ones. In Meta R-CNN, aligned
features are reweighted with global vectors computed from the similarity matrix
between query and support features. This last operation may be redundant as the
similarity information is already embedded into the aligned features, whereas
background attenuation leverages new information. 

\begin{figure}
    \centering
    \includegraphics[width=\textwidth, trim=0 0 0 5,clip]{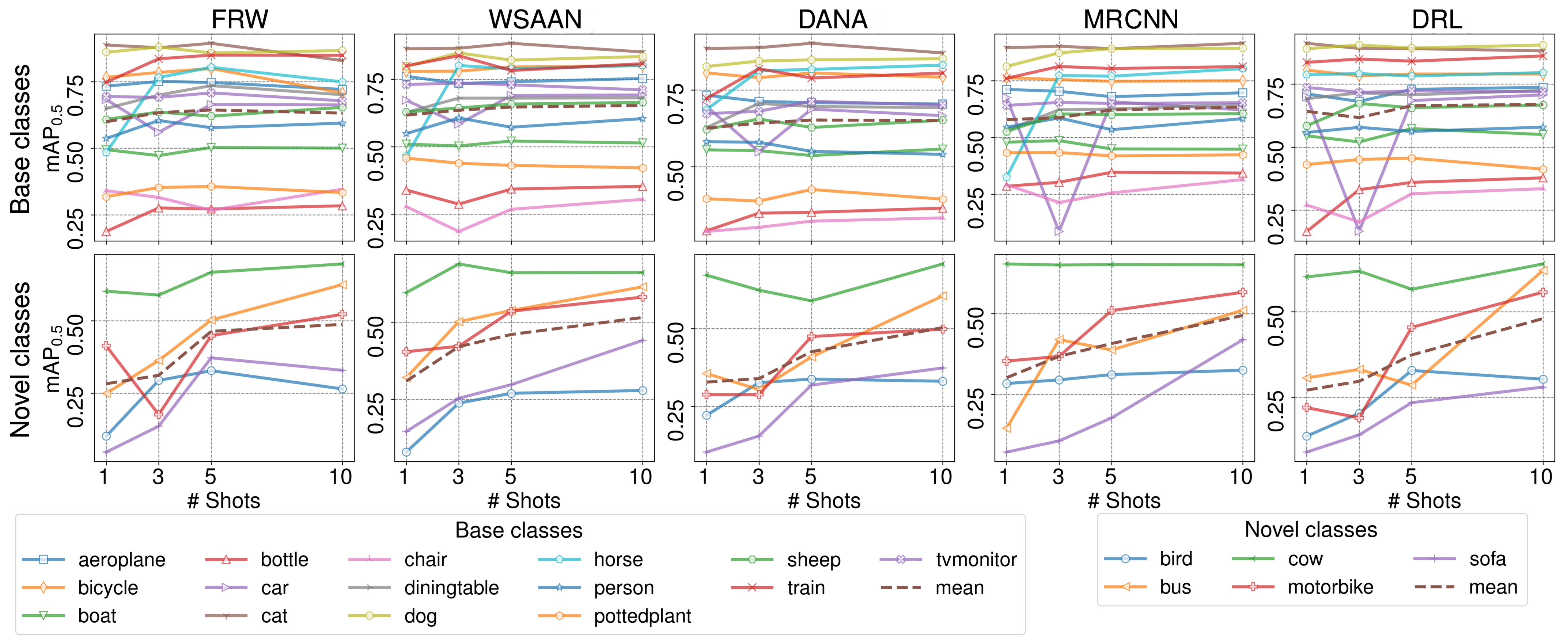}
    \caption[Influence of the number of shots on FSOD performance on Pascal VOC
    per class]{$\text{mAP}_{0.5}$ on Pascal VOC against the number of shots for
    each class and each method. Dashed lines represent average performance on
    all classes, either base classes \textbf{(top)} or novel classes
    \textbf{(bottom)}.}
    \label{fig:aaf_map_per_class}
\end{figure}

From this comparison, one can conclude that both global attention and spatial
alignment are beneficial for FSOD. However, these improvements may not always be
compatible, as shown by the results of Meta R-CNN. Hence, the design of each
component must be done carefully so that spatial alignment, global attention,
and fusion work in unison.

Another set of experiments is conducted on MS COCO dataset. Only the two
best-performing methods on Pascal VOC are selected and trained on MS COCO
following the same experimental setup. The results are summarized in
\cref{tab:aaf_result_coco}. The mAP values are reported following standards from
Pascal VOC ($\text{mAP}_{0.5}$ with one IoU threshold), and MS COCO
($\text{mAP}_{0.5:0.95}$ with several thresholds). MS COCO is a much more
difficult detection benchmark; therefore the numbers of shots is adjusted to 1,
5, 10, and 30 shots. These results comfort the conclusion obtained on Pascal
VOC: the framework is flexible enough to implement various FSOD techniques that
achieve competitive results with state-of-the-art. As for Pascal VOC the models
achieve better performance with more shots. However, unlike on Pascal VOC, base
classes also benefit significantly from a larger number of examples on MS COCO.
MS COCO is more difficult, therefore, the information extracted from the
supports better helps the models. Finally, WSAAN outperforms DANA on Pascal VOC
but performs slightly worse on MS COCO. It can be noted that the results
obtained on a dataset cannot be extrapolated to another without taking into
account the characteristics of the datasets. A method that performs best on a
dataset is not guaranteed to do so on another dataset. This reinforces the need
of a flexible framework that allows fair and easy comparison between FSOD
methods. That way, the best-performing method can be easily selected for a given
problem. Without such a framework, it is difficult to find out from the
literature which method is the most promising for a given application as most
FSOD works focus on natural images. For COSE, this framework is highly valuable
as it will serve as an objective comparison tool for attention-based FSOD
methods.

\begin{table}[]
    \centering
    \resizebox{0.7\textwidth}{!}{%
    \begin{tabular}{@{\hspace{2mm}}ccccccccccccc@{\hspace{2mm}}}
    \toprule[1pt]
    &                        &  & \multicolumn{4}{c}{\textbf{WSAAN} \cite{xiao2020fsod}}                                   &  & \multicolumn{4}{c}{\textbf{DANA} \cite{chen2021should}}                                     \\
                             &  & \multicolumn{2}{c}{$\text{mAP}_{0.5}$} &  &  \multicolumn{2}{c}{$\text{mAP}_{0.5:0.95}$} &  & \multicolumn{2}{c}{$\text{mAP}_{0.5}$}  &  & \multicolumn{2}{c}{$\text{mAP}_{0.5:0.95}$}    \\ \midrule
    $\boldsymbol{K}$         &  & \underline{Base}  &  \underline{Novel} &  & \underline{Base}        & \underline{Novel}  &  & \underline{Base}   & \underline{Novel}  &  & \underline{Base}   & \underline{Novel}         \\
    1                        &  & 33.51	            & 11.97	             &  & 20.12                   &	6.57	           &  & \bbf{35.52}	           & \rbf{14.54}	            & & \bbf{21.30}	            & \rbf{7.77}\\
    5                        &  & 39.88	            & 19.86	             &  & 23.62                   &	10.54              &  &	\bbf{42.79}              & 	\rbf{22.17}           & & 	\bbf{25.19}           & 	\rbf{11.90}\\
    10                       &  & 40.87	            & 21.42	             &  & 24.38                   &	11.49              &  &	\bbf{43.00}              & 	\rbf{23.70}           & & 	\bbf{25.57}           & 	\rbf{12.89}\\
    30                       &  & 41.54	            & 22.21	             &  & 24.74                   &	12.08              &  &	\bbf{43.54}              & 	\rbf{24.36}           & & 	\bbf{25.96}           & 	\rbf{13.35}\\ \bottomrule[1pt]
    \end{tabular}%
    } \caption[Comparison of FSOD performance on MS COCO between WSAAN and
    DANA]{Performance comparison between WSAAN \cite{xiao2020fsod} and DANA
    \cite{chen2021should} on MS COCO. $\text{mAP}_{0.5:0.95}$ (MS COCO mAP, with
    IoU thresholds ranging from 0.5 to 0.95) and $\text{mAP}_{0.5}$ values are
    reported for base and novel classes separately and for different numbers of
    shots: $K \in \{1, 5, 10, 30\}$.}
    \label{tab:aaf_result_coco}
    \end{table}

From these experiments on natural images, it seems clear that DANA performs
best. Therefore, it highlights the importance of feature alignment for
query-support matching. Global attention loses spatial information in support
features which is detrimental to detection. However, global attention methods
should not be overlooked. WSAAN shows impressive performance and even
outperforms slightly DANA on Pascal VOC. It could be interesting to combine both
methods, but this does not seem to be trivial as demonstrated by the results of
Meta R-CNN, which leverages the alignment from DANA and the attention
from FRW, but does not yield better results.

\subsubsection{Aerial images}

To our knowledge, very few works evaluate FSOD methods on aerial images at the
time we proposed the AAF framework. Among those we select FSOD-RSI
\cite{xiao2020few}, which simply applies FRW to aerial images (we will refer to
it as FRW), WSAAN \cite{deng2020few} and our PFRCNN.
In addition, we include DANA inside this comparison as it was the
best-performing technique on natural images. All these methods are evaluated on
different datasets, making their performance comparison challenging. Using the
proposed AAF framework, we compare the performance of these methods on both DOTA
and DIOR. These methods are reimplemented inside the framework and all other
design choices are fixed during the experiments (as described in
\cref{sec:aaf_implementation_details}). \cref{tab:aaf_result_aerial} regroups the
results of the comparison. These results show a slight improvement over the
state-of-the-art on DIOR (WSAAN \cite{deng2020few}). Our implementation of WSAAN
outperforms (8 mAP$_{0.5}$ points on base classes and 0.4 on novel classes) the result
reported in the original paper. However, the attention mechanism employed in
WSAAN is not optimal for aerial images. WSAAN is outperformed by both FRW and
DANA. While this was expected for DANA in the light of results from
\cref{sec:aaf_natural_images}, it was not for FRW. The superiority of DANA over the other
methods on DOTA and DIOR is clear and coherent with the results on natural
images. The more sophisticated attention mechanism, in particular the alignment,
from DANA is better at extracting and leveraging the information from the
support examples. Hence, the detection performance is higher. It is particularly
beneficial for small numbers of shots: the extracted information is semantically
robust.

\begin{table}[]
    \centering
    \resizebox{\textwidth}{!}{%
    \begin{tabular}{@{}ccccccccccccccccccccccccccc@{}}
        \toprule[1pt]
        \multicolumn{1}{l}{} & \multicolumn{11}{c}{\textbf{DOTA}}                                                                                                                                                                                              &                      & \multicolumn{14}{c}{\textbf{DIOR}}                                                                                                                                                                                                                     \\ \cmidrule(lr){2-12} \cmidrule(lr){14-27}
                             & \multicolumn{2}{c}{\textbf{FRW}}                      & \textbf{} & \multicolumn{2}{c}{\textbf{WSAAN}}                                        & \textbf{} & \multicolumn{2}{c}{\textbf{DANA}}                                       & \textbf{} & \multicolumn{2}{c}{\textbf{PFRCN}}   &                      & \multicolumn{2}{c}{\textbf{FRW}}                     & \textbf{} & \multicolumn{2}{c}{\textbf{WSAAN}}                     & \textbf{} & \multicolumn{2}{c}{\textbf{DANA}}                                       & \textbf{} & \multicolumn{2}{c}{\textbf{PFRCN}}  &  & \multicolumn{2}{c}{\textbf{WSAAN$^\dagger$}}                       \\ \midrule
        $\boldsymbol{K}$     & \underline{Base} & \underline{Novel}                  &           & \underline{Base}                     & \underline{Novel}                  &           & \underline{Base}                   & \underline{Novel}                  &           & \underline{Base} & \underline{Novel}                            &                      & \underline{Base} & \underline{Novel}                 &           & \underline{Base}                    & \underline{Novel}&           & \underline{Base}                   & \underline{Novel}                  &           & \underline{Base} & \underline{Novel}                           &  & \underline{Base}   & \underline{Novel}                                     \\
        1                    & 47.24                    & \rbf{13.35}                      &           & 45.55                    & 12.19                                          &           & \bbf{49.38}                              & 12.52                              & \textbf{} & 34.98                     & 7.51                                &                      & 56.67                    & 16.92    &           & 56.41                    & 15.48              &           & \bbf{58.78}                    & \rbf{20.64}      &           & 40.66                     & 6.07                                     &  & -                  & -                                                     \\
        3                    & 46.50                    & \rbf{25.32}                      &           & 44.18                    & 24.42                                          &           & \bbf{49.67}                              & 20.70                              &           & 34.58                     & 9.33                                &                      & 58.05                    & 25.08    &           & 51.72                    & 13.84              &           & \bbf{59.14}                    & \rbf{27.26}      &           & 40.48                     & 7.51                                     &  & -                  & -                                                     \\
        5                    & 48.60                    & 29.57                      &           & 47.56                    & \rbf{31.44}                                          &           & \bbf{52.49}                              & 24.96                              &           & 36.09                     & 11.33                               &                      & 60.75                    & 32.58    &           & 60.79                    & 30.38              &           & \bbf{62.12}                    & \rbf{34.16}      &           & 41.97                     & 8.55                                     &  & -                  & 0.25                                                  \\
        10                   & 49.04                    & 35.29                      &           & 46.72                    & 35.12                                          &           & \bbf{53.99}                              & \rbf{36.50}                              &           & 36.32                     & 11.55                               &                      & 61.30                    & 37.29    &           & \bbf{62.79}                    & 32.38              &           & 62.71                    & \rbf{38.18}      &           & 42.37                     & 9.16                                     &  & 0.54               & 0.32                                                  \\ \bottomrule[1pt]
        \end{tabular}%
    } \caption[FSOD performance of AAF Framework Implementation on aerial images]{Comparison of $\text{mAP}_{0.5}$ of several methods on DOTA and
    DIOR datasets. For each method, mAP is reported for different number of
    shots $K \in \{1, 3, 5, 10\}$ and separately for base and novel classes.
    Blue and red values represent the best performance on base and novel classes
    respectively, for each dataset. $^\dagger$ denotes results taken directly
    from the original papers.}
    \label{tab:aaf_result_aerial}
    \end{table}

These results confirm the analysis conducted in \cref{chap:aerial_diff}, the
performance gap between the classical baseline (i.e. FCOS) and the few-shot
approaches is larger on aerial images. On natural images, the performance drop
between the few-shot approach and the regular baseline is nearly inexistent for
base classes and around 25\% for novel classes. On aerial images these drops are
largely increased: $\sim 15\%$ and $\sim 50\%$ for base and novel classes
respectively. This can be guessed from
\cref{tab:aaf_result_voc,tab:aaf_result_aerial}, but detailed gaps are provided
in \cref{tab:aaf_rmap_values}. Following the analysis from
\cref{chap:aerial_diff}, the main reason behind this performance gap between
natural and aerial images is the size of the objects in the image. Small objects
are much more difficult to detect and also are poor representatives of a
semantic class as they contain little information. Therefore, the most sensible
direction to pursue is to design new attention mechanisms specifically built for
small objects. As it happens, this will be discussed in the next section.

\begin{table}[!h]
    \centering
    \resizebox{0.7\textwidth}{!}{%
    \begin{tabular}{@{\hspace{2mm}}llccccccccl@{\hspace{2mm}}}
    \toprule[1pt]
                                    &               & \multicolumn{2}{c}{\textbf{DOTA}} &  & \multicolumn{2}{c}{\textbf{DIOR}} & \textbf{} & \multicolumn{2}{c}{\textbf{Pascal VOC}} &  \\
                                    &               & \textbf{Base}   & \textbf{Novel}  &  & \textbf{Base}   & \textbf{Novel}  &           & \textbf{Base}      & \textbf{Novel}     &  \\ \cmidrule(r){1-10}
    \multirow{5}{*}{\rotatebox[origin=c]{90}{\textbf{$\text{mAP}_{\mathbf{0.5}}$}}}  & FCOS baseline & 60.87            & 69.69            &  & 72.82	&81.48 & &	65.47	&68.02   &  \\ \cmidrule(lr){2-10}
                                    & FRW           & 49.04	          &35.29	        &  &61.30	         &37.29            &           &63.16	            &48.71               &  \\
                                    & WSAAN         & 46.72	          &35.12	        &  &\bbf{62.79}	         &32.38            &           &\bbf{65.27}	            &51.70               &  \\
                                    & DANA          & \bbf{53.99}	          &\rbf{36.50}	        &  &62.71	         &\rbf{38.18}            &           &65.17	            &\rbf{52.26}               &  \\
    \multirow{5}{*}{\rotatebox[origin=c]{90}{\textbf{RmAP (\%)}}}                               &               &                 &                 &  &                 &                 &           &                    &                    &  \\
                                    & FRW           & -19.43          & -49.36          &  & -15.83          & -54.23          &           & -3.52              & -28.39             &  \\
                                    & WSAAN         & -23.24          & -49.60          &  & -13.78          & -60.26          &           & -0.30              & -24.00             &  \\
                                    & DANA          & -11.30          & -47.63          &  & -13.88          & -53.14          &           & -0.46              & -23.17             &  \\ 
                                    &               &                 &                 &  &                 &                 &           &                    &                    &  \\\bottomrule[1pt]
    \end{tabular}%
    }
    \caption{$\text{mAP}_{0.5}$ and RmAP values for some reimplemented methods
    and XQSA with $K=10$ shots.}
    \label{tab:aaf_rmap_values}
    \end{table}

\section{Cross-Scales Query-Support Alignment for Small FSOD}
\label{sec:xscale}
\vspace{-1em}

From the analysis in \cref{chap:aerial_diff} and the previous section, it is
clear that a new attention mechanism specifically designed for small objects is
required to get reasonable performance on aerial images. To this end, we propose
a novel alignment method that combines information from multiple scales:
Cross-Scales Query-Support Alignment (XQSA). This differs from existing methods
which often work independently at different scales. Conversely, XQSA combines
the information from various scales and sources (\ie query and support images).
Its original motivation is to unlock matching support examples with query
objects belonging to the same class even though their sizes differ. With
existing methods, this was prohibited as same-class objects with different sizes
have non-similar features and are not matched by similarity-based attention
mechanisms.

\subsection{XQSA definition}
\label{sec:xqsa_definition}
In this section, we detail the functioning of our proposed Cross-Scales
Query-Support Alignment module. First, features are extracted from the query and
support images with a backbone network $f$. In our implementation, $f$ is a
ResNet-50 with an FPN attached. It outputs feature maps at three distinct
resolution levels:
\begin{align}
    \{F_{q,0}, F_{q,1}, F_{q,2}\}  &= f(I_q), \\
    \{F^c_{s,0}, F^c_{s,1}, F^c_{s,2}\}  &= f(I_s^c).
\end{align}

All query features $F_{q,i} \in \mathbb{R}^{w_{q,i}\times h_{q,i} \times d}$,
for $i \in \{0,1,2\}$ (i.e. from different levels) are flattened and
concatenated into a unique representation $F_q \in \mathbb{R}^{n_q\times d}$, with
$n_q = \sum_i w_{q,i} h_{q,i}$. Here, $w_{q,i}$ and $h_{q,i}$ denote the size of the
query feature map at level $i$, this size depends on the query image size and
the stride of the corresponding level in the backbone. The same operation is
performed for all support features $F_{s,i}^c$. When more than one shot is
available per class, support features are average at each level $i$:
\begin{equation}
    F_{s,i}^c = \frac{1}{K} \sum\limits_{k=1}^K F_{s,i}^{c, k}.
\end{equation}

Then, following the ViT paradigm, the support and query features are linearly
projected into the \textit{queries}, \textit{keys} and \textit{values} matrices
$Q$, $K$ and $V$. Specifically, the query features are used to produce the
\textit{queries} while \textit{keys} and \textit{values} are computed from the
support features:

\vspace{8mm}
\begin{align}
    Q = \tikzmarknode{g}{\highlight{Brown}{$F_q$}} W_Q &= [\tikzmarknode{s}{\highlight{Violet}{$F_{q,0}, F_{q,1}, F_{q,2}$}}] \tikzmarknode{n}{\highlight{Green}{$W_Q$}}, \\
    K^c = \tikzmarknode{r}{\highlight{Brown}{$F_s^c$}} W_K &= [\tikzmarknode{l}{\highlight{Violet}{$F_{s,0}^{c}, F_{s,1}^{c}, F_{s,2}^{c}$}}] \tikzmarknode{n}{\highlight{Green}{$W_K$}}, \\
    V^c = \tikzmarknode{r}{\highlight{Brown}{$F_s^c$}} W_V &= [\tikzmarknode{l}{\highlight{Violet}{$F_{s,0}^{c}, F_{s,1}^{c}, F_{s,2}^{c}$}}] \tikzmarknode{t}{\highlight{Green}{$W_V$}},
\end{align}
\begin{tikzpicture}[overlay,remember picture,>=stealth,nodes={align=left,inner ysep=1pt},<-]
    \path (s.north) ++ (0,1em) node[anchor=south west,color=Violet!80] (lambda){\footnotesize Per level features};
    \draw [color=Violet!80](s.north) |- ([xshift=-0.3ex,color=Violet]lambda.north east);
    \path (g.north) ++ (0,1em) node[anchor=south east,color=Brown!80] (gfeat){\footnotesize Concatenated \\\footnotesize multiscale features};
    \draw [color=Brown!80](g.north) |- ([xshift=-0.3ex,color=Brown]gfeat.north west);
    \path (t.south) ++ (0,-1em) node[anchor=north east,color=Green!80] (mat){\footnotesize Learned projection matrices};
    \draw [color=Green!80](t.south) |- ([xshift=-0.3ex,color=Green]mat.south west);
\end{tikzpicture}

where $W_Q$, $W_K$ and $W_V$ are learnable projection matrices, which are
implemented as linear layers in practice. From this, an affinity matrix is
computed between the queries and the keys, and then the aligned support features $A_s^c$
are computed as: 
\begin{align}
    \lambda^c_s &= \text{Softmax}(\frac{Q{K^c}^T}{\sqrt{d}}), \\
    A_s^c &= \lambda^c_s V^c.
\end{align}

\begin{figure}
    \centering
    \includegraphics[width=\textwidth]{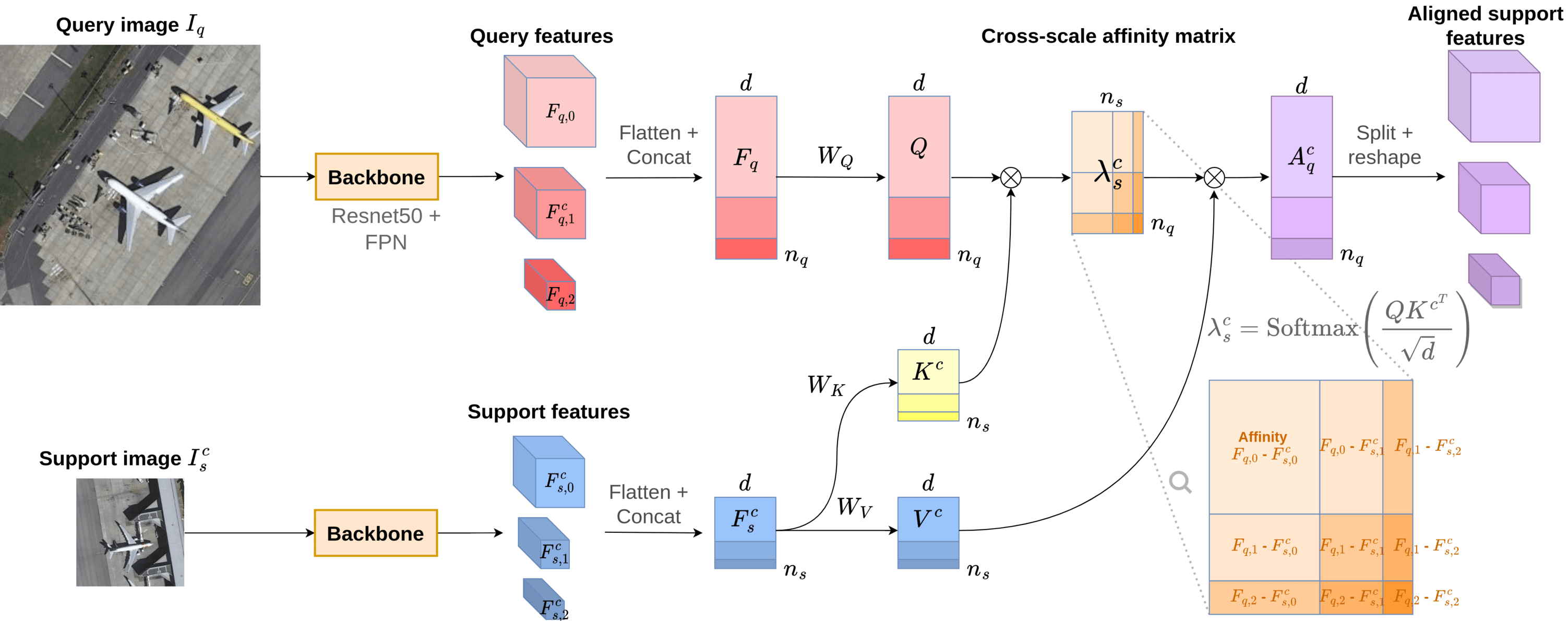}
    \caption[Cross-scales Query-Support Alignment diagram]{Diagram illustrating
    the proposed cross-scales query-support alignment method. Features are
    extracted from the query and support images at multiple scales and combined
    to form an affinity matrix. For each query feature position, the affinity is
    computed with any position in the support features. This allows object
    matching across different feature levels.}
    \label{fig:aaf_xscale}
\end{figure}

For completeness with the definition of the AAF framework, $\lambda_q = \textrm{I}$,
meaning that the query features are not modified. The aligned features are
finally processed by a two-layer MLP with skip connections. LayerNorm
\cite{ba2016layer} is applied before alignment and the MLP. These supplementary
computations can be seen as fusion operations in the AAF framework,
similar to what was proposed in \cite{li2020one, han2021meta,wu2021universal}.
This resembles the ViT attention, but with a major difference, it combines
features from different images and different levels (see \cref{fig:aaf_xscale}).
This allows better object matching when there are size discrepancies between
support and query images.

Small objects have a limited footprint in feature maps which make them hard to
detect but also hard to match with support examples. XQSA's multiscale
alignment enhances the chances of matching as each query feature is compared with
support features at all scales. 
Finally, in order to fairly compare XQSA with DANA, we leverage their BackGround
Attenuation block (BGA) on the support features before alignment. They conduct a
thorough ablation study which shows the positive impact of BGA on the few-shot
performance of their method. We also carry out an ablation study about our
cross-scales method (see \cref{tab:xqsa_ablation}) and find that BGA also improves
performance in this case. XQSA is implemented inside the AAF framework, split
into three modules: alignment, attention and fusion, following the description
from \cref{sec:framework}. BGA is implemented within the global attention block,
which can handle as well self-attention module as well and is applied before
alignment.

\subsection{Ablation study XQSA}
\label{sec:xqsa_ablation}
To confirm the benefits of each component of our attention methods, we conduct a
brief ablation experiment, adding separately the different modules of our
proposed attention mechanism. The ablation is conducted on DOTA and the results
are available in \cref{tab:xqsa_ablation}. From this table, it is clear that
each component plays a role in the improved performance of our method. Both the
fusion (with the MLP) and the skip connections around fusion and alignment are
beneficial for the performance on novel classes. It is worth noting that
Background Attenuation proposed by \cite{chen2021should} helps both for base and
novel classes, which confirms the experiments conducted by the authors of this
work.

\begin{table}[h]
    \centering
    \resizebox{0.7\textwidth}{!}{%
    \begin{tabular}{@{\hspace{2mm}}llcccc@{}}
    \toprule[1pt]
    Baseline      &  & \checkmark                & \checkmark                & \checkmark                & \checkmark                         \\
    Cross-scale Alignment                &  &                           & \checkmark                & \checkmark                & \checkmark                         \\
    Fusion Layer          &  &                           &                           & \checkmark                & \checkmark                         \\
    Query-Support Self-Attention     &  &                           &                           &                           & \checkmark                         \\ \midrule
    \textbf{Base classes}     &  & \multicolumn{1}{r}{49.20} & \multicolumn{1}{r}{49.46} & \multicolumn{1}{r}{49.13} & \multicolumn{1}{r}{\bbf{51.11}} \\
    \textbf{Novel classes}    &  & \multicolumn{1}{r}{36.52} & \multicolumn{1}{r}{38.84} & \multicolumn{1}{r}{40.31} & \multicolumn{1}{r}{\rbf{41.01}} \\ \bottomrule[1pt]
    \end{tabular}%
    } \caption[XQSA components ablation study]{Ablation study of the XQSA attention method on DOTA dataset.
    $\text{mAP}_{0.5}$ scores are reported for base and novel classes with
    $K=10$ shots.}
    \label{tab:xqsa_ablation}
    \end{table}

\subsection{Application to aerial and natural images}
To assess the capabilities of the proposed method, we compare it with the best
methods from \cref{sec:aaf_fair_comparison}: FRW and DANA on DOTA, DIOR, Pascal
VOC and MS COCO. The results of these experiments are available in
\cref{tab:xscale_res}. The mAP values are reported separately for small
($\sqrt{wh} < 32$), medium ($32 \leq \sqrt{wh} < 96$) and large ($\sqrt{wh} \geq
96$) objects. Hence, the methods can be compared specifically on small objects.
XQSA performs consistently better on small and medium novel objects, compared
with FRW and DANA. This performance gain comes at the expense of slightly lower
detection quality for larger objects. In XQSA, the shallow query features are
compared to all support features (\ie not only shallow support maps). As deeper
maps are smaller, this increases moderately the number of potential detections
for small objects. However, for large objects, the deep query feature map is
compared with all support maps, including the shallow ones. It increases a lot
the number of potential matches between query and support features (see an
illustration of this phenomenon in \cref{fig:xqsa_asymmetry}). For large objects
that are already well detected, this mostly adds wrong matches and deteriorates
the performance. For small objects, however, this is useful as very few correct
matches are found by current FSOD methods. A potential solution for this issue
would be to down-weight the contributions of shallower features in the affinity
matrix's bottom rows (\ie the left and bottom blocks of the matrix). The
affinity matrix could even be made upper triangular to avoid taking into account the
contributions of shallower levels at all. For similar reasons, XQSA demonstrates
a slight drop in base classes performance. However, the actual goal of few-shot
learning is to maximize performance on novel classes. The large number of
available examples during base training is enough to learn robust query-support
matching even for small objects. However, our goal here is to improve the
generalization capabilities of the model on novel objects. The performance on
base classes is simply a safety check and relates more to the Generalized FSOD
problem (see \cref{sec:fsod_extensions}).

\begin{table}[]
    \centering
    
    \resizebox{\textwidth}{!}{%
    \begin{tabular}{@{}ccccccccccccccccccccc@{}}
    \toprule[1pt]
                                   &      & \multicolumn{4}{c}{\textbf{DOTA}} & \textbf{} & \multicolumn{4}{c}{\textbf{DIOR}} &  & \multicolumn{4}{c}{\textbf{Pascal VOC}} &  & \multicolumn{4}{c}{\textbf{MS COCO}} \\ \cmidrule(lr){3-6} \cmidrule(lr){8-11} \cmidrule(lr){13-16} \cmidrule(l){18-21} 
                                   &      & \textbf{All}    & \textbf{S}      & \textbf{M}      & \textbf{L}      &           & \textbf{All}    & \textbf{S}      & \textbf{M}      & \textbf{L}      &  & \textbf{All}    & \textbf{S}      & \textbf{M}      & \textbf{L}       &  & \textbf{All}    & \textbf{S}      & \textbf{M}      & \textbf{L}      \\ \midrule
    \multirow{3}{*}{\thead{\large Base\\ \large Classes}}  & \textbf{FRW}  & 49.04  & 25.48  & 59.17  & 63.37  &           & 62.20  & 8.21   & 48.66  & 80.67  &  & 63.21    & 15.67    & 47.94   & \bbf{81.73}   &  & 29.03   & 13.08   & 35.87   & 48.00  \\
                                   & \textbf{DANA} & \bbf{53.99}  & \bbf{36.98}  & \bbf{62.33}  & \bbf{70.39}  &           & \bbf{62.71}  & \bbf{10.92}  & \bbf{49.34}  & \bbf{83.17}  &  & \bbf{65.17}    & \bbf{18.14}    & \bbf{50.58}   & 80.11   &  & \bbf{38.14}   & \bbf{23.30}   & \bbf{51.85}   & \bbf{56.38}  \\
                                   & \textbf{XQSA} & 51.11  & 26.10  & 59.41  & 64.30  &           & 59.88  & 10.64  & 45.69  & 82.34  &  & 62.13    & 15.60    & 48.64   & 75.94   &  & 31.56   & 16.13   & 40.13   & 49.83\\ \midrule
     \multirow{3}{*}{\thead{\large Novel\\ \large Classes}} & \textbf{FRW}  & 35.29  & 13.99  & 34.11  & 59.31  &           & 37.29  & 2.48   & 33.74  & 59.38  &  & 48.72    & 16.44    & 26.71   & \rbf{68.27}   &  & 24.09   & 11.53   & 22.45   & \rbf{38.69}  \\
                                   & \textbf{DANA} & 36.58  & 14.32  & 40.28  & \rbf{64.65}  &           & 38.18  & 3.21   & 34.91  & \rbf{60.99}  &  & 52.26    & 10.05    & 24.67   & 67.23   &  & 24.75   & 12.01   & \rbf{29.40}   & 37.95  \\
                                   & \textbf{XQSA} & \rbf{41.00}  & \rbf{17.84}  & \rbf{44.57}  & 54.46  &           & \rbf{41.51}  & \rbf{4.12}   & \rbf{40.69}  & 58.21  &  & \rbf{53.94}    & \rbf{19.46}    & \rbf{34.86}   & 66.14   &  & \rbf{25.03}   & \rbf{12.57}   & 26.05   & 38.55   \\ \bottomrule[1pt]
    \end{tabular}%
    } \caption[Performance comparison between XQSA, FRW, and DANA on DOTA, DIOR,
    Pascal VOC and MS COCO]{Performance comparison between XQSA, FRW, and DANA.
    $\text{mAP}_{0.5}$ values are reported separately for base (top) and novel
    (bottom) classes on DOTA, DIOR, Pascal VOC, and MS COCO with $K=10$ shots.
    mAP values are reported for All, Small ($\sqrt{wh} < 32$), Medium ($32 \leq
    \sqrt{wh} < 96$) and Large ($\sqrt{wh} \geq 96$) objects.}
    \label{tab:xscale_res}
    \end{table}

\begin{figure}
    \centering
    \includegraphics[width=\textwidth]{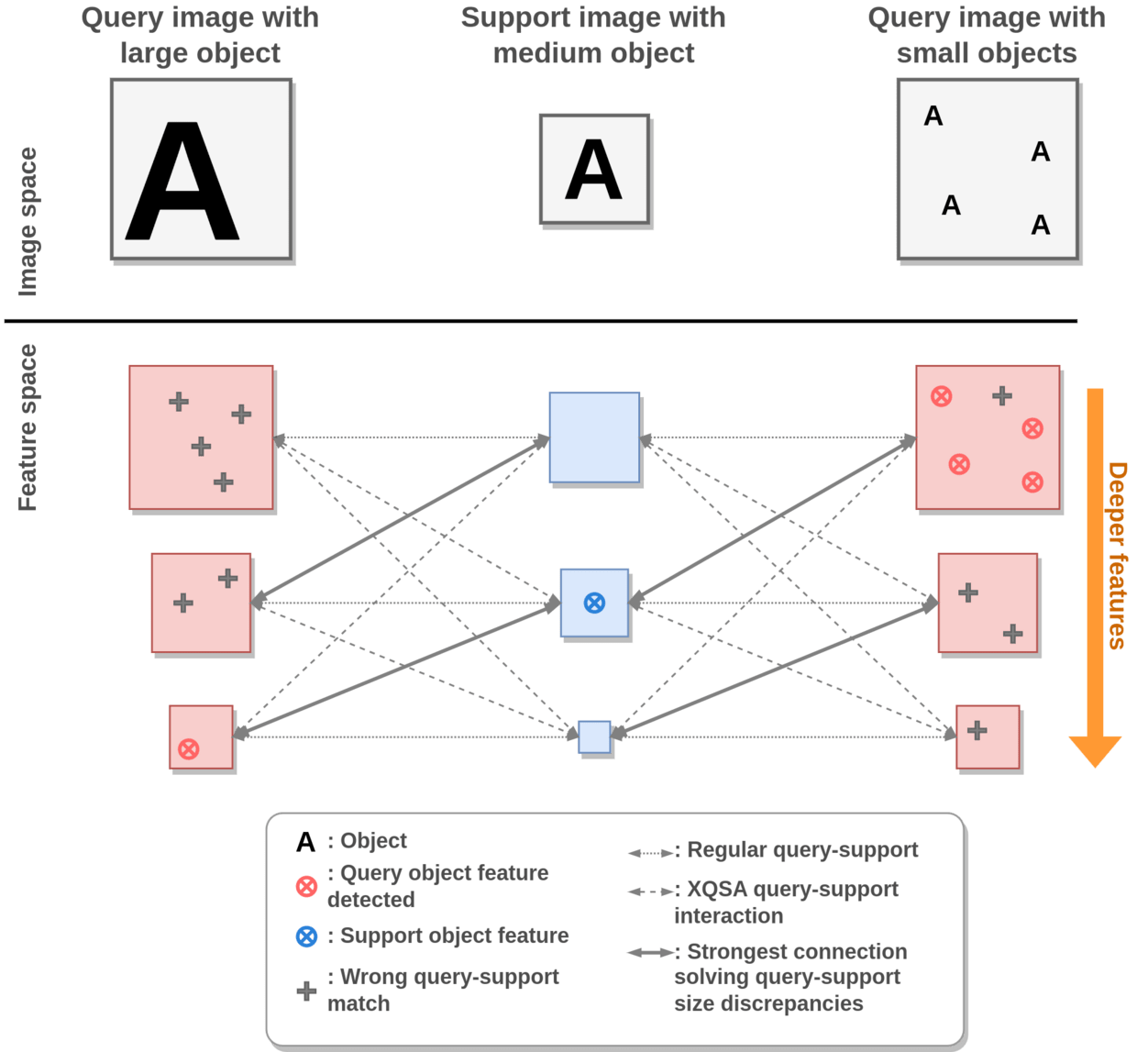}
    \caption{XQSA small and large objects matching asymmetry.}
    \label{fig:xqsa_asymmetry}
\end{figure}

\begin{table}[]
    \centering
    
    \resizebox{\textwidth}{!}{%
    \begin{tabular}{@{}ccccccccccccccccccccc@{}}
    \toprule[1pt]
                                                            &      & \multicolumn{4}{c}{\textbf{DOTA}} & \textbf{} & \multicolumn{4}{c}{\textbf{DIOR}} &  & \multicolumn{4}{c}{\textbf{Pascal VOC}} &  & \multicolumn{4}{c}{\textbf{MS COCO}} \\ \cmidrule(lr){3-6} \cmidrule(lr){8-11} \cmidrule(lr){13-16} \cmidrule(l){18-21} 
                                                            &      & \textbf{All}    & \textbf{S}      & \textbf{M}      & \textbf{L}      &           & \textbf{All}    & \textbf{S}      & \textbf{M}      & \textbf{L}      &  & \textbf{All}    & \textbf{S}      & \textbf{M}      & \textbf{L}       &  & \textbf{All}    & \textbf{S}      & \textbf{M}      & \textbf{L}      \\ \midrule
    \multirow{3}{*}{\thead{\large Base\\ \large Classes}}   & \textbf{FRW}  &  23.18	& 8.60	 &27.84	  &32.22   &           & 35.60	& 2.60	 & 23.04  & 50.82  &  & 37.93	 & 6.54	    & 22.84	  & 50.49   &  & 15.60	 & 5.47	   & 18.84	 &27.83   \\
                                                            & \textbf{DANA} &  \bbf{26.63}	& \bbf{11.43}  &\bbf{30.73}	  &\bbf{37.62}   &           & \bbf{36.39}	& 3.48	 & \bbf{24.93}  & \bbf{52.31}  &  & \bbf{39.12}	 & \bbf{7.28	}    & \bbf{25.37}	  & \bbf{51.39}   &  & \bbf{22.46}	 & \bbf{10.22}   & \bbf{29.72}	 &\bbf{36.51}   \\
                                                            & \textbf{XQSA} &  25.30	& 8.85	 &28.78	  &34.64   &           & 34.80	& \bbf{3.54	} & 22.90  & 51.47  &  & 27.45	 & 3.18	    & 16.60	  & 36.76   &  & 11.37	 & 4.44	   & 14.18	 &31.97   \\ \midrule
     \multirow{3}{*}{\thead{\large Novel\\ \large Classes}} & \textbf{FRW}  &  15.99	& 4.25	 &14.09	  &29.65   &           & 20.00	& 0.48	 & 17.00  & 33.30  &  & 29.09	 & 5.64	    & 12.21	  & 40.03   &  & 12.41	 & 4.84	   & 10.90	 &20.82    \\
                                                            & \textbf{DANA} &  17.17	& 5.60	 &20.44	  &\rbf{32.40}   &           & 20.35	& 0.78	 & 17.49  & 34.01  &  & \rbf{31.75}	 & 5.23	    & 11.09	  & \rbf{43.38}   &  & \rbf{13.44}	 & \rbf{5.30	}   & \rbf{15.03}	 &\rbf{21.47}    \\
                                                            & \textbf{XQSA} &  \rbf{21.04}	& \rbf{7.91	} &\rbf{25.18}	  &26.49   &           & \rbf{22.78}	& \rbf{0.97	} & \rbf{20.97}  & \rbf{34.78}  &  & 25.07	 & \rbf{6.40	}    & \rbf{12.74}	  & 35.15   &  & 10.33	 & 4.87	   & 10.04	 &16.72    \\ \bottomrule[1pt]
    \end{tabular}%
    } \caption[Performance comparison between XQSA, FRW, and DANA on DOTA, DIOR,
    Pascal VOC and MS COCO]{Performance comparison between XQSA, FRW, and DANA.
    $\text{mAP}_{0.5:0.95}$ values are reported separately for base (top) and
    novel (bottom) classes on DOTA, DIOR, Pascal VOC, and MS COCO with $K=10$
    shots. mAP values are reported for All, Small ($\sqrt{wh} < 32$), Medium
    ($32 \leq \sqrt{wh} < 96$) and Large ($\sqrt{wh} \geq 96$) objects.}
    \label{tab:xscale_res_095}
    \end{table}

Looking at the performance on all objects, disregarding their size, the proposed
alignment technique significantly improves the detection quality for aerial
images. Using XQSA in the AAF framework increased novel class mAP by 5 and 4
points on DOTA and DIOR, respectively. As it works better on small objects but
worse on large objects, it is less appropriate for natural images. As a
consequence, it shows lower improvements for Pascal VOC and MS COCO. Overall,
XQSA largely improves on existing works for aerial images. On DIOR, this
corresponds to a 10 mAP point increase compared to previous state-of-the-art
\cite{deng2020few}. However, this is not sufficient to fill the performance gap
with natural images as presented in \cref{fig:xqsa_rmap_all_hist}. This figure
extends \cref{fig:fsod_baseline_res} with XQSA results. While XQSA improves upon
other methods on aerial images, it is still far behind the performance of the
regular baseline. XQSA is better for small and medium objects but at the cost of lower
performance on large objects and base classes. Progress is still required to get
more versatile FSOD solutions able to handle small, medium, and large objects at
the same time. 

\begin{figure}[]
    \centering
    \includegraphics[width=\columnwidth]{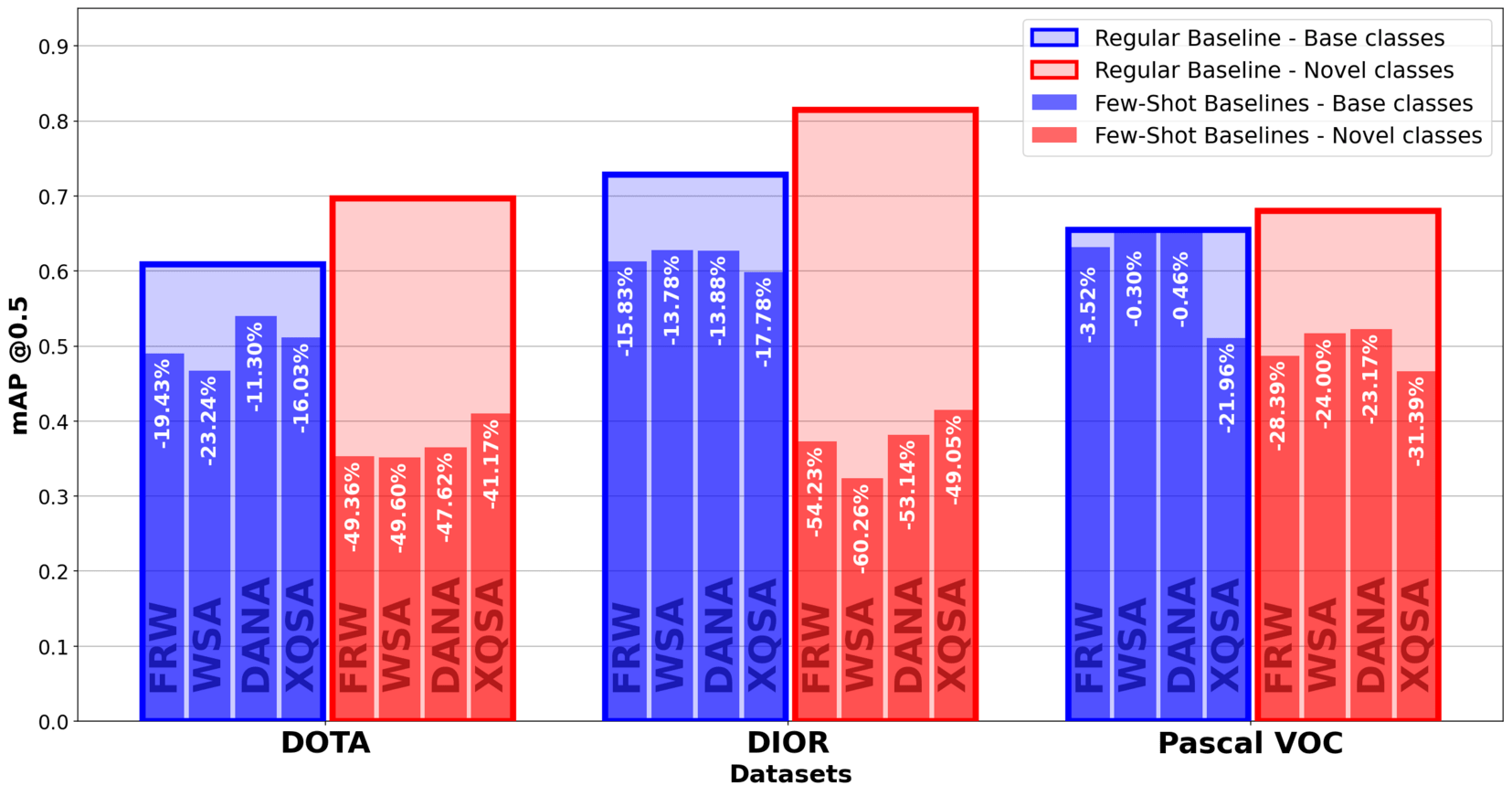}
    \caption[FSOD performance comparison with regular baseline on DOTA, DIOR and
    Pascal VOC.]{$\text{mAP}_{0.5}$ and corresponding RmAP values of the four
    best performing methods from all our experiments. All methods are trained
    within our proposed AAF framework with data augmentation which explains
    slightly higher performance for FRW and WSA. 10 shots are available for each
    class at inference time.}
    \label{fig:xqsa_rmap_all_hist}
\end{figure}

In complement to \cref{tab:xscale_res}, we also provide the comparison between
XQSA, DANA and FRW with $\text{mAP}_{0.5:0.95}$ metric and on the four datasets.
These results are provided in \cref{tab:xscale_res_095}. $\text{mAP}_{0.5:0.95}$
is a more demanding metric for object detection. It is especially hard for small
objects as a few pixels shift from ground truth can greatly reduce the IoU and
therefore lead to a missed detection. This intensifies as the IoU threshold
increases in the mAP computation.

For DOTA and DIOR, results with
$\text{mAP}_{0.5:0.95}$ are in agreement with results from \cref{tab:xscale_res}
(i.e. with $\text{mAP}_{0.5}$). However, XQSA does not perform better than DANA
on Pascal VOC and MS COCO novel classes with $\text{mAP}_{0.5:0.95}$. This is
mainly due to the metric being too strict on small objects. This questions the
soundness of these metrics for FSOD, especially when dealing with small objects.
We will tackle this question in \cref{chap:siou_metric}.

\subsection{XQSA Implementation Challenges and Extensions}
XQSA is inspired by the ViT and resembles some FSOD techniques that leverage
transformer attention as well. It is well-known that these mechanisms are
computationally heavy and scale quadratically in terms of the number of
locations in the feature maps. Here, it is slightly different as we combine the
features from the support and query images. Support images are $4 \times$
smaller than query images (see \cref{sec:aaf_aug_crop}), therefore the
complexity of the attention module is greatly reduced. According to the
notations introduced in \cref{sec:xqsa_definition}, the complexity of the XQSA
alignment block is $O(n_q n_s)$. As $n_s$ is way smaller than $n_q$ (roughly 16
times), this is way better than a self-attention mechanism which would be
$O(n_q^2)$. However, with XQSA, as features from all scales are concatenated,
this remains computationally heavy, especially as this process must be
repeated for each support image. Unfortunately, this does not scale well as we
increase the number of support examples. Time complexity rapidly becomes
prohibitive, but memory complexity is more limiting. The gradients of the
Softmax used for the computation of the similarity matrices are extremely large
($O(n_qn_s^2)$) and do not fit on GPU memory when the number of support examples
increases. To bypass these limitations, we propose several tricks.

\begin{itemize}[nolistsep]
    \item[-] \textbf{Pytorch manual gradient computation}: automatic
    differentiation in Pytorch is not always optimal. When successive
    computations involve the same gradients, Pytorch often computes and stores
    them separately, wasting precious resources. To this end, we re-implemented
    the XQSA block as a self-contained operation, with custom gradients
    computation to prevent duplicated gradients. This results in slight memory
    and performance gains, but it is still not enough to scale efficiently with
    the number of support examples. 
    \item[-] \textbf{Pytorch gradient checkpointing}: Pytorch has an API to
    checkpoint the gradients during the backward pass. It copies back the
    gradients on the CPU memory to prevent overflowing the GPU memory. It
    solves the out-of-memory issues, but makes the training much slower.
    \item[-] \textbf{Deformable XQSA block}: The alignment block is expensive
    due to the comparison between all query feature locations and all support
    locations. Thus, we extend the XQSA block with a deformable attention
    mechanism, inspired from ConViT \cite{d2021convit}. Specifically, we
    introduce an inductive bias inside the attention module by adding a
    layer that selects the locations that will be compared between support and
    query feature maps. This resembles Deformable Convolutions
    \cite{dai2017deformable} and Deformable-DETR \cite{zhu2021deformable}. While
    this was solving both the memory overflow and training slowness, we were not
    able to achieve reasonable detection performance with it.
    \item[-] \textbf{Support class averaging}: while this solution seems
    sub-optimal, it saves a lot of time and memory by avoiding a lot
    of computation. Of course, it does not completely solve the scaling issues
    (\eg as the number of support classes increases). However, it allows
    adding a large number of support images per class without any issues and
    performs well on the FSOD task. 
\end{itemize}

Finally, we only keep the support class averaging as it is the simplest and
best-performing alternative tested. 

\section{Qualitative Comparison within the AAF framework}
\vspace{-1em}
In this section, we provide a qualitative comparison of the four attention-based
methods that we compared previously in \cref{fig:xqsa_rmap_all_hist}: FRW,
WSAAN, DANA and ours XQSA. To get a fair comparison, we sampled the same support
examples for each method and performed the detection on the same query images. For
convenience, we split the comparison in two, first on base classes and then on
novel classes. In both cases, we provide the support examples used for the
detection in the first figure and the detection on a handful of images in the
second figure. These comparisons are visible in
\cref{fig:xqsa_base_support,fig:xqsa_base_detection} for base classes and in
\cref{fig:xqsa_novel_support,fig:xqsa_novel_detection} for novel classes. 

\subsection{Base Classes Detection Quality}

From \cref{fig:xqsa_base_detection}, it is quite difficult to assess which method
is superior to the other. It seems quite obvious that DANA and XQSA perform
slightly better than FRW and WSAAN. However, there are cases where neither DANA
nor XQSA is the best (see the second row for instance). This is in line with the
quantitative results from the previous sections. The gap between these methods
on base classes is tight, it is therefore quite difficult to correctly assess
the quality of these methods from only a handful of examples. It is worth noting
that these techniques work quite well on the base classes. Of course, there are
some false positives, but the number of false negatives is limited, which is
quite important for intelligence applications.  

\begin{figure}
    \centering
    \includegraphics[width=\textwidth, trim=100 40 100 40, clip]{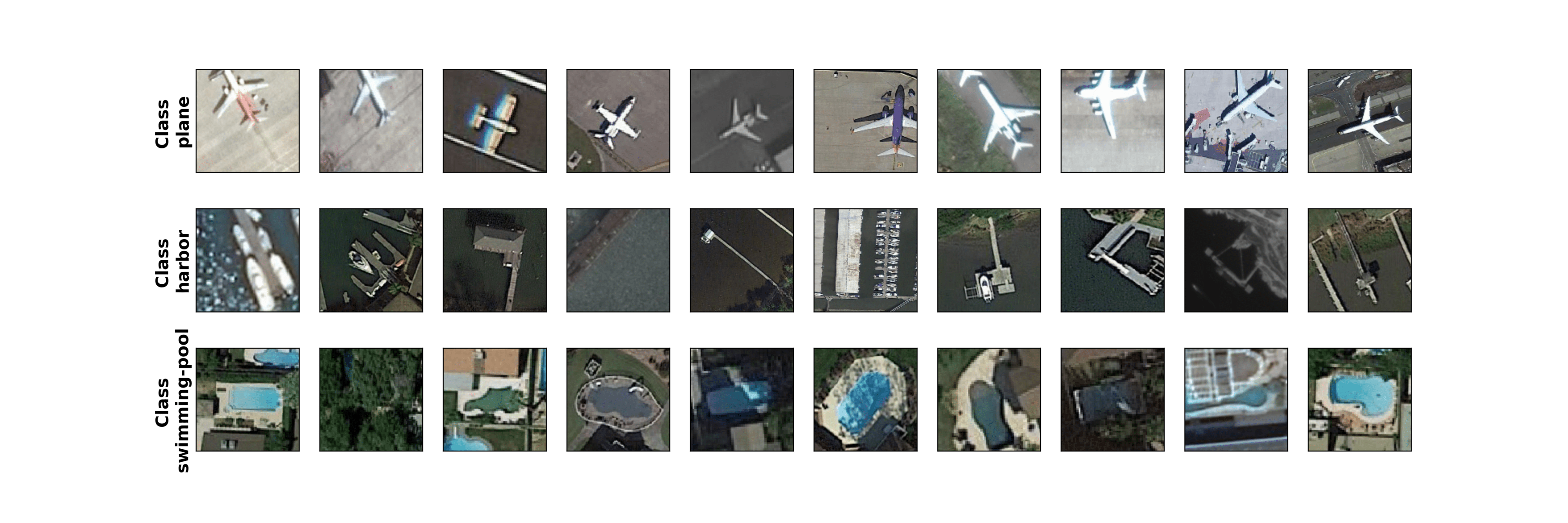}
    \caption[Support examples for base classes for qualitative assessment of
    attention-based methods]{Support examples for base classes ($K=10$). These
    are the examples used during the detection inference that produced results
    from \cref{fig:xqsa_novel_detection}.}
    \label{fig:xqsa_base_support}
\end{figure}

\begin{figure}
    \centering
    \includegraphics[width=\textwidth, trim=120 150 110 150, clip]{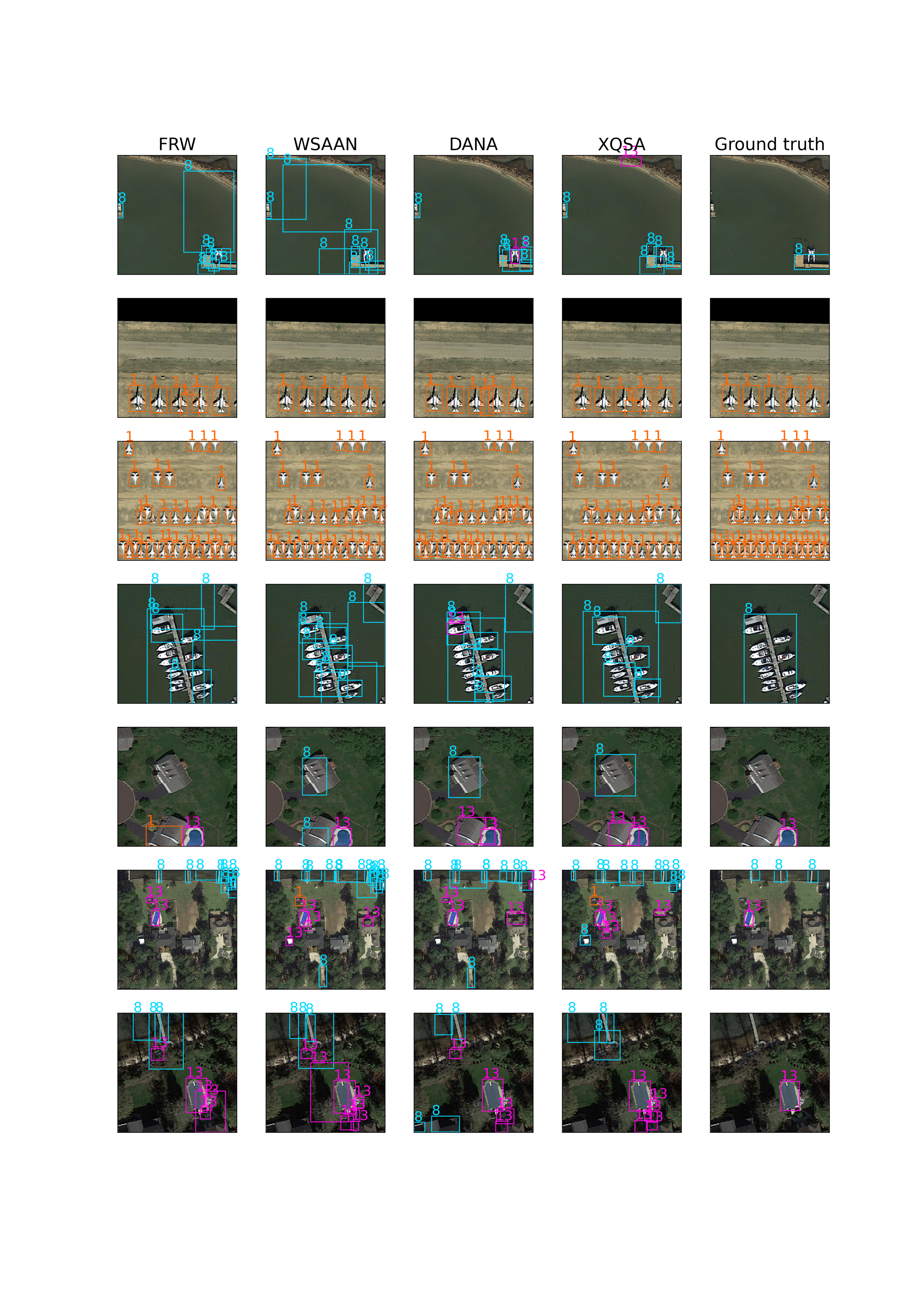}
    \caption{Qualitative assessment of the detection quality of FRW, WSAAN,
    DANA and XQSA on DOTA with $K=10$ shots on base classes.}
    \label{fig:xqsa_base_detection}
\end{figure}

\subsection{Novel Classes Detection Quality}
With novel classes, the performance differences are more visible than with base
classes. First, it can be seen that XQSA produces less false positives which
indicates a higher accuracy. Then, XQSA also provides more small detection than
other methods. It sometimes gives false positives, but overall it improves the
detection of small targets (see the last row for instance). Interestingly, XQSA
and DANA are less sensitive to partial objects in the images (see first row).
This may be explained by the spatial information kept in the query-support
combination compared to FRW and WSAAN. As the whole object is available in the
support examples, DANA and XQSA match all parts of the support features with the
query features and detect only entire objects. It seems that this kind of
matching is much more robust than the global attention alone (FRW and WSAAN) as
they showcase many more false positives overall. Of course, the detection of
small objects is still very challenging in the FSOD setting which explains the
relatively poor detection quality in these images. It is not easy to objectively
determine which method is the best at this, but a slight advantage seems in
favor of XQSA, confirming quantitative results.

\section{Conclusion}
\vspace{-1em}
In this chapter, we have introduced a highly modular framework for implementing
attention-based FSOD methods. First, this framework allows fair comparison
between the various attention mechanisms proposed in the literature. From our
analysis, it seems that spatial alignment is crucial to achieving high-quality
FSOD, mostly because it does not lose the spatial information contained in the
support examples. Secondly, the AAF framework is a practical tool to design new
attention mechanisms. For that matter, we developed a novel cross-scales
alignment layer within the framework to specifically increase the detection
performance on the small objects. The so-called XQSA alignment allows us to
achieve large improvements compared to the contemporary literature on several
datasets. It works especially well on aerial images as they contain smaller and
more objects than natural images. Specifically, XQSA outperforms the
state-of-the-art on DOTA and DIOR datasets at that time. 

Nevertheless, the attention-based methods are not completely satisfactory from
an industrial perspective. While they achieve reasonable performance on aerial
images, they have some disadvantages. First, even if they can adapt to novel
classes from a few support examples at test time, they still require extensive
fine-tuning to perform correctly. This fine-tuning can take up to several hours,
which is not convenient for "on-the-fly" adaptation. Then, the episodic training
strategy is somewhat cumbersome and generates unrealistic scenarios. Indeed,
during each episode, query images are sampled so that they contain at least one
instance of one of the episode classes. In real-case applications, no object of
interest can be visible in an image which makes the detection task more
challenging. But most importantly, only the episode classes are detected during
the episode, the detection task is, therefore, simpler as it only classifies
objects among a smaller number of classes. Given these drawbacks, we investigate
in the next chapter FSOD methods that do not employ the episodic training
strategy and therefore solely rely on a fine-tuning scheme.   

\begin{figure}
    \centering
    \includegraphics[width=\textwidth, trim=100 40 100 40, clip]{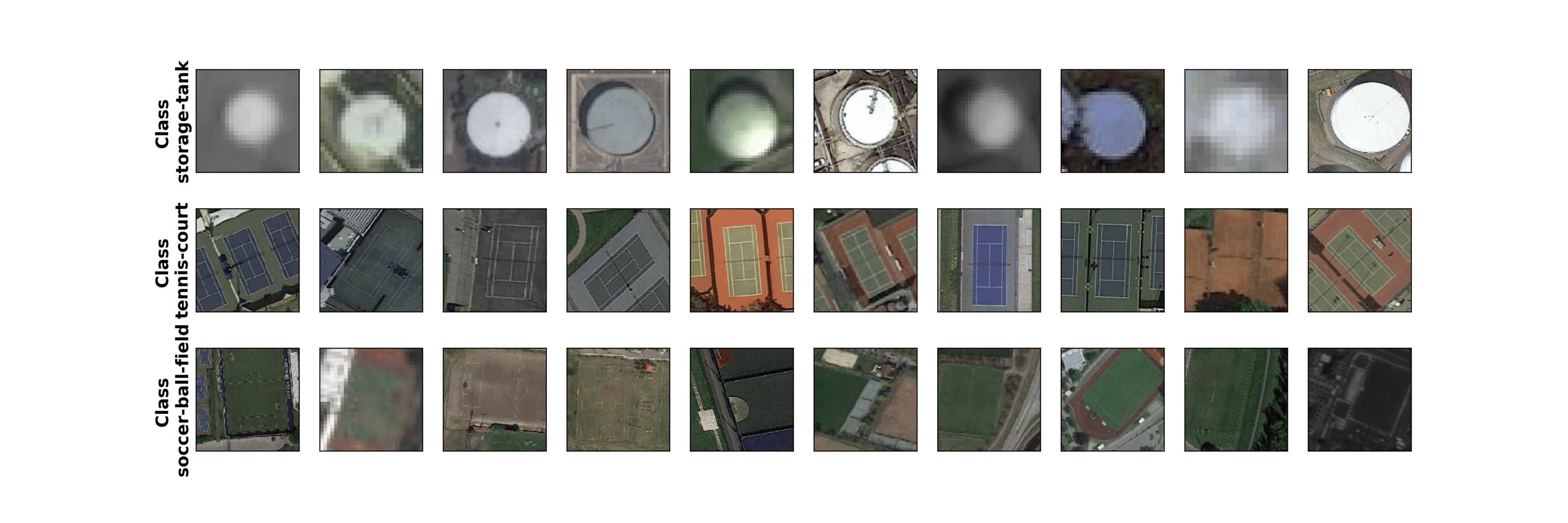}
    \caption[Support examples for novel classes for qualitative assessment of attention-based methods]{Support examples for novel classes ($K=10$). These are the examples used 
    during the detection inference that produced results from \cref{fig:xqsa_novel_detection}.}
    \label{fig:xqsa_novel_support}
\end{figure}

\begin{figure}
    \centering
    \includegraphics[width=\textwidth, trim=120 150 110 150, clip]{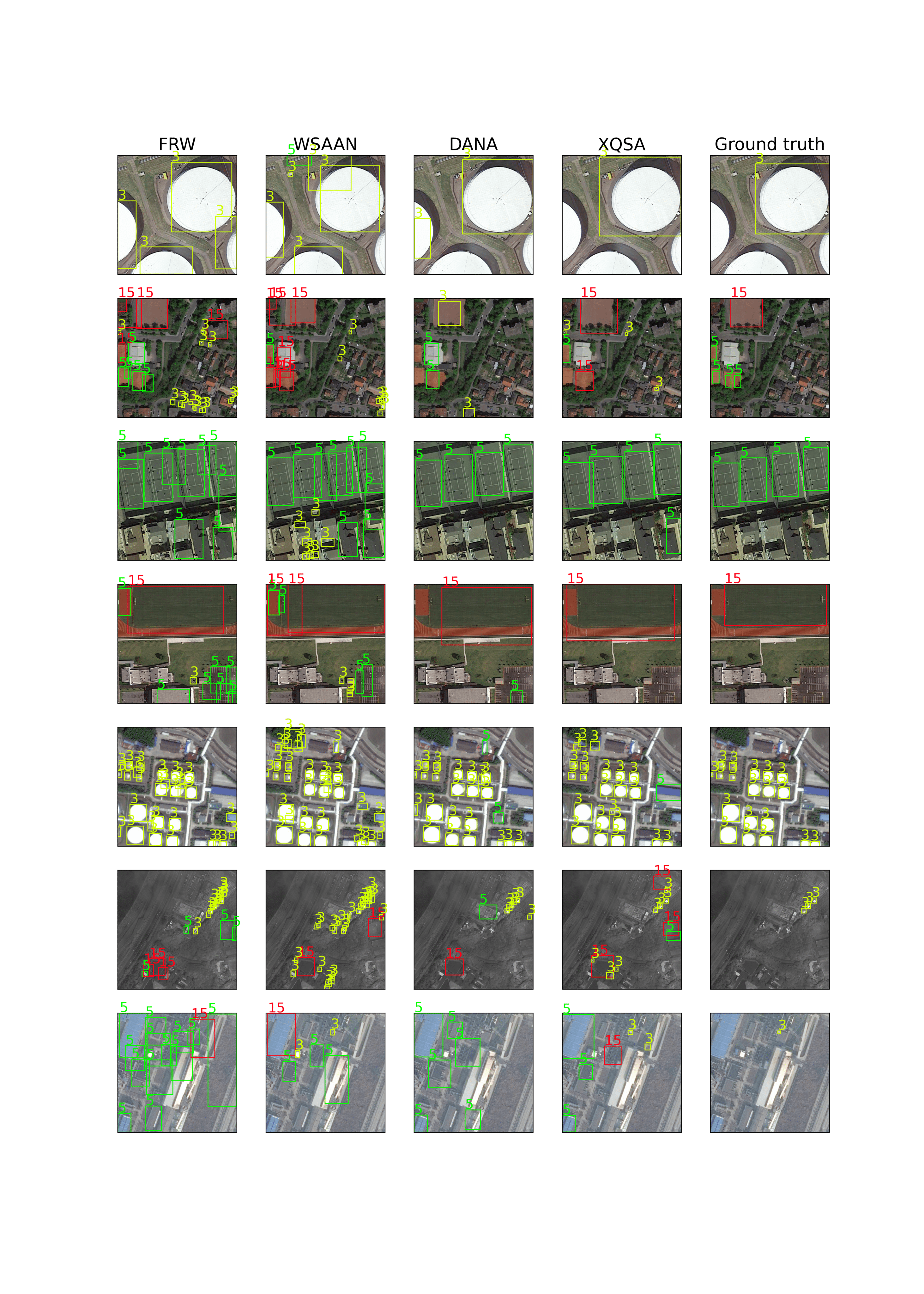}
    \caption{Qualitative assessment of the detection quality of FRW, WSAAN,
     DANA and XQSA on DOTA with $K=10$ shots on novel classes.}
    \label{fig:xqsa_novel_detection}
\end{figure}

\makeatletter
\let\savedchap\@makechapterhead
\def\@makechapterhead{\vspace*{-2cm}\savedchap}
\chapter{Few-Shot Diffusion Detector via Fine-Tuning}%
\let\@makechapterhead\savedchap
\makeatletter
\label{chap:diffusion}

\chapabstract{Previous chapters explore few-shot object detection with metric
learning and attention-based techniques. This chapter focuses on the
last major approach for FSOD: fine-tuning. Based on DiffusionDet, a recent
detection framework leveraging diffusion models, we build a simple but efficient
fine-tuning strategy. The resulting method, called FSDiffusionDet, achieves
state-of-the-art FSOD on aerial datasets and competitive performance on natural
images. Extensive experimental studies explore the design choices of the
fine-tuning strategy to better understand the key components required to achieve
such quality. Finally, these impressive results allow considering more complex
settings such as cross-domain scenarios, which are especially relevant for COSE.} 
\vspace{1em}

\PartialToC

In \cref{chap:prcnn,chap:aaf} we have proposed respectively metric-learning and
attention-based approaches to tackle the FSOD problem. Both of these directions
were sensible choices given the state of the FSOD literature at the beginning of
this project. Since then, however, fine-tuning approaches have gained a lot of
interest with competitive performance and reduced complexity. Following this
trend, we explore in this chapter a simple fine-tuning strategy for FSOD. Based
on the recent DiffusionDet \cite{chen2022diffusiondet} model, we propose an
effective fine-tuning scheme for FSOD which outperforms all previous methods on
DOTA and DIOR datasets while being competitive with state-of-the-art on natural
images. We begin with a brief presentation of the Diffusion Probabilistic Models
(DPM) and their recent progress in various generative tasks. Then we present in
detail DiffusionDet, which tackles OD with a refreshing perspective, as a
box-denoising problem. Following this, we present our fine-tuning strategy
called Few-Shot DiffusionDet (FSDiffusionDet) and the results of multiple
experiments conducted to improve our strategy. Given the impressive performance
of FSDiffusionDet in the few-shot regime, we broaden the scope of our analysis
and study the more challenging Cross-Domain FSOD task. We emphasize that the
first two sections of this chapter present existing works in the literature,
while the last three sections discuss our contributions: the FSDiffusionDet
strategy, thorough experimental analysis of the strategy on several datasets,
and its application in Cross-Domain scenarios.

\section{Diffusion Probabilistic Models Principle}
\label{sec:ddpm_principle}
\vspace{-1em}
Diffusion Probabilistic Models have been introduced in 2015 by
Sohl-Dickstein et al. \cite{sohl2015deep}. Their principle is simple, to
approximate a complex and intractable probability distribution, they model a
diffusion process from the original distribution to a normal distribution as
gradual Gaussian noise addition. Then, the goal is to find the reverse process
to approximate the original distribution by iterative denoising. In this
section, we present the main concepts of the DPM \cite{sohl2015deep} and their
recent advances in generative tasks, mostly led by Denoising Diffusion
Probabilistic Models (DDPM) \cite{ho2020denoising}. 

\subsection{Forward and Reverse Diffusion Processes}
First, let's introduce a few notations. Our objective is to be able to
efficiently sample elements $x$ from a distribution $\mathcal{P}$. When
$\mathcal{P}$ is an arbitrary distribution, its Probability Density Function
(PDF) is intractable and sampling often relies on expensive Monte Carlo
techniques. Here, we suppose that there exists a random process $q$ that
transforms $\mathcal{P}$ into a normal distribution: 
\begin{equation}
    q(x) \sim \mathcal{N}(\mathbf{0}, \mathbf{I}).
\end{equation}

This is called the Diffusion Process and refers to the eponym physical
phenomenon. The main hypothesis is to assume that this process is a Markov Chain
that adds Gaussian noise progressively:
\begin{equation}
    x_T = q(x_{1:T} | x_0), 
\end{equation}

where $x_0 \sim \mathcal{P}$ is an element sampled from the original
distribution and $x_T \sim \mathcal{N}(0, \mathbf{I})$ is sampled from
a normal distribution. $T$ denotes here the number of steps in the Markov chain
and $q(x_{1:T})$ represents the joint distribution of variables $x_1$ to $x_T$.
The diffusion process adds Gaussian noise iteratively and therefore is defined as: 
\begin{equation}
    q(x_{1:T} | x_0) := \prod_{t=1}^T q(x_t|x_{t-1}) \quad \text{with } q(x_t|x_{t-1}) := \mathcal{N}(\sqrt{1 - \beta_t }x_t, \beta_t \mathbf{I}).
\end{equation}

The $\beta_t$ denote the variance schedule, \ie the amount of gaussian noise
added at each step. For convenience, we also define $\alpha_t$ and $\bar{\alpha}_t$: 
\begin{align}
    \alpha_t &= 1 - \beta_t, \\
    \bar{\alpha}_t &= \prod_{s=1}^t \alpha_s.
\end{align}

$q$ is called the \textit{forward} diffusion process as it
transforms $x_0$ into noise (this is true only asymptotically when $T
\to \infty$). Thus, we can write for $1 \leq t \leq T$:
\begin{align}
    x_t &= \sqrt{\alpha_t}x_{t-1} + \sqrt{1 - \alpha_t} \epsilon_{t-1}, \\
        &= \sqrt{\alpha_t\alpha_{t-1}}x_{t-2} + \sqrt{1 - \alpha_t\alpha_{t-1}} \epsilon_{t-2}, \\
        &= \dots \nonumber\\ 
        &= \sqrt{\bar{\alpha_t}}x_0 + \sqrt{1 - \bar{\alpha_t}} \epsilon_0, \label{eq:diff_xt}
\end{align}

where the $\epsilon_{i}$ ($0\leq i \leq t-1$) are sampled from a normal distribution. 

The \textit{reverse process} instead transforms Gaussian noise into
elements sampled from $\mathcal{P}$. It is also a Markov chain
with Gaussian transitions (this is ensured for sufficiently small $\beta_t$): 
\begin{equation}
    q(x_{0:T}) := \prod_{t=1}^T q(x_{t-1}|x_t) \;\; \text{with} \;\; q(x_{t-1}|x_t) := \mathcal{N}(\mu_t, \mathbf{\Sigma_t}).
\end{equation}

Unfortunately, the reverse process, or rather $\mu_t$ and $\mathbf{\Sigma_t}$,
are highly intractable and cannot be easily estimated. However, we can
approximate this process with a parametrized model $p_{\theta}$:
\begin{equation}
    p_{\theta}(x_{0:T}) := p(x_T) \prod_{t=1}^T p_{\theta} (x_t|x_{t-1}) \;\; \text{with} \;\; p_{\theta}(x_t|x_{t-1}) := \mathcal{N}(\mu_{\theta}(x_t, t), \mathbf{\Sigma_{\theta}}(x_t, t)).
\end{equation}

Hence, if we find an optimal set of parameters $\theta$ so that the model is
able to approximate the reverse process expectation $\mu_t$ and variance
$\mathbf{\Sigma_t}$ from variable $x_t$ and timestep $t$, then the reverse
process can be computed. Now, we must derive efficient ways to estimate the
reverse process. 

One solution is to leverage deep neural networks which are well
suited for these kinds of tasks. We will explain how such models can be trained
to approximate the reverse process conditional probabilities $q(x_{t-1}|x_t)$.
Note that the reverse process is initialized with a normal distribution: $p(x_T)
= \mathcal{N}(\mathbf{0}, \mathbf{I})$. The forward and reverse diffusion
processes are illustrated in \cref{fig:diff_processes}, in the context of image
denoising. Before jumping to the next section to see how we can estimate such
models, we can observe that the reverse process distribution is tractable when
conditioned on $x_0$, it will be useful for training:
\begin{align}
    \label{eq:reverse_conditioned}
    q(x_{t-1} | x_t, x_0) &:= \mathcal{N}(\tilde{\mu}_t(x_t, x_0), \tilde{\beta}_t\mathbf{I}), \\[2mm]
    \text{with} \; &\tilde{\mu}_t(x_t, x_0) = \frac{\sqrt{\bar{\alpha}_{t-1}} \beta_t}{1- \bar{\alpha}_{t}} x_0 + \frac{\sqrt{\alpha_t}(1- \bar{\alpha}_{t-1})}{1- \bar{\alpha}_{t}} x_t,\\[2mm]
    \text{and} \; &\tilde{\beta}_t =  \frac{1- \hat{\alpha}_{t-1}}{1- \hat{\alpha}_{t}} \beta_t.
\end{align}

\begin{figure}[h]
    \centering
    \includegraphics[width=\textwidth]{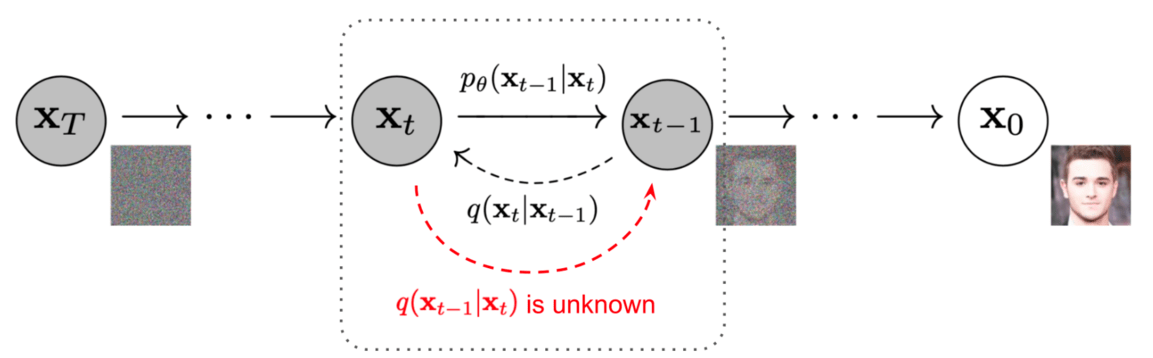}
    \vspace{2mm}
    \caption[Diffusion processes illustration]{Diffusion Processes illustration in the context of image
    denoising. Image taken from \cite{ho2020denoising,weng2021diffusion}.}
    \label{fig:diff_processes}
\end{figure}

\subsection{DDPM Training and Sampling}
Now that we have introduced the forward and reverse diffusion processes, we can
dive into the core of DDPM and see how we can train models to approximate the
reverse process. The overall goal is to maximize the log-likelihood of the
modeled data distribution $p_{\theta}(x_0)$: 
\begin{equation}
    \theta^* = \argmax_{\theta} \log(p_{\theta}(x_0)).
\end{equation}

Unfortunately, the log-likelihood is not easily optimizable and several tricks
are required to get a more tractable objective. First, \cite{sohl2015deep} makes
use of the well-known Evidence Lower BOund \cite{kingma2013auto} (ELBO), which
is lower bound to log-likelihood objective and more easily computable. In
practice, we minimize the negative log-likelihood and leverage the evidence upper
bound:
\begin{align}
    - \log(p_{\theta}(x_{1:T})) &= -\log \left( \int p_{\theta}(x_{0:T}) dx_{1:T} \right), \\
                                &= -\log \left( \int p_{\theta}(x_{0:T}) \frac{q(x_{1:T}|x_0)}{q(x_{1:T}|x_0)} dx_{1:T} \right),\\
                                &= -\log \mathbb{E}_q \left[\frac{p_{\theta}(x_{0:T})}{q(x_{1:T}|x_0)}\right], \\
                                &\leq \mathbb{E}_q \left[-\log \frac{p_{\theta}(x_{0:T})}{q(x_{1:T}|x_0)}\right] = \mathcal{L}_{\text{ELBO}}.
\end{align}

Now, using the definition of the process $p_{\theta}(x_{0:T})$ and
$q(x_{1:T}|x_0)$, and Bayes' rule, we can split the objective for each timestep
$t$ and make it tractable: 
\begin{equation}
    \begin{aligned}
        \mathcal{L}_{\text{ELBO}} = \mathbb{E}_q \left[\tikzmarknode{g}{\highlight{darkgray}{$-\log \frac{p(x_T)}{q(x_T|x_0)}$}} + \sum\limits_{t>1} \tikzmarknode{l}{\highlight{MidnightBlue}{$-\log \frac{p_{\theta}(x_{t-1}|x_t)}{q(x_{t}|x_{t-1}, x_0)}$}} \tikzmarknode{m}{\highlight{Bittersweet}{$- \log p_{\theta}(x_{0}|x_1)$}}\right]. \label{eq:diff_objective}
    \end{aligned}
\end{equation}
\begin{tikzpicture}[overlay,remember picture,>=stealth,nodes={align=left,inner ysep=1pt},<-]
    \path (g.south) ++ (0,-1em) node[anchor=north east,color=darkgray!80] (gfeat){\footnotesize Objective for timestep $T$, $\mathcal{L}_T$ };
    \draw [color=darkgray!80](g.south) |- ([xshift=-0.3ex,color=darkgray]gfeat.south west);
    \path (l.south) ++ (0,-2.5em) node[anchor=north east,color=MidnightBlue!80] (lfeat){\footnotesize Objective for timestep $t-1$, $\mathcal{L}_{t-1}$};
    \draw [color=MidnightBlue!80](l.south) |- ([xshift=-0.3ex,color=MidnightBlue]lfeat.south west);
    \path (m.south) ++ (0,-1em) node[anchor=north east,color=Bittersweet!80] (mfeat){\footnotesize $\mathcal{L}_0$ };
    \draw [color=Bittersweet!80](m.south) |- ([xshift=-0.3ex,color=Bittersweet]mfeat.south west);
\end{tikzpicture}
\vspace{1.5em}

The previous equation is only valid because $q(x_{0:T})$ if a Markov chain and
because the Markov Property states that $q(x_t|x_{t-1}, x_0) = q(x_t|x_{t-1})$.
This is crucial, otherwise, the Bayes' rule introduces intractable terms
($q(x_t)$ and $q(x_{t-1})$). We refer the reader to Appendix A from
\cite{ho2020denoising} and Appendix B from \cite{sohl2015deep} for the detailed
derivation of \cref{eq:diff_objective}. Finally, the terms highlighted in gray
and blue in the above \cref{eq:diff_objective} can be interpreted as
Kullback-Leiber Divergence terms:
\begin{equation}
    \begin{aligned}
        \mathcal{L}_{\text{ELBO}} = \tikzmarknode{g}{\highlight{darkgray}{$D_{\text{KL}}\left(q(x_T|x_0)||p(x_T)\right)$}} &+ \sum\limits_{t>1} \tikzmarknode{l}{\highlight{MidnightBlue}{$D_{\text{KL}}\left(q(x_{t}|x_{t-1}, x_0)||p_{\theta}(x_{t-1}|x_t)\right)$}} \label{eq:diff_objective_2}\\
                     &+\tikzmarknode{m}{\highlight{Bittersweet}{$\mathbb{E}_q \left[ - \log p_{\theta}(x_{0}|x_1)\right]$}}.
    \end{aligned}
\end{equation}
\begin{tikzpicture}[overlay,remember picture,>=stealth,nodes={align=left,inner ysep=1pt},<-]
    \path (g.south) ++ (0,-1em) node[anchor=north east,color=darkgray!80] (gfeat){\footnotesize $\mathcal{L}_T$ };
    \draw [color=darkgray!80](g.south) |- ([xshift=-0.3ex,color=darkgray]gfeat.south west);
    \path (l.south) ++ (1em,-1em) node[anchor=north west,color=MidnightBlue!80] (lfeat){\footnotesize $\mathcal{L}_{t-1}$};
    \draw [color=MidnightBlue!80]([xshift=1em]l.south) |- ([xshift=-0.3ex,color=MidnightBlue]lfeat.south east);
    \path (m.south) ++ (0,-1em) node[anchor=north west,color=Bittersweet!80] (mfeat){\footnotesize $\mathcal{L}_0$ };
    \draw [color=Bittersweet!80](m.south) |- ([xshift=-0.3ex,color=Bittersweet]mfeat.south east);
\end{tikzpicture}

These KL divergence terms are easy to compute as they compare only Gaussian
distributions. This gives an easy-to-optimize upper bound to train the diffusion
models. Diffusion models are meant to approximate the reverse diffusion process.
One way to achieve this with neural network models is to use the
reparametrization trick introduced in \cite{kingma2013auto}. The idea is to
train a neural network to output the mean and variance parameters of a Gaussian
distribution and sample elements from the estimated distribution. That way, the
gradients can be computed through the stochastic sampling operation. Here, we
specifically learn two models able to predict the mean and variance parameters
$\mu_{\theta}(x_t, t)$ and $\mathbf{\Sigma_{\theta}}(x_t, t)$ conditioned on
$x_t$ and timestep $t$. These two estimators are trained following the ELBO
objective which compares the estimated and reverse process distributions at each
timestep. In practice, this is achieved by randomly sampling a timestep and
optimizing the model with the corresponding loss function.  

To stabilize the training, authors from \cite{ho2020denoising} propose to fix
$\mathbf{\Sigma_{\theta}}(x_t, t) = \beta_t^2\mathbf{I}$ and introduce a few
simplifications in the loss. First, as the loss only involve KL divergence
between Gaussians, it can be written analytically:
\begin{align}
    \mathcal{L}_t = \mathbb{E}_{t, x_0, \epsilon} \left[\frac{1}{2 \|\Sigma_{\theta}\|_2^2} \| \tilde{\mu}_t(x_t, x_0) - \mu_{\theta}(x_t, t)\|^2 \right].
\end{align}

Another simplification follows from observing that $\tilde{\mu}_t(x_t,
x_0)$ and $\mu_{\theta}(x_t, t)$ can be re-written as:
\begin{align}
    \tilde{\mu}_t(x_t,x_0) &= \frac{1}{\sqrt{\alpha_t}}\left(x_t - \frac{1 - \alpha_t}{\sqrt{1 - \bar{\alpha_t}}} \epsilon_t\right), \\
    \mu_{\theta}(x_t, t) &= \frac{1}{\sqrt{\alpha_t}}\left(x_t - \frac{1 - \alpha_t}{\sqrt{1 - \bar{\alpha_t}}} \epsilon_{\theta}(x_t, t)\right).
\end{align}

This follows from \cref{eq:diff_xt,eq:reverse_conditioned}. Given this
expression of $\mu_{\theta}(x_t, t)$, it is only necessary for the model to
estimate $\epsilon_{\theta}(x_t, t)$, the amount of gaussian noise added to
$x_{t-1}$ to produce $x_t$. This explains why diffusion models are especially
well suited for denoising applications. Therefore, the loss can be further
simplified as (using \cref{eq:diff_xt}): 
\begin{align}
    \mathcal{L}_t^{\text{simple}} &= \mathbb{E}_{t, x_0, \epsilon} \left[ \| \epsilon_t - \epsilon_{\theta}(x_t, t)\|^2 \right],\\
                                  &= \mathbb{E}_{t, x_0, \epsilon} \left[ \| \epsilon_t - \epsilon_{\theta}(\sqrt{\bar{\alpha_t}}x_0 + \sqrt{1 - \bar{\alpha_t}} \epsilon_t , t)\|^2 \right].
\end{align}

Note that the scaling term has been omitted from the last equation as the
authors from \cite{ho2020denoising} obtained better results without it. 
The training procedure can be summarized in \cref{alg:diff_training}.  
Finally, once the model is trained, the sampling can be done from random noise
and repeatedly applying the reverse process. Using the model to estimate the
noise added at each time step, we have: 

\begin{equation}
    x_{t-1} = \frac{1}{\sqrt{\alpha_t}} \left(x_t - \frac{1-\alpha_t}{\sqrt{1-\bar{\alpha_t}}} \epsilon_{\theta}(x_t, t)\right).
\end{equation}

The sampling procedure is then defined in \cref{alg:diff_sampling}. Please note
that the formalism employed in this section is identical to the one used in
\cite{ho2020denoising}. We could have referred the reader directly to this
article, but it seemed essential to recall the basic principles of diffusion
models. For completeness, we also cite the excellent
blogpost \cite{weng2021diffusion} from which we drew some inspiration for the two previous
sections. 

\begin{algorithm}[]
    \caption{Diffusion Training procedure}\label{alg:diff_training}
    \begin{algorithmic}
    \While{not converged}    
    \State $x_0 \sim q(x_0)$
    \State $t \sim \text{Uniform}(\{1,..., T\})$
    \State $\epsilon \sim \mathcal{N}(0, \mathbf{I})$
    Take a gradient descent step on $\nabla_{\theta} \| \epsilon - \epsilon_{\theta}(\sqrt{\bar{\alpha_t}}x_0 + \sqrt{1 - \bar{\alpha_t}} \epsilon , t)\|^2$
    \EndWhile
 
    \end{algorithmic}
    \label{alg:diff_training}
\end{algorithm}

\begin{algorithm}
    \caption{Diffusion Sampling Procedure}\label{alg:diff_sampling}
    \begin{algorithmic}
    \State $x_T \sim \mathcal{N}(0, \mathbf{I})$
    \For{$t=T,...,1$}    
    \State $z \sim \mathcal{N}(0, \mathbf{I})$ 
    \State $x_{t-1} = \frac{1}{\sqrt{\alpha_t}} \left(x_t - \frac{1-\alpha_t}{\sqrt{1-\bar{\alpha_t}}} \epsilon_{\theta}(x_t, t)\right) + \beta_t^2 z$
    \EndFor
    \State \textbf{return} $x_0$
    \end{algorithmic}
    \label{alg:diff_sampling}
\end{algorithm}

\subsection{Recent Advances with Diffusion Models}
In the previous sections, we have presented the diffusion models and how they
can be trained to learn complex data distributions. It was originally leveraged
for image generation in \cite{ho2020denoising}, which samples 2D random noise
and progressively generates a sensible image. Plenty of consecutive works have
done the same with various improvements. In DDPM, the authors employ a denoising
U-Net to estimate the noise at each time step. This U-Net is replaced with
visual transformers in recent diffusion models \cite{rombach2022high}. With the
above formulation, the sampling is expensive as it requires iterative
application of the denoising model to the whole image. Instead, Latent Diffusion
Models (LDM) \cite{rombach2022high} proposes to apply the diffusion process to
the latent space to greatly reduce sampling time. The authors leverage an
encoder-decoder scheme to map the image space onto the latent space and back. In
addition, their formulation is well suited for latent manipulation and
conditioning the generation process with additional information such as text,
other images, layout, etc. Other approaches speed-up DMs with improved sampling
such as strided sampling schedule \cite{nichol2021improved}, ODE-based sampling
\cite{lu2022dpm,lu2022dpm_plus}, and careful variance scheduling
\cite{kong2021fast}. Alternatively, some contributions reconsider the denoising
diffusion process and leverage other corruption processes such as blurring
\cite{hoogeboom2022blurring} or masking \cite{daras2022soft}. Another approach
is to leverage non-Markovian diffusion process with for
instance Denoising Diffusion Implicit Models (DDIM) \cite{song2020denoising}.
Similarly to LDM, \cite{ho2022cascaded} derives a cascaded framework to scale up
the generated image size. With these iterative improvements, DMs largely
outperform the state-of-the-art image generation in terms of quality. Up to now,
this field was mostly dominated by GANs (\eg
\cite{karras2019style,Karras2019stylegan2}). GANs run faster, but the gap is
getting smaller and DMs overcome the GANs main issues: lack of
diversity, training instabilities and mode collapse. 

These techniques recently got a lot of attention out of the computer vision
field with their association with Large-Language-Models (LLMs) (\eg CLIP
\cite{radford2021learning}, GPT \cite{brown2020language} or T5
\cite{raffel2020exploring}). These models are referred to as Visual Language
Models (VLM), and combine the rich semantic latent space of LLMs with image
representation to perform text-to-image generation. They are embodied by Dall-E
\cite{ramesh2021zero}, Flamingo \cite{alayrac2022flamingo} and Imagen
\cite{saharia2022photorealistic}, among others. These models are able to
generate almost indistinguishable images (at least for the human eye) in an
extremely controllable way. It is great for plenty of applications, including for
creative purposes. However, it also has a large societal impact as such models can
easily be misused (\eg for deepfake generation) and are subject to questionable
biases. As an example for the previous claims, \cref{fig:diffusion_examples}
provides a few examples of real and fake images generated with various VLMs,
guessing correctly which pictures are fake is quite challenging. 

\begin{figure}
    \centering
    \includegraphics[width=\textwidth]{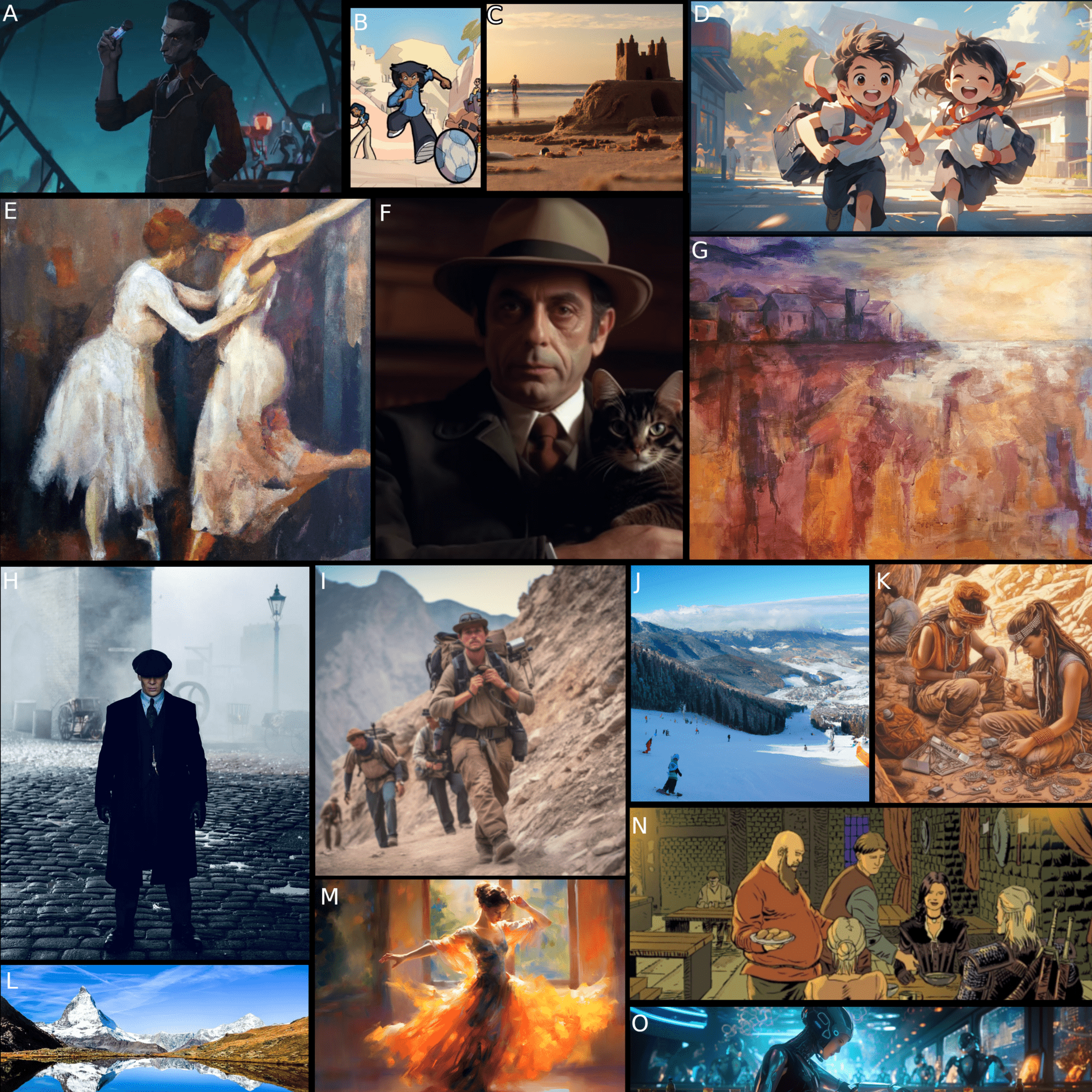}
    \caption[Examples of images generated with Diffusion Models]{Examples of
    real and fake images generated with diffusion approaches. We encourage the
    reader to guess which images are real and which are fake. We provide the list
    of answers in a footnote\footnotemark on the next page to prevent any confirmation bias.}
    \label{fig:diffusion_examples}
\end{figure}

Of course, image generation is not the only task that can be handled with
diffusion models. Generation of all kinds of modalities can be performed with
DMs: music \cite{mittal2021symbolic, liu2022diffsinger}, voice
\cite{huang2022prodiff}, text \cite{lovelace2022latent}, time series
\cite{li2022generative} or graphs \cite{vignac2022digress}. In addition, DMs can
also be leveraged for non-generative tasks, such as image translation
\cite{sasaki2021unit}, inverse image problems \cite{rombach2022high}
(application to inpainting), 3D modelling \cite{karnewar2023holodiffusion}. Last
but not least, they can also be used for predictive tasks such as segmentation
\cite{rahman2023ambiguous}, and, of course, detection with DiffusionDet
\cite{chen2022diffusiondet}. Thus, the next section will be dedicated to
explaining the principle behind DiffusionDet.

\footnotetext{Real images are images:
A, B, G, H, L and N, others have been generated with Midjourney or Dall-E.}

\section{DiffusionDet for Object Detection}
\vspace{-1em}
DiffusionDet \cite{chen2022diffusiondet} is a recently proposed model for object
detection. It tackles the OD task using a generative approach instead of seeing
it as a regression task. The latter predicts the box coordinates from the input
image while the former generates the box coordinates conditioned on the image.
The difference is subtle, but not seeing the detection as a regression problem
unlocks new designs. The main idea of DiffusionDet is to apply the diffusion
principle to the box generation. Random boxes are first sampled, and a model is
trained to refine iteratively the size and position of the boxes so that they
localize the objects in the input image, this is illustrated in
\cref{fig:diffusiondet_principle}. Specifically, the boxes are iteratively
denoised by the model. The diffusion process considered here is the same as in
Denoising Diffusion Implicit Models (DDIM) \cite{song2020denoising}, which as
mentioned in the previous section, proposes a non-Markovian forward process
that leads to the same objective as DDPM. The non-Markovian property of the
novel diffusion allows for much faster denoising. DDIM sampling is then
leveraged in DiffusionDet to iteratively denoise the boxes. 

\begin{figure}
    \centering
    \includegraphics[width=\textwidth]{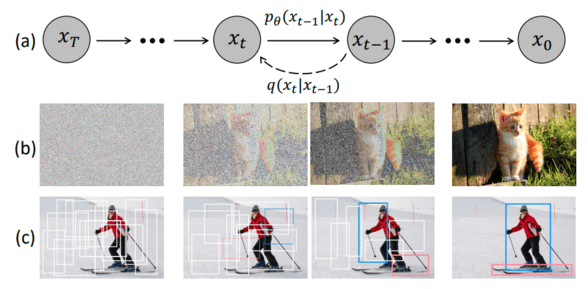}
    \caption[DiffusionDet principle]{Illustration of DiffusionDet principle,
     figure taken from \cite{chen2022diffusiondet}. (a) diffusion process, (b)
     diffusion process for image generation with DDPM and (c) DiffusionDet
     principle, random boxes are denoised to locate objects in the image.}
     \label{fig:diffusiondet_principle}
\end{figure}

Specifically, the denoising part of DiffusionDet is a lightweight hybrid
network, it consists of a self-attention layer (transformer-like) followed by a
dynamic layer (called an Instance Interaction layer). The diffusion/detection head
is finally split into two branches, one for classification and one for
regression. Both branches are implemented as small MLPs. The input to the head
is computed from the input images features extracted with a backbone network.
The backbone is a ResNet-50 with a three level FPN attached on top. Before
being fed to the detection head, object features are pooled from the entire
feature map with RoI Align module. The detection head processes object features
independently, but the Instance Interaction layer enables interactions between
instances. The detection head is applied iteratively to refine the bounding
boxes. The initial bounding boxes are sampled randomly from a normal
distribution. The regression branch of the head is trained to predict the noise
between the true boxes and the current boxes. After each iteration, the boxes
are updated following the DDIM sampling strategy. Only a small number of
iterations is required to get satisfactory boxes (the original paper provides
experiments with between 1 and 8 iterations). A renewal process also replaces
boxes with small confidence scores with random ones after each iteration to
prevent duplicated or erroneous boxes. The dynamic layer injects features from
the previous iteration into the computation of the adjusted boxes. The current
time step is encoded into a time embedding using a lightweight MLP. These
embeddings are then used to compute scale and shift vectors to transform the
object features and condition the model. We provide a detailed architecture
diagram in \cref{fig:diffusiondet_architecture}.

\begin{figure}
    \centering
    \includegraphics[width=\textwidth]{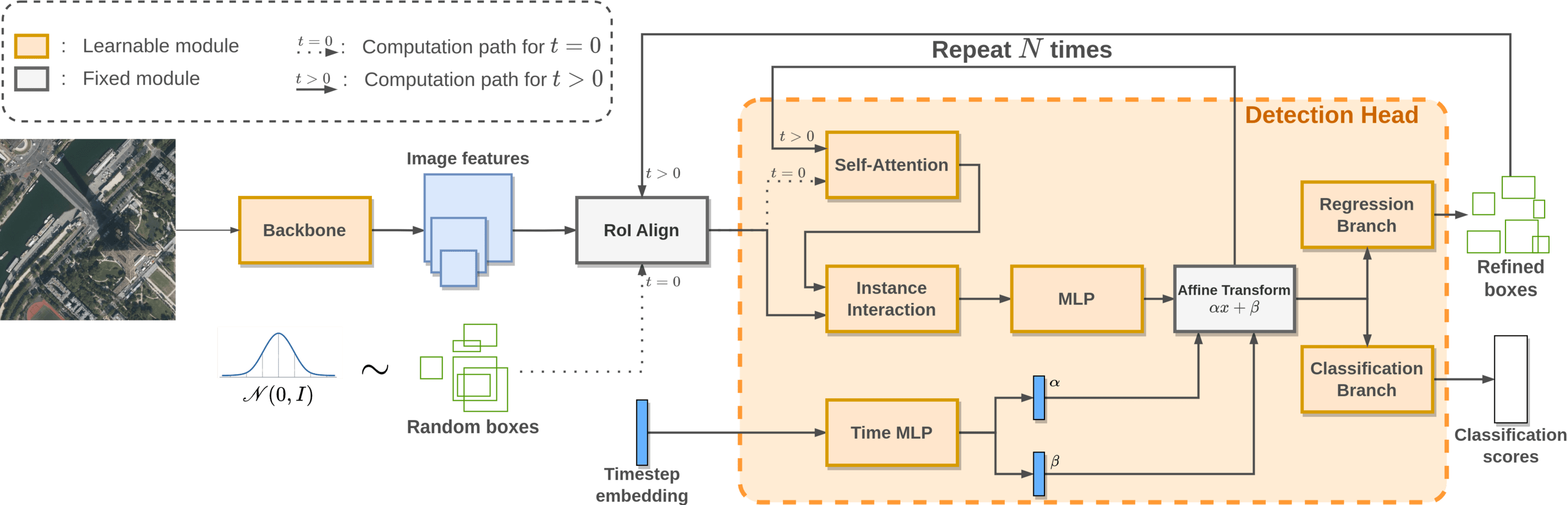}
    \caption[DiffusionDet architecture]{DiffusionDet architecture and detailed detection head design.}
    \label{fig:diffusiondet_architecture}
\end{figure}

The training is done in a similar fashion as in DDIM, except for the loss
function which is designed for object detection. First, a timestep is sampled
randomly, then the right amount of noise is added to the ground truth boxes and
the model is optimized using a classical loss function for detection (a
combination of Generalized IoU, L1 loss on box coordinates and cross-entropy for
the labels). As the number of predicted boxes is fixed, a set-to-set matcher is
employed to build target-prediction pairs with similar functioning as in DETR
\cite{carion2020end}. The loss is then computed on the selected pairs as done in
any detection framework.

DiffusionDet has a hybrid structure, it only has one stage, yet it does not
predict boxes and labels densely as common one-stage detectors. The box
denoising formulation allows for replacing the first stage (\ie the RPN) with a
much more naive approach: random box sampling. The iterative refinement of the
boxes is able to compensate for the poor initial positioning of the boxes. In a
sense, DiffusionDet resembles two-stage detection frameworks that leverage
iterative bounding box regression such as CRAFT \cite{yang2016craft}, Object
Detection via Multi-region model \cite{gidaris2015object}, or Cascade R-CNN
\cite{cai2018cascade}. The main difference between these models and DiffusionDet
is the direct box prediction. Instead of outputting refined box location,
DiffusionDet is trained to predict the shift between the current boxes and the
corresponding ground truth. This does not seem significant, but it is much more
adapted to the iterative regression procedure. First iterative methods propose
using the same detection head repeatedly to get better and better boxes.
However, the head is trained directly to output correct boxes, no matter how off
they are in the first place. DiffusionDet instead conditions the model with the
timestep embedding so that it knows how much noise should be removed from the
boxes. With this trick, it can reuse the same head without any issues.
Conversely, Cascade R-CNN makes use of decoupled heads for each iteration to
account for different refinement magnitudes; however, it significantly increases
the model size.

\section{Few-Shot DiffusionDet}
\vspace{-1em}
Now that we have reviewed the basic principles of diffusion models and presented
DiffusionDet we can see how it can be leveraged for FSOD. In this chapter, we
propose an adaptation of DiffusionDet in the few-shot setting, based on
fine-tuning. Fine-tuning has become increasingly popular in FSOD throughout the
last two years (see \cref{tab:fsod_comparison}). Simple fine-tuning strategies
are now competitive with elaborated attention mechanisms. Another motivation for
trying a fine-tuning approach is to study techniques from the three main
directions in FSOD. \cref{chap:prcnn,chap:aaf} respectively focus on
metric-learning and attention-based methods. The last kind of FSOD approach in
the literature relies on fine-tuning. While such a strategy is not very
innovative, this is one of the first applications of a diffusion-based approach
to a few-shot predictive task. In addition to the fine-tuning strategy, we
propose a transductive inference scheme to boost the performance of the
fine-tuned model. However, these are only preliminary work and do not yield the
expected results yet. Finally, we also investigate an attention-based extension
of DiffusionDet, while promising on paper the experimental study demonstrates
poor results.

\subsection{Fine-tuning Strategy for Few-Shot DiffusionDet}
We present in this section the fine-tuning strategy that we propose to
adapt DiffusionDet to the few-shot regime:
\vspace{-1em}
\begin{enumerate}[nolistsep]
    \item Train DiffusionDet in a regular fashion on a dataset containing only
    examples of the base classes.
    \item Once base training is done, replace the classification layer with a
    randomly initialized layer with as many output neurons as the number of
    novel classes.
    \item Freeze the entire backbone and let only the detection head be
    updatable. This choice is not optimal and will be discussed in the next
    section, however, we present here the baseline configuration. 
    \item Reset the learning rate scheduler, so it goes again through a warmup
    phase. The scheduling is a simple linear warmup starting at $\frac{1}{1000}$
    of the base Learning Rate (LR) and linearly increasing up to the base LR
    value during 1000 iterations. 
    \item Fine-tune the model with $K$ images for each novel class. All
    instances of the novel classes are kept while instances from other classes
    are discarded. This corresponds to the distractor-free sampling scheme
    discussed in \cref{sec:fsod_distractors}.

\end{enumerate}

As our goal is not to tackle the Generalized Few-Shot setting, we are mostly
interested in the performance on novel classes. Of course, one might want to
detect base classes as well, in this case, it is possible to keep a
version of the model after base training and leverage it for base classes
detection. Of course, it would require twice as much time to perform the
inference to detect both base and novel classes, but this is a mild compromise
compared to common issues raised by G-FSOD.

Fine-tuning is part of most FSOD methods as the adaptation of the regression
part of the models cannot be easily done on the fly (conversely to the
classification part). However, fine-tuning attention-based or metric learning
models is often quite long in comparison with "simple" fine-tuning strategies
which directly fine-tune object detectors on the support set without expensive
additional components (\eg a query-support attention block). This makes the
fine-tuning faster and unlocks much quicker iterations and experiments.
Nevertheless, fine-tuning approaches cannot be adapted at inference time and
therefore, it is difficult to measure the robustness to various support
examples. Thus, multiple fine-tunings are required to get a relevant evaluation
of a model, otherwise, the randomness of the support set can introduce some
variance and the comparison is less reliable. In practice, fine-tuning with
different support does not significantly change the performance of FSDiffusionDet.

While the proposed strategy is fairly simple, it yields impressive results. We
provide in \cref{tab:diff_baseline_comp} a comparison between the FSDiffusionDet
baseline and the discussed methods from previous chapters. It outperforms
largely the metric-learning and attention-based methods on aerial images. On
natural images, the gains are reduced but FSDiffusionDet is still superior
(especially for MS COCO where the problem is now a 20-ways detection problem and
not a 5-ways task as in previous chapters). A detailed analysis is conducted in
the next section to understand why it performs so well and how it can be
improved further. 

\begin{table}[]
    \centering
    \resizebox{0.80\textwidth}{!}{%
    \begin{tabular}{@{\hspace{2mm}}ccccccccc@{\hspace{2mm}}}
    \toprule[1pt]
              & \multicolumn{2}{c}{\textbf{DOTA}}        & \multicolumn{2}{c}{\textbf{DIOR}}        & \multicolumn{2}{c}{\textbf{Pascal VOC}} & \multicolumn{2}{c}{\textbf{MS COCO}}     \\ \midrule
    \textbf{Method}    & Base           & Novel          & Base           & Novel          & Base           & Novel         & Base           & Novel          \\ \midrule
    FRW       & 49.04          & 35.29          & 61.30          & 37.29          & 63.21          & 48.72         & 29.03          & 24.09          \\
    DANA      & 53.99          & 36.50          & 62.71          & 38.18          & 65.17          & 52.26         & 38.14          & 24.75          \\
    WSAAN     & 46.72          & 35.12          & 62.79          & 32.38          & 65.27          & 51.70         & 40.87          & 21.42          \\
    PFRCNN    & 36.32          & 11.55          & 42.37          & 9.16           & -              & -             & -              & -              \\
    XQSA      & 51.11          & 41.00          & 59.88          & 41.51          & 62.13          & \rbf{53.94}         & 31.56          & \rbf{25.03} \\ \midrule
    FSDiffDet & \bbf{69.58} & \rbf{52.05} & \bbf{81.71} & \rbf{54.32} & \bbf{74.63}          & 52.64         & \bbf{51.91} & 24.99          \\ \bottomrule[1pt]
    \end{tabular}%
    } \caption[FSDiffusionDet baseline compared with other FSOD
    methods]{FSDiffusionDet baseline compared with other FSOD methods. mAP is
    reported with a 0.5 IoU threshold and all methods leverage 10 shots.}
    \label{tab:diff_baseline_comp}
    \end{table}

Another advantage of FSDiffusionDet compared with attention methods is its
memory efficiency. Indeed, query-support combination blocks and support
embedding models require a lot of memory while training and often scale linearly
with the number of classes ($N$) and the number of shots $K$. FSDiffusionDet is
not limited by the number of shots and therefore, we can explore much higher
shot settings than with metric-learning or attention-based methods.
\cref{tab:fsdiffdet_shots} provides the novel classes performance on our four
datasets of interest. The base class performance is not reported here as they do
not depend on the number of shots, they can be found in the last row of
\cref{tab:diff_baseline_comp}. One can observe a smooth increase in performance
with the number of shots with a plateau above 50 shots. To better visualize this
trend and compare it with attention-based methods studied in \cref{chap:aaf},
we plot in \cref{fig:diff_perf_vs_shot} the performance against the number of
shots. From this, it can be seen that the performance is much lower in the
one-shot setting with FSDiffusionDet compared to attention-based approaches.
However, FSDiffusionDet quickly catches up and outperforms largely other methods
in higher shots settings. In addition, we can observe a much quicker increase in
performance as the number of shots increases with FSDiffusionDet. This is a
highly desirable property in an industrial application because this means that
the model has more potential for improvements. On the contrary, attention-based
approaches do not display such a strong trend, they are better suited for extremely
low-shot regimes, but become less effective with higher shots. 

\begin{table}[]
    \centering
    \resizebox{0.5\textwidth}{!}{%
    \begin{tabular}{@{\hspace{2mm}}ccccc@{\hspace{2mm}}}
    \toprule[1pt]
    $\bm{K}$ & \textbf{DOTA} & \textbf{DIOR} & \textbf{Pascal VOC} & \textbf{MS COCO} \\ \midrule
    1                & 4.19          & 27.17         & 22.24               & 7.43             \\
    2                & 9.83          & 40.31         & 31.98               & 12.45            \\
    3                & 27.61         & 43.54         & 29.52               & 15.75            \\
    5                & 39.00         & 46.92         & 38.08               & 19.33            \\
    10               & 52.05         & 54.32         & 52.64               & 24.99            \\
    20               & 62.79         & 60.24         & 59.26               & 28.76            \\
    30               & 67.32         & 65.28         & 64.19               & 31.19            \\
    50               & 71.91         & 71.21         & 67.81               & 34.64            \\
    100              & 72.27         & 77.05         & 71.31               & 38.77            \\ \bottomrule[1pt]
    \end{tabular}%
    } \caption[FSDiffusionDet performance against shots]{Influence of the number of
    shots on the few-shot object detection performance of FSDiffusionDet on
    DOTA, DIOR, Pascal VOC and MS COCO. Performance is reported with mAP$_{0.5}$.}
    \label{tab:fsdiffdet_shots}
\end{table}

\begin{figure}
    \centering
    \includegraphics[width=0.85\textwidth]{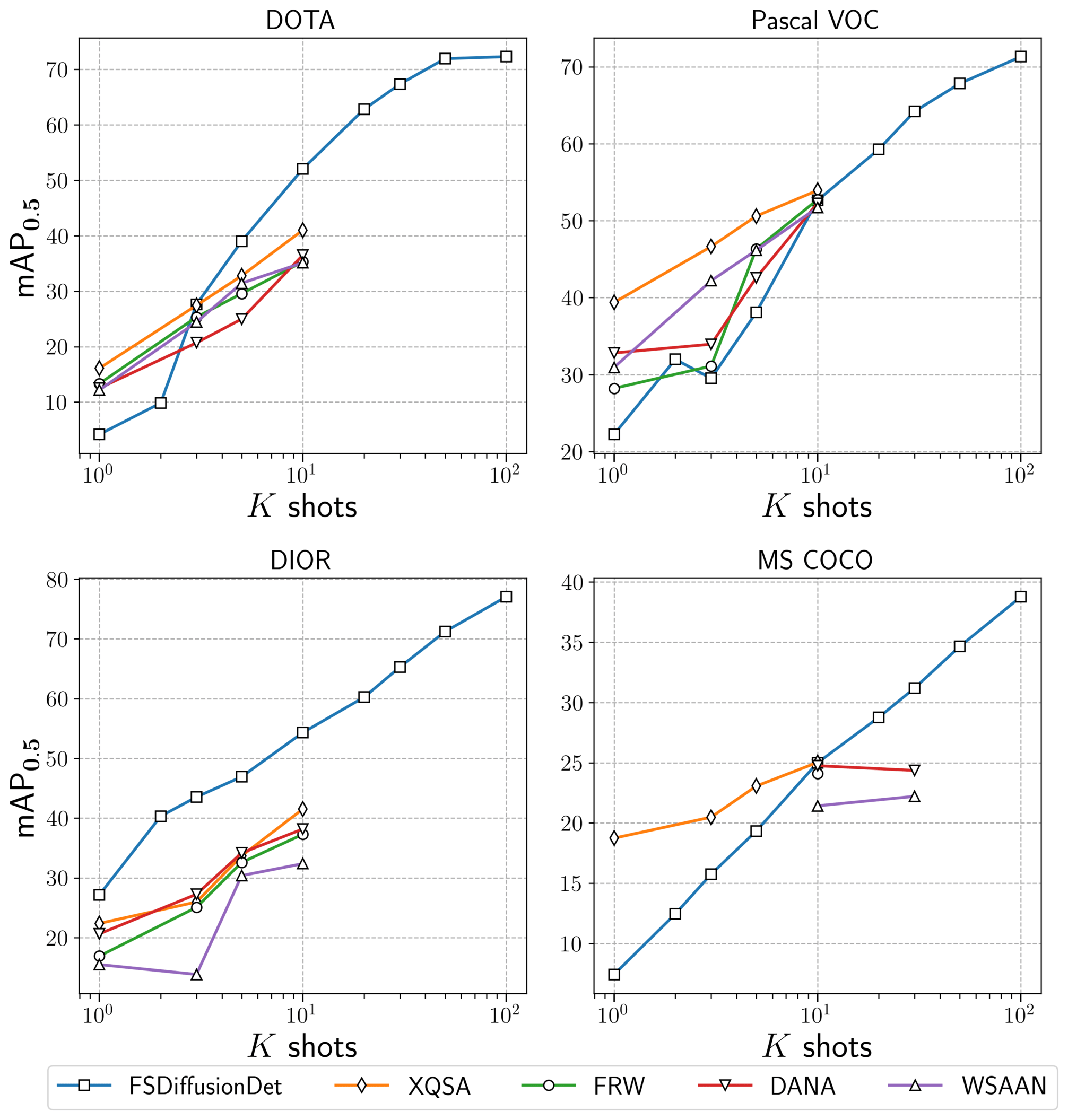}
    \caption[Influence of the number of shots on performance]{Performance of
    FSDiffusionDet, XQSA, FRW, DANA and WSAAN on DOTA, DIOR, Pascal VOC and MS
    COCO against the number of shots. Performance is reported with mAP$_{0.5}$.}
    \label{fig:diff_perf_vs_shot}
\end{figure}

\subsection{Experimental Study of FSDiffusionDet}
In the previous section, we have presented an efficient fine-tuning strategy for
DiffusionDet along with an analysis of its few-shot performance on several
datasets. However, this is only a baseline for FSDiffusionDet and its
performance can be improved further. Its fast training time allows for
conducting more experiments, which was too expensive with metric-learning and
attention-based methods. Thus, we present in this section a series of
experiments that we conducted to explore the capabilities of FSDiffusionDet but
also to answer more general questions about FSOD.

\subsubsection{Backbone Weights Initialization}
First, in classification, it is now well-known that using self-supervised
pre-trained backbones often boost a lot the few-shot performance of a method
\cite{chen2020simple,gidaris2019boosting,caron2021emerging,radford2021learning}.
While plenty of studies show this for classification, for detection, the
transferability of the learned features is not so evident. As a matter of fact,
a few contributions actually show that using such backbones is sub-optimal.  For
instance, InsLoc \cite{yang2021instance} and SoCO \cite{wei2021aligning} propose
object-level self-supervised techniques and prove empirically that image-level
SSL is not optimal and in some cases can even be detrimental to the detection
task. In this section, we study empirically the influence of using SSL
pre-trained weights for the backbone initialization, before base training and
fine-tuning. We do not consider the recent object-level techniques and instead
leverage four different initialization strategies for the backbone weights:
\vspace{-0.5em}
\begin{itemize}[nolistsep]
    \item[-] \textbf{Scratch}: weights are randomly initialized.
    \item[-] \textbf{ImageNet}: weights are initialized from a ResNet-50 trained in a supervised manner for ImageNet classification.
    \item[-] \textbf{DINO}: weights from the DINO \cite{caron2021emerging} pretraining on ImageNet.
    \item[-] \textbf{CLIP}: weights taken from the CLIP \cite{radford2021learning} model, trained in a contrastive way on a 400 million image-text pairs dataset.
\end{itemize}

Then, FSDiffusionDet is trained following our two steps training scheme (base
training and fine-tuning). The results are available in
\cref{tab:diff_backbone_weights}. From this, it is quite clear that training
from scratch is not a sensible option, even though base training is correct, the
fine-tuning on base classes provides really poor detection performance. Then,
between Imagenet, DINO and CLIP the differences are tight. Of course, CLIP's
weights are slightly worse than ImageNet and DINO, but it is still a strong
baseline. Between ImageNet and DINO, however, it is difficult to conclude as
both achieve the best performance on one aerial and one natural dataset. As the
performance gap is thin between ImageNet and DINO, we choose to conduct our next
experiments with ImageNet weights which have stood the test of time and are now
the default choice in computer vision.   

\begin{table}[]
    \centering
    \resizebox{0.6\textwidth}{!}{%
    \begin{tabular}{@{\hspace{2mm}}ccccc@{\hspace{2mm}}}
    \toprule[1pt]
    \textbf{Backbones} & \textbf{DOTA} & \textbf{DIOR} & \textbf{Pascal VOC} & \textbf{MS COCO} \\ \midrule
    Scratch            & 7.28          & 8.72          & 13.72               & 0.38             \\
    ImageNet           & \textbf{52.05}& 54.32         & 52.64               & \textbf{24.99}   \\
    DINO               & 46.84         & \textbf{55.88}& \textbf{54.58}      & 23.94            \\
    CLIP               & 40.36         & 51.61         & 49.81               & 19.83            \\ \bottomrule[1pt]
    \end{tabular}%
    } \caption[Study of the influence of the backbone pre-training]{Study of the
    influence of the backbone pre-training. mAP$_{0.5}$ is provided only for base
    classes, therefore the blue and red colors to distinguish between base and
    novel classes are no longer required.}
    \label{tab:diff_backbone_weights}
    \end{table}

\subsubsection{Plasticity Analysis}
The number of parameters frozen in the model is sometimes called the
\textit{plasticity} of the model in the continual learning field, but this
concept may also be useful in the few-shot setting. For simplicity, we measure
the plasticity of the model as the ratio between the number of trainable
parameters over the total number of parameters. Plasticity close to 1 means
that the model is malleable and could learn new complex tasks. However, when it
is close to 0, the model can barely change and learning new tasks may be
difficult. 

The plasticity is commonly explored in fine-tuning strategies for few-shot
tasks. The underlying principle is that the task is learned during base
training, and fine-tuning is only used to adapt the task to novel
classes. Hence, the behavior of the entire model should not change
dramatically. Therefore, the plasticity of the models is often quite low in the
FS literature. In practice, the early stages of the model are kept frozen while
only the deeper layers are trained. This trick is well-motivated as it
drastically reduces the capacity of the model and thus prevents overfitting,
which is particularly severe in low-shot regimes. In addition, it also reduces
catastrophic forgetting, which can be quite a challenge in G-FSOD. 

However, this may be inadequate for the detection task. As the detection is
primarily a problem of finding what is and what is not an object of interest,
the backbone is trained as a feature filter. Features from classes of interest
are highlighted while others are faded out. Conversely, for classification,
backbones are not required to learn such a filtering process as all classes are
"of interest", and there is always one object of interest in the image. In
addition, some recent experiments \cite{lee2022rethinking} about freezing
settings in cross-domain scenarios show that improved performance is achieved
with increased plasticity. In this work, the authors only study three freezing
settings: fine-tuning only the last layer (as proposed in TFA
\cite{wang2020frustratingly}), fine-tuning only the detection head (proposed in
FSCE \cite{sun2021fsce}), and fine-tuning the whole model (their proposition).
Here, we investigate the freezing setting in a more detailed manner with several
intermediary setups, but the main difference is that we conduct this analysis on
the same image domain: 
\vspace{-0.5em}
\begin{itemize}[nolistsep]
    \item[-] \textbf{Fine-Tune last layer only}: fine-tune only the last regression and classification layers.
    \item[-] \textbf{Fine-Tune head only}: fine-tune only the detection head.
    \item[-] \textbf{Up to stage $i$}: Freeze backbone up to stage number $i$ ($i \in
    \llbracket 1, 5 \rrbracket$ as ResNets have 5 stages).
    \item[-] \textbf{Fine-Tune whole}: fine-tune the whole model.
    \item[-] \textbf{Bias only}: fine-tune only the backbone biases.
    \item[-] \textbf{BatchNorm only}: fine-tune only the backbone BatchNorm parameters.
\end{itemize}

The results of this comparison can be found in
\cref{tab:diff_freezing_sweetspot}. It reports the mAP with different freezing
strategies on DOTA, DIOR, Pascal VOC and MS COCO. Additionally, the plasticity
rate is reported for each freezing strategy. Two distinct behaviors are observed
here. First, on DOTA, as the plasticity increases, the FSOD performance
increases as well. On DIOR, Pascal VOC and MS COCO, lower plasticity is
optimal (fine-tuning the detection head and the last stage of the backbone).
Therefore, the fine-tuning strategy cannot be set once for all datasets. It is
therefore crucial to understand what differs in DOTA from the other datasets. As
the task and the images remain similar between base training and fine-tuning,
the only source of variability comes from the class splits. Our hypothesis here
is that base and novel classes in DOTA are less compatible (\ie less alike) than
in the other datasets. For Pascal VOC, we briefly discuss this aspect in
\cref{sec:aaf_natural_images}, where we observed surprisingly high mAP for the
novel class \textit{sheep} as the class \textit{horse} was in the base set. A
more quantitative way of measuring the compatibility between the base and novel
class sets would be required to draw reliable conclusions about this. We are
currently working on this.

In addition, we observe that fine-tuning the backbone entirely is often a
sub-optimal choice. Instead, higher (or at least competitive) results are
achieved by fine-tuning only the biases or the batch normalization parameters of
the backbone. Fine-tuning only the biases or the batch normalization parameters
in the backbone does not change much the plasticity as only a few parameters are
concerned, yet it seems to provide a beneficial adaptability to the entire
backbone. On DIOR, Pascal VOC and MS COCO, it provides very high mAP compared to
other settings with similar plasticity. Finally, fine-tuning only the very last
layer of the classification and regression branches is completely sub-optimal.
Strangely, this contradicts some FSOD models that adopt this strategy and
achieve reasonable performance (\eg TFA \cite{wang2020frustratingly}). With
FSDiffusionDet, this strategy achieves extremely poor detection, having too
small plasticity must be avoided. Thus, a plasticity compromise must be found
depending on the dataset and its split compatibility.

\begin{table}[]
    \centering
    \resizebox{0.85\textwidth}{!}{%
    \begin{tabular}{@{\hskip 2mm}lccccc@{\hskip 2mm}}
    \toprule[1pt]
    \textbf{Freezing point}& \textbf{Plasticity rate}   & \textbf{DOTA} & \textbf{DIOR} & \textbf{Pascal VOC} & \textbf{MS COCO} \\ \midrule
    FT whole               & 100.00 \%             & 60.09         &  52.17             & 43.10           &  17.15             \\
    Bias only              & 35.98 \%              & \textbf{60.45}&  55.12             & 49.90           &  20.19             \\
    BatchNorm only         & 35.97 \%              & 59.35         &  55.63             & 51.96           &  19.70             \\
    Up to stage 1          & 99.98 \%              & 58.85         &  53.37             & 43.81           &  17.72             \\
    Up to stage 2          & 99.47 \%              & 57.41         &  53.21             & 41.23           &  17.73             \\
    Up to stage 3          & 96.57 \%              & 59.88         &  54.36             & 47.57           &  19.49             \\
    Up to stage 4          & 79.66 \%              & 56.13         &  \textbf{57.51}    & \textbf{53.72}  &  \textbf{21.88}    \\
    FT head only           & 35.97 \%              & 51.82         &  55.70             & 51.72           &  19.96             \\
    FT last layer only     & 0.03 \%               & 0.05          &  0.11              & 0.53            &   0.01             \\ \bottomrule[1pt]
    \end{tabular}%
    } \caption[Plasticity Analysis]{Influence of the amount of plasticity on the
    FS performance on DOTA, DIOR, Pascal VOC and MS COCO. mAP is reported with a 0.5 IoU threshold.}
    \label{tab:diff_freezing_sweetspot}
    \end{table}

\subsubsection{Number of Proposals}
Another set of experiments explores the influence of the number of
\textit{proposals} for FSOD. The proposals are the boxes sampled at the
beginning of the diffusion process. The number of proposals $N_p$ represents the
maximum number of objects that the model can detect in one image. This number is
chosen large compared to the average number of objects in the images.
Intuitively, sampling more random boxes reduces the chances of missing an
object. However, having a higher number of proposals generates more duplicates
which can be detrimental as well. More proposals also lead to a higher training
time and memory usage as the denoising process is applied on all boxes. The
critical parts are the self-attention layers that scale in $O(N_p^2)$. Thus, we
investigate the few-shot performance of FSDiffusionDet with various numbers of
proposals. The results of these experiments are available in
\cref{tab:fsdiff_proposals}. We notice two different behaviors between natural
and aerial images. For natural images (Pascal VOC and MS COCO), it seems better
to set the number of proposals relatively low compared to aerial images. This
makes sense as there are more objects in aerial images. For natural images, the
detection quality increases as the number of proposals is reduced, and it may be
relevant to test what happens with even fewer proposals. However, with aerial
images, the performance does not seem to correlate well with the number of
proposals. It is relevant to mention that the results on MS COCO are opposite to
what the authors of DiffusionDet found in the regular data regime (increasing
the number of proposals increases the mAP). This could be explained by the
reduced number of objects in the images, as in the few-shot regime we consider
only the novel classes, many instances are discarded and fewer proposals are
required to detect the objects.

\begin{table}[]
    \centering
    \resizebox{0.69\textwidth}{!}{%
    \begin{tabular}{@{\hspace{2mm}}ccccc@{\hspace{2mm}}}
    \toprule[1pt]
    \textbf{\# of Proposals} & \textbf{DOTA} & \textbf{DIOR} & \textbf{Pascal VOC} & \textbf{MS COCO} \\ \midrule
    \textbf{200}             & 41.57         & 52.92         & \textbf{52.86}           & \textbf{23.24}         \\
    \textbf{250}             & 47.97         & 47.62         & 52.28           & 22.61         \\
    \textbf{300}             & \textbf{55.76}         & 51.77         & 51.81           & 22.46         \\
    \textbf{350}             & 52.27         & 50.41         & 50.63           & 22.13         \\
    \textbf{400}             & 46.49         & 49.98         & 50.55           & 20.04         \\
    \textbf{450}             & 53.11         & 53.07         & 51.06           & 20.48         \\
    \textbf{500}             & 52.03         & \textbf{55.31}         & 51.44           & 20.25         \\ \bottomrule[1pt]
    \end{tabular}%
    } \caption[Number of proposals study]{Analysis of
    FSDiffusionDet performance (mAP$_{0.5}$) against the number of proposals on DOTA, DIOR,
    Pascal VOC and MS COCO datasets.}
    \label{tab:fsdiff_proposals}
    \end{table}

\subsubsection{Other Experiments and Future Directions}
\label{sec:diff_experiments}
In addition to the previous experiments, we conduct several other studies to
further improve the detection capabilities of FSDiffusionDet. However, some of
these studies did not yield very relevant insights, some others were not
explored deeply enough due to the time constraint of this PhD project. We briefly
present these experiments that will pave the way for future improvements of
FSDiffusionDet.

\textbf{Learning Rate Sweeping.}\\
First, just as proposals, freezing sweet spot and backbone pre-training, we
studied the influence of the Learning Rate (LR) and its schedule on the FSOD
performance. Indeed, the choice of the LR value during the fine-tuning is not
trivial. Therefore, to make sure we get a good fit for our experiments we
conduct a LR sweeping, \ie we try several different values for the LR. This
experiment is conducted only on DOTA with $K=10$ for simplicity, but in theory,
it should be done for every new experiment. Indeed, following the nomenclature
from the recent Deep Learning Tuning Playbook \cite{tuningplaybookgithub}, the
learning rate is a nuisance parameter, meaning that to make a fair comparison
between various settings, the optimal LR should be found for all runs
individually. As we change some hyperparameters, it is likely that the optimal
learning rate changes as well, therefore fair comparison can only be achieved if
the LR is optimal for all runs, \eg the optimal learning rate for $K=1$ or
$K=100$ shots may not be the same. Even though fine-tuning methods are fast to
adapt, running such an LR analysis is very expensive. Nevertheless, running
an LR sweeping on DOTA provides insights into how it influences the FSOD
performance. The results can be found in \cref{tab:diff_lr_sweep}, and show an
optimal value around $5\mathrm{e}{-5}$. But most importantly, it shows a
relatively large area where performance is satisfactory. This comforts us in our
choice of fixing the LR for all our experiments. While this is probably not the
optimal choice, it is reasonable. We also tried a cosine annealing scheduler,
but it yields consistently inferior results and was then rejected. Its only
advantage is that it seems to deal better with higher LR, which makes the
training slightly faster.

\begin{table}[]
    \centering
    \resizebox{0.6\textwidth}{!}{%
    \begin{tabular}{@{\hspace{2mm}}ccc@{\hspace{2mm}}}
        \toprule[1pt]
        Learning rate         & \textbf{Constant Schedule} & \textbf{Cosine Annealing} \\ \midrule
        $1\mathrm{e}{-6}$ & 39.04                      & 29.45                     \\
        $5\mathrm{e}{-6}$ & 49.06                      & 45.32                     \\
        $1\mathrm{e}{-5}$  & 52.31                      & 49.33                     \\
        $5\mathrm{e}{-5}$  & \textbf{53.46}             & \textbf{52.99}            \\
        $1\mathrm{e}{-4}$   & 53.25                      & 51.96                     \\
        $5\mathrm{e}{-4}$   & 49.33                      & 52.51                     \\
        $1\mathrm{e}{-3}$    & 47.45                      & 47.82                     \\ \bottomrule[1pt]
        \end{tabular}%
        } \caption[Learning rate sweeping on DOTA]{Learning rate sweeping on DOTA
    dataset with $K=10$ shots. Two distinct schedulers are considered: constant
    and cosine annealing. Performance is reported with mAP$_{0.5}$.}
    \label{tab:diff_lr_sweep}
    \end{table}

\textbf{Proposal Prior Distribution.}\\
In DiffusionDet, the coordinates of the proposals are sampled randomly following
a normal distribution. The coordinates of the boxes are clamped with a scale
parameter $\varsigma$ to make sure the center of each box remains within the
image limits. Specifically, we have: 
\begin{align}
    w &= \left(\text{clamp}(\epsilon_w, -\varsigma, \varsigma) / \varsigma + 1 \right) / 2, \\
    h &= \left(\text{clamp}(\epsilon_h, -\varsigma, \varsigma) / \varsigma + 1 \right) / 2, \\
    x &= \left(\text{clamp}(\epsilon_x, -\varsigma, \varsigma) / \varsigma + 1 \right) / 2 - \frac{w}{2}, \\
    y &= \left(\text{clamp}(\epsilon_y, -\varsigma, \varsigma) / \varsigma + 1 \right) / 2 - \frac{h}{2}, \\
    &\text{with} \; \epsilon_x, \epsilon_y, \epsilon_w, \epsilon_h \sim \mathcal{N}(0,1).
\end{align}

By default, $\varsigma$ is set to 2, but it would be interesting to explore how
it changes the FSOD performance. Indeed, as $\varsigma \to 0$ the boxes are more
and more identical and their centers tend to approach the image corners, as
$\varsigma \to \infty$, the boxes are more and more aligned with the image
center. It could also be relevant to explore the use of a uniform sampling
instead of a Gaussian distribution. This would prevent having a bias toward the
image center as it happens with high values of $\varsigma$. In fact, the use of
$\varsigma$ close to 1, is relatively close to a uniform distribution. This is
illustrated in \cref{fig:diff_ddim_initial_distribution}.

The setting of $\varsigma$ slightly changes the diffusion process as the boxes
are generated from clamped Gaussian distributions and not regular Gaussians.
Most of the derivations detailed in \cref{sec:ddpm_principle} hold only for
gaussian distributions and therefore using small values of $\varsigma$ or
uniform prior may disrupt the diffusion process. We did not have time to conduct
these experiments yet, but this is planned as future work. Another consideration
is the size of the generated boxes. In the above, the size of the proposal is
randomly sampled and has an expected value of half the image size. This may not
be optimal, especially when applying FSDiffusionDet on aerial images with small
objects. Hence, we propose to introduce a proposal scaling parameter $\varpi$
that divides the width and height of the proposals: 
\begin{align}
    w' &= \frac{w}{\varpi}, \\
    h' &= \frac{h}{\varpi},
\end{align}

where $w'$ and $h'$ are the scaled width and height of the sampled boxes. Of
course, as small proposals may not cover the whole image, their number must be
increased to prevent missing objects. Going further, we can also imagine a
mixture of width and height distribution to sample proposals with significant
size differences, which is not achieved in practice yet as shown in
\cref{fig:diff_ddim_initial_distribution}. \cref{fig:diff_ddim_scales_proposals}
shows how the proposals would change with $\varpi$, with $\varsigma=2$ fixed.
\begin{figure}
    \centering
    \includegraphics[width=\textwidth]{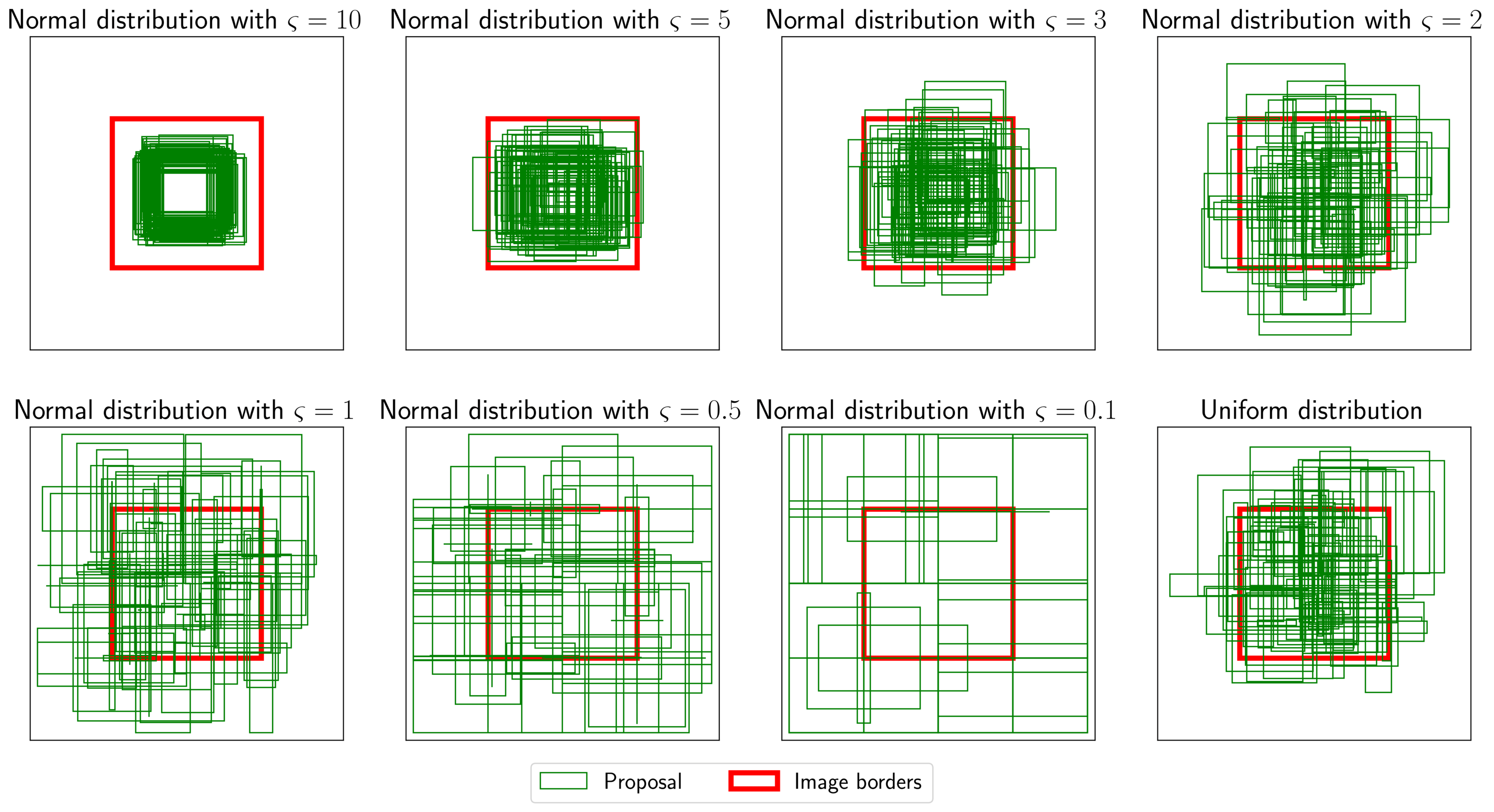}
    \caption[DiffusionDet initial random boxes]{DiffusionDet initial random
    boxes with various values of $\varsigma$, the parameter that controls the
    spread over the images. 75 proposals are sampled per image.}
    \label{fig:diff_ddim_initial_distribution}
    \vspace{1cm}
    \includegraphics[width=\textwidth]{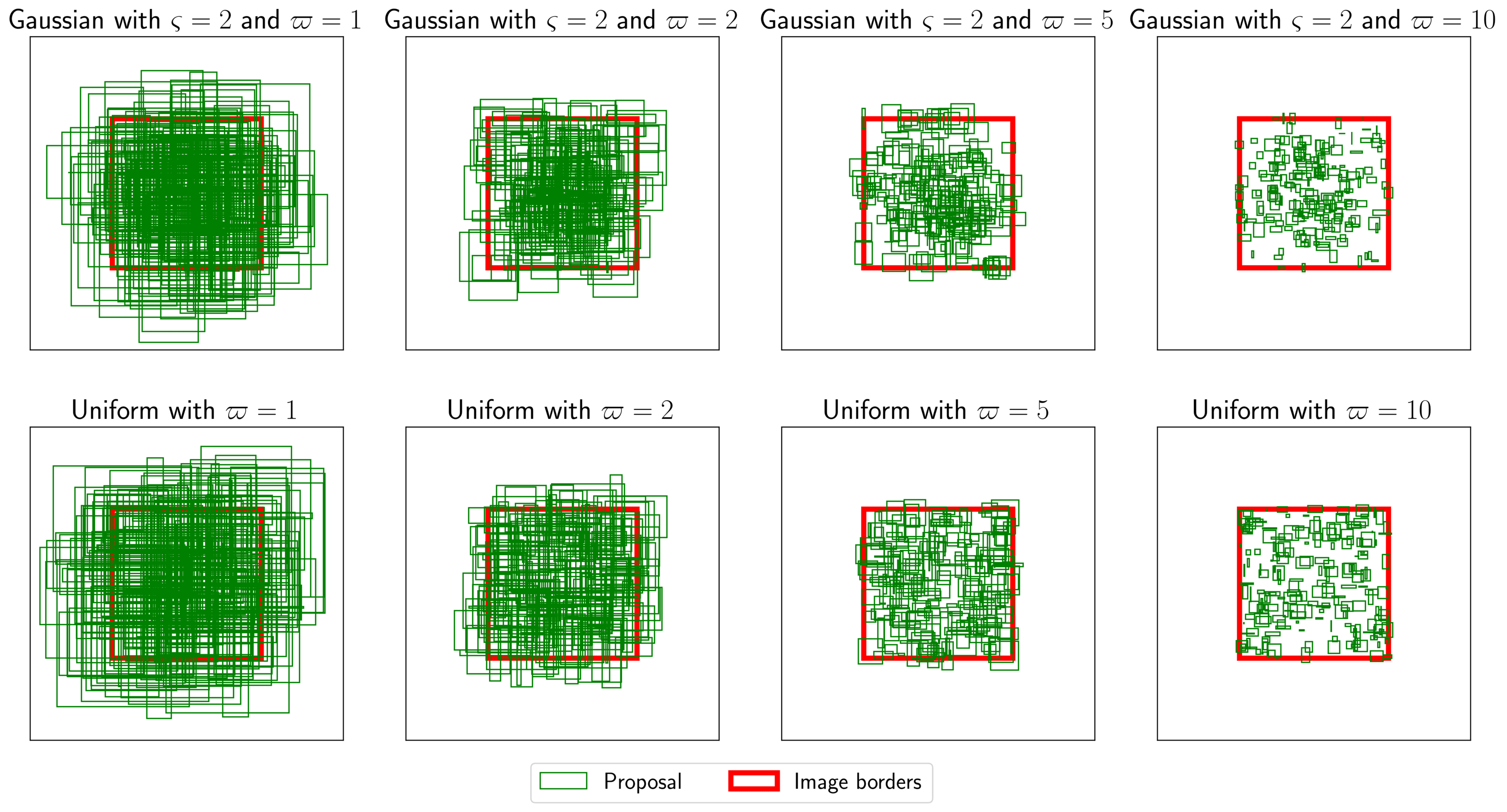}
    \caption[DiffusionDet scaled proposals]{Influence of $\varpi$ on the size of
    the proposals. Note that here 200 proposals are sampled, for visualization purposes.}
    \label{fig:diff_ddim_scales_proposals}
\end{figure}

\textbf{Transductive Inference.}\\
To further improve the FSOD performance of FSDiffusionDet, we also consider
designing a transductive inference scheme inside the detection framework. To our
knowledge, this would be a first in the FSOD domain. Of course, the transductive
setting is slightly different from the regular few-shot inference as it requires
access to a large set of query images during the inference. The goal is to
detect objects in these images, but these unlabelled images can be leveraged to
improve the detection on the entire set. This setup makes a lot of sense for
COSE's application. Indeed, the very large images of COSE cannot be processed as
a whole, instead, they must be cropped into smaller patches. This means that a
relatively large number of images are to be processed at the same time ($ 11600
\times 8700$ pixels images can be cropped into roughly 400 patches of $512
\times 512$ pixels). Hence, studying transductive inference is of particular
interest to the company.

First of all, in classification, the transductive inference is often used in
replacement of a fine-tuning step and allows for direct adaptation of a model to a
new task or domain (see \cref{sec:transductive_cls} for more details). This is
prohibited in the detection context as the regression branch must be fine-tuned
anyway. Thus, our goal is to improve the classification part once the model has
been fine-tuned. To this end, we propose to adapt LaplacianShot
\cite{ziko2020laplacian} to work with representations of objects instead of
representations of entire images. Specifically, the Laplacian Shot Module (LSM)
replaces the classification layer of the model. As input, LSM receives a set of
all objects representation detected in the query set (using the boxes produced
by the regression branch) and the representation of the annotated support
examples. Then, LSM optimizes an objective function to find an optimal label
assignment $\mathcal{Y}$ of all query examples (in our case all objects found in the query
set):
\begin{equation}
    \mathcal{L}_{\text{LSM}}(\mathcal{Y}) = \sum\limits_{i=1}^{|Q|}\sum\limits_{c=1}^{|\mathcal{C}_{\text{novel}}|} l_i^c d(z_i, m_c) + \frac{1}{2} \sum\limits_{i,j} \eta(z_i, z_j) l_i^Tl_j,
\end{equation}

where $z_q$ and $l_q$ are the extracted features and classification score vector
for object $q$, respectively. $Q$ represents the set of object representations
in the whole query set, $d$ is a distance measure (\eg the euclidean distance
between objects representations), and $\eta$ is a similarity function in the
embedding space (in practice, it is chosen as a binary $k$-NN, \ie a vector has
a similarity of 1 with its $k$ nearest neighbors and 0 with all others).
Finally, $m_c$ is the representation of class $c$, it is computed as the average
over the support representations of that class. Intuitively, this objective
function finds a compromise between assigning to an unlabelled object the label
of the closest support example and assigning the same label as its neighbors. 

As a first comparison, we leverage four distinct inference setups for detection: 
\begin{itemize}[nolistsep]
    \item[-] \textbf{Fine-tuning Inference (FI)}: boxes and classification
    scores output by the fine-tuned model.
    \item[-] \textbf{Transductive Inference (TI)}: boxes from fine-tuned model
    and classification score from the Laplacian Shot Module.
    \item[-] \textbf{Hybrid Transductive Inference (HTI)}: boxes from fine-tuned
    model and classification scores as a combination (\eg element-wise
    multiplication) of LSM and fine-tuned model.
    \item[-] \textbf{Optimal Classification (OC)}: boxes from the model and
    optimally matched labels from the ground truth. It can be seen as an
    \textbf{oracle}, it is a performance upper-bound given the quality of the
    regression.
\end{itemize}

We assume above that for the detection task, the regression branch must be
fine-tuned otherwise performance is highly degraded. To confirm this assumption,
we compare the four inference settings described above using a model that has
only been base-trained against a model that was fine-tuned on the novel classes.
This is done on DOTA with $K=10$ shots. The results are available in
\cref{tab:diff_transductive_naive}. From that table, it is clear that the
fine-tuning of the regression branch is crucial to achieving reasonable
performance. In particular, the oracle (OC) is highly degraded when the
regression branch is not fine-tuned. Interestingly, in this case, TI achieves
higher mAP than the non-fine-tuned classification branch of the model. It makes
sense as this layer is only initialized with random weights. However, one can
see that the TI is largely under the FI when using the fine-tuned regression
branch. The classification made by the LSM is therefore worse than the
fine-tuned classification branch. A quick investigation of the classification
scores shows that the scores output by TI differ significantly from FI (see
\cref{fig:diff_transductive_score_hist}). The fine-tuned model outputs a large
number of very small classification scores which mostly correspond to background
objects. Hence, they are filtered out by the post-processing (score thresholding
and NMS), and the remaining ones will have a negligible impact on the mAP
computation. TI, however, outputs much higher scores, with a greater variance.
It struggles to distinguish foreground and background objects, and for good
reason, it was designed for classification and not detection. To this end, we
propose HTI, a hybrid classification inference that leverages both the score
from the fine-tuned model and the transductive inference. Hopefully, it will fix
the mistakes from the fine-tuned classification layer while avoiding the TI's
pitfall. To do so, we simply multiply the scores from FI and TI together. Thus,
the good foreground/background distinction from FI is embedded in the new score
distribution. This helps a lot for the classification; however, it is still
under the FI performance. From this, it seems that TI is only detrimental to
classification. TSNE representations of the embedding space help to make sense
of these results. These can be found in \cref{fig:diff_t_tsne} for FI, TI and
HTI, and the Oracle. 

One can see large patches of the cluster representing the class 2 misclassified
as class 14 by both TI and HTI. The prototype of class 2 seems closer to these
misclassified points in the TSNE visualizations, but the distances must be
interpreted carefully as two dimensions are not enough to represent the entire
complexity of the representation space (dimension 256). Yet, it seems that the
distributions of the classes in the embedding space are multimodal, therefore it
may be impossible to accurately classify the objects with only one prototype per
class. Leveraging multiple prototypes per class should be investigated in future
work. 

\begin{figure}
    \centering
    \includegraphics[width=\textwidth]{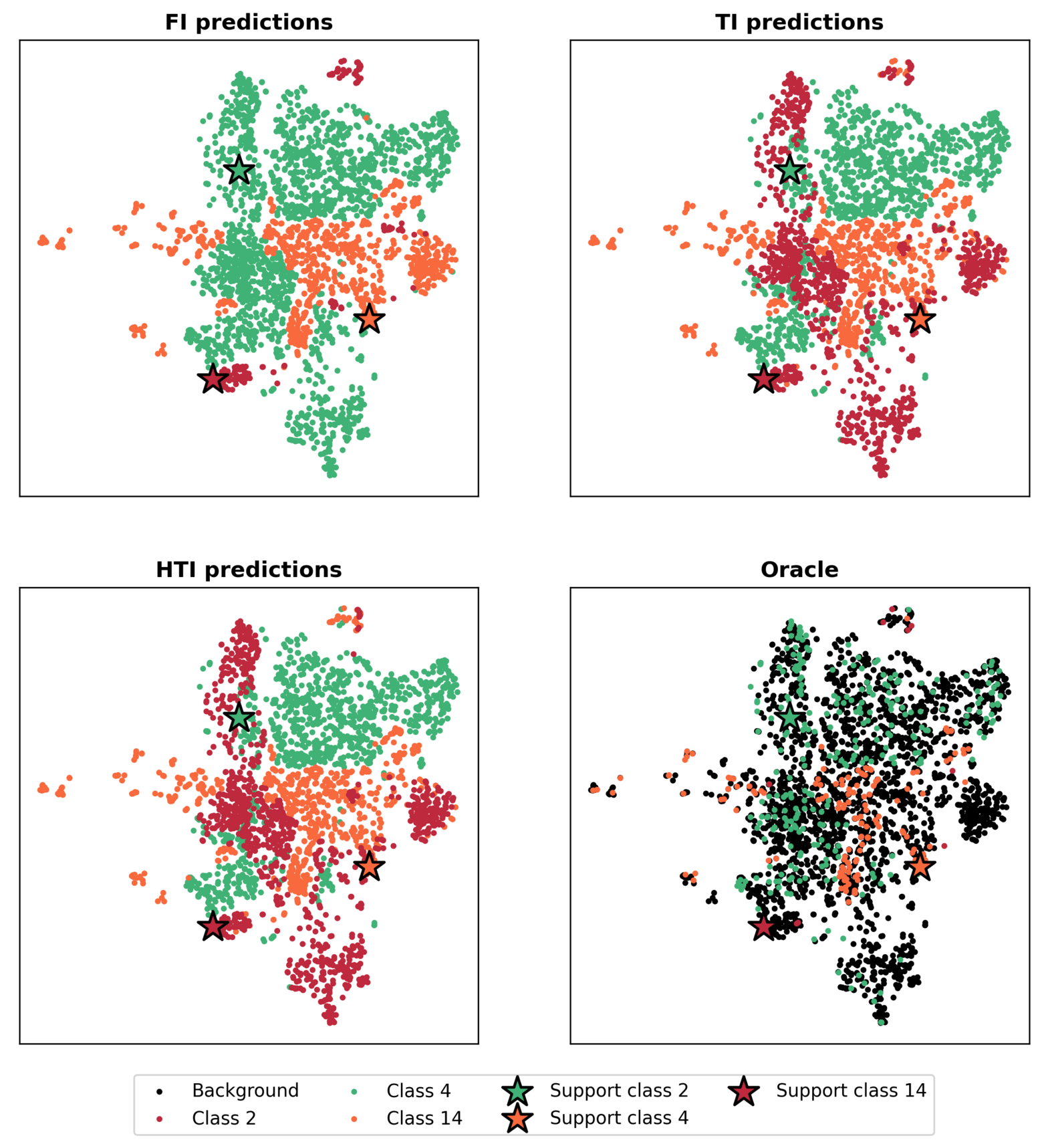}
    \caption[TSNE visualizations of transductive inference]{Comparisons of the
    TSNE visualizations of the fine-tuned model predictions (FI), transductive
    inference (TI), hybrid inference (HTI) and the oracle. Note that background
    predictions are only available for the oracle as the models' inferences only
    provide class scores.}
    \label{fig:diff_t_tsne}
\end{figure}

\begin{figure}
    \centering
    \includegraphics[width=\textwidth]{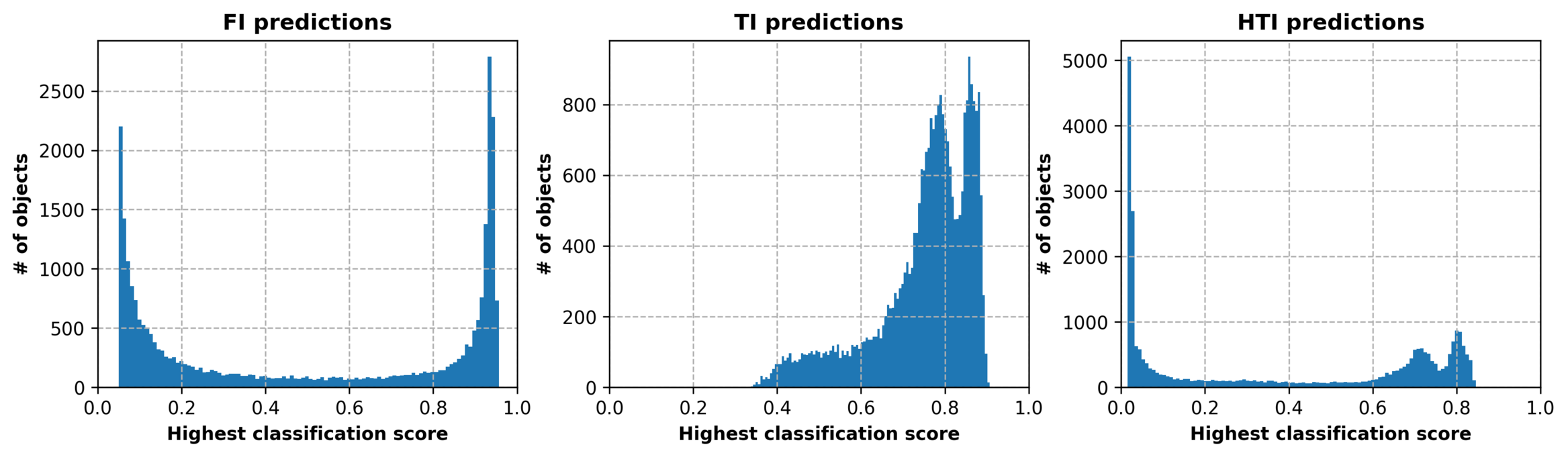}
    \caption[Transductive cores histograms comparison]{Comparisons of the
    scores histograms of fine-tuned model predictions (FI), transductive
    inference (TI), and hybrid inference (HTI).}
    \label{fig:diff_transductive_score_hist}
\end{figure}

In addition to HTI, we also tried to filter the objects with low FI scores (with
a threshold at 0.05), this greatly helps for the transductive inference although
it reduces slightly the performance of FI (some objects are correctly detected
but have a low score but are filtered anyway). With filtering and HTI, we
almost reach the same performance as the FI which is encouraging. However, the
goal is to benefit from the transductive inference and this is not achieved yet.
More analysis needs to be conducted to better understand the reason behind the
poor classification score of TI. In addition, the transductive inference should
be extended to account for the detection challenges: a great
foreground/background imbalance and an increased intra-class diversity. To this
end, we have a few ideas that we did not have time to explore yet: 

\begin{itemize}[nolistsep]
    \item[-] \textbf{Geometric priors}: leverage the geometrical features of the
    objects (\eg size, aspect ratio, etc) to find outliers and exclude them
    without using a hard score thresholding. This could significantly reduce the
    number of objects and help to filter ill-formed boxes.
    \item[-] \textbf{Multiple prototypes per class}: instead of aggregating all support
    examples of one class as one prototype, use multiple prototypes per class
    and extend LSM as a mixture model.
    \item[-] \textbf{Background class}: Introduce a background class within the LSM module to prevent the
    poor foreground/background distinction. 
\end{itemize}

These tracks will be explored during the last months of this PhD project and
will hopefully improve further the detection capabilities of FSDiffusionDet. In
addition, we plan to apply the transductive inference in Cross-Domain scenarios,
as it could help against the performance drop in these settings. 

\begin{table}[]
    \centering
    \resizebox{0.6\textwidth}{!}{%
    \begin{tabular}{@{\hspace{2mm}}ccccc@{\hspace{2mm}}}
    \toprule[1pt]
    \multicolumn{1}{c}{\textbf{Training Strategy}}            & \textbf{FI} & \textbf{TI} & \textbf{HTI} & \textbf{OC} \\ \midrule
    Base training only              & 0.04        & \textbf{3.79 }       & 2.56          & 25.42       \\
    Base training + FT              & \textbf{58.98}       & 33.61       & 53.90         & 67.35       \\ 
    Base training + FT + Filtering  & \textbf{57.95}       & 43.63       & 57.00         & 67.35       \\\bottomrule[1pt]
    \end{tabular}%
    } \caption[Transductive inference against fine-tuned model]{Naive comparison
    between a model only trained on base classes against a model that has been
    fine-tuned as well. All 4 inference setups are compared as well. Bold values
    represent the best-performing method between FI, TI and HTI, the Oracle
    (OC) is not included. Performance is reported as the mAP$_{0.5}$}
    \label{tab:diff_transductive_naive}
    \end{table}

\textbf{Support Attention.}\\
Finally, we also try to extend FSDiffusionDet with an attention mechanism (\eg
XQSA). The main motivation is to be able to compare the influence of the
detection framework on attention-based FSOD methods (in \cref{chap:aaf}, we
studied the influence of the attention mechanism with a fixed detection
framework). Given the impressive results of FSDiffusionDet with a simple
fine-tuning strategy, this is promising. 

Thus, we extend the detection head with a query-support block which is meant to
incorporate the support features within the detection head. The head is then
split and boxes are produced for each class independently (following the
attention-based FSOD principle, see \cref{fig:aaf_attention_principle}).
Unfortunately, this does not yield satisfactory results and slows down the
training a lot. Considering the time already spent on attention mechanisms since
the beginning of the project and the very good performance of FSDiffusionDet
with the fine-tuning, we decided not to explore this direction further.
Nonetheless, this remains an interesting direction for future work.

\subsection{Comparison with existing FSOD Methods}
In the previous section, we explored several design choices for FSDiffusionDet
and analyzed how they influence the detection performance on novel classes. We
compare here the best settings for FSDiffusionDet according to our experiments
conducted in the previous section. These experiments are averaged over 5
distinct seeds to get more reliable results. This contrasts with the above
experiments which are mostly done with one seed only. However, the limited
variance observed over the multiple runs confirms that previous results are
reliable as well. FSDiffusionDet is compared with PFRCNN and XQSA that we
proposed in \cref{chap:prcnn,chap:aaf} and some relevant works from the
literature. This comparison can be found in \cref{tab:diff_main_comparison}.
This table also includes the Small, Medium, and Large size distinctions from the
previous chapters. FSDiffusionDet largely outperforms other methods disregarding
the object sizes. For small objects, FSDiffusionDet achieves impressive
performance on aerial images but lags slightly behind XQSA on natural images. It
is particularly noteworthy as it was not designed specifically for small object
detection. Another surprising result can be observed on DOTA where medium size
objects are better detected than large ones, which is not the case for other
datasets. This is unusual compared to all our previous experiments, including
attention-based methods. It might result from having too few proposals boxes
($N_p=300$ in \cref{tab:diff_main_comparison} for DOTA), then the model can only
focus on small and medium objects as they are more numerous than larger ones.
Given the size distribution in DOTA, this is a better compromise as it yields
higher overall mAP. From \cref{tab:diff_main_comparison}, it seems that
FSDiffusionDet performs slightly worse than XQSA, DANA and FRW on MS COCO. This
is not true as FSDiffusionDet tackles MS COCO as a 20-ways detection problem
whereas other approaches only consider an easier 5-ways problem. It is not
possible to perform 5-ways episodic evaluation with FSDiffusionDet as all
classes must be included during fine-tuning. However, it could be interesting to
observe how well attention-based methods perform in the 20-ways settings (it is
often challenging to do so with attention-based methods due to memory
constraints and long inference time).

\begin{table}[]
    \centering
    \resizebox{\textwidth}{!}{%
    \begin{tabular}{@{\hspace{1mm}}lcccccccccccccccc@{\hspace{1mm}}}
    \toprule[1pt]
                   & \multicolumn{4}{c}{\textbf{DOTA}}                                 & \multicolumn{4}{c}{\textbf{DIOR}}                                 & \multicolumn{4}{c}{\textbf{Pascal VOC}}                               & \multicolumn{4}{c}{\textbf{MS COCO}}                                 \\
                   & All            & S              & M              & L              & All            & S              & M              & L              & All            & S              & M              & L              & All            & S              & M              & L              \\ \midrule
    FRW            & 35.29          & 13.99          & 34.11          & 59.31          & 37.29          & 2.48           & 33.74          & 59.38          & 48.72          & 16.44          & 26.71          & 68.27          & 24.09          & 11.53          & 22.45          & \textbf{38.69} \\
    DANA           & 36.50          & 14.32          & 40.28          & \textbf{64.65} & 38.18          & 3.21           & 34.91          & 60.99          & 52.26          & 10.05          & 24.67          & 67.23          & 24.75          & 12.01          & 29.40          & 37.95          \\
    WSAAN          & 35.12          & -              & -              & -              & 32.38          & -              & -              & -              & 51.70          & -              & -              & -              & 21.42          & -              & -              & -              \\ \midrule
    PFRCNN         & 11.55          & -              & -              & -              & 9.16           & -              & -              & -              & -              & -              & -              & -              & -              & --             & -              & -              \\
    XQSA           & 41.00          & 17.84          & 44.57          & 54.46          & 41.51          & 4.12           & 40.69          & 58.21          & 53.94          & \textbf{19.46} & \textbf{34.86} & 66.14          & \textbf{25.03} & \textbf{12.57} & \textbf{26.05} & 38.55          \\
    FSDiffusionDet & \textbf{57.93} & \textbf{45.99} & \textbf{61.33} & 53.25          & \textbf{55.80} & \textbf{14.66} & \textbf{54.14} & \textbf{72.82} & \textbf{55.80} & 15.05          & 30.20          & \textbf{69.64} & 24.03          & 5.17           & 19.23          & 38.62          \\ \bottomrule[1pt]
    \end{tabular}%
    } \caption[Detection results of FSDiffusionDet on DOTA, DIOR, Pascal VOC and
    MS COCO datasets]{Detection results of FSDiffusionDet on DOTA, DIOR, Pascal
    VOC and MS COCO datasets. The models employed to produce this figure have
    been finetuned with $K=10$ shots and following the best fine-tuning strategy
    found in \cref{sec:diff_experiments} for each dataset. The mAP$_{0.5}$ is reported
    as an average of over 5 distinct runs. The top rows include methods from the
    literature while the bottom rows designate our proposed methods.}
    \label{tab:diff_main_comparison}
    \end{table}

\begin{figure}
    \centering
    \includegraphics[width=\textwidth]{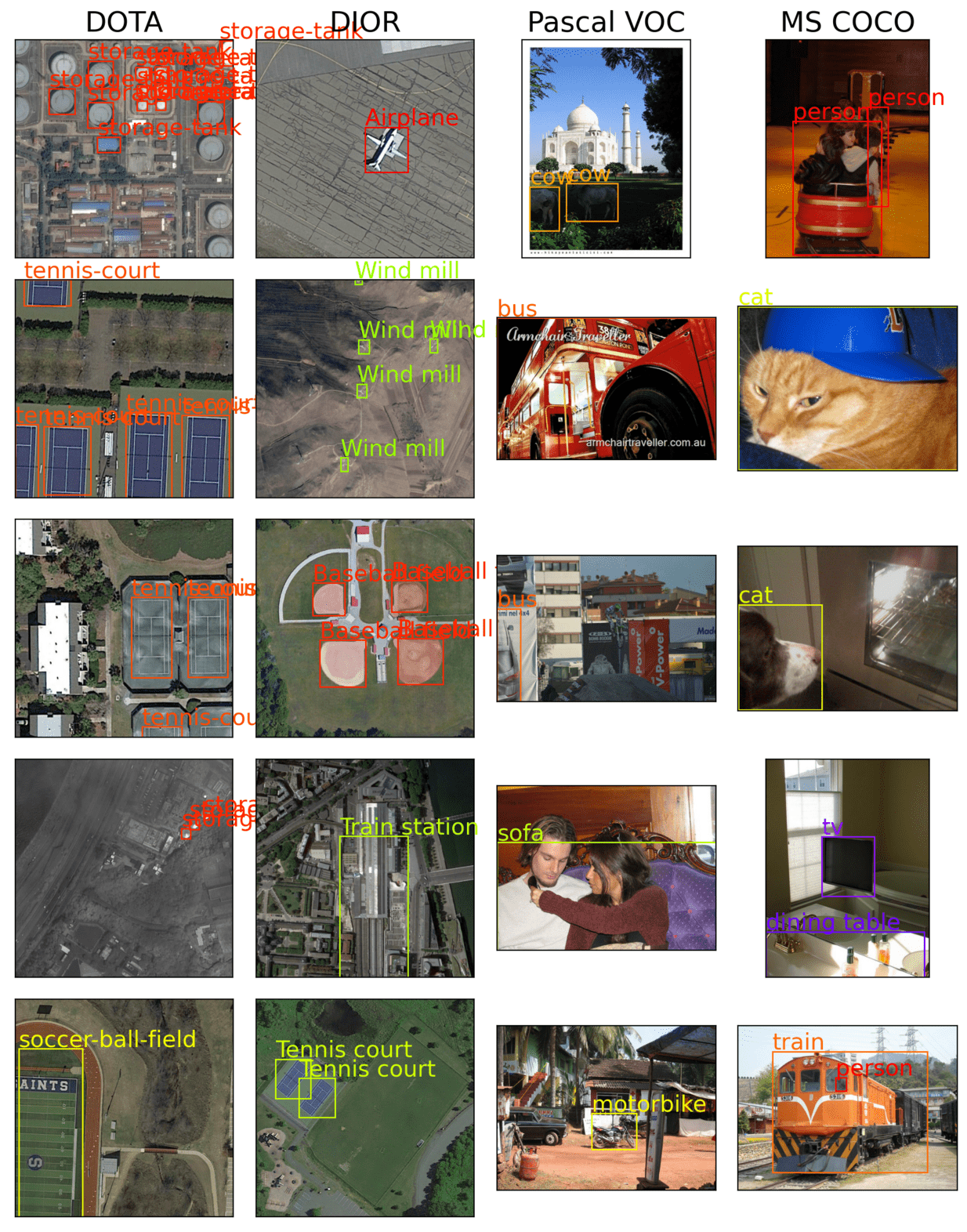}
    \caption[FSDiffusionDet qualitative detection results]{Qualitative detection results of FSDiffusionDet on DOTA,
    DIOR, Pascal VOC and MS COCO datasets. The models employed to produce this
    figure have been fine-tuned with $K=10$ shots and following the best
    fine-tuning strategy found in \cref{sec:diff_experiments} for each dataset.}
    \label{fig:diff_quantitative_res}
\end{figure}

Finally, we also provide a qualitative assessment of the performance of
FSDiffusionDet on DOTA, DIOR, Pascal VOC and MS COCO in
\cref{fig:diff_quantitative_res}. This figure presents novel class detection
results on 5 images from each dataset. It is certainly stronger than all
previously studied methods (see \cref{fig:xqsa_novel_detection}). However, it is
not perfect. Some objects are misclassified (see the third row in MS COCO
column), some are not detected (DOTA, first row), and there are still false
detections (Pascal VOC, 4th row). Nonetheless, these qualitative results are
much better than our other approaches and can be considered for actual
industrial applications. It strengthens the need for more elaborate fine-tuning
strategies and insight into how to design them without too much trial and error.
FSDiffusionDet achieves impressive FSOD results, especially on DOTA and DIOR,
but this performance was achieved through expensive exploration. It would be of
great help to know in advance how a strategy will perform on a given dataset. We
started some investigation in this direction with the design of a compatibility
score between base and novel classes, taking into account both the overall shift
and discrepancies in the class structures. Lastly, FSDiffusionDet's results
are strong enough to tackle more complex scenarios such as Few-Shot Cross-Domain
Adaptation. This will be explored in \cref{sec:cda_analysis} and should be
continued in future work as well.

\section{Application to Cross-Domain FSOD}
\label{sec:cda_analysis}
\vspace{-1em}
Given the impressive performance of FSDiffusionDet on DOTA and DIOR, it seems
tempting to try more difficult setups. Up to now, the methods studied in this
project were barely reaching a satisfactory point from an applicative
perspective. With FSDiffusionDet, we are past that, and can now consider the
Cross-Domain setting. Cross-Domain is especially important for COSE, given the
prohibited access to test-time images. The ability to adapt to new domains would
be an extremely valuable property for a surveillance system such as CAMELEON. Of
course, the domain change would be limited in COSE's applications as the only
change between two missions would be the general aspect of the image (\ie weather,
GSD, luminosity, etc.). However, the images will always be aerial taken pointing
nadir. 

In this section, we tackle the challenging Cross-Domain Few-Shot Object
Detection (CD-FSOD) task which is barely untouched in the literature. To this
end, we focus on two distinct scenarios, one introduced by
\cite{lee2022rethinking} with a first training on MS COCO and one specifically
designed for COSE's applications where both the source and target domains are
aerial datasets. For both scenarios, we first present the dataset used as source
and target domains and the experimental setup. Then we provide some experimental
results with the FSDiffusionDet baseline. These results are preliminary and
promising, further experimentation in this direction is required to better
understand this task and further improve FSDiffusionDet in this context.
Therefore, we end this section with a summary of the future work that is
planned.

\subsection{MS COCO $\to$ Anything}

First, we study a general Cross-Domain (CD) setting introduced in the literature
by \cite{lee2022rethinking}. It consists of training first on MS COCO and then
fine-tuning on another dataset with a restricted number of shots. Unlike in the
FSOD setting, there is no separation between base and novel classes in CD, all
classes of the target domain are considered novel. The benchmark introduced by
\cite{lee2022rethinking} contains a list of 10 datasets (VisDrone2019
\cite{pengfei2021visdrone}, DeepFruits, iWildCam \cite{beery2021iwildcam},
SIXray \cite{miao2019sixray}, Fashionpedia \cite{jia2020fashionpedia},
Oktoberfest \cite{tum2019oktoberfest}, LogoDet-3K \cite{wang2022logodet},
CrowdHuman \cite{shao2018crowdhuman}, ClipArt \cite{inoue2018cross}, KITTI
\cite{Geiger2012CVPR}).

\subsubsection{Cross-Domain Scenarios}
The pool of datasets proposed by \cite{lee2022rethinking} has a large variety of
images, therefore it constitutes a relevant benchmark for CD-FSOD. However, it
does not contain any aerial dataset. VisDrone is an aerial image dataset but
differs greatly from DOTA or DIOR as its images are taken from a much lower
altitude and contain perspective. In addition, 10 datasets make the experiments
expensive to run. Thus, we propose a lighter benchmark using VisDrone2019,
DeepFruits, SixRay, DOTA, DIOR and ClipArt. We emphasize that dealing with that
many datasets is quite challenging as almost every dataset has its own
annotation format and data structure. When \cite{lee2022rethinking} proposed
this benchmark, the authors only provided the list of datasets, without any
information about their preparation and split, which makes their experiments
very hard to reproduce. On the contrary, we propose for the convenience of
future research on CD-FSOD a prepared version of this
"meta-dataset”\footnote{\href{https://gitlab.sorbonne-paris-nord.fr/labcom-iriser/public/super_pycocotools}{Link
to the Meta-Dataset and Python API package}} under the same format (the MS
COCO format, which is relatively common in the OD community). In addition to the
prepared meta-dataset, we also
extend the popular Python package \textit{pycocotools} to help
load and explore the datasets. 

As mentioned above, we study here 6 cross-domain scenarios with a common base
training on MS-COCO. Every scenario has a different target domain represented
with one of the following datasets. \cref{fig:cd_fsod_coco_scenarios} provides
some image examples (without annotations) for VisDrone2019, DeepFruits, SixRay,
and ClipArt. We refer the reader to \cref{fig:od_dataset_aerial} for a
presentation of DOTA and DIOR. For convenience, we denote these 6 scenarios as
COCO $\to$ X scenarios. 

\begin{figure}
    \centering
    \includegraphics[width=\textwidth]{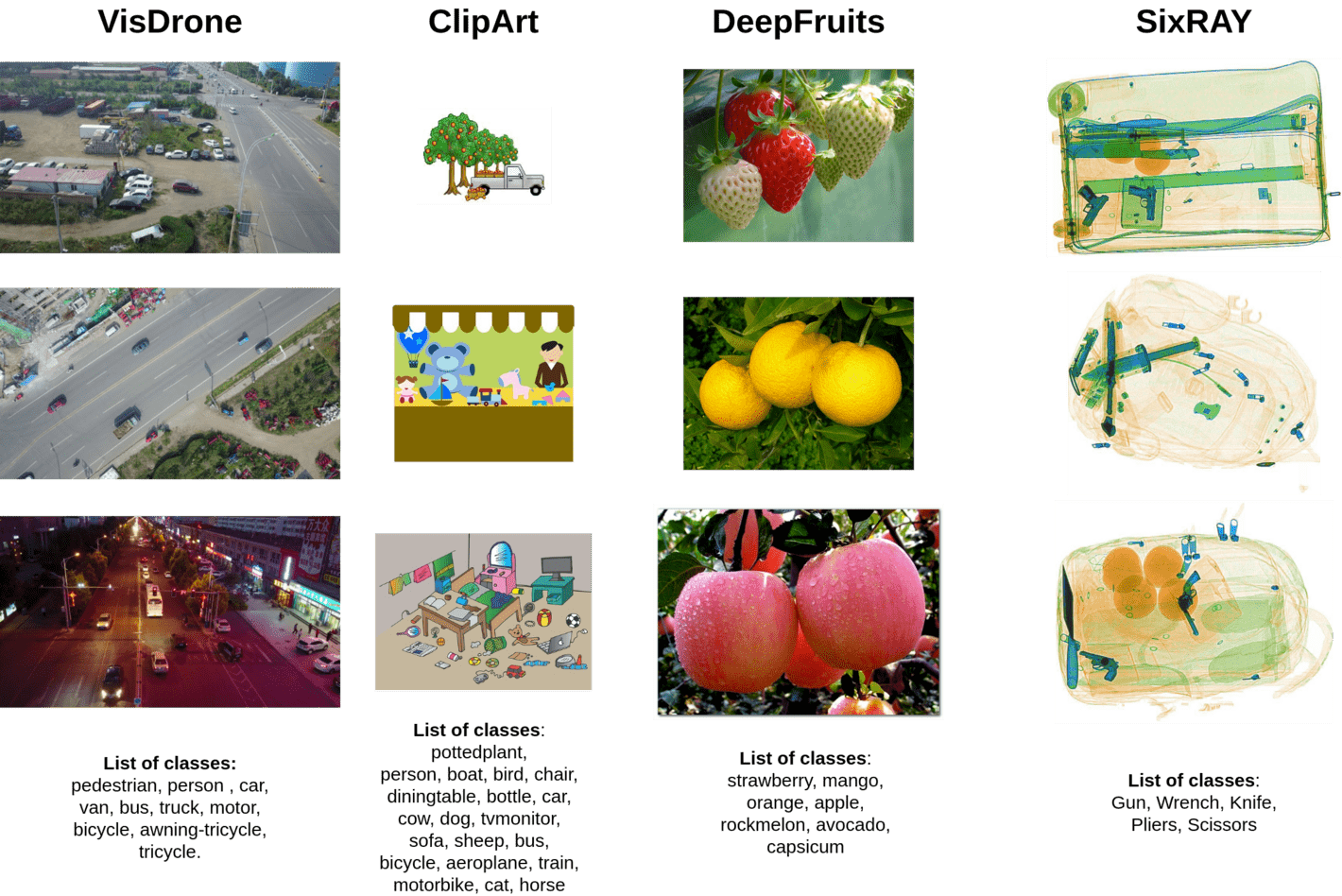}
    \caption[Presentation of Cross-Domain datasets]{Presentation of the datasets
     used in our Cross-Domain benchmark: VisDrone2019, DeepFruits, SixRay,
     and ClipArt.}
    \label{fig:cd_fsod_coco_scenarios}
\end{figure}

\subsubsection{Experimental Results}
We only experimented with the CD scenario with the baseline of FSDiffusionDet.
Specifically, only the detection head is fine-tuned, with 500 proposals. The
weights of the model are initialized using the pre-trained weights available on
the DiffusionDet repository, except for the last layer of the classification
branch which is initialized randomly with the right number of outputs. While
this is certainly not optimal for every scenario, this gives a strong baseline
to compare with in future experiments. 

The results of our experimentation with the cross-domain scenarios can be found
in \cref{tab:cd_fsod_results} and \cref{fig:cd_coco_performance}. As for the
regular FS experiments, a clear pattern is visible as the number of shots
increases. Obviously, the more shots, the better the detection. However, this
pattern differs from one scenario to another. For instance, COCO $\to$ ClipArt
scenario sees very little improvements as the number of shots increases unlike
COCO $\to$ DOTA and COCO $\to$ SIXRay. It is also noteworthy to observe the
different behaviors between DOTA and DIOR. Even if DIOR is an "easier" dataset
than DOTA (in the sense that higher performance is achieved on DIOR in a regular
detection setting), there is a larger difference in the COCO $\to$ X
cross-domain scenario. Relatively low performance is observed for ClipArt and
VisDrone, this is probably due to differences in data preparation compared with
\cite{lee2022rethinking}. As the author did not provide any information about all
the datasets' splits and preparation, we can only guess what they did. For ClipArt,
it is slightly different as they leveraged a GAN-augmented version of the
dataset which might sensibly boost the detection performance.  
Finally, for each scenario 5 distinct training were done with varying seeds to
check the consistency of our results. \cref{tab:cd_fsod_results} gives the
average over the 5 runs and a 95\% confidence interval. A limited variance
between different runs is observed, this means that FSDiffusionDet is not very
sensitive to the examples chosen in the support set. This is a crucial property
as some few-shot methods depend a lot on the choice of the support set. Of
course, most of our experiments should be repeated the same way to strengthen
the results, but this quickly becomes expensive in terms of computing resources.  

These preliminary results are promising, FSDiffusionDet achieves satisfactory
performance with only a few-annotated examples on various datasets. These
datasets are constituted of various kinds of images, therefore it demonstrates
well the adaptation capabilities of FSDiffusionDet. Now for COSE's application,
these results on aerial datasets are particularly encouraging. FSDiffusionDet
achieves impressive performance with only a base training on MS COCO and few
examples of either DOTA or DIOR. Thus, this model could be rapidly fine-tuned
for a specific mission by the forces without declassifying any image. Of course,
this is only a baseline and FSDiffusionDet can surely be improved further. In
addition, this scenario starts with a base training on natural images, which is
probably not optimal, instead, we could leverage an aerial dataset as a source
model as well. Incidentally, this will be the subject of the next section. 

\begin{table}[]
    \centering
    \resizebox{0.9\textwidth}{!}{%
    \begin{tabular}{@{\hspace{2mm}}ccccccc@{\hspace{2mm}}}
    \toprule[1pt]
    \multicolumn{1}{l}{\textbf{$K$ Shots}} & \textbf{DIOR} & \textbf{DOTA} & \textbf{DeepFruits} & \textbf{SIXRay} & \textbf{ClipArt} & \textbf{VisDrone}\\ \midrule
    \textbf{1}                           & 11.10 $\pm$ 0.32      & 4.03  $\pm$ 0.26        & 38.47 $\pm$ 1.42         & 4.80  $\pm$ 0.87      &  2.09$\pm$  0.19 & 2.83$\pm$ 0.17       \\
    \textbf{5}                           & 30.42 $\pm$ 0.69      & 14.45 $\pm$ 0.43        & 55.58 $\pm$ 1.36         & 13.25  $\pm$1.14       & 5.26 $\pm$ 0.15 & 5.74$\pm$ 0.22       \\
    \textbf{10}                          & 38.73 $\pm$ 0.65      & 25.02 $\pm$ 0.65        & 68.37 $\pm$ 2.01         & 21.26  $\pm$1.33       & 5.69 $\pm$ 0.10 & 7.50$\pm$ 0.10       \\
    \textbf{20}                          & 48.23 $\pm$ 0.33      & 33.31 $\pm$ 0.46        & 73.95 $\pm$ 0.53         & 30.06  $\pm$1.09       & 6.10 $\pm$ 0.22 & 9.14$\pm$ 0.35       \\
    \textbf{50}                          & 56.97 $\pm$ 0.60      & 43.23 $\pm$ 0.68        & 76.65 $\pm$ 0.78         & 41.93  $\pm$1.02       & 6.44 $\pm$ 0.16 & 11.47$\pm$ 0.27       \\ \bottomrule[1pt]
    \end{tabular}%
    } \caption[COCO $\to$ X cross-domain FSDiffusionDet performance]{Cross-domain performance
    results on 6 scenarios COCO $\to$ DIOR / DOTA / DeepFruits / SIXRay / ClipArt / VisDrone.
    Results are given for different numbers of shots. Experiments are repeated 5
    times for each scenario and shot setting. The average mAP$_{0.5}$ is reported with a 95\% confidence interval.}
    \label{tab:cd_fsod_results}
    \end{table}

\begin{figure}
    \centering
    \includegraphics[width=0.7\textwidth]{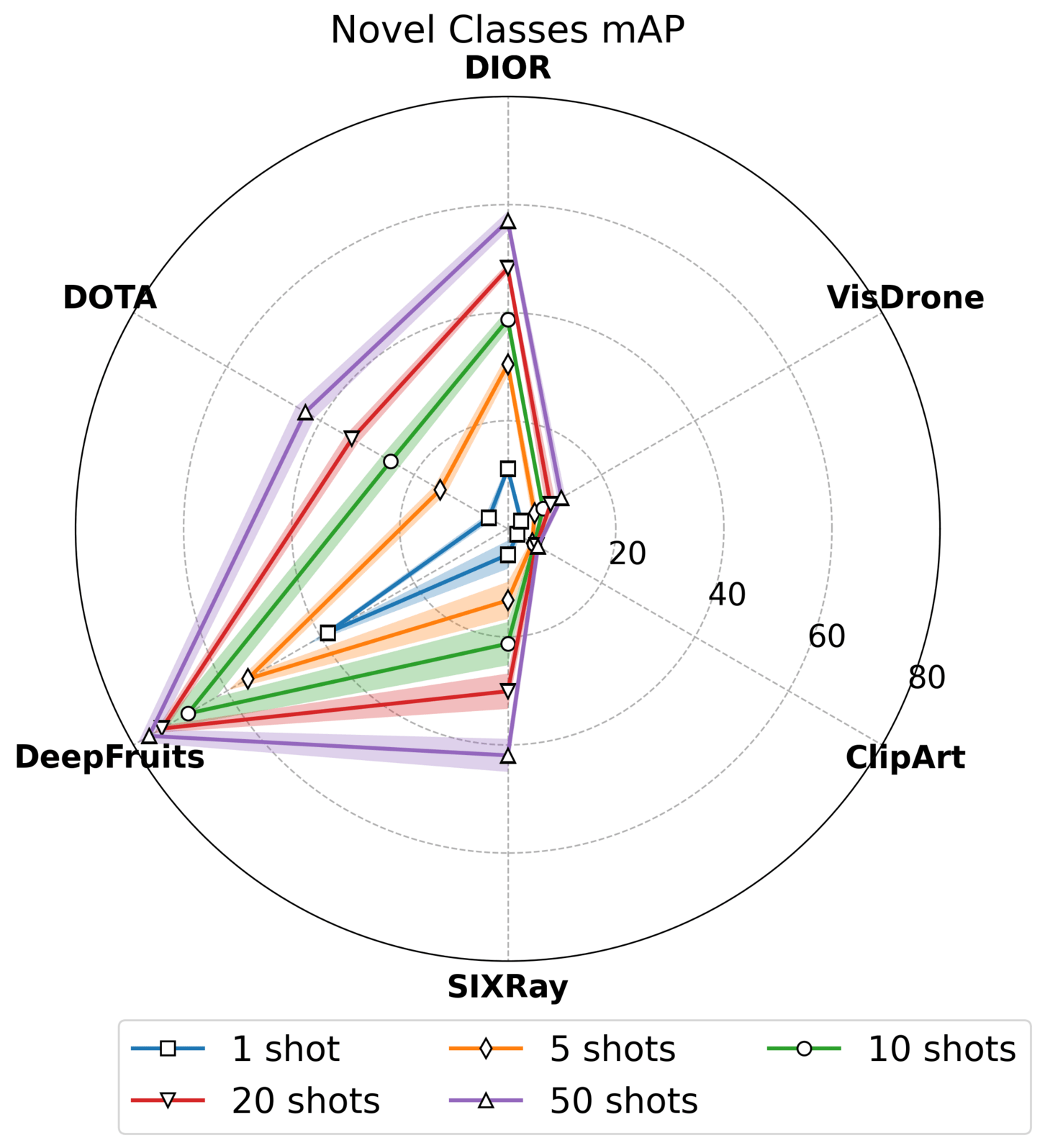}
    \caption[COCO $\to$ X cross-domain performance spider plot]{Cross-domain
    performance of FSDiffusionDet on multiple scenarios with MS COCO as the
    source domain. Light areas denote the 95\% confidence interval. Concentric
    circles indicate mAP$_{0.5}$ levels.}
    \label{fig:cd_coco_performance}
\end{figure}

\subsection{Aerial Cross-Domain}

Besides experimenting with COCO $\to$ X scenario, we propose another setting
specifically designed for aerial images and COSE's application. The idea is to
leverage two distinct aerial datasets as source and target domains. In
particular, we use DOTA and DIOR to get two scenarios: DOTA $\to$ DIOR and DIOR
$\to$ DOTA. Of course, it would be interesting to leverage other kinds of
datasets as well (\eg xView, VisDrone, etc.), especially as DOTA and DIOR are
very similar (mostly overhead urban images and shared classes). Yet, this gives
insights into how FSDiffusionDet behaves in fairly simple cross-domain
scenarios. We report experiments with these two scenarios in
\cref{tab:cd_fsod_dior2dota,tab:cd_fsod_dota2dior}. The results are given for
multiple numbers of shots ranging from 1 to 50. In addition, we studied two
freezing strategies by fine-tuning only the detection head (\ie frozen backbone)
or the whole model (fully fine-tuned).

\begin{table}[]
    \centering
    \resizebox{0.75\textwidth}{!}{%
    \begin{tabular}{@{\hspace{2mm}}cccccccccc@{\hspace{2mm}}}
    \multicolumn{1}{l}{} & \multicolumn{8}{c}{{\ul \textbf{DIOR $\to$ DOTA}}}                                                           \\
    \multicolumn{1}{l}{} & \multicolumn{4}{c}{\textbf{Backbone frozen}}              & \multicolumn{4}{c}{\textbf{Fully fine-tuned}}                \\ \midrule[1pt]
    \textbf{$K$ shots}     & \textbf{All} & \textbf{S} & \textbf{M} & \textbf{L}& & \textbf{All} & \textbf{S} & \textbf{M} & \textbf{L} \\ \midrule
    \textbf{1}           & \textbf{5.41 }        & 2.72       & 6.28       & 4.51      & & 5.09         & 3.08       & 6.72       & 4.07       \\
    \textbf{5}           & \textbf{25.88}        & 16.99      & 31.47      & 22.50     & & 24.90        & 15.85      & 29.67      & 22.27      \\
    \textbf{10}          & 31.99                 & 17.64      & 36.90      & 31.23     & & \textbf{33.30}        & 15.97      & 37.13      & 32.45      \\
    \textbf{20}          & 38.77                 & 21.68      & 46.49      & 34.79     & & \textbf{41.30}        & 21.97      & 45.90      & 41.08      \\
    \textbf{50}          & 44.07                 & 29.22      & 52.66      & 41.00     & & \textbf{49.22}        & 29.41      & 55.94      & 52.82      \\ \bottomrule[1pt]
    \end{tabular}%
    } \caption[DIOR to DOTA Cross-Domain results]{FSDiffusionDet Cross-domain
    results on the scenario DIOR $\to$ DOTA. Two settings are compared: with the
    backbone frozen (left) and the backbone fully fine-tuned (right). Bold
    values denote the best setting for overall performance on novel classes.
    Performance is reported with mAP$_{0.5}$ values.}
    \label{tab:cd_fsod_dior2dota}
    \end{table}

The key takeaway from this experiment is that higher performance is achieved in
the aerial cross-domain scenarios than with the COCO $\to$ X scenario. It seems
more profitable to perform base training on a source domain that is similar to
the target domain. Of course, compared to the FS performance on DOTA and DIOR,
lower quality is achieved in cross-domain scenarios. This is explained first
because the images from the two datasets differ, but also because the task is
now slightly more complex as all classes of the target dataset are novel. The
detection task becomes a 16-way $K$-shots problem in DIOR $\to$ DOTA scenario
for instance. In the regular FS setting studied throughout this project, only
three classes were selected as novel classes for DOTA, making the classification
much easier. Then, from \cref{tab:cd_fsod_dior2dota,tab:cd_fsod_dota2dior}, a
contradiction arises, in DIOR $\to$ DOTA scenario, the fully fine-tuned model
outperforms the model with the backbone frozen, which agrees with the
experiments from \cref{sec:diff_experiments}. However, in the DOTA $\to$ DIOR
scenario, the inverse is observed. This clearly shows that the freezing sweet
spot depends on the source and target domains and cannot be set once and for
all. It works in the case of regular few-shot when the source and target domains
are identical. So, one could expect the same behavior in cross-domain scenarios
with similar source and target domains as it is with DOTA and DIOR. However, in
the regular few-shot setting, the number of classes in the source domain (\ie
the base classes) is always larger than the number of classes in the target
domain (\ie the novel classes). Here, DIOR has more classes than DOTA and this
difference may explain the opposite results between the two CD scenarios.
Specifically, fine-tuning the model entirely may be beneficial only when the
target domain contains fewer classes than the source domain. It could also be
caused by different class separations between the datasets. If classes are
easily differentiable in DIOR but not in DOTA, it might be difficult to transfer
from DIOR to DOTA. These are only conjectures, and they should be taken carefully
especially as complex interactions between source and target classes may also
cause such behavior. More experiments would be required to analyze and
understand this surprising result.

\begin{table}[]
    \centering
    \resizebox{0.75\textwidth}{!}{%
    \begin{tabular}{@{\hspace{2mm}}cccccccccc@{\hspace{2mm}}}
    \multicolumn{1}{l}{} & \multicolumn{8}{c}{{\ul \textbf{DOTA $\to$ DIOR}}}                                                           \\
    \multicolumn{1}{l}{} & \multicolumn{4}{c}{\textbf{Backbone frozen}}              & \multicolumn{4}{c}{\textbf{Fully fine-tuned}}                \\ \midrule[1pt]
    \textbf{$K$ shots}     & \textbf{All} & \textbf{S} & \textbf{M} & \textbf{L}& & \textbf{All} & \textbf{S} & \textbf{M} & \textbf{L} \\ \midrule
    \textbf{1}           & \textbf{20.18}        & 5.53       & 16.96      & 23.43    &  & 9.40         & 3.86       & 9.15       & 8.95       \\
    \textbf{5}           & \textbf{34.43}        & 9.99       & 31.12      & 47.03    &  & 29.57        & 8.70       & 25.80      & 35.76      \\
    \textbf{10}          & \textbf{41.48}        & 12.85      & 36.62      & 53.85    &  & 38.44        & 10.50      & 32.58      & 47.27      \\
    \textbf{20}          & \textbf{49.00}        & 16.39      & 40.23      & 62.79    &  & 45.36        & 15.29      & 36.51      & 55.05      \\
    \textbf{50}          & \textbf{54.07}        & 18.70      & 43.83      & 67.58    &  & 53.51        & 19.49      & 41.27      & 63.04      \\ \bottomrule[1pt]
    \end{tabular}%
    } \caption[DOTA to DIOR Cross-Domain results]{FSDiffusionDet Cross-domain
    results on the scenario DOTA $\to$ DIOR. Two settings are compared: with the
    backbone frozen (left) and the backbone fully fine-tuned (right). Bold
    values denote the best setting for overall performance on novel classes.
    Performance is reported with mAP$_{0.5}$ values.}
    \label{tab:cd_fsod_dota2dior}
    \end{table}

\subsection{Cross-Domain Perspectives}
The previous sections have been devoted to cross-domain experiments. These are
preliminary but interesting results. They give insight into how difficult this
setup is and how fine-tuning strategies can perform. However, plenty of
experiments are still necessary. We detail here some of the most relevant
perspectives for future CD-FSOD research that we briefly hinted in the previous sections:

\begin{enumerate}[nolistsep]
    \item \textbf{Comparison with other FSOD methods}: it would be interesting
    to compare with other existing FSOD methods, in particular, with
    attention-based techniques that we studied in depth in \cref{chap:aaf}. In
    addition, studying other fine-tuning approaches is required to validate the
    results found in our experiments.
    \item \textbf{Transductive inference}: Even though our naive transductive
    detection did not outperform the fine-tuning strategy in FSOD, it could help
    in cross-domain scenarios. Indeed, leveraging query images during inference
    can reduce the discrepancies between source and target domains and improve
    performance. This has been empirically shown for the classification task,
    but it remains to be adapted for detection. 
    \item \textbf{Source-target domain compatibility score}:
    \cite{lee2022rethinking} proposes to choose the fine-tuning sweet spot
    according to the distance between the source and target domains.
    Specifically, more plasticity is required when domains are farther apart.
    They compute such distance as the recall of a pre-trained detection model on
    MS COCO, applied to the target dataset in a class-agnostic manner. This
    could be generalized to any source and target domains with a detection model
    trained on the source domain. However, we would like to emphasize that this
    should not be called a distance measure between domains as it does not
    satisfy the symmetry property. Instead, it is a \textit{compatibility
    measure} as it evaluates how beneficial the source domain is for the
    adaptation to the target domain.   
    Our cross-domain scenarios on aerial images clearly demonstrate this, as we
    obtain contradictory conclusions for DOTA $\to$ DIOR and DIOR $\to$ DOTA
    scenarios. More plasticity is required for DIOR $\to$ DOTA than for DOTA
    $\to$ DIOR, hence, the \textit{compatibility measure} cannot be the same for
    these two scenarios. In addition, this distance is highly influenced by the
    detector chosen in the first place, in particular, some models are known to
    output a lot of duplicate boxes which often boost the recall significantly
    (see such an analysis in \cite{jena2023beyond}). It would be helpful to come
    up with a properly defined compatibility measure for a given scenario that
    does not rely on a detection model and gives coherent hints to obtain an
    optimal fine-tuning strategy. This measure should also be able to assess the
    compatibility of the base and novel class set in the regular FSOD setting,
    as a special case of the cross-domain scenario (source and target domains
    are identical but the classes change). We are currently working on such a
    compatibility score based on an overall discrepancy measure between source
    and target domains and a source-target classes compatibility score.   
\end{enumerate}

\section{Conclusion}
\vspace{-1em}
In this chapter, we have presented thoroughly the basic principle of diffusion
models and how they can be leveraged for detection. Then, we proposed a
simple fine-tuning strategy to apply DiffusionDet in the few-shot setting.
FSDiffusionDet achieves sensibly higher performance than all previous methods
studied in this PhD on aerial images. To understand why, we conducted extensive
experimental studies on crucial design choices of our strategy. It highlighted a
strong but complex connection between the plasticity of the model and the detection
performance. Finally, we applied FSDiffusionDet in several cross-domain
scenarios and observed promising results. Again, the plasticity has a great
influence on the performance and more experiments must be conducted to
understand this relation completely. A possible direction would be to design a
compatibility measure between domains and between sets of classes to determine
the optimal amount of plasticity required for a given scenario.

\addtocontents{toc}{\protect\pagebreak}
\part{Rethinking Intersection Over Union}%
\label{part:iou}

\makeatletter
\let\savedchap\@makechapterhead
\def\@makechapterhead{\vspace*{-2.8cm}\savedchap}
\chapter[Scale-Adaptative Intersection Over Union]{\resizebox{\linewidth}{!}{Scale-Adaptative Intersection Over Union}}
\let\@makechapterhead\savedchap
\makeatletter
\label{chap:siou_metric}

\chapabstract{Intersection over Union (IoU) is not an optimal box similarity measure
for evaluating and training object detectors. For evaluation, it is too strict
with small objects and does not align well with human perception. For training,
it provides a poor balance between small and large objects to the detriment of
small ones. We propose Scale-adaptative Intersection over Union (SIoU), a
parametric alternative that solves the shortcomings of IoU. We provide
empirical and theoretical arguments for the superiority of SIoU through in-depth
analysis of various criteria.} 
\textit{
\begin{itemize}[noitemsep]
    \item[\faPaperPlaneO] P. Le Jeune and A. Mokraoui, "Rethinking Intersection Over Union for Small Object Detection in Few-Shot Regime", Submitted at the International Conference on Computer Vision 2023 (ICCV).
    \item[\faFileTextO] P. Le Jeune and A. Mokraoui, "Extension de l'\textit{Intersection over Union} pour améliorer la détection d'objets de petite taille en régime d'apprentissage few-shot", GRETSI 2023, XXIXème Colloque Francophone de Traitement du Signal et des Images, Grenoble, France.
\end{itemize}
}
\vspace{0.5em}
\PartialToCCh

Intersection over Union (IoU) is a box similarity criterion, it measures how
well two bounding boxes overlap each other. We already defined it in
\cref{chap:od}, but in this chapter we explore its properties thoroughly and
highlight some of its flaws when employed as a loss function or as a cornerstone
of the evaluation process of detection models. These flaws mainly concern small
objects for which IoU is too strict. Therefore, it is particularly relevant to
tackle these issues for aerial images and COSE's applications. To address these
weaknesses, we propose Scale-Adaptive Intersection over Union (SIoU), a
parameterizable criterion that can be set to favor small objects as needed. We
start by defining and analyzing the IoU and its variants. Then, we propose our
novel criterion SIoU and its properties. \cref{sec:criterion_analysis} presents
an original empirical and theoretical study of several box similarity criteria
and argues for the superiority of SIoU. Finally, we conduct a user study and
experimental analysis to further consolidate the advantages of SIoU over IoU. 

\section{Analysis of Intersection over Union}
\vspace{-1em}
In this section, we first review the definition of IoU and present some of its
variants that are available in the literature. Then, we analyze why IoU is not
optimal for small objects.

\subsection{Intersection over Union and its Variants}
To begin, let us review the definition of existing criteria for box similarity.
Originally, the IoU is defined as the
intersection area of two sets divided by the area of their union: 
\begin{equation}
    \text{IoU}(A, B) = \frac{|A \cap B|}{|A \cup B|},
\end{equation}

\noindent
where $A$ and $B$ are two sets. Even if there are plenty of applications where
IoU is useful (\eg in statistics where IoU is better known as the Jaccard
index), we are mostly interested here in its application in computer vision. In
this case, $A$ and $B$ are sets of pixels, and the IoU measures how close they
are. When $A$ and $B$ are rectangular boxes, IoU can be computed easily with
simple operations on box coordinates (see \cref{eq:iou_bbox}). This explains why
IoU is such a widespread criterion for object detection. It is used as a loss
function ($\mathcal{L}_{\text{reg}} = 1 - \text{IoU}$) by several well
established detection frameworks (\eg \cite{yu2016unitbox, tian2019fcos}). IoU
is also involved in the process of example selection during training of most
detection methods, \ie all the ones inspired either by Faster R-CNN
\cite{ren2015faster} or YOLO \cite{redmon2016you}. In these frameworks,
regression loss is computed from the coordinates of proposed boxes and ground
truth. Not all pairs of proposals and ground truth are kept for the computation.
Only proposals with a sufficient IoU with a ground truth box are selected.
Finally, IoU is also used at the heart of the evaluation process. A proposed box
is considered a positive detection if it meets two conditions: 1) an IoU greater
than a given threshold with a ground truth box, and 2) the same label as this
ground truth (see \cref{sec:od_evaluation}). 

Several attempts were made to improve IoU but existing works mostly focus on the
regression loss part, disregarding the other IoU uses in the detection
task. First, \cite{rezatofighi2019generalized} proposed a generalized version of
IoU which yields negative values when boxes do not overlap:
\begin{equation}
    \text{GIoU}(A,B) = \text{IoU}(A,B) - \frac{|C \backslash (A \cup B)|}{|C|},
\end{equation}

\noindent
where $C$ is the convex hull around $A$ and $B$. This criterion is employed as a
loss function by several detection frameworks \cite{tian2019fcos,
bochkovskiy2020yolov4, zhang2020bridging}. It is sometimes also combined with
other regression loss as in \cite{li2020generalized, carion2020end}, which both
combine it with an L1 regression on box coordinates. Combining IoU loss with
other regression terms was also proposed by \cite{zheng2020distance}. They
introduce two losses Distance-IoU (DIoU) and Complete-IoU which respectively add
an L2 regression term and an aspect ratio penalty to the IoU loss. Recently,
$\alpha$-IoU \cite{he2021alpha} extends DIoU \cite{zheng2020distance} by
proposing a family of losses following the same structure as DIoU with the IoU
term raised to the power $\alpha$:
\begin{equation}
    \alpha\text{-IoU}(A, B) = \text{IoU}(A, B)^{\alpha}.
\end{equation}
Balanced IoU (BIoU) also extends upon DIoU by
measuring shifts between the corners of the boxes instead of their centers.     
Alternatively, Bounded IoU \cite{tychsen2018improving} computes an IoU upper
bound between a proposal and a ground truth. Other approaches, such as Scale
Balanced Loss \cite{sun2020scale}, try to design distance-based loss functions
which share properties with IoU, especially its scale-invariance.   

All these IoU variants are proposed to improve the regression part of the
models. However, IoU is involved in other parts of the framework including
example selection, Non-Maximal Suppression, and evaluation. A recent user
study \cite{strafforello2022humans} indicates that IoU does not completely align
with human perception. Humans have strong positional and size preferences based on
conceptual information contained in the boxes. It suggests that IoU is not
an optimal choice either for example selection or for evaluation as it
will lead to detections that do not satisfy human users.

\subsection{Inadequation of IoU for Small Objects and Few-Shot Regime}
Object detection is a fundamental task in industry and has applications in many
domains such as medical imaging, agriculture, or autonomous driving. However,
it is often impracticable or too expensive to build sufficiently large annotated
datasets to train detection models. It is therefore crucial to improve
data-efficient approaches and particularly Few-Shot Object Detection (FSOD)
methods. However, the limited number of examples provides poor supervision and
prevents the model to learn accurate localization, which is especially
problematic for small objects. This difficulty greatly intensifies in the
few-shot regime as shown by \cref{chap:aerial_diff}. Designing FSOD methods
specifically for the detection of small objects partially solves this issue (see
\cref{sec:xqsa_definition}), but is not enough. One of the reasons for the poor
FSOD performance on small objects is the extensive use of the IoU. Just as for
detection, most FSOD pipelines employ IoU as a regression loss
\cite{yu2016unitbox,tian2019fcos}; for example selection
\cite{ren2015faster,redmon2016you,liu2016ssd}; or as an evaluation criterion,
but IoU is not an optimal choice when dealing with small objects.

\begin{figure}
    \centering
    \begin{subfigure}[t]{0.46\textwidth}
        \vskip 0pt
        \centering
        \includegraphics[trim=0 5 0 5, clip,width=\textwidth]{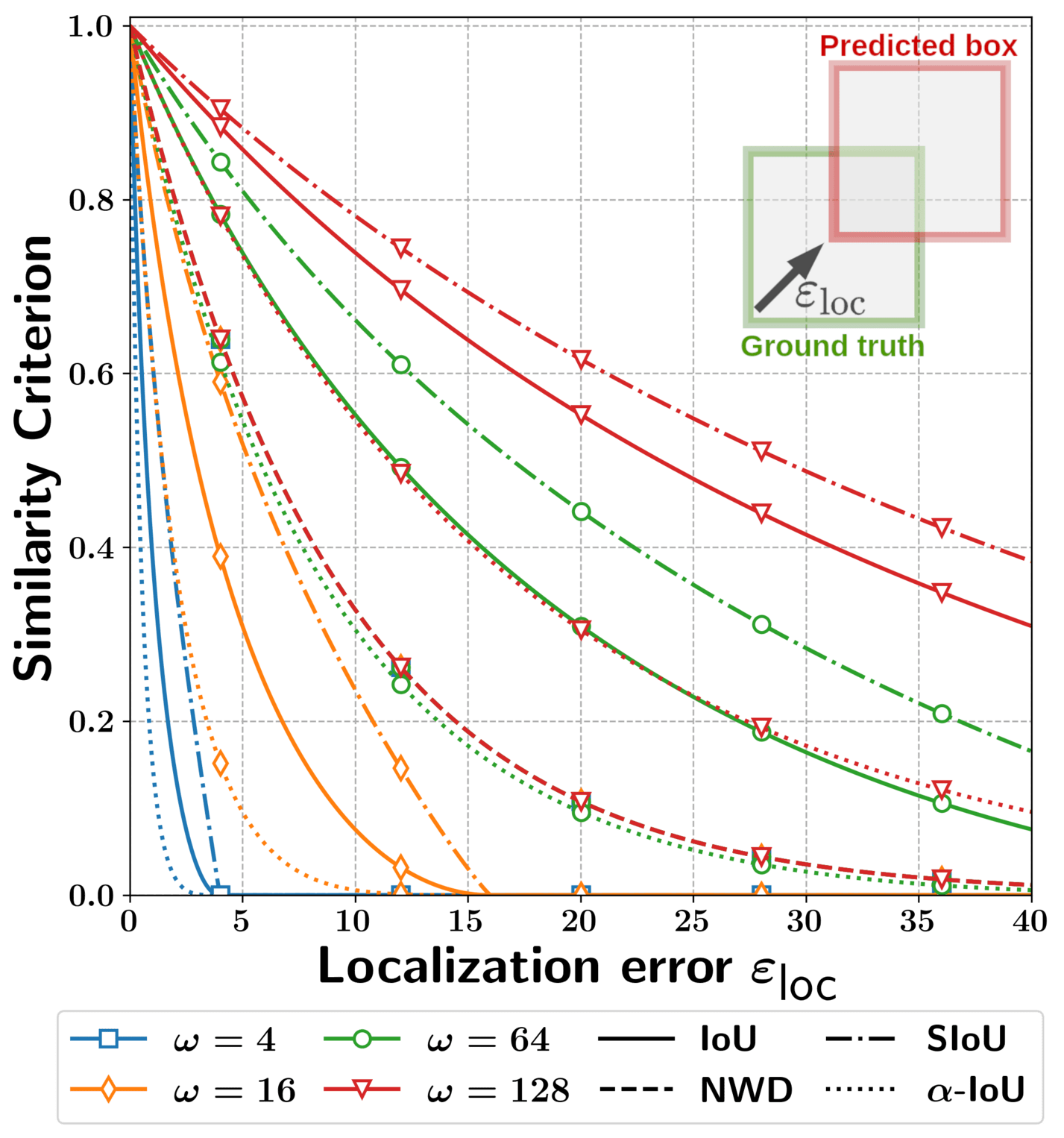}
        \label{fig:pixel_shift}
    \end{subfigure}
    \begin{subfigure}[t]{0.48\textwidth}
        \vskip 0pt
        \centering
        \includegraphics[trim=0 0 0 7, clip,width=0.99\textwidth]{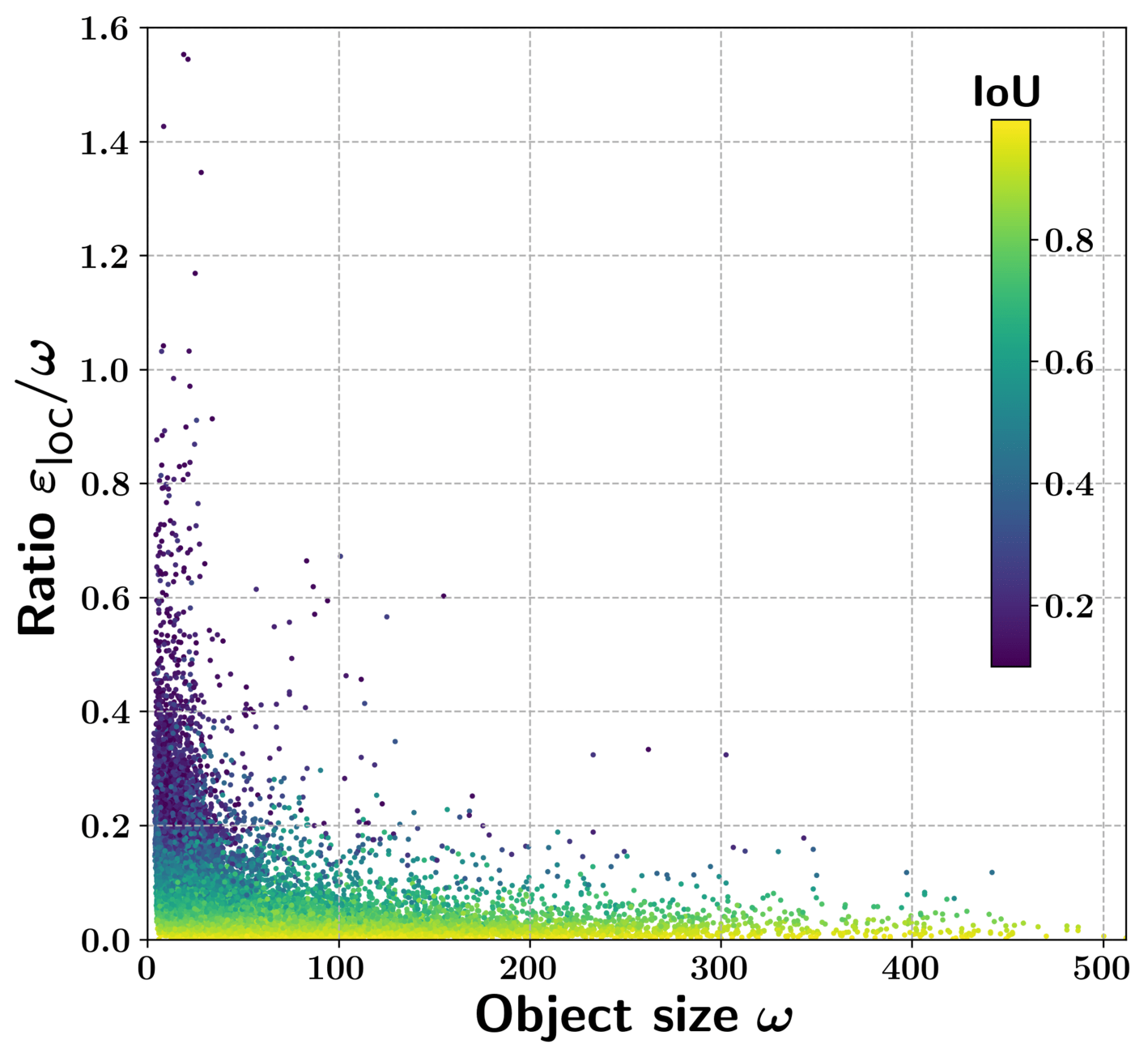}
        \vskip 6pt
        \label{fig:model_inaccuracy}
    \end{subfigure}
    \vspace{-4mm}
    \caption[IoU small object inadequacy]{\textbf{(Left)} Evolution of IoU, NWD
    \cite{wang2021nwd}, the proposed SIoU and $\alpha$-IoU \cite{he2021alpha}
    when a box is shifted from the ground truth box by
    $\varepsilon_{\text{loc}}$ pixels, for various box width $\omega \in \{4,
    16, 64, 128\}$ (boxes are squares). \textbf{(Right)} Ratio between pixel
    localization error $\varepsilon_{\text{loc}}$ and object size $\omega$ for a
    trained detection model on DOTA dataset. Each point represents the
    localization error of one object in DOTA test set.}
    \label{fig:loc_error_evo}
\end{figure}

IoU has a remarkable property: scale invariance. It means that scaling all
coordinates of two bounding boxes by the same amount will not change their IoU.
At first glance, this seems a desirable property as all objects will be treated
identically no matter their size. In practice, it has a fundamental drawback:
small boxes are prone to large IoU changes from only small position or size
modifications. To clarify, let us consider a simple example. Two
square boxes of width $\omega$ are shifted diagonally by
$\varepsilon_{\text{loc}}$ pixels. In this setup, a 1-pixel shift leads to a
larger decrease in IoU when boxes are smaller. This comes from the scale
invariance property, IoU stays constant as the ratio
$\frac{\varepsilon_{\text{loc}}}{\omega}$ remains fixed. 
However, this ratio is not constant for trained detection models, it increases
as objects get smaller (see \cref{fig:loc_error_evo}, right), leading to lower
IoU values for smaller objects. Hence, small objects are much more likely to
fall under the IoU thresholds which decide if a box is a true or false
detection, even though being satisfactory from a human perspective (see the user
study in \cref{sec:user_study}). In addition, \cref{sec:shift_analysis,sec:inaccuracy}
explore the resilience of various criteria to localization inaccuracies and
confirm that IoU is not an optimal box similarity measure. 

Only a handful of works question the adequation of IoU for object detection.
Among those, \cite{rezatofighi2019generalized} proposed a generalization of IoU
when boxes do not overlap, \cite{wang2021nwd} introduced a novel loss function
to target small objects, and \cite{du2021scale} proposed a Scale-Sensitive IoU
which extended CIoU with an area regulatory factor. In addition,
\cite{strafforello2022humans} showed that human perception and IoU are not fully
aligned. This lack of interest in new criterion design is explained by the great
detection performance in the regular setting (\ie natural images with sufficient
annotations). In the few-shot regime, and when targets are small, the flaws of
IoU become critical. Therefore,         
we revisit IoU to improve FSOD methods and focus on aerial images which mostly
contain small objects. We propose Scale-adaptive Intersection over Union (SIoU),
a novel criterion that can replace IoU for training and evaluating detection
models. However, for training we mostly aim at few-shot detection models as small
objects are particularly difficult for them. To demonstrate the superiority of
the proposed SIoU, \cref{sec:criterion_analysis} compares it with various
existing criteria. This section analyzes criteria distributions when
exposed to randomly shifted boxes. To our knowledge, this is the first attempt
to empirically and theoretically study the distributions of these criteria. The
conclusions of this analysis are then compared with human perception through a
user study which shows that SIoU aligns better with human appraisal than IoU
(see \cref{sec:user_study}). The comparison of these criteria also highlights that
SIoU as a loss function can guide training towards small objects better than
other criteria and in a more controlled fashion. SIoU loss can be tuned to
improve the detection of small objects just as it can be tuned to align with
human perception. Finally, these analyses are confirmed by extensive experiments
on both aerial images (DOTA and DIOR datasets) and natural images (Pascal VOC
and COCO datasets).

\section{Scale-Adaptive Intersection over Union}
\label{sec:siou}
\vspace{-1em}
\subsection{Definition of the novel box similarity criterion}
Before introducing the proposed criterion, let us define two bounding boxes
$b_1=[x_1, y_1, w_1, h_1]^T$ and $b_2 = [x_2, y_2, w_2, h_2]^T$ (the prediction
box and ground truth respectively), following the box definition from
\cref{chap:od}. Similarly, the adjectives small, medium, and large keep the same
meaning as in previous chapters: the box $b_i$ is \textit{small} if
$\sqrt{w_ih_i} \leq 32$ pixels, \textit{medium} if $32 < \sqrt{w_ih_i} \leq 96$,
and \textit{large} if $\sqrt{w_ih_i} > 96$.

IoU is scale-invariant, hence if $\text{IoU}(b_1,b_2) =
u$, scaling all coordinates of both boxes by the same factor $k$ will produce
the same IoU: 
\begin{equation}
    \text{IoU}(b_1, b_2) =\text{IoU}(kb_1, kb_2) = u.
\end{equation} 

However, detection models are not scale-invariant, they do not localize equally
well small and large objects. \cref{fig:loc_error_evo} (right) clearly shows
that the ratio between the localization error ($\varepsilon_{\text{loc}} = \|
b_1 - b_2\|_1 $) and the object size ($\omega=\sqrt{w_2h_2}$) increases as the
object becomes smaller. This figure is made with a detection model trained on
DOTA with all annotations. Each point represents the ratio
$\frac{\varepsilon_{\text{loc}}}{\omega}$ for one object in the test set. If the
detection model was indeed scale-invariant, the ratio should not change
significantly with the object sizes. Hence, because of the scale-invariance
property, IoU scores are lower for small objects. It then has several consequences:

\begin{enumerate}[nolistsep]
    \item Bounding boxes output by the model are not considered positive examples during evaluation.
    \item Bounding boxes are not selected as positive examples for loss
    computation, which biases the training towards larger objects.
    \item NMS does not filter duplicates of small boxes as their overlap is not high enough.
\end{enumerate}

A way to alleviate these issues is to relax the invariance property of the
IoU so it favors more small objects without penalizing large ones. To this end,
we propose a novel criterion called Scale-adaptive Intersection over Union
(SIoU): 
\begin{equation}
    \begin{aligned}
        \text{SIoU}(b_1, b_2) &= \text{IoU}(b_1, b_2) ^ p  \label{eq:p}\\
        \text{with} \quad\quad p &= 1 - \gamma \exp\left({-\frac{\sqrt{w_1h_1 + w_2h_2}}{\sqrt{2}\kappa}}\right),
    \end{aligned}
\end{equation}

\noindent
where $p$ is a function of the object sizes. Thus, the scores are rescaled
according to the object size. $\gamma \in ]-\infty, 1]$ and $\kappa > 0$ are two
parameters that control the strength and direction of the rescaling (hence, $p
\geq 0$). $\gamma$ governs the scaling for small objects while $\kappa$ controls
how fast the behavior of regular IoU is recovered for large objects.
\cref{fig:influence_gamma_kappa} (left) in \cref{sec:gamma_kappa} shows the
evolution of $p$ with object size for various $\gamma$ and $\kappa$. For
convenience, we will denote the average object size \ie, the average size of boxes
$b_1$ and $b_2$, by $\tau = \frac{w_1h_1 + w_2h_2}{2}$.

Of course, there are many valid choices for the exponent $p$. However, we want
to ensure some properties for SIoU, which translate into constraints for $p$:
\begin{itemize}[nolistsep]
    \item[-] SIoU should either be higher or lower than IoU when objects are
    small, but should remain finite, so $p(0) \in \mathbb{R}_+^*$.
    \item[-] For large objects, SIoU should behave like IoU, $\lim\limits_{\tau
    \to \infty}p(\tau) = 1$.
    \item[-] To prevent complete inversion of the order and smooth changes, p
    should be positive, continuous, and monotonic.   
\end{itemize}

Thus, an exponential response is a natural choice for the design of $p$.
Similar forms could be achieved with hyperbolic functions. For instance,
$p(\tau) = 1 - \frac{\gamma}{1+\kappa\tau}$ would be a sensible alternative. An
inconvenient of these designs is the possibility to only focus on either small
or large objects. This is mainly due to the monotonicity of $p$. It can be
relaxed to unlock the possibility of targeting objects of a specific size, for
instance, with a bell-shaped exponent \eg $p(\tau) = 1 - \gamma
\exp(-\kappa(\tau - \tau_{\text{target}})^2)$. Where $\kappa$ could be understood
as a bandwidth parameter around objects of size $\tau_{\text{target}}$. We did
not investigate the design of $p$, but experimenting with it would be relevant to
better understand the balance between small and large objects during training.

\subsection{SIoU Properties}
This new criterion follows the same structure as $\alpha$-IoU
\cite{he2021alpha}, but differs greatly as it sets different powers for
different object sizes. SIoU provides a solution for small object detection in
the few-shot regime while $\alpha$-IoU only aims to improve general detection.
However, SIoU inherits a few properties from $\alpha$-IoU. 
\vspace{1em}
\begin{property}[SIoU Relaxation]
    \label{property:relaxation}
    Let $b_1$ and $b_2$ be two bounding boxes and introduce $\tau = \frac{w_1h_1 +
    w_2h_2}{2}$ their average area. SIoU preserves the behavior of IoU in
    certain cases such as:
   \begin{itemize}[nolistsep]
        \item[-] $\textup{IoU}(b_1, b_2) = 0 \Rightarrow \textup{SIoU}(b_1, b_2) =\textup{IoU}(b_1, b_2) = 0$
        \item[-] $\textup{IoU}(b_1, b_2) = 1 \Rightarrow\textup{SIoU}(b_1, b_2) =\textup{IoU}(b_1, b_2) = 1$
        \item[-] $\lim\limits_{\tau \to +\infty} \textup{SIoU}(b_1, b_2) =\textup{IoU}(b_1,b_2)$
        \item[-] $\lim\limits_{\kappa\to 0} \textup{SIoU}(b_1, b_2) =\textup{IoU}(b_1,b_2)$
   \end{itemize}
\end{property}

\cref{property:relaxation} shows that SIoU is sound: it equals IoU when boxes have no
intersection and when they perfectly overlap. Therefore, the associated loss
function (see \cref{property:l_g_reweight}) will take maximal values for boxes that
do not overlap and minimum values for identical boxes. In addition, SIoU
behaves similarly to IoU when dealing with large objects (\ie when $\tau \to \infty$).
When boxes are large, the power $p$ that rescales the IoU is close to 1. Hence,
this change of criterion only impacts small objects. However, when discussing
the properties of SIoU, the limit between small/medium/large objects is relative
to the choice of $\kappa$. If $\kappa \gg \sqrt{wh}$, even large objects will be
rescaled. On the contrary, when $\kappa \xrightarrow {}0$, all objects are
treated as large and are not rescaled. In practice, $\kappa$ and $\gamma$ are
chosen empirically, but \cref{sec:criterion_analysis} provides useful insights
for the choice of these parameters.  
\vspace{1em}
\begin{property}[Loss and gradients reweighting]
    \label{property:l_g_reweight}
    Let $\mathcal{L}_{\textup{IoU}}(b_1, b_2) = 1 - \textup{IoU}(b_1, b_2)$ and
    $\mathcal{L}_{\textup{SIoU}}(b_1, b_2) = 1 - \textup{SIoU}(b_1, b_2)$ be the
    loss functions associated respectively with IoU and SIoU. Let us denote the
    ratio between SIoU and IoU losses by 
    $\mathcal{W}_{\mathcal{L}}(b_1, b_2) =
    \frac{\mathcal{L}_{\textup{SIoU}}(b_1, b_2)}{\mathcal{L}_{\textup{IoU}}(b_1,
    b_2)}$. 
    Similarly, 
    $\mathcal{W}_{\mathcal{\nabla}}(b_1, b_2) =
    \frac{|\nabla\mathcal{L}_{\textup{SIoU}}(b_1, b_2)|}{|\nabla\mathcal{L}_{\textup{IoU}}(b_1,
    b_2)|}$ 
    denotes the ratio of gradients generated from SIoU and IoU losses: 
    \begin{align}
        \mathcal{W}_{\mathcal{L}}(b_1, b_2) &= \frac{1- \textup{IoU}(b_1, b_2)^p}{1-\textup{IoU}(b_1, b_2)}, \\
        \mathcal{W}_{\mathcal{\nabla}}(b_1, b_2) &= p\textup{IoU}(b_1, b_2)^{p-1},
    \end{align}

    \noindent
    $\mathcal{W}_{\mathcal{L}}$ and $\mathcal{W}_{\mathcal{\nabla}}$ are
    increasing (resp. decreasing) functions of IoU when $p\geq 1$ (resp. $p <
    1$) which is satisfied when $\gamma \leq 0$ (resp. $\gamma > 0$). As the IoU
    goes to 1, $\mathcal{W}_{\mathcal{L}}$ and $\mathcal{W}_{\mathcal{\nabla}}$
    approaches $p$: 
    \begin{align}
        \lim\limits_{\textup{IoU}(b_1, b_2) \to 1}\mathcal{W}_{\mathcal{L}}(b_1, b_2) &=  p, \\
        \lim\limits_{\textup{IoU}(b_1, b_2) \to 1}\mathcal{W}_{\mathcal{\nabla}}(b_1, b_2) &=  p.
    \end{align}
\end{property}

We employ the same tools as in \cite{he2021alpha} to analyze how SIoU affects
the losses and associated gradients. We show in property 2 that their results
hold for a non-constant power $p$ as well. From this, it can be observed that
when IoU is close to 1, losses and gradients are both rescaled by $p$. Hence,
the gradients coming from objects of different sizes will be rescaled
differently. The setting of $\gamma$ and $\kappa$ allows to balance the training
towards specific object sizes. Experimental results are provided in
\cref{sec:results} to support these findings. Proofs for properties 1 and 2 are
available in \cref{app:properties}.   

However, \textit{order preservingness} is not satisfied by using power value
changing with the size of the objects. This property ensures that the order
given by the IoU is preserved with the novel criterion, \eg $\text{IoU}(b_1,b_2) <
\text{IoU}(b_1,b_3) \Rightarrow \alpha\text{-IoU}(b_1,b_2) <
\alpha\text{-IoU}(b_1,b_3)$. $\alpha$-IoU preserves the order of IoU, but SIoU
does not. We show in \cref{app:properties} that even though this property is not
always satisfied, a large proportion of boxes meet the conditions for the order
to hold. \\


\subsection{Extensions and Generalization of SIoU}
Finally, SIoU can very well be extended as IoU was with GIoU or DIoU. Note that
we only focus on GIoU extension here as DIoU and its variants are composite loss
(\ie sum of multiple loss functions). We provide here an extension following
GIoU as it appears especially well-designed for small object detection. When
detecting small targets, it is easier for a model to completely miss the object,
producing an IoU of 0 no matter how far the predicted box is. On the contrary,
GIoU yields negative values for non-intersecting boxes. This produces more
relevant guidance during the early phase of training when the model outputs
poorly located boxes. Therefore, we extend SIoU by raising GIoU to the same
power $p$ as in \cref{eq:p}:%
\begin{equation}
    \text{GSIoU}(b_1, b_2) = \begin{cases}\text{GIoU}(b_1,b_2) ^ p  &\text{if } \text{GIoU}(b_1,b_2) \geq 0 \\
                                            -|\text{GIoU}(b_1,b_2)| ^ p  &\text{if } \text{GIoU}(b_1,b_2) < 0 \end{cases}. 
\end{equation}

\section{Scale-Adaptive Criteria Analysis}
\label{sec:criterion_analysis}
\vspace{-1em}

This section analyzes both empirically and theoretically the behaviors of IoU,
GIoU \cite{rezatofighi2019generalized}, $\alpha$-IoU \cite{he2021alpha}, NWD
\cite{wang2021nwd}, SIoU and GSIoU. We investigate the desirable properties of
such criteria for model training and performance evaluation.

\subsection{Response Analysis to Box Shifting}
\label{sec:shift_analysis}
As mentioned in \cref{sec:siou}, IoU drops dramatically when the localization
error increases for small objects. Shifting a box a few pixels off the ground
truth can result in a large decrease in IoU, without diminishing the quality of
the detection from a human perspective. This is depicted in
\cref{fig:loc_error_evo} (left), where plain lines represent the evolution of
IoU for various object sizes. These curves are generated by diagonally shifting
a box away from the ground truth. Boxes are squares, but similar curves would be
observed otherwise. In this plot, boxes have the same size, so when
there is no shift in between ($\varepsilon_{\text{loc}}=0$), IoU equals 1.
However, if the sizes of the boxes differ by a ratio $r$, IoU would peak at
$1/r^2$. Other line types represent other criteria. SIoU decreases slower than
IoU when $\varepsilon_{\text{loc}}$ increases and this is especially true when
boxes are small. This holds because $\gamma>0$, if it was negative, SIoU would
adopt the opposite behavior. In addition, the gap between IoU and SIoU is
larger when objects are small. Only NWD shares this property, but it only
appears when boxes have different sizes (all lines coincide for NWD). Hence,
SIoU is the only criterion that allows controlling its decreasing rate, \ie how
much SIoU is lost for a 1-pixel shift. As GIoU and GSIoU values range in
$[-1,1]$, they were not included in \cref{fig:loc_error_evo}, but for
completeness, they are plotted in \cref{fig:shift_gsiou} along with other
criteria. 

\begin{figure}[h]
    \centering
    \includegraphics[width=\columnwidth]{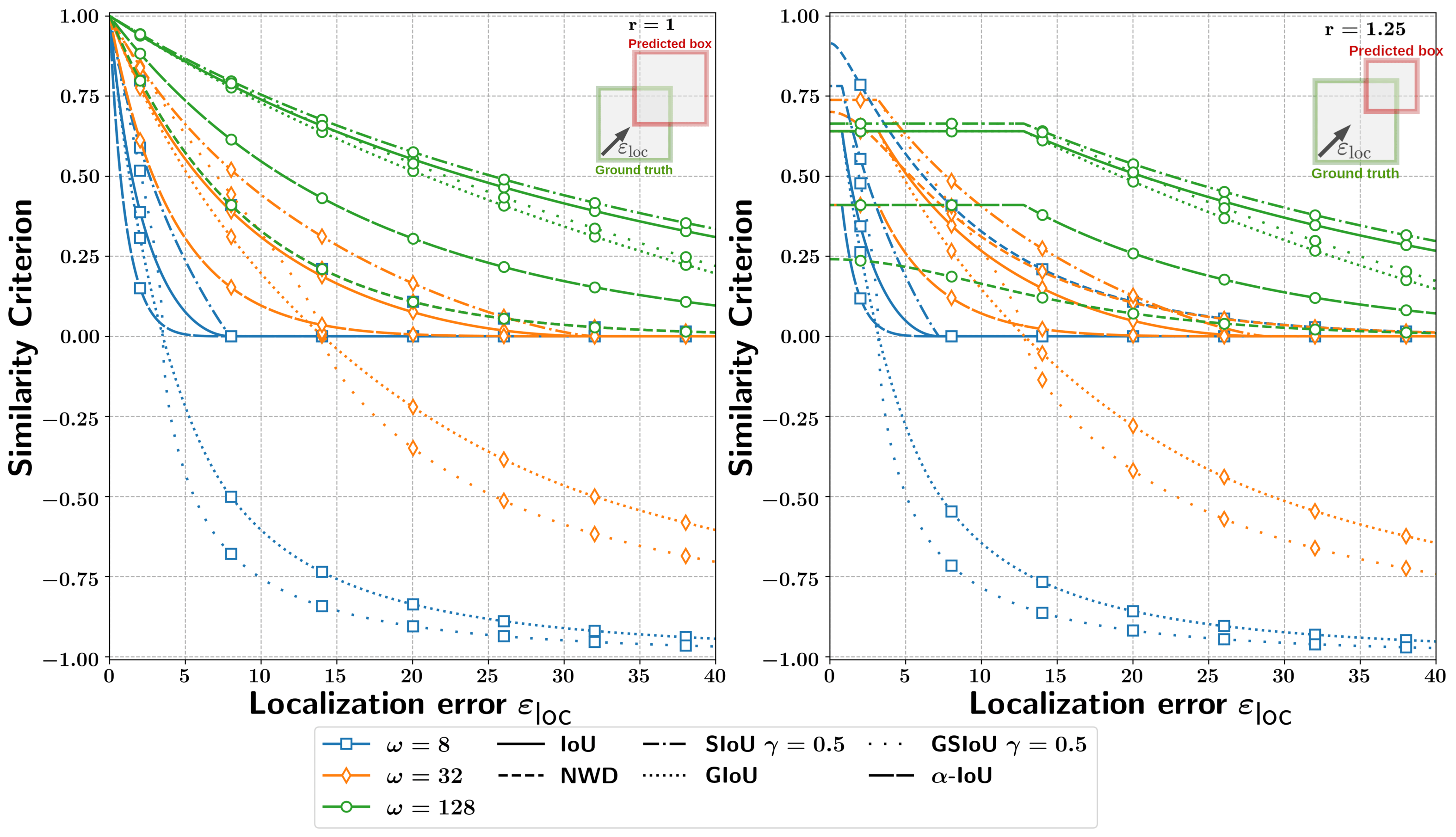}
    \caption[Critera's values against box localization error]{Evolution of
    various criteria (IoU, GIoU, and GSIoU) when a box is shifted from the
    ground truth box by $\rho$ pixels for various box sizes $\omega \in \{4, 16,
    64, 128\}$. With boxes of the same size \textbf{(left)} and different sizes
    \textbf{(right)}.}
    \label{fig:shift_gsiou}
\end{figure}

\subsection{Resilience Analysis to Detector Inaccuracy}
\label{sec:inaccuracy}
Knowing how a criterion responds to shifts and size variations is important to
understand what makes a sensible box similarity measure. Pushing beyond the shift
analysis, we study empirically and theoretically the criteria's distributions
when exposed to detector inaccuracies, \ie randomly shifted boxes. This setting
mimics the inaccuracy of the model either during training or at test time.

\subsubsection{Empirical Protocol}
\label{sec:empirical_analysis}
To simplify, let us suppose that all boxes are squares of the same size $\omega$
and can be shifted only horizontally. Similar results are observed by relaxing
these constraints, see \cref{sec:inaccuracy_assumptions}. A box is then entirely
defined by its horizontal position $x$ and its width $\omega$. If a detector is
not perfect, it will produce bounding boxes slightly shifted horizontally from
the ground truth. To model the detector's inaccuracy, we suppose that the
predicted box position is randomly sampled from a Gaussian distribution centered
on the ground truth location (which is chosen as 0 without loss of generality):
$X \sim \mathcal{N}(0, \sigma^2)$ where $\sigma$ controls how inaccurate the
model is. We are interested in the distribution of $\mathfrak{C} \in
\{\text{IoU}, \text{GIoU}, \text{SIoU}, \text{GSIoU}, \alpha\text{-IoU},
\text{NWD}\}$ and how it changes with $\omega$. To this end, let $Z =
\mathfrak{C}(X)$. More precisely, we are interested in the Probability Density
Function (PDF) of $Z$ and its two first moments (which exist because
$\mathfrak{C}$ is continuous and bounded). 

\cref{fig:criteria_analysis} gathers the results of this analysis. It shows the
pdf of each criterion for various box sizes (left) along with the evolution of
the expectation and standard deviation of $Z$ against $\omega$ (middle and
right). Specifically, we randomly sample a large number of boxes and compute the
associated criteria values for all $\mathfrak{C}$ and boxes. Then, the average
and standard deviation are computed to estimate the moment of the criteria'
pdfs. This process is repeated for various box sizes $\omega$ to understand how
it changes the behaviors of the criteria. From this, it can be noted that the
size of the boxes has a large influence on the distributions of all criteria.
The expected values of all criteria are monotonically increasing with object
size. In particular, small objects have lower expected IoU values than larger
ones. This is consistent with the initial assessment from
\cref{fig:loc_error_evo} (right) and it validates the choice of $\sigma$
constant for this study (although \cref{sec:inaccuracy} discusses this
assumption).

\begin{figure}
    \centering
    \includegraphics[width=\textwidth]{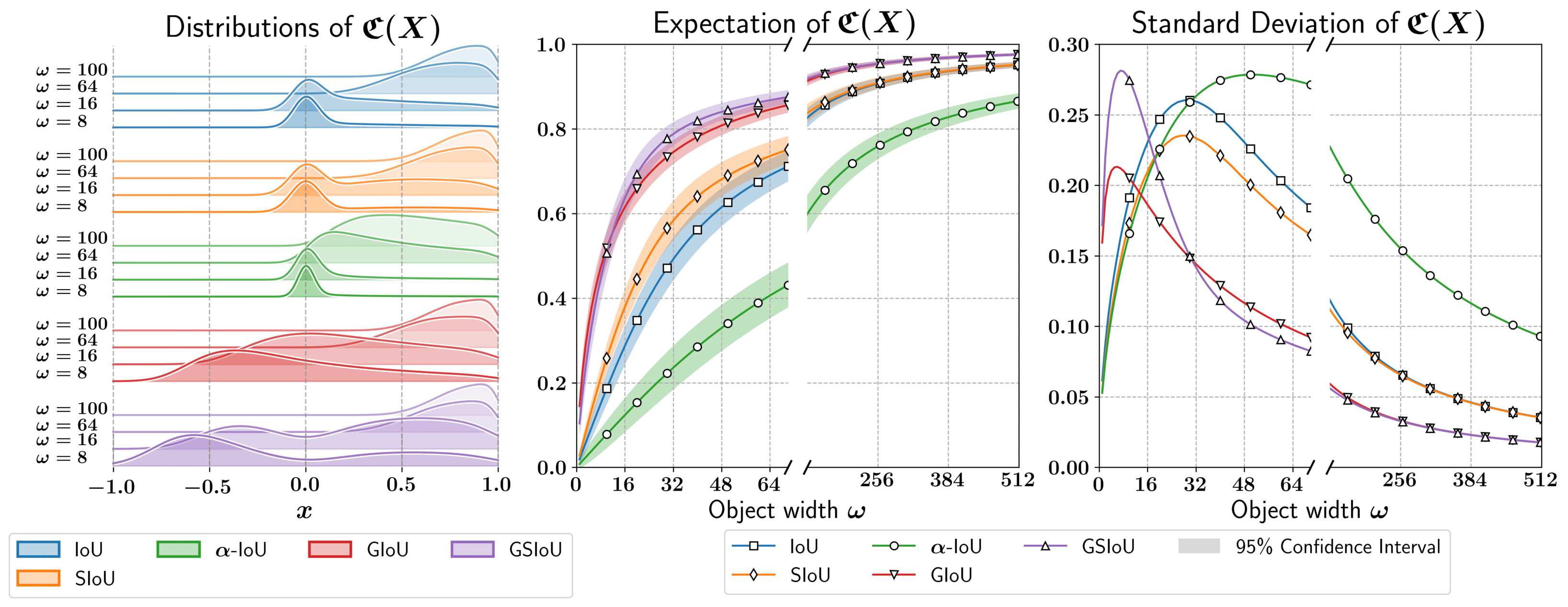}
    \caption[Criteria's distribution comparison ($\gamma=-4$ and $\kappa=16$ for SIoU and GSIoU)]{Analysis of the distribution of IoU, SIoU, GIoU, GSIoU and
            $\alpha$-IoU when computed on inaccurately positioned boxes. This is
            done by observing the probability distribution functions for various
            $\omega$ values \textbf{(left)}, the expectation \textbf{(middle)}
            and standard deviation \textbf{(right)} for all criteria. For SIoU
            and GSIoU, we fixed $\gamma=0.5$ and $\kappa=64$, for $\alpha$-IoU,
            $\alpha=3$ (as recommended in the original paper
            \cite{he2021alpha}). The inaccuracy of the detector is set to
            $\sigma=16$. Note that the empirical pdfs were smoothed using a
            Kernel Density Estimator method. This affects particularly IoU, SIoU
            , and $\alpha$-IoU for which the actual pdf is defined only on $[0,1]$. For
            the sake of visualization, GIoU and GSIoU were rescaled between 0
            and 1 for the expectation and standard deviation plots. }
    \label{fig:criteria_analysis}
    \vspace{-1em}
\end{figure}

When building detection models, we hope to detect equally well objects of all
sizes, this means having a constant expected IoU, no matter the objects' size.
This would require the localization error to be an affine function of $\omega$.
Of course, the localization error of the detector is likely to depend on
$\omega$. However, it cannot be an affine function, otherwise, small objects
would be perfectly detected, which is not observed (see
\cref{fig:loc_error_evo}, right). As SIoU has larger expected values than IoU
for small objects, it can compensate for their larger localization errors. The
setting of $\gamma$ and $\kappa$ allows controlling how much small objects are
favored (see \cref{fig:influence_gamma_kappa}). NWD
is not included in these plots as its expected value and variance are constant
when dealing with same-size boxes.

\subsubsection{Influence Analysis on the Performance Evaluation}
If the expected value of a criterion is too small, it is likely that the boxes
will be considered negative detections during evaluation and therefore reduce
the performance. Therefore, having a criterion with larger expected values for
small objects would better reflect the true performance of a detector. One might
think that it would be equivalent to scale-adaptive IoU thresholds during the
evaluation, but this is not completely true as the variance of the criteria also
differs.

Having an accurate criterion (\ie with low variance) is crucial for evaluation.
Let us take a detector that produces well-localized boxes on average, \ie on
average the criterion computed between the boxes and their corresponding ground
truths is above a certain threshold. As the detector is not perfect, it will
randomly produce boxes slightly better or slightly worse than the average. If
the criterion has a high variance, it will be more likely that poor boxes get
scores below the criterion threshold and therefore will be considered negative
detections. This will reduce the performance of the detector even though on
average, it meets the localization requirements. In addition, a criterion with a
higher variance will be less reliable and would produce more inconsistent
evaluations of a model. The fact that the IoU variance is high for small objects
partly explains why detectors have much lower performance on these objects.
Hence, SIoU seems more adapted for evaluation. Of course, using this criterion
for evaluation will attribute higher scores for less precise localization of
small objects. However, this aligns better with human perception as demonstrated
in \cref{sec:user_study}. Employing SIoU in the evaluation process also allows
tweaking it for the needs of a specific application. 

\vspace{1em}
\subsubsection{Influence Analysis on Training}
All criteria discussed above are employed as regression losses in the
literature. The loss associated with each criterion $\mathfrak{C}$ is
$\mathcal{L}_{\mathfrak{C}}(b_1,b_2) = 1 - \mathfrak{C}(b_1,b_2)$. Therefore,
the expected value of the criterion determines the expected value of the loss
and thus the magnitude of the gradients. Large values of the criterion give low
values for the loss. Now, as the expected values of the criteria change with the
object size, the expected values of the losses also change. Small objects
generate greater loss values than larger ones on average. However, this is
balanced by the fact that fewer small objects are selected as positive examples
because the IoU is involved in the selection process. To achieve better
detection, training must focus more on small objects. One way to ensure this is
to set larger loss values for small objects. Thus, the equilibrium is
shifted toward smaller objects and gradients will point to regions where the
loss of small objects is lower. As shown in \cref{fig:criteria_analysis_4}
(\cref{sec:gamma_kappa}), with the right parameters, SIoU can do that. It
attributes lower values for small objects while keeping similar values for large
ones. The contrast between small and large objects is accentuated and
optimization naturally focuses on smaller objects. SIoU's parameters control
which object size gets more emphasis. This is closely linked to
\cref{property:l_g_reweight} which states that employing SIoU (compared to IoU)
reweights the loss and the gradient by $p$. If $\gamma < 0$, $p$ decreases with
the size of the objects and thus the optimization focuses on small objects. This
also explains why generalizations of existing criteria (\ie with negative values
for non-overlapping boxes) often outperform their vanilla version. Taking IoU
and GIoU as examples, the gap between their expected values for small and
large objects is greater with GIoU. It nudges the optimization towards small
objects.  

\subsubsection{Inaccuracy Tolerance Assumptions}
\label{sec:inaccuracy_assumptions}
\noindent Several assumptions were made in \cref{sec:inaccuracy} to analyze the
criteria for box similarity:
\begin{enumerate}[nolistsep, topsep=-3mm]
    \item Boxes are shifted only horizontally. 
    \item Boxes have the same size.
    \item The detector's inaccuracy is fixed and does not depend on the object size.
\end{enumerate}
The first two assumptions are relatively harmless. Allowing diagonal shifts
simply accelerates the IoU drop rate. A 1-pixel diagonal shift is equivalent to
a vertical and a horizontal shift. Intuitively, this is similar to a 2-pixels
horizontal shift. However, this is not true because, with a 1-pixel diagonal
shift, the area of intersection decreases slower than with a 2-pixels horizontal
shift. Following the notations from \cref{sec:inaccuracy}, the intersection
between two boxes of width $\omega$ diagonally shifted by $\rho$ pixels is
$(\omega - \rho)^2 = \omega^2 - 2\omega\rho + \rho^2$ while the intersection
between same boxes horizontally shifted by $2\rho$ pixels is $\omega(\omega -
2\rho) = \omega^2 - 2\omega\rho$. To ensure that this does not question the
conclusions of \cref{sec:empirical_analysis}, \cref{fig:diagonal_shift_distri}
compares the expected values and variances of IoU and GIoU with horizontal and
diagonal shifts. Similar behaviors are observed with and without diagonal
shifting. The only difference is that the expected values of the criteria for
diagonally shifted boxes are lower as the shifts get larger. It also increases
the variances as the distributions are more spread. Then, relaxing the second
constraint results in slightly different distributions, but with similar
behavior. Having boxes of different sizes only changes the maximum value of the
criteria. If boxes have different sizes, the maximum value must be smaller than
1. Therefore, the expected values approach smaller values than 1 as objects get
larger. The variance is reduced as the range of criteria values is smaller (see
\cref{fig:different_size_boxes}). 

\begin{figure}
    \centering
    \begin{subfigure}[t]{0.7\columnwidth}
        \centering
        \includegraphics[trim=490 0 0 0, clip,width=\columnwidth]{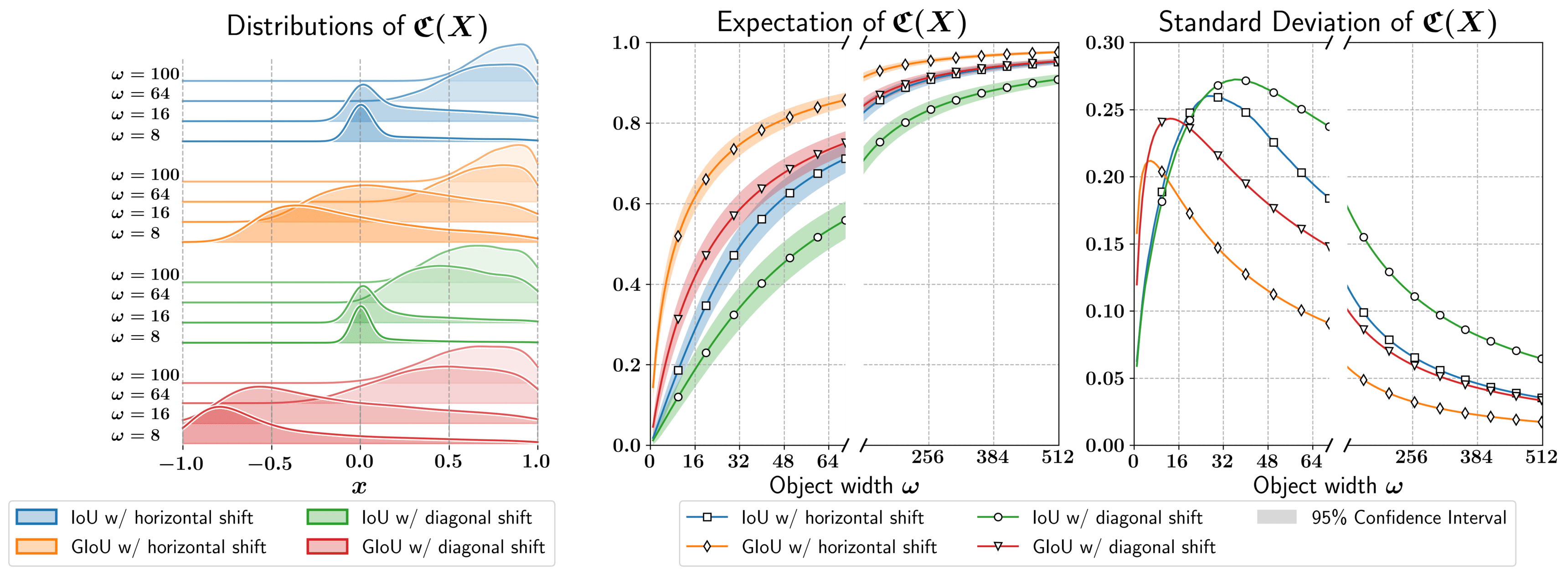}
        \caption{IoU and GIoU expected values and standard deviation with horizontally and diagonally shifted boxes.}
        \label{fig:diagonal_shift_distri}
    \end{subfigure}%
    
    \begin{subfigure}[b]{0.7\columnwidth}
        \centering
        \includegraphics[trim=430 0 0 0, clip,width=\columnwidth]{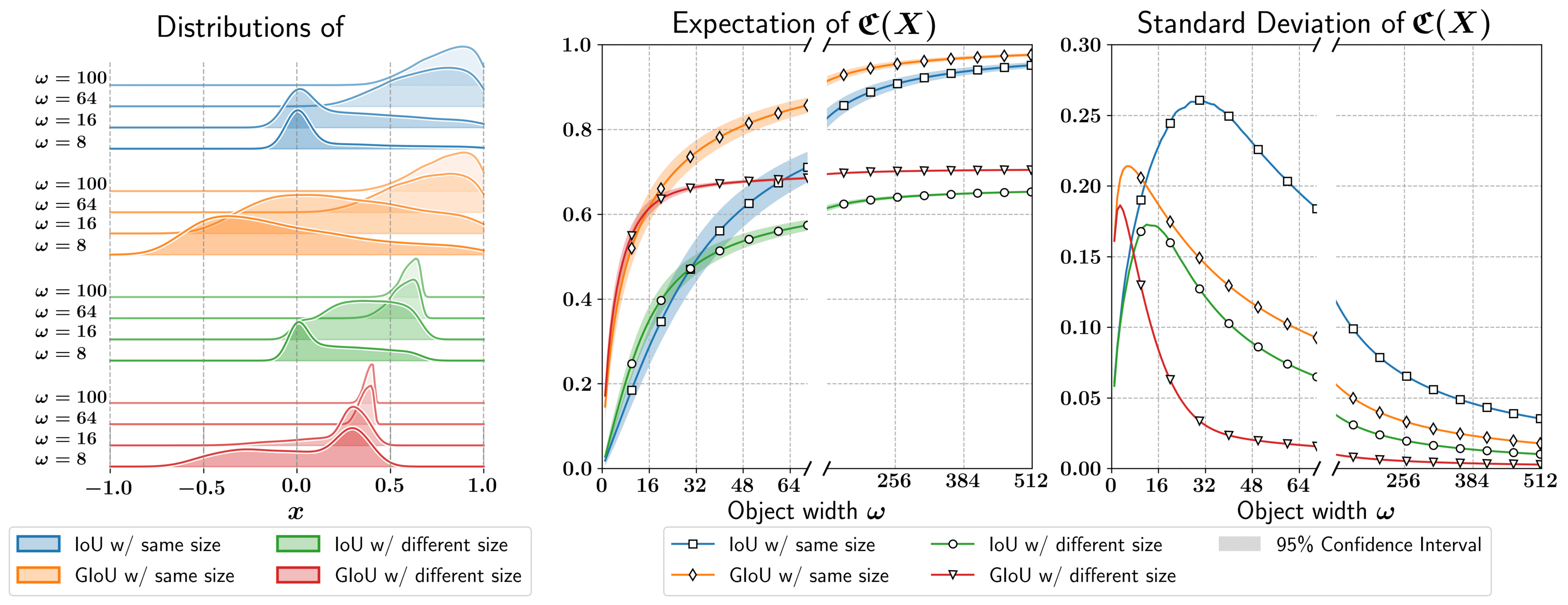}
        \caption{IoU and GIoU expected values and standard deviation with and without boxes of the same size.}
        \label{fig:different_size_boxes}
    \end{subfigure}
    \begin{subfigure}[b]{0.7\columnwidth}
        \centering
        \includegraphics[trim=490 0 0 0, clip,width=\columnwidth]{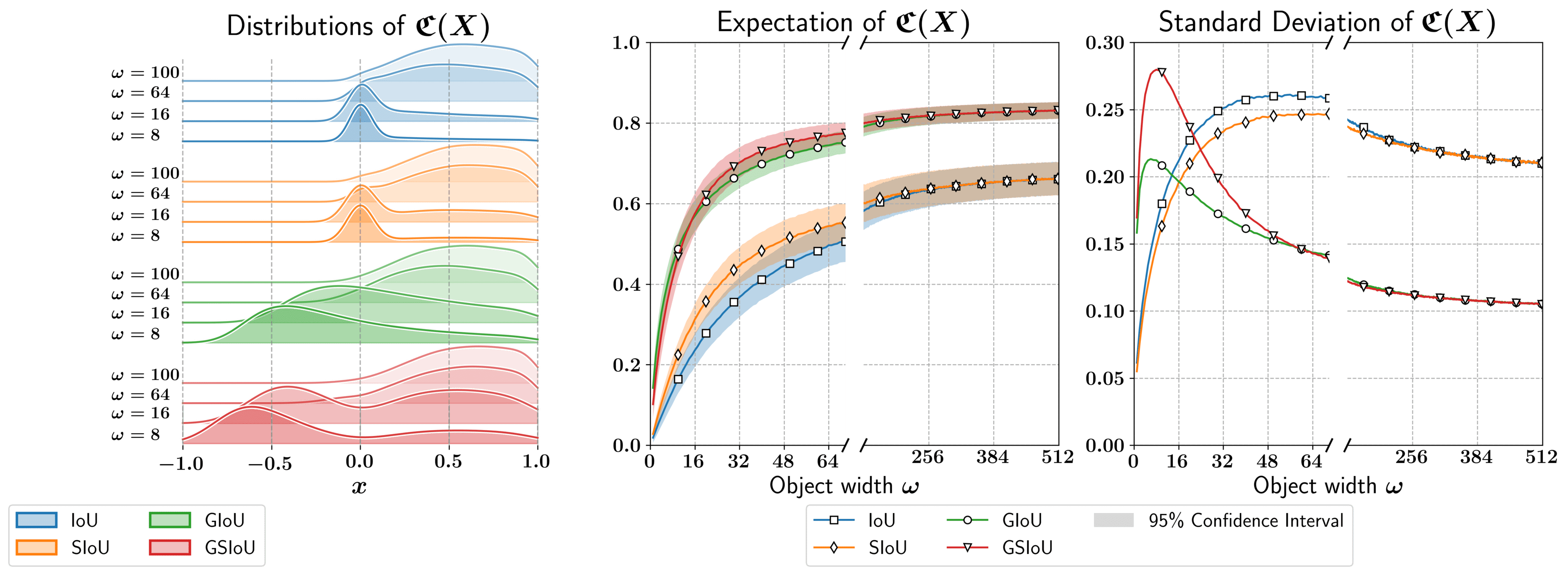}
        \caption{IoU, SIoU, GIoU, and GSIoU expected values and standard
        deviation with the detector's inaccuracy modeled as an affine function,
        $\sigma(\omega) = \sigma_0 + \lambda \omega$ ($ \sigma_0 = 16$,
        $\lambda=\frac{1}{4}$).}
        \label{fig:affine_inaccuracy}
    \end{subfigure}
    \caption{Relaxing the constraints for criteria' distribution analysis.}
    \label{fig:relax_assumptions}
\end{figure}

It is less straightforward that the conclusions hold without the last
assumption. In the analysis from \cref{sec:empirical_analysis}, we assume that
the inaccuracy of the detector is fixed. This means that the detector generates
randomly shifted boxes by the same number of pixels on average no matter the
size of the object. This is certainly false, in practice, in terms of absolute
distance, detectors are better with smaller objects. However, the inaccuracy
cannot simply be proportional to object sizes because small objects would then
be perfectly detected. Thus, we tried to change the inaccuracy of the detector
as an affine function of the box width: $\sigma(\omega) = \sigma_0 + \lambda
\omega$. We choose to set $\sigma_0$ to the fixed value of $\sigma$ used in
\cref{sec:empirical_analysis} and $\lambda=1/4$. This setting reflects better
the inaccuracy of a true detector. The expected values and standard deviations
of IoU, SIoU, GIoU, and GSIoU with this inaccuracy setting are plotted in
\cref{fig:affine_inaccuracy}. The main difference with fixed inaccuracy is that
expected values do not approach 1 as object size gets larger, instead they tend
towards lower values. It also leads to non-zero variance for large objects.
However, for small objects, the curves of the different criteria are mostly
unchanged, and the conclusions formulated in \cref{sec:empirical_analysis} are
still valid.

\subsection{Theoretical study of GIoU}
In the previous section, we derived the statistics of several box similarity
criteria from empirical simulations. However, the criteria probability
distribution functions and first moments can also be derived theoretically. We
provide such results for GIoU in \cref{prop:distribution_giou}. 
\begin{proposition}[GIoU's distribution]
    \label{prop:distribution_giou}
    Let $b_1 = (0,y_1,w_1,h_1)$ be a bounding box horizontally centered and $b_2
    = (X,y_2,w_2,h_2)$ another bounding box randomly positioned, with $X\sim
    \mathcal{N}(0, \sigma^2)$ and $\sigma \in \mathbb{R}^*_+$. Let's suppose that
    the boxes are identical squares, shifted only horizontally (\ie $w_1 = w_2 = h_1 =
    h_2$ and $y_1 = y_2$).
    \noindent
    Let $Z = \mathfrak{C}(X)$, where $\mathfrak{C}$ is the generalized intersection over union. The
    probability density function of $Z$ is given by:
    \vspace{-0.5em}
    \begin{align}
        d_Z(z) &= \frac{2\omega}{(1+z)^2\sqrt{2\pi} \sigma} \exp\left(-\frac{1}{2}\left[\frac{\omega(1-z)}{\sigma(1+z)}\right]^2\right).
    \end{align}
    \vspace{-0.5em}
    \noindent
    The two first moments of $Z$ exist and are given by:
    \begin{align}
        \mathbb{E}[Z] &= \frac{2}{\pi^{3/2}} G^{2,3}_{3,2}\left(2a^2 \bigg\rvert \begin{matrix}0 & \frac{1}{2} & \frac{1}{2} \\ \frac{1}{2}& 0 \end{matrix}\right), \\
        \mathbb{E}[Z^2] &= 1- \frac{8a}{\sqrt{2\pi}} + \frac{16a^2}{\pi^{3/2}} G^{2,3}_{3,2}\left(2a^2 \bigg\rvert \begin{matrix}-1 & \frac{1}{2} & -\frac{1}{2} \\ \frac{1}{2}  & 0 \end{matrix}\right),
    \end{align}

    \noindent
    where $G$ is the Meijer G-function \cite{meijer2013} (see its definition in
    \cref{eq:meijer_g}, \cref{app:theoretical_giou_proof}) and
    $a=\frac{\sigma}{\omega}$.
\end{proposition}

The proof of this proposition and derivations for other criteria are available
in \cref{app:theoretical_giou_proof}. The theoretical expressions completely
agree with empirical results, which confirms the soundness of our simulations. 

Other criteria do not have closed forms for their first and second moments.
Nonetheless, we provide in \cref{tab:pdf_crit} their expressions keeping the
integrals as simple as possible, which allows relatively easy evaluation. In
addition, we provide the expression of the pdf for each criterion. The setup
remains identical as in \cref{prop:distribution_giou}, the boxes are only
horizontally shifted and have the same width $\omega$. For clarity, we also give
simple expressions of each criterion in such a setup (see \cref{tab:pdf_crit}).  

\begin{sidewaystable}[]
    \centering
    \resizebox{\textwidth}{!}{%
    \begin{tabular}{@{}ccccccccccccccccc@{}}
    \toprule[1.5pt]
                    & \textbf{Criterion Expression} & \textbf{Probability Density Function} & $\boldsymbol{\mathbb{E}[Z]}$ & $\boldsymbol{\mathbb{E}[Z^2]}$ \\ \midrule[1.5pt]
    \textbf{IoU}             & \parbox{0.25\textwidth}{\begin{align*}\mathfrak{C} \colon \mathbb{R} &\to [0,1] \\x &\mapsto \max\left(0, \frac{\omega- |x|}{\omega + |x|}\right)\end{align*}} &\parbox{0.25\textwidth}{\begin{align*}d_Z(z) = 2\bigg[(&1 - F_X(\omega\frac{1-z}{1+z}) \delta_0(z) \\&+ \mathds{1}_{\mathbb{R}_+}(z) \frac{4\omega}{(1+z)^2} d_X(\omega\frac{1-z}{1+z}))\bigg]\end{align*}  } &     \parbox{0.25\textwidth}{\begin{align*}\mathbb{E}[Z] = \frac{4}{\sqrt{2\pi}a}\int_{0}^{1} \frac{1}{1+u}e^{-\frac{u^2}{2a^2}}\,du  - \text{erf}(\frac{1}{\sqrt{2}a})\\\end{align*}} &     \parbox{0.25\textwidth}{\begin{align*}\mathbb{E}[Z^2]= \text{erf}(\frac{1}{\sqrt{2}a}) - \frac{8a}{\sqrt{2\pi}a}\bigg[&1- \frac{1}{2}e^{-\frac{1}{2a^2}}\\&-2\int_{0}^{1} \frac{1}{(1+u)^3}e^{-\frac{u^2}{2a^2}}\,du \bigg]\end{align*}} &\\\bottomrule[1pt]
    \textbf{GIoU}           & \parbox{0.25\textwidth}{\begin{align*}\mathfrak{C} \colon \mathbb{R} &\to [-1,1] \\x &\mapsto \frac{\omega- |x|}{\omega + |x|}\end{align*}} &\parbox{0.25\textwidth}{\begin{align*}d_Z(z) &= \frac{4\omega}{(1+z)^2\sqrt{2\pi} \sigma} \exp\left(-\frac{1}{2}\left[\frac{\omega(1-z)}{\sigma(1+z)}\right]^2\right)\end{align*}} &     \parbox{0.25\textwidth}{\begin{align*}\mathbb{E}[Z] &= \frac{2}{\pi^{3/2}} G^{2,3}_{3,2}\left(2a^2 \bigg\rvert \begin{matrix}0 & \frac{1}{2} & \frac{1}{2} \\ \frac{1}{2}& 0 \end{matrix}\right) \\\end{align*}} &     \parbox{0.25\textwidth}{\begin{align*}\mathbb{E}[Z^2] &= 1- \frac{8a}{\sqrt{2\pi}} + \frac{16a^2}{\pi^{3/2}} G^{2,3}_{3,2}\left(2a^2 \bigg\rvert \begin{matrix}-1 & \frac{1}{2} & -\frac{1}{2} \\ \frac{1}{2}  & 0 \end{matrix}\right)\end{align*}} &\\\bottomrule[1pt]
    \textbf{SIoU}            & \parbox{0.25\textwidth}{\begin{align*}\mathfrak{C} \colon \mathbb{R} &\to [0,1] \\x &\mapsto \begin{cases} \left(\frac{\omega- |x|}{\omega + |x|}\right)^p \,&\text{if  } \omega- |x| \geq 0\\ 0\,&\text{otherwise}\end{cases}\end{align*}} &\parbox{0.25\textwidth}{\begin{align*}d_Z(z) = 2\bigg[(&1 - F_X(\omega\frac{1-z^{1/p}}{1+z^{1/p}}) \delta_0(z) \\&+ \mathds{1}_{\mathbb{R}_+}(z) \frac{4\omega z^{1/p -1 }}{(1+z^{1/p})^2} d_X(\omega\frac{1-z^{1/p}}{1+z^{1/p}}))\bigg]\end{align*}} &     \parbox{0.25\textwidth}{\begin{align*}\mathbb{E}[Z] &= 2\omega\int_{0}^{1} \big(\frac{1-u}{1+u}\big)^p e^{-\frac{u^2}{2a^2}}\,du\end{align*}} &     \parbox{0.25\textwidth}{\begin{align*}\mathbb{E}[Z^2] &= 2\omega\int_{0}^{1} \big(\frac{1-u}{1+u}\big)^{2p} e^{-\frac{u^2}{2a^2}}\,du\end{align*}} &\\\bottomrule[1pt]
    \textbf{GSIoU}            & \parbox{0.25\textwidth}{\begin{align*}\mathfrak{C} \colon \mathbb{R} &\to [-1,1] \\x &\mapsto \begin{cases} \left(\frac{\omega- |x|}{\omega + |x|}\right)^p \,&\text{if  } \omega- |x| \geq 0\\ -\left(\frac{|x| - \omega}{\omega + |x|}\right)^p \,&\text{otherwise}\end{cases}\end{align*}} &\parbox{0.25\textwidth}{\begin{align*}d_Z(z) = 2\bigg[(&1 - F_X(\omega\frac{1-z^{1/p}}{1+z^{1/p}}) \delta_0(z) \\&+ \mathds{1}_{\mathbb{R}_+}(z) \frac{4\omega z^{1/p -1 }}{(1+z^{1/p})^2} d_X(\omega\frac{1-z^{1/p}}{1+z^{1/p}}))\\&-( 1 - F_X(\omega\frac{1+|z|^{1/p}}{1-|z|^{1/p}}) \delta_0(z)) \\&+ \mathds{1}_{\mathbb{R}_-}(z) \frac{4\omega |z|^{1/p -1 }}{(1-|z|^{1/p})^2} d_X(\omega\frac{1+|z|^{1/p}}{1-|z|^{1/p}}))\bigg]\end{align*}} &     \parbox{0.25\textwidth}{\begin{align*}\mathbb{E}[Z] = 2\omega\bigg[&\int_{0}^{1} \big(\frac{1 - u}{1 +u }\big)^p e^{-\frac{u^2}{2a^2}}\,du \\- &\int_{1}^{+\infty} \big(\frac{1-u}{1+u}\big)^p e^{-\frac{u^2}{2a^2}}\,du \bigg]\end{align*} } &     \parbox{0.25\textwidth}{\begin{align*}\mathbb{E}[Z^2] = 2\omega\bigg[&\int_{0}^{1} \big(\frac{1 - u}{1 +u  }\big)^{2p} e^{-\frac{u^2}{2a^2}}\,du \\&- \int_{1}^{+\infty} \big(\frac{1-u}{1+u}\big)^{2p} e^{-\frac{u^2}{2a^2}}\,du \bigg]\\\end{align*}} &\\\bottomrule[1.5pt]
                \end{tabular}%
    } \caption[Criteria statistical characteristics]{Criteria expression,
    probability distribution and first two moments for IoU, GIoU, SIoU, and
    GSIoU. These are valid for the comparison of same size square boxes of width
    $\omega$ randomly shifted horizontally. Random shifts are sampled from a
    centered Gaussian distribution of variance $\sigma^2$ and
    $a=\sigma/\omega$. $F_X$ and $d_X$ are respectively the cumulative and
    probability density function of $X$.}
    \label{tab:pdf_crit}
    \end{sidewaystable}

\subsection{Influence of $\gamma$ and $\kappa$ on SIoU and GSIoU}
\label{sec:gamma_kappa}

In the previous discussion, SIoU and GSIoU are parametrized with $\gamma=0.5$
and $\kappa=64$; however, these two parameters have an influence on the
distribution and moments of SIoU and GSIoU. First, following the analysis from
\cref{sec:empirical_analysis}, \cref{fig:influence_gamma_kappa} investigates the
influence of $\gamma$ and $\kappa$ on SIoU behavior. \cref{fig:gamma_influence}
shows the value of $p$, the expectation and variance of SIoU against object size
for $\gamma \in \{1.0, 0.75, 0.5, 0.1, 0.0, -0.25, -2, -4\}$. $p$ is a function
of the average area of the boxes ($\tau = \frac{w_1h_1 + w_2h_2}{2})$ and for
simplicity we suppose here that the boxes are squares of the same width
$\omega$, hence $\sqrt{\tau} = \omega$. Then $p$ can be viewed as a function of
$\omega$: $p(\omega) = 1 - \gamma\exp(-\omega / \kappa)$. For negative values of
$\gamma$, $p$ decreases from $p(0) = 1 - \gamma$ to 1, small objects get higher
exponents in comparison with larger objects. On the contrary, when $\gamma > 0$,
$p$ increases from $p(0) = 1 - \gamma$ to 1. Changing $\gamma$ also influences
the distribution of SIoU. As $\gamma$ increases, the expected value for small
objects increases as well, while the variance decreases. 

\begin{figure}
    \centering
    \begin{subfigure}[t]{\textwidth}
        \centering
        \includegraphics[width=\textwidth]{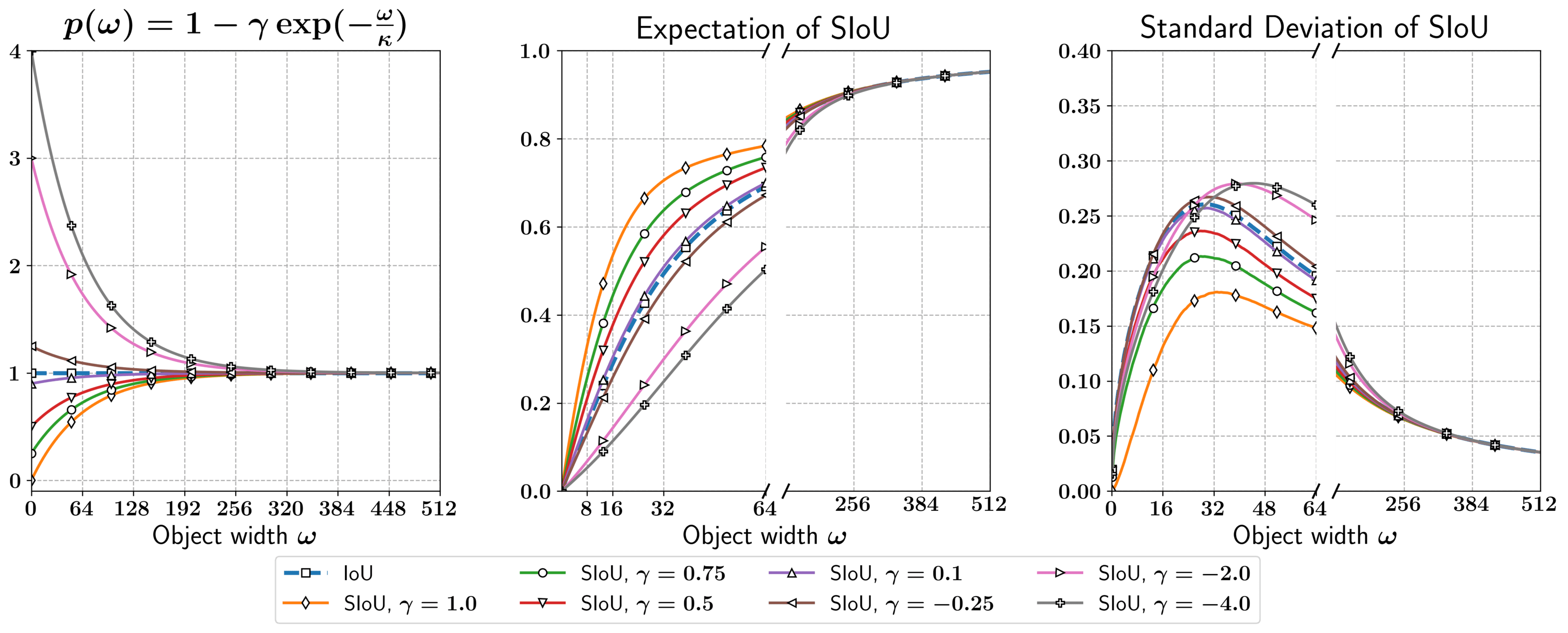}
        \caption{$\gamma$'s influence, with $\kappa=64$}
        \label{fig:gamma_influence}
    \end{subfigure}%
    
    \begin{subfigure}[b]{\textwidth}
        \centering
        \includegraphics[width=\textwidth]{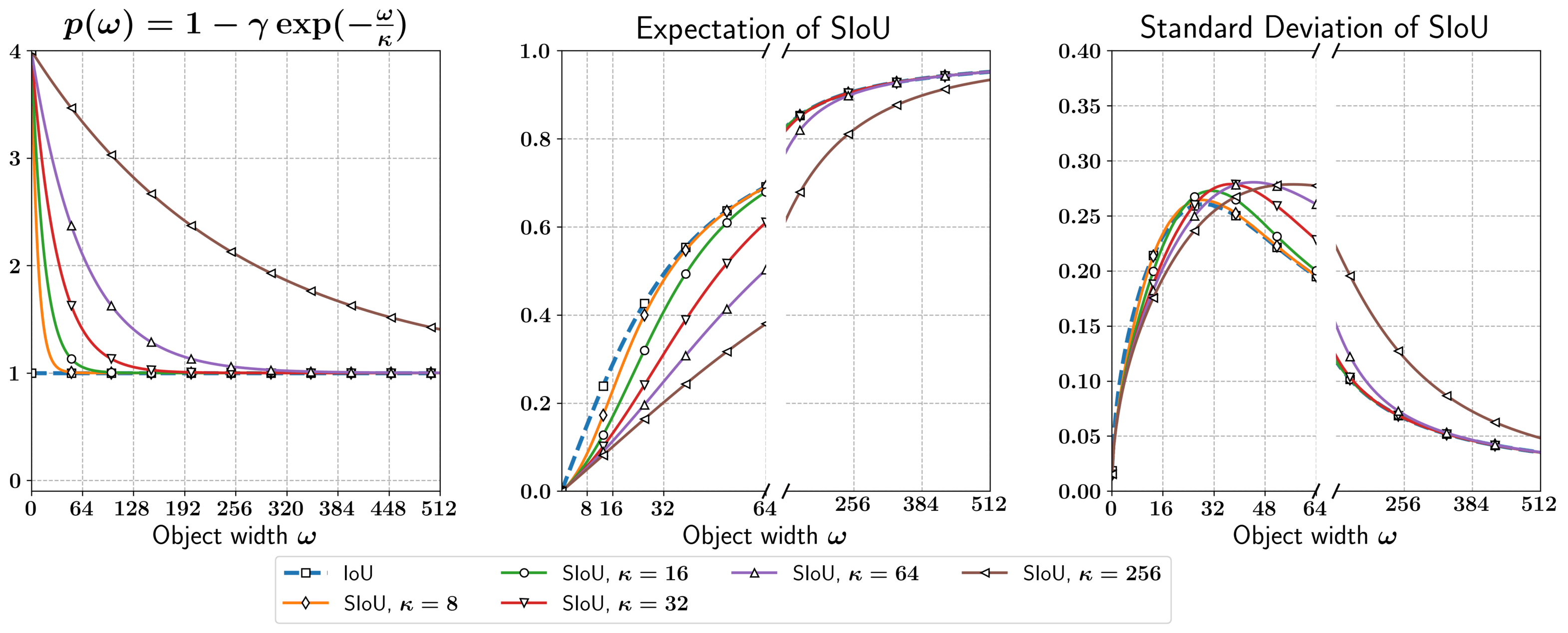}
        \caption{$\kappa$'s influence, with $\gamma=-3$}
        \label{fig:kappa_influence}
    \end{subfigure}
    \caption{Influence of $\gamma$ and $\kappa$ on the expected value and standard deviation of SIoU.}
    \label{fig:influence_gamma_kappa}
    \end{figure}

\cref{fig:kappa_influence} shows the same curves for $\kappa \in
\{8,16,32,64,256\}$. $\kappa$ controls how fast $p$ approaches 1 and
therefore, changing $\kappa$ simply shifts the curves of expectation and
variance accordingly. As $\kappa$ increases, IoU's behavior is retrieved for
larger objects reducing the expected value of SIoU. The variance is not changed
much by $\kappa$, but it slightly shifts the maximum of the curve, \ie the object
size for which SIoU's variance is maximum.

\cref{fig:criteria_analysis_4} also provides pdfs plots for various object
sizes for SIoU and GSIoU, in addition to expectation and variance comparison
between existing criteria. For this figure, $\gamma=-3$ and $\kappa=16$. This
figure echoes \cref{fig:criteria_analysis} which plots the same curves but with
$\gamma=0.5$ and $\kappa=64$.

\begin{figure}
    \centering
    \includegraphics[width=\textwidth]{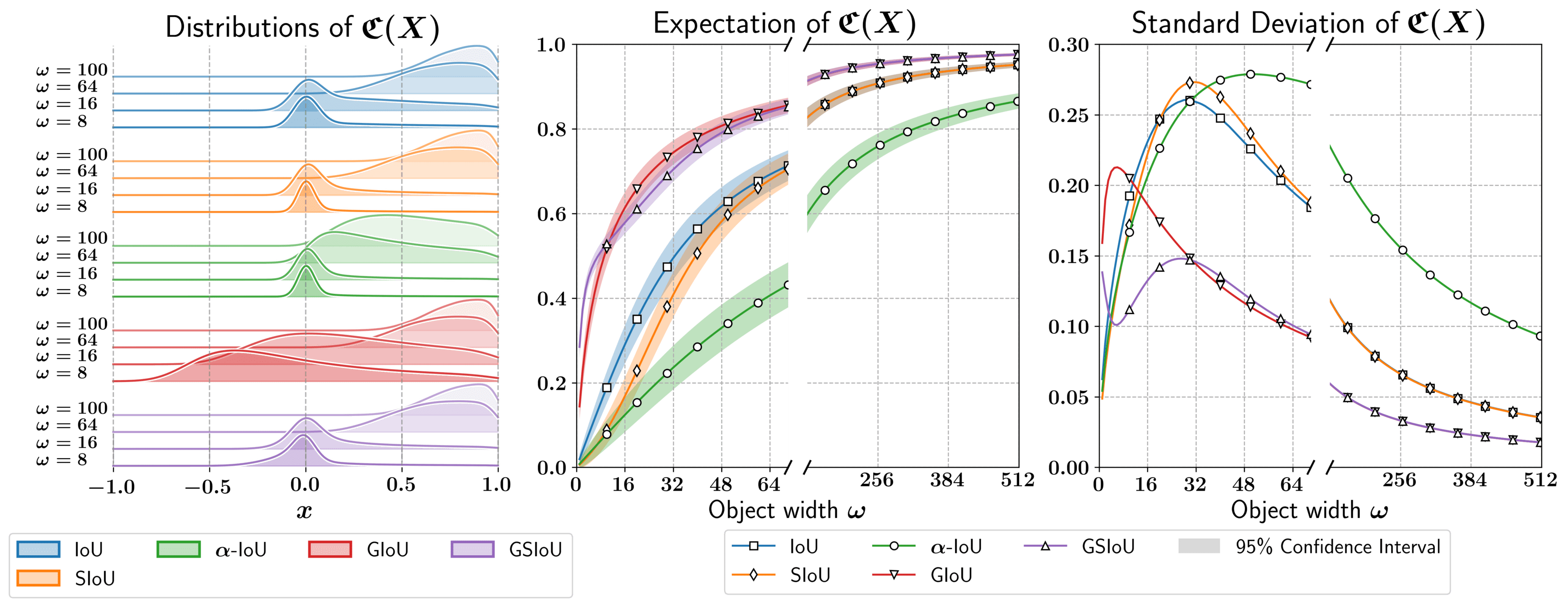}
    \caption[Criteria's distribution comparison ($\gamma=-4$ and $\kappa=16$ for
            SIoU and GSIoU)]{Analysis of the distribution of IoU, SIoU, GIoU,
            GSIoU and
            $\alpha$-IoU when computed on inaccurately positioned boxes. This is
            done by observing the probability distribution functions (pdfs) for various
            $\omega$ values \textbf{(left)}, the expectation \textbf{(middle)}
            and standard deviation \textbf{(right)} for all criteria. For SIoU
            and GSIoU, we fixed $\gamma=-4$ and $\kappa=16$, for $\alpha$-IoU,
            $\alpha=3$ (as recommended in the original paper
            \cite{he2021alpha}). The inaccuracy of the detector is set to
            $\sigma=16$. Note that the empirical pdfs were smoothed using a
            Kernel Density Estimator method. This affects particularly IoU, SIoU
            and $\alpha$-IoU as the actual PDF is defined only on $[0,1]$. For
            the sake of visualization, GIoU and GSIoU were rescaled between 0
            and 1 for the expectation and standard deviation plots. }
    \label{fig:criteria_analysis_4}
\end{figure}

\section{SIoU Alignment with Human Perception}
\label{sec:user_study}
\vspace{-1em}
As discussed in \cref{sec:empirical_analysis}, having an accurate criterion \ie
one with low variance, is crucial for evaluation. However, such a criterion must
also align with human perception. Most image processing models are destined to
assist human users. Thus, to maximize the usefulness of such models, the
evaluation process should align as closely as possible with human perception. To
assess the agreement between the criteria and human perception, we conducted a
user study in which participants had to rate on a 1 to 5 scale (\ie from
\textit{very poor} to \textit{very good}) how a bounding box localizes an object.
Specifically, an object is designated by a green ground truth box and a red box
is randomly sampled around the object (\ie with random IoU with the ground
truth). Then, the participants rate how well the red box localizes the object
within the green one. The study gathered 75 different participants and more than
3000 individual answers. We present here the main conclusion of this study.

\subsection{User Study Presentation}
\subsubsection{Experimental Protocol}
To carry out the user study about detection preferences, we developed a Web
App\footnote{The web app is available
\href{http://pierlj.pythonanywhere.com/}{here}.} to gather participants'
answers. Each participant had to sign in with a brief form. They are asked about
their age and whether they are familiar with image analysis. This information is
only meant to detect rating differences between different population groups (see
\cref{fig:hr_vs_crit}, last two rows). After completing the form, participants
are brought to the rating page (see \cref{fig:user_study_example}). On this
page, one image is visible with two bounding boxes drawn on it. A green one,
which represents the ground truth annotation of an object, and a red one
randomly shifted and deformed. Each participant must rate how well the red box
is detecting the object inside the green box. The rating is done on a 5-levels
scale, going from \textit{very poor} to \textit{very good}. A set of 50
different images is shown to each participant. After 25 images, the experiment
changes slightly: the background image is replaced by a completely black image.
This would remove any contextual bias coming from the variety of objects inside
the green box. We refer to the two phases of the experiment as respectively, the
phases with and without context. The red boxes are sampled around the green box,
but to enforce a uniform distribution of the IoU with the green box, a random
IoU value $u $ is first uniformly sampled between 0 and 1. Then, we randomly
generate a red box that has an IoU $u$ with the green box (direct box
sampling does not produce uniformly distributed IoU values). Participants are
instructed to answer quickly and are provided with examples for each possible
rating (see \cref{fig:user_study_example}). The images and the annotations are
randomly picked from the DOTA dataset.

\begin{figure}
    \centering
    \includegraphics[width=\textwidth]{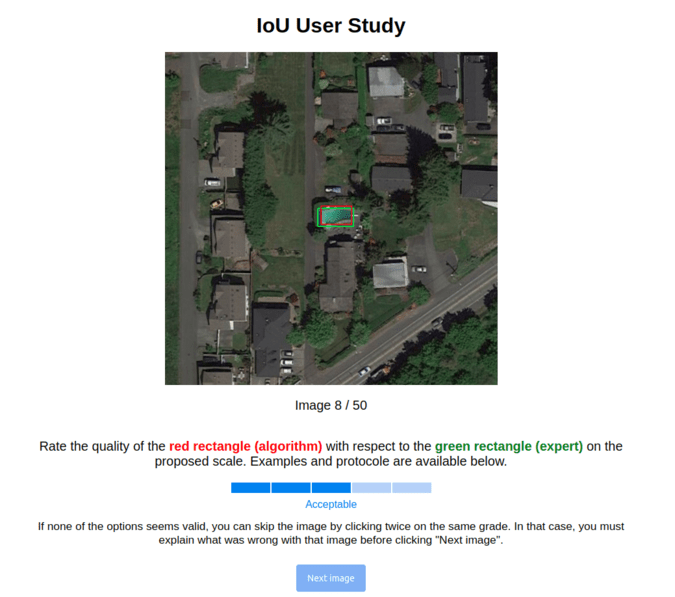}
    \caption[User study rating page]{Rating page of the Web App with an example
    of an image and two rectangular bounding boxes. The user is asked to rate the
    quality of the red box compared to the green one on a 5-levels scale
    going from \textit{very poor} to \textit{very good}.}
    \label{fig:user_study_example}
\end{figure}

\begin{figure}
    \centering
    \includegraphics[width=\textwidth]{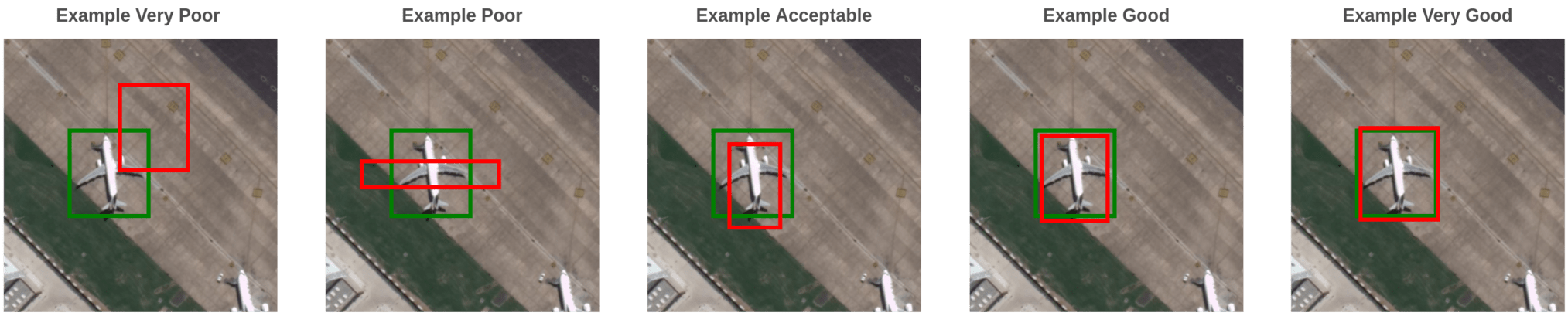}
    \caption[User study examples]{Examples given to the participants of the user
    study. The IoU between the green and red boxes are 0.1, 0.25, 0.5, 0.75, and
    0.9 for the ratings from very poor to very good respectively.}
    \label{fig:user_study_example}
\end{figure}

\subsubsection{General Statistics about Participants}
The study gathered 75 participants and a total of 3136 individual answers
(because some participants did not complete the entire experiment). The age of
the participants ranges from 21 to 64 years old with an average of 31.
Approximately half (37) of the participants are versed in computer vision
or image analysis, we will refer to this group as the \textit{expert group}. On
average, the response time is 10.3s per evaluated image during the first phase
of the experiment (when a background image is visible). It drops to 7.2s when
the image is replaced with a uniform background during the second phase. This
time difference suggests that humans do take into account the contextual
information of the image inside their decision-making process, which agrees with
the findings of \cite{strafforello2022humans}.

\subsection{User Study Insights} 
Human perception does not fully align with IoU. People tend to be more lenient
than IoU towards small objects. Specifically, comparing a small and a large box
with the same IoU with respect to their own ground truth, people will rate the
small one better. This suggests that IoU is too strict for small objects in
comparison with human perception. From a human perspective, precise localization
seems less important for small objects. \cref{fig:rating_vs_crit} represents the
relative gap of IoU (left) and SIoU (right) values for each object size and
rating. The relative differences $c_{s,r}$ are computed against the average IoU (or SIoU)
value per rating: 
\begin{equation}
    c_{s,r} =
\frac{\mathfrak{C}_{s,r} -
\sum\limits_s\mathfrak{C}_{s,r}}{\sum\limits_s\mathfrak{C}_{s,r}},
\end{equation}
\noindent
where $\mathfrak{C}_{s,r}$ is the average criterion value ($\mathfrak{C} \in
\{\text{IoU}, \text{SIoU}\}$) for objects of size $s$ and rating $r$. IoU values
for small objects (in orange) are lower than for large objects (in red) for all
rating $r$. For a human to give a rating $r$ to a box, it requires that a box
overlaps less with the ground truth (according to IoU) if the boxes are small.
SIoU compensates for this trend (see \cref{fig:rating_vs_crit}). 

However, according to SIoU, the same overlapping value triggers the same human
rating no matter the object sizes (see \cref{fig:rating_vs_crit}). Or at least
the required criterion value gap between small and large objects is reduced. We
can observe that it is more difficult to compensate the gaps for higher rating
values (especially with $r=4$). This means that human appraisal is not
identically in favor of small objects. Instead, it seems more and more in favor
of the small objects until the boxes are nearly identical (rating $r=5$). It
would be relevant to extend SIoU further to take this into account and achieve
even better alignment with human perception. Anyway, SIoU is much better aligned
than IoU with the human rating as for a rating $r$, the difference of SIoU for
small and large objects is below 5\% while with IoU, it is always above 15\%.
This means that with IoU, if we have two predicted boxes for one ground truth,
the IoUs of the predicted boxes with the ground truth can vary from about 15\%
without changing the human rating. Such a difference in SIoU would result in
different human ratings for the two predicted boxes. This phenomenon is much
more problematic with $\alpha$-IoU and NWD. This makes them poor choices for
the evaluation process as they are poorly aligned with human perception.

Similar charts are available in \cref{fig:rating_vs_crit_siou_gamma}, only with SIoU
but with various values of $\gamma$. In the previous paragraph, we set
$\gamma=0.2$. Choosing higher $\gamma$ values would reverse the trend and
produce a criterion even more lenient than humans for small objects. It will
also decrease further SIoU's variance. However, this setting has been chosen to
maximize the alignment with human perception. SIoU with $\gamma=0.2$ correlates
better with human rating compared with other criteria. As the rating is an
ordered categorical variable, we use the Kendall rank correlation to make the
comparison. The correlation between the human rating $r$ and each criterion can
be found in \cref{tab:correlation}. SIoU with $\gamma=0.2$ and $\kappa=64$
aligns best with human perception and has a low variance (see
\cref{sec:gamma_kappa}). This showcases the superiority of SIoU over existing
criteria. It should be preferred over IoU to assess the performance of models on
all visual tasks that employ IoU within their evaluation process. It
supports recent findings that show misalignment between IoU and human preference
\cite{strafforello2022humans}.

\begin{figure}[t]
    \centering
    \includegraphics[width=\textwidth]{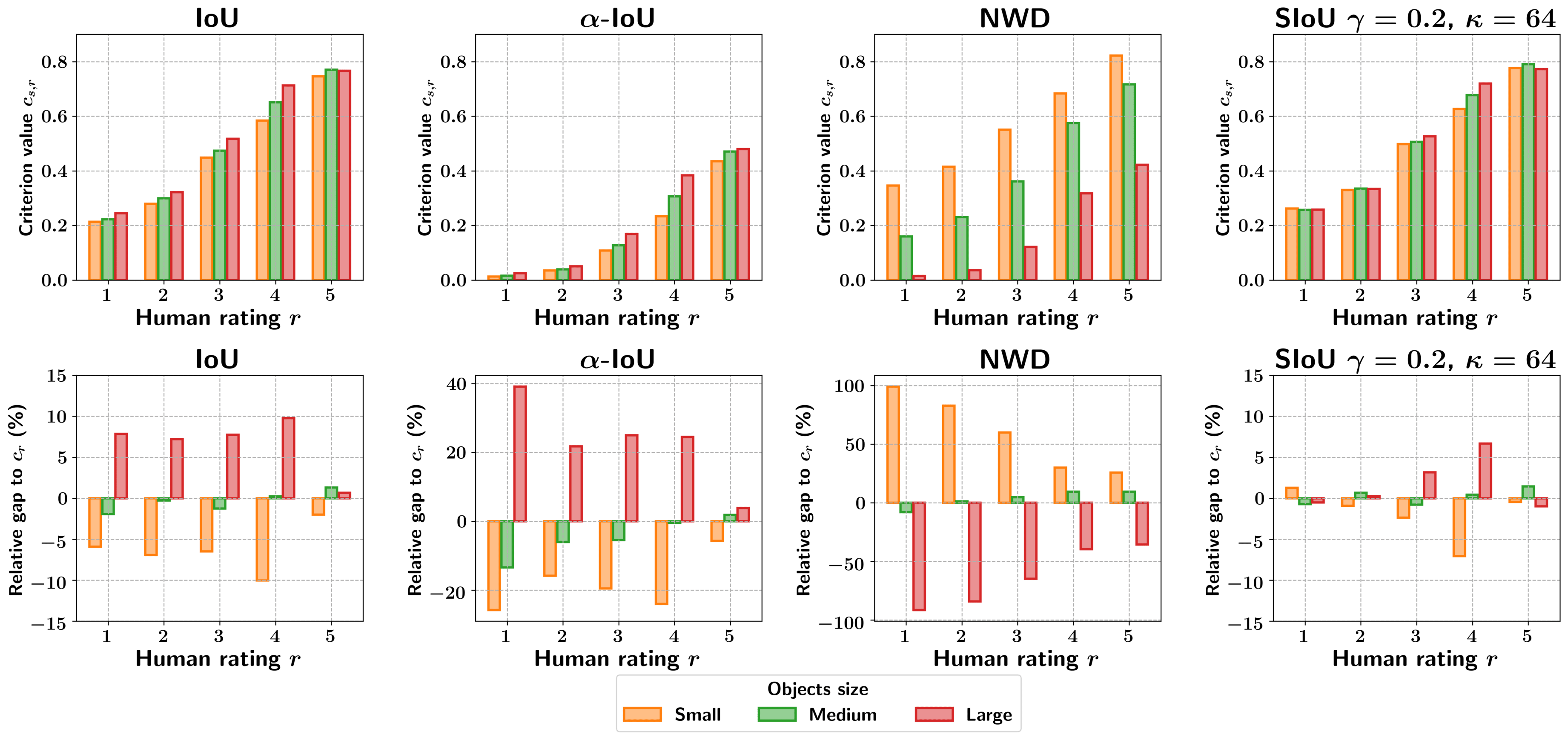}
    \caption[Criteria' alignment with human ratings]{Criteria' scores for different object sizes and human ratings $r
    \in \{1,2,3,4,5\}$ (\textbf{top}). Relative gap with the criterion value
    averaged over the object sizes ($c_{r} = 1/3(c_{S,r} + c_{M,r} + c_{L,r})$) (\textbf{bottom}).}
    \label{fig:rating_vs_crit}
\end{figure}

\begin{figure}[t]
    \centering
    \includegraphics[width=\textwidth]{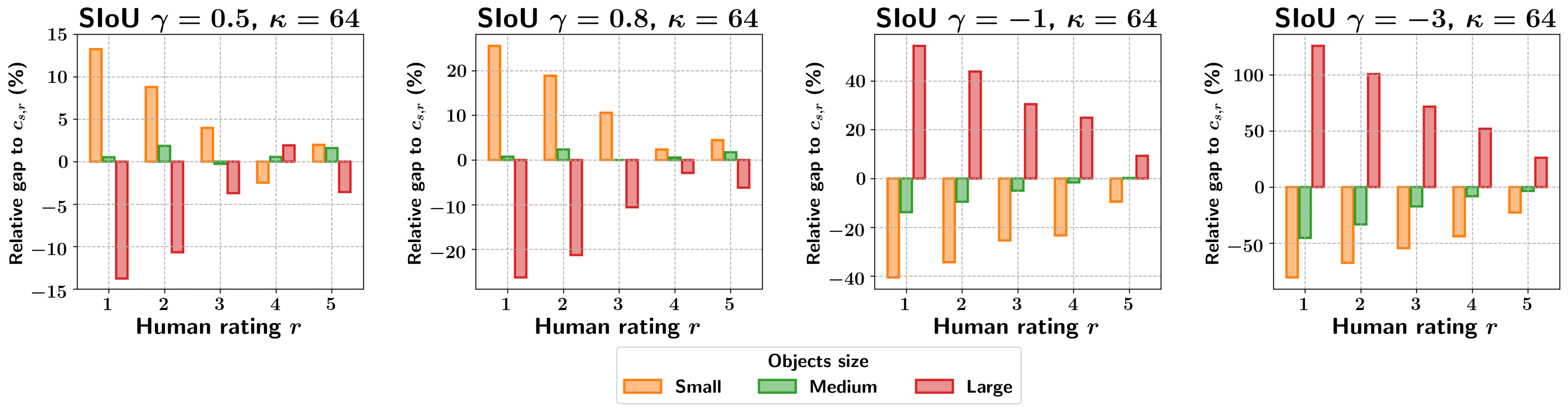}
    \caption[SIoU alignment with various $\gamma$ values]{Relative gap with the SIoU values averaged over the object sizes,
    for various $\gamma$ values.}
    \label{fig:rating_vs_crit_siou_gamma}
\end{figure}

\begin{table}[]
    \centering
    \resizebox{0.5\textwidth}{!}{%
    \begin{tabular}{@{\hspace{2mm}}cccccc@{\hspace{2mm}}}
    \toprule[1pt]
    \textbf{}                            &$\boldsymbol{r}$ & \textbf{IoU}  & \textbf{SIoU}       & $\boldsymbol{\alpha}$\textbf{-IoU} & \textbf{NWD} \\ \bottomrule
    $\boldsymbol{r}$                     & 1.000           & 0.674         & \textbf{0.701}      & 0.674                              & 0.550        \\ \midrule
    \textbf{IoU}                         & 0.674           & 1.000         & 0.892               & 0.997                              & 0.474        \\
    \textbf{SIoU}                        & 0.701           & 0.892         & 1.000               & 0.892                              & 0.576        \\
    $\boldsymbol{\alpha}$\textbf{-IoU}   & 0.674           & 0.997         & 0.892               & 1.000                              & 0.472        \\
    \textbf{NWD}                         & 0.550           & 0.474         & 0.576               & 0.472                              & 1.000        \\ \bottomrule[1pt]
    \end{tabular}%
    }
    \caption[Criteria' correlation with human rating]{Kendall's $\tau$
    correlation between various criteria and human rating $r$. For SIoU,
    $\gamma=0.2$ and $\kappa=64$, for $\alpha$-IoU, $\alpha=3$.}
    \label{tab:correlation}
\end{table}

\subsubsection{Factor Analysis}
To validate our previous experiments, we investigate the potential
influence of external factors on human ratings which could bias the results.
Specifically, we are interested here in several variables: the object size, the
presence of contextual information, the expertise, and the age of the
participants. Of course, the object size seems to have an influence on the human
rating (according to the previous section), but the idea is to show it
quantitatively. To this end, \cref{tab:avg_rating} gathers the average rating
$r$ under different groupings (by object size, presence of contextual
information, the expertise of the participants, and age of the participants). In
addition, the average value for each criterion is given for each group. IoU
value is close to 0.5 for every group as expected (boxes were chosen to have a
uniform IoU distribution). However, values of other criteria vary from one group
to another. This is especially true for scale-dependent criteria (SIoU and NWD)
on different object size groups. To check whether the different groups are
statistically different, we conducted one-way ANOVA tests on the four variables
from \cref{tab:avg_rating}. The results confirm that the mean ratings for
various object sizes are statistically different ($p<8.4\times 10^{-26}$
\footnote{$p$ stands here for the $p$-value of the statistical test here, not
the exponent from SIoU's definition.}). The tests find no statistical differences
for the participant expertise ($p<0.47$), and the presence of contextual
information ($p<0.28$). However, there is a significant difference between age
groups ($p<0.02$), but its influence on the ratings is limited compared
to object sizes. This confirms values inside \cref{tab:avg_rating} as the older
age groups tend to give lower ratings. Our goal here is not to infer anything
about the reasons for this difference, but this fact should be kept in mind
before drawing any conclusion. In addition, this suggests the need of distinct
alignments given what population group is the end user of a system. Having a
parameterizable evaluation process (\eg with SIoU) could help design models that
better satisfy their users.

\begin{table}[]
    \centering
    \resizebox{0.70\textwidth}{!}{%
    \begin{tabular}{@{\hspace{2mm}}ccccccc@{\hspace{2mm}}}
    \toprule[1pt]
    \textbf{}                             & \textbf{}     & $\boldsymbol{r}$ & \textbf{IoU} & \textbf{SIoU} & \textbf{NWD} & $\boldsymbol \alpha$\textbf{-IoU} \\ \midrule[1pt]
    \multirow{3}{*}{\textbf{Object size}} & Small         & 3.406                            & 0.507                             & 0.550                            & 0.610                                   & 0.203                \\
                                          & Medium        & 3.158                            & 0.502                             & 0.532                            & 0.424                                   & 0.199                \\
                                          & Large         & 2.824                            & 0.491                             & 0.500                            & 0.151                                   & 0.189                \\\midrule
    \multirow{2}{*}{\textbf{Context}}     & w/o context   & 3.144                            & 0.504                             & 0.531                            & 0.397                                   & 0.197                \\
                                          & w/ context    & 3.112                            & 0.496                             & 0.523                            & 0.390                                   & 0.197                \\\midrule
    \multirow{2}{*}{\textbf{Expertise}}   & Inexperienced & 3.104                            & 0.493                             & 0.520                            & 0.392                                   & 0.194                \\
                                          & Expert        & 3.152                            & 0.507                             & 0.535                            & 0.395                                   & 0.200                \\\midrule
    \multirow{3}{*}{\textbf{Age}}         & (10, 25]      & 3.215                            & 0.504                             & 0.531                            & 0.397                                   & 0.196                \\
                                          & (25, 40]      & 3.078                            & 0.496                             & 0.524                            & 0.390                                   & 0.198                \\
                                          & (40, 65]      & 3.085                            & 0.501                             & 0.529                            & 0.394                                   & 0.197                \\ \bottomrule[1pt]
    \end{tabular}%
    }
    \caption[Influence of external variables on human rating]{Average rating and criteria values for different groupings of the
    variables of interest (object size, presence of contextual information,
    expertise and age of the participants).}
    \label{tab:avg_rating}
    \end{table}

To visualize better the alignment of the various criteria with the human
perception, \cref{fig:hr_vs_crit} plots the rating values against the criteria
value. For clarity, random vertical shifts are added to rating
values to distinguish between the values of each variable and data
points. From this figure, it is clear that the IoU is not a perfect criterion as
a wide range of IoU values is attributed to the same rating value. It seems also
clear that contextual information and participant expertise do not introduce
much change in the human rating. However, age does have a small influence on the
rating, but this is more blatant with the object size: the average IoU value for
a rating $r$ decreases with the object size (this is visible with the black
vertical lines in \cref{fig:rating_vs_crit}). This completely agrees with the
statistical test results. SIoU compensates for this trend and produces more
aligned averages for the different object sizes. NWD has the same effect but
largely reverses the trend in the other direction. 

\begin{figure}
    \centering
    \includegraphics[width=\textwidth]{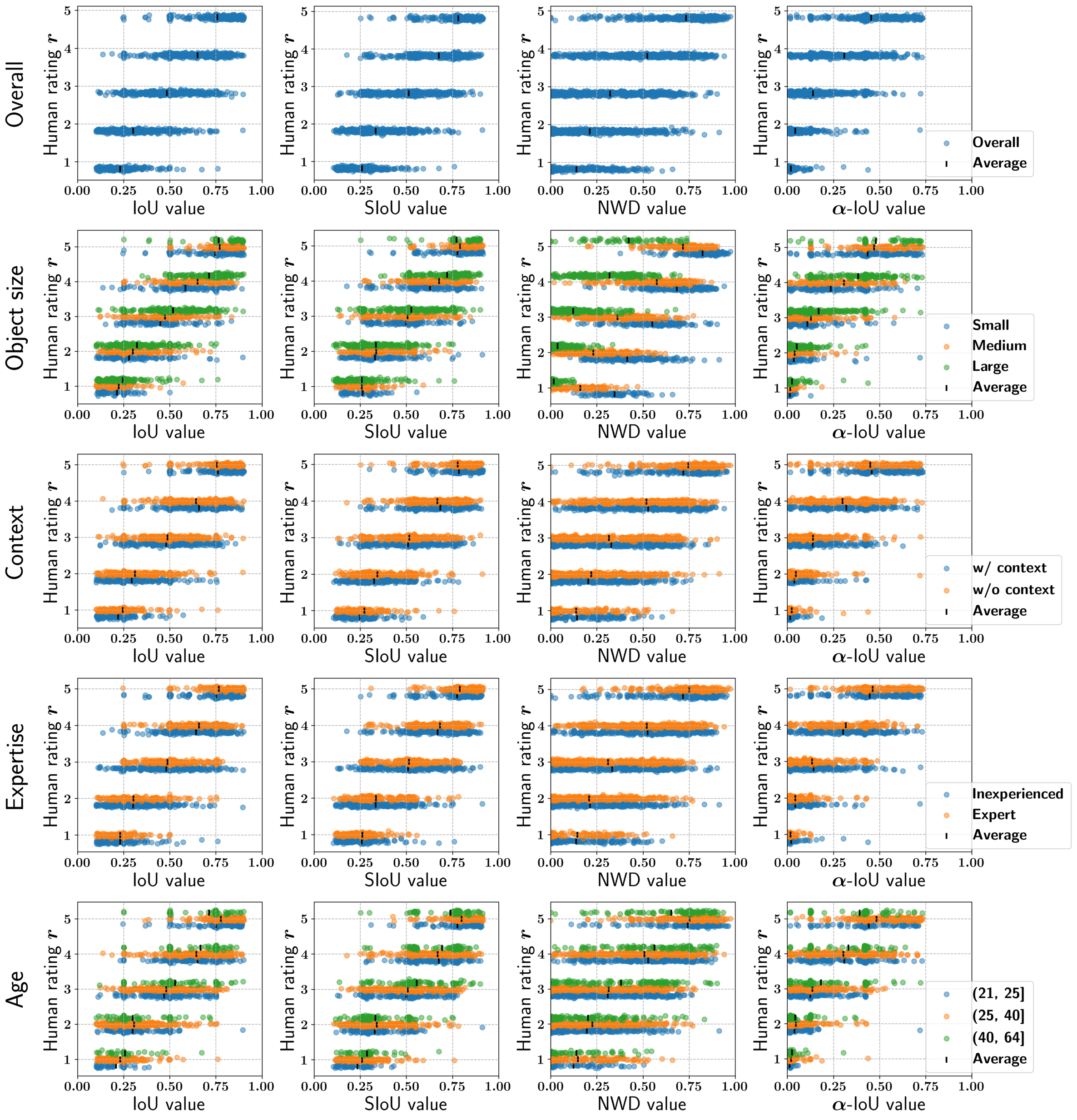}
    \caption[User study factor analysis]{Rating against IoU, SIoU ($\gamma=0.2$, $\kappa=64$), NWD and
    $\alpha$-IoU ($\alpha=3$) values, overall and for different groupings of the
    variables of interest (object size, presence of contextual information,
    expertise and age of the participants). Colors represent different values for
    each variable. A legend for each row is included in the right-most column
    of the figure. For the \textit{Age} variable, the participants have been
    separated into three groups of the same size. }
    \label{fig:hr_vs_crit}
    \vspace{2cm}
\end{figure}

\begin{table}[h]
    \centering
    \resizebox{0.5\columnwidth}{!}{%
    \begin{tabular}{@{\hspace{2mm}}cccccc@{\hspace{2mm}}}
    \toprule[1pt]
    \textbf{}                        & $\boldsymbol{r}$ & \textbf{IoU} & \textbf{SIoU} & \textbf{NWD} & $\boldsymbol \alpha$\textbf{-IoU} \\ \midrule[1pt]
    \multirow{5}{*}{\parbox{2cm}{\textbf{Small \\Objects}}}  & 1         & 0.214        & 0.262         & 0.346        & 0.013               \\
                                     & 2         & 0.279        & 0.330         & 0.415        & 0.035               \\
                                     & 3         & 0.449        & 0.498         & 0.551        & 0.108               \\
                                     & 4         & 0.584        & 0.627         & 0.683        & 0.234               \\
                                     & 5         & 0.746        & 0.776         & 0.822        & 0.435               \\ \midrule
    \multirow{5}{*}{\parbox{2cm}{\textbf{Medium \\Objects}}} & 1         & 0.223        & 0.257         & 0.160        & 0.016               \\
                                     & 2         & 0.299        & 0.335         & 0.230        & 0.039               \\
                                     & 3         & 0.474        & 0.506         & 0.361        & 0.127               \\
                                     & 4         & 0.651        & 0.677         & 0.575        & 0.306               \\
                                     & 5         & 0.771        & 0.791         & 0.716        & 0.470               \\\midrule
    \multirow{5}{*}{\parbox{2cm}{\textbf{Large \\Objects}}}  & 1         & 0.245        & 0.258         & 0.015        & 0.025               \\
                                     & 2         & 0.322        & 0.334         & 0.036        & 0.051               \\
                                     & 3         & 0.517        & 0.527         & 0.121        & 0.168               \\
                                     & 4         & 0.713        & 0.720         & 0.318        & 0.383               \\
                                     & 5         & 0.766        & 0.772         & 0.422        & 0.480               \\\bottomrule[1pt]
    \end{tabular}%
    } \caption[Average criteria values for various object size]{Average criteria
    (IoU, SIoU, NWD and $\alpha$-IoU) values for different object sizes and
     ratings.}
    \label{tab:avg_iou_object_size}
    \end{table}

\section{Experimental Results}
\label{sec:results}
\vspace{-1.5em}
To support our analysis from \cref{sec:criterion_analysis}, we conduct various
experiments, mainly on aerial images with DOTA and DIOR datasets. To showcase
the versatility of SIoU, we also experiment with natural images on Pascal VOC
and COCO datasets. We would like to emphasize that the goal of SIoU, as a loss, is
primarily to improve the performance of FSOD and not regular object detection.
While SIoU is beneficial for any detection task for evaluation purposes, it is
designed to address the extreme challenge of detecting small objects in the
few-shot regime. Therefore, most of our experiments focus on the few-shot
setting. However, we also report results in regular object detection to display
the potential of SIoU. For the few-shot experiments, we choose Cross-Scale Query
Support Alignment (XQSA) (see \cref{sec:xscale}) as a baseline, but a comparison
with the other attention methods from \cref{chap:aaf} is available in
\cref{sec:fs_experiment_siou}. 

\subsection{Comparison with Existing Criteria}
To begin, we compare the few-shot performance on DOTA with various loss
functions based on the criteria discussed in \cref{sec:criterion_analysis}. The
result of these experiments is available in \cref{tab:criterion_comparison}. The
criteria are divided into two groups, generalized (\ie which is not 0 when boxes
do not overlap; it therefore includes NWD) and vanilla criteria. As discussed
in \cref{sec:empirical_analysis}, the generalized versions of the criteria
outperform their original counterparts and therefore should be compared
separately. Scale-adaptive criteria (SIoU and GSIoU) largely outperform other
losses on novel classes and especially on small objects. For SIoU and GSIoU, we
choose $\gamma=-3$ and $\kappa=16$ according to a series of experiments
conducted on DOTA to determine their optimal values (see
\cref{sec:gamma_kappa}). It is important to point out the relatively good
performance of NWD despite not checking all the desirable properties highlighted
in \cref{sec:criterion_analysis}.

\begin{table}[]
    \centering
    \resizebox{0.75\textwidth}{!}{%
    \begin{tabular}{@{\hspace{2mm}}cccccccccc@{\hspace{2mm}}}
    \toprule[1pt]
    \multicolumn{1}{l}{\textbf{}} & \multicolumn{4}{c}{\textbf{Base classes}}          & & \multicolumn{4}{c}{\textbf{Novel Classes}}          \\ 
    \multicolumn{1}{l}{\textbf{Loss}} & \textbf{All} & \textbf{S} & \textbf{M} & \textbf{L}& & \textbf{All} & \textbf{S} & \textbf{M} & \textbf{L} \\ \midrule[1pt]
    \textbf{IoU}                  & 50.67        & \bbf{25.83}      & 57.49      & 68.24     & & 32.41        & 10.06      & 47.87      & 67.09      \\
    $\alpha$\textbf{-IoU}         & 46.72        & 13.24      & 55.21      & \bbf{69.94}     & & 33.95        & 12.58      & 46.58      & \rbf{74.50}      \\ 
    \textbf{SIoU}                 & \bbf{53.62}        & 24.07      & \bbf{61.91}      & 67.34     & & \rbf{39.05}        & \rbf{16.59}      & \rbf{54.42}      & \rbf{74.49}      \\ \midrule
    \textbf{NWD}                  & 50.79        & 19.19      & 58.90      & 67.90     & & 41.65        & 28.26      & 50.16      & 65.06      \\
    \textbf{GIoU}                 & 52.41        & \bbf{26.94}      & 61.17      & 63.00     & & 41.03        & 24.01      & \rbf{52.13}      & 69.78      \\
    \textbf{GSIoU}                & \bbf{52.91}        & 22.14      & \bbf{61.19}      & \bbf{66.02}     & & \rbf{45.88}        & \rbf{34.83}      & 51.26      & \rbf{70.78}      \\ \bottomrule[1pt]
    \end{tabular}%
    }
    \caption[Experimental comparison of SIoU with other criteria]{Few-shot performance comparison between several criteria: IoU,
    $\alpha$-IoU, SIoU, NWD, GIoU, and GSIoU trained on DOTA. mAP is reported
    with a 0.5 IoU threshold for small (S), medium (M), large (L), and all objects.}
    \label{tab:criterion_comparison}
    \vspace{-0.5em}
    \end{table}

\subsection{Application on Aerial and Natural Images}
As the previous set of experiments was only carried out on DOTA, we showcase the
versatility of GSIoU on three other datasets: DIOR, Pascal VOC and COCO. As it
is clear that generalized criteria achieve higher performance, the comparison
here is only done between GIoU and GSIoU. The results of this comparison are
available in \cref{tab:siou_dataset_comparison}.

The large improvements found for DOTA translate to DIOR as well, especially for
small objects. For Pascal VOC and COCO, similar gains are observed for small
objects, but the improvement is limited overall (\ie disregarding the object
size) as natural datasets contain larger objects. It is worth mentioning that
these improvements on Pascal VOC and COCO require a different tuning of SIoU.
$\gamma=-3$ and $\kappa=16$ produce mitigated results with these datasets, and
$\gamma=-1$ and $\kappa=64$ is a more sensible choice. This was predictable as
the objects in Pascal VOC and COCO are substantially larger than in DOTA and
DIOR. This can also explain the slightly smaller gains on DIOR compared to DOTA.
Finding optimal values of $\gamma$ and $\kappa$ could yield slightly better
performance on DIOR. The right balance depends on the proportion of small,
medium and large objects in the datasets. With natural images which contain
fewer small objects, the training balance does not need to be shifted as much as
for aerial images. 

Obviously, one may see this as a constraint: SIoU
introduces two novel hyper-parameters and their values change depending on which
dataset is used. However, the tuning of SIoU is straightforward, as lower values
of $\gamma$ and $\kappa$ skew the training towards smaller objects, and in
practice few experiments are enough to find near-optimal values. Given the
impressive gains obtained on small objects in the few-shot regime, it is
certainly worth it. Clearly, it would be even better to extend SIoU to be
parameter-free or to have a pre-defined way of setting $\gamma$ and $\kappa$,
for instance, based on the object size distribution inside a dataset.

\begin{table}[]
    \centering
    \resizebox{0.8\textwidth}{!}{%
    \begin{tabular}{@{\hspace{2mm}}ccccccccccc@{\hspace{2mm}}}
    \toprule[1pt]
    \textbf{}                        & \textbf{} & \multicolumn{4}{c}{\textbf{Base classes}}          & & \multicolumn{4}{c}{\textbf{Novel Classes}} \\ 
    \textbf{}                        & \textbf{XQSA} & \textbf{All} & \textbf{S} & \textbf{M} & \textbf{L}& & \textbf{All}  & \textbf{S}  & M     & L    \\\midrule[1pt]
    \multirow{2}{*}{\textbf{DOTA}}   &  w/ GIoU      & 52.41        & \bbf{26.94}      & 61.17      & 63.00     & & 41.03          & 24.01         & \rbf{52.13}  & 69.78 \\
                                     &  w/ GSIoU     & \bbf{52.91}        & 22.14      & \bbf{61.19}      & \bbf{66.02}     & & \rbf{45.88}          & \rbf{34.83}         & 51.26  & \rbf{70.78} \\ \midrule
    \multirow{2}{*}{\textbf{DIOR}}   &  w/ GIoU      & 58.90        & 10.38      & 40.76      & 80.44     & & 47.93          & 9.85          & 47.61  & 68.40 \\
                                     &  w/ GSIoU     & \bbf{60.29}        & \bbf{11.28}      & \bbf{43.24}      & \bbf{81.63}     & & \rbf{52.85}          & \rbf{13.78}         & \rbf{53.73}  & \rbf{71.22} \\ \midrule
    \multirow{2}{*}{\textbf{Pascal}} &  w/ GIoU      & 51.09        & \bbf{13.93}      & \bbf{40.26}      & 62.01     & & 48.42          & 18.44         & 36.06  & 59.99 \\
                                     &  w/ GSIoU     & \bbf{54.47}        & 13.88      & 40.13      & \bbf{66.82}     & & \rbf{55.16}          & \rbf{22.94}         & \rbf{36.24}  & \rbf{67.40} \\ \midrule
    \multirow{2}{*}{\textbf{COCO}}&  w/ GIoU      & 19.15        & \bbf{8.72}     & 22.50      & 30.59                        & & 26.25          & 11.96         & 23.95  & 38.60 \\
                                     &  w/ GSIoU     & \bbf{19.57}        & 8.41      & \bbf{23.02}      & \bbf{31.07}     & & \rbf{27.11}          & \rbf{12.81}         & \rbf{26.02}  & \rbf{39.20} \\ \bottomrule[1pt]
    \end{tabular}%
    } \caption[GSIoU Loss performance on DOTA, DIOR, Pascal VOC and
    COCO]{Few-shot performance on four datasets: DOTA, DIOR, Pascal VOC and
    COCO. GIoU and GSIoU losses are compared. mAP is reported with a 0.5 IoU
    threshold and for all object sizes.}
    \vspace{-1.5em}
    \label{tab:siou_dataset_comparison}
    \end{table}

\subsection{$\gamma$ and $\kappa$ influence on FSOD Performance}
As discussed above, the setting of $\gamma$ and $\kappa$ is
crucial for training. Therefore, we conducted various experiments on DOTA to
find the best parameters for GSIoU loss. We intentionally include extreme values
of $\gamma$ and $\kappa$ to demonstrate the behavior of SIoU. The results can be
found in \cref{tab:gamma_perf,tab:kappa_perf}. This shows that the optimal
values for DOTA dataset are $\gamma=-3$ and $\kappa=8$. However, $\kappa=16$ is
also a good choice and is more consistent across datasets. Thus, we choose to
keep $\gamma=-3$ and $\kappa=16$ for other experiments. An exhaustive grid
search should be done to find even better settings. Our search was sparse and a
better combination of $\gamma$ and $\kappa$ probably exists, yet our sub-optimal
setup already yields significant improvement for small object detection in the
FS regime.
    
    \begin{table}[h]
        \centering
        \resizebox{0.73\textwidth}{!}{%
        \begin{tabular}{@{\hspace{2mm}}ccccccccc@{\hspace{2mm}}}
            \toprule[1pt]
            \multicolumn{1}{l}{}      & \multicolumn{4}{c}{\textbf{Base classes}}                         & \multicolumn{4}{c}{\textbf{Novel Classes}}                                                                 \\ \midrule
            $\bm{\gamma}$ & \textbf{All}   & \textbf{S}     & \textbf{M}     & \textbf{L}     & \textbf{All}   & \textbf{S}                  & \textbf{M}                   & \textbf{L}                   \\
            \textbf{0.5}                       & 47.09          & 21.29          & 54.67          & 65.48          & 30.50          & 8.83                        &44.97                         & 65.89 \\
            \textbf{0.25}                      & 45.94          & 21.60          & 54.39          & 63.40          & 30.96          & 12.53                       & 42.37                        & 64.14                        \\
            \textbf{0}                         & 52.41          & 26.94          & 61.17          & 63.00          & 41.03          & 24.01                       & 52.13                        & 69.78                        \\
            \textbf{-0.5}                      & 52.80          & \bbf{27.16} & 61.19          & 64.61          & 41.06          & 25.20                       & 50.18                        & \rbf{72.04}               \\
            \textbf{-1}                        & 53.03          & 23.20          & 61.53          & 66.68          & 42.77          & 27.55                       & \rbf{52.01}               & 70.76                        \\
            \textbf{-2}                        & \bbf{54.06} & 23.68          & \bbf{62.69} & 66.62          & 43.67          & 30.04                       & 51.69                        & 69.66                        \\
            \textbf{-3}                        & 52.91          & 22.14          & 61.19          & 66.02          & \rbf{45.88} & \rbf{34.83}              & 51.26                        & 70.78                        \\
            \textbf{-4}                        & 53.59          & 22.50          & 62.48          & 66.18          & 42.43          & 27.56                       & 51.79                        & 68.70                        \\
            \textbf{-9}                        & 53.11          & 20.98          & 62.13          & \bbf{67.00} & 42.63          & 30.53                       & 48.89                        & 68.62                        \\ \bottomrule[1pt]
            \end{tabular}%
        } \caption[Influence of $\gamma$ on the FSOD performance]{Evolution of
        the few-shot performance (XQSA with GSIoU loss) on DOTA for various
        values of $\gamma$ ($\kappa=16$ is fixed). mAP is reported with a 0.5 IoU
        threshold and for all object sizes.}
        \label{tab:gamma_perf}
        \end{table}

     \begin{table}[h]
        \centering
        \resizebox{0.73\textwidth}{!}{%
        \begin{tabular}{@{\hspace{2mm}}ccccccccc@{\hspace{2mm}}}
        \toprule[1pt]
        \multicolumn{1}{l}{\textbf{}} & \multicolumn{4}{c}{\textbf{Base classes}}           & \multicolumn{4}{c}{\textbf{Novel Classes}}          \\ \midrule
        \textbf{$\boldsymbol \kappa$}                & \textbf{All} & \textbf{S} & \textbf{M} & \textbf{L} & \textbf{All} & \textbf{S} & \textbf{M} & \textbf{L} \\
        \textbf{4}                    & 51.65        & 21.50      & 59.76      & 65.85      & 42.98        & 30.33      & 48.57      & \rbf{73.41}      \\
        \textbf{8}                    & 52.70        & 21.96      & 61.49      & 66.43      & \rbf{44.16}        & \rbf{31.35}      & 50.70      & 71.99      \\
        \textbf{16}                   & \bbf{54.06}        & \bbf{23.68}      & 62.69      & 66.62      & 43.67        & 30.04      & 51.69      & 69.66      \\
        \textbf{32}                   & 53.88        & 22.33      & \bbf{63.00}      & \bbf{67.35}      & 37.36        & 23.65      & 44.60      & 66.29      \\
        \textbf{64}                   & 52.82        & 21.79      & 61.46      & 66.77      & 43.68        & 29.43      & \rbf{52.47}      & 69.46      \\
        \textbf{128}                  & 53.42        & 21.73      & 62.90      & 66.75      & 41.32        & 26.85      & 49.40      & 70.38      \\ \bottomrule[1pt]
        \end{tabular}%
        } \caption[Influence of $\kappa$ on the FSOD performance]{Evolution of
        the few-shot performance (XQSA with GSIoU loss) for various values of
        $\kappa$ ($\gamma= -2$ is fixed).}
        \label{tab:kappa_perf}
        \vspace{-2mm}
        \end{table}

    \paragraph*{Influence of $\boldsymbol \gamma$ and $\boldsymbol \kappa$ on FSOD performance on Pascal VOC dataset}
    The search conducted above was carried out on DOTA dataset. It transposes nicely
    on DIOR dataset as well. Yet, these two datasets are similar. They
    both contain aerial images with some classes in common, but most importantly,
    the size of their objects are highly similar. When applied to different
    datasets, these results may not hold. For instance, Pascal VOC requires other
    combinations of $\gamma$ and $\kappa$ to outperform the training with GIoU. This is
    shown in \cref{tab:pascal_search}.

    \begin{table}[h]
        \centering
        \resizebox{0.88\textwidth}{!}{%
        \begin{tabular}{@{\hspace{2mm}}ccccccccc@{\hspace{2mm}}}
        \toprule[1pt]
                                        & \multicolumn{4}{c}{\textbf{Base Classes}} & \multicolumn{4}{c}{\textbf{Novel Classes}} \\ \midrule
        \textbf{Loss function}          & \textbf{All}    & \textbf{S}      & \textbf{M}      & \textbf{L}     & \textbf{All}    & \textbf{S}      & \textbf{M}      & \textbf{L}      \\
        \textbf{GIoU}                           & 51.09  & \bbf{13.93}  & 40.26  & 62.01 & 48.42  & 18.44  & 36.06  & 59.99  \\
        \textbf{GSIoU} $\boldsymbol \gamma=-3$, $\boldsymbol \kappa=16$  & 45.22  & 10.06  & 34.85  & 57.10 & 43.16  & 14.89  & 33.92  & 54.16  \\
        \textbf{GSIoU} $\boldsymbol \gamma=-1$, $\boldsymbol \kappa=64$  & 54.47  & 13.88  & 40.13  & 66.82 & 55.16  & \rbf{22.94}  & 36.24  & 67.40  \\
        \textbf{GSIoU} $\boldsymbol \gamma=0.5$, $\boldsymbol \kappa=64$ & \bbf{56.97}  & 13.88  & \bbf{40.75}  & \bbf{70.31} & \rbf{55.36}  & 20.25  & \rbf{36.85}  & \rbf{68.05}  \\ \bottomrule[1pt]
        \end{tabular}
        } \caption[Influence of $\gamma$ and $\kappa$ on Pascal VOC]{Few-shot
        performance on Pascal VOC dataset with different values of $\gamma$ and
        $\kappa$.}
        \label{tab:pascal_search}
        \vspace{-2mm}
        \end{table}
    
\subsection{Additional Experiments with GSIoU Loss}
\label{sec:experimental_results}

\subsubsection{Changing the Few-Shot Approach}
\label{sec:fs_experiment_siou}
To support the versatility of GSIoU, we also experiment with several few-shot
approaches. We selected three FSOD techniques that we implemented within the AAF
framework: Feature Reweighting \cite{kang2019few} (FRW),  Dual-Awareness Atention
\cite{chen2021should} (DANA) and our Cross-scale  Query-Support Alignment
(XQSA). We train all of them with GIoU and GSIoU as regression losses and
provide the results in \cref{tab:fsod_method_com}. Except for DANA, GSIoU
sensibly improves detection performance, especially for small objects. The
results with DANA are surprising, and it would be of great interest to
investigate the reasons behind this below-par performance. 
        
\begin{table}[ht]
    \centering
    \resizebox{0.85\textwidth}{!}{%
    \begin{tabular}{@{\hspace{2mm}}ccccccccccc@{\hspace{2mm}}}
    \toprule[1pt]
    \textbf{}                      & \textbf{} & \multicolumn{4}{c}{\textbf{Base classes}}          & & \multicolumn{4}{c}{\textbf{Novel Classes}}          \\ 
    \textbf{}                      & \textbf{XQSA} & \textbf{All} & \textbf{S} & \textbf{M} & \textbf{L}& & \textbf{All} & \textbf{S} & \textbf{M} & \textbf{L} \\ \midrule[1pt]
    \multirow{2}{*}{\textbf{FRW}}  & w/ GIoU      & \bbf{34.60}        & \bbf{16.15}      & \bbf{48.61}      & \bbf{59.00}     & & 32.00        & 15.29      & \rbf{44.50}      & 54.77      \\
                                    & w/ GSIoU     & 30.36        & 11.94      & 44.30      & 54.87     & & \rbf{32.94}        & \rbf{16.69}      & 42.87      & \rbf{62.64}      \\ \midrule
    \multirow{2}{*}{\textbf{DANA}} & w/ GIoU      & 48.09        & 27.34      & \bbf{66.06}      & \bbf{68.00}     & & \rbf{44.49}        & \rbf{30.10}      & 52.24      & 74.40      \\
                                    & w/ GSIoU     & \bbf{50.10}        & \bbf{32.19}      & 65.46      & 67.77     & & 41.40        & 21.07      & \rbf{54.80}      & \rbf{75.23}      \\ \midrule
    \multirow{2}{*}{\textbf{XQSA}} & w/ GIoU      & \bbf{45.30}        & \bbf{26.94}      & 61.17      & 63.00     & & 41.03        & 24.01      & \rbf{52.13}      & 69.78      \\
                                    & w/ GSIoU     & 43.42        & 22.14      & \bbf{61.19}      & \bbf{66.02}     & & \rbf{45.88}        & \rbf{34.83}      & 51.26      & \rbf{70.78}      \\ \bottomrule[1pt]
    \end{tabular}%
    } \caption[Application of SIoU loss with various FSOD methods]{Performance
    comparison with three different FSOD methods: Feature Reweighting
    \cite{kang2019few} (FRW), Dual Awareness Attention \cite{chen2021should}
    (DANA) and our Cross-scale Query-Support Alignment (XQSA), trained with GIoU
    and GSIoU. mAP is reported with a 0.5 \textbf{IoU threshold} for small (S),
    medium (M), large (L) and all objects.}
    \label{tab:fsod_method_com}
    \end{table}

\subsubsection{Regular Object Detection on DOTA and DIOR}
GSIoU is not only beneficial for FSOD, but it also improves the
performance of regular object detection methods. \cref{tab:regular_od} compares
the performance of FCOS \cite{tian2019fcos} trained on DOTA and DIOR with GIoU
and GSIoU. The same pattern is visible as we get better performance with GSIoU.
However, the gain for small objects is not as large as for FSOD. Nevertheless,
it suggests that other tasks relying on IoU could also benefit from GSIoU. 

Obviously, further experiments are required to showcase the superiority of
SIoU/GSIoU in regular data settings, especially with other detection frameworks
and datasets. We did investigate with YOLO and DiffusionDet, but we only
achieved mitigated results, even sometimes in favor of GIoU. One possibility
for this is the use of IoU in the example selection process. In most detection
frameworks, only the best predicted bounding boxes, according to the IoU, are
selected to compute the loss (this is detailed in \cref{sec:training_od} and
especially \cref{tab:od_matching_strat}). As mentioned previously, it would be
relevant to study the impact of SIoU on this process as well. However, FCOS, on
which the AAF framework is based, does not rely on IoU for the example
selection. Instead, it selects as positive examples all predicted boxes whose
center falls into a ground truth box. Hence, an IoU-based example selector
probably hinders the benefits of SIoU while the FCOS's selection is likely a
better match. This should be investigated in depth in future work. 

\begin{table}[]
    \centering
    \resizebox{0.65\columnwidth}{!}{%
    \begin{tabular}{@{\hspace{2mm}}cccccccccc@{\hspace{2mm}}}
    \toprule[1pt]
    \textbf{}      & \multicolumn{4}{c}{\textbf{DOTA}}                  & & \multicolumn{4}{c}{\textbf{DIOR}}          \\ 
    \textbf{FCOS}      & \textbf{All} & \textbf{S} & \textbf{M} & \textbf{L}& & \textbf{All} & \textbf{S} & \textbf{M} & L \\ \midrule
    \textbf{w/ GIoU}  & 34.9         & 17.4       & 36.6       & 43.3  & & 48.1        &10.1       &40.3       &63.2   \\
    \textbf{w/ GSIoU} & \textbf{36.8}         & \textbf{17.5}       & \textbf{40.4}       & \textbf{45.2}  & & \textbf{49.2}        & \textbf{11.0}      & \textbf{41.2}     &\textbf{66.1}   \\ \bottomrule[1pt]
    \end{tabular}%
    } \caption[SIoU for regular object detection]{Regular Object Detection
    performance on DOTA and DIOR datasets with GIoU and GSIoU ($\gamma=-3$ and
    $\kappa=16$) losses. mAP is computed with several IoU thresholds (0.5 to
    0.95) as it is commonly done in regular detection.}
    \label{tab:regular_od}
    \end{table}

\subsection{Evaluation with SIoU}
\label{sec:siou_eval}
In this section, we present some of the results reported in previous section using SIoU as
the evaluation criterion. Specifically, instead of choosing an IoU threshold to
decide if a box is a positive or negative detection, an SIoU threshold is
employed. For the sake of comparison, we kept the same thresholds as in
\cref{tab:criterion_comparison,tab:regular_od,tab:siou_dataset_comparison}, \ie 0.5
for Few-Shot methods and 0.5:0.95 for regular object detection. The results are
available in
\cref{tab:criteria_comp_siou,tab:dataset_comparison_siou,tab:regular_od_siou}.
The conclusions from \cref{sec:results} still hold, and the superiority of GSIoU
over other criteria is clear. However, a few changes are noticeable. First, SIoU loss
seems to perform better than IoU. This is expected since the model is directly
optimized to satisfy this criterion.
Then, when evaluated with SIoU, models trained with NWD perform well.
Indeed, NWD puts a lot of emphasis on size matching during
training, and less on position. Therefore, it is logical to observe better
performance compared to other losses when using SIoU as the evaluation criterion. 

One crucial point is that SIoU evaluation mostly changes the score for small
objects. SIoU behaves like IoU for large objects, therefore relatively small
changes are visible for medium and large objects. Overall, the scores are higher
than with IoU as the expected value of SIoU is higher than IoU. The important
point to note is that the gap between small and large objects performance is
reduced and aligns better with human perception.

\begin{table}[]
    \centering
    \resizebox{0.7\textwidth}{!}{%
    \begin{tabular}{@{\hspace{2mm}}cccccccccc@{\hspace{2mm}}}
    \toprule[1pt]
                       & \multicolumn{4}{c}{\textbf{Base classes}}          & & \multicolumn{4}{c}{\textbf{Novel Classes}}                           \\
    \textbf{Loss}      & \textbf{All} & \textbf{S} & \textbf{M} & \textbf{L}& & \textbf{All} & \textbf{S} & \textbf{M} & \textbf{L}  \\  \midrule[1pt]
    \textbf{IoU}       & 55.81        & 35.03      & 62.57      & 70.05     & & 39.10        & 18.58      & 53.93      & 68.83           \\
    \textbf{$\boldsymbol \alpha$-IoU} & 53.05        & 20.60      & 61.05     & \bbf{72.41}     & & 41.93        & 20.99      & 55.74   & 76.79         \\
    \textbf{SIoU}      & \bbf{59.77}        & 36.38      & 67.29      & 70.06     & & 49.51        & 31.06      & 62.53      & \rbf{77.24}           \\ \midrule
    \textbf{NWD}       & 58.80        & 34.16      & 66.81      & 70.05     & & 53.66        & 42.02      & 62.53      & 68.92           \\
    \textbf{GIoU}      & 59.27        & \bbf{44.07}      & \bbf{66.91}      & 65.46     & & 49.02        & 35.10      & 57.58      & 74.30           \\
    \textbf{GSIoU}     & 59.32        & 35.32      & 66.29      & 69.03     & & \rbf{57.70}        & \rbf{46.77}      & \rbf{65.56}      & 73.67           \\ \bottomrule[1pt]
    \end{tabular}%
    }
    \caption[Criteria comparison with SIoU threshold during evaluation]{Few-shot performance comparison between several criteria: IoU,
    $\alpha$-IoU, SIoU, NWD, GIoU and GSIoU trained on DOTA. mAP is reported
    with a 0.5 \textbf{SIoU threshold} for small (S), medium (M), large (L), and all objects.}
    \label{tab:criteria_comp_siou}
    \end{table}

\begin{table}[]
        \centering
        \resizebox{0.65\textwidth}{!}{%
        \begin{tabular}{@{\hspace{2mm}}cccccccccc@{\hspace{2mm}}}
        \toprule[1pt]
                                           & \multicolumn{4}{c}{\textbf{DOTA}}                                                                                           & \multicolumn{4}{c}{\textbf{DIOR}}                                                                                           \\ \midrule
        \textbf{FCOS}                          & \textbf{All}             & \multicolumn{1}{c}{\textbf{S}} & \multicolumn{1}{c}{\textbf{M}} & \multicolumn{1}{c}{\textbf{L}}& & \textbf{All}             & \multicolumn{1}{c}{\textbf{S}} & \multicolumn{1}{c}{\textbf{M}} & \multicolumn{1}{c}{\textbf{L}} \\
        \multicolumn{1}{c}{\textbf{w/ GIoU}}  & \multicolumn{1}{r}{43.9} & 27.4                           & 46.5                           & 47.2                          & & \multicolumn{1}{r}{54.5} & 17.6                           & 49.8                           & 66.4                           \\
        \multicolumn{1}{c}{\textbf{w/ GSIoU}} & \multicolumn{1}{r}{\textbf{45.4}} & \textbf{27.7}                           & \textbf{50.2}                           & \textbf{49.2}                          & & \multicolumn{1}{r}{\textbf{55.4}} & \textbf{18.0}                             & \textbf{50.1}                           & \textbf{69.2}                           \\ \bottomrule[1pt]
        \end{tabular}%
        } \caption[Regular object detection performance with SIoU threshold
        during evaluation]{Regular Object Detection performance on DOTA and DIOR
        datasets with GIoU and GSIoU ($\gamma=-3$ and $\kappa=16$) losses. mAP
        is computed with several \textbf{SIoU thresholds} (0.5 to 0.95) as it is
        commonly done in regular detection.}
        \label{tab:regular_od_siou}
        \end{table}

\begin{table}[]
    \centering
    \resizebox{0.85\textwidth}{!}{%
    \begin{tabular}{@{\hspace{2mm}}ccccccccccc@{\hspace{2mm}}}
    \toprule[1pt]
    \textbf{}                        & \textbf{} & \multicolumn{4}{c}{\textbf{Base classes}}          & & \multicolumn{4}{c}{\textbf{Novel Classes}} \\ 
    \textbf{}                        & \textbf{XQSA} & \textbf{All} & \textbf{S} & \textbf{M} & \textbf{L}& & \textbf{All}  & \textbf{S}  & M     & L    \\\midrule[1pt]
    \multirow{2}{*}{\textbf{DOTA}}   &  w/ GIoU      & 59.27        & \bbf{44.07}      & 66.91      & 65.46     & & 49.02        & 35.10      & 57.58      & \rbf{74.30}      \\
                                     &  w/ GSIoU     & \bbf{59.32} & 35.32      & \bbf{66.29}      & \bbf{69.03}     & & \rbf{57.70}        & \rbf{46.77}      & \rbf{65.56}      & 73.67      \\\midrule
    \multirow{2}{*}{\textbf{DIOR}}   &  w/ GIoU      & 62.06        & 17.49      & 45.55      & 82.22     & & 53.81        & 23.79      & 53.46      & 71.63      \\
                                     &  w/ GSIoU     & \bbf{63.81}        & \bbf{17.77}      & \bbf{49.62}      & \bbf{82.53}     & & \rbf{58.79}        & \rbf{25.60}      & \rbf{59.28}      & \rbf{73.78}      \\\midrule
    \multirow{2}{*}{\textbf{Pascal}} &  w/ GIoU      & 55.51        & 26.10      & \bbf{46.82}      & 64.31     & & 52.43        & 28.97      & 40.73      & 62.58      \\
                                     &  w/ GSIoU     & \bbf{58.74}        & \bbf{27.47}      & 46.56      & \bbf{68.93}     & & \rbf{58.92}        & \rbf{31.36}      & \rbf{41.65}      & \rbf{69.71}      \\ \midrule
    \multirow{2}{*}{\textbf{COCO}} &  w/ GIoU      & 21.46                 & 12.77               & 24.79               & 31.86              & & 29.21          & 17.36         & 27.62         & 40.05   \\
                                   &  w/ GSIoU     & \bbf{21.97}        & \bbf{12.80}      & \bbf{25.72}      & \bbf{32.35}     & & \rbf{29.94} & \rbf{18.87}& \rbf{29.93}& \rbf{40.47}   \\ \bottomrule[1pt]
    \end{tabular}%
    } \caption[Application on DOTA, DIOR, Pascal VOC and COCO with SIoU
    threshold during evaluation]{Few-shot performance on three datasets: DOTA,
    DIOR, Pascal VOC and COCO. GIoU and GSIoU losses are compared. mAP is
    reported with a 0.5 \textbf{SIoU threshold} and for various object sizes.}
    \label{tab:dataset_comparison_siou}
    \end{table}

\subsection{Discussions and Limitations}
As mentioned in \cref{sec:empirical_analysis} SIoU is a better choice for
performance analysis. However, as IoU is almost the only choice in literature
for evaluation, we must use it as well for a fair comparison with existing
works. Nonetheless,
\cref{tab:criteria_comp_siou,tab:dataset_comparison_siou,tab:regular_od_siou}
provide results from
\cref{tab:criterion_comparison,tab:regular_od,tab:siou_dataset_comparison} using
SIoU as the evaluation criterion in \cref{sec:siou_eval}. They agree with the
IoU evaluation and strengthen the conclusions of our experiments. While these
results are promising, we must emphasize a few limitations of SIoU and our
study. First, SIoU requires a slight tuning to get the best performance, even if
that tuning is quite straightforward and mostly depends on the size distribution
in the target images. SIoU allows being more lenient with small objects for
evaluation ($\gamma \geq 0$), and stricter for training ($\gamma \leq 0$) to
prioritize the detection of small targets. However, this is a small price to pay
compared to the performance gains obtained on aerial datasets and especially on
small objects. Second, another limitation is its application to regular object
detection. While this works relatively well with FCOS, it does not show
consistent results with other frameworks. It could likely be explained as, in
these frameworks, IoU plays a crucial role in the example selection and loss
computation. More investigation is required to answer this question. Similarly,
Non-Maximal Suppression leveraged IoU as well, and therefore, an implementation
of NMS with SIoU instead could also help greatly for the detection task,
especially for small objects. Finally, we discussed the alignment of SIoU with
human perception for evaluation, \ie inside the computation of the mAP. However,
recent works \cite{oksuz2018localization,jena2023beyond} question the soundness
of this metric and propose alternatives that are not necessarily based on IoU.
It would be relevant to study them as well and understand how well they align
with human perception in order to design more user-oriented detection models.

\section{Conclusion}
\vspace{-1em}
In this chapter, we highlighted the weaknesses of Intersection over Union both
for training and evaluating few-shot object detection models. As an alternative,
we proposed Scale-adaptative Intersection over Union (SIoU), a criterion that
changes with the object size. We performed an in-depth empirical and theoretical
study of several criteria and showed that SIoU has desirable properties for model
evaluation that other criteria have not. This is confirmed by a user study that
shows a better alignment of SIoU with human appreciation. In addition, we
experimented thoroughly with SIoU as a loss function and obtained impressive
performance gains on small object detection in the few-shot regime. This is
particularly helpful for applications on aerial images, especially as it is
compatible with the attention mechanism that we presented in \cref{chap:aaf} to
improve small object detection in the few-shot regime.

\part{Prototyping and Industrial Application}%
\label{part:application}

\chapter{Integration in COSE Prototypes}
\label{chap:integration}

\chapabstract{Detection models are often heavy and are not well suited for
    COSE's application. In this chapter, we first present in detail the CAMELEON
    system and its constraints. Then, we study the influence of the model size
    on the performance and present useful tools and tricks to accelerate the
    inference. Finally, we explain how the detection models are deployed inside
    the CAMELEON prototype and how they perform on aerial images.} \PartialToC

Up to this point, the contributions of this project have been mostly
research-oriented and a significant amount of work is still required to apply
the developed techniques on real-case scenarios. Therefore, in this chapter, we
present the engineering part of this project, which focuses on applying object
detection algorithms inside the CAMELEON system. In particular, we detail the
architecture of the CAMELEON system and the constraints associated. In light of
these constraints, we can adapt object detection algorithms for COSE's
applications. This mainly involves reducing the size of the models and their
inference speed while preserving their detection quality. Of course, this is no
easy task, but fortunately, there exist tools to help with this process. We
present these tools and their principle briefly, before explaining how they can
be leveraged for deploying object detection models on edge devices within the
CAMELEON system. Finally, we provide some guidelines for future improvements of
optimized detection models and how to deploy few-shot models as well. 

\section{CAMELEON Aerial Intelligence System}
\vspace{-1em}
\subsection{Presentation of the system}
CAMELEON is a high-resolution aerial camera system that aims to be embedded on
various types of carriers such as tactical helicopters (\eg Airbus H215), patrol
or intelligence aircraft (\eg Airbus A400M or ATL-2), and tactical drones (\eg
Safran's Patroller). Its main objectives are to provide precise 2D and 3D models
of geographical areas at needs. Often, satellite imagery is insufficient as it
has a low ground resolution, can be outdated, or simply unavailable (\eg due to
weather conditions). In such situations, airborne reconnaissance systems are
vital. In practice, these kinds of systems have military applications (theater
cartography, tactical intelligence, etc.) but also civil and commercial
(maritime surveillance, search and rescue, fire monitoring, etc.). Such aerial
surveillance systems already exist (\eg GlobalScanner product by COSE as well).
However, their specifications are often limited in light of the recent progress
in sensor resolution and quality. As an example, optronics systems often provide
extreme ground resolution but are limited to a small area. CAMELEON aims at
improving the compromise between the swath and ground resolution. Specifically,
CAMELEON will embed up to 6 camera sensors (each producing $\sim$ 100M pixels
images). This allows for an extra wide field of view and significant
overlapping between images (which is crucial for 3D modelization). This involves
dealing with huge amounts of data, which requires carefully designed hardware
and software. The CAMELEON system is constituted of two major components, the
on-board components and the ground component. Both will be presented in the
following sections.

The on-board segment of CAMELEON is the heart of the system, it includes a
sensor block, a dedicated computer and a user interface:
\vspace{-0.5em}
\begin{itemize}[nolistsep]
    \item[-] \textbf{Sensor Block}: this unit gathers all the sensors of the
    system along with the mechanical parts and motors used for controlling their
    orientations. All the sensors are attached at the core of a 3-axis gimbal
    suspension. These three axes are controlled by motors to compensate for
    parasite motions (low and high frequencies) of the carrier and to be able to
    orient the sensors in any direction. In addition, a forward motion
    compensation module equips each sensor to improve image quality. Without
    this, the image quality is degraded due to motion blur induced by the
    displacement of the carrier during the exposition of the sensor. The Sensor
    Block also contains an Inertial Navigation System (\ie an Inertial
    Measurement Unit coupled with a GNSS sensor), whose goal is to precisely
    measure the position and orientation of the camera to determine accurately
    the position of pointed objects on the ground.   
    \item[-] \textbf{On-board Computer}: it manages the data stream from the sensors to
    the User Interface and memory storage. Specifically, it consists of an
    \textit{host} system that controls an FPGA programmed to read into the
    camera memory buffer and retrieves the images in the host memory. The host is
    then in charge of storing the images and their associated metadata in a
    geospatial database. It also feeds the user interface with the latest image
    for visualization. Finally, it also controls a set of lightweight GPUs
    (Nvidia Xavier or Orin) to execute complex treatments such as object
    detection. A special memory allocation allows for fast data transfer from
    the host to the kernel to process the image from the system in real-time.
    This is illustrated in \cref{fig:cam_host_system}, which depicts a
    simplified overview of the system's architecture.
    \item[-] \textbf{User Interface}: it is an application that displays the images and
    metadata acquired by the system to the user. It also monitors a set of
    variables about the flight and the mission. 
\end{itemize}

\begin{figure}
    \centering
    \includegraphics[width=0.7\textwidth]{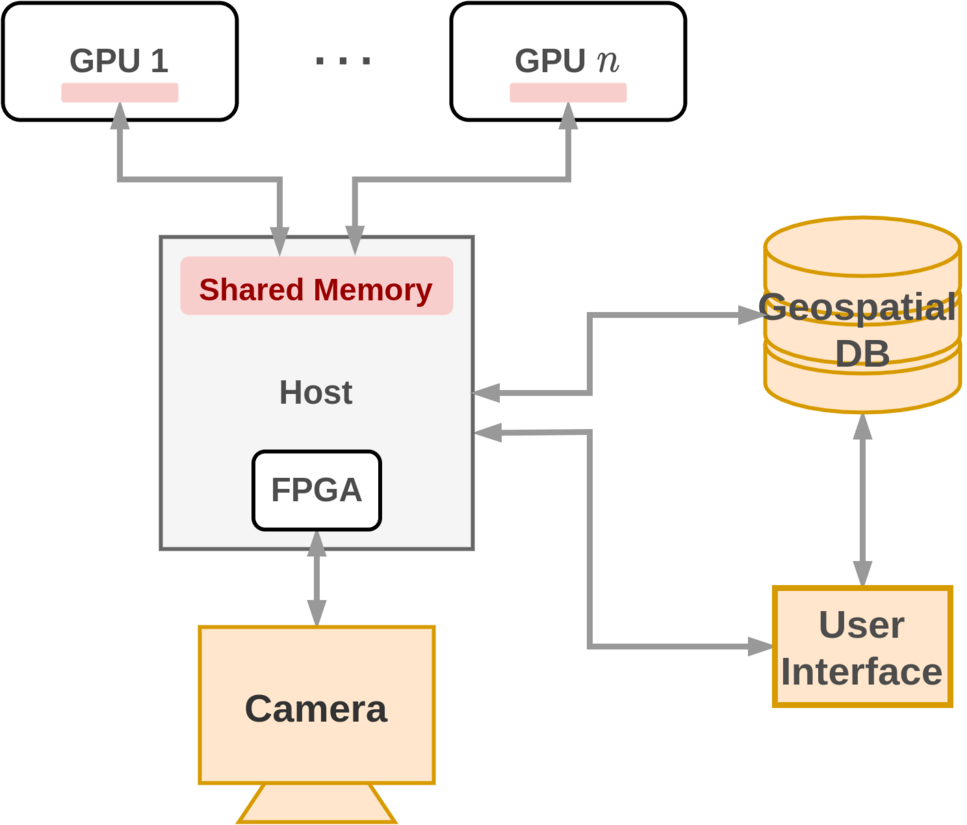}
    \caption[On-board computer architecture]{Simplified illustration of the on-board
    computer architecture in CAMELEON.}
    \label{fig:cam_host_system}
\end{figure}

The Graphical Processing Units (GPUs) selected to be part of the system are the
Nvidia edge computing devices Xavier and Orin (see \cref{fig:cam_gpus}). These
are lightweight and power-efficient GPUs that can be used as an independent
device or as an end-point within a more elaborate system. These GPUs have the
computation capacity of mid-range commercial GPU (\eg Nvidia RTX 3060) while
consuming less than 50W, which is roughly four times as efficient. That makes
them particularly well-suited for COSE application as only limited resources are
available inside the carrier but a strong computation power is still required.
Nevertheless, such an efficiency is not enough to run regular detection models
in real-time on 100 of megapixels images. Thus, a lot of effort is required to
adapt the models and their inference, we will elaborate this in
\cref{sec:on_edge_detection}.   

\begin{figure}
    \centering
    \includegraphics[width=0.7\textwidth]{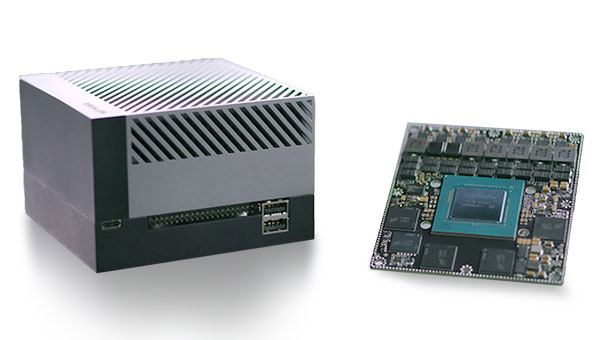}
    \caption[Nvidia AGX Orin]{Nvidia AGX Orin, development kit (left) and production (right) GPUs.}
    \label{fig:cam_gpus}
\end{figure}

The second component of the CAMELEON system is mainly constituted of the ground
station. It is a high-power workstation, and its goal is to store and process
data between missions. The low resources available during the flight are not
sufficient for heavy treatments such as 3D modelization of the overflown areas.
Mission disks filled during the flight inside the on-board segment can be
transferred to the ground station for further analysis of the collected data.
The ground station can also be used to execute more demanding detection
algorithms or advanced treatments (\eg segmentation, change detection, etc.).

\subsection{CAMELEON Image Specifications and Constraints}
In addition to the limited computation resource available on-board, CAMELEON has
requirements that make object detection even more challenging. To assist the
operators on-board efficiently, objects must be detected in real-time. The
cameras are set to acquire at least one image every second, and each image is
$11600\times 8700$. These images are 16-bit Bayer matrices, which amount to
roughly 192 megabytes per image. This means that more than 1 GB of data is
produced every second in the system when all the cameras are used. However, the
current prototype is designed with only one camera, which makes the real-time
detection a little less challenging. Aside from the model inference, dealing
with such a large amount of data requires carefully designed datastreams between
the various parts of the system to retrieve, store and display the images
produced by the cameras. This is what motivated the use of a Field Programmable
Gate Array (FPGA) to orchestrate the transfer between the camera and the host. 

Now, achieving real-time detection, \ie processing one image per second, under
such constraint is challenging and cannot be done with regular-size detection
models (\eg with ResNet backbone) implemented in Python. The fastest
implementations in Pytorch (\eg YOLOv5 \cite{yolov5}) are able to process one
$640\times 640$ pixels image in about 20ms on a high-end server GPU (Nvidia
V100), with a backbone comparable in size to a ResNet-50. One image of CAMELEON
is equivalent to 250 images$640\times 640$ pixels, which would require about 5s
to process. Besides, this does not count data transfer time, pre- and
post-processing which must also fit under the one-second time limit.
Furthermore, the high-end GPU used for these benchmarks (available on YOLOv5's
repository\footnote{\href{https://github.com/ultralytics/yolov5}{https://github.com/ultralytics/yolov5}}
) have much more computing power than the Nvidia Xavier selected for CAMELEON
(10 times more according to the theoretical capabilities on Nvidia's website).
Therefore, one must adapt the detection models significantly to fulfill
CAMELEON's requirements. This can be done in two different ways: first reducing
the model size and second accelerating the inference. Both approaches are
employed in CAMELEON's prototype, and they will be discussed in the following
two sections.

\section{Reducing Object Detection Model Size for Edge Computing}
\label{sec:on_edge_detection}
\vspace{-1em}
The most straightforward way to increase the throughput of a detection model is
to reduce its size. However, this often comes at the cost of lower accuracy. In
this section, we analyze to find the speed/accuracy tradeoff of
YOLOv5 \cite{yolov5} on aerial images. Then we propose a simple knowledge
distillation approach to improve this tradeoff and achieve higher detection
quality at a fixed size. 

\subsection{Object Detection Accuracy/Speed Tradeoff}

First, we compare the speed/accuracy tradeoff for multiple detection frameworks
on natural images. The results of this comparison can be found in
\cref{fig:depl_tradeoff_natural}. To make this figure, we collected the
performance metrics and model size information directly from the articles
presenting the various detection methods. We select the mAP with various IoU
thresholds on MS COCO. This is rather simple as this is the most common
evaluation benchmark in the detection literature. Then, the most relevant way to
assess the speed of the models is to measure the latency (in ms) which
represents the time required to process one image. However there are multiple
complications with this measure: it depends greatly on the hardware used, the
size of the input images and even the version of the library used during
inference (in a less pronounced manner). In addition, the latency is not always
reported in the articles which makes the task even more difficult. Thus, we
leverage a surrogate for the latency: the total number of parameters in the
model. Of course, it does not correlate fully with the latency, but it is a
sound approximation, and it is easier to collect. In
\cref{fig:depl_tradeoff_natural}, we plot the performance against the number of
parameters (left) and against the latency (right). There are missing values in
the latency plot as this measure was not reported in the original articles. We
compare 8 distinct detection frameworks: YOLOv5 \cite{yolov5}, CornerNet
\cite{law2018cornernet}, CenterNet \cite{duan2019centernet}, Mask R-CNN
\cite{he2017mask}, DETR \cite{carion2020end}, FCOS \cite{tian2019fcos},
SwinTransformers \cite{liu2021swin} (based on Cascade R-CNN), DynamicHeads
\cite{dai2021dynamic}, and DiffusionDet \cite{chen2022diffusiondet}. The
conclusion of this comparison is blatant: YOLOv5 has a much better
speed/accuracy tradeoff than other detection frameworks. Of course, our
comparison is not exhaustive, yet it includes recent one-stage detectors that
are fast and perform well. We are aware that very recent developments in the
YOLO family outperform YOLOv5 (\eg YOLOv8
\footnote{\href{https://github.com/ultralytics/ultralytics}{Link to YOLOv8
repository}}); however, we could not include them in our comparison. 

Another reason that nudges our choice toward YOLOv5 is the greater variety of
model sizes that they propose. Most detection models are proposed with two
different sizes, usually by employing two distinct backbones (\eg ResNet-50 and
ResNet-101), but keeping the detection head unchanged. With YOLOv5, the whole
network is modified accordingly, including the head. This provides smoother
model size modifications and greater flexibility.     

\begin{figure}
    \centering
    \includegraphics[width=\textwidth]{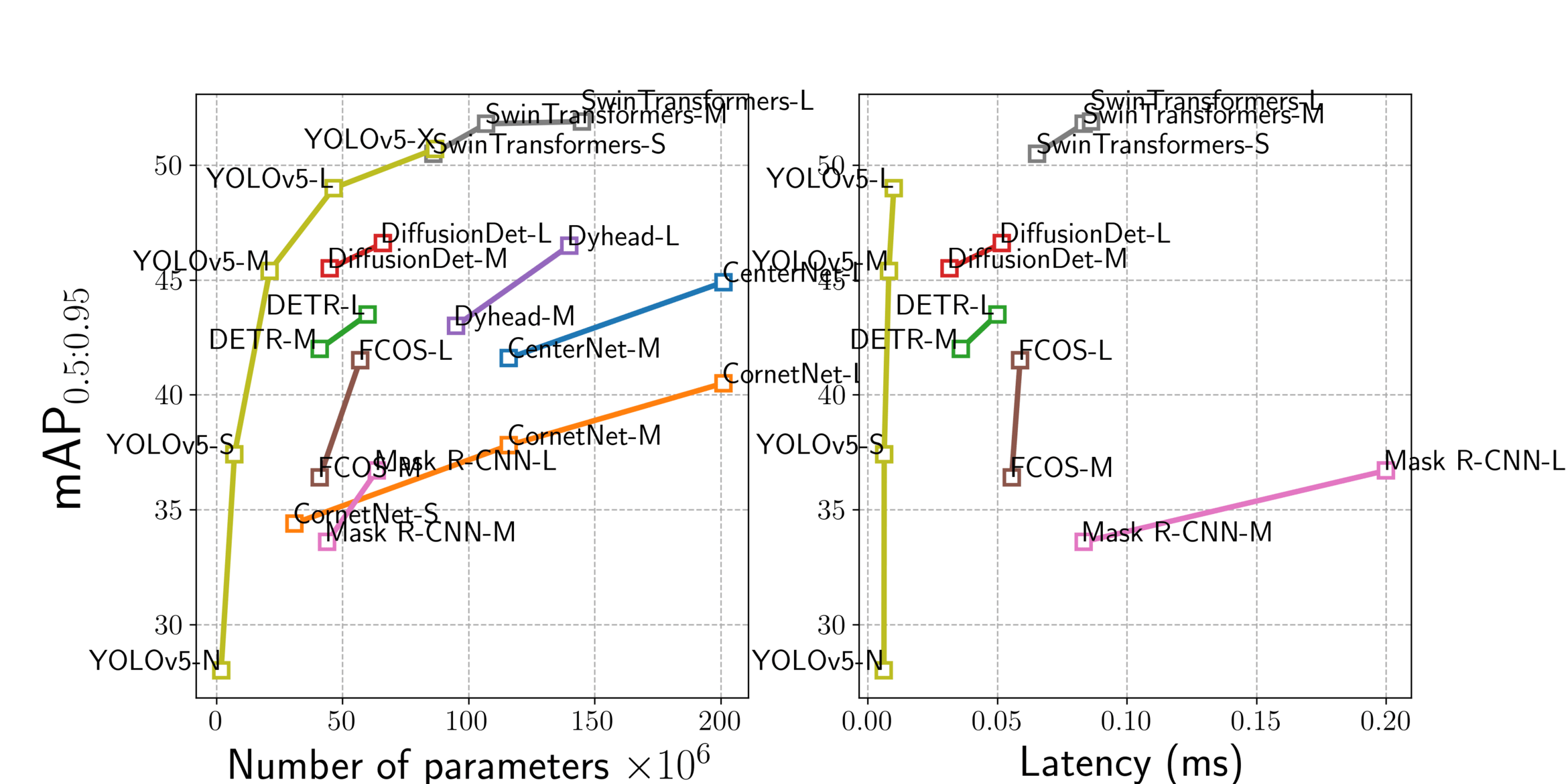}
    \caption[Inference speed / detection accuracy tradeoffs
    comparison]{Inference speed / detection accuracy tradeoffs comparison for
    multiple detection frameworks on MS COCO. Latency, \ie wall-clock time for inference of
    one image is directly used (right) but this measure is not always reported
    in the literature. As a surrogate, the number of parameters is used to have a
    more thorough overview (left). N, S, M, L and X designate the various model
    sizes.}
    \label{fig:depl_tradeoff_natural}
\end{figure}

The comparison from \cref{fig:depl_tradeoff_natural} demonstrates the
superiority of YOLOv5 over other approaches on MS COCO. However, it is necessary
to confirm that the same behavior is observed on aerial images as well. To this
end, we train the different versions of YOLOv5 on DOTA and DIOR datasets and
plot similar curves as in \cref{fig:depl_tradeoff_natural}. The resulting plots
are available in \cref{fig:depl_tradeoff_aerial}. In this analysis, we include
even smaller models than the \textit{nano} version of YOLOv5 (YOLOv5-N). We call
these models YOLOv5-P and YOLOv5-F (for Pico and Femto following the
nomenclature from YOLOv5). These models have respectively 0.68 and 0.32 millions
of parameters which is much lower compared to traditional detection models.
Nevertheless, strong performance is achieved on DOTA and DIOR datasets (see
\cref{tab:depl_aerial_yolo}). Also, YOLOv5-X is not included in the analysis as
it is too large for COSE's application.

\begin{table}[]
    \centering
    \resizebox{0.7\textwidth}{!}{%
    \begin{tabular}{@{\hspace{2mm}}ccccc@{\hspace{2mm}}}
    \toprule[1pt]
    \textbf{Model} & \textbf{\# params} & \textbf{Latency (ms)} & \textbf{mAP DOTA} & \textbf{mAP DIOR} \\ \midrule
    YOLOv5-F        & 3.19e+5           & 5.2                   & 42.8             & 41.4             \\
    YOLOv5-P        & 6.75e+5           & 5.4                   & 49.1             & 63.2             \\
    YOLOv5-N        & 1.78e+6           & 6.0                   & 68.5             & 80.3             \\
    YOLOv5-S        & 7.05e+6           & 6.6                   & 72.7             & 85.2             \\
    YOLOv5-M        & 2.09e+7           & 8.7                   & 74.6             & 88.1             \\
    YOLOv5-L        & 4.62e+7           & 11.1                  & 75.4             & 89.1             \\ \bottomrule[1pt]
    \end{tabular}%
    } \caption[Performance on DOTA and DIOR for various YOLOv5
    sizes]{Performance (mAP$_{0.5}$) on DOTA and DIOR for various YOLOv5
    sizes. Numbers of parameters and latency are provided along with the
    performance measures. Latency is computed with 512$\times$ 512 images on a
    RTX 3090.}
    \label{tab:depl_aerial_yolo}
    \end{table}

\begin{figure}
    \centering
    \includegraphics[width=\textwidth]{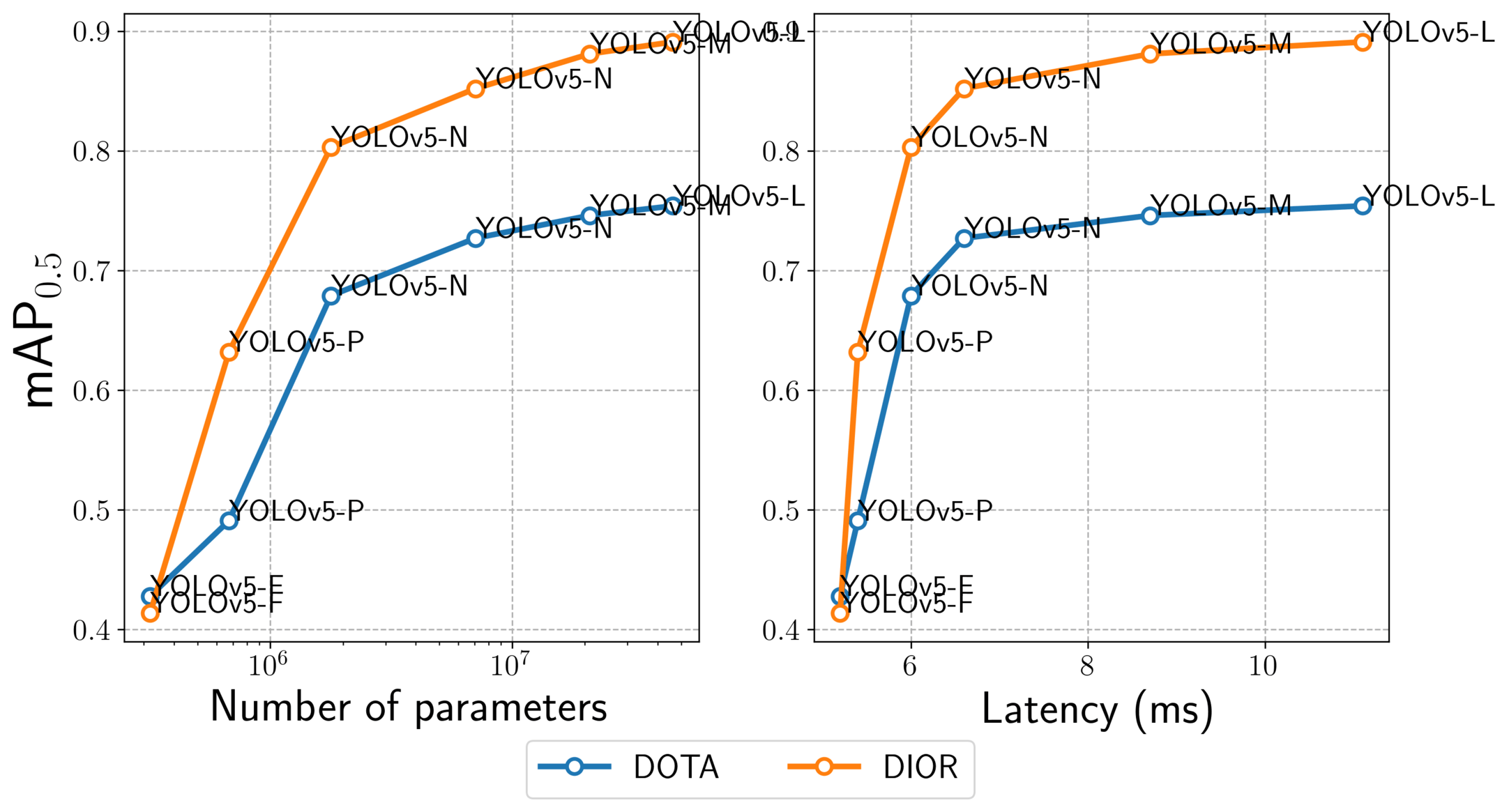}
    \caption[YOLOv5 speed/accuracy tradeoff on aerial images]{YOLOv5 tradeoff
     between model size and performance (left). Between latency and performance
     (right). 6 model sizes are compared: Femto, Pico, Nano, Small, Medium and
     Large on DOTA and DIOR. The horizontal axis of the left plot is in
     log-scale to better distinguish between the smallest models. Latency is
     computed with 512$\times$ 512 images on a RTX 3090.}
     \label{fig:depl_tradeoff_aerial}
\end{figure}

From \cref{fig:depl_tradeoff_aerial,tab:depl_aerial_yolo}, it is clear that the
detection performance is strongly correlated with the number of parameters of
the model and that this connection holds also for very small models (< 1M
parameters). However, it seems that the performance drops faster under a certain
model size. This is observed both for DOTA and DIOR, even though it is more
pronounced with DIOR. We also observed that the latency plot does not follow
identically the parameters plot. For instance, with YOLOv5-P, the latency does
not follow the scaling down of the model size exactly. This is probably due to
computation overhead and synchronization inside the model which prevents faster
inference. Given these tradeoff curves, YOLOv5-N is the most promising model as
it is closest to the top-left corner of the plots. However, it is relevant to
investigate other model sizes and verify how well they comply with COSE's
constraints.

\subsection{Knowledge Distillation}

Of course, the smaller the models, the lower the detection performance. However,
there exist techniques to boost the performance of any network when we have
access to a similar but larger model with increased performance. This is called
\textit{Knowledge Distillation} (KD). The main principle behind this technique
is to train a \textit{student} model to mimic a \textit{teacher} which is often
larger and has better performance. It was first introduced by Hinton et al.
\cite{hinton2015distilling} in 2015. Originally, it consisted in training the
student model with an additional loss measuring how close the logits from the
student and the teacher are. Then, it was extended multiple times, with for
instance intermediary layer activation distillation \cite{romero2015fitnets},
relational distillation \cite{park2019relational}, adversarial distillation
\cite{wang2018kdgan} or multiple teachers \cite{liu2020adaptive}. Most of these
techniques are designed for classification applications and provide limited
performance boosts for detection models. Fortunately, knowledge distillation can
be extended for fine-grained tasks and in particular for object detection
\cite{wang2019distilling,li2017mimicking,chawla2021data,wei2018quantization,wang2020gan}.
For more details about existing knowledge distillation methods, we defer the
reader to this complete survey \cite{gou2021knowledge}. 

As a first try with KD, we applied Fine-grained Feature Imitation (FFI)
\cite{wang2019distilling} to improve the training of YOLOv5-P and YOLOv5-N with
larger teachers trained on DOTA and DIOR. FFI proposes an additional loss
function that measures the disparity between the teacher and student feature
maps, but the computation only includes regions close to a ground truth annotation:

\begin{align}
    \mathcal{L}_{FFI} &= \frac{1}{2m} \sum\limits_{i=1}^W\sum\limits_{j=1}^H \mathcal{M}_{i,j}\| f_{\text{adapt}}(s_{i,j}) - t_{i,j} \|_2^2, \, \text{where  }m =  \sum\limits_{i=1}^W\sum\limits_{j=1}^H \mathcal{M}_{i,j}.
\end{align}

Here, $\mathcal{M}$ is the imitation mask which is 1 where in neighboring
regions of each ground truth annotation and 0 elsewhere. $H$ and $W$ represent
the height and width of the feature map respectively.  $f_{\text{adapt}}$ is a
mapping function that converts the feature of the student, denoted $s_{i,j}$, to
the same size as the teacher's ones (in terms of the number of channels),
denoted $t_{i,j}$. Indeed, features maps of the student often have fewer channel,
which prevents direct comparison. The intuition behind this loss is that student
outputs should only match teacher outputs in regions where there is an object of
interest. The background is too noisy and the student model may not have the
capacity to mimic the teacher everywhere, it should focus only in relevant
regions.

The results obtained with aerial images do not agree with the performance gains
reported in the original paper \cite{wang2019distilling} for natural images (see
\cref{tab:kd_res}). Significant performance drops are observed with distillation
on aerial images. Similar drops are observed for YOLOv5-P and YOLOv5-N both on
DOTA and DIOR datasets. KD methods designed for natural images may not be well
adapted for aerial images, one reason for this could be the presence of much
smaller objects which have smaller and noisier representations inside the
feature maps. In a sense, that could be linked to the difficulty of applying
FSOD methods on aerial images, as described in \cref{chap:aerial_diff}. 
These results deserve to be investigated further in future work as
distillation is a promising direction for detection improvements.

\begin{table}[]
    \centering
    \resizebox{0.75\textwidth}{!}{%
    \begin{tabular}{@{\hspace{2mm}}clcccclcccc@{\hspace{2mm}}}
    \toprule[1pt]
                              & \multicolumn{1}{c}{} & \multicolumn{4}{c}{\textbf{DOTA}} &  & \multicolumn{4}{c}{\textbf{DIOR}} \\ \cmidrule(lr){3-6} \cmidrule(l){8-11} 
                              & \multicolumn{1}{c}{} & \textbf{All}    & \textbf{S}      & \textbf{M}      & \textbf{L}      &  & \textbf{All}    & \textbf{S}      & \textbf{M}      & \textbf{L}      \\ \midrule
    \multirow{2}{*}{\textbf{YOLOv5-P}} & w/o KD               & \textbf{49.1}   & \textbf{30.8}   & \textbf{53.8}   & \textbf{48.9}   &  & \textbf{63.2}    & \textbf{21.3}    & \textbf{50.0}   & \textbf{76.6}   \\
                              & w/ KD                         & 46.4   & 27.9   & 51.7   & 45.9   &  & 48.5    & 17.5    & 43.0   & 56.4   \\ \midrule
    \multirow{2}{*}{\textbf{YOLOv5-N}} & w/o KD               & \textbf{67.9}   & \textbf{51.0}   & \textbf{70.9}   & \textbf{69.8}   &  & \textbf{80.6}    & \textbf{35.9}    & \textbf{69.2}   & \textbf{90.8}      \\
                              & w/ KD                         & 47.5   & 30.3   & 52.4   & 45.5   &  & 55.9    & 22.0    & 51.4   & 65.3      \\ \bottomrule[1pt]
    \end{tabular}%
    } \caption[Knowledge Distillation performance comparison on DOTA and DIOR]{Detection performance (mAP$_{0.5}$) comparison with and without
    Knowledge Distillation (KD). Performance is reported for different objects
    sizes and on two datasets, DOTA and DIOR. Both YOLOv5-P and YOLOv5-N are
    compared.}
    \label{tab:kd_res}
    \end{table}

\section{Inference Acceleration with TensorRT}
\vspace{-1em}
TensorRT is a tool provided by Nvidia to optimize the inference of deep learning
models on Nvidia's hardware. Recent Nvidia GPUs have dedicated modules for deep
learning inference, called \textit{Tensor Cores}. They contrast from
\textit{CUDA cores}, which are principally made for parallel computing. TensorRT
unlocks the potential of the tensor cores by generating an \textit{engine} that
can be run most efficiently on a specific GPU. This differs from the
vanilla Python inference (\ie with any deep learning library), which calls CUDA
kernels that are executed on the CUDA cores of the GPU. TensorRT also leverages
the CUDA cores and thus leverages the maximum GPU capabilities. In addition,
TensorRT proposes several optimization tricks for increasing the
inference speed even more. We detail such tricks in the following paragraphs and
explain how they can be used for object detection models. 

Specifically, TensorRT takes as input a neural network model from any common
library in a suitable format such as Open Neural Network Exchange (ONNX) and
converts it into an \textit{engine}. This engine contains the initial network
along with specific instructions about how to run inference in an efficient way
given the available hardware. TensorRT provides APIs for various programming
languages (\eg Python and C++) to run the engine directly from any application.

\paragraph*{Operation Fusion and Scheduling}
TensorRT merges different layers and re-organizes the forward pass of the model.
Specifically, the forward pass of a model can be represented as a computation
graph (see \cref{fig:trt_fusion}). Each node represents a layer and the edges
represent the data stream between layers. One of the objectives of TensorRT is
to optimize the computation graph of the model. First, when several consecutive
layers are composed (\ie the output of the first layer is the only input of the
second layer), they can be fused vertically. Instead of creating distinct CUDA
kernels for each layer, only one kernel is created, saving a lot of time,
particularly in data transfer. This is illustrated in \cref{fig:trt_fusion}
(right) where, convolution, bias, and activation layers are merged as "CBR"
blocks. Then, layers can also be merged horizontally. Horizontal fusion happens
when similar operations are in parallel and have the same input (\eg three
distinct parallel $1\times 1$ convolutions as illustrated in
\cref{fig:trt_fusion}, right). This also reduces the creation of unnecessary
CUDA kernels and maximizes the utilization of the resources. Finally, when there
are parallel paths in the computation graph, the order in which the different
branches can be irrelevant. In this case, some clever scheduling can maximize
the utilization of the GPU and speed-up inference. This scheduling is not
trivial as it is subject to the memory constraint of the hardware, and all
parallel operations cannot be done at once. 

TensorRT searches for optimal scheduling and fusion according to the available
hardware and the input size that will be used at inference (\eg batch size and
image size). This implies that the input size is fixed for a given engine, but
it is a small price compared to the benefits of the TensorRT conversion.
TensorRT does have a variable size option, but it does not seem to be compatible
with the reduced precision that we will detail in the following paragraphs.

\begin{figure}
    \centering
    \includegraphics[width=\textwidth]{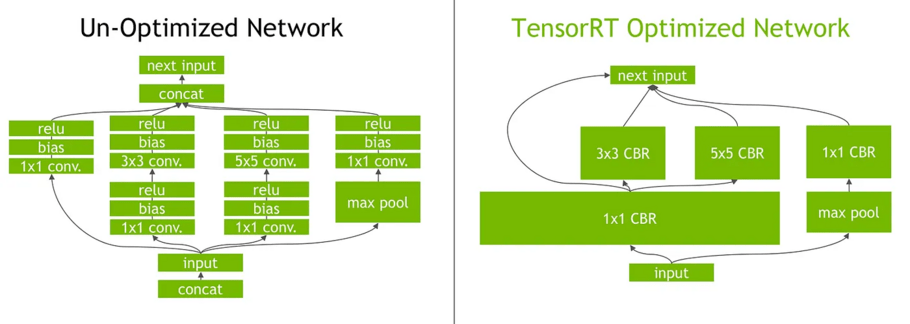}
    \caption[Illustration of a computation graph and its optimization by
    TensorRT]{Illustration of a computation graph and its optimization by
    TensorRT. Figure taken from an Nvidia technical
    \href{https://developer.nvidia.com/blog/production-deep-learning-nvidia-gpu-inference-engine/}{blog
    post}.}
    \label{fig:trt_fusion}
\end{figure}

\paragraph*{Floating point precision}
Then, TensorRT reduces the precision of the model's weights and activations. In
Python, common libraries store models and images as \textit{float} numbers which
take 4 bytes of memory (\ie 32 bits, denoted F32). TensorRT can convert a model with
half-precision floats (denoted F16), which only take 16 bits of memory. This reduces the size
of the model by a factor of two, but most importantly it reduces the computation
power required to perform a forward pass as well. Indeed, arithmetic operations
are faster with half-precision numbers as they require fewer basic operations.
Thus, it sensibly reduces the overall inference time. This trick is becoming
quite common as well for training neural networks; however, mixed precision (\ie
keeping some parts of the models as regular floats) is required to prevent
convergence issues \cite{micikevicius2017mixed}. TensorRT instead converts the
whole model into half-precision. In practice, this loss of information is not an
issue for inference and performance is almost unchanged.  

\paragraph*{Integer Quantization and Calibration}
Pushing even further, TensorRT enables the conversion of the model in 8-bit
integers. This again reduces the amount of memory and computation power required
to store the model and perform inference. However, the model's weights and
intermediary computation can only take 256 different values with integer
precision. This is certainly not enough to represent the entire range of values
found in a model. A solution for this is to map the weights and activations of
the model into the $\llbracket -2^{b-1}, 2^{b-1}-1\rrbracket$ interval, where
$b$ is the number of bits used for the quantization (generally $b=8$). If the
dynamic range of the weights is $[\alpha, \beta]$, this can be achieved with a
linear transformation:

\begin{equation}
    S(x) = ax + b, \quad\text{with } a= \frac{2^b-1}{\beta-\alpha}, \quad\text{and }  b= -\frac{\alpha(2^b-1)}{\beta-\alpha} - 2^{b-1}.
\end{equation}

Hence, the quantization and dequantization functions can be written as: 

\begin{align}
    \bar{x} = Q(x) &= \left\lfloor S(x)\right\rfloor = \left\lfloor ax + b\right\rfloor,\\
    x &\approx S^{-1}(\bar{x}) = \frac{1}{a}(\bar{x}-b).
\end{align}

This can be further simplified when both the weights interval and the quantized
interval are centered (\eg $\alpha = \beta$): $S^{-1}(\bar{x}) =
\frac{1}{a}\bar{x}$. In this case, one can easily see that the matrix
multiplication, the basic operation of neural networks' forward pass can be
computed using mostly integer multiplications and additions. If we define three
matrices $X \in \mathbb{R}^{n\times p}$, $Y \in \mathbb{R}^{p\times m}$, and $Z
\in \mathbb{R}^{n\times m}$, such that $Z = XY$, then we have:

\begin{align}
    z_{i,j} = \sum\limits_k x_{i,k}y_{k,j} \approx \sum\limits_k S_x^{-1}(\bar{x}_{i,k})S_y^{-1}(\bar{y}_{k,j}) = \frac{1}{a_xa_y} \sum\limits_k \bar{x_{i,k}}\bar{y_{k,j}},
\end{align}

where $a_x$ and $a_y$ are the quantization scale parameters for the respective
quantization of matrices $X$ and $Y$. As hinted in the previous equation,
different quantization functions are necessary to keep the flexibility of the
model. The number of parameters inside a model is large compared to the number
of quantized values (256 for integer quantization). Quantizing all weights of
the model at once would result in a much looser approximation of the actual
weight values. Instead, quantization is performed at the layer-level or even
lower (at the column or channel level). 

Similarly, activations of the model are also quantized per layer. However, the
range of the activation highly depends on the input of the model and cannot be
known in advance. To alleviate this, a calibration cache can be computed from a
\textit{calibration set}. TensorRT automatically builds this cache by running a
forward pass on all images of the calibration set. Of course, this set must
contain images that are similar to the images on which the model will be
deployed. Concerning COSE, this complexifies the domain adaptation problem as
the calibration of the model should be adapted to the domain. Even if it only
requires non-annotated images to compute the calibration cache, it can sometimes
be challenging to obtain an appropriate calibration set due to confidentiality
constraints. The resulting calibration cache gathers information about the
quantization of the activation of the different layers of the model. It is then
used during inference, which slightly increases memory usage during inference. 

We briefly presented the linear quantization technique, yet TensorRT also
provides more elaborated quantization to minimize the loss of information. For
instance, it has an Entropy Calibration technique that adapts the density of
bins according to the density of weights values. It is very similar to entropy
coding techniques that assign smaller code to the most frequent symbols.

\paragraph*{C++ Implementation}
Finally, the TensorRT conversion allows for using C++ as a backend instead of
Python. This is not strictly speaking a TensorRT trick, but it significantly reduces
the inference time. As TensorRT provides a C++ API, the whole
inference pipeline can be written in C++ as well. C++ is known to be much faster
than Python and implementing pre- and post-processing in this language can speed
up the inference. 

\section{Object Detection Pipeline for CAMELEON Prototype}
\vspace{-1em}
Now that we have discussed various improvements for accelerating the inference
time of the detection model, we present how we make use of them specifically for
COSE application. First, we select the most interesting models from our analysis
in \cref{sec:on_edge_detection}: YOLOv5-N and YOLOv5-P. Both achieve fast
inference while preserving high detection performance. We then perform the
conversion in a TensorRT engine to speed up the inference. This was done only
for the model trained with DOTA dataset, but it would be identical with
DIOR.

\subsection{Deployment with TensorRT}

We started with the nano version as the goal is to satisfy the time constraint
of processing one image per second. If it does not fit under the time
constraint, we would do the same with the pico version. With the help of
TensorRT, we reduce the precision of the model into 8-bits integers. The
calibration is performed with DOTA validation images. However, to find the most
efficient utilization of the GPUs, it was necessary to explore what were the
best combinations of batch size and input image size. CAMELEON's images have
11600 $\times$ 8700 pixels and doing the inference on the entire image at once
demands too much GPU memory, even with our smallest models. The images must be
tiled before being fed to the model. The tiling has a major downside: it can cut
object in half which make them more difficult to detect. Adding some overlapping
between tiles can prevent this but it also increases exponentially the amount of
data to process (see \cref{fig:overlap_influence}). In addition, having to many
tiles requires processing multiple batches. This is equivalent to performing the
inference multiple times adding a small overhead at each iteration. Thus, we
want to minimize the overlap and have tiles as large as possible. 

\begin{figure}
    \centering
    \includegraphics[width=\textwidth]{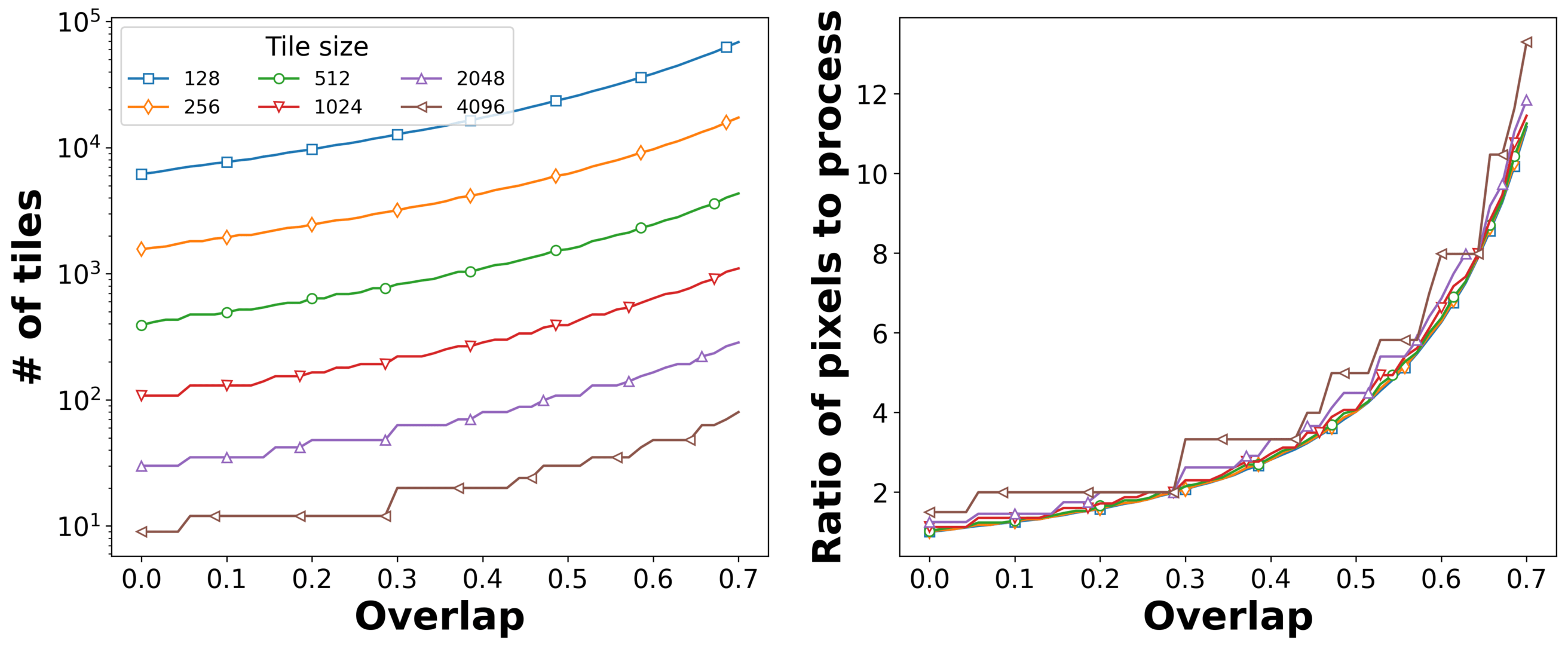}
    \caption[Influence of overlap on the amount of data to process]{Show the number of tiles against the overlap, for different
    tile size (left) and the ratio of pixels that need to be processed given the
    overlap (right). As an example, choosing an overlap of 0.7 will produce 10
    times more data to process compared with the original image. The
    computations were done with CAMELEON image size (11600 $\times$ 8700 pixels).}
    \label{fig:overlap_influence}
\end{figure}

We tried several combinations of batch size and tile size, but this is a tedious
task as an entire engine must be created for each setting. We did not mention it
before, but the engine generation is a heavy optimization process that requires
several hours on the Xavier GPUs. Based on these insights and multiple practical
trials, we set the tile size to half the height of a CAMELEON image: 4350, with
a batch size of 1. This maximizes the utilization of the GPU, requires few
inferences and limits object splitting. This results in 6 tiles organized as
shown in \cref{fig:depl_tiling} with slight horizontal overlap.      

\begin{figure}
    \centering
    \includegraphics[width=\textwidth]{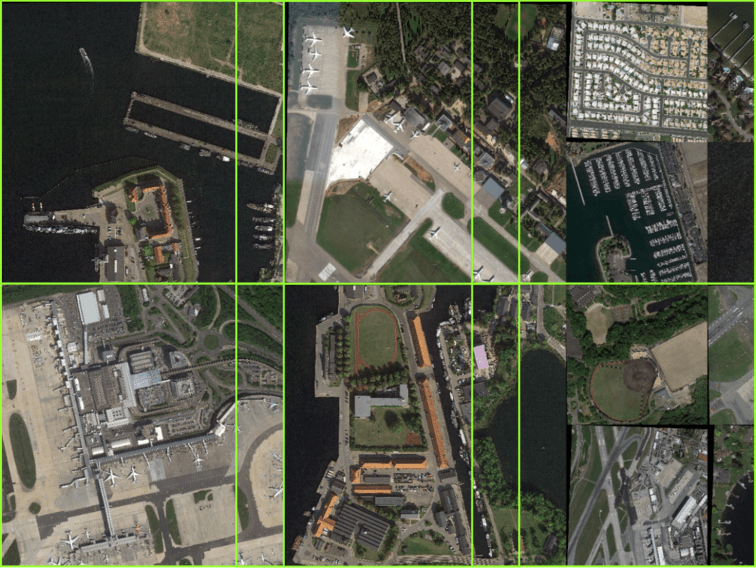}
    \caption[CAMELEON tiling]{Illustration of the tiling implemented in our
    prototype, the background image is a test image, it is constituted of DOTA
    images pasted together to get the right dimensions. Green squares represent the tiles.}
    \label{fig:depl_tiling}
\end{figure}

Overlapping also induces the need for NMS after the detection. Indeed, in areas
that will be processed multiple times, detections can be duplicated. To remove
these undesired boxes, an NMS operation must be performed at the whole image
level. As it scales in $O(M^2)$ with $M$ the number of detected boxes, it can
become expensive. In addition, it depends on the number of detected objects in
the image, which can be considerable in CAMELEON-size images. Fortunately, there
exist fast GPU implementations of the NMS that can be leveraged easily. It could
be further optimized by first performing the NMS on each tile individually and
then only in the overlapping areas to drastically reduce the number of boxes.
Nonetheless, the NMS computation time is not prohibitive as it is and the test
images contain far more objects than what will be encountered in real-case
applications. Therefore, we choose not to spend effort on improving the NMS
process. The pre- and post-processing are less expensive than the inference
itself, but they still consume a significant amount of time. To accelerate them,
we choose to implement them in C++ as it is much faster than Python, especially
for data management and transfer.  

\subsection{Detection Performance Evaluation and Profiling}
Now that we have a TensorRT engine, we must validate its capacity, both in terms
of inference speed and detection quality. Indeed, the precision reduction of the
model often leads to reduced performance. 

First, we check if the inference is fast enough to comply with COSE's
constraints. To this end, a simple profiling of the various steps of the
execution is realized. Specifically, we isolate the data copy, tiling,
pre-processing, inference and post-processing. The results of this profiling can
be found in \cref{tab:depl_profiling}. The timings are averaged over 100 runs on
an Nvidia AGX Xavier in high consumption mode (30W + overclocking). It results
in the overall processing of one image in roughly 750ms. This falls under the 1s
barrier and thus fulfills the application constraints. Larger versions of YOLOv5
do not respect this constraint. However, it would be relevant to investigate
smaller ones to get degraded performance mode to process more than one image per
second. The current prototype has only one sensor, but the final system will fly
with 6 identical cameras. While it may not be realistic to process all these
images in real-time it might be useful to have access to faster models able to
manage the image from more than one sensor. We also experiment with a new
generation of embedded GPU: the Nvidia AGX Orin, which recently replaced the
Xavier. Compatibility issues between TensorRT and the new GPU prevent us from
experimenting more with it, but we were able to create a similar engine for the
Orin. This produces significant inference speed gains as we process an entire image
in a little less than 500ms. Yet, the Orin consumes more power than the Xavier
(50W). 

\begin{table}[]
    \centering
    \resizebox{0.65\textwidth}{!}{%
    \begin{tabular}{@{\hspace{2mm}}llcl@{\hspace{2mm}}}
    \toprule
    \multicolumn{2}{l}{\textbf{Processing step}}      &  & \textbf{Time (ms)}      \\ \midrule
    CPU to GPU copy    &                     &  & 87 $\pm$ 0.5         \\
    Tiling             &                     &  & 1  $\pm$ 0.1         \\[3mm]
    \multirow{4}{*}{\parbox{3.5cm}{Inference per tile \newline $\times 6$}}  & \textit{Pre-processing }     &  & \textit{24 $\pm$ 3 }          \\
            & \textit{Inference  }         &  & \textit{73 $\pm$ 1 }          \\
                       & \textit{GPU to CPU copy }    &  & \textit{10 $\pm$ 0.5 }        \\
                       & \textbf{Subtotal}               &  & \textbf{642}            \\[3mm]
    Post-processing    &                     &  & 16 $\pm$ 1           \\ \midrule
    \multicolumn{2}{l}{\textbf{Total}} &  & \textbf{745} $\pm$ \textbf{10} \\ \bottomrule
    \end{tabular}%
    } 
    \vspace{3mm}
    \caption[CAMELEON inference profiling]{Profiling for the inference of YOLOv5-N engine with integer
    precision on one CAMELEON image. Timings are measured on a Nvidia AGX
    Xavier.}
    \label{tab:depl_profiling}
    \end{table}

To assess how the precision reduction changes both the inference speed and
detection performance, we compare multiple deployment strategies with the
YOLOv5-N model. Specifically, we compare the inference speed of the Pytorch
model (with F32 and F16 precisions) with the TensorRT engines (F32, F16 and INT8
precisions). The inference times are measured both on an RTX 3090 and an Nvidia
AGX Xavier, with distinct image widths: 512 and 4096 pixels. The results of this
comparison are available in \cref{tab:model_precision_comparison}. First, there
is an impressive speed gap between the TensorRT engines and Pytorch models:
TensorRT engines are much faster. This is expected given all the optimizations
conducted and the relative slowness of Python. However, it is worth noting
different behaviors between the two GPUs. Of course, the RTX 3090 is faster, but
changing the engine precision from F32 to F16 does not produce similar gains for
the AGX Xavier. This is explained by the change of microarchitecture between
these two GPUs. The AGX Xavier is based on the Volta architecture while the RTX
3090 leverages the newer Ampere architecture. In particular, Ampere introduces a
new generation of Tensor core, specifically designed for F32 computations.
Reducing the precision of the models generates a slight overhead and higher gains
are observed when the GPU utilization is higher, \ie when using larger images or
increased batch size. This is why we report the results both for 512 and 4096
image widths. Inference gains are more important with larger image sizes,
fortunately for our application. Then, we also report the mAP for each model. We
observe no performance drop when switching from F32 to F16. But, there is a
slight decrease of mAP with the quantization in INT8. This is expected as a lot
of approximations are made in the computations. On the other hand, we also
observe a performance drop with the TensorRT conversion, with the F32 engine
this is not expected as no approximation should be done. While the drop is
acceptable, it should be investigated in depth to better understand its origin.
Finally, we report some qualitative results of the INT8 inference on a
CAMELEON-size image in \cref{fig:inference_xavier}. The inference on the whole
image was performed in less than a second on an AGX Xavier. 

\begin{table}[]
    \centering
    \vspace{8mm}
    \resizebox{0.9\textwidth}{!}{%
    \begin{tabular}{@{\hspace{2mm}}ccccclc@{\hspace{2mm}}}
    \toprule[1pt]
         & \multicolumn{2}{c}{\textbf{RTX 3090 Latency}} & \multicolumn{2}{c}{\textbf{AGX Xavier Latency}}       &  &    \\ \midrule 
    Image Size (pixels)    & 512          & 4096          & 512         & 4096         &  &  \textbf{mAP}$_{\bm{0.5}}$     \\ \midrule 
    Pytorch F32   & 6            & 36.3          & 87.5        & 2638.6       &  & 0.679 \\
    Pytorch F16   & 8.3          & 20.4          & 90.7        & 2616.5       &  & 0.679 \\
    TensorRT F32  & 0.75         & 17.16         & 4.65        & 199.22       &  & 0.565 \\
    TensorRT F16  & 0.44         & 7.22          & 4.64        & 106.2        &  & 0.565 \\
    TensorRT INT8 & 0.43         & 5.77          & 2.48        & 68.65        &  & 0.523 \\ \bottomrule[1pt]
    \end{tabular}%
    } 
    \vspace{2mm}
    \caption[Engine inference speed comparison]{Inference speed comparison (in ms) between various
    deployment strategies and precisions for YOLOv5-N. Inference times are
    measured with two image widths and on two GPUs. A high-end desktop GPU: RTX
    3090, and an embedded GPU: AGX Xavier. mAP with a 0.5 IoU threshold is also
    reported on DOTA.}
    \label{tab:model_precision_comparison}
    \end{table}

\begin{figure}
    \centering
    \includegraphics[width=\textwidth]{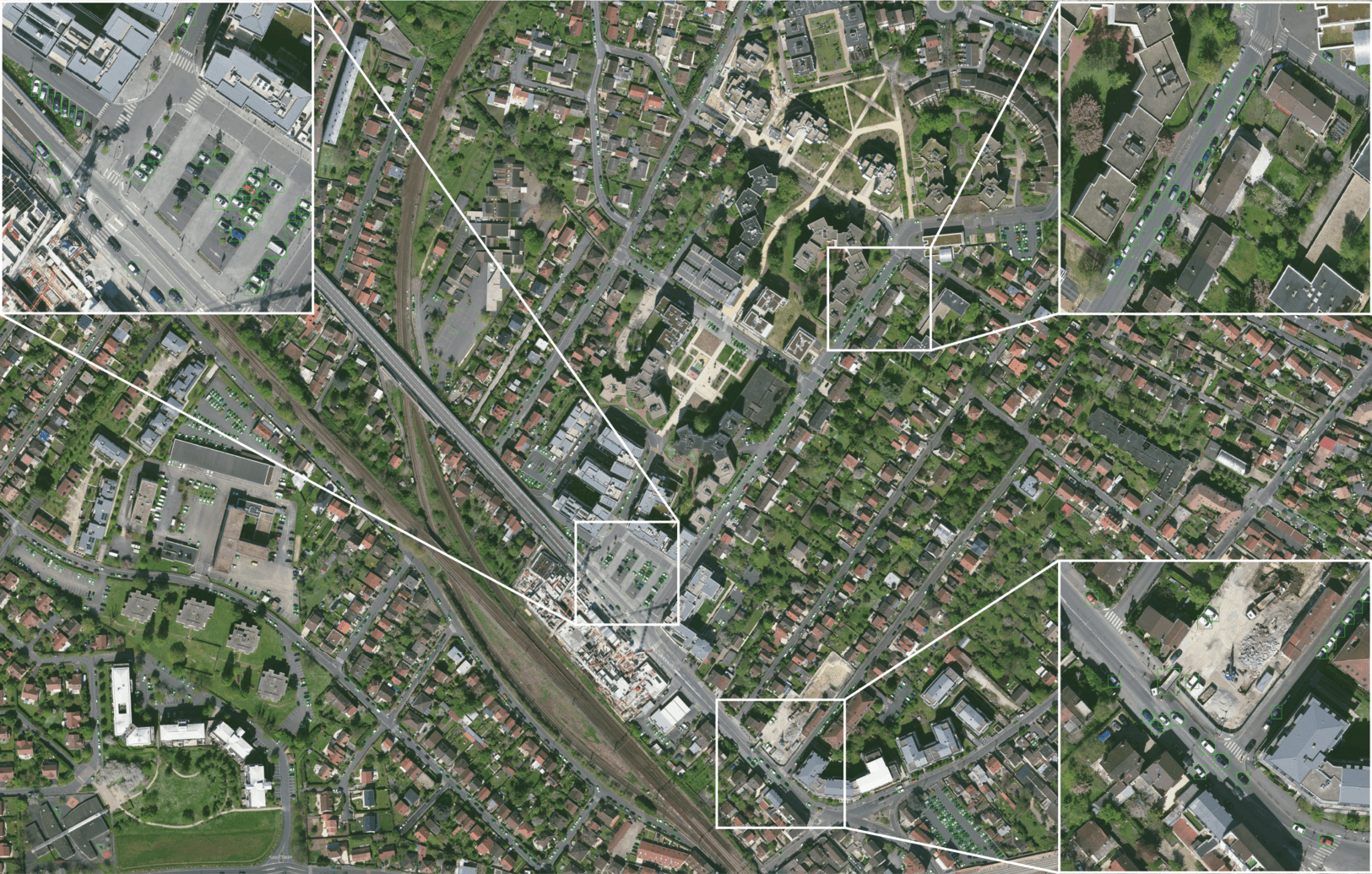}
    \caption[Qualitative detection performance of the deployed detection
    engine]{Qualitative detection performance on CAMELEON-size image with
    YOLOv5-N INT8 engine.}
    \label{fig:inference_xavier}
\end{figure}

More experiments are planned in future work to find an even better compromise
between inference speed and detection quality, especially with the Orin GPUs.
However, the compromise attained with YOLOv5-N is already satisfactory from an
industrial perspective and will be integrated into the first CAMELEON prototype.
The connection with the CAMELEON user interface \cref{fig:cameleon_ihm} has
already been done and the models wait for the first flight to be tested on real
images.   

\begin{figure}
    \centering
    \includegraphics[width=\textwidth]{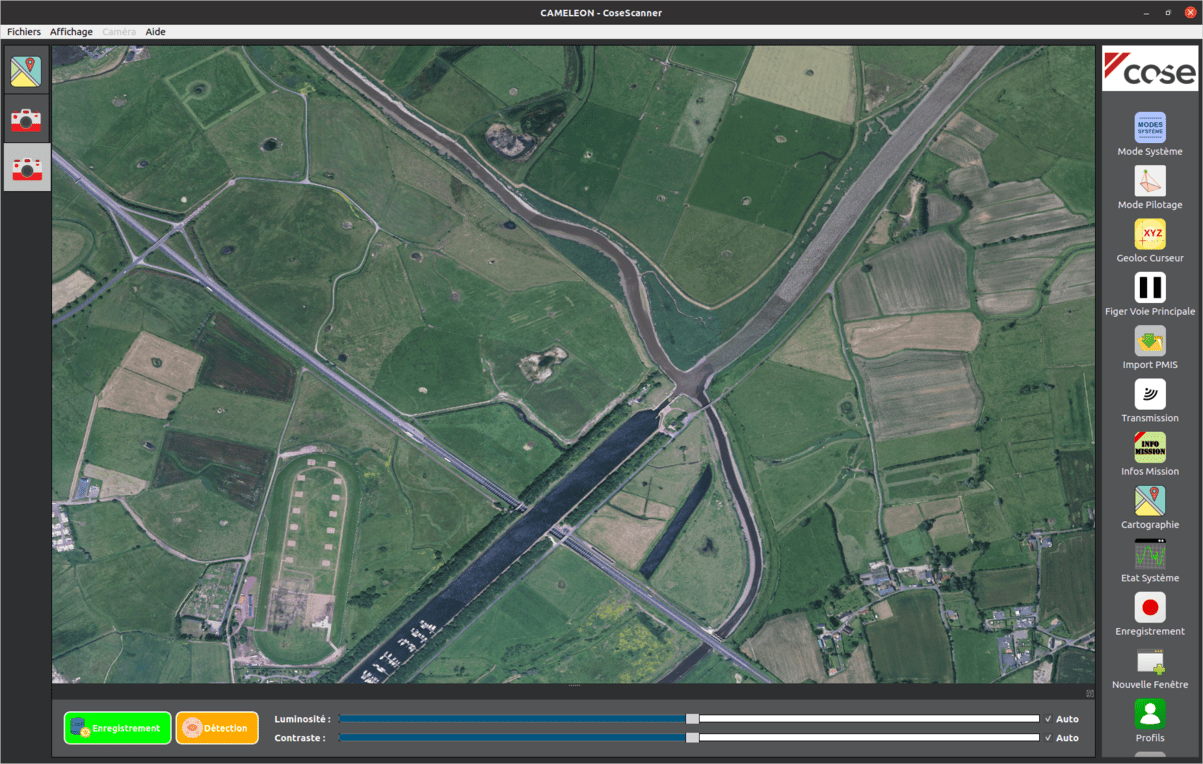}
    \caption{CAMELEON user interface.}
    \label{fig:cameleon_ihm}
\end{figure}

\section{Conclusion}
\vspace{-1em}
In this chapter, we have presented some tools and tricks that are leveraged to
greatly accelerate the inference of neural networks. Thanks to these tools, we
are able to achieve real-time detection on images of hundreds of megapixels,
embedded on a low-consumption GPU. Of course, better compromises could be found
with a more complete study of the YOLOv5 framework and especially by exploring
the influence of the model size on the performance with greater granularity.
Aside from finding a better compromise, it would be relevant to investigate
deeper knowledge distillation and unlock its full potential for object
detection. It is also required to explore new detection frameworks (\eg YOLOv8),
in particular with the recent AGX Orin. Its improved computation power could
allow for larger model sizes and increased performance while satisfying the time
constraint. Finally, this deployment has been conducted for regular detection
models, it should be done as well for Few-Shot detection models, even though it
complexifies greatly the analysis and the creation of TensorRT engines. While it
is unrealistic to deploy attention-based models in this manner (due to increased
inference time), fine-tuned models are much more adapted. It would increase the
adaptation time by adding the engine creation after fine-tuning, but the
inference will remain unchanged. Of course, it would be necessary to study the
influence of the model size on the performance and see if that is compatible
with CAMELEON's requirements.

\bookmarksetup{startatroot}
\makeatletter
\@addtoreset{table}{section}
\AtBeginEnvironment{table}{\ifnumequal{\value{table}}{0}{\addtocontents{lot}{\protect\addvspace{10pt}}}{}}
\makeatother
\stopcontents[chapters]
\chapter*{Conclusion and Perspectives}
\addcontentsline{toc}{chapter}{Conclusion and Perspectives}
\stepcounter{chapter}


In a nutshell, we first summarize the various contributions made
to the Few-Shot Object Detection field. Then, we take a step back from these
contributions and discuss what are the most promising research directions that
should be followed in future work. Finally, we present the remaining industrial
challenges that COSE will face before achieving robust and adaptable object
detection in an embedded environment.

\section*{Contribution Summary}
\addcontentsline{toc}{section}{Contribution summary}
\vspace{-1em}
In the first part of the manuscript we have reviewed thoroughly the literature
about Object Detection, Few-Shot Learning and especially Few-Shot Object
Detection. These three domains are the foundations of this PhD project. The
knowledge presented in the corresponding chapters is crucial for the development
of future FSOD algorithms and to meet the industrial needs of COSE. These three
fields are growing rapidly and it is necessary to remain up-to-date with
state-of-the-art both for academic and industrial research. It is especially
important for FSOD as it is a very recent problem and lacks consensus on how to
address it. Even if the FSOD literature is growing, it is still a small field
and most contributions are primarily designed for natural images. However, it is
not guaranteed that such techniques will perform well on other kinds of images.
In \cref{chap:aerial_diff}, we highlight especially that with an in-depth
analysis of the FSOD performance on aerial images, the performance drops
significantly when the methods are applied to these kinds of images. The main
reason behind this phenomenon is the smaller object size in aerial images. Even
if small objects are already difficult for regular object detection, the
challenge is much greater in the few-shot regime. In addition,
\cref{part:literature} highlights the organization of the FSOD literature into
three kinds of approaches. We divided \cref{part:contributions_fsod} into three
chapters accordingly, each focusing on a different FSOD approach: metric
learning, attention-based methods and fine-tuning strategies. Specifically, in
\cref{chap:prcnn}, we have proposed an original FSOD method entirely based on
metric learning. It embeds prototypical networks into the well-known Faster
R-CNN detection model. This naive approach achieves mitigated results but is
highly adaptable, it can adapt to novel classes without fine-tuning or heavy
computation. The experiments conducted with this model are instructive for the
development of future FSOD methods. Then, in \cref{chap:aaf}, we focused on the
attention-based mechanisms for FSOD. We proposed a modular framework to
implement and compare existing attention-based FSOD methods. Thorough
experiments on both aerial and natural images showed the superiority of local
attention mechanisms, called alignment. To reduce the performance gap between
natural and aerial images, a novel alignment technique is developed within the
framework to specifically address the detection of small objects. The resulting
FSOD approach outperforms existing work in the literature on aerial image
datasets as it improves largely the detection of small targets. Finally,
\cref{chap:diffusion} studied the last FSOD approach: the fine-tuning strategy.
Leveraging the recent DiffusionDet detector based on diffusion models, we
propose a simple fine-tuning approach that outperforms significantly previous
techniques. The simplicity of this approach and its impressive detection quality
allows for tackling more difficult scenarios such as Cross-Domain FSOD.
Fine-tuning based FSOD is also better suited for transductive inference than
metric-learning or attention approaches as it does not change the detection
model much. Both transductive inference and Cross-Domain scenarios are of
particular interest to COSE as they match better the real application. Our
experiments in these directions show promising results and should be extended in
future work. All detection models heavily rely on the Intersection over Union
(IoU) for training and evaluation. However, we showed in \cref{chap:siou_metric}
that it is not an optimal choice, especially when dealing with small objects.
Thus, it becomes critical when applying FSOD to aerial images. Therefore, we
proposed Scale-adaptative Intersection over Union to replace IoU. It
significantly improves the training of few-shot object detectors, as it allows
for a better balance between small and large objects. In addition, it aligns
better with human perception and is then a better choice for the evaluation of
object detectors. Then, with a more industrial mindset, we optimized and
deployed several object detectors inside the CAMELEON prototype. We explained
our process, the tools leveraged, and the compromises made in
\cref{chap:integration}. This chapter focuses only on regular object detection
as it is a first step before deploying more complex models able to generalize
either to new classes or new domains.

\begin{table}[h]
    \resizebox{\textwidth}{!}{%
    \begin{tabular}{@{\hspace{2mm}}crccc@{\hspace{2mm}}}
    \toprule[1pt]
                                 &                                                                                 & \textbf{Prototypical Faster R-CNN} & \textbf{XQSA}      & \textbf{FSDiffusionDet} \\ \midrule
                                 & \textbf{Approach}                                                               & Metric-Learning           & Attention & Fine-tuning    \\ \midrule
    \multirow{4}{*}{\textbf{Performance}} & Novel classes performance                                              & \rtd \rtd                 & \gtu      & \gtu\gtu              \\
                                 & Base classes performance                                                        & \rtd                      & \gtu      & \gtu\gtu              \\
                                 & \begin{tabular}[c]{@{}r@{}}Extremely low-shot \\[-1.5mm] performance\end{tabular}       & \rtd                      & \gtu      & \rtd              \\
                                 & Inference speed                                                                 & \gtu                      & \rtd      & \gtu              \\ \midrule
    \multirow{6}{*}{\textbf{Flexibility}} & Class scalability                                                      & \gtu                      & \rtd      & \gtu              \\
                                 & Shot scalability                                                                & \rtd                      & \gtu      & \gtu\gtu              \\
                                 & Fixed class number                                                              & \textcolor{OliveGreen}{\textbf{No}}& \textcolor{OliveGreen}{\textbf{No}}        & \textcolor{BrickRed}{\textbf{Yes}}            \\
                                 & Test-time adaptation                                                            & \textcolor{OliveGreen}{\textbf{Yes}} & \textcolor{BrickRed}{\textbf{No}}       & \textcolor{BrickRed}{\textbf{No}}             \\
                                 & \begin{tabular}[c]{@{}r@{}}Cross Domain Adaptation \\[-1.1mm] Capabilities\end{tabular} & Not tested                & Not tested         & \begin{tabular}[c]{@{}c@{}}\textcolor{OliveGreen}{\textbf{Promising}} \\[-1.5mm] \textcolor{OliveGreen}{\textbf{results}}\end{tabular}            \\ \midrule
    \multirow{4}{*}{\textbf{Training}}    & Base training time                                                     & \rtd                      & \rtd      & \gtu              \\
                                 & Fine-tuning required                                                            & \textcolor{OliveGreen}{\textbf{No}}& \textcolor{BrickRed}{\textbf{Yes}}       & \textcolor{BrickRed}{\textbf{Yes}}            \\
                                 & Fine-tuning time                                                                & \gtu                      & \rtd \rtd & \rtd              \\ 
                                 & Episodic training$^{\ast }$                                                               & \textcolor{BrickRed}{\textbf{Yes}} & \textcolor{BrickRed}{\textbf{Yes}} & \textcolor{OliveGreen}{\textbf{No}}             \\ \midrule
    \multirow{2}{*}{\textbf{Deployment}}  & Technical complexity$^{\star}$                                         & \gtu                      & \rtd         & \gtu\gtu              \\
                                 & TensorRT optimization$^{\dagger}$                                               & \rtd                      & \rtd\rtd         & \gtu              \\ \bottomrule[1pt]
    \end{tabular}%
    } \caption[Proposed methods overall comparison]{Comparison table between the
    three proposed approaches in this thesis: Prototypical Faster R-CNN, XQSA
    and FSDiffusionDet. The ratings are given according to the experiments
    conducted throughout this project and the insights generated. Green and red
    colors denote pros and cons, respectively, best viewed in colors. $^{\ast
    }$Episodic training is more complex and often subject to instabilities.
    $^{\star}$Technical complexity of the methods is a subjective criterion that
    measures how intricate a method is. $^{\dagger}$TensorRT optimization is not
    compatible with custom layers or complex data streams which greatly
    complexifies the optimization of more elaborated methods such as
    attention-based approaches.}
    \label{tab:conclusion_tab}
    \end{table}

Finally, we take a step back and compare the three approaches that we proposed
in this thesis: Prototypical Faster R-CNN, XQSA and FSDiffusionDet. Each has its
pros and cons, even if PFRCNN does not perform well it can be adapted at
test time which is a highly desirable property. On the contrary, FSDiffusionDet
is less flexible but achieves significantly higher performance. This comparison
is available in \cref{tab:conclusion_tab}, it compares the three methods on
various criteria, grouped into four categories: Performance, Flexibility,
Training and Deployment. These categories encompass the desired capabilities for
a detector model inside the CAMELEON system. 

\vspace{-0.5em}
\section*{Future Research Tracks}
\addcontentsline{toc}{section}{Future Research Tracks}
\vspace{-1.3em}
In this PhD, various directions have been explored for improving Few-Shot Object
Detection. We believe that we have provided significant contributions and
answered relevant questions in the FSOD field. However, our analysis raises new
questions and problems. First, we have proposed several improvements toward the
detection of small objects in the few-shot regime. Nonetheless, there is still a
significant performance gap between small and large objects and more effort
should be put in this direction. To this end, we have explored the attention
mechanisms and loss function design, but we have not looked into the training
examples selection strategy even though it may have a significant influence on
the small and large objects' balance during training. Incidentally, it would be
relevant to replace IoU inside the example selection process with SIoU. This
would favor smaller objects and certainly improves overall performance on aerial
images. At least, it would offer more control over the training balance between
small and large objects. SIoU is a controllable criterion, but it has a
drawback, it requires the setting of two parameters and their optimal choices
require some trials and errors. It would be of great use to come up with an
automatic strategy that could provide the optimal parameters for a given problem
or dataset. 

The main motivation of Few-Shot Learning is to mimic the human ability to learn
from very few examples. Of course, we are still far from solving FSOD, and a lot
of efforts are still needed to close the gap with fully supervised learning, not
to mention human-level perception. Nevertheless, the improvements made through
this project encourage the study of more complex yet more realistic settings.
One is particularly relevant for COSE's use case: Cross-Domain FSOD. This
problem is still mostly untouched in the literature. We conducted some early
experiments in this direction in \cref{chap:diffusion} and demonstrated
promising results. However, dedicated designs are required to get real
improvements as applying FSOD methods directly on this task is certainly not
optimal. Many directions are worth a try, taking inspiration from classification
methods, \eg generative modeling or discrepancy-based adaptation. The
transductive inference is another promising approach, it makes some additional
assumptions, but it is still realistic at least from COSE's perspective The
transductive approach has not been explored for detection and could bring
desirable properties to the FSOD field, especially test-time adaptability which
is still very challenging.

\vspace{-0.5em}
\section*{Remaining Industrial Challenges}
\addcontentsline{toc}{section}{Remaining Industrial Challenges}
\vspace{-1.3em}
Now from COSE's perspective, plenty of challenges are still to be addressed.
First, the current CAMELEON prototype only has a single camera but the final
product is meant to have five more. It would probably be quite difficult to achieve
real-time detection on all images. Of course, we could continue optimizing the
models and find better speed/accuracy tradeoffs. The recent advances in the YOLO
family are promising, and substantial gains are achieved with newer GPUs.
However, it may not be sufficient to multiply the throughput of the detection
pipeline by six. Fortunately, as there is some overlapping between the images, it
is not required to process all six data streams independently. Instead, we could
determine the overlapping areas to avoid processing them multiple times. But
that is not trivial and requires image registration techniques which are also
expensive, especially at the size of CAMELEON's images. 

While the current prototype achieves satisfactory detection performance, there
is still room for improvement. Our attempt to apply Knowledge Distillation (KD)
to aerial images has been fruitless. Nevertheless, distillation is a promising
direction to improve detection performance while maintaining a fixed computation
budget. Some effort should be spent on developing novel KD methods that are also
beneficial with aerial images and small objects.  

Currently, only regular object detection algorithms have been deployed inside
the prototype. While this is certainly a major step, it is still limited by the
training classes and domain. As it is not possible to know in advance the kind
of images that will be encountered during missions, mainly because of
confidentiality constraints, we cannot guarantee the quality of the detections
in operation. It would be extremely valuable to achieve real-time class and
domain adaptation from a few examples. It would result in a single model that
could be adapted on-the-fly (literally) according to the operation's needs.
However, state-of-the-art FSOD is still far from this and COSE should probably
not invest in this direction in the short run, especially as such adaptative
methods will probably come at the cost of lower detection quality and slower
inference. Instead, COSE should focus on developing an efficient fine-tuning
platform that could be used by the forces to train their own models for specific
missions. It could not be done during a mission, but afterward using the images
collected and a few annotations. It would help to analyze the mission data
faster and could produce new real-time models for subsequent similar missions.
The key would be to conceive a unique application that can handle image
annotation, model fine-tuning and deployment, all at once.

\cleardoublepage
\phantomsection
\addcontentsline{toc}{chapter}{Bibliography}
\printbibliography

@article{ren2015faster,
  title={Faster r-cnn: Towards real-time object detection with region proposal networks},
  author={Ren, Shaoqing and He, Kaiming and Girshick, Ross and Sun, Jian},
  journal={Advances in neural information processing systems},
  volume={28},
  pages={91--99},
  year={2015}
}

@inproceedings{karlinsky2019repmet,
  title={Repmet: Representative-based metric learning for classification and few-shot object detection},
  author={Karlinsky, Leonid and Shtok, Joseph and Harary, Sivan and Schwartz, Eli and Aides, Amit and Feris, Rogerio and Giryes, Raja and Bronstein, Alex M},
  booktitle={Proceedings of the IEEE/CVF Conference on Computer Vision and Pattern Recognition},
  pages={5197--5206},
  year={2019}
}

@inproceedings{fan2020few,
  title={Few-shot object detection with attention-RPN and multi-relation detector},
  author={Fan, Qi and Zhuo, Wei and Tang, Chi-Keung and Tai, Yu-Wing},
  booktitle={Proceedings of the IEEE/CVF Conference on Computer Vision and Pattern Recognition},
  pages={4013--4022},
  year={2020}
}

@inproceedings{kang2019few,
  title={Few-shot object detection via feature reweighting},
  author={Kang, Bingyi and Liu, Zhuang and Wang, Xin and Yu, Fisher and Feng, Jiashi and Darrell, Trevor},
  booktitle={Proceedings of the IEEE/CVF International Conference on Computer Vision},
  pages={8420--8429},
  year={2019}
}

@inproceedings{wu2020meta,
  title={Meta-rcnn: Meta learning for few-shot object detection},
  author={Wu, Xiongwei and Sahoo, Doyen and Hoi, Steven},
  booktitle={Proceedings of the 28th ACM International Conference on Multimedia},
  pages={1679--1687},
  year={2020}
}

@inproceedings{snell2017prototypical,
 author = {Snell, Jake and Swersky, Kevin and Zemel, Richard},
 booktitle = {Advances in Neural Information Processing Systems},
 editor = {I. Guyon and U. V. Luxburg and S. Bengio and H. Wallach and R. Fergus and S. Vishwanathan and R. Garnett},
 pages = {},
 title = {Prototypical Networks for Few-shot Learning},
 volume = {30},
 year = {2017}
}

@inproceedings{girshick2014rich,
  title     = {Rich feature hierarchies for accurate object detection and semantic segmentation},
  author    = {Girshick, Ross and Donahue, Jeff and Darrell, Trevor and Malik, Jitendra},
  booktitle = {Proceedings of the IEEE conference on computer vision and pattern recognition},
  pages     = {580--587},
  year      = {2014}
}

@inproceedings{girshick2015fast,
  title={Fast r-cnn},
  author={Girshick, Ross},
  booktitle={Proceedings of the IEEE international conference on computer vision},
  pages={1440--1448},
  year={2015}
}

@inproceedings{lin2017feature,
  title={Feature pyramid networks for object detection},
  author={Lin, Tsung-Yi and Doll{\'a}r, Piotr and Girshick, Ross and He, Kaiming and Hariharan, Bharath and Belongie, Serge},
  booktitle={Proceedings of the IEEE conference on computer vision and pattern recognition},
  pages={2117--2125},
  year={2017}
}

@inproceedings{carion2020end,
  title={End-to-end object detection with transformers},
  author={Carion, Nicolas and Massa, Francisco and Synnaeve, Gabriel and Usunier, Nicolas and Kirillov, Alexander and Zagoruyko, Sergey},
  booktitle={European Conference on Computer Vision},
  pages={213--229},
  year={2020},
  organization={Springer}
}

@inproceedings{tian2019fcos,
  title={Fcos: Fully convolutional one-stage object detection},
  author={Tian, Zhi and Shen, Chunhua and Chen, Hao and He, Tong},
  booktitle={Proceedings of the IEEE/CVF International Conference on Computer Vision},
  pages={9627--9636},
  year={2019}
}

@inproceedings{redmon2016you,
  title={You only look once: Unified, real-time object detection},
  author={Redmon, Joseph and Divvala, Santosh and Girshick, Ross and Farhadi, Ali},
  booktitle={Proceedings of the IEEE conference on computer vision and pattern recognition},
  pages={779--788},
  year={2016}
}

@inproceedings{chen2018lstd,
  title={Lstd: A low-shot transfer detector for object detection},
  author={Chen, Hao and Wang, Yali and Wang, Guoyou and Qiao, Yu},
  booktitle={Proceedings of the AAAI Conference on Artificial Intelligence},
  volume={32},
  number={1},
  year={2018}
}

@inproceedings{wang2020frustratingly,
    title={Frustratingly Simple Few-Shot Object Detection},
    author={Wang, Xin and Huang, Thomas E. and  Darrell, Trevor and Gonzalez, Joseph E and Yu, Fisher},
    booktitle = {International Conference on Machine Learning (ICML)},
    month = {July},
    year = {2020}
}

@ARTICLE{deng2020few,
  author={Li, Xiang and Deng, Jingyu and Fang, Yi},
  journal={IEEE Transactions on Geoscience and Remote Sensing}, 
  title={Few-Shot Object Detection on Remote Sensing Images}, 
  year={2021},
  volume={},
  number={},
  pages={1-14},
  doi={10.1109/TGRS.2021.3051383}}

@inproceedings{chen2020simple,
  title={A simple framework for contrastive learning of visual representations},
  author={Chen, Ting and Kornblith, Simon and Norouzi, Mohammad and Hinton, Geoffrey},
  booktitle={International conference on machine learning},
  pages={1597--1607},
  year={2020},
  organization={PMLR}
}

@inproceedings{sung2018learning,
  title={Learning to compare: Relation network for few-shot learning},
  author={Sung, Flood and Yang, Yongxin and Zhang, Li and Xiang, Tao and Torr, Philip HS and Hospedales, Timothy M},
  booktitle={Proceedings of the IEEE conference on computer vision and pattern recognition},
  pages={1199--1208},
  year={2018}
}

@inproceedings{vinyals2016matching,
 author = {Vinyals, Oriol and Blundell, Charles and Lillicrap, Timothy and kavukcuoglu, koray and Wierstra, Daan},
 booktitle = {Advances in Neural Information Processing Systems},
 editor = {D. Lee and M. Sugiyama and U. Luxburg and I. Guyon and R. Garnett},
 pages = {},
 title = {Matching Networks for One Shot Learning},
 volume = {29},
 year = {2016}
}

@inproceedings{ravi2016optimization,
  title={Optimization as a model for few-shot learning},
  author={Larochelle, S Ravi H},
  booktitle={5th International Conference on Learning Representations, ICLR 2017, Toulon, France, April 24-26, 2017, Conference Track Proceedings},
  year={2017}
}

@inproceedings{fan2020fsod,
  title={Few-Shot Object Detection with Attention-RPN and Multi-Relation Detector},
  author={Fan, Qi and Zhuo, Wei and Tang, Chi-Keung and Tai, Yu-Wing},
  booktitle={CVPR},
  year={2020}
}

@article{xiao2020fsod,
  title={Few-Shot Object Detection With Self-Adaptive Attention Network for Remote Sensing Images},
  author={Zixuan Xiao and Jiahao Qi and Wei Xue and P. Zhong},
  journal={IEEE Journal of Selected Topics in Applied Earth Observations and Remote Sensing},
  year={2021},
  volume={14},
  pages={4854-4865}
}

@inproceedings{kim2020few,
  title={Few-Shot Object Detection via Knowledge Transfer},
  author={Kim, Geonuk and Jung, Hong-Gyu and Lee, Seong-Whan},
  year={2020},
  booktitle={2020 IEEE International Conference on Systems, Man, and Cybernetics (SMC)},
  pages={3564--3569}
}

@inproceedings{he2016deep,
  title={Deep residual learning for image recognition},
  author={He, Kaiming and Zhang, Xiangyu and Ren, Shaoqing and Sun, Jian},
  booktitle={Proceedings of the IEEE conference on computer vision and pattern recognition},
  pages={770--778},
  year={2016}
}

@inproceedings{caron2018deep,
  title={Deep clustering for unsupervised learning of visual features},
  author={Caron, Mathilde and Bojanowski, Piotr and Joulin, Armand and Douze, Matthijs},
  booktitle={Proceedings of the European Conference on Computer Vision (ECCV)},
  pages={132--149},
  year={2018}
}

@inproceedings{xia2018dota,
  title={DOTA: A large-scale dataset for object detection in aerial images},
  author={Xia, Gui-Song and Bai, Xiang and Ding, Jian and Zhu, Zhen and Belongie, Serge and Luo, Jiebo and Datcu, Mihai and Pelillo, Marcello and Zhang, Liangpei},
  booktitle={Proceedings of the IEEE Conference on Computer Vision and Pattern Recognition},
  pages={3974--3983},
  year={2018}
}

@inproceedings{hoffer2015deep,
  title={Deep metric learning using triplet network},
  author={Hoffer, Elad and Ailon, Nir},
  booktitle={International workshop on similarity-based pattern recognition},
  pages={84--92},
  year={2015},
  organization={Springer}
}

@article{everingham2010pascal,
  title={The pascal visual object classes (voc) challenge},
  author={Everingham, Mark and Van Gool, Luc and Williams, Christopher KI and Winn, John and Zisserman, Andrew},
  journal={International journal of computer vision},
  volume={88},
  number={2},
  pages={303--338},
  year={2010},
  publisher={Springer}
}

@inproceedings{he2017mask,
  title={Mask r-cnn},
  author={He, Kaiming and Gkioxari, Georgia and Doll{\'a}r, Piotr and Girshick, Ross},
  booktitle={Proceedings of the IEEE international conference on computer vision},
  pages={2961--2969},
  year={2017}
}

@inproceedings{yan2019meta,
  title={Meta r-cnn: Towards general solver for instance-level low-shot learning},
  author={Yan, Xiaopeng and Chen, Ziliang and Xu, Anni and Wang, Xiaoxi and Liang, Xiaodan and Lin, Liang},
  booktitle={Proceedings of the IEEE/CVF International Conference on Computer Vision},
  pages={9577--9586},
  year={2019}
}

@article{kirkpatrick2017overcoming,
  title={Overcoming catastrophic forgetting in neural networks},
  author={Kirkpatrick, James and Pascanu, Razvan and Rabinowitz, Neil and Veness, Joel and Desjardins, Guillaume and Rusu, Andrei A and Milan, Kieran and Quan, John and Ramalho, Tiago and Grabska-Barwinska, Agnieszka and others},
  journal={Proceedings of the national academy of sciences},
  volume={114},
  number={13},
  pages={3521--3526},
  year={2017},
  publisher={National Acad Sciences}
}

@inproceedings{finn2017model,
  title={Model-agnostic meta-learning for fast adaptation of deep networks},
  author={Finn, Chelsea and Abbeel, Pieter and Levine, Sergey},
  booktitle={International Conference on Machine Learning},
  pages={1126--1135},
  year={2017},
  organization={PMLR}
}

@article{bromley1993signature,
  title={Signature verification using a" siamese" time delay neural network},
  author={Bromley, Jane and Guyon, Isabelle and LeCun, Yann and S{\"a}ckinger, Eduard and Shah, Roopak},
  journal={Advances in neural information processing systems},
  volume={6},
  pages={737--744},
  year={1993}
}

@incollection{wallach2019one,
  title     = {One-Shot Object Detection with Co-Attention and Co-Excitation},
  author    = {Hsieh, Ting-I and Lo, Yi-Chen and Chen, Hwann-Tzong and Liu, Tyng-Luh},
  booktitle = {Advances in Neural Information Processing Systems 32},
  year      = {2019}
}

@article{li2020one,
  title   = {One-shot object detection without fine-tuning},
  author  = {Li, Xiang and Zhang, Lin and Chen, Yau Pun and Tai, Yu-Wing and Tang, Chi-Keung},
  journal = {arXiv preprint arXiv:2005.03819},
  year    = {2020}
}

@inproceedings{sun2021fsce,
  title     = {FSCE: Few-shot object detection via contrastive proposal encoding},
  author    = {Sun, Bo and Li, Banghuai and Cai, Shengcai and Yuan, Ye and Zhang, Chi},
  booktitle = {Proceedings of the IEEE/CVF Conference on Computer Vision and Pattern Recognition},
  pages     = {7352--7362},
  year      = {2021}
}

@inproceedings{wang2018non,
  title     = {Non-local neural networks},
  author    = {Wang, Xiaolong and Girshick, Ross and Gupta, Abhinav and He, Kaiming},
  booktitle = {Proceedings of the IEEE conference on computer vision and pattern recognition},
  pages     = {7794--7803},
  year      = {2018}
}

@inproceedings{vaswani2017attention,
  title     = {Attention is all you need},
  author    = {Vaswani, Ashish and Shazeer, Noam and Parmar, Niki and Uszkoreit, Jakob and Jones, Llion and Gomez, Aidan N and Kaiser, {\L}ukasz and Polosukhin, Illia},
  booktitle = {Advances in neural information processing systems},
  pages     = {5998--6008},
  year      = {2017}
}

@inproceedings{han2021meta,
  title     = {Meta r-cnn: Towards general solver for instance-level low-shot learning},
  author    = {Yan, Xiaopeng and Chen, Ziliang and Xu, Anni and Wang, Xiaoxi and Liang, Xiaodan and Lin, Liang},
  booktitle = {Proceedings of the IEEE/CVF International Conference on Computer Vision},
  pages     = {9577--9586},
  year      = {2019}
}

@article{liu2021dynamic,
  title   = {Dynamic Relevance Learning for Few-Shot Object Detection},
  author  = {Liu, Weijie and Li, Haohe and Yu, Shenghao and Chen, Song and Ye, Xulun and Wu, Jiafei and others},
  journal = {arXiv preprint arXiv:2108.02235},
  year    = {2021}
}

@inproceedings{xiao2020few,
  title        = {Few-shot object detection and viewpoint estimation for objects in the wild},
  author       = {Xiao, Yang and Marlet, Renaud},
  booktitle    = {European Conference on Computer Vision},
  pages        = {192--210},
  year         = {2020},
  organization = {Springer}
}

@article{gao2021fast,
  title     = {A Fast and Accurate Few-Shot Detector for Objects with Fewer Pixels in Drone Image},
  author    = {Gao, Yuxuan and Hou, Runmin and Gao, Qiang and Hou, Yuanlong},
  journal   = {Electronics},
  volume    = {10},
  number    = {7},
  pages     = {783},
  year      = {2021},
  publisher = {Multidisciplinary Digital Publishing Institute}
}

@article{zhang2021meta,
  author  = {Zhang, Gongjie and Luo, Zhipeng and Cui, Kaiwen and Lu, Shijian and Xing, Eric P.},
  journal = {IEEE Transactions on Pattern Analysis and Machine Intelligence},
  title   = {{Meta-DETR}: Image-Level Few-Shot Detection with Inter-Class Correlation Exploitation},
  year    = {2022},
  doi     = {10.1109/TPAMI.2022.3195735}
}

@article{chen2021should,
  author  = {Chen, Tung-I and Liu, Yueh-Cheng and Su, Hung-Ting and Chang, Yu-Cheng and Lin, Yu-Hsiang and Yeh, Jia-Fong and Chen, Wen-Chin and Hsu, Winston},
  journal = {IEEE Transactions on Multimedia},
  title   = {Dual-Awareness Attention for Few-Shot Object Detection},
  year    = {2021},
  volume  = {},
  number  = {},
  pages   = {1-1},
  doi     = {10.1109/TMM.2021.3125195}
}

@article{xu2021few,
  title     = {Few-Shot Object Detection via Sample Processing},
  author    = {Xu, Honghui and Wang, Xinqing and Shao, Faming and Duan, Baoguo and Zhang, Peng},
  journal   = {IEEE Access},
  volume    = {9},
  pages     = {29207--29221},
  year      = {2021},
  publisher = {IEEE}
}

@inproceedings{chu2021joint,
  author    = {Chu, Jinghui and Feng, Jiawei and Jing, Peiguang and Lu, Wei},
  booktitle = {2021 IEEE International Conference on Image Processing (ICIP)},
  title     = {Joint Co-Attention And Co-Reconstruction Representation Learning For One-Shot Object Detection},
  year      = {2021},
  volume    = {},
  number    = {},
  pages     = {2229-2233},
  doi       = {10.1109/ICIP42928.2021.9506387}
}

@inproceedings{zhang2021pnpdet,
  title     = {PNPDet: Efficient few-shot detection without forgetting via plug-and-play sub-networks},
  author    = {Zhang, Gongjie and Cui, Kaiwen and Wu, Rongliang and Lu, Shijian and Tian, Yonghong},
  booktitle = {Proceedings of the IEEE/CVF Winter Conference on Applications of Computer Vision},
  pages     = {3823--3832},
  year      = {2021}
}

@inproceedings{wu2021universal,
  title     = {Universal-Prototype Enhancing for Few-Shot Object Detection},
  author    = {Wu, Aming and Han, Yahong and Zhu, Linchao and Yang, Yi},
  booktitle = {Proceedings of the IEEE/CVF International Conference on Computer Vision},
  pages     = {9567--9576},
  year      = {2021}
}

@inproceedings{wu2020multi,
  title        = {Multi-scale positive sample refinement for few-shot object detection},
  author       = {Wu, Jiaxi and Liu, Songtao and Huang, Di and Wang, Yunhong},
  booktitle    = {European Conference on Computer Vision},
  pages        = {456--472},
  year         = {2020},
  organization = {Springer}
}

@inproceedings{fan2021generalized,
  title     = {Generalized Few-Shot Object Detection without Forgetting},
  author    = {Fan, Zhibo and Ma, Yuchen and Li, Zeming and Sun, Jian},
  booktitle = {Proceedings of the IEEE/CVF Conference on Computer Vision and Pattern Recognition},
  pages     = {4527--4536},
  year      = {2021}
}

@inproceedings{lin2017focal,
  title     = {Focal loss for dense object detection},
  author    = {Lin, Tsung-Yi and Goyal, Priya and Girshick, Ross and He, Kaiming and Doll{\'a}r, Piotr},
  booktitle = {Proceedings of the IEEE international conference on computer vision},
  pages     = {2980--2988},
  year      = {2017}
}

@inproceedings{law2018cornernet,
  title     = {Cornernet: Detecting objects as paired keypoints},
  author    = {Law, Hei and Deng, Jia},
  booktitle = {Proceedings of the European conference on computer vision (ECCV)},
  pages     = {734--750},
  year      = {2018}
}

@inproceedings{duan2019centernet,
  title     = {Centernet: Keypoint triplets for object detection},
  author    = {Duan, Kaiwen and Bai, Song and Xie, Lingxi and Qi, Honggang and Huang, Qingming and Tian, Qi},
  booktitle = {Proceedings of the IEEE/CVF International Conference on Computer Vision},
  pages     = {6569--6578},
  year      = {2019}
}

@inproceedings{dai2021dynamic,
  title     = {Dynamic Head: Unifying Object Detection Heads with Attentions},
  author    = {Dai, Xiyang and Chen, Yinpeng and Xiao, Bin and Chen, Dongdong and Liu, Mengchen and Yuan, Lu and Zhang, Lei},
  booktitle = {Proceedings of the IEEE/CVF Conference on Computer Vision and Pattern Recognition},
  pages     = {7373--7382},
  year      = {2021}
}

@inproceedings{bertinetto2016learning,
  title     = {Learning feed-forward one-shot learners},
  author    = {Bertinetto, Luca and Henriques, Jo{\~a}o F and Valmadre, Jack and Torr, Philip and Vedaldi, Andrea},
  booktitle = {Advances in neural information processing systems},
  pages     = {523--531},
  year      = {2016}
}

@article{doersch2020crosstransformers,
  title   = {Crosstransformers: spatially-aware few-shot transfer},
  author  = {Doersch, Carl and Gupta, Ankush and Zisserman, Andrew},
  journal = {arXiv preprint arXiv:2007.11498},
  year    = {2020}
}

@article{wang2020generalizing,
  title     = {Generalizing from a few examples: A survey on few-shot learning},
  author    = {Wang, Yaqing and Yao, Quanming and Kwok, James T and Ni, Lionel M},
  journal   = {ACM Computing Surveys (CSUR)},
  volume    = {53},
  number    = {3},
  pages     = {1--34},
  year      = {2020},
  publisher = {ACM New York, NY, USA}
}

@inproceedings{lin2014microsoft,
  title        = {Microsoft coco: Common objects in context},
  author       = {Lin, Tsung-Yi and Maire, Michael and Belongie, Serge and Hays, James and Perona, Pietro and Ramanan, Deva and Doll{\'a}r, Piotr and Zitnick, C Lawrence},
  booktitle    = {European conference on computer vision},
  pages        = {740--755},
  year         = {2014},
  organization = {Springer}
}

@article{huang2021survey,
  title   = {A Survey of Self-Supervised and Few-Shot Object Detection},
  author  = {Huang, Gabriel and Laradji, Issam and Vazquez, David and Lacoste-Julien, Simon and Rodriguez, Pau},
  journal = {arXiv preprint arXiv:2110.14711},
  year    = {2021}
}

@article{li2020object,
  title     = {Object detection in optical remote sensing images: A survey and a new benchmark},
  author    = {Li, Ke and Wan, Gang and Cheng, Gong and Meng, Liqiu and Han, Junwei},
  journal   = {ISPRS Journal of Photogrammetry and Remote Sensing},
  volume    = {159},
  pages     = {296--307},
  year      = {2020},
  publisher = {Elsevier}
}

@article{yang2020restoring,
  title   = {Restoring negative information in few-shot object detection},
  author  = {Yang, Yukuan and Wei, Fangyun and Shi, Miaojing and Li, Guoqi},
  journal = {Advances in neural information processing systems},
  volume  = {33},
  pages   = {3521--3532},
  year    = {2020}
}

@inproceedings{zhang2021hallucination,
  title     = {Hallucination improves few-shot object detection},
  author    = {Zhang, Weilin and Wang, Yu-Xiong},
  booktitle = {Proceedings of the IEEE/CVF Conference on Computer Vision and Pattern Recognition},
  pages     = {13008--13017},
  year      = {2021}
}

@article{ba2016layer,
  title   = {Layer normalization},
  author  = {Ba, Jimmy Lei and Kiros, Jamie Ryan and Hinton, Geoffrey E},
  journal = {arXiv preprint arXiv:1607.06450},
  year    = {2016}
}

@inproceedings{wolf2021double,
  title     = {Double Head Predictor based Few-Shot Object Detection for Aerial Imagery},
  author    = {Wolf, Stefan and Meier, Jonas and Sommer, Lars and Beyerer, J{\"u}rgen},
  booktitle = {Proceedings of the IEEE/CVF International Conference on Computer Vision},
  pages     = {721--731},
  year      = {2021}
}

@inproceedings{bar2022detreg,
  title     = {Detreg: Unsupervised pretraining with region priors for object detection},
  author    = {Bar, Amir and Wang, Xin and Kantorov, Vadim and Reed, Colorado J and Herzig, Roei and Chechik, Gal and Rohrbach, Anna and Darrell, Trevor and Globerson, Amir},
  booktitle = {Proceedings of the IEEE/CVF Conference on Computer Vision and Pattern Recognition},
  pages     = {14605--14615},
  year      = {2022}
}

@article{huang2021few,
  title     = {Few-Shot Object Detection on Remote Sensing Images via Shared Attention Module and Balanced Fine-Tuning Strategy},
  author    = {Huang, Xu and He, Bokun and Tong, Ming and Wang, Dingwen and He, Chu},
  journal   = {Remote Sensing},
  volume    = {13},
  number    = {19},
  pages     = {3816},
  year      = {2021},
  publisher = {MDPI}
}

@misc{xu2022simpl,
  doi       = {10.48550/ARXIV.2106.15681},
  url       = {https://arxiv.org/abs/2106.15681},
  author    = {Xu, Yang and Huang, Bohao and Luo, Xiong and Bradbury, Kyle and Malof, Jordan M.},
  title     = {SIMPL: Generating Synthetic Overhead Imagery to Address Zero-shot and Few-Shot Detection Problems},
  publisher = {arXiv},
  year      = {2021},
  copyright = {arXiv.org perpetual, non-exclusive license}
}

@article{cao2021few,
  title   = {Few-Shot Object Detection via Association and DIscrimination},
  author  = {Cao, Yuhang and Wang, Jiaqi and Jin, Ying and Wu, Tong and Chen, Kai and Liu, Ziwei and Lin, Dahua},
  journal = {Advances in Neural Information Processing Systems},
  volume  = {34},
  pages   = {16570--16581},
  year    = {2021}
}

@article{wu2021generalized,
  title   = {Generalized and discriminative few-shot object detection via SVD-dictionary enhancement},
  author  = {Wu, Aming and Zhao, Suqi and Deng, Cheng and Liu, Wei},
  journal = {Advances in Neural Information Processing Systems},
  volume  = {34},
  pages   = {6353--6364},
  year    = {2021}
}

@article{liu2021gendet,
  title     = {Gendet: Meta learning to generate detectors from few shots},
  author    = {Liu, Liyang and Wang, Bochao and Kuang, Zhanghui and Xue, Jing-Hao and Chen, Yimin and Yang, Wenming and Liao, Qingmin and Zhang, Wayne},
  journal   = {IEEE Transactions on Neural Networks and Learning Systems},
  year      = {2021},
  publisher = {IEEE}
}

@inproceedings{han2022few,
  title     = {Few-shot object detection with fully cross-transformer},
  author    = {Han, Guangxing and Ma, Jiawei and Huang, Shiyuan and Chen, Long and Chang, Shih-Fu},
  booktitle = {Proceedings of the IEEE/CVF Conference on Computer Vision and Pattern Recognition},
  pages     = {5321--5330},
  year      = {2022}
}

@inproceedings{li2021transformation,
  title     = {Transformation invariant few-shot object detection},
  author    = {Li, Aoxue and Li, Zhenguo},
  booktitle = {Proceedings of the IEEE/CVF Conference on Computer Vision and Pattern Recognition},
  pages     = {3094--3102},
  year      = {2021}
}

@inproceedings{liu2018path,
  title     = {Path aggregation network for instance segmentation},
  author    = {Liu, Shu and Qi, Lu and Qin, Haifang and Shi, Jianping and Jia, Jiaya},
  booktitle = {Proceedings of the IEEE conference on computer vision and pattern recognition},
  pages     = {8759--8768},
  year      = {2018}
}

@inproceedings{viola2001rapid,
  title        = {Rapid object detection using a boosted cascade of simple features},
  author       = {Viola, Paul and Jones, Michael},
  booktitle    = {Proceedings of the 2001 IEEE computer society conference on computer vision and pattern recognition. CVPR 2001},
  volume       = {1},
  pages        = {I--I},
  year         = {2001},
  organization = {Ieee}
}

@article{sermanet2013overfeat,
  title   = {Overfeat: Integrated recognition, localization and detection using convolutional networks},
  author  = {Sermanet, Pierre and Eigen, David and Zhang, Xiang and Mathieu, Micha{\"e}l and Fergus, Rob and LeCun, Yann},
  journal = {arXiv preprint arXiv:1312.6229},
  year    = {2013}
}

@article{salton1983introduction,
  title   = {Introduction to modern information retrieval},
  author  = {Salton, Gerard},
  journal = {McGraw-Hill},
  year    = {1983}
}

@inproceedings{oksuz2018localization,
  title     = {Localization recall precision (LRP): A new performance metric for object detection},
  author    = {Oksuz, Kemal and Cam, Baris Can and Akbas, Emre and Kalkan, Sinan},
  booktitle = {Proceedings of the European conference on computer vision (ECCV)},
  pages     = {504--519},
  year      = {2018}
}

@article{jena2023beyond,
  author  = {Jena, Rohit and Zhornyak, Lukas and Doiphode, Nehal and Chaudhari, Pratik and Buch, Vivek and Gee, James and Shi, Jianbo},
  title   = {Beyond mAP: Towards better evaluation of instance segmentation},
  journal = {CVPR},
  year    = {2023}
}

@article{zou2023object,
  title     = {Object detection in 20 years: A survey},
  author    = {Zou, Zhengxia and Chen, Keyan and Shi, Zhenwei and Guo, Yuhong and Ye, Jieping},
  journal   = {Proceedings of the IEEE},
  year      = {2023},
  publisher = {IEEE}
}

@article{zaidi2022survey,
  title     = {A survey of modern deep learning based object detection models},
  author    = {Zaidi, Syed Sahil Abbas and Ansari, Mohammad Samar and Aslam, Asra and Kanwal, Nadia and Asghar, Mamoona and Lee, Brian},
  journal   = {Digital Signal Processing},
  pages     = {103514},
  year      = {2022},
  publisher = {Elsevier}
}

@article{goldstein1972man,
  title     = {Man-Machine Interaction in Human-Face Identification},
  author    = {Goldstein, A Jay and Harmon, Leon D and Lesk, Ann B},
  journal   = {Bell System Technical Journal},
  volume    = {51},
  number    = {2},
  pages     = {399--427},
  year      = {1972},
  publisher = {Wiley Online Library}
}

@article{goldstein1971identification,
  title     = {Identification of human faces},
  author    = {Goldstein, A Jay and Harmon, Leon D and Lesk, Ann B},
  journal   = {Proceedings of the IEEE},
  volume    = {59},
  number    = {5},
  pages     = {748--760},
  year      = {1971},
  publisher = {IEEE}
}

@incollection{kaya1972basic,
  title     = {A basic study on human face recognition},
  author    = {Kaya, Y and Kobayashi, K},
  booktitle = {Frontiers of pattern recognition},
  pages     = {265--289},
  year      = {1972},
  publisher = {Elsevier}
}

@article{Yuille1989FeatureEF,
  title   = {Feature extraction from faces using deformable templates},
  author  = {Alan Loddon Yuille and Peter W. Hallinan and David S. Cohen},
  journal = {International Journal of Computer Vision},
  year    = {1989},
  volume  = {8},
  pages   = {99-111}
}

@article{turk1991eigenfaces,
  title     = {Eigenfaces for recognition},
  author    = {Turk, Matthew and Pentland, Alex},
  journal   = {Journal of cognitive neuroscience},
  volume    = {3},
  number    = {1},
  pages     = {71--86},
  year      = {1991},
  publisher = {MIT Press One Rogers Street}
}

@article{bichsel1994human,
  title     = {Human face recognition and the face image Set s topology},
  author    = {Bichsel, Martin and Pentland, Alex P},
  journal   = {CVGIP: Image understanding},
  volume    = {59},
  number    = {2},
  pages     = {254--261},
  year      = {1994},
  publisher = {Elsevier}
}

@article{rowley1998neural,
  title     = {Neural network-based face detection},
  author    = {Rowley, Henry A and Baluja, Shumeet and Kanade, Takeo},
  journal   = {IEEE Transactions on pattern analysis and machine intelligence},
  volume    = {20},
  number    = {1},
  pages     = {23--38},
  year      = {1998},
  publisher = {IEEE}
}

@article{vaillant1994original,
  title     = {Original approach for the localisation of objects in images},
  author    = {Vaillant, Régis and Monrocq, Christophe and Le Cun, Yann},
  journal   = {IEE Proceedings-Vision, Image and Signal Processing},
  volume    = {141},
  number    = {4},
  pages     = {245--250},
  year      = {1994},
  publisher = {IET}
}

@article{sung1998example,
  title     = {Example-based learning for view-based human face detection},
  author    = {Sung, K-K and Poggio, Tomaso},
  journal   = {IEEE Transactions on pattern analysis and machine intelligence},
  volume    = {20},
  number    = {1},
  pages     = {39--51},
  year      = {1998},
  publisher = {IEEE}
}

@inproceedings{moghaddam1995probabilistic,
  title        = {Probabilistic visual learning for object detection},
  author       = {Moghaddam, Baback and Pentland, Alex},
  booktitle    = {Proceedings of IEEE international conference on computer vision},
  pages        = {786--793},
  year         = {1995},
  organization = {IEEE}
}

@inproceedings{papageorgiou1998general,
  title        = {A general framework for object detection},
  author       = {Papageorgiou, Constantine P and Oren, Michael and Poggio, Tomaso},
  booktitle    = {Sixth International Conference on Computer Vision (IEEE Cat. No. 98CH36271)},
  pages        = {555--562},
  year         = {1998},
  organization = {IEEE}
}

@inproceedings{ronfard2002learning,
  title        = {Learning to parse pictures of people},
  author       = {Ronfard, Remi and Schmid, Cordelia and Triggs, Bill},
  booktitle    = {Computer Vision—ECCV 2002: 7th European Conference on Computer Vision Copenhagen, Denmark, May 28--31, 2002 Proceedings, Part IV 7},
  pages        = {700--714},
  year         = {2002},
  organization = {Springer}
}

@inproceedings{mikolajczyk2004human,
  title        = {Human detection based on a probabilistic assembly of robust part detectors},
  author       = {Mikolajczyk, Krystian and Schmid, Cordelia and Zisserman, Andrew},
  booktitle    = {Computer Vision-ECCV 2004: 8th European Conference on Computer Vision, Prague, Czech Republic, May 11-14, 2004. Proceedings, Part I 8},
  pages        = {69--82},
  year         = {2004},
  organization = {Springer}
}

@article{papageorgiou2000trainable,
  title     = {A trainable system for object detection},
  author    = {Papageorgiou, Constantine and Poggio, Tomaso},
  journal   = {International journal of computer vision},
  volume    = {38},
  pages     = {15--33},
  year      = {2000},
  publisher = {Springer}
}

@inproceedings{dalal2005histograms,
  title        = {Histograms of oriented gradients for human detection},
  author       = {Dalal, Navneet and Triggs, Bill},
  booktitle    = {2005 IEEE computer society conference on computer vision and pattern recognition (CVPR'05)},
  volume       = {1},
  pages        = {886--893},
  year         = {2005},
  organization = {IEEE}
}

@inproceedings{felzenszwalb2008discriminatively,
  title        = {A discriminatively trained, multiscale, deformable part model},
  author       = {Felzenszwalb, Pedro and McAllester, David and Ramanan, Deva},
  booktitle    = {2008 IEEE conference on computer vision and pattern recognition},
  pages        = {1--8},
  year         = {2008},
  organization = {Ieee}
}

@article{girshick2011object,
  title   = {Object detection with grammar models},
  author  = {Girshick, Ross and Felzenszwalb, Pedro and McAllester, David},
  journal = {Advances in neural information processing systems},
  volume  = {24},
  year    = {2011}
}

@article{krizhevsky2017imagenet,
  title     = {Imagenet classification with deep convolutional neural networks},
  author    = {Krizhevsky, Alex and Sutskever, Ilya and Hinton, Geoffrey E},
  journal   = {Communications of the ACM},
  volume    = {60},
  number    = {6},
  pages     = {84--90},
  year      = {2017},
  publisher = {AcM New York, NY, USA}
}

@article{uijlings2013selective,
  title     = {Selective search for object recognition},
  author    = {Uijlings, Jasper RR and Van De Sande, Koen EA and Gevers, Theo and Smeulders, Arnold WM},
  journal   = {International journal of computer vision},
  volume    = {104},
  pages     = {154--171},
  year      = {2013},
  publisher = {Springer}
}

@article{he2015spatial,
  title     = {Spatial pyramid pooling in deep convolutional networks for visual recognition},
  author    = {He, Kaiming and Zhang, Xiangyu and Ren, Shaoqing and Sun, Jian},
  journal   = {IEEE transactions on pattern analysis and machine intelligence},
  volume    = {37},
  number    = {9},
  pages     = {1904--1916},
  year      = {2015},
  publisher = {IEEE}
}

@inproceedings{redmon2017yolo9000,
  title     = {YOLO9000: better, faster, stronger},
  author    = {Redmon, Joseph and Farhadi, Ali},
  booktitle = {Proceedings of the IEEE conference on computer vision and pattern recognition},
  pages     = {7263--7271},
  year      = {2017}
}

@article{redmon2018yolov3,
  title   = {Yolov3: An incremental improvement},
  author  = {Redmon, Joseph and Farhadi, Ali},
  journal = {arXiv preprint arXiv:1804.02767},
  year    = {2018}
}

@article{bochkovskiy2020yolov4,
  title   = {Yolov4: Optimal speed and accuracy of object detection},
  author  = {Bochkovskiy, Alexey and Wang, Chien-Yao and Liao, Hong-Yuan Mark},
  journal = {arXiv preprint arXiv:2004.10934},
  year    = {2020}
}

@article{yolox2021,
  title   = {YOLOX: Exceeding YOLO Series in 2021},
  author  = {Ge, Zheng and Liu, Songtao and Wang, Feng and Li, Zeming and Sun, Jian},
  journal = {arXiv preprint arXiv:2107.08430},
  year    = {2021}
}

@article{wang2021you,
  title   = {You only learn one representation: Unified network for multiple tasks},
  author  = {Wang, Chien-Yao and Yeh, I-Hau and Liao, Hong-Yuan Mark},
  journal = {arXiv preprint arXiv:2105.04206},
  year    = {2021}
}

@article{long2020pp,
  title   = {PP-YOLO: An effective and efficient implementation of object detector},
  author  = {Long, Xiang and Deng, Kaipeng and Wang, Guanzhong and Zhang, Yang and Dang, Qingqing and Gao, Yuan and Shen, Hui and Ren, Jianguo and Han, Shumin and Ding, Errui and others},
  journal = {arXiv preprint arXiv:2007.12099},
  year    = {2020}
}

@article{li2022yolov6,
  title   = {YOLOv6: A single-stage object detection framework for industrial applications},
  author  = {Li, Chuyi and Li, Lulu and Jiang, Hongliang and Weng, Kaiheng and Geng, Yifei and Li, Liang and Ke, Zaidan and Li, Qingyuan and Cheng, Meng and Nie, Weiqiang and others},
  journal = {arXiv preprint arXiv:2209.02976},
  year    = {2022}
}

@article{wang2022yolov7,
  title   = {YOLOv7: Trainable bag-of-freebies sets new state-of-the-art for real-time object detectors},
  author  = {Wang, Chien-Yao and Bochkovskiy, Alexey and Liao, Hong-Yuan Mark},
  journal = {arXiv preprint arXiv:2207.02696},
  year    = {2022}
}

@inproceedings{liu2016ssd,
  title        = {Ssd: Single shot multibox detector},
  author       = {Liu, Wei and Anguelov, Dragomir and Erhan, Dumitru and Szegedy, Christian and Reed, Scott and Fu, Cheng-Yang and Berg, Alexander C},
  booktitle    = {Computer Vision--ECCV 2016: 14th European Conference, Amsterdam, The Netherlands, October 11--14, 2016, Proceedings, Part I 14},
  pages        = {21--37},
  year         = {2016},
  organization = {Springer}
}

@article{simonyan2014very,
  title   = {Very deep convolutional networks for large-scale image recognition},
  author  = {Simonyan, Karen and Zisserman, Andrew},
  journal = {arXiv preprint arXiv:1409.1556},
  year    = {2014}
}

@inproceedings{xie2017aggregated,
  title     = {Aggregated residual transformations for deep neural networks},
  author    = {Xie, Saining and Girshick, Ross and Doll{\'a}r, Piotr and Tu, Zhuowen and He, Kaiming},
  booktitle = {Proceedings of the IEEE conference on computer vision and pattern recognition},
  pages     = {1492--1500},
  year      = {2017}
}

@article{zagoruyko2016wide,
  title   = {Wide residual networks},
  author  = {Zagoruyko, Sergey and Komodakis, Nikos},
  journal = {arXiv preprint arXiv:1605.07146},
  year    = {2016}
}

@inproceedings{tan2019efficientnet,
  title        = {Efficientnet: Rethinking model scaling for convolutional neural networks},
  author       = {Tan, Mingxing and Le, Quoc},
  booktitle    = {International conference on machine learning},
  pages        = {6105--6114},
  year         = {2019},
  organization = {PMLR}
}

@inproceedings{ghiasi2019fpn,
  title     = {Nas-fpn: Learning scalable feature pyramid architecture for object detection},
  author    = {Ghiasi, Golnaz and Lin, Tsung-Yi and Le, Quoc V},
  booktitle = {Proceedings of the IEEE/CVF conference on computer vision and pattern recognition},
  pages     = {7036--7045},
  year      = {2019}
}

@article{fu2017dssd,
  title   = {Dssd: Deconvolutional single shot detector},
  author  = {Fu, Cheng-Yang and Liu, Wei and Ranga, Ananth and Tyagi, Ambrish and Berg, Alexander C},
  journal = {arXiv preprint arXiv:1701.06659},
  year    = {2017}
}

@article{xu2020gliding,
  title     = {Gliding vertex on the horizontal bounding box for multi-oriented object detection},
  author    = {Xu, Yongchao and Fu, Mingtao and Wang, Qimeng and Wang, Yukang and Chen, Kai and Xia, Gui-Song and Bai, Xiang},
  journal   = {IEEE transactions on pattern analysis and machine intelligence},
  volume    = {43},
  number    = {4},
  pages     = {1452--1459},
  year      = {2020},
  publisher = {IEEE}
}

@article{ding2018learning,
  title   = {Learning RoI transformer for detecting oriented objects in aerial images},
  author  = {Ding, Jian and Xue, Nan and Long, Yang and Xia, Gui-Song and Lu, Qikai},
  journal = {arXiv preprint arXiv:1812.00155},
  year    = {2018}
}

@inproceedings{xie2021oriented,
  title     = {Oriented R-CNN for object detection},
  author    = {Xie, Xingxing and Cheng, Gong and Wang, Jiabao and Yao, Xiwen and Han, Junwei},
  booktitle = {Proceedings of the IEEE/CVF International Conference on Computer Vision},
  pages     = {3520--3529},
  year      = {2021}
}

@article{liu2017learning,
  title   = {Learning a rotation invariant detector with rotatable bounding box},
  author  = {Liu, Lei and Pan, Zongxu and Lei, Bin},
  journal = {arXiv preprint arXiv:1711.09405},
  year    = {2017}
}

@article{singh2021scale,
  title     = {Scale normalized image pyramids with AutoFocus for object detection},
  author    = {Singh, Bharat and Najibi, Mahyar and Sharma, Abhishek and Davis, Larry S},
  journal   = {IEEE Transactions on Pattern Analysis and Machine Intelligence},
  volume    = {44},
  number    = {7},
  pages     = {3749--3766},
  year      = {2021},
  publisher = {IEEE}
}

@article{singh2018sniper,
  title   = {Sniper: Efficient multi-scale training},
  author  = {Singh, Bharat and Najibi, Mahyar and Davis, Larry S},
  journal = {Advances in neural information processing systems},
  volume  = {31},
  year    = {2018}
}

@article{kisantal2019augmentation,
  title   = {Augmentation for small object detection},
  author  = {Kisantal, Mate and Wojna, Zbigniew and Murawski, Jakub and Naruniec, Jacek and Cho, Kyunghyun},
  journal = {arXiv preprint arXiv:1902.07296},
  year    = {2019}
}

@article{rabbi2020small,
  title     = {Small-object detection in remote sensing images with end-to-end edge-enhanced GAN and object detector network},
  author    = {Rabbi, Jakaria and Ray, Nilanjan and Schubert, Matthias and Chowdhury, Subir and Chao, Dennis},
  journal   = {Remote Sensing},
  volume    = {12},
  number    = {9},
  pages     = {1432},
  year      = {2020},
  publisher = {MDPI}
}

@inproceedings{Bai2018SODMTGANSO,
  title     = {SOD-MTGAN: Small Object Detection via Multi-Task Generative Adversarial Network},
  author    = {Yancheng Bai and Yongqiang Zhang and Mingli Ding and Bernard Ghanem},
  booktitle = {ECCV},
  year      = {2018}
}

@inproceedings{Ferdous2019SuperRD,
  title     = {Super resolution-assisted deep aerial vehicle detection},
  author    = {Syeda Nyma Ferdous and Moktari Mostofa and Nasser M. Nasrabadi},
  booktitle = {Defense + Commercial Sensing},
  year      = {2019}
}

@article{wang2021nwd,
  title   = {Detecting Tiny Objects in Aerial Images: A Normalized Wasserstein Distance and A New Benchmark},
  author  = {Xu, Chang and Wang, Jinwang and and Yang, Wen and Yu, Huai and Yu, Lei and Xia, Gui-Song},
  journal = {ISPRS Journal of Photogrammetry and Remote Sensing (ISPRS J P \& RS)},
  year    = {2022}
}

@article{deng2022extended,
  author  = {Deng, Chunfang and Wang, Mengmeng and Liu, Liang and Liu, Yong and Jiang, Yunliang},
  journal = {IEEE Transactions on Multimedia},
  title   = {Extended Feature Pyramid Network for Small Object Detection},
  year    = {2022},
  volume  = {24},
  number  = {},
  pages   = {1968-1979},
  doi     = {10.1109/TMM.2021.3074273}
}

@article{li2021gsdet,
  author  = {Li, Wei and Wei, Wei and Zhang, Lei},
  journal = {IEEE Transactions on Image Processing},
  title   = {GSDet: Object Detection in Aerial Images Based on Scale Reasoning},
  year    = {2021},
  volume  = {30},
  number  = {},
  pages   = {4599-4609},
  doi     = {10.1109/TIP.2021.3073319}
}

@article{dosovitskiy2020vit,
  title   = {An Image is Worth 16x16 Words: Transformers for Image Recognition at Scale},
  author  = {Dosovitskiy, Alexey and Beyer, Lucas and Kolesnikov, Alexander and Weissenborn, Dirk and Zhai, Xiaohua and Unterthiner, Thomas and  Dehghani, Mostafa and Minderer, Matthias and Heigold, Georg and Gelly, Sylvain and Uszkoreit, Jakob and Houlsby, Neil},
  journal = {ICLR},
  year    = {2021}
}

@inproceedings{parmar2018image,
  title        = {Image transformer},
  author       = {Parmar, Niki and Vaswani, Ashish and Uszkoreit, Jakob and Kaiser, Lukasz and Shazeer, Noam and Ku, Alexander and Tran, Dustin},
  booktitle    = {International conference on machine learning},
  pages        = {4055--4064},
  year         = {2018},
  organization = {PMLR}
}

@inproceedings{zhu2021deformable,
  title     = {Deformable {\{}DETR{\}}: Deformable Transformers for End-to-End Object Detection},
  author    = {Xizhou Zhu and Weijie Su and Lewei Lu and Bin Li and Xiaogang Wang and Jifeng Dai},
  booktitle = {International Conference on Learning Representations},
  year      = {2021},
  url       = {https://openreview.net/forum?id=gZ9hCDWe6ke}
}

@inproceedings{dai2017deformable,
  title     = {Deformable convolutional networks},
  author    = {Dai, Jifeng and Qi, Haozhi and Xiong, Yuwen and Li, Yi and Zhang, Guodong and Hu, Han and Wei, Yichen},
  booktitle = {Proceedings of the IEEE international conference on computer vision},
  pages     = {764--773},
  year      = {2017}
}

@article{jia2022detrs,
  title   = {DETRs with Hybrid Matching},
  author  = {Jia, Ding and Yuan, Yuhui and He, Haodi and Wu, Xiaopei and Yu, Haojun and Lin, Weihong and Sun, Lei and Zhang, Chao and Hu, Han},
  journal = {arXiv preprint arXiv:2207.13080},
  year    = {2022}
}

@inproceedings{touvron2021training,
  title        = {Training data-efficient image transformers \& distillation through attention},
  author       = {Touvron, Hugo and Cord, Matthieu and Douze, Matthijs and Massa, Francisco and Sablayrolles, Alexandre and J{\'e}gou, Herv{\'e}},
  booktitle    = {International conference on machine learning},
  pages        = {10347--10357},
  year         = {2021},
  organization = {PMLR}
}

@inproceedings{d2021convit,
  title        = {Convit: Improving vision transformers with soft convolutional inductive biases},
  author       = {d’Ascoli, St{\'e}phane and Touvron, Hugo and Leavitt, Matthew L and Morcos, Ari S and Biroli, Giulio and Sagun, Levent},
  booktitle    = {International Conference on Machine Learning},
  pages        = {2286--2296},
  year         = {2021},
  organization = {PMLR}
}

@inproceedings{bao2022beit,
  title     = {{BE}iT: {BERT} Pre-Training of Image Transformers},
  author    = {Hangbo Bao and Li Dong and Songhao Piao and Furu Wei},
  booktitle = {International Conference on Learning Representations},
  year      = {2022},
  url       = {https://openreview.net/forum?id=p-BhZSz59o4}
}

@inproceedings{liu2021swin,
  title     = {Swin transformer: Hierarchical vision transformer using shifted windows},
  author    = {Liu, Ze and Lin, Yutong and Cao, Yue and Hu, Han and Wei, Yixuan and Zhang, Zheng and Lin, Stephen and Guo, Baining},
  booktitle = {Proceedings of the IEEE/CVF international conference on computer vision},
  pages     = {10012--10022},
  year      = {2021}
}

@inproceedings{liu2022swin,
  title     = {Swin transformer v2: Scaling up capacity and resolution},
  author    = {Liu, Ze and Hu, Han and Lin, Yutong and Yao, Zhuliang and Xie, Zhenda and Wei, Yixuan and Ning, Jia and Cao, Yue and Zhang, Zheng and Dong, Li and others},
  booktitle = {Proceedings of the IEEE/CVF conference on computer vision and pattern recognition},
  pages     = {12009--12019},
  year      = {2022}
}

@inproceedings{caron2021emerging,
  title     = {Emerging properties in self-supervised vision transformers},
  author    = {Caron, Mathilde and Touvron, Hugo and Misra, Ishan and J{\'e}gou, Herv{\'e} and Mairal, Julien and Bojanowski, Piotr and Joulin, Armand},
  booktitle = {Proceedings of the IEEE/CVF international conference on computer vision},
  pages     = {9650--9660},
  year      = {2021}
}

@misc{oquab2023dinov2,
  title         = {DINOv2: Learning Robust Visual Features without Supervision},
  author        = {Maxime Oquab and Timothée Darcet and Théo Moutakanni and Huy Vo and Marc Szafraniec and Vasil Khalidov and Pierre Fernandez and Daniel Haziza and Francisco Massa and Alaaeldin El-Nouby and Mahmoud Assran and Nicolas Ballas and Wojciech Galuba and Russell Howes and Po-Yao Huang and Shang-Wen Li and Ishan Misra and Michael Rabbat and Vasu Sharma and Gabriel Synnaeve and Hu Xu and Hervé Jegou and Julien Mairal and Patrick Labatut and Armand Joulin and Piotr Bojanowski},
  year          = {2023},
  eprint        = {2304.07193},
  archiveprefix = {arXiv},
  primaryclass  = {cs.CV}
}

@misc{kirillov2023segment,
  title         = {Segment Anything},
  author        = {Alexander Kirillov and Eric Mintun and Nikhila Ravi and Hanzi Mao and Chloe Rolland and Laura Gustafson and Tete Xiao and Spencer Whitehead and Alexander C. Berg and Wan-Yen Lo and Piotr Dollár and Ross Girshick},
  year          = {2023},
  eprint        = {2304.02643},
  archiveprefix = {arXiv},
  primaryclass  = {cs.CV}
}

@inproceedings{liu2022convnet,
  title     = {A convnet for the 2020s},
  author    = {Liu, Zhuang and Mao, Hanzi and Wu, Chao-Yuan and Feichtenhofer, Christoph and Darrell, Trevor and Xie, Saining},
  booktitle = {Proceedings of the IEEE/CVF Conference on Computer Vision and Pattern Recognition},
  pages     = {11976--11986},
  year      = {2022}
}

@article{wang2022internimage,
  title   = {Internimage: Exploring large-scale vision foundation models with deformable convolutions},
  author  = {Wang, Wenhai and Dai, Jifeng and Chen, Zhe and Huang, Zhenhang and Li, Zhiqi and Zhu, Xizhou and Hu, Xiaowei and Lu, Tong and Lu, Lewei and Li, Hongsheng and others},
  journal = {arXiv preprint arXiv:2211.05778},
  year    = {2022}
}

@article{chen2022diffusiondet,
  title   = {Diffusiondet: Diffusion model for object detection},
  author  = {Chen, Shoufa and Sun, Peize and Song, Yibing and Luo, Ping},
  journal = {arXiv preprint arXiv:2211.09788},
  year    = {2022}
}

@inproceedings{yu2016unitbox,
  title     = {Unitbox: An advanced object detection network},
  author    = {Yu, Jiahui and Jiang, Yuning and Wang, Zhangyang and Cao, Zhimin and Huang, Thomas},
  booktitle = {Proceedings of the 24th ACM international conference on Multimedia},
  pages     = {516--520},
  year      = {2016}
}

@inproceedings{rezatofighi2019generalized,
  title     = {Generalized intersection over union: A metric and a loss for bounding box regression},
  author    = {Rezatofighi, Hamid and Tsoi, Nathan and Gwak, JunYoung and Sadeghian, Amir and Reid, Ian and Savarese, Silvio},
  booktitle = {Proceedings of the IEEE/CVF conference on computer vision and pattern recognition},
  pages     = {658--666},
  year      = {2019}
}

@inproceedings{zheng2020distance,
  title     = {Distance-IoU loss: Faster and better learning for bounding box regression},
  author    = {Zheng, Zhaohui and Wang, Ping and Liu, Wei and Li, Jinze and Ye, Rongguang and Ren, Dongwei},
  booktitle = {Proceedings of the AAAI conference on artificial intelligence},
  volume    = {34},
  number    = {07},
  pages     = {12993--13000},
  year      = {2020}
}

@article{he2021alpha,
  title   = {Alpha-IoU: A Family of Power Intersection over Union Losses for Bounding Box Regression},
  author  = {He, Jiabo and Erfani, Sarah and Ma, Xingjun and Bailey, James and Chi, Ying and Hua, Xian-Sheng},
  journal = {Advances in Neural Information Processing Systems},
  volume  = {34},
  pages   = {20230--20242},
  year    = {2021}
}

@article{huber1992robust,
  title     = {Robust estimation of a location parameter},
  author    = {Huber, Peter J},
  journal   = {Breakthroughs in statistics: Methodology and distribution},
  pages     = {492--518},
  year      = {1992},
  publisher = {Springer}
}

@inproceedings{gupta2019lvis,
  title     = {{LVIS}: A Dataset for Large Vocabulary Instance Segmentation},
  author    = {Gupta, Agrim and Dollar, Piotr and Girshick, Ross},
  booktitle = {Proceedings of the {IEEE} Conference on Computer Vision and Pattern Recognition},
  year      = {2019}
}

@inproceedings{Shao_2019_ICCV,
  author    = {Shao, Shuai and Li, Zeming and Zhang, Tianyuan and Peng, Chao and Yu, Gang and Zhang, Xiangyu and Li, Jing and Sun, Jian},
  title     = {Objects365: A Large-Scale, High-Quality Dataset for Object Detection},
  booktitle = {Proceedings of the IEEE/CVF International Conference on Computer Vision (ICCV)},
  month     = {October},
  year      = {2019}
}

@inproceedings{Geiger2012CVPR,
  author    = {Andreas Geiger and Philip Lenz and Raquel Urtasun},
  title     = {Are we ready for Autonomous Driving? The KITTI Vision Benchmark Suite},
  booktitle = {Conference on Computer Vision and Pattern Recognition (CVPR)},
  year      = {2012}
}

@inproceedings{yu2020bdd100k,
  title     = {Bdd100k: A diverse driving dataset for heterogeneous multitask learning},
  author    = {Yu, Fisher and Chen, Haofeng and Wang, Xin and Xian, Wenqi and Chen, Yingying and Liu, Fangchen and Madhavan, Vashisht and Darrell, Trevor},
  booktitle = {Proceedings of the IEEE/CVF conference on computer vision and pattern recognition},
  pages     = {2636--2645},
  year      = {2020}
}

@article{sun2022fair1m,
  title     = {FAIR1M: A benchmark dataset for fine-grained object recognition in high-resolution remote sensing imagery},
  author    = {Sun, Xian and Wang, Peijin and Yan, Zhiyuan and Xu, Feng and Wang, Ruiping and Diao, Wenhui and Chen, Jin and Li, Jihao and Feng, Yingchao and Xu, Tao and others},
  journal   = {ISPRS Journal of Photogrammetry and Remote Sensing},
  volume    = {184},
  pages     = {116--130},
  year      = {2022},
  publisher = {Elsevier}
}

@article{lam2018xview,
  title   = {xview: Objects in context in overhead imagery},
  author  = {Lam, Darius and Kuzma, Richard and McGee, Kevin and Dooley, Samuel and Laielli, Michael and Klaric, Matthew and Bulatov, Yaroslav and McCord, Brendan},
  journal = {arXiv preprint arXiv:1802.07856},
  year    = {2018}
}

@inproceedings{Shanshan2017CVPR,
  author    = {Shanshan Zhang and Rodrigo Benenson and Bernt Schiele},
  title     = {CityPersons: A Diverse Dataset for Pedestrian Detection},
  booktitle = {CVPR},
  year      = {2017}
}

@inproceedings{yu2020scale,
  title     = {Scale match for tiny person detection},
  author    = {Yu, Xuehui and Gong, Yuqi and Jiang, Nan and Ye, Qixiang and Han, Zhenjun},
  booktitle = {Proceedings of the IEEE/CVF winter conference on applications of computer vision},
  pages     = {1257--1265},
  year      = {2020}
}

@article{shao2018crowdhuman,
  title   = {CrowdHuman: A Benchmark for Detecting Human in a Crowd},
  author  = {Shao, Shuai and Zhao, Zijian and Li, Boxun and Xiao, Tete and Yu, Gang and Zhang, Xiangyu and Sun, Jian},
  journal = {arXiv preprint arXiv:1805.00123},
  year    = {2018}
}

@inproceedings{mundhenk2016large,
  title        = {A large contextual dataset for classification, detection and counting of cars with deep learning},
  author       = {Mundhenk, T Nathan and Konjevod, Goran and Sakla, Wesam A and Boakye, Kofi},
  booktitle    = {Computer Vision--ECCV 2016: 14th European Conference, Amsterdam, The Netherlands, October 11-14, 2016, Proceedings, Part III 14},
  pages        = {785--800},
  year         = {2016},
  organization = {Springer}
}

@misc{tum2019oktoberfest,
  title         = {Oktoberfest Food Dataset},
  author        = {Alexander Ziller and Julius Hansjakob and Vitalii Rusinov and Daniel Z\"ugner and Peter Vogel and Stephan G\"unnemann},
  year          = {2019},
  eprint        = {1912.05007},
  archiveprefix = {arXiv},
  primaryclass  = {cs.CV}
}

@inproceedings{inoue2018cross,
  title     = {Cross-domain weakly-supervised object detection through progressive domain adaptation},
  author    = {Inoue, Naoto and Furuta, Ryosuke and Yamasaki, Toshihiko and Aizawa, Kiyoharu},
  booktitle = {Proceedings of the IEEE conference on computer vision and pattern recognition},
  pages     = {5001--5009},
  year      = {2018}
}

@article{wang2022logodet,
  title     = {Logodet-3k: A large-scale image dataset for logo detection},
  author    = {Wang, Jing and Min, Weiqing and Hou, Sujuan and Ma, Shengnan and Zheng, Yuanjie and Jiang, Shuqiang},
  journal   = {ACM Transactions on Multimedia Computing, Communications, and Applications (TOMM)},
  volume    = {18},
  number    = {1},
  pages     = {1--19},
  year      = {2022},
  publisher = {ACM New York, NY}
}

@inproceedings{miao2019sixray,
  title     = {Sixray: A large-scale security inspection x-ray benchmark for prohibited item discovery in overlapping images},
  author    = {Miao, Caijing and Xie, Lingxi and Wan, Fang and Su, Chi and Liu, Hongye and Jiao, Jianbin and Ye, Qixiang},
  booktitle = {Proceedings of the IEEE/CVF conference on computer vision and pattern recognition},
  pages     = {2119--2128},
  year      = {2019}
}

@article{sun2022drone,
  author  = {Sun, Yiming and Cao, Bing and Zhu, Pengfei and Hu, Qinghua},
  journal = {IEEE Transactions on Circuits and Systems for Video Technology},
  title   = {Drone-Based RGB-Infrared Cross-Modality Vehicle Detection Via Uncertainty-Aware Learning},
  year    = {2022},
  volume  = {32},
  number  = {10},
  pages   = {6700-6713},
  doi     = {10.1109/TCSVT.2022.3168279}
}

@article{lecun1998gradient,
  title     = {Gradient-based learning applied to document recognition},
  author    = {LeCun, Yann and Bottou, L{\'e}on and Bengio, Yoshua and Haffner, Patrick},
  journal   = {Proceedings of the IEEE},
  volume    = {86},
  number    = {11},
  pages     = {2278--2324},
  year      = {1998},
  publisher = {Ieee}
}

@article{rawat2017deep,
  title     = {Deep convolutional neural networks for image classification: A comprehensive review},
  author    = {Rawat, Waseem and Wang, Zenghui},
  journal   = {Neural computation},
  volume    = {29},
  number    = {9},
  pages     = {2352--2449},
  year      = {2017},
  publisher = {MIT Press}
}

@article{grill2020bootstrap,
  title   = {Bootstrap your own latent-a new approach to self-supervised learning},
  author  = {Grill, Jean-Bastien and Strub, Florian and Altch{\'e}, Florent and Tallec, Corentin and Richemond, Pierre and Buchatskaya, Elena and Doersch, Carl and Avila Pires, Bernardo and Guo, Zhaohan and Gheshlaghi Azar, Mohammad and others},
  journal = {Advances in neural information processing systems},
  volume  = {33},
  pages   = {21271--21284},
  year    = {2020}
}

@inproceedings{he2020momentum,
  title     = {Momentum contrast for unsupervised visual representation learning},
  author    = {He, Kaiming and Fan, Haoqi and Wu, Yuxin and Xie, Saining and Girshick, Ross},
  booktitle = {Proceedings of the IEEE/CVF conference on computer vision and pattern recognition},
  pages     = {9729--9738},
  year      = {2020}
}

@article{caron2020unsupervised,
  title   = {Unsupervised learning of visual features by contrasting cluster assignments},
  author  = {Caron, Mathilde and Misra, Ishan and Mairal, Julien and Goyal, Priya and Bojanowski, Piotr and Joulin, Armand},
  journal = {Advances in neural information processing systems},
  volume  = {33},
  pages   = {9912--9924},
  year    = {2020}
}

@inproceedings{radford2021learning,
  title        = {Learning transferable visual models from natural language supervision},
  author       = {Radford, Alec and Kim, Jong Wook and Hallacy, Chris and Ramesh, Aditya and Goh, Gabriel and Agarwal, Sandhini and Sastry, Girish and Askell, Amanda and Mishkin, Pamela and Clark, Jack and others},
  booktitle    = {International conference on machine learning},
  pages        = {8748--8763},
  year         = {2021},
  organization = {PMLR}
}

@article{song2022comprehensive,
  title   = {A comprehensive survey of few-shot learning: Evolution, applications, challenges, and opportunities},
  author  = {Song, Yisheng and Wang, Ting and Mondal, Subrota K and Sahoo, Jyoti Prakash},
  journal = {arXiv preprint arXiv:2205.06743},
  year    = {2022}
}

@inproceedings{yoo2018efficient,
  title     = {Efficient k-shot learning with regularized deep networks},
  author    = {Yoo, Donghyun and Fan, Haoqi and Boddeti, Vishnu and Kitani, Kris},
  booktitle = {Proceedings of the AAAI Conference on Artificial Intelligence},
  volume    = {32},
  number    = {1},
  year      = {2018}
}

@article{hazan2010direct,
  title   = {Direct loss minimization for structured prediction},
  author  = {Hazan, Tamir and Keshet, Joseph and McAllester, David},
  journal = {Advances in neural information processing systems},
  volume  = {23},
  year    = {2010}
}

@article{triantafillou2017few,
  title   = {Few-shot learning through an information retrieval lens},
  author  = {Triantafillou, Eleni and Zemel, Richard and Urtasun, Raquel},
  journal = {Advances in neural information processing systems},
  volume  = {30},
  year    = {2017}
}

@inproceedings{munkhdalai2018rapid,
  title        = {Rapid adaptation with conditionally shifted neurons},
  author       = {Munkhdalai, Tsendsuren and Yuan, Xingdi and Mehri, Soroush and Trischler, Adam},
  booktitle    = {International conference on machine learning},
  pages        = {3664--3673},
  year         = {2018},
  organization = {PMLR}
}

@inproceedings{chopra2005learning,
  title        = {Learning a similarity metric discriminatively, with application to face verification},
  author       = {Chopra, Sumit and Hadsell, Raia and LeCun, Yann},
  booktitle    = {2005 IEEE Computer Society Conference on Computer Vision and Pattern Recognition (CVPR'05)},
  volume       = {1},
  pages        = {539--546},
  year         = {2005},
  organization = {IEEE}
}

@inproceedings{koch2015siamese,
  title        = {Siamese neural networks for one-shot image recognition},
  author       = {Koch, Gregory and Zemel, Richard and Salakhutdinov, Ruslan and others},
  booktitle    = {ICML deep learning workshop},
  volume       = {2},
  number       = {1},
  year         = {2015},
  organization = {Lille}
}

@inproceedings{liu2020prototype,
  title        = {Prototype rectification for few-shot learning},
  author       = {Liu, Jinlu and Song, Liang and Qin, Yongqiang},
  booktitle    = {Computer Vision--ECCV 2020: 16th European Conference, Glasgow, UK, August 23--28, 2020, Proceedings, Part I 16},
  pages        = {741--756},
  year         = {2020},
  organization = {Springer}
}

@article{ren2018meta,
  title   = {Meta-learning for semi-supervised few-shot classification},
  author  = {Ren, Mengye and Triantafillou, Eleni and Ravi, Sachin and Snell, Jake and Swersky, Kevin and Tenenbaum, Joshua B and Larochelle, Hugo and Zemel, Richard S},
  journal = {arXiv preprint arXiv:1803.00676},
  year    = {2018}
}

@inproceedings{deuschel2021multi,
  title     = {Multi-prototype few-shot learning in histopathology},
  author    = {Deuschel, Jessica and Firmbach, Daniel and Geppert, Carol I and Eckstein, Markus and Hartmann, Arndt and Bruns, Volker and Kuritcyn, Petr and Dexl, Jakob and Hartmann, David and Perrin, Dominik and others},
  booktitle = {Proceedings of the IEEE/CVF International Conference on Computer Vision},
  pages     = {620--628},
  year      = {2021}
}

@article{oreshkin2018tadam,
  title   = {Tadam: Task dependent adaptive metric for improved few-shot learning},
  author  = {Oreshkin, Boris and Rodr{\'\i}guez L{\'o}pez, Pau and Lacoste, Alexandre},
  journal = {Advances in neural information processing systems},
  volume  = {31},
  year    = {2018}
}

@inproceedings{lee2018gradient,
  title        = {Gradient-based meta-learning with learned layerwise metric and subspace},
  author       = {Lee, Yoonho and Choi, Seungjin},
  booktitle    = {International Conference on Machine Learning},
  pages        = {2927--2936},
  year         = {2018},
  organization = {PMLR}
}

@inproceedings{franceschi2018bilevel,
  title        = {Bilevel programming for hyperparameter optimization and meta-learning},
  author       = {Franceschi, Luca and Frasconi, Paolo and Salzo, Saverio and Grazzi, Riccardo and Pontil, Massimiliano},
  booktitle    = {International Conference on Machine Learning},
  pages        = {1568--1577},
  year         = {2018},
  organization = {PMLR}
}

@article{kim2018bayesian,
  title   = {Bayesian model-agnostic meta-learning},
  author  = {Kim, Taesup and Yoon, Jaesik and Dia, Ousmane and Kim, Sungwoong and Bengio, Yoshua and Ahn, Sungjin},
  journal = {arXiv preprint arXiv:1806.03836},
  year    = {2018}
}

@article{behl2019alpha,
  title   = {Alpha maml: Adaptive model-agnostic meta-learning},
  author  = {Behl, Harkirat Singh and Baydin, At{\i}l{\i}m G{\"u}ne{\c{s}} and Torr, Philip HS},
  journal = {arXiv preprint arXiv:1905.07435},
  year    = {2019}
}

@article{fallah2021generalization,
  title   = {Generalization of model-agnostic meta-learning algorithms: Recurring and unseen tasks},
  author  = {Fallah, Alireza and Mokhtari, Aryan and Ozdaglar, Asuman},
  journal = {Advances in Neural Information Processing Systems},
  volume  = {34},
  pages   = {5469--5480},
  year    = {2021}
}

@phdthesis{schmidhuber1987evolutionary,
  title  = {Evolutionary principles in self-referential learning, or on learning how to learn: the meta-meta-... hook},
  author = {Schmidhuber, J{\"u}rgen},
  year   = {1987},
  school = {Technische Universit{\"a}t M{\"u}nchen}
}

@article{bertinetto2018meta,
  title   = {Meta-learning with differentiable closed-form solvers},
  author  = {Bertinetto, Luca and Henriques, Joao F and Torr, Philip HS and Vedaldi, Andrea},
  journal = {arXiv preprint arXiv:1805.08136},
  year    = {2018}
}

@article{rusu2018meta,
  title   = {Meta-learning with latent embedding optimization},
  author  = {Rusu, Andrei A and Rao, Dushyant and Sygnowski, Jakub and Vinyals, Oriol and Pascanu, Razvan and Osindero, Simon and Hadsell, Raia},
  journal = {arXiv preprint arXiv:1807.05960},
  year    = {2018}
}

@article{thrun1998learning,
  title     = {Learning to learn: Introduction and overview},
  author    = {Thrun, Sebastian and Pratt, Lorien},
  journal   = {Learning to learn},
  pages     = {3--17},
  year      = {1998},
  publisher = {Springer}
}

@article{mishra2017simple,
  title   = {A simple neural attentive meta-learner},
  author  = {Mishra, Nikhil and Rohaninejad, Mostafa and Chen, Xi and Abbeel, Pieter},
  journal = {arXiv preprint arXiv:1707.03141},
  year    = {2017}
}

@inproceedings{zhao2018dynamic,
  title     = {Dynamic conditional networks for few-shot learning},
  author    = {Zhao, Fang and Zhao, Jian and Yan, Shuicheng and Feng, Jiashi},
  booktitle = {Proceedings of the European conference on computer vision (ECCV)},
  pages     = {19--35},
  year      = {2018}
}

@inproceedings{santoro2016meta,
  title        = {Meta-learning with memory-augmented neural networks},
  author       = {Santoro, Adam and Bartunov, Sergey and Botvinick, Matthew and Wierstra, Daan and Lillicrap, Timothy},
  booktitle    = {International conference on machine learning},
  pages        = {1842--1850},
  year         = {2016},
  organization = {PMLR}
}

@article{ramalho2019adaptive,
  title   = {Adaptive posterior learning: few-shot learning with a surprise-based memory module},
  author  = {Ramalho, Tiago and Garnelo, Marta},
  journal = {arXiv preprint arXiv:1902.02527},
  year    = {2019}
}

@inproceedings{munkhdalai2017meta,
  title        = {Meta networks},
  author       = {Munkhdalai, Tsendsuren and Yu, Hong},
  booktitle    = {International conference on machine learning},
  pages        = {2554--2563},
  year         = {2017},
  organization = {PMLR}
}

@inproceedings{xu2017few,
  title     = {Few-shot object recognition from machine-labeled web images},
  author    = {Xu, Zhongwen and Zhu, Linchao and Yang, Yi},
  booktitle = {Proceedings of the IEEE Conference on Computer Vision and Pattern Recognition},
  pages     = {1164--1172},
  year      = {2017}
}

@article{kaiser2017learning,
  title   = {Learning to remember rare events},
  author  = {Kaiser, {\L}ukasz and Nachum, Ofir and Roy, Aurko and Bengio, Samy},
  journal = {arXiv preprint arXiv:1703.03129},
  year    = {2017}
}

@article{vapnik1999overview,
  title     = {An overview of statistical learning theory},
  author    = {Vapnik, Vladimir N},
  journal   = {IEEE transactions on neural networks},
  volume    = {10},
  number    = {5},
  pages     = {988--999},
  year      = {1999},
  publisher = {IEEE}
}

@inproceedings{joachims1999transductive,
  title     = {Transductive inference for text classification using support vector machines},
  author    = {Joachims, Thorsten and others},
  booktitle = {Icml},
  volume    = {99},
  pages     = {200--209},
  year      = {1999}
}

@article{cortes1995support,
  title     = {Support-vector networks},
  author    = {Cortes, Corinna and Vapnik, Vladimir},
  journal   = {Machine learning},
  volume    = {20},
  pages     = {273--297},
  year      = {1995},
  publisher = {Springer}
}

@article{zhou2003learning,
  title   = {Learning with local and global consistency},
  author  = {Zhou, Dengyong and Bousquet, Olivier and Lal, Thomas and Weston, Jason and Sch{\"o}lkopf, Bernhard},
  journal = {Advances in neural information processing systems},
  volume  = {16},
  year    = {2003}
}

@article{belhasin2022transboost,
  title   = {TransBoost: Improving the Best ImageNet Performance using Deep Transduction},
  author  = {Belhasin, Omer and Bar-Shalom, Guy and El-Yaniv, Ran},
  journal = {arXiv preprint arXiv:2205.13331},
  year    = {2022}
}

@article{nichol2018first,
  title   = {On first-order meta-learning algorithms},
  author  = {Nichol, Alex and Achiam, Joshua and Schulman, John},
  journal = {arXiv preprint arXiv:1803.02999},
  year    = {2018}
}

@article{dhillon2019baseline,
  title   = {A baseline for few-shot image classification},
  author  = {Dhillon, Guneet S and Chaudhari, Pratik and Ravichandran, Avinash and Soatto, Stefano},
  journal = {arXiv preprint arXiv:1909.02729},
  year    = {2019}
}

@inproceedings{qiao2019transductive,
  title     = {Transductive episodic-wise adaptive metric for few-shot learning},
  author    = {Qiao, Limeng and Shi, Yemin and Li, Jia and Wang, Yaowei and Huang, Tiejun and Tian, Yonghong},
  booktitle = {Proceedings of the IEEE/CVF international conference on computer vision},
  pages     = {3603--3612},
  year      = {2019}
}

@article{boudiaf2020information,
  title   = {Information maximization for few-shot learning},
  author  = {Boudiaf, Malik and Ziko, Imtiaz and Rony, J{\'e}r{\^o}me and Dolz, Jos{\'e} and Piantanida, Pablo and Ben Ayed, Ismail},
  journal = {Advances in Neural Information Processing Systems},
  volume  = {33},
  pages   = {2445--2457},
  year    = {2020}
}

@inproceedings{ziko2020laplacian,
  title        = {Laplacian regularized few-shot learning},
  author       = {Ziko, Imtiaz and Dolz, Jose and Granger, Eric and Ayed, Ismail Ben},
  booktitle    = {International conference on machine learning},
  pages        = {11660--11670},
  year         = {2020},
  organization = {PMLR}
}

@article{hou2019cross,
  title   = {Cross attention network for few-shot classification},
  author  = {Hou, Ruibing and Chang, Hong and Ma, Bingpeng and Shan, Shiguang and Chen, Xilin},
  journal = {Advances in Neural Information Processing Systems},
  volume  = {32},
  year    = {2019}
}

@inproceedings{kim2019edge,
  title     = {Edge-labeling graph neural network for few-shot learning},
  author    = {Kim, Jongmin and Kim, Taesup and Kim, Sungwoong and Yoo, Chang D},
  booktitle = {Proceedings of the IEEE/CVF conference on computer vision and pattern recognition},
  pages     = {11--20},
  year      = {2019}
}

@article{liu2018learning,
  title   = {Learning to propagate labels: Transductive propagation network for few-shot learning},
  author  = {Liu, Yanbin and Lee, Juho and Park, Minseop and Kim, Saehoon and Yang, Eunho and Hwang, Sung Ju and Yang, Yi},
  journal = {arXiv preprint arXiv:1805.10002},
  year    = {2018}
}

@inproceedings{lazarou2021iterative,
  title     = {Iterative label cleaning for transductive and semi-supervised few-shot learning},
  author    = {Lazarou, Michalis and Stathaki, Tania and Avrithis, Yannis},
  booktitle = {Proceedings of the ieee/cvf international conference on computer vision},
  pages     = {8751--8760},
  year      = {2021}
}

@inproceedings{rodriguez2020embedding,
  title        = {Embedding propagation: Smoother manifold for few-shot classification},
  author       = {Rodr{\'\i}guez, Pau and Laradji, Issam and Drouin, Alexandre and Lacoste, Alexandre},
  booktitle    = {Computer Vision--ECCV 2020: 16th European Conference, Glasgow, UK, August 23--28, 2020, Proceedings, Part XXVI 16},
  pages        = {121--138},
  year         = {2020},
  organization = {Springer}
}

@inproceedings{hu2021graph,
  title        = {Graph-based interpolation of feature vectors for accurate few-shot classification},
  author       = {Hu, Yuqing and Gripon, Vincent and Pateux, St{\'e}phane},
  booktitle    = {2020 25th International Conference on Pattern Recognition (ICPR)},
  pages        = {8164--8171},
  year         = {2021},
  organization = {IEEE}
}

@article{kye2020meta,
  title   = {Meta-learned confidence for few-shot learning},
  author  = {Kye, Seong Min and Lee, Hae Beom and Kim, Hoirin and Hwang, Sung Ju},
  journal = {arXiv preprint arXiv:2002.12017},
  year    = {2020}
}

@article{geng2020recent,
  title     = {Recent advances in open set recognition: A survey},
  author    = {Geng, Chuanxing and Huang, Sheng-jun and Chen, Songcan},
  journal   = {IEEE transactions on pattern analysis and machine intelligence},
  volume    = {43},
  number    = {10},
  pages     = {3614--3631},
  year      = {2020},
  publisher = {IEEE}
}

@article{boudiaf2023open,
  title   = {Open-Set Likelihood Maximization for Few-Shot Learning},
  author  = {Boudiaf, Malik and Bennequin, Etienne and Tami, Myriam and Toubhans, Antoine and Piantanida, Pablo and Hudelot, C{\'e}line and Ayed, Ismail Ben},
  journal = {arXiv preprint arXiv:2301.08390},
  year    = {2023}
}

@inproceedings{bendale2016towards,
  title     = {Towards open set deep networks},
  author    = {Bendale, Abhijit and Boult, Terrance E},
  booktitle = {Proceedings of the IEEE conference on computer vision and pattern recognition},
  pages     = {1563--1572},
  year      = {2016}
}

@inproceedings{jo2018open,
  title        = {Open set recognition by regularising classifier with fake data generated by generative adversarial networks},
  author       = {Jo, Inhyuk and Kim, Jungtaek and Kang, Hyohyeong and Kim, Yong-Deok and Choi, Seungjin},
  booktitle    = {2018 IEEE international conference on acoustics, speech and signal processing (ICASSP)},
  pages        = {2686--2690},
  year         = {2018},
  organization = {IEEE}
}

@article{martin2022towards,
  title   = {Towards Practical Few-Shot Query Sets: Transductive Minimum Description Length Inference},
  author  = {Martin, S{\'e}gol{\`e}ne and Boudiaf, Malik and Chouzenoux, Emilie and Pesquet, Jean-Christophe and Ayed, Ismail},
  journal = {Advances in Neural Information Processing Systems},
  volume  = {35},
  pages   = {34677--34688},
  year    = {2022}
}

@inproceedings{gidaris2018dynamic,
  title     = {Dynamic few-shot visual learning without forgetting},
  author    = {Gidaris, Spyros and Komodakis, Nikos},
  booktitle = {Proceedings of the IEEE conference on computer vision and pattern recognition},
  pages     = {4367--4375},
  year      = {2018}
}

@article{robins1995catastrophic,
  title     = {Catastrophic forgetting, rehearsal and pseudorehearsal},
  author    = {Robins, Anthony},
  journal   = {Connection Science},
  volume    = {7},
  number    = {2},
  pages     = {123--146},
  year      = {1995},
  publisher = {Taylor \& Francis}
}

@inproceedings{li2019learn,
  title        = {Learn to grow: A continual structure learning framework for overcoming catastrophic forgetting},
  author       = {Li, Xilai and Zhou, Yingbo and Wu, Tianfu and Socher, Richard and Xiong, Caiming},
  booktitle    = {International Conference on Machine Learning},
  pages        = {3925--3934},
  year         = {2019},
  organization = {PMLR}
}

@article{wang2018deep,
  title     = {Deep visual domain adaptation: A survey},
  author    = {Wang, Mei and Deng, Weihong},
  journal   = {Neurocomputing},
  volume    = {312},
  pages     = {135--153},
  year      = {2018},
  publisher = {Elsevier}
}

@inproceedings{tzeng2015simultaneous,
  title     = {Simultaneous deep transfer across domains and tasks},
  author    = {Tzeng, Eric and Hoffman, Judy and Darrell, Trevor and Saenko, Kate},
  booktitle = {Proceedings of the IEEE international conference on computer vision},
  pages     = {4068--4076},
  year      = {2015}
}

@inproceedings{gebru2017fine,
  title     = {Fine-grained recognition in the wild: A multi-task domain adaptation approach},
  author    = {Gebru, Timnit and Hoffman, Judy and Fei-Fei, Li},
  booktitle = {Proceedings of the IEEE international conference on computer vision},
  pages     = {1349--1358},
  year      = {2017}
}

@inproceedings{motiian2017unified,
  title     = {Unified deep supervised domain adaptation and generalization},
  author    = {Motiian, Saeid and Piccirilli, Marco and Adjeroh, Donald A and Doretto, Gianfranco},
  booktitle = {Proceedings of the IEEE international conference on computer vision},
  pages     = {5715--5725},
  year      = {2017}
}

@inproceedings{long2015learning,
  title        = {Learning transferable features with deep adaptation networks},
  author       = {Long, Mingsheng and Cao, Yue and Wang, Jianmin and Jordan, Michael},
  booktitle    = {International conference on machine learning},
  pages        = {97--105},
  year         = {2015},
  organization = {PMLR}
}

@article{tzeng2014deep,
  title   = {Deep domain confusion: Maximizing for domain invariance},
  author  = {Tzeng, Eric and Hoffman, Judy and Zhang, Ning and Saenko, Kate and Darrell, Trevor},
  journal = {arXiv preprint arXiv:1412.3474},
  year    = {2014}
}

@inproceedings{ghifary2014domain,
  title        = {Domain adaptive neural networks for object recognition},
  author       = {Ghifary, Muhammad and Kleijn, W Bastiaan and Zhang, Mengjie},
  booktitle    = {PRICAI 2014: Trends in Artificial Intelligence: 13th Pacific Rim International Conference on Artificial Intelligence, Gold Coast, QLD, Australia, December 1-5, 2014. Proceedings 13},
  pages        = {898--904},
  year         = {2014},
  organization = {Springer}
}

@inproceedings{long2017deep,
  title        = {Deep transfer learning with joint adaptation networks},
  author       = {Long, Mingsheng and Zhu, Han and Wang, Jianmin and Jordan, Michael I},
  booktitle    = {International conference on machine learning},
  pages        = {2208--2217},
  year         = {2017},
  organization = {PMLR}
}

@article{zellinger2017central,
  title   = {Central moment discrepancy (cmd) for domain-invariant representation learning},
  author  = {Zellinger, Werner and Grubinger, Thomas and Lughofer, Edwin and Natschl{\"a}ger, Thomas and Saminger-Platz, Susanne},
  journal = {arXiv preprint arXiv:1702.08811},
  year    = {2017}
}

@inproceedings{zhuang2015supervised,
  title     = {Supervised representation learning: Transfer learning with deep autoencoders},
  author    = {Zhuang, Fuzhen and Cheng, Xiaohu and Luo, Ping and Pan, Sinno Jialin and He, Qing},
  booktitle = {Twenty-fourth international joint conference on artificial intelligence},
  year      = {2015}
}

@inproceedings{sun2016return,
  title     = {Return of frustratingly easy domain adaptation},
  author    = {Sun, Baochen and Feng, Jiashi and Saenko, Kate},
  booktitle = {Proceedings of the AAAI conference on artificial intelligence},
  volume    = {30},
  number    = {1},
  year      = {2016}
}

@article{rozantsev2018beyond,
  title     = {Beyond sharing weights for deep domain adaptation},
  author    = {Rozantsev, Artem and Salzmann, Mathieu and Fua, Pascal},
  journal   = {IEEE transactions on pattern analysis and machine intelligence},
  volume    = {41},
  number    = {4},
  pages     = {801--814},
  year      = {2018},
  publisher = {IEEE}
}

@article{li2016revisiting,
  title   = {Revisiting batch normalization for practical domain adaptation},
  author  = {Li, Yanghao and Wang, Naiyan and Shi, Jianping and Liu, Jiaying and Hou, Xiaodi},
  journal = {arXiv preprint arXiv:1603.04779},
  year    = {2016}
}

@inproceedings{tzeng2017adversarial,
  title     = {Adversarial discriminative domain adaptation},
  author    = {Tzeng, Eric and Hoffman, Judy and Saenko, Kate and Darrell, Trevor},
  booktitle = {Proceedings of the IEEE conference on computer vision and pattern recognition},
  pages     = {7167--7176},
  year      = {2017}
}

@inproceedings{isola2017image,
  title     = {Image-to-image translation with conditional adversarial networks},
  author    = {Isola, Phillip and Zhu, Jun-Yan and Zhou, Tinghui and Efros, Alexei A},
  booktitle = {Proceedings of the IEEE conference on computer vision and pattern recognition},
  pages     = {1125--1134},
  year      = {2017}
}

@article{goodfellow2020generative,
  title     = {Generative adversarial networks},
  author    = {Goodfellow, Ian and Pouget-Abadie, Jean and Mirza, Mehdi and Xu, Bing and Warde-Farley, David and Ozair, Sherjil and Courville, Aaron and Bengio, Yoshua},
  journal   = {Communications of the ACM},
  volume    = {63},
  number    = {11},
  pages     = {139--144},
  year      = {2020},
  publisher = {ACM New York, NY, USA}
}

@article{liu2016coupled,
  title   = {Coupled generative adversarial networks},
  author  = {Liu, Ming-Yu and Tuzel, Oncel},
  journal = {Advances in neural information processing systems},
  volume  = {29},
  year    = {2016}
}

@inproceedings{CycleGAN2017,
  title     = {Unpaired Image-to-Image Translation using Cycle-Consistent Adversarial Networks},
  author    = {Zhu, Jun-Yan and Park, Taesung and Isola, Phillip and Efros, Alexei A},
  booktitle = {Computer Vision (ICCV), 2017 IEEE International Conference on},
  year      = {2017}
}

@inproceedings{hoffman2018cycada,
  title        = {Cycada: Cycle-consistent adversarial domain adaptation},
  author       = {Hoffman, Judy and Tzeng, Eric and Park, Taesung and Zhu, Jun-Yan and Isola, Phillip and Saenko, Kate and Efros, Alexei and Darrell, Trevor},
  booktitle    = {International conference on machine learning},
  pages        = {1989--1998},
  year         = {2018},
  organization = {Pmlr}
}

@inproceedings{sohl2015deep,
  title        = {Deep unsupervised learning using nonequilibrium thermodynamics},
  author       = {Sohl-Dickstein, Jascha and Weiss, Eric and Maheswaranathan, Niru and Ganguli, Surya},
  booktitle    = {International Conference on Machine Learning},
  pages        = {2256--2265},
  year         = {2015},
  organization = {PMLR}
}

@article{ho2020denoising,
  title   = {Denoising diffusion probabilistic models},
  author  = {Ho, Jonathan and Jain, Ajay and Abbeel, Pieter},
  journal = {Advances in Neural Information Processing Systems},
  volume  = {33},
  pages   = {6840--6851},
  year    = {2020}
}

@article{zhang2023diffusion,
  title   = {Diffusion-based Target Sampler for Unsupervised Domain Adaptation},
  author  = {Zhang, Yulong and Chen, Shuhao and Zhang, Yu and Lu, Jiangang},
  journal = {arXiv preprint arXiv:2303.12724},
  year    = {2023}
}

@inproceedings{ghifary2016deep,
  title        = {Deep reconstruction-classification networks for unsupervised domain adaptation},
  author       = {Ghifary, Muhammad and Kleijn, W Bastiaan and Zhang, Mengjie and Balduzzi, David and Li, Wen},
  booktitle    = {Computer Vision--ECCV 2016: 14th European Conference, Amsterdam, The Netherlands, October 11--14, 2016, Proceedings, Part IV 14},
  pages        = {597--613},
  year         = {2016},
  organization = {Springer}
}

@article{bousmalis2016domain,
  title   = {Domain separation networks},
  author  = {Bousmalis, Konstantinos and Trigeorgis, George and Silberman, Nathan and Krishnan, Dilip and Erhan, Dumitru},
  journal = {Advances in neural information processing systems},
  volume  = {29},
  year    = {2016}
}

@inproceedings{kim2017learning,
  title        = {Learning to discover cross-domain relations with generative adversarial networks},
  author       = {Kim, Taeksoo and Cha, Moonsu and Kim, Hyunsoo and Lee, Jung Kwon and Kim, Jiwon},
  booktitle    = {International conference on machine learning},
  pages        = {1857--1865},
  year         = {2017},
  organization = {PMLR}
}

@article{triantafillou2019meta,
  title   = {Meta-dataset: A dataset of datasets for learning to learn from few examples},
  author  = {Triantafillou, Eleni and Zhu, Tyler and Dumoulin, Vincent and Lamblin, Pascal and Evci, Utku and Xu, Kelvin and Goroshin, Ross and Gelada, Carles and Swersky, Kevin and Manzagol, Pierre-Antoine and others},
  journal = {arXiv preprint arXiv:1903.03096},
  year    = {2019}
}

@inproceedings{fu2021meta,
  title     = {Meta-fdmixup: Cross-domain few-shot learning guided by labeled target data},
  author    = {Fu, Yuqian and Fu, Yanwei and Jiang, Yu-Gang},
  booktitle = {Proceedings of the 29th ACM International Conference on Multimedia},
  pages     = {5326--5334},
  year      = {2021}
}

@article{zhang2017mixup,
  title   = {mixup: Beyond empirical risk minimization},
  author  = {Zhang, Hongyi and Cisse, Moustapha and Dauphin, Yann N and Lopez-Paz, David},
  journal = {arXiv preprint arXiv:1710.09412},
  year    = {2017}
}

@inproceedings{zhao2021domain,
  title     = {Domain-adaptive few-shot learning},
  author    = {Zhao, An and Ding, Mingyu and Lu, Zhiwu and Xiang, Tao and Niu, Yulei and Guan, Jiechao and Wen, Ji-Rong},
  booktitle = {Proceedings of the IEEE/CVF Winter Conference on Applications of Computer Vision},
  pages     = {1390--1399},
  year      = {2021}
}

@inproceedings{chen2022cross,
  title        = {Cross-Domain Cross-Set Few-Shot Learning via Learning Compact and Aligned Representations},
  author       = {Chen, Wentao and Zhang, Zhang and Wang, Wei and Wang, Liang and Wang, Zilei and Tan, Tieniu},
  booktitle    = {Computer Vision--ECCV 2022: 17th European Conference, Tel Aviv, Israel, October 23--27, 2022, Proceedings, Part XX},
  pages        = {383--399},
  year         = {2022},
  organization = {Springer}
}

@inproceedings{zhuo2022tgdm,
  title     = {TGDM: Target Guided Dynamic Mixup for Cross-Domain Few-Shot Learning},
  author    = {Zhuo, Linhai and Fu, Yuqian and Chen, Jingjing and Cao, Yixin and Jiang, Yu-Gang},
  booktitle = {Proceedings of the 30th ACM International Conference on Multimedia},
  pages     = {6368--6376},
  year      = {2022}
}

@article{ganin2016domain,
  title     = {Domain-adversarial training of neural networks},
  author    = {Ganin, Yaroslav and Ustinova, Evgeniya and Ajakan, Hana and Germain, Pascal and Larochelle, Hugo and Laviolette, Fran{\c{c}}ois and Marchand, Mario and Lempitsky, Victor},
  journal   = {The journal of machine learning research},
  volume    = {17},
  number    = {1},
  pages     = {2096--2030},
  year      = {2016},
  publisher = {JMLR. org}
}

@inproceedings{fu2022me,
  title     = {ME-D2N: Multi-Expert Domain Decompositional Network for Cross-Domain Few-Shot Learning},
  author    = {Fu, Yuqian and Xie, Yu and Fu, Yanwei and Chen, Jingjing and Jiang, Yu-Gang},
  booktitle = {Proceedings of the 30th ACM International Conference on Multimedia},
  pages     = {6609--6617},
  year      = {2022}
}

@inproceedings{li2021universal,
  title     = {Universal representation learning from multiple domains for few-shot classification},
  author    = {Li, Wei-Hong and Liu, Xialei and Bilen, Hakan},
  booktitle = {Proceedings of the IEEE/CVF International Conference on Computer Vision},
  pages     = {9526--9535},
  year      = {2021}
}

@inproceedings{oh2022refine,
  title     = {ReFine: Re-randomization before Fine-tuning for Cross-domain Few-shot Learning},
  author    = {Oh, Jaehoon and Kim, Sungnyun and Ho, Namgyu and Kim, Jin-Hwa and Song, Hwanjun and Yun, Se-Young},
  booktitle = {Proceedings of the 31st ACM International Conference on Information \& Knowledge Management},
  pages     = {4359--4363},
  year      = {2022}
}

@article{yao2021cross,
  title   = {Cross-domain few-shot learning with unlabelled data},
  author  = {Yao, Fupin},
  journal = {arXiv preprint arXiv:2101.07899},
  year    = {2021}
}

@article{zhang2022few,
  title   = {Few-Shot Adaptation of Pre-Trained Networks for Domain Shift},
  author  = {Zhang, Wenyu and Shen, Li and Zhang, Wanyue and Foo, Chuan-Sheng},
  journal = {arXiv preprint arXiv:2205.15234},
  year    = {2022}
}

@inproceedings{hu2022adversarial,
  title        = {Adversarial Feature Augmentation for Cross-domain Few-Shot Classification},
  author       = {Hu, Yanxu and Ma, Andy J},
  booktitle    = {Computer Vision--ECCV 2022: 17th European Conference, Tel Aviv, Israel, October 23--27, 2022, Proceedings, Part XX},
  pages        = {20--37},
  year         = {2022},
  organization = {Springer}
}

@inproceedings{perez2020incremental,
  title     = {Incremental few-shot object detection},
  author    = {Perez-Rua, Juan-Manuel and Zhu, Xiatian and Hospedales, Timothy M and Xiang, Tao},
  booktitle = {Proceedings of the IEEE/CVF Conference on Computer Vision and Pattern Recognition},
  pages     = {13846--13855},
  year      = {2020}
}

@article{chen2022few,
  title     = {Few Shot Object Detection for SAR Images via Feature Enhancement and Dynamic Relationship Modeling},
  author    = {Chen, Shiqi and Zhang, Jun and Zhan, Ronghui and Zhu, Rongqiang and Wang, Wei},
  journal   = {Remote Sensing},
  volume    = {14},
  number    = {15},
  pages     = {3669},
  year      = {2022},
  publisher = {MDPI}
}

@article{ouyang2022few,
  title     = {Few-shot object detection based on positive-sample improvement},
  author    = {Ouyang, Yan and Wang, Xin-qing and Hu, Rui-zhe and Xu, Hong-hui},
  journal   = {Defence Technology},
  year      = {2022},
  publisher = {Elsevier}
}

@inproceedings{zhao2022semantic,
  title     = {Semantic-aligned fusion transformer for one-shot object detection},
  author    = {Zhao, Yizhou and Guo, Xun and Lu, Yan},
  booktitle = {Proceedings of the IEEE/CVF Conference on Computer Vision and Pattern Recognition},
  pages     = {7601--7611},
  year      = {2022}
}

@article{liu2023transformation,
  title   = {Transformation-Invariant Network for Few-Shot Object Detection in Remote Sensing Images},
  author  = {Liu, Nanqing and Xu, Xun and Celik, Turgay and Gan, Zongxin and Li, Heng-Chao},
  journal = {arXiv preprint arXiv:2303.06817},
  year    = {2023}
}

@inproceedings{yin2022sylph,
  title     = {Sylph: A hypernetwork framework for incremental few-shot object detection},
  author    = {Yin, Li and Perez-Rua, Juan M and Liang, Kevin J},
  booktitle = {Proceedings of the IEEE/CVF Conference on Computer Vision and Pattern Recognition},
  pages     = {9035--9045},
  year      = {2022}
}

@inproceedings{lee2022few,
  title     = {Few-shot object detection by attending to per-sample-prototype},
  author    = {Lee, Hojun and Lee, Myunggi and Kwak, Nojun},
  booktitle = {Proceedings of the IEEE/CVF Winter Conference on Applications of Computer Vision},
  pages     = {2445--2454},
  year      = {2022}
}

@inproceedings{zhang2022kernelized,
  title     = {Kernelized few-shot object detection with efficient integral aggregation},
  author    = {Zhang, Shan and Wang, Lei and Murray, Naila and Koniusz, Piotr},
  booktitle = {Proceedings of the IEEE/CVF Conference on Computer Vision and Pattern Recognition},
  pages     = {19207--19216},
  year      = {2022}
}

@inproceedings{jiang2023few,
  title     = {Few-Shot Object Detection via Improved Classification Features},
  author    = {Jiang, Xinyu and Li, Zhengjia and Tian, Maoqing and Liu, Jianbo and Yi, Shuai and Miao, Duoqian},
  booktitle = {Proceedings of the IEEE/CVF Winter Conference on Applications of Computer Vision},
  pages     = {5386--5395},
  year      = {2023}
}

@inproceedings{zhao2021morphable,
  title     = {Morphable detector for object detection on demand},
  author    = {Zhao, Xiangyun and Zou, Xu and Wu, Ying},
  booktitle = {Proceedings of the IEEE/CVF International Conference on Computer Vision},
  pages     = {4771--4780},
  year      = {2021}
}

@article{lu2022decoupled,
  title     = {Decoupled metric network for single-stage few-shot object detection},
  author    = {Lu, Yue and Chen, Xingyu and Wu, Zhengxing and Yu, Junzhi},
  journal   = {IEEE Transactions on Cybernetics},
  volume    = {53},
  number    = {1},
  pages     = {514--525},
  year      = {2022},
  publisher = {IEEE}
}

@inproceedings{kaul2022label,
  title     = {Label, verify, correct: A simple few shot object detection method},
  author    = {Kaul, Prannay and Xie, Weidi and Zisserman, Andrew},
  booktitle = {Proceedings of the IEEE/CVF conference on computer vision and pattern recognition},
  pages     = {14237--14247},
  year      = {2022}
}

@inproceedings{li2021few,
  title     = {Few-shot object detection via classification refinement and distractor retreatment},
  author    = {Li, Yiting and Zhu, Haiyue and Cheng, Yu and Wang, Wenxin and Teo, Chek Sing and Xiang, Cheng and Vadakkepat, Prahlad and Lee, Tong Heng},
  booktitle = {Proceedings of the IEEE/CVF Conference on Computer Vision and Pattern Recognition},
  pages     = {15395--15403},
  year      = {2021}
}

@article{wang2022context,
  title     = {Context Information Refinement for Few-Shot Object Detection in Remote Sensing Images},
  author    = {Wang, Yan and Xu, Chaofei and Liu, Cuiwei and Li, Zhaokui},
  journal   = {Remote Sensing},
  volume    = {14},
  number    = {14},
  pages     = {3255},
  year      = {2022},
  publisher = {MDPI}
}

@article{zhou2022few,
  title     = {Few-shot object detection via context-aware aggregation for remote sensing images},
  author    = {Zhou, Yong and Hu, Han and Zhao, Jiaqi and Zhu, Hancheng and Yao, Rui and Du, Wen-Liang},
  journal   = {IEEE Geoscience and Remote Sensing Letters},
  volume    = {19},
  pages     = {1--5},
  year      = {2022},
  publisher = {IEEE}
}

@article{lin2023effective,
  title   = {An Effective Crop-Paste Pipeline for Few-shot Object Detection},
  author  = {Lin, Shaobo and Wang, Kun and Zeng, Xingyu and Zhao, Rui},
  journal = {arXiv preprint arXiv:2302.14452},
  year    = {2023}
}

@article{dong2022incremental,
  title   = {Incremental-detr: Incremental few-shot object detection via self-supervised learning},
  author  = {Dong, Na and Zhang, Yongqiang and Ding, Mingli and Lee, Gim Hee},
  journal = {arXiv preprint arXiv:2205.04042},
  year    = {2022}
}

@inproceedings{liu2022novel,
  title        = {Novel instance mining with pseudo-margin evaluation for few-shot object detection},
  author       = {Liu, Weijie and Wang, Chong and Yu, Shenghao and Tao, Chenchen and Wang, Jun and Wu, Jiafei},
  booktitle    = {ICASSP 2022-2022 IEEE International Conference on Acoustics, Speech and Signal Processing (ICASSP)},
  pages        = {2250--2254},
  year         = {2022},
  organization = {IEEE}
}

@inproceedings{she2022fast,
  title        = {Fast Hierarchical Learning for Few-Shot Object Detection},
  author       = {She, Yihang and Bhat, Goutam and Danelljan, Martin and Yu, Fisher},
  booktitle    = {2022 IEEE/RSJ International Conference on Intelligent Robots and Systems (IROS)},
  pages        = {1993--2000},
  year         = {2022},
  organization = {IEEE}
}

@inproceedings{wu2022multi,
  title        = {Multi-faceted Distillation of Base-Novel Commonality for Few-Shot Object Detection},
  author       = {Wu, Shuang and Pei, Wenjie and Mei, Dianwen and Chen, Fanglin and Tian, Jiandong and Lu, Guangming},
  booktitle    = {Computer Vision--ECCV 2022: 17th European Conference, Tel Aviv, Israel, October 23--27, 2022, Proceedings, Part IX},
  pages        = {578--594},
  year         = {2022},
  organization = {Springer}
}

@inproceedings{guirguis2022cfa,
  title     = {CFA: Constraint-based Finetuning Approach for Generalized Few-Shot Object Detection},
  author    = {Guirguis, Karim and Hendawy, Ahmed and Eskandar, George and Abdelsamad, Mohamed and Kayser, Matthias and Beyerer, J{\"u}rgen},
  booktitle = {Proceedings of the IEEE/CVF Conference on Computer Vision and Pattern Recognition},
  pages     = {4039--4049},
  year      = {2022}
}

@article{kohler2021few,
  title   = {Few-Shot Object Detection: A Comprehensive Survey},
  author  = {Kohler, Mona and Eisenbach, Markus and Gross, Horst-Michael},
  journal = {IEEE transactions on neural networks and learning systems},
  year    = {2021}
}

@article{liu2022few,
  title     = {Few-Shot Object Detection in Remote Sensing Image Interpretation: Opportunities and Challenges},
  author    = {Liu, Sixu and You, Yanan and Su, Haozheng and Meng, Gang and Yang, Wei and Liu, Fang},
  journal   = {Remote Sensing},
  volume    = {14},
  number    = {18},
  pages     = {4435},
  year      = {2022},
  publisher = {MDPI}
}

@inproceedings{zhao2022oa,
  title     = {OA-FSUI2IT: A Novel Few-Shot Cross Domain Object Detection Framework with Object-Aware Few-Shot Unsupervised Image-to-Image Translation},
  author    = {Zhao, Lifan and Meng, Yunlong and Xu, Lin},
  booktitle = {Proceedings of the AAAI Conference on Artificial Intelligence},
  volume    = {36},
  number    = {3},
  pages     = {3426--3435},
  year      = {2022}
}

@inproceedings{gao2022acrofod,
  title        = {AcroFOD: An Adaptive Method for Cross-Domain Few-Shot Object Detection},
  author       = {Gao, Yipeng and Yang, Lingxiao and Huang, Yunmu and Xie, Song and Li, Shiyong and Zheng, Wei-Shi},
  booktitle    = {Computer Vision--ECCV 2022: 17th European Conference, Tel Aviv, Israel, October 23--27, 2022, Proceedings, Part XXXIII},
  pages        = {673--690},
  year         = {2022},
  organization = {Springer}
}

@inproceedings{nakamura2022few,
  title     = {Few-shot Adaptive Object Detection with Cross-Domain CutMix},
  author    = {Nakamura, Yuzuru and Ishii, Yasunori and Maruyama, Yuki and Yamashita, Takayoshi},
  booktitle = {Proceedings of the Asian Conference on Computer Vision},
  pages     = {1350--1367},
  year      = {2022}
}

@article{xiong2022cd,
  title   = {CD-FSOD: A Benchmark for Cross-domain Few-shot Object Detection},
  author  = {Xiong, Wuti and Liu, Li},
  journal = {arXiv preprint arXiv:2210.05311},
  year    = {2022}
}

@inproceedings{lee2022rethinking,
  title        = {Rethinking Few-Shot Object Detection on a Multi-Domain Benchmark},
  author       = {Lee, Kibok and Yang, Hao and Chakraborty, Satyaki and Cai, Zhaowei and Swaminathan, Gurumurthy and Ravichandran, Avinash and Dabeer, Onkar},
  booktitle    = {Computer Vision--ECCV 2022: 17th European Conference, Tel Aviv, Israel, October 23--27, 2022, Proceedings, Part XX},
  pages        = {366--382},
  year         = {2022},
  organization = {Springer}
}

@inproceedings{rahman2019transductive,
  title     = {Transductive learning for zero-shot object detection},
  author    = {Rahman, Shafin and Khan, Salman and Barnes, Nick},
  booktitle = {Proceedings of the IEEE/CVF International Conference on Computer Vision},
  pages     = {6082--6091},
  year      = {2019}
}

@article{zhou2022fsods,
  title     = {FSODS: A Lightweight Metalearning Method for Few-Shot Object Detection on SAR Images},
  author    = {Zhou, Zheng and Chen, Jie and Huang, Zhixiang and Wan, Huiyao and Chang, Pei and Li, Zhao and Yao, Baidong and Wu, BoCai and Sun, Long and Xing, Mengdao},
  journal   = {IEEE Transactions on Geoscience and Remote Sensing},
  volume    = {60},
  pages     = {1--17},
  year      = {2022},
  publisher = {IEEE}
}

@inproceedings{wang2019meta,
  title     = {Meta-learning to detect rare objects},
  author    = {Wang, Yu-Xiong and Ramanan, Deva and Hebert, Martial},
  booktitle = {Proceedings of the IEEE/CVF International Conference on Computer Vision},
  pages     = {9925--9934},
  year      = {2019}
}

@inproceedings{han2022meta,
  title     = {Meta faster r-cnn: Towards accurate few-shot object detection with attentive feature alignment},
  author    = {Han, Guangxing and Huang, Shiyuan and Ma, Jiawei and He, Yicheng and Chang, Shih-Fu},
  booktitle = {Proceedings of the AAAI Conference on Artificial Intelligence},
  volume    = {36},
  number    = {1},
  pages     = {780--789},
  year      = {2022}
}

@inproceedings{bansal2018zero,
  title     = {Zero-shot object detection},
  author    = {Bansal, Ankan and Sikka, Karan and Sharma, Gaurav and Chellappa, Rama and Divakaran, Ajay},
  booktitle = {Proceedings of the European Conference on Computer Vision (ECCV)},
  pages     = {384--400},
  year      = {2018}
}

@article{han2022multimodal,
  title   = {Multimodal few-shot object detection with meta-learning based cross-modal prompting},
  author  = {Han, Guangxing and Ma, Jiawei and Huang, Shiyuan and Chen, Long and Chellappa, Rama and Chang, Shih-Fu},
  journal = {arXiv preprint arXiv:2204.07841},
  year    = {2022}
}

@inproceedings{gao2022decoupling,
  title     = {Decoupling Classifier for Boosting Few-shot Object Detection and Instance Segmentation},
  author    = {Gao, Bin-Bin and Chen, Xiaochen and Huang, Zhongyi and Nie, Congchong and Liu, Jun and Lai, Jinxiang and Jiang, Guannan and Wang, Xi and Wang, Chengjie},
  booktitle = {Advances in Neural Information Processing Systems},
  year      = {2022}
}

@inproceedings{qiao2021defrcn,
  title     = {Defrcn: Decoupled faster r-cnn for few-shot object detection},
  author    = {Qiao, Limeng and Zhao, Yuxuan and Li, Zhiyuan and Qiu, Xi and Wu, Jianan and Zhang, Chi},
  booktitle = {Proceedings of the IEEE/CVF International Conference on Computer Vision},
  pages     = {8681--8690},
  year      = {2021}
}

@inproceedings{wang2019few,
  title     = {Few-shot adaptive faster r-cnn},
  author    = {Wang, Tao and Zhang, Xiaopeng and Yuan, Li and Feng, Jiashi},
  booktitle = {Proceedings of the IEEE/CVF Conference on Computer Vision and Pattern Recognition},
  pages     = {7173--7182},
  year      = {2019}
}

@misc{artaxor2020,
	title={Arthropod Taxonomy Orders Object Detection Dataset},
	url={https://www.kaggle.com/dsv/1240192},
	DOI={10.34740/KAGGLE/DSV/1240192},
	publisher={Kaggle},
	author={Geir Drange},
	year={2020}
}

@inproceedings{jiang2021underwater,
  title     = {Underwater species detection using channel sharpening attention},
  author    = {Jiang, Lihao and Wang, Yi and Jia, Qi and Xu, Shengwei and Liu, Yu and Fan, Xin and Li, Haojie and Liu, Risheng and Xue, Xinwei and Wang, Ruili},
  booktitle = {Proceedings of the 29th ACM International Conference on Multimedia},
  pages     = {4259--4267},
  year      = {2021}
}

@article{kingma2014adam,
  title   = {Adam: A method for stochastic optimization},
  author  = {Kingma, Diederik P and Ba, Jimmy},
  journal = {arXiv preprint arXiv:1412.6980},
  year    = {2014}
}

@article{deng2012mnist,
  title     = {The mnist database of handwritten digit images for machine learning research},
  author    = {Deng, Li},
  journal   = {IEEE Signal Processing Magazine},
  volume    = {29},
  number    = {6},
  pages     = {141--142},
  year      = {2012},
  publisher = {IEEE}
}

\cleardoublepage
\appendix

\renewcommand{\chaptermark}[1]{\markboth{Appendix~\thechapter~-~ #1}{}}
\phantomsection
\addcontentsline{toc}{part}{Appendices}
\chapter{Proofs of SIoU's Properties}
\label{app:properties}

In this appendix, we provide the proofs for
\cref{property:relaxation,property:l_g_reweight}, and discuss the
\textit{order-preservigness} of SIoU.

\begin{appproperty}[SIoU Relaxation]
    \label{app_property:app_relaxation}
    Let $b_1$ and $b_2$ be two bounding boxes and introduce $\tau = \frac{w_1h_1 +
    w_2h_2}{2}$ their average area. SIoU preserves the behavior of IoU in
    certain cases such as:
   \begin{itemize}
    \setlength\itemsep{0em}
        \item[-]$\textup{IoU}(b_1, b_2) = 0 \Rightarrow \textup{SIoU}(b_1, b_2) =\textup{IoU}(b_1, b_2) = 0$
        \item[-]$\textup{IoU}(b_1, b_2) = 1 \Rightarrow\textup{SIoU}(b_1, b_2) =\textup{IoU}(b_1, b_2) = 1$
        \item[-]$\lim\limits_{\tau \to +\infty} \textup{SIoU}(b_1, b_2) =\textup{IoU}(b_1,b_2)$
        \item[-]$\lim\limits_{\kappa\to 0} \textup{SIoU}(b_1, b_2) =\textup{IoU}(b_1,b_2)$
   \end{itemize}
\end{appproperty}

\begin{proof}
    First let recall the expression of SIoU, $\text{SIoU}(b_1, b_2) =
    \text{IoU}(b_1, b_2) ^ p$ with $p = 1 - \gamma
    \exp\left({-\frac{\sqrt{\tau}}{\kappa}}\right)$. $\tau >0$ because boxes
    cannot be empty and as $\gamma \in ]-\infty, 1]$ and $\kappa \in
    \mathbb{R}_+^*$, we have $p>0$.\\

    \vspace{5mm}
    \noindent
    From this, the first two items of \cref*{property:relaxation} follow clearly.\\ 

    \noindent
    The two other points hold because the function $f \colon x \mapsto
    \text{IoU}(b_1, b_2)^x $ is continuous on $\mathbb{R}$ for any couple of
    boxes $b_1$ and $b_2$ ($\text{IoU}(b_1, b_2) \in [0,1]$) and  
    $\lim\limits_{\tau \to \infty} p = \lim\limits_{\kappa \to 0} p = 1$.

\end{proof}

\pagebreak
\begin{appproperty}[Loss and gradients reweighting]
    \label{property:app_l_g_reweight}
    Let $\mathcal{L}_{\textup{IoU}}(b_1, b_2) = 1 - \textup{IoU}(b_1, b_2)$ and
    $\mathcal{L}_{\textup{SIoU}}(b_1, b_2) = 1 - \textup{SIoU}(b_1, b_2)$ be the
    loss functions associated respectively with IoU and SIoU. Let denote the
    ratio between SIoU and IoU losses by 
    $\mathcal{W}_{\mathcal{L}}(b_1, b_2) =
    \frac{\mathcal{L}_{\textup{SIoU}}(b_1, b_2)}{\mathcal{L}_{\textup{IoU}}(b_1,
    b_2)}$. 
    Similarly, 
    $\mathcal{W}_{\mathcal{\nabla}}(b_1, b_2) =
    \frac{|\nabla\mathcal{L}_{\textup{SIoU}}(b_1, b_2)|}{|\nabla\mathcal{L}_{\textup{IoU}}(b_1,
    b_2)|}$ 
    denotes the ratio of gradients generated from SIoU and IoU losses: 
    \begin{align}
        \mathcal{W}_{\mathcal{L}}(b_1, b_2) &= \frac{1- \textup{IoU}(b_1, b_2)^p}{1-\textup{IoU}(b_1, b_2)}, \\
        \mathcal{W}_{\mathcal{\nabla}}(b_1, b_2) &= p\textup{IoU}(b_1, b_2)^{p-1},
    \end{align}

    \noindent
    $\mathcal{W}_{\mathcal{L}}$ and $\mathcal{W}_{\mathcal{\nabla}}$ are
    increasing (resp. decreasing) functions of IoU when $p\geq 1$ (resp. $p <
    1$) which is satisfied when $\gamma \leq 0$ (resp. $\gamma > 0$). As the IoU
    goes to 1, $\mathcal{W}_{\mathcal{L}}$ and $\mathcal{W}_{\mathcal{\nabla}}$
    approaches $p$: 
    \begin{align}
        \lim\limits_{\textup{IoU}(b_1, b_2) \to 1}\mathcal{W}_{\mathcal{L}}(b_1, b_2) &=  p, \\
        \lim\limits_{\textup{IoU}(b_1, b_2) \to 1}\mathcal{W}_{\mathcal{\nabla}}(b_1, b_2) &=  p.
    \end{align}

\end{appproperty}

\begin{proof}
    Let denote the $\text{IoU}(b_1, b_2)$ by $\mu$, and define two functions $ f
    \colon \mu \mapsto 1- \mu = \mathcal{L}_{\text{IoU}}(b_1, b_2)$ and $g
    \colon \mu \mapsto 1- \mu^p = \mathcal{L}_{\text{SIoU}}(b_1, b_2)$. \\
    
    \noindent
    $f$ and $g$ are differentiable on $[0,1]$ and $\lim\limits_{\mu \to 1}
    f(\mu) = \lim\limits_{\mu \to 1} g(\mu) = 0$. This holds because $p$ is
    independent of the IoU (\ie $\mu$).
    
    \noindent
    Therefore L'Hôpital's rule can be applied:\\
    $\lim\limits_{\mu \to 1}\mathcal{W}_{\mathcal{L}} = \lim\limits_{\mu \to
    1}\frac{g(\mu)}{f(\mu)} = \lim\limits_{\mu \to 1}\frac{g'(\mu)}{f'(\mu)} =
    \lim\limits_{\mu \to 1} p\mu^{p-1} = p$.
    
    \noindent
    The expression of the second ratio $\mathcal{W}_{\mathcal{\nabla}}(b_1,
    b_2)$ follows directly as $|\nabla\mathcal{L}_{\textup{SIoU}}(b_1, b_2)| =
    g'(\mu)$ and $|\nabla\mathcal{L}_{\textup{IoU}}(b_1, b_2)| = f'(\mu)$, hence
    $\lim\limits_{\mu \to 1}\mathcal{W}_{\nabla} = \lim\limits_{\mu \to 1}
    p\mu^{p-1} = p$.
\end{proof}

\noindent
\textbf{Order-preservigness} \\
Let us take three boxes $b_1$, $b_2$, and $b_3$. Order-preservigness does not
hold for SIoU. Therefore $\text{IoU}(b_1, b_2) \leq  \text{IoU}(b_1, b_3)$ does
not imply $\text{SIoU}(b_1, b_2) \leq  \text{SIoU}(b_1, b_3)$. However, this
property is often true in practice. Denoting by $\tau_{i,j}$ the average area
between boxes $i$ and $j$, we can study the conditions for the order to hold. We
will also note $p_{i,j} = 1- \gamma
\exp\left({-\frac{\sqrt{\tau_{i,j}}}{\kappa}}\right)$\\

\noindent
Let's suppose, without loss of generality, that $\tau_{1,2} \leq \tau_{1,3}$.
Otherwise, cases 1 and 2 would be swapped. \\

\noindent
\textbf{Case 1: } $\gamma \leq 0$\\
We have $p_{1,2} > p_{1,3}$ as $\tau_{1,2} \leq \tau_{1,3}$ and $\gamma \leq 0$.
Therefore, $p_{1,2} = p_{1,3} + \varepsilon$, with $\varepsilon > 0$. \\

\noindent
Then,
\begin{equation*}
    \begin{aligned}
        \text{IoU}(b_1,b_2)^{p_{1,2}} &= \text{IoU}(b_1,b_2)^{p_{1,3} +\varepsilon}\\
                                        & = \text{IoU}(b_1,b_2)^{p_{1,3}}\text{IoU}(b_1,b_2)^{\varepsilon}\\
                                        &\leq \text{IoU}(b_1,b_2)^{p_{1,3}} \\
                                        &\leq \text{IoU}(b_1,b_3)^{p_{1,3}}
    \end{aligned}
\end{equation*}

\noindent
Line 3 holds because $0 < \text{IoU}(b_1,b_2)^{\varepsilon} \leq 1$. Line 4 is
true because $\text{IoU}(b_1, b_2) \leq  \text{IoU}(b_1, b_3)$ and the function
$h\colon x \mapsto x^{p_{1,3}}$ is monotonically increasing. Hence, when $\gamma
\leq 0$ the order is preserved.\\

\noindent
\textbf{Case 2: } $\gamma > 0$\\
We have $p_{1,3} > p_{1,2}$ as $\tau_{1,3} \leq \tau_{1,2}$ and $\gamma > 0$.
Therefore, $p_{1,3} = p_{1,2} + \varepsilon$, with $\varepsilon > 0$. \\

\noindent
In this case, the order does not always hold, counter-examples can be found.
However, it is useful to study the conditions for it to hold:
$$\text{IoU}(b_1, b_2)^{p_{1,2}} \leq  \text{IoU}(b_1, b_3)^{p_{1,3}} \Leftrightarrow \frac{\ln
\big(\text{IoU}(b_1, b_2)\big)}{\ln\big(\text{IoU}(b_1, b_3)\big)} \geq \frac{p_{1,3}}{p_{1,2}}$$

\noindent
In practice, the right condition is often true as $p_{1,2}$ and $p_{1,3}$ are
close due to scale matching, a trick present in most detection frameworks to
prevent comparison of proposals and ground truth of very different sizes. In
addition, the ratio of log values gets large rapidly, even if the gap between
$\text{IoU}(b_1, b_2)$ and $\text{IoU}(b_1, b_3)$ is small, the ratio $\frac{\ln
\big(\text{IoU}(b_1, b_2)\big)}{\ln\big(\text{IoU}(b_1, b_3)\big)}$ can be
large.\\

\noindent
The order-preservigness property holds in many cases. This is sensible as
IoU is still a reliable metric that has been used extensively for the training
and evaluation of detection models. However, in the rare cases where this order
is broken, the $\text{IoU}(b_1, b_2)$ and $\text{IoU}(b_1, b_3)$ are close, so
IoU does not discriminate much between the boxes. On the contrary, SIoU prefers
the smallest one (or largest one, according to the choice of $\gamma$). This
stronger discrimination is probably beneficial for training and evaluation.

\chapter{Theoretical GIoU's Distribution Analysis}
\label{app:theoretical_giou_proof}
\vspace{-2em}
In this appendix, we give a proof of the formulas of the probability density
function, expected value, and variance of GIoU. We also provide some
non-closed-form expressions for other criteria.

\begin{appproposition}[GIoU's distribution]
    \label{prop:appdistribution_giou}
    Let $b_1 = (0,y_1,w_1,h_1)$ be a bounding box horizontally centered and $b_2
    = (X,y_2,w_2,h_2)$ another bounding box randomly positioned, with $X\sim
    \mathcal{N}(0, \sigma^2)$ and $\sigma \in \mathbb{R}^*_+$. Let's suppose that
    the boxes are identical squares, shifted only horizontally (\ie $w_1 = w_2 = h_1 =
    h_2$ and $y_1 = y_2$). \\
    \noindent
    Let $Z = \mathfrak{C}(X)$, where $\mathfrak{C}$ is the generalized intersection over union. The
    probability density function of $Z$ is given by:
    \begin{align}
        d_Z(z) &= \frac{4\omega}{(1+z)^2\sqrt{2\pi} \sigma} \exp\left(-\frac{1}{2}\left[\frac{\omega(1-z)}{\sigma(1+z)}\right]^2\right).
    \end{align}

    \noindent
    The two first moments of $Z$ exist and are given by:
    \begin{align}
        \mathbb{E}[Z] &= \frac{2}{\pi^{3/2}} G^{2,3}_{3,2}\left(2a^2 \bigg\rvert \begin{matrix}0 & \frac{1}{2} & \frac{1}{2} \\ \frac{1}{2}& 0 \end{matrix}\right), \\
        \mathbb{E}[Z^2] &= 1- \frac{8a}{\sqrt{2\pi}} + \frac{16a^2}{\pi^{3/2}} G^{2,3}_{3,2}\left(2a^2 \bigg\rvert \begin{matrix}-1 & \frac{1}{2} & -\frac{1}{2} \\ \frac{1}{2}  & 0 \end{matrix}\right),
    \end{align}

    \noindent
    where $G$ is the Meijer G-function \cite{meijer2013} and $a=\frac{\sigma}{\omega}$. 
\end{appproposition}

\begin{proof}
    First, in the setup described in \cref{prop:distribution_giou}, GIoU can be
    expressed in terms of the width of the boxes $\omega$ and the shift in
    between $x$. Let's call this function $\mathfrak{C}$: 
    \begin{align*}
        \mathfrak{C} \colon \mathbb{R} &\to [-1,1] \\
                 x &\mapsto \frac{\omega- |x|}{\omega + |x|}
    \end{align*}

    \noindent
    The shifts are sampled from a Gaussian distribution: $X \sim \mathcal{N}(0,
    \sigma^2)$, therefore we are interested in the distribution of the variable
    $Z=\mathfrak{C}(X)$.

    \begingroup
    \allowdisplaybreaks
    The cumulative density function of $Z$ is given by:
    \begin{align*}
        F_Z(z) = P(Z \leq z) &= P( \frac{\omega- |X|}{\omega + |X|} \leq y)\\
                            & = P( \omega\frac{1- z}{1+z} \leq |X|) \\
                            & = 2 P( \omega\frac{1- z}{1+z} \leq X) \\
                            & = 2 (1 - F_X(\omega\frac{1- z}{1+z})) \\
                            & = 2 (1 - F_X(g(z)))
    \end{align*}
    With $g(z) = \omega(\frac{1- z}{1+z})$.
    Hence, the density function of $Z$ can be derived by taking the derivative of $F_Z$:

    \begin{align*}
        d_Z(z) &= \frac{d}{dz} F_Z(z) \\
                &= -2 g'(z) F_X'(g(y)) \\
                &= \frac{4\omega}{(1+z)^2\sqrt{2\pi} \sigma} \exp\left(-\frac{1}{2}\left[\frac{\omega(1-z)}{\sigma(1+z)}\right]^2\right)
    \end{align*}

    To determine the first and second moments of $Z$, we make use of the change of
    variable formula: 
    \begin{align*}
        \mathbb{E}[Z] = \mathbb{E}[\mathfrak{C}(X)] = \int_{-\infty}^{+\infty} \mathfrak{C}(x)d_X(x)\,dx \\
        \mathbb{E}[Z^2] = \mathbb{E}[\mathfrak{C}(X)^2] = \int_{-\infty}^{+\infty} \mathfrak{C}(x)^2d_X(x)\,dx
    \end{align*}

    Let's start with $\mathbb{E}[Z]$:
    \begin{align*}
        \mathbb{E}[Z]& = \int_{-\infty}^{+\infty} \mathfrak{C}(x)d_X(x)\,dx \\
                    &= \int_{-\infty}^{+\infty} \frac{1}{\sqrt{2\pi} \sigma}\frac{\omega- |x|}{\omega + |x| } e^{-\frac{x^2}{2\sigma^2}}\,dx \\
                    &= \frac{2}{\sqrt{2\pi} \sigma}\int_{0}^{+\infty}  \frac{\omega- x}{\omega + x}e^{-\frac{x^2}{2\sigma^2}}\,dx \\
                    &= \frac{2}{\sqrt{2\pi} \sigma}\int_{0}^{+\infty}  \frac{2\omega- (\omega + x)}{\omega + x}e^{-\frac{x^2}{2\sigma^2}}\,dx \\
                    &= \frac{2}{\sqrt{2\pi} \sigma}\left[2\int_{0}^{+\infty}  \frac{\omega}{\omega + x}e^{-\frac{x^2}{2\sigma^2}}\,dx - \int_{0}^{+\infty}  e^{-\frac{x^2}{2\sigma^2}}\,dx\right] \\
                    &= \frac{4}{\sqrt{2\pi} \sigma}\int_{0}^{+\infty}  \frac{\omega}{\omega + x}e^{-\frac{x^2}{2\sigma^2}}\,dx - 1 \\
                    &= \frac{4}{\sqrt{2\pi} a}\int_{0}^{+\infty}  \frac{1}{1 + u}e^{-\frac{u^2}{2a^2}}\,du - 1 \\
                    &= \frac{4}{\sqrt{2\pi} a} \frac{\sqrt{2}a}{2\pi} G^{2,3}_{3,2}\left(2a^2 \bigg\rvert \begin{matrix}0 & \frac{1}{2} & \frac{1}{2} \\ \frac{1}{2}& 0 \end{matrix}\right) - 1 \\
                    &= \frac{2}{\pi^{3/2}} G^{2,3}_{3,2}\left(2a^2 \bigg\rvert \begin{matrix}0 & \frac{1}{2} & \frac{1}{2} \\ \frac{1}{2}& 0 \end{matrix}\right) -1
        \vspace{-0.5em}
    \end{align*}

    \noindent
    From line 2 to 3, we used the parity of function $\mathfrak{C}$, between 6
    and 7, a change of variable $u=x/\omega$ is done and $a$ is set to $\sigma /
    \omega$. Finally, in the second-to-last line, we identify a Meijer-G
    function \cite{meijer2013} evaluated at $2a^2$.  Unfortunately, there exist
    no closed-form of the integral $ \int_{0}^{+\infty}  \frac{1}{1 +
    u}e^{-\frac{u^2}{2a^2}}\,dx$. In this case, a Mellin transform of this
    integral allows recognizing a Meijer-G function. For other criteria, their
    first two moments cannot be expressed in a similar closed form. That is why we only
    provide the theoretical expressions of the expectation and variance of GIoU.

    A similar derivation leads to the expression of the second moment of $Z$:
    \begin{align*}
        \mathbb{E}[Z^2]& = \int_{-\infty}^{+\infty} \mathfrak{C}(x)^2d_X(x)\,dx \\
                    &= \int_{-\infty}^{+\infty} \frac{1}{\sqrt{2\pi} \sigma}\big(\frac{\omega- |x|}{\omega + |x| }\big)^2 e^{-\frac{x^2}{2\sigma^2}}\,dx \\
                    &= \frac{2}{\sqrt{2\pi} \sigma}\int_{0}^{+\infty}  \big(\frac{\omega- x}{\omega + x}\big)^2e^{-\frac{x^2}{2\sigma^2}}\,dx \\
                    &= 1 - \frac{8\omega}{\sqrt{2\pi} \sigma}\int_{0}^{+\infty}  \frac{x}{(\omega + x)^2}e^{-\frac{x^2}{2\sigma^2}}\,dx \\
                    &= 1 - \frac{8}{\sqrt{2\pi} a}\left[ a^2 -2\sigma^2\int_{0}^{+\infty}  \frac{1}{(\omega + x)^3}e^{-\frac{x^2}{2\sigma^2}}\,dx\right] \\
                    &= 1 - \frac{8}{\sqrt{2\pi} a}\left[ a^2 -2a^2\int_{0}^{+\infty}  \frac{1}{(1 + u)^3}e^{-\frac{u^2}{2a^2}}\,du\right] \\
                    &=1- \frac{8a}{\sqrt{2\pi}} + \frac{16a^2}{\pi^{3/2}} G^{2,3}_{3,2}\left(2a^2 \bigg\rvert \begin{matrix}-1 & \frac{1}{2} & -\frac{1}{2} \\ \frac{1}{2}  & 0 \end{matrix}\right)   
        \vspace{-0.5em}
    \end{align*}
    
    From line 2 to 3, we again use the parity of $\mathfrak{C}$, from 4 to 5, an
    integration by parts is done, and finally, from 5 to 6, we apply the change of
    variable $u=x/\omega$. Once again, we get an integral that does not have any
    closed form but can be expressed by another Meijer-G function.
    \endgroup

    \setlength{\belowdisplayskip}{3pt}
    For completeness, we recall here the definition of the Meijer-G function: 
    \begin{align}
        \label{eq:meijer_g}
        G^{m,n}_{p,q}\left(z \bigg\rvert \begin{matrix}a_1 & ... & a_p \\ b_1 & ... & b_q \end{matrix}\right) =\frac{1}{2 \pi i} \int_{L}\frac{\prod\limits_{j=1}^m\Gamma(b_j - s)\prod\limits_{j=1}^n\Gamma(1 - a_j + s)}{\prod\limits_{j=m+1}^q\Gamma(1 - b_j + s)\prod\limits_{j=n+1}^p\Gamma(a_j - s)} z^s\, ds,
        \vspace{-1em}
    \end{align}
    
    where $L$ is the integration path and $\Gamma$ is the gamma function. $m$,
    $n$, $p$ and $q$ are integers while $a_j$ and $b_j$ are real or complex
    numbers. There are some constraints on these parameters, but we do not detail
    them here, they can be found in \cite{meijer2013}. 
\end{proof}

\newpage\null\thispagestyle{empty}\newpage
\end{document}